\documentclass[a4paper,11pt]{article}
\usepackage{graphicx} 
\usepackage[a4paper,left=2cm,right=2cm,top=3cm,bottom=3cm]{geometry}
\usepackage{natbib}
\usepackage{amsfonts}
\usepackage{amsmath}
\usepackage{amsthm}
\usepackage{import}
\usepackage{enumerate}
\usepackage[dvipsnames]{xcolor}
\usepackage{authblk, bm, bbm, soul}
\usepackage{amssymb}
\usepackage{enumitem}
\usepackage{multirow}
\usepackage{capt-of}
\usepackage{booktabs}
\usepackage{hyperref}
\usepackage{cleveref}
\usepackage{natbib}
\hypersetup{colorlinks=true, citecolor=blue, linkcolor = blue, urlcolor = blue}
\usepackage{stackrel}
\usepackage{algorithm}
\usepackage{multirow}
\usepackage{dsfont}
\usepackage{bbm}
\usepackage{algpseudocode}
\usepackage[toc,page]{appendix}

\newtheorem{theorem}{Theorem}[]
\newtheorem{definition}{Definition}[]

\newtheorem{assumption}{Assumption}
\newtheorem{proposition}[theorem]{Proposition}
\newtheorem{corollary}[theorem]{Corollary}
\newtheorem{lemma}[theorem]{Lemma}
\theoremstyle{remark}
\newtheorem{remark}{Remark}

\global\long\def\real{\mathbb{R}}

\global\long\def\prob{\mathbb{P}}
\global\long\def\mean {\mathbb{E}}
\global\long\def\var{\mathrm{var}}
\global\long\def\cov{\mathrm{cov}}
\global\long\def\indc{\mathds{1}}

\global\long\def\Xspace{\mathcal{X}}
\global\long\def\Aspace{\mathcal{A}}
\global\def\data{\mathcal D}

\DeclareMathOperator*{\argmin}{arg\,min}

\allowdisplaybreaks

\title{Fairness-aware Bayes optimal functional classification}
\author[1]{Xiaoyu Hu\footnote{Equal contribution}}
\author[2]{Gengyu Xue$^{\ast}$}
\author[3]{Zhenhua Lin}
\author[2]{Yi Yu}

\affil[1]{School of Mathematics and Statistics, Xi'an Jiaotong University}
\affil[2]{Department of Statistics, University of Warwick}
\affil[3]{Department of Statistics and Data Science, National University of Singapore}

\date{\today}

\begin{document}
\maketitle

\begin{abstract}
    Algorithmic fairness has become a central topic in machine learning, and mitigating disparities across different subpopulations has emerged as a rapidly growing research area. In this paper, we systematically study the classification of functional data under fairness constraints, ensuring the disparity level of the classifier is controlled below a pre-specified threshold. We propose a unified framework for fairness-aware functional classification, tackling an infinite-dimensional functional space, addressing key challenges from the absence of density ratios and intractability of posterior probabilities, and discussing unique phenomena in functional classification. We further design a post-processing algorithm Fair Functional Linear Discriminant Analysis classifier (Fair-FLDA), which targets at homoscedastic Gaussian processes and achieves fairness via group-wise thresholding. Under weak structural assumptions on eigenspace, theoretical guarantees on fairness and excess risk controls are established. As a byproduct, our results cover the excess risk control of the standard FLDA as a special case, which, to the best of our knowledge, is first time seen. Our theoretical findings are complemented by extensive numerical experiments on synthetic and real datasets, highlighting the practicality of our designed algorithm.     
\end{abstract}

\section{Introduction}
Driven by technological advancements that enable high-resolution data collection and analysis, functional data analysis (FDA) has gained increasing attention over the past two decades. A wide range of statistical research has been carried out in this area, and we refer readers to \citet{wang2016functional} for a comprehensive review of recent developments. 

Across a variety of statistical tasks, functional classification has emerged as one of the central focuses, with applications in many areas including neuroscience \citep[e.g.][]{heinrichs2023functional,lila2024interpretable}, genetics \citep[e.g.][]{coffey2014clustering}, handwriting recognition \citep[e.g.][]{hubert2017multivariate} and others. Various classifiers have been proposed and thoroughly studied in the literature, among which are classifiers based on projection \citep[e.g.][]{delaigle2012achieving,kraus2019classification}, Radon--Nikodym derivatives \citep[e.g.][]{berrendero2018use,torrecilla2020optimal}, principal component score densities \citep[e.g.][]{dai2017optimal} and partial least squares \citep[e.g.][]{preda2007pls}, to name but a few. 

Despite the satisfactory performance of the existing classifiers, empirical studies have shown that many algorithms might inherit bias from data and raise ethical issues \citep[e.g.][]{angwin2022machine}. Organisations like the White House \citep{whitehouse2016bigdata} have called for mitigating discrimination and prompting fairness during decision-making. In response, ensuring that algorithms do not disadvantage any vulnerable groups has become a pressing priority in both research and practice.  Substantial efforts have been directed toward reducing bias in existing methods.

In the existing literature, algorithmic fairness has been extensively considered, particularly within the settings of classification \citep[e.g.][]{yang2020fairness,jiang2020wasserstein,wei2021optimized,zeng2024bayes,zeng2024minimax, hou2024finite} and regression \citep[e.g.][]{chzhen2022minimax, fukuchi2023demographic,xu2023fair}. These algorithms can be classified into pre-, in- and post-processing procedures. Pre-processing methods aim to modify training data prior to model training, allowing models to learn from debiased inputs \citep[e.g.][]{calmon2017optimized,johndrow2019algorithm}. In-processing ones handle fairness constraints during the training step. Common strategies include fairness-constrained optimisation \citep[e.g.][]{narasimhan2018learning,celis2019classification} and fairness penalised objective functions \citep[e.g.][]{cho2020fair}. Post-processing one, by contrast, seek to reduce disparities by modifying predictions after training is completed \citep[e.g.,][]{kim2019multiaccuracy,li2022fairee}.

The most related work to this paper is \citet{zeng2024bayes}, in which they develop a framework for Bayes-optimal fair classifiers under finite-dimensional feature spaces with a strong reliance on posterior probabilities, which are unfortunately intractable in functional spaces.  This handicaps the direct application of \citet{zeng2024bayes} to functional data.  In this paper, we resort to the Radon--Nikodym derivative, which is used naturally as a substitute and plays a central role in characterising optimal decision rules under fairness constraints in the functional setting. Additional challenges, further comparisons and practical considerations are provided in \Cref{remark_compare}.

In fact, even without fairness constraints, the characterisation of excess risk for functional classification is unresolved in general settings with unknown eigenfunctions of the covariance operator.  Addressing fairness in this setting is inherently more challenging, and quantifying the trade-off between fairness and accuracy is theoretically more demanding and remains largely unexplored.

\subsection{List of contributions}
In this paper, we study the problem of optimal binary classification for functional data under various fairness constraints. Specifically, we focus on the case when the sensitive attribute is binary and the probability measures of two classes of standard features within each sensitive group are mutually absolutely continuous. The main contributions of this paper are summarised as follows. 

Firstly, to the best of our knowledge, this is the first study to explore fair classification for functional data.
We propose a unified framework for constructing the fair Bayes-optimal classifier for functional data, providing a functional data-tailored treatment of fairness-aware classification problems. 

Secondly, when the non-sensitive features are assumed to be Gaussian processes, we introduce a post-processing algorithm, the Fair Functional Linear Discriminant Analysis classifier (Fair-FLDA) in \Cref{alg:plug-in}, which effectively enforces fairness by group-wise thresholding. Our algorithm accounts for the most general setting where we assume all model parameters, including group-wise covariance functions and their eigenvalues and eigenfunctions, to be unknown. 

Thirdly, we further establish the finite-sample theoretical guarantee for the proposed algorithm in terms of both fairness and excess risk control, ensuring our algorithm Fair-FLDA not only adheres to the specified fairness constraint with high probability, but also achieves a satisfactory classification performance, with the cost of fairness explicitly quantified. As a byproduct, our results cover the special case of functional classification without fairness, which serves as a complement of \citet{wang2021optimal} under a more general setting when eigenfunctions are assumed to be unknown. 

Finally, the proposed algorithm is validated through extensive numerical experiments on both simulated and real datasets, further supporting our theoretical findings and highlighting their practicality.

\medskip 
\noindent \textbf{Notation.} In this paper, for a positive integer $a$, denote $[a] = \{1, \ldots, a\}$. For $a,b \in \mathbb{R}$, let $a \vee b = \max\{a,b\}$ and $a\wedge b = \min\{a,b\}$. For two sequences of positive numbers $\{a_n\}$ and $\{b_n\}$, denote $a_n \lesssim b_n$ (or $a_n = O(b_n)$) and $a_n \asymp b_n$, if there exists some constants $c,C > 0$ such that $a_n/b_n \leq C$ and $c \leq a_n/b_n \leq C$. Write $a_n \ll b_n$, if $a_n/b_n \rightarrow 0$ as $n \rightarrow \infty$. For an $\mathbb{R}$-valued random variable $X$ and $k \in [2]$, let $\|X\|_{\psi_k}$ denote the Orlicz-$\psi_k$ norm, i.e.~$\|X\|_{\psi_k} = \inf \{t > 0: \mathbb{E}[\exp(\{|X|/t\}^k)]\leq 2\}$. For a sequence of random variables $\{X_n\}$ and positive numbers $\{a_n\}$, denote $X_n = O_\mathrm{p}(a_n)$ if $\lim_{M \rightarrow \infty}\lim\sup_n \mathbb{P}(|X_n|\geq M a_n)=0$. For any two $\sigma$-finite measures $\mu$ and $\nu$, denote $\mu \ll \nu$ if $\mu$ is absolutely continuous with respect to $\nu$ and write $\mathrm{d}\mu/\mathrm{d}\nu$ the Radon--Nikodym derivative; write $\mu\sim \nu$ if they are equivalent. Let $L^2([0,1])$ be the space of square-integrable functions on $[0,1]$. For $f\in L^2([0,1])$, denote $\|f\|^2_{L^2} = \int_{0}^1 f^2(s)\,\mathrm{d}s$.  For $f,g \in L^2([0,1])$, denote the inner product by $\langle f, g \rangle_{L^2} = \int_{0}^1 f(s)g(s)\;\mathrm{d}s$. For any bivariate kernel function $K : [0,1]^2\rightarrow \mathbb{R}_+$, let $\mathcal{H}(K)$ denote the reproducing kernel Hilbert spaces (RKHS) generated by $K$. For $f \in \mathcal{H}(K)$, denote $\|f\|^2_{K} = \sum_{j=1}^\infty \langle f, \phi_j\rangle_{L^2}^2/\lambda_j$ its RKHS norm, where $\{\phi_j\}_{j=1}^\infty$ and $\{\lambda_j\}_{j=1}^\infty$ are obtained by Mercer's decomposition of $K$: $K(s,t) = \sum_{j=1}^\infty \lambda_j\phi_j(s)\phi_j(t)$, $s,t \in [0,1]$.

\section{Fair Bayes optimal classifier} 

In this section, we first formally introduce the functional classification model under fairness constraints in \Cref{section_setup}, then present a unified framework to obtain $f^\star_{D, \delta}$ in \Cref{section_fair_bayes} and introduce our algorithm Fair Functional Linear Discriminant Analysis classifier (Fair-FLDA) in \Cref{section_algorithm}.

\subsection{Problem setup} \label{section_setup}

Suppose that we have $n$ independent and identically distributed samples $\data = \{(X_i, A_i, Y_i), i\in[n]\}$, where $X_i \in L^2([0,1])$ is the standard functional feature, $A_i \in \{0, 1\}$ is the sensitive feature (e.g.~gender or race) and $Y_i \in \{0, 1\}$ is the binary label. Let $\mathcal F$ be the class of measurable functions $f: L^2([0,1]) \times \{0,1\} \to [0, 1]$ and our goal is to identify a randomised classifier $f^\star \in \mathcal{F}$, as defined in \Cref{def:random_classifier}, such that the misclassification error 
\begin{equation}\label{eq-misclassification-error}
    R(f^\star) = \mathbb{P}(\widehat{Y}_{f^\star}(X,A) \neq Y)   
\end{equation}
is minimised subject to a specified fairness constraint.
\begin{definition}[Randomised classifier]\label{def:random_classifier}
    For any $x \in L^2([0,1])$ and $a \in \{0,1\}$, a randomised classifier $f\in \mathcal F$ is a measurable function such that $f(x, a) = \prob(\widehat Y_f =1 | X=x, A=a)$, where $\widehat{Y}_f = \widehat{Y}_f(x,a)$ is the predicted label, i.e.~$\widehat Y_f |\{X=x, A=a\} \sim \mathrm{Bernoulli}\;(f(x,a))$.
\end{definition}

For $a,y \in \{0,1\}$, let $P_{a,y}$ be the distribution of the random process $X$ given $A=a$ and $Y=y$. We additionally assume that $P_{a,1} \sim P_{a,0}$, thus the Radon--Nikodym derivative $\eta_a(X) = \mathrm{d}P_{a,1}(X)/\mathrm{d}P_{a,0}$ is well defined. Let $\pi_{a,y} = \prob(A=a, Y=y)$ and $\pi_a = \mathbb{P}(A=a)$, then a straightforward calculation shows that for each sensitive group, the Bayes rule that minimises the misclassification error is given by 
\begin{equation}\label{eq_bayes_no_fair}
    f^*(x,a) = \indc\{\eta_a(X) \geq \pi_{a,0}/\pi_{a,1}\},
\end{equation}
where $\indc\{\cdot\}$ denotes the indicator function \citep[e.g.][]{berrendero2018use}. However, the above classifier does not take fairness into account. To address this, in order to mitigate bias across groups, existing literature proposed various notions of parity, some of which are listed below.
\begin{definition}\label{def-notions-of-parity}
    A classifier $f$ is said to satisfy (i) equality of opportunity \citep[e.g.][]{hardt2016equality} if the true positive rates are the same among protected groups, i.e.~$\prob(\widehat Y_f =1 | A=0, Y=1) = \prob(\widehat Y_f =1 | A=1, Y=1)$, (ii) predictive equality \citep[e.g.][]{corbett2017algorithmic} if the false positive rates are the same among protected groups. i.e.~$\prob(\widehat Y_f=1|A=0, Y=0) = \prob(\widehat Y_f=1| A=1, Y=0)$, and (iii) demographic parity \citep[e.g.][]{cho2020fair} if its prediction $\widehat Y_f$ is independent of the sensitive attribute $A$, i.e.~$\prob(\widehat Y_f=1| A=a) = \prob(\widehat Y_f=1)$, for $a \in \{0,1\}$.
\end{definition}

Enforcing exact parity may lead to a substantial loss in accuracy. In the existing literature, one popular approach to balance this trade-off is to instead control the disparity measure, $\mathrm{D}: \mathcal{F} \rightarrow [-1,1]$, i.e.~upper bounding the difference in quantities of interest between the sensitive groups. The disparity measures corresponding to the notions of parity in \Cref{def-notions-of-parity} are presented in  \Cref{def_disparity_measure} for completeness.

\begin{definition} \label{def_disparity_measure}
    For a given classifier $f \in \mathcal F$, the disparity of opportunity (DO), predictive disparity (PD) and demographic disparity (DD) are defined as
	\begin{align*} 
    	\mathrm{DO}(f) & = \prob\{\widehat Y_f(X, 1)=1| A=1, Y=1\} - \prob\{\widehat Y_f(X,0)=1 | A=0, Y=1\} , \\
		\mathrm{PD}(f) & = \prob\{\widehat Y_f(X, 1)=1|A=1, Y=0\} - \prob\{\widehat{Y}_f(X,0)=1|A=0, Y=0\},\\
        \mathrm{DD}(f) & = \prob\{\widehat Y_f(X, 1)=1| A=1\} - \prob\{\widehat Y_f(X,0)=1 | A=0\}.
	\end{align*}
\end{definition}
For any disparity measure $\mathrm{D}$, any tolerance level $\delta\geq 0$, we seek the $\delta$-fair Bayes optimal classifier~$f^\star_{\mathrm{D}, \delta}$, such that the misclassification error defined in \eqref{eq-misclassification-error}, is minimised over all classifiers that satisfy the $\delta$-disparity, i.e.
\begin{equation} \label{eq_bayes_fair}
    f^\star_{\mathrm{D}, \delta} \in \argmin_{f\in\mathcal F} \big\{ R(f) : |\mathrm{D}(f)| \le \delta\big\}.
\end{equation}

\subsection{Unified framework for fairness-aware functional classification} \label{section_fair_bayes}
To characterise the $\delta$-fair Bayes optimal classifier defined in \eqref{eq_bayes_fair}, in this section, we apply the generalised Neyman--Pearson lemma (see \Cref{lem:GNP} in \Cref{appendix_technical}) and exploit Radon--Nikodym derivatives, providing a unified framework for fairness-aware functional classification.

The applications of the generalised Neyman--Pearson lemma are available when both the objective function and the constraint in \eqref{eq_bayes_fair} are linear in the classifier.  As for the objective function, it follows from \Cref{l:risk_densityratio} in \Cref{appendix_bayes_optimal_auxiliary}, we can rewrite the misclassification error in \eqref{eq-misclassification-error} as
\begin{align*}
	R(f)  =  \sum_{a \in \{0,1\}} \int_{\Xspace} f(x,a) \bigg\{ \pi_{a,0} - \pi_{a,1} \frac{\mathrm{d}P_{a,1}}{\mathrm{d}P_{a,0}}(x) \bigg\} \mathrm{d}P_{a,0}(x) +  \prob(Y=1),
\end{align*}
which is linear in the classifier. 
As for the constraint, we consider the class of bilinear disparity measures defined below.
\begin{definition} \label{def:linear_disparity}
	For all probability measures $P$ and $f\in\mathcal F$, a disparity measure $\mathrm{D}: \mathcal F \to [-1,1]$ is bilinear in the classifiers $f$ and $\mathrm{d}P_{a,1}/\mathrm{d}P_{a,0}$ if there exist $s_{\mathrm{D}, a}, b_{\mathrm{D}, a} \in \real$ such that
	\[ 
        \mathrm{D}(f) = \sum_{a \in \{0,1\}} \int_{\Xspace} f(x,a) \bigg\{s_{\mathrm{D}, a} \frac{\mathrm{d}P_{a,1}}{\mathrm{d}P_{a,0}}(x) + b_{\mathrm{D}, a}\bigg\} \mathrm{d}P_{a,0}(x).
    \]    
\end{definition}
\Cref{def:linear_disparity} is also considered in \citet{zeng2024bayes} and hold for many commonly used disparity measures in the existing literature, including those defined in \Cref{def_disparity_measure}.  This is justified below, with its proof deferred to \Cref{appendix_proof_bilinear}.
\begin{proposition} \label{prop:bilinear_coef}
The disparity measures $\mathrm{DO}, \mathrm{PD}$ and $\mathrm{DD}$ defined in \Cref{def_disparity_measure} are bilinear with $s_{\mathrm{DO}, a} = 2a-1$, $b_{\mathrm{DO}, a}=0$; $s_{\mathrm{PD}, a}=0$, $b_{\mathrm{PD}, a} = 2a-1$; and $s_{\mathrm{DD}, a} = (2a-1)\pi_{a,1}/\pi_a$, $b_{\mathrm{DD}, a} = (2a-1)\pi_{a,0}/\pi_a$, for $a \in \{0, 1\}$.
\end{proposition}

With both the misclassification error and disparity measures linear in the classifier, the generalised Neyman--Pearson lemma unlocks $f^\star_{\mathrm{D}, \delta}$, a closed-form solution to the $\delta$-fair Bayes optimal classifier, presented in \Cref{thm:fair_bayes_opt}.

\begin{theorem} \label{thm:fair_bayes_opt}
Assume that $\mathrm{d}P_{a,1}/\mathrm{d}P_{a,0}$ is a continuous random variable for $a \in \{0, 1\}$. For any $\tau \in \mathbb{R}$ and a given bilinear disparity measure $\mathrm{D}$ in \Cref{def:linear_disparity}, denote the classifier
\begin{align} \label{eq:f_bayesform}
    g_{\mathrm{D}, \tau}(x,a) = \begin{cases} \vspace{0.5em}
            1, &(\pi_{a,1}-\tau s_{\mathrm{D},a}) \frac{\mathrm{d}P_{a,1}}{\mathrm{d}P_{a,0}}(x) \geq  \pi_{a,0}+\tau b_{\mathrm{D},a},\\ 
            0, &(\pi_{a,1}-\tau s_{\mathrm{D},a}) \frac{\mathrm{d}P_{a,1}}{\mathrm{d}P_{a,0}}(x) < \pi_{a,0}+\tau b_{\mathrm{D},a},
        \end{cases}
    \end{align}
    and let $\mathrm{D}(\tau) = \mathrm{D}(g_{\mathrm{D}, \tau})$. Then, for any $\delta \geq 0$, the $\delta$-fair Bayes optimal classifier is $f_{\mathrm{D}, \delta}^\star = g_{\mathrm{D}, \tau_{\mathrm{D}, \delta}^\star}$, where 
    \begin{align} \label{eq:taustar}
        \tau_{\mathrm{D}, \delta}^\star = \argmin_{\tau \in \real} \big\{ |\tau|: \, |\mathrm{D}(\tau)| \leq \delta\big\}.
    \end{align}
\end{theorem}
The proof of \Cref{thm:fair_bayes_opt} is in \Cref{appendix_proof_optimal_bayes}. The condition that $\mathrm{d}P_{a,1}/\mathrm{d}P_{a,0}$ is continuous holds for homogeneous Gaussian processes considered in \Cref{section_algorithm}. We remark that our analysis can be easily extended to the case when $\mathrm{d}P_{a,1}/\mathrm{d}P_{a,0}$ is discontinuous, by including a randomised decision rule on the set where $(\pi_{a,1}-\tau s_{\mathrm{D},a}) \mathrm{d}P_{a,1}(x) /\mathrm{d}P_{a,0} = \pi_{a,0}+\tau b_{\mathrm{D},a}$.

At a high level, \Cref{thm:fair_bayes_opt} states that the $\delta$-fair Bayes optimal classifier is shifted from the Bayes classifier \eqref{eq_bayes_no_fair} by linear factors $s_{\mathrm{D}, a}$ and $b_{\mathrm{D}, a}$.  The linear shift is ensured from the bilinearity of the disparity measure $\mathrm{D}$ and the shift level $\tau$ is further optimised in \eqref{eq:taustar}.  The optimisation in \eqref{eq:taustar}, at the core is to minimise the misclassification error among all that satisfy the fairness constraints. The detailed form is a consequence of behaviours of the misclassification error $R(g_{\mathrm{D}, \tau})$ and disparity measure $\mathrm{D}(g_{\mathrm{D}, \tau})$ as functions of $\tau$.  We show in \Cref{prop:D_and_R} in \Cref{appendix_bayes_optimal_auxiliary} that, the disparity $\mathrm{D}(\tau)$ is continuous and non-increasing, while the misclassification $R(g_{\mathrm{D}, \tau})$ is non-increasing in $(-\infty, 0)$ and non-decreasing in $(0, \infty)$. As a result, finding an optimal threshold $\tau$ that minimises the misclassification error reduces to minimising $|\tau|$.

For $\mathrm{D} \in \{\mathrm{DO}, \mathrm{PD}, \mathrm{DD}\}$ as defined in \Cref{def_disparity_measure}, it holds that $\pi_{a,1} - \tau_{\mathrm{D}, \delta}^\star s_{\mathrm{D},a}>0$ and $\pi_{a,0}+\tau_{\mathrm{D}, \delta}^\star b_{\mathrm{D},a}>0$ (see \Cref{l_range_tau} in \Cref{appendix_bayes_optimal_auxiliary}). Consequently, the fair Bayes-optimal classifier $f_{\mathrm{D}, \delta}^\star$ is
\begin{align} \label{eq_f_star}
    f_{\mathrm{D}, \delta}^\star(x,a)= \begin{cases} \vspace{0.5em}
        1, & \frac{\mathrm{d}P_{a,1}}{\mathrm{d}P_{a,0}}(x) \geq  \frac{\pi_{a,0}+\tau_{\mathrm{D}, \delta}^\star b_{\mathrm{D},a}}{\pi_{a,1} - \tau_{\mathrm{D}, \delta}^\star s_{\mathrm{D},a}},\\
        0,&\frac{\mathrm{d}P_{a,1}}{\mathrm{d}P_{a,0}}(x) <  \frac{\pi_{a,0}+\tau_{\mathrm{D}, \delta}^\star b_{\mathrm{D},a}}{\pi_{a,1} - \tau_{\mathrm{D}, \delta}^\star s_{\mathrm{D},a}}.
    \end{cases}
\end{align}
Compared to the Bayes classifier without fairness constraints in \eqref{eq_bayes_no_fair}, mitigating disparity is achieved by adjusting the classification thresholds, from $\pi_{a,0}/\pi_{a,1}$ to $(\pi_{a,0}+\tau_{\mathrm{D}, \delta}^\star b_{\mathrm{D},a})/(\pi_{a,1} - \tau_{\mathrm{D}, \delta}^\star s_{\mathrm{D},a})$, with the shift determined by the chosen disparity measure and the underlying population distributions. When $\tau^\star_{\mathrm{D},\delta} =0$, i.e.~$|\mathrm{D}(0)|\leq \delta$, we recover the Bayes classifier $f^\star(x,a)$ in \eqref{eq_bayes_no_fair}, which is, in this case, automatically fair. We write the classifier in \eqref{eq_bayes_no_fair} as $f^\star_{\mathrm{D},\infty}$ in the rest of the paper. 

\begin{remark} \label{remark_compare}
Our framework is inspired by the one in \citet{zeng2024bayes}, which is also an application of the generalised Neyman--Pearson lemma.  The key difference between \citet{zeng2024bayes} and \Cref{thm:fair_bayes_opt} lies in the generalisation of the classifier to a functional feature space.  When moving to infinite-dimensional spaces, the absence of a default base measure leads to the use of the Radon--Nikodym derivative $\mathrm{d}P_{a,1}(x)/\mathrm{d}P_{a,0}$.  It is a more natural functional used for functional classification, rather than the posterior probabilities $\mathbb{P}(Y=1|A=a, X=x)$ considered in \citet{zeng2024bayes}.  In particular, in important cases such as when the functions are Gaussian processes, $\mathrm{d}P_{a,1}(x)/\mathrm{d}P_{a,0}$ is analytically tractable but not $\mathbb{P}(Y=1|A=a, X=x)$.  In the cases when $\mathrm{d}P_{a,1}(x)/\mathrm{d}P_{a,0}$ is not tractable, there are ample tools for its approximation \citep[e.g.][]{bongiorno2016classification, dai2017optimal}.

\end{remark}

\subsection{Fair functional linear discriminant analysis classifier for Gaussian processes} \label{section_algorithm}

As a concrete and important example, we focus on a specific setting in which the functional features are modelled as Gaussian processes.  We propose the Fair Functional Linear Discriminant Analysis classifier in \Cref{alg:plug-in}, featuring a plug-in estimator built on $f^{\star}_{\mathrm{D}, \delta}$ proposed in \Cref{thm:fair_bayes_opt} and, specifically, in \eqref{eq_f_star}.

In addition to the problem setup in \Cref{section_setup}, for any collection of observations $(X,A,Y) \in \mathcal{D}$, we assume that the functional feature $X$ is a Gaussian process, i.e.~$\{X|A=a,Y=y\} \sim\mathcal{GP}(\mu_{a,y}, K_{a,y})$, where $\mu_{a,y}(t) = \mathbb{E}\{ X(t) | A=a, Y=y\} $ denote the mean function and $K_{a,y}(s,t) = \mathbb{E}[ \{X(s)-\mu_{a,y}(s)\}\{X(t)-\mu_{a,y}(t)\} | A=a, Y=y ]$ denote the covariance function, $s,t \in [0,1]$ and $a,y \in\{0,1\}$. For simplicity, we only consider a homoscedastic setting within each group in this paper, i.e.~$K_{a,0} = K_{a,1} = K_a$.  The covariance function $K_a$, consequently, admits the spectral expansion $ K_a(s,t) = \sum_{j=1}^\infty \lambda_{a,j}\phi_{a,j}(s)\phi_{a,j}(t)$, where $\lambda_{a,1} \geq \lambda_{a,2} \geq \cdots \geq 0$ are the eigenvalues and $\{\phi_{a,j}\}_{j\in \mathbb{N}_+}$ are the collection of eigenfunctions. In the absence of fairness constraints, the functional linear discriminant analysis (FLDA) classifier defined in \eqref{eq_bayes_no_fair} is known to be optimal to minimise the misclassification error \citep[e.g.][]{berrendero2018use}.

To account for fairness and construct a plug-in type classifier for $f^\star_{\mathrm{D}, \delta}$, it is essential to evaluate the Radon--Nikodym derivative $\mathrm{d}P_{a,1}/\mathrm{d}P_{a,0}$, which plays a central role in the decision rule. Although such derivatives are typically intractable, they are analytically available under the Gaussian setting. In particular, by standard results of Gaussian measures \citep[e.g.~Theorem 1 in][]{berrendero2018use}, the distributions $P_{a,0}$ and $P_{a,1}$ are mutually continuous if and only if the mean difference $\mu_{a,1}-\mu_{a,0}$ belongs to the RKHS space $\mathcal{H}(K_a)$.  The Radon--Nikodym derivative consequently admits that 
\begin{align*}
    \frac{\mathrm{d}P_{a,1}}{\mathrm{d}P_{a,0}}(X) = \exp \bigg\{  \sum_{j=1}^\infty \frac{(\zeta_{a, j}-\theta_{a,0, j})(\theta_{a, 1, j} - \theta_{a,0, j})}{\lambda_{a,j}} - \frac{1}{2} \sum_{j=1}^\infty \frac{ (\theta_{a,1,j} - \theta_{a,0, j})^2 }{\lambda_{a,j}} \bigg\},
\end{align*}
where $\zeta_{a, j}= \langle X, \phi_{a,j} \rangle_{L^2}$ is the principal component scores of $X$ and $\theta_{a,y,j} = \langle\mu_{a,y}, \phi_{a,j}\rangle_{L^2}$ are the coefficients of the mean functions projected onto the eigenfunctions of $K_a$.

\begin{remark}[Perfect classification]\label{rmk:perfect}
    In the problem of functional classification without fairness, the intrinsic infinite-dimensional nature of functional data gives rise to vanishing misclassification errors under certain scenarios. Such a phenomenon is first discussed in \citet{delaigle2012achieving} and is commonly known as perfect classification in the existing literature. As discussed in \citet{berrendero2018use}, for homogeneous Gaussian processes, perfect classification arises when the class distributions $P_{a,1}$ and $P_{a,0}$ are mutually singular, i.e.~$\mu_{a,1}-\mu_{a,0} \notin \mathcal{H}(K_a)$, in which case the Radon--Nikodym derivative $\mathrm{d}P_{a,1}/\mathrm{d}P_{a,0}$ fails to exist. We remark that in this paper, we restrict our theoretical analysis to the more challenging regime of imperfect classification, where the classification error does not vanish. Notably, in the perfect classification regime, the FLDA classifier in \eqref{eq_bayes_no_fair} is automatically fair when the disparity measure $\mathrm{D} \in \{\mathrm{DO}, \mathrm{PD}\}$. When $\mathrm{D} = \mathrm{DD}$, \eqref{eq_bayes_no_fair} is automatically fair if $|\mathbb{P}(Y =1|A=1) -\mathbb{P}(Y=1|A=0)| \leq \delta$. Further insights into the phenomenon of automatic fair are supported by numerical experiments presented in \Cref{sec:perfect}.
\end{remark}

To estimate $\mathrm{d}P_{a,1}/\mathrm{d}P_{a,0}$ in practice, we assume the availability of an additional dataset, $\widetilde{\mathcal{D}} = \{(\widetilde{X}_i,\widetilde{A}_i,\widetilde{Y}_i)\}$, which is drawn independently from the same distribution as $\mathcal{D}$. We refer to $\widetilde{\mathcal{D}}$ as the training data, used to estimate $\widehat{\eta}_a$ in the initial classifier, and to $\mathcal{D}$ as the calibration data, subsequently used to post-process the initial classifier by selecting the adjusted threshold $\widehat{\tau}_{\mathrm{D},\delta}$. 

We decompose the calibration data as $\mathcal{D} = \mathcal{D}_{0,1} \cup \mathcal{D}_{0,0} \cup \mathcal{D}_{1,1} \cup \mathcal{D}_{1,0}$, where for $a, y \in \{0, 1\}$, let $\mathcal{D}_{a,y} = \{(X_i, A_i = a, Y_i = y)\}$, $n_{a,y} = |\mathcal{D}_{a,y}|$ and $n = \sum_{a,y} n_{a,y}$. The training data are written as $\widetilde{\mathcal{D}} = \widetilde{\mathcal{D}}_{0,1} \cup \widetilde{\mathcal{D}}_{0,0} \cup \widetilde{\mathcal{D}}_{1,1} \cup \widetilde \data_{1,0}$ with sizes $\widetilde{n}_{a,y} = |\widetilde{\mathcal{D}}_{a,y}|$ and total size $\widetilde{n} = \sum_{a,y} \widetilde{n}_{a,y}$. For notational clarity, we denote the $i$-th feature in $\mathcal{D}_{a,y}$ as $X^i_{a,y}$ for $i \in [n_{a,y}]$, and write $X$ from $\widetilde{\data}_{a,y}$ as $\widetilde{X}^i_{a,y}$ for $i\in [\widetilde{n}_{a,y}]$. The resulting plug-in classifier is detailed in Algorithm~\ref{alg:plug-in}, with its theoretical guarantees discussed in \Cref{section_theory}.
\begin{algorithm}[!htbp]
	\caption{Fair Functional Linear Discriminant Analysis classifier. \label{alg:plug-in}}
	\begin{algorithmic}
		\Require Training data $\widetilde \data_{0,1} \cup \widetilde \data_{0,0} \cup \widetilde \data_{1,1} \cup \widetilde \data_{1,0}$, calibration data $\data_{0,0} \cup \data_{0,1} \cup \data_{1,0} \cup \data_{1,1}$, disparity level $\delta$, level of truncation $J$.
		\State \textbf{S1.} Estimating Radon--Nikodym derivatives $\eta_a $ and class probabilities $\pi_{a, y}$ using training data.
		\State  For $a, y\in\{0,1\}$, calculate $\widehat \pi_{a, y} = \widetilde n_{a, y}/\widetilde n$, $\widehat{\mu}_{a, y}(t) = (\widetilde n_{a, y})^{-1}\sum_{i=1}^{\widetilde n_{a, y}} \widetilde X_{a, y}^i(t)$ and 
			\begin{equation*}
			\widehat{K}_{a}(s,t) = \sum_{y \in \{0,1\}} \frac{\widetilde n_{a,y}}{\widetilde n_{a,0}+\widetilde n_{a,1}} \frac{1}{\widetilde n_{a, y} -1} \sum_{i=1}^{\widetilde n_{a,y}}\{\widetilde X^i_{a,y}(s) - \widehat{\mu}_{a,y}(s)\}\{\widetilde X^i_{a,y}(t) - \widehat{\mu}_{a,y}(t)\}. 
		\end{equation*}
		\State Estimate empirical eigenvalues $\{\widehat{\lambda}_{a,j}\}_{j =1}^J$ and eigenfunctions $\{\widehat{\phi}_{a,j}\}_{j=1}^J$ of $\widehat{K}_a$ by spectral expansion.
		\State Calculate $\widehat{\theta}_{a,y,j} = \int_0^1 \widehat{\mu}_{a,y}(t)\widehat{\phi}_{a,j}(t)\, \mathrm{d}t$. For any function $X$, denote
		\begin{equation*}
			\widehat \eta_a(X) = \exp\bigg\{\sum_{j=1}^J \frac{(\widehat{\theta}_{a,1,j} - \widehat{\theta}_{a,0,j} )(\int_{0}^1 X(t) \widehat{\phi}_{a,j}(t) \,\mathrm{d}t  - \widehat{\theta}_{a,0,j})}{\widehat{\lambda}_{a,j}} - \frac{1}{2}\sum_{j=1}^J\frac{( \widehat{\theta}_{a,1,j}- \widehat{\theta}_{a,0,j})^2}{\widehat{\lambda}_{a,j}}\bigg\} .
		\end{equation*}
		
		\State \textbf{S2.} Estimating the optimal threshold using calibration data.
		\State Let $\widehat g_{\mathrm{D}, \tau}(x,a) = \indc \big\{ (\widehat{\pi}_{a,1} - \tau s_{\mathrm{D}, a}) \widehat \eta_a(x) > \widehat{\pi}_{a,0} +\tau b_{\mathrm{D},a} \big\}$.
		\State With $\int_\Xspace f(x, a) \;\mathrm{d}\widehat P_{a, y} = n_{a,y}^{-1}\sum_{i=1}^{n_{a,y}} f(X_{a, y}^i, a)$, calculate
		\[ \widehat{\mathrm{D}}(\tau) = \sum_{a \in \{0,1\}} \bigg\{ \int_\Xspace \widehat g_{\mathrm{D}, \tau}(x, a) s_{\mathrm{D}, a} \;\mathrm{d}\widehat P_{a, 1}(x) + \int_\Xspace \widehat g_{\mathrm{D}, \tau}(x, a) b_{\mathrm{D}, a} \;\mathrm{d}\widehat P_{a, 0}(x)   \bigg\}. \]		
		\State Set $\widehat{\tau}_{\mathrm{D},\delta} = \argmin_{\tau \in \mathbb{R}}\{|\tau|:\, |\widehat{\mathrm{D}}(\tau)| \leq \delta \}$. 
		\Ensure $\widehat{f}_{\mathrm{D},\delta}(x,a)$, with
			\begin{align*}
			\widehat{f}_{\mathrm{D}, \delta}(x,a) = \indc\big\{ (\widehat{\pi}_{a,1} - \widehat \tau_{\mathrm{D}, \delta} s_{\mathrm{D}, a}) \widehat \eta_a(x) > \widehat{\pi}_{a,0} + \widehat \tau_{\mathrm{D}, \delta} b_{\mathrm{D},a} \big\}.
		\end{align*} 
	\end{algorithmic}
\end{algorithm}

\section{Theoretical Properties} \label{section_theory}

For a general bilinear disparity measure defined in \Cref{def:linear_disparity}, we provide the theoretical guarantees on the fairness and excess risk control of the Fair-FLDA algorithm (\Cref{alg:plug-in}) in Theorems \ref{thm_fair_guarantee} and \ref{thm_fairness_general}, with a special case regarding the disparity of opportunity in \Cref{thm_fairness}.  
To kick off, we list a few assumptions.

\begin{assumption}[Class probabilities] \label{a_class_prob}
Assume that there exist absolute constants $0<C_p \leq C_p'<1$,  such that the class probabilities satisfy $0 < C_p \leq \pi_{a,y}\leq C_p'<1$, $a,y \in \{0,1\}$. 
\end{assumption}

\begin{assumption}[Gaussian processes]\label{a_data}
Assume that the standard features $\{\widetilde{X}^i_{a,y}\}_{i\in [\widetilde{n}_{a,y}]}\cup \{X^i_{a,y}\}_{i\in [n_{a,y}]}$ are collections of Gaussian processes $\mathcal{GP}(\mu_{a,y},K_{a})$ with continuous trajectories, $a,y \in \{0,1\}$.  In addition, assume the following holds for any $a \in \{0,1\}$.
    \begin{enumerate}[label=\textbf{\alph*.}]
        \item \label{a_data_cov}(Covariance function.)  The covariance function $K_a$ is continuous and there exist absolute constants $C_{\lambda}, C'_{\lambda}>0$ such that the eigenvalues of the covariance operator are decreasing with $j^{-\alpha} \geq C_{\lambda}\lambda_{a,j} \geq C'_{\lambda}\lambda_{a,j+1} + j^{-\alpha-1}$ for $\alpha >1$ and $j \in \mathbb{N}_+$.
    
        \item \label{a_data_snr} (Signal-to-noise ratio.) There exists an absolute constant $C_K>0$ such that $\|\mu_{a,1} -\mu_{a,0}\|^2_{K_a} \geq C_K$.
    
        \item \label{a_data_mean_decay} (Mean difference.) There exists a constant $C_\mu >0$ such that for any $j \in \mathbb{N_+}$, it holds that $ |\langle \mu_{a,1} - \mu_{a,0}, \phi_{a,j}\rangle_{L^2}| \leq C_\mu j^{-\beta}$ with $\beta > (\alpha+1)/2$. 
    \end{enumerate}
\end{assumption}

In \Cref{a_class_prob}, we assume that the class probabilities are bounded away from $0$ and $1$, essential to ensure that a sufficient number of samples for each group can be observed in both the training and calibration steps. 

Assumption \ref{a_data} characterises the properties of the functional features. 
In particular, \Cref{a_data}\ref{a_data_cov} specifies the decaying rate of eigenvalues and quantifies the spacings between two consecutive eigenvalues.  This is commonly seen in the FDA literature \citep[e.g.][]{hall2007methodology,dou2012estimation} involving eigenfunction estimations. \Cref{a_data}\ref{a_data_snr}~imposes a lower bound on the magnitude of the signal-to-noise ratio $\|\mu_{a,1} -\mu_{a,0}\|^2_{K_a}$. We thus exclude the trivial classification regime and preclude the classifier from degenerating into random guessing. \Cref{a_data}\ref{a_data_mean_decay} enforces the alignment between the mean difference and the eigenspace, with larger values of $\beta$ indicating better alignment. Assumptions \ref{a_data}\ref{a_data_cov}, \ref{a_data}\ref{a_data_snr}~and \ref{a_data}\ref{a_data_mean_decay}~jointly imply that $\|\mu_{a,1} -\mu_{a,0}\|^2_{K_a} \asymp 1$, and that the tail sum satisfies $\sum_{j=J+1}^\infty (\theta_{a,1,j}-\theta_{a,0,j})^2/\lambda_j \lesssim J^{\alpha-2\beta+1}$ for any $J \in \mathbb{N}_+$. Throughout the paper, we assume that the decay rates $\alpha$ and $\beta$ are invariant for $a \in \{0,1\}$, but this assumption can be easily generalised to allow for different decaying rates between groups.

With the above assumptions, we first establish the fairness guarantee of \Cref{alg:plug-in}.

\begin{theorem}[Fairness guarantee] \label{thm_fair_guarantee}
    For any $\delta >0$ and bilinear disparity measure $\mathrm{D}$ in \Cref{def:linear_disparity}, let $\widehat{f}_{\mathrm{D}, \delta}$ denote the classifier output by \Cref{alg:plug-in}, we have, for any $\eta \in (0,1/2)$,~that
    \begin{equation*}
        \mathbb{P}\Big\{|\mathrm{D}(\widehat{f}_{\mathrm{D}, \delta})|\leq \delta + C\sqrt{\frac{\log(1/\eta)}{n}}\Big\} \geq 1-\eta,
    \end{equation*}
    where $C>0$ is an absolute constant.
\end{theorem}
\Cref{thm_fair_guarantee} is proved in \Cref{appendix_fair_guarantee} and demonstrates that with probability at least $1 - \eta$, after the calibration step, the disparity level of $\widehat{f}_{\mathrm{D}, \delta}$ does not exceed the pre-specified level $\delta$ by a small offset term up to $O(\sqrt{\log(1/\eta)/n})$. The magnitude of the offset term is quantified by the high probability upper bound on $|\widehat{\mathrm{D}}(\widehat{\tau}_{\mathrm{D},\delta})-\mathrm{D}(\widehat{\tau}_{\mathrm{D},\delta})|$, which measures the deviation of the empirical distribution from its corresponding population counterpart.

If one instead insists on controlling the population unfairness $\mathrm{D}(\widehat{f}_{\mathrm{D}, \delta})$ below $\delta$, then provided that $\sqrt{\log(1/\eta)/n} \lesssim \delta$, it suffices to adjust the input $\delta$ to $\delta - C\sqrt{\log(1/\eta)/n}$, i.e.~set $\widehat{\tau}_{\mathrm{D},\delta} = \argmin_{\tau \in \mathbb{R}}\{|\tau|: \, |\widehat{\mathrm{D}}(\tau)| \leq \delta- C\sqrt{\log(1/\eta)/n}\}$ in \textbf{S2} of \Cref{alg:plug-in}. We refer to the resulting method as the Fair Functional Linear Discriminant Analysis classifier with Calibration (Fair-$\mathrm{FLDA_c}$). Numerical results comparing the performance of Fair-FLDA and Fair-$\mathrm{FLDA_c}$ are presented in \Cref{section_numerical}.

We next present the excess risk control of our method. As a preliminary step, we establish the misclassification error of the oracle classifier in \Cref{thm_misclassification}.
\begin{proposition}[Misclassification error] \label{thm_misclassification}
    Under the model setup in \Cref{section_algorithm}, for any $\delta \geq 0$ and bilinear disparity measure $\mathrm{D}$ such that $\pi_{a,1} - \tau_{\mathrm{D}, \delta}^\star s_{\mathrm{D},a}>0$ and $\pi_{a,0}+\tau_{\mathrm{D}, \delta}^\star b_{\mathrm{D},a}>0$, the corresponding misclassification error for $f_{D, \delta}^\star$ defined in \Cref{thm:fair_bayes_opt} is given by 
    \begin{align*}
	R(f_{\mathrm{D}, \delta}^\star)  =\;& \sum_{a \in \{0,1\}} \pi_{a, 0} \Phi\bigg[ - \frac{\|\mu_{a,1} -\mu_{a,0}\|_{K_a}}{2} - \frac{\log\big\{ (\pi_{a,0}+\tau_{\mathrm{D}, \delta}^\star b_{\mathrm{D},a})/(\pi_{a,1} - \tau_{\mathrm{D}, \delta}^\star s_{\mathrm{D},a}) \big\}}{ \|\mu_{a,1} -\mu_{a,0}\|_{K_a} }  \bigg] \\
	&  +\sum_{a \in \{0,1\}} \pi_{a,1}  \Phi\bigg[ - \frac{\|\mu_{a,1} -\mu_{a,0}\|_{K_a}}{2} + \frac{\log\big\{ (\pi_{a,0}+\tau_{\mathrm{D}, \delta}^\star b_{\mathrm{D},a})/(\pi_{a,1} - \tau_{\mathrm{D}, \delta}^\star s_{\mathrm{D},a}) \big\}}{ \|\mu_{a,1} -\mu_{a,0}\|_{K_a} }  \bigg],
    \end{align*}
    where $\Phi$ is the cumulative distribution function of the standard normal distribution.
\end{proposition}
See \Cref{appendix_proof_misclassification_error} for the proof of \Cref{thm_misclassification}. Compared to the standard excess risk without fairness constraints, $R(f^\star_{\mathrm{D},\infty})$ \citep[e.g.~Theorem 2 in][]{berrendero2018use}, the cost of fairness constraints is quantified by the logarithmic term involving $\tau^\star_{\mathrm{D},\delta}$, resulted from the adjusted group-wise thresholds in~\eqref{eq_f_star}. Note that as shown in \Cref{prop:D_and_R} in \Cref{appendix_bayes_optimal_auxiliary}, a decrease in $\delta$, i.e.~stronger fairness constraints, leads to an increase in $|\tau^\star_{\mathrm{D},\delta}|$, hence a larger misclassification error $R(f_{\mathrm{D}, \delta}^\star)$. When $f^\star_{\mathrm{D},\infty}$ is automatically fair, i.e.~$\tau^*_{\mathrm{D},\delta}=0$, \Cref{thm_misclassification} recovers Theorem 2 in \citet{berrendero2018use}. 

Moving towards excess risk controls, as discussed in \citet{zeng2024minimax}, for any $f\in \mathcal{F}$, the traditional excess risk $R(f) -R(f_{\mathrm{D}, \delta}^\star)$ may be negative as $f_{\mathrm{D}, \delta}^\star$ does not necessarily minimise the excess risk, i.e.~$f_{\mathrm{D}, \delta}^\star \notin \argmin_{f\in \mathcal{F}} R(f)$.  To make a meaningful control of the misclassification error, we resort to the quantity $|R(f) -R(f_{\mathrm{D}, \delta}^\star)|$, which can be further decomposed as the non-negative fairness-aware excess risk $d_E(f, f^{\star}_{\mathrm{D}, \delta})$ (defined in \Cref{def-fairness-aware-excess-risk}) and a disparity cost $\tau_{\mathrm{D}, \delta}^\star \{\mathrm{D}(f_{\mathrm{D}, \delta}^\star) - \mathrm{D}(f)\}$, that
\begin{equation}\label{eq-R(f)-decomp}
    |R(f) -R(f_{\mathrm{D}, \delta}^\star)| = |d_E(f, f_{\mathrm{D}, \delta}^\star) + \tau_{\mathrm{D}, \delta}^\star \{\mathrm{D}(f_{\mathrm{D}, \delta}^\star) - \mathrm{D}(f)\}| \leq  d_E(f, f_{\mathrm{D}, \delta}^\star) + |\tau_{\mathrm{D}, \delta}^\star| |\mathrm{D}(f_{\mathrm{D}, \delta}^\star) - \mathrm{D}(f)|.
\end{equation}
The derivation of \eqref{eq-R(f)-decomp} is directly via the definition below.

\begin{definition}[Fairness-aware excess risk] \label{def-fairness-aware-excess-risk}
    For $\delta \geq 0$, let $f_{\mathrm{D}, \delta}^\star$ be a $\delta$-fair Bayes optimal classifier defined in \eqref{eq_bayes_fair}, recalling $\tau^\star_{\mathrm{D},\delta}$ in \eqref{eq:taustar} and $s_{\mathrm{D}, a}$, $b_{\mathrm{D},a}$ in \Cref{def:linear_disparity}. For any classifier $f: \Xspace \times \{0,1\} \rightarrow [0,1]$, the fairness-aware excess risk, $d_E(f, f_{\mathrm{D}, \delta}^\star)$,  is defined as
	\begin{align*}
		&d_E(f, f_{\mathrm{D}, \delta}^\star)\\
        =\;& \sum_{a \in \{0,1\}}  \int_\Xspace\big\{ f(x, a) - f_{\mathrm{D}, \delta}^\star(x, a) \big\} \Big[ \big(\pi_{a, 0} + \tau_{\mathrm{D}, \delta}^\star b_{\mathrm{D},a}\big)- \big(\pi_{a,1}-\tau_{\mathrm{D}, \delta}^\star s_{\mathrm{D}, a}\big)\frac{\mathrm{d}P_{a, 1}}{\mathrm{d}P_{a,0}}(x)\Big]\; \mathrm{d}P_{a, 0}(x).
	\end{align*}
\end{definition}
We are now ready to present the excess risk control for \Cref{alg:plug-in} in \Cref{thm_fairness_general}.
\begin{theorem} \label{thm_fairness_general}
     Denote $\epsilon_\pi$, $\epsilon_\eta$ and $\epsilon_{\mathrm{D}}$ the estimation error related to $\widehat{\pi}_{a,y}$, $\widehat{\eta}_a$ and $\widehat{\mathrm{D}}$, i.e.~for any small $\eta \in (0,1/2)$, $a,y \in \{0,1\}$ and $X \sim \mathcal{GP}(\mu_{a,y},K_a)$, it holds with probability at least $1-\eta/3$ that
     \begin{equation} \label{thm_fairness_general_eq1}
         |\widehat{\pi}_{a,y}-\pi_{a,y}| \leq \epsilon_\pi,\quad |\log\{\widehat{\eta}_a(X)\} -\log\{\eta_a(X)\}|  \leq \epsilon_\eta \quad \mbox{and} \quad \sup_{\tau \in \mathbb{R}}|\widehat{\mathrm{D}}(\tau) - \mathrm{D}(\tau)|\leq \epsilon_\mathrm{D}.
     \end{equation}
    Suppose that Assumptions \ref{a_class_prob} and \ref{a_data} hold; and that for any $\delta \geq 0$, any bilinear disparity measure $\mathrm{D}$ defined in \Cref{def:linear_disparity}, satisfying that (i) $\mathrm{D}(0)\notin (\delta - \epsilon_\mathrm{D},\delta] \cup [-\delta, -\delta+\epsilon_\mathrm{D})$, (ii) $\pi_{a,1} - \tau_{\mathrm{D}, \delta}^\star s_{\mathrm{D},a} \geq c_1$, $\pi_{a,0}+\tau_{\mathrm{D}, \delta}^\star b_{\mathrm{D},a} \geq c_1$, with $\max\{|s_{\mathrm{D},a}|,|b_{\mathrm{D},a}|\} \leq c_2$, and (iii) $|\mathrm{D}(\tau_{\mathrm{D}, \delta}^\star) - \mathrm{D}(\tau_{\mathrm{D}, \delta}^\star+\xi)| \geq C_\mathrm{D}|\xi|^\frac{1}{\gamma}$ for any~$\xi$ in the small neighbourhood of $0$ and some $\gamma >0$,
   where $c_1, c_2 ,C_\mathrm{D}>0$ are absolute constants.
    
    It holds with probability at least $1-\eta$, $\eta \in (0,\epsilon_\pi + \epsilon_\eta+\epsilon_\tau)$,  that the classifier $\widehat{f}_{\mathrm{D}, \delta}$ output by \Cref{alg:plug-in} satisfies that
    \begin{equation} \label{eq-thm5-Rf-upper-bound}            
        |R(\widehat{f}_{\mathrm{D}, \delta}) -R(f_{\mathrm{D}, \delta}^\star)| \lesssim d_E(f, f_{\mathrm{D}, \delta}^\star) + |\tau_{\mathrm{D}, \delta}^\star|\sqrt{\frac{\log(1/\eta)}{n}},
    \end{equation}
    where $d_E(\widehat{f}_{\mathrm{D}, \delta}, f_{\mathrm{D}, \delta}^\star) \lesssim (\epsilon_\pi + \epsilon_\eta+\epsilon_\tau)^2$ with $\epsilon_\tau = |\widehat{\tau}_{\mathrm{D},\delta} - \tau_{\mathrm{D},\delta}^\star|  \lesssim \epsilon_\mathrm{D}^{\gamma}\indc\{\tau_{\mathrm{D}, \delta}^\star \neq 0\}$.    
\end{theorem}

\Cref{thm_fairness_general} provides a general characterisation of fairness-aware excess risk and traditional excess risk when $\mathrm{D}(0)$ is not too close to $\delta$, i.e.~when $f^\star_{\mathrm{D},\infty}$ is either sufficiently fair or unfair. We impose two extra assumptions for $\mathrm{D}$, with the first one controlling the behaviours of $\tau^\star_{\mathrm{D},\delta}$ near the boundaries. This condition is used to ensure that $\widehat{\tau}_{\mathrm{D},\delta}$ lies within the range of its estimated counterparts, i.e.~$\widehat{\pi}_{a,1} - \widehat{\tau}_{\mathrm{D}, \delta} s_{\mathrm{D},a}>0$ and $\widehat{\pi}_{a,0}+\widehat{\tau}_{\mathrm{D}, \delta} b_{\mathrm{D},a}>0$. This guarantees that the misclassification error $R(\widehat{f}_{\mathrm{D},\delta})$ retains a well-structured form, analogous to that in \Cref{thm_misclassification}. The second assumption controls the steepness of $\mathrm{D}$ in a small neighbourhood of $\tau^\star_{\mathrm{D},\delta}$, with larger values of $\gamma$ corresponding to greater steepness. In the fair-impacted case, i.e.~$\tau^\star_{\mathrm{D},\delta} \neq 0$, if $\mathrm{D}$ is potentially very flat near $\tau^\star_{\mathrm{D},\delta}$, i.e.~$\gamma$ is small, the estimation problem becomes more challenging, resulting in a larger estimation error $\epsilon_\tau$ (as illustrated in \Cref{fig:D_tau}). We remark that when $\mathrm{D}$ is explicitly given, both of the above assumptions can be verified in most of the cases, see \Cref{thm_fairness} as an example. 

\begin{figure*}[!htbp]
	\begin{center}
		\newcommand{\thiswidth}{0.25\linewidth}
		\newcommand{\thisgap}{0mm}
		\begin{tabular}{cc}
			\hspace{\thisgap}\includegraphics[width=\thiswidth]{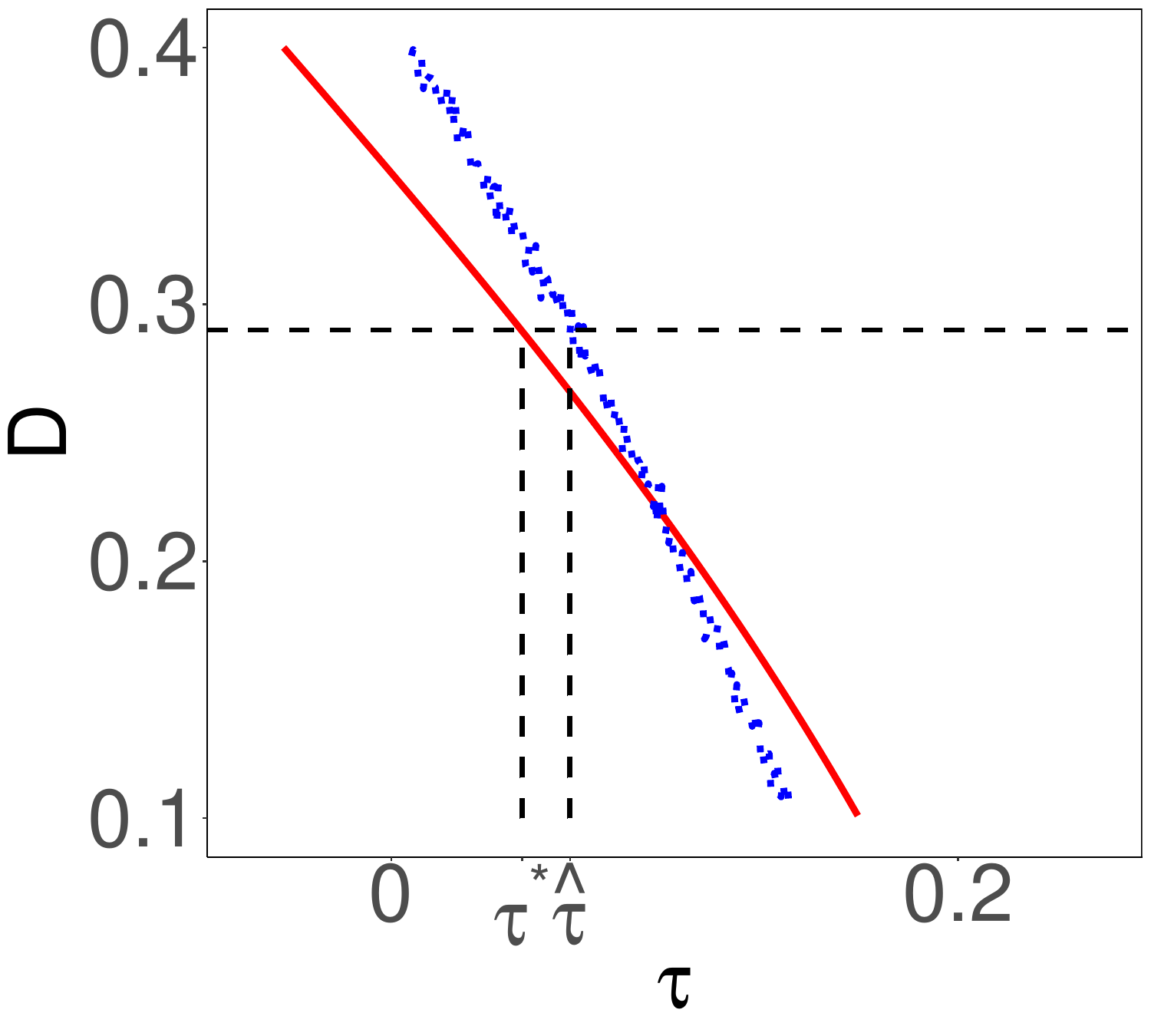} &
			\hspace{\thisgap}\includegraphics[width=\thiswidth]{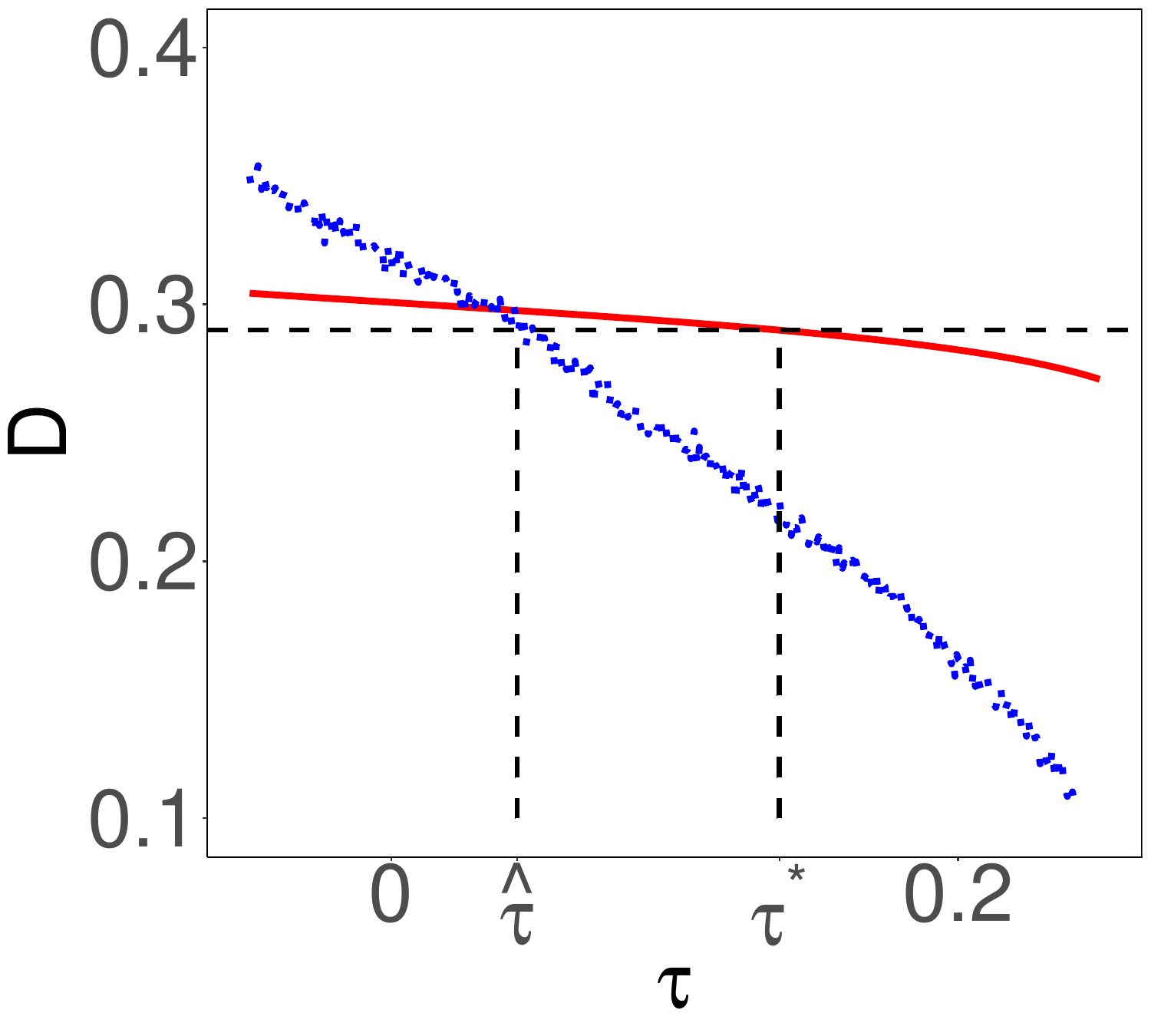} 
		\end{tabular}
		\caption{Effects of steepness of disparity levels on the estimation error of $\tau^\star_{\mathrm{D},\delta}$. Left and right panels illustrate steep and flat $\mathrm{D}(\cdot)$.  Red solid line: $\mathrm{D}(\cdot)$.  Blue dotted line: $\widehat{\mathrm{D}}(\cdot)$.}
		\label{fig:D_tau}
	\end{center}
\end{figure*}

When $\tau^\star_{\mathrm{D},\delta} \neq 0$, the disparity cost in the right-hand side of \eqref{eq-thm5-Rf-upper-bound} is a direct consequence of the fairness guarantees in \Cref{thm_fair_guarantee}. The fairness-aware excess risk, $d_E(\widehat{f}_{\mathrm{D}, \delta}, f_{\mathrm{D}, \delta}^\star)$, consists of two components, with $(\epsilon_\pi + \epsilon_\eta)^2$ capturing the intrinsic excess risk and $\epsilon_\mathrm{D}^{2\gamma}$ reflecting the cost of fairness. This is the first time seen in the FDA literature, echoing the same spirit in the finite-dimensional classification literature \citep{zeng2024minimax,hou2024finite}. Note that, when $\mathrm{D}$ is chosen as one of the disparity measures in \Cref{def_disparity_measure}, $\epsilon_\mathrm{D}$ can be explicitly controlled as $\epsilon_\mathrm{D} \asymp \epsilon_\pi + \epsilon_\eta +\sqrt{\log(1/\eta)/n}$.

When $\tau_{\mathrm{D},\delta}^\star  =0$, our results unveil the excess risk bound for FLDA, i.e.~$R(\widehat{f}_{\mathrm{D}, \infty}) -R(f_{\mathrm{D}, \infty}^\star) \lesssim (\epsilon_\pi + \epsilon_\eta\big)^2$.  Although without fairness constraints, such a result is also the first time seen in the FDA literature. To the best of our knowledge, excess risk control for functional classification has previously been considered in \citet{meister2016optimal} and \citet{wang2021optimal}, with the former relying on smoothness assumptions of the functional densities and decay rates of the metric entropy of the functional space, and the latter assuming that the eigenfunctions of the covariance operator are explicitly known. Instead, our result in \eqref{eq-thm5-Rf-upper-bound} is established under a more general setting, requiring only weak structural assumptions on the eigenspace. A more detailed comparison with \citet{wang2021optimal} is presented following \Cref{thm_fairness}.

To further illustrate our results, we apply the framework in \Cref{thm_fairness_general} to a setting where $\mathrm{D} = \mathrm{DO}$, the disparity of opportunity. The corresponding results are summarised in \Cref{thm_fairness}. 
\begin{corollary}[Excess risk control under disparity of opportunity] \label{thm_fairness}
    Suppose that the training and calibration data $\mathcal{D} \cup \widetilde{\mathcal{D}}$ are generated under Assumptions \ref{a_class_prob} and \ref{a_data}.  For any $\delta \geq 0$, the classifier $\widehat{f}_{\mathrm{DO}, \delta}$ output by \Cref{alg:plug-in} satisfies the following properties. 
    \begin{enumerate}
        \item  \label{thm_fairness_1}For any truncation level $J \in \mathbb{N}_+$ in \textbf{S1} in \Cref{alg:plug-in} such that 
        \begin{equation} \label{thm_fairness_eq_J}
            J \gtrsim \log^2(J) \;\; \text{and} \;\; J^{2\alpha+2}\log^2(J)\log(\widetilde{n}/\widetilde{\eta}) \lesssim \widetilde{n},
        \end{equation}
        and any arbitrarily small constants $\eta \in (0,n^{-1/2} \wedge \widetilde{n}^{(\alpha-2\beta+1)/(2\beta-\alpha)})$, denote 
        \begin{align*}
        \epsilon_\eta = \begin{cases}\vspace{0.5em}
               \sqrt{\frac{J^{\alpha-2\beta+4}\log(\widetilde{n}/\eta)\log(1/\eta)}{\widetilde{n}}} + \sqrt{J^{\alpha-2\beta+1}\log(1/\eta)}, & \frac{\alpha+1}{2}< \beta\leq \frac{\alpha+2}{2},\\ \vspace{0.5em}
               \sqrt{\frac{J^2\log(\widetilde{n}/\eta)\log(1/\eta)}{\widetilde{n}}}+ \sqrt{J^{\alpha-2\beta+1}\log(1/\eta)}, & \frac{\alpha+2}{2}< \beta\leq \frac{\alpha+3}{2},\\
                \sqrt{\frac{J\log(\widetilde{n}/\eta)\log(1/\eta)}{\widetilde{n}}}+ \sqrt{J^{\alpha-2\beta+1}\log(1/\eta)},& \beta > \frac{\alpha+3}{2}.
            \end{cases}
    \end{align*}
    Then, it holds with probability at least $1-\eta$ that 
        \begin{align*}
            &d_E(\widehat{f}_{\mathrm{DO}, \delta}, f_{\mathrm{DO}, \delta}^\star) \lesssim \epsilon_\eta^2 + \frac{ \log(1/\eta)\indc\{\tau_{D, \delta}^\star \neq 0\}}{n},
        \end{align*}
        if we additionally assume that $\mathrm{DO}(0)\notin (\delta - \epsilon_\eta -\sqrt{\log(1/\eta)/n},\delta] \cup [-\delta, -\delta+\epsilon_\eta +\sqrt{\log(1/\eta)/n})$.

        \item \label{thm_fairness_2} If we further assume that $n \asymp \widetilde{n} \asymp N$ up to poly-logrithmic factors and select the truncation level $J$ in \Cref{alg:plug-in} as
        \[J \asymp N^{\frac{1}{2\alpha+2}} \cdot \indc\Big\{\frac{\alpha+1}{2}< \beta< \frac{3\alpha+2}{2}\Big\} + N^{\frac{1}{2\beta-\alpha}} \cdot \indc\Big\{\beta \geq \frac{3\alpha+2}{2}\Big\},\]
        then it holds that 
        \[d_E(\widehat{f}_{\mathrm{DO}, \delta}, f_{\mathrm{DO}, \delta}^\star) = \begin{cases}\vspace{0.3em}
                O_\mathrm{p}\Big(N^{\frac{\alpha-2\beta+1}{2\alpha+2}}\Big),  & \frac{\alpha+1}{2}< \beta< \frac{3\alpha+2}{2},\\ 
                O_\mathrm{p}\Big(N^{\frac{\alpha-2\beta+1}{2\beta-\alpha}}\Big), & \beta \geq \frac{3\alpha+2}{2},
            \end{cases}\]
        and
        \[|R(\widehat{f}_{\mathrm{DO}, \delta}) -R(f_{\mathrm{DO}, \delta}^\star)| \leq d_E(\widehat{f}_{\mathrm{DO}, \delta}, f_{\mathrm{DO}, \delta}^\star) + |\tau_{\mathrm{DO}, \delta}^\star|O_\mathrm{p}\big(N^{-\frac{1}{2}}\big).\]
    \end{enumerate}    
\end{corollary}

We defer the proof of \Cref{thm_fairness} to \Cref{appendix_proof_thm_fairness}.
\Cref{thm_fairness}.\ref{thm_fairness_1} shows that the upper bound is of the form,
\begin{equation*}
    d_E(\widehat{f}_{\mathrm{DO}, \delta}, f_{\mathrm{DO}, \delta}^\star) \leq \mbox{variance to estimate $\eta_a$}+\mbox{squared bias due to truncation}+\mbox{cost of fairness},
\end{equation*}
which further highlights the role of the truncation parameter $J$ in determining the final convergence rate through the underlying bias-variance trade-off. Moreover, as the mean difference aligns more strongly with the eigenspace, i.e.~as $\beta$ increases, a smaller estimation variance is observed. 

In \Cref{thm_fairness}.\ref{thm_fairness_2}, we further elaborate on the optimal choice of $J$ and present the corresponding excess risk upper bounds for different values of $\beta$. Although \Cref{thm_fairness}.\ref{thm_fairness_1} classifies the excess risk into three regimes based on the relationship between $\alpha$ and $\beta$, with the extra condition on $J$ in \eqref{thm_fairness_eq_J}, we can only select $J$ to its maximum $N^{1/(2\alpha+2)}$ whenever $(\alpha+1)/2< \beta <(3\alpha+2)/2$, in which the approximation error $J^{\alpha-2\beta+1}$ always dominants. Note that the cost of fairness is masked in the setting when $n \asymp \widetilde{n} \asymp N$, which shares the same finding as \citet{hou2024finite}.

\Cref{thm_fairness} is the first time providing finite-sample guarantees for functional classification under fairness constraints.  There is no existing work in FDA with fairness constraints, and even for the FDA work without fairness, such detailed characterisation is novel in the FDA literature.  Without any predecessors regarding the former, we list some relevant works with the latter point.

The most relevant work is \citet{wang2021optimal}, which studies the binary classification problem for Gaussian processes under the assumption that the underlying eigenfunctions are known. In the case when $\tau^\star_{\mathrm{D},\delta} =0$, our results when $\beta > (\alpha+3)/2$ recover the excess risk rate established in \citet{wang2021optimal}, provided that $J^{2\alpha+2} \lesssim \widetilde{n}$ up to poly-logarithmic factors. We would like to remark that the upper bound on $J$ arises from the need to estimate eigenfunctions. Similar conditions can also be seen in previous work like \citet{hall2007methodology} and \citet{dou2012estimation}. 

Another relevant work is \citet{tony2019high}, in which they study the linear discriminant analysis classifier for $J$-dimensional Gaussian random variables. When $\beta > (\alpha+3)/2$, our key variance term~$J/\widetilde{n}$, resulting from estimating the first $J$ leading terms in $\eta_a$, aligns with the minimax optimal result derived in \citet{tony2019high}. At a high level, $f^\star_{\mathrm{D},\delta}$ can be viewed as applying LDA to an infinite number of principal component scores. By focusing only on the first $J$ scores in the construction of $\widehat{f}_{\mathrm{D},\delta}$, we recover the variance term, which aligns with the corresponding variance term in $J$-dimensional LDA. This phenomenon further supports the perspective in the existing literature that, when functional data are fully observed, the problem complexity becomes comparable to that of finite-dimensional settings, e.g.~functional mean estimation in \citet{zhang2016sparse} and varying coefficient model estimation in \citet{xue2024optimal}.

\section{Numerical experiments} \label{section_numerical}
In this section, we conduct numerical experiments to demonstrate the feasibility of \Cref{alg:plug-in} and support our theoretical findings in \Cref{section_theory}. Experiments with simulated and real data are demonstrated in Sections \ref{sec:sim} and \ref{sec:real-data}, respectively. All the codes and implementations can be found at \url{https://github.com/hxyXiaoyuHu/Fair-FLDA}.

\subsection{Simulated data analysis}\label{sec:sim}

Generate $(Y, A) \in \{0, 1\}^{\otimes 2}$ according to the distributions $\prob(A=1)=0.7, \prob(Y=1|A=0)=0.4$ and $\prob(Y=1|A=1)=0.7$.
Given $Y=y$ and $A=a$, generate the functional covariate $X_{a,y}(t)$ as $X_{a,y}(t) = \mu_{a, y}(t) + \sum_{k=1}^{50} \zeta_{a,k} \phi_k(t)$, where $\phi_k(t) = \sqrt{2}\cos(k \pi t)$, $\zeta_{a,k} \sim N(0, \lambda_{a,k})$, $\lambda_{0,k} = k^{-2}$, $\lambda_{1,k} = 2k^{-2}$, and the mean functions are specified as follows,
\[ \mu_{0,0}=\mu_{1,0}=0, ~~ \mu_{0,1}(t) = \sum_{k=1}^{50} 0.8(-1)^k k^{-\beta} \phi_k(t), ~~ \mu_{1,1}(t) = \sum_{k=1}^{50} \sqrt{2}(-1)^k k^{-\beta} \phi_k(t). \]
 
Let $n$ denote the size of the training sample. 
We implement the proposed fairness-aware classifier under two calibration settings,
\begin{itemize}
	\item Fair-FLDA: calibration constant set to 0;
	\item Fair-$\mathrm{FLDA_c}$: calibration constant set to $\min\{\sqrt{2\log(1/\rho)/n}, \delta\}$, with $\rho=0.05$.
\end{itemize}
Truncation levels are selected via 5-fold cross-validation, specifically by minimising the average classification error associated with the unconstrained classifier.

During implementation, the training set $\mathcal D$ is randomly split into  two equal-sized subsets, $\mathcal D_1 \cup \mathcal D_2$. One subset is used to estimate $\widehat \eta_a$ and $\widehat \pi_{a,y}$, while the other is used to estimate the threshold $\widehat \tau$.
Let $\widehat f_1$ denote the classifier estimated using $\mathcal D_1$ for model estimation and $\mathcal D_2$ for threshold calibration, and let~$\widehat f_2$ denote the classifier constructed with the roles of $\mathcal D_1$  and $\mathcal D_2$ reversed.
To mitigate the randomness caused by random splitting, we adopt a cross-fitting approach and define the final probabilistic classifier as the average $\widehat f = (\widehat f_1 + \widehat f_2)/2$. 

For comparisons, we include the classical functional linear discriminant classifier (FLDA). All methods are evaluated on an independent test set of size 5000. We repeat the experiments 500 times. For each fitted classifier $f$, we compute the unfairness metric $\mathcal{U}_{D}(f) = |D(f)|$ and report both its empirical median $\mathcal{U}_{D, 50}$ and empirical $95\%$ quantile $\mathcal{U}_{D, 95}$, along with the median classification error.

For conciseness, we present the key results in the 1st-3rd columns in Figure \ref{fig:results_main}, while additional results regarding effects of sample sizes ($n$), alignment of mean difference ($\beta$), model misspecification (non-Gaussianity) and perfect classification are deferred to \Cref{sec:apx_numerical}.
The classification errors of the proposed classifiers Fair-LDA and Fair-$\mathrm{FLDA_c}$ exhibit a non-increasing trend as $\delta$ grows, mirroring the behaviour of the oracle Bayes classifier. As expected, the error of FLDA remains constant across different values of $\delta$, as it does not incorporate any fairness constraint. When $\delta$ is small, corresponding to fair-impacted regimes, the classification errors of our methods decrease with increasing $\delta$, due to the less stringent fairness constraints.  Once $\delta$ exceeds a certain threshold, the fairness constraint becomes inactive, and the classification errors of the fairness-aware classifiers converge to that of the unconstrained FLDA.  

In terms of disparity control, the FLDA consistently fails to meet fairness requirements. In contrast, the Fair-FLDA maintains the desired median disparity level, while Fair-$\mathrm{FLDA_c}$ typically achieves median disparity levels strictly below $\delta$. Moreover, as established in Theorem \ref{thm_fair_guarantee}, the empirical 95\% quantile of disparity under Fair-FLDA may slightly exceed $\delta$, whereas the Fair-$\mathrm{FLDA_c}$ effectively corrects for this offset, achieving probabilistic control of the disparity below $\delta$ with probability at least 95\%. As $\delta$ increases beyond a critical threshold, both fairness-aware classifiers gradually reduce to the unconstrained classifier FLDA, and their corresponding disparity levels stabilize accordingly. 
Notably, both Fair-FLDA and Fair-$\mathrm{FLDA_c}$ satisfy their respective fairness criteria without significant compromise in classification accuracy.

\begin{figure}[!htbp]
    \centering
    {\scriptsize
    \hspace{0.01\textwidth}
    \makebox[0.46\textwidth][c]{\textbf{Simulated data}}
    \hspace{0.05\textwidth}%
    \makebox[0.46\textwidth][c]{\textbf{Real data}}\\[0.5ex]
    }
    \includegraphics[width=0.156\textwidth]{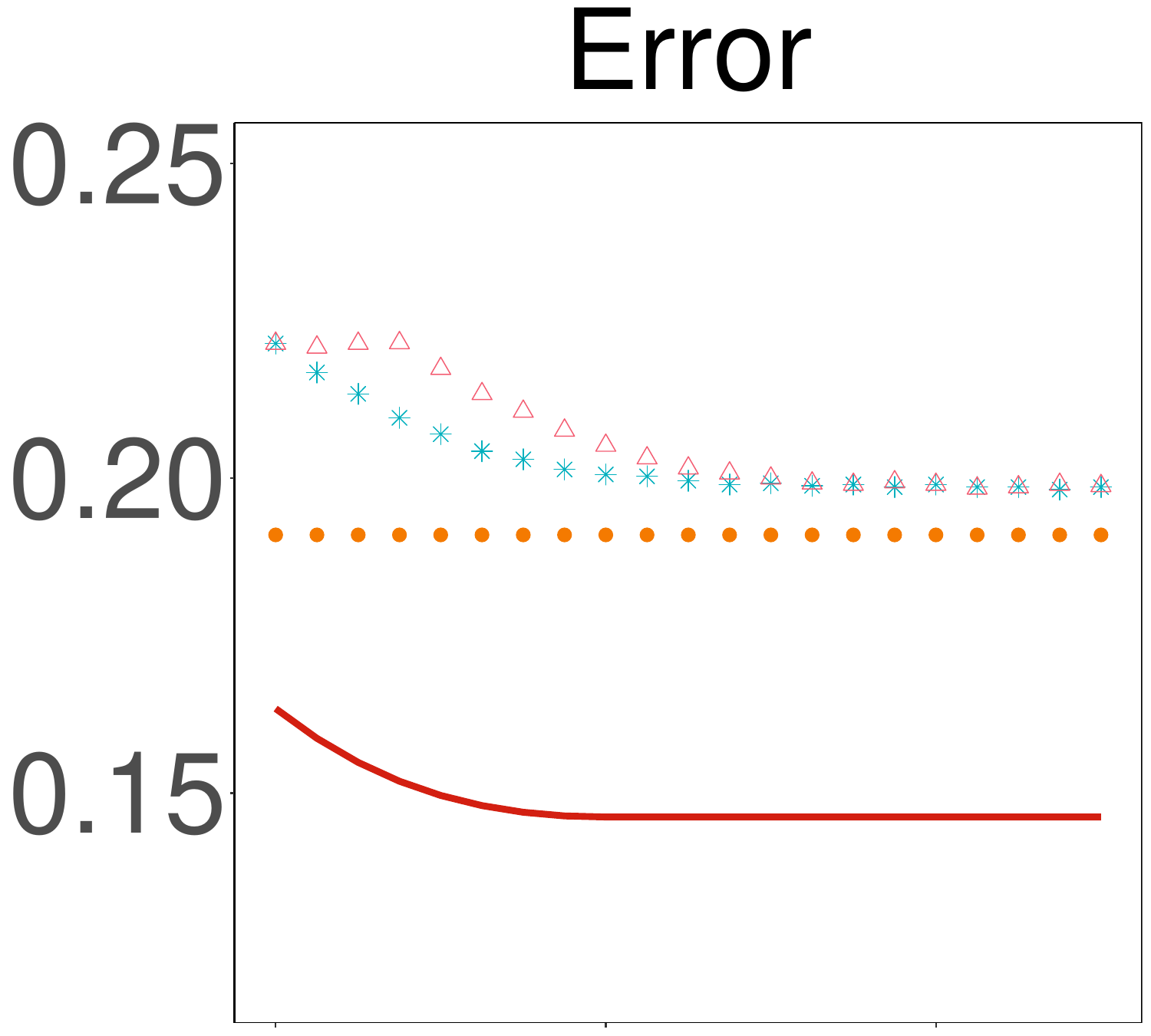}
	\includegraphics[width=0.156\textwidth]{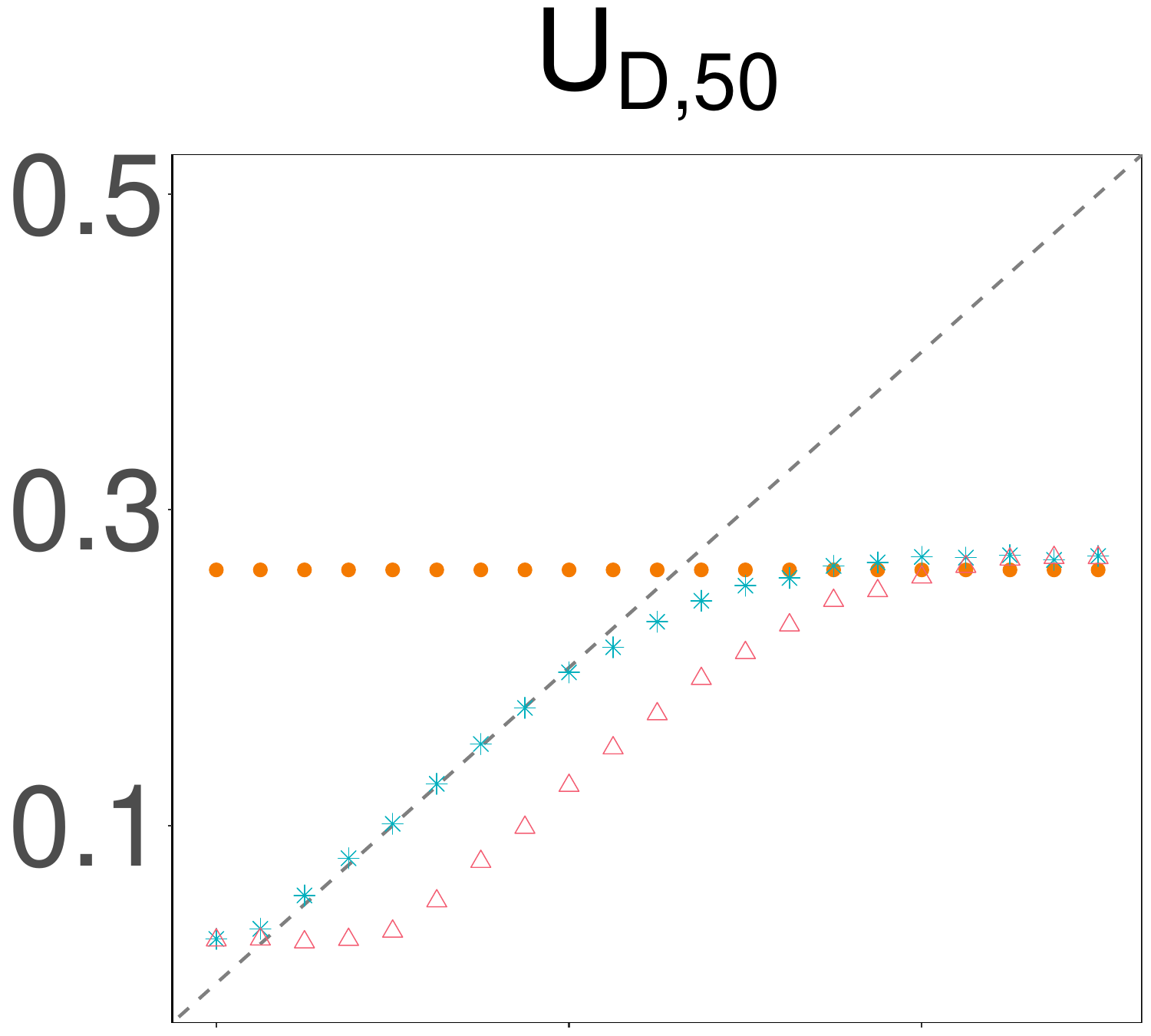}
	\includegraphics[width=0.156\textwidth]{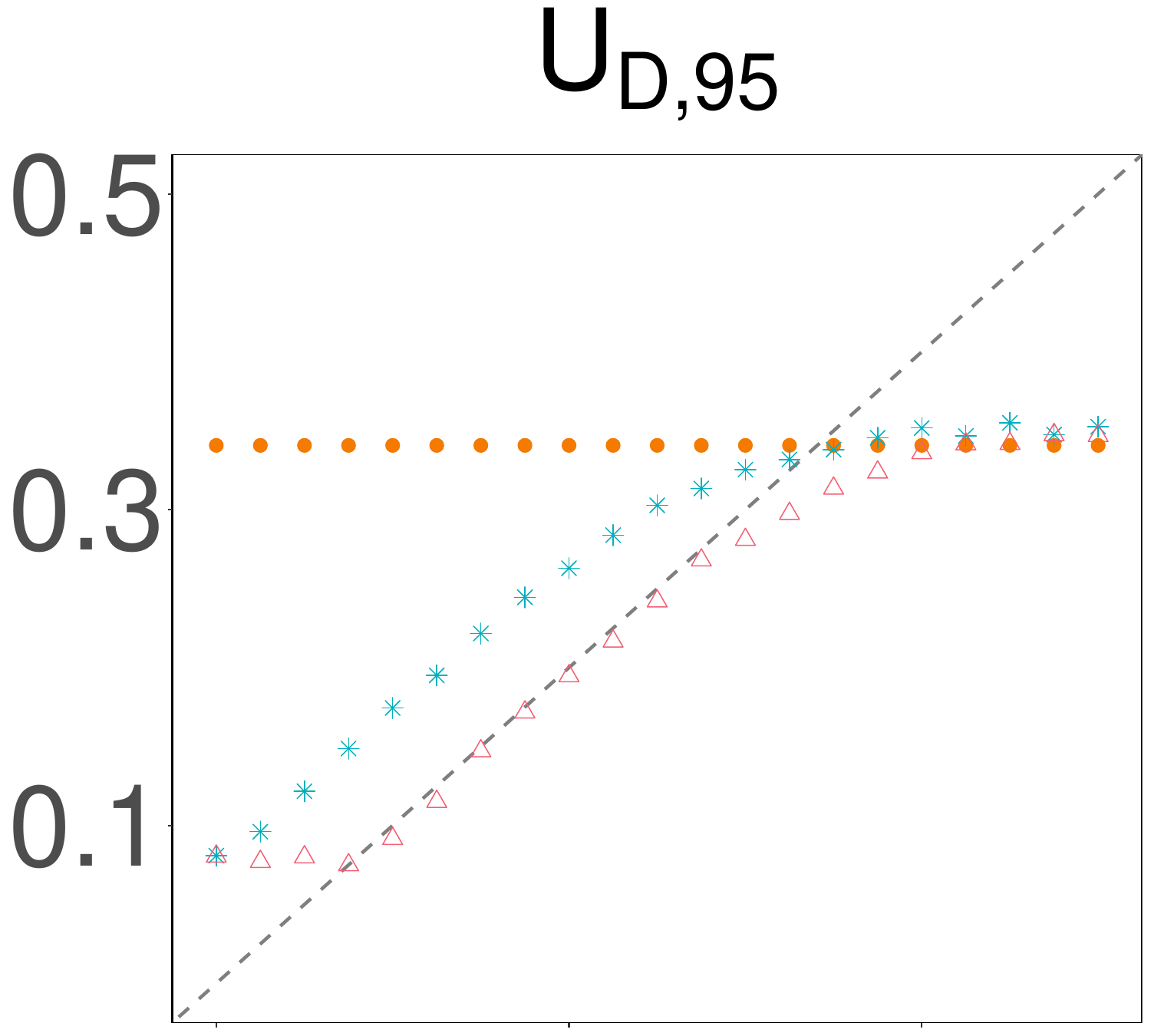} 
    \hspace{0.02\textwidth}
    \includegraphics[width=0.156\textwidth]{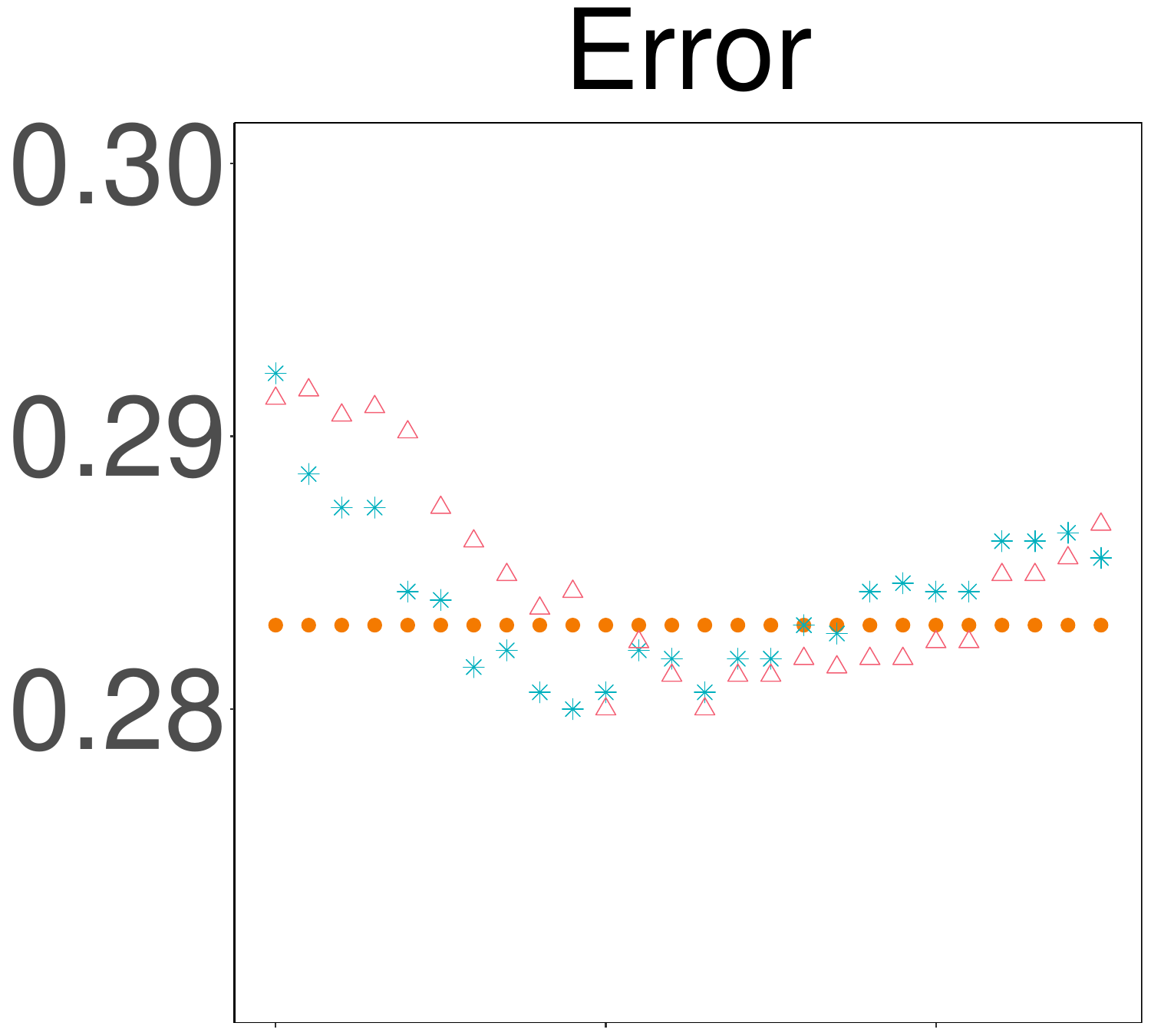}
	\includegraphics[width=0.156\textwidth]{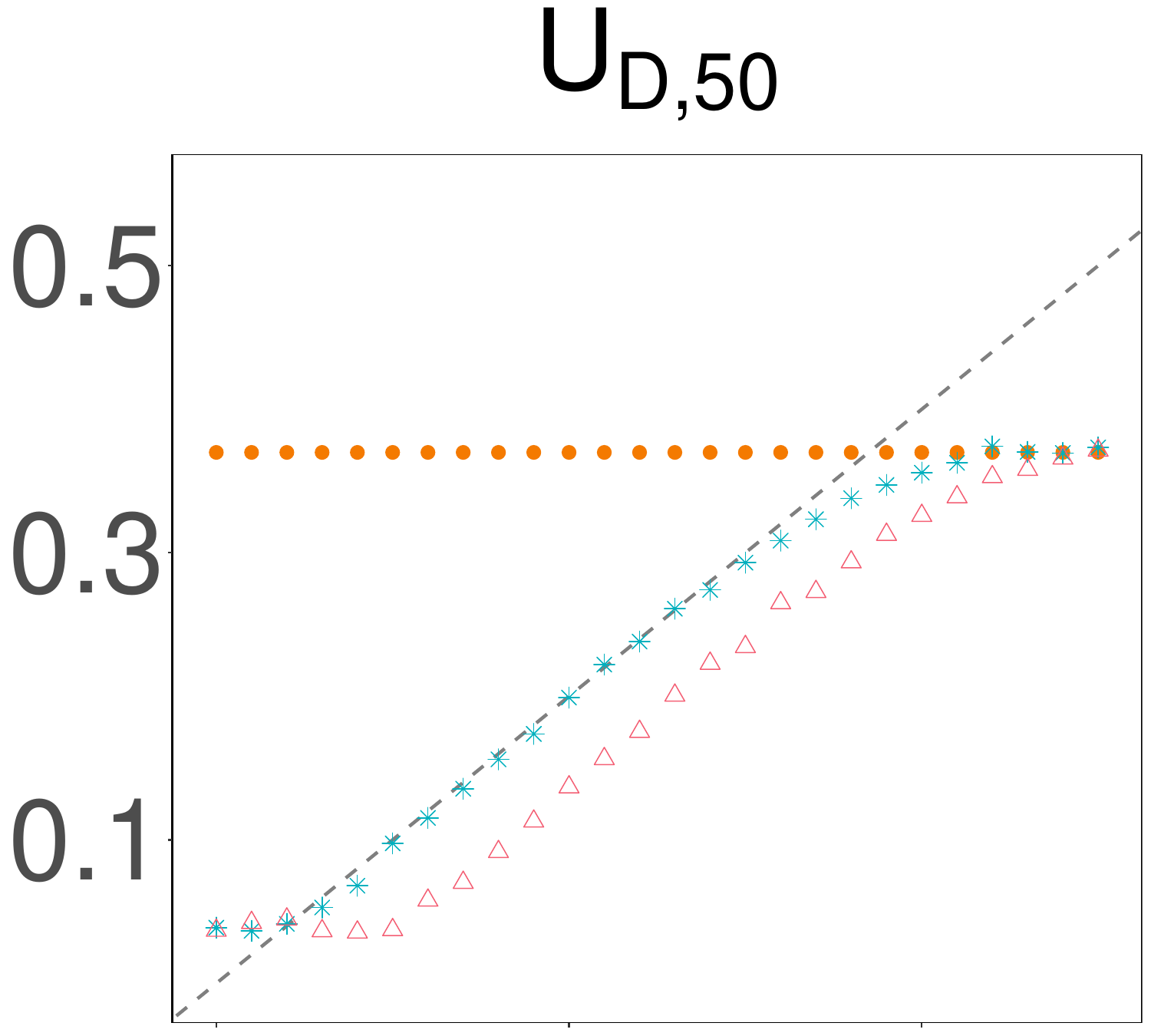} 
	\includegraphics[width=0.156\textwidth]{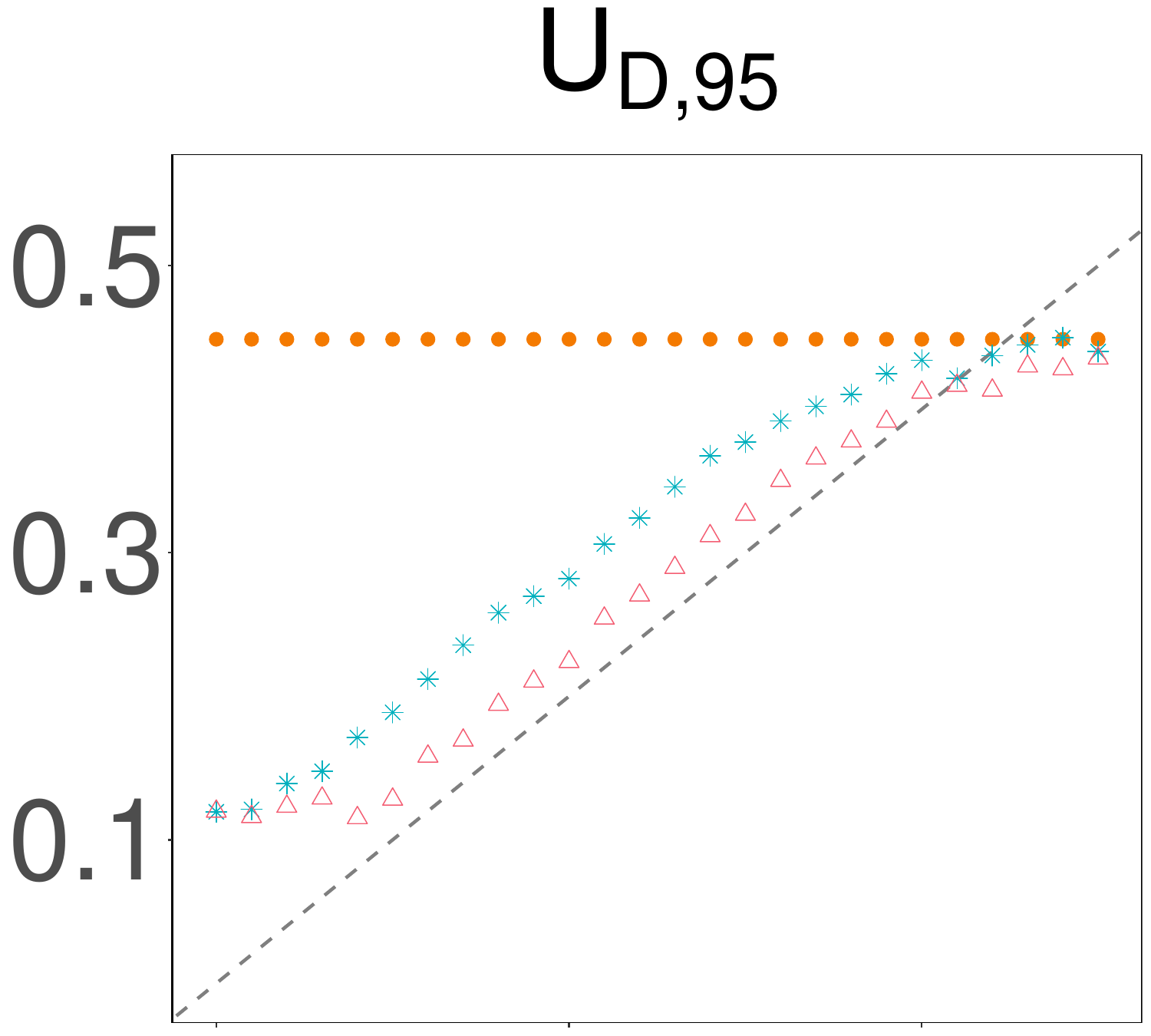} 

    \includegraphics[width=0.156\textwidth]{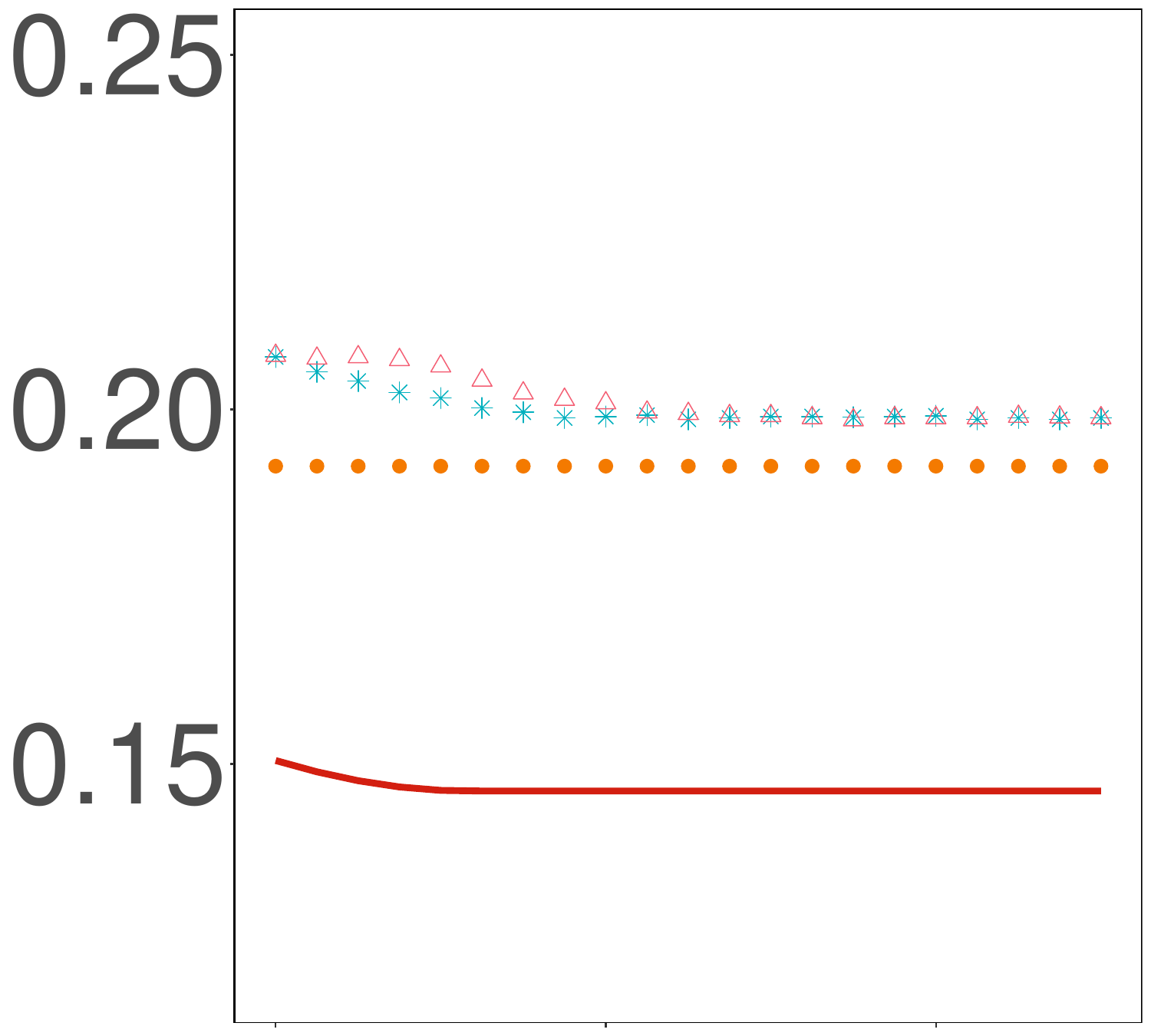}
	\includegraphics[width=0.156\textwidth]{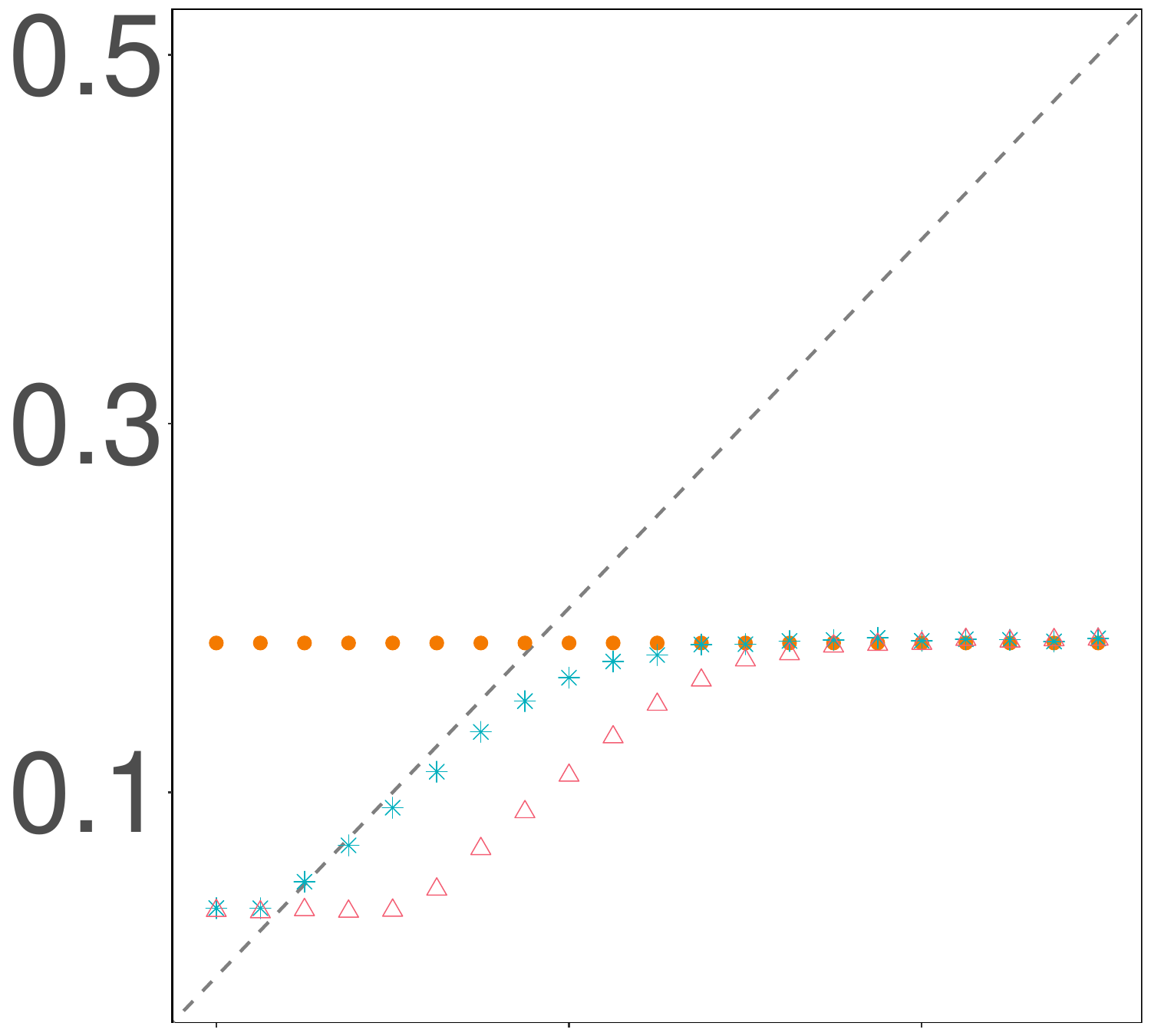}
	\includegraphics[width=0.156\textwidth]{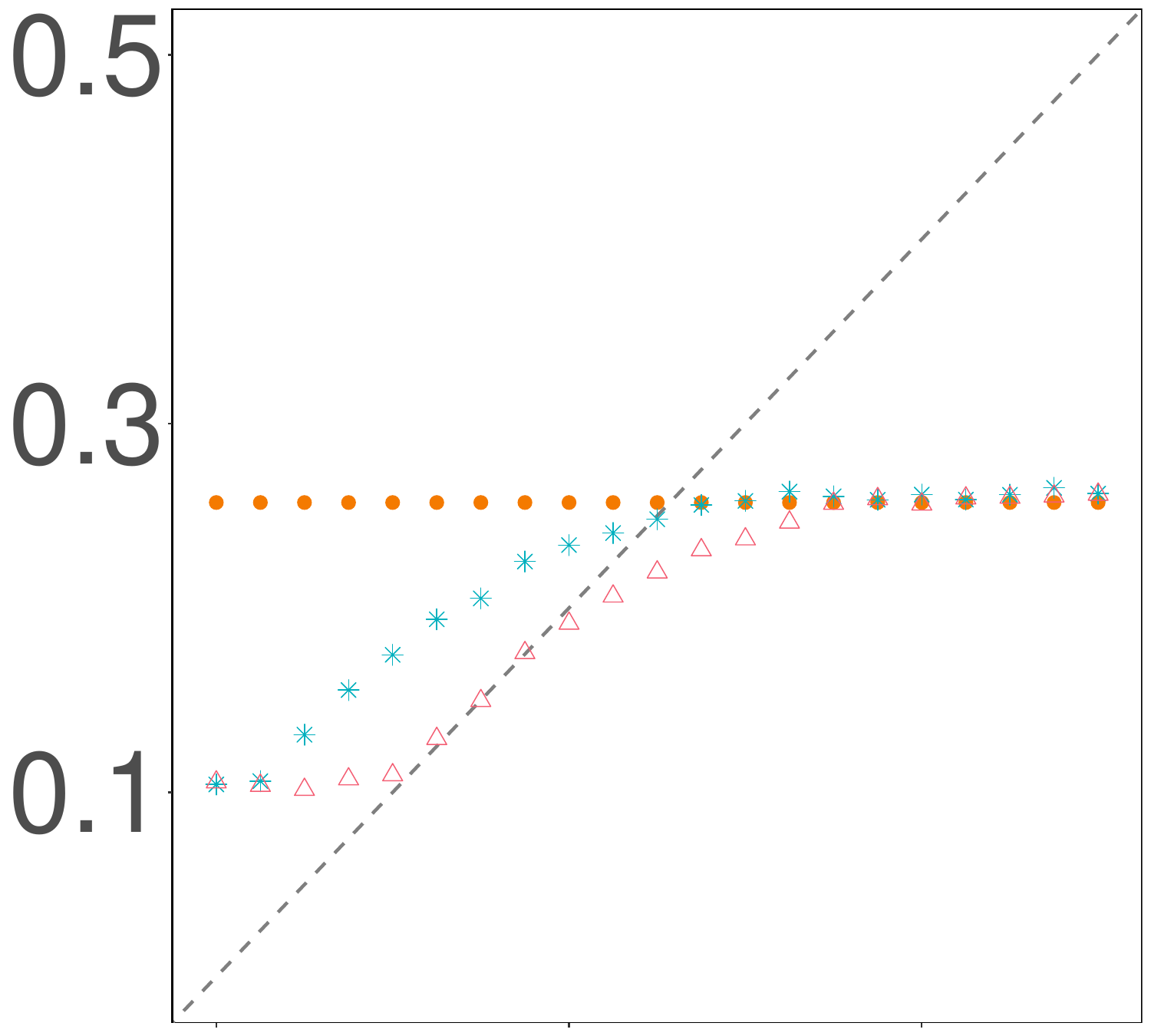}  
    \hspace{0.02\textwidth}
    \includegraphics[width=0.156\textwidth]{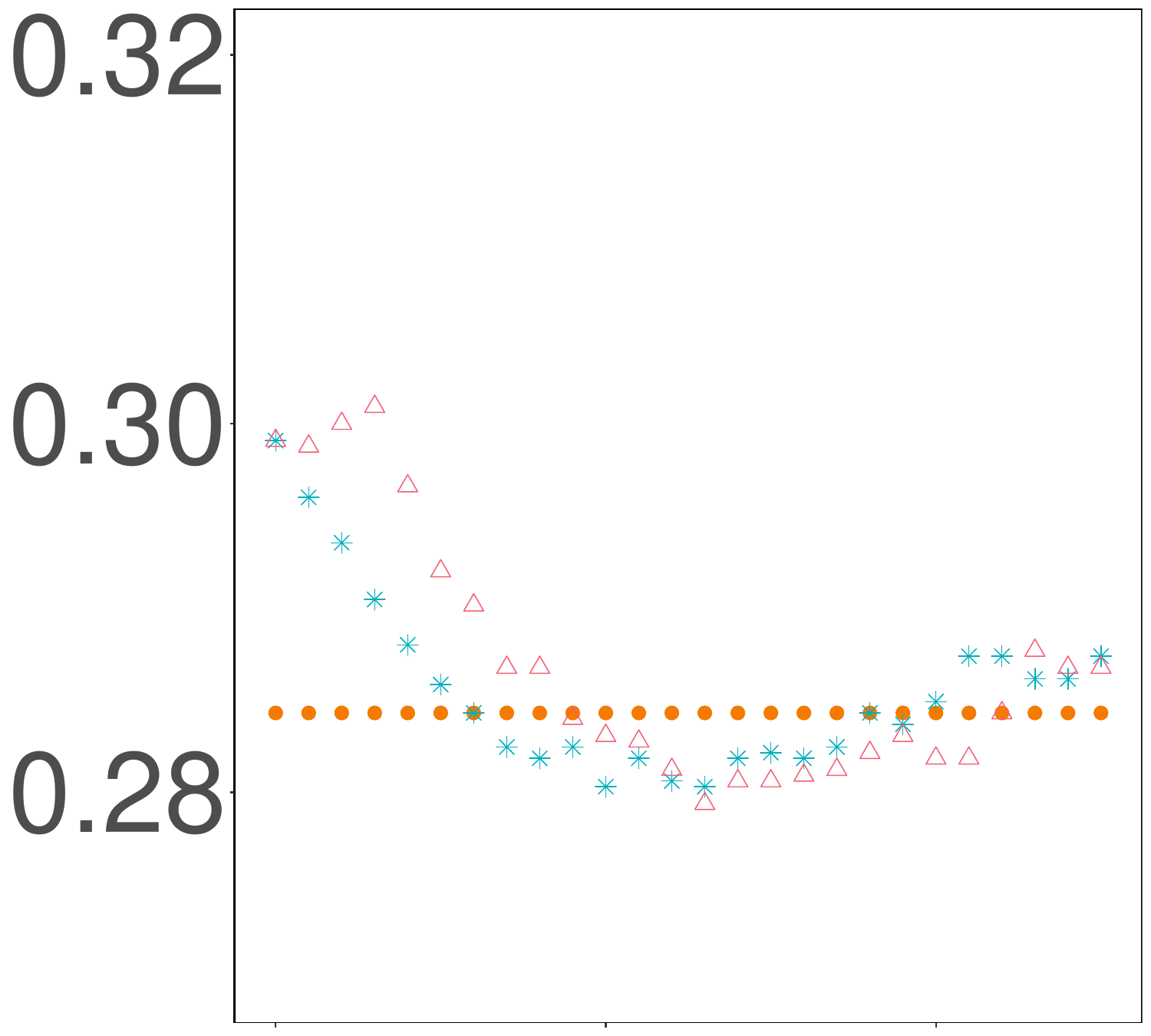}
	\includegraphics[width=0.156\textwidth]{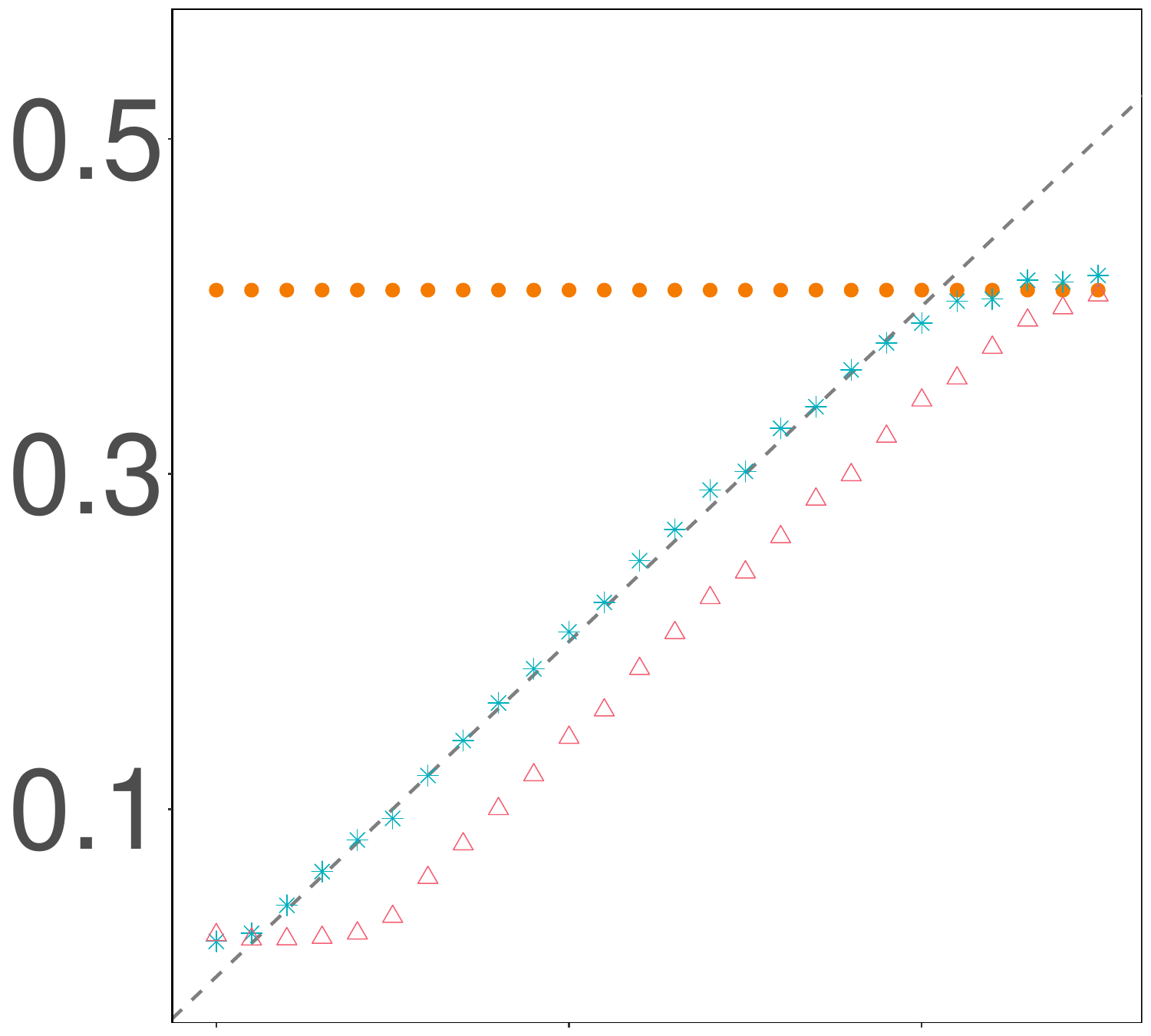} 
	\includegraphics[width=0.156\textwidth]{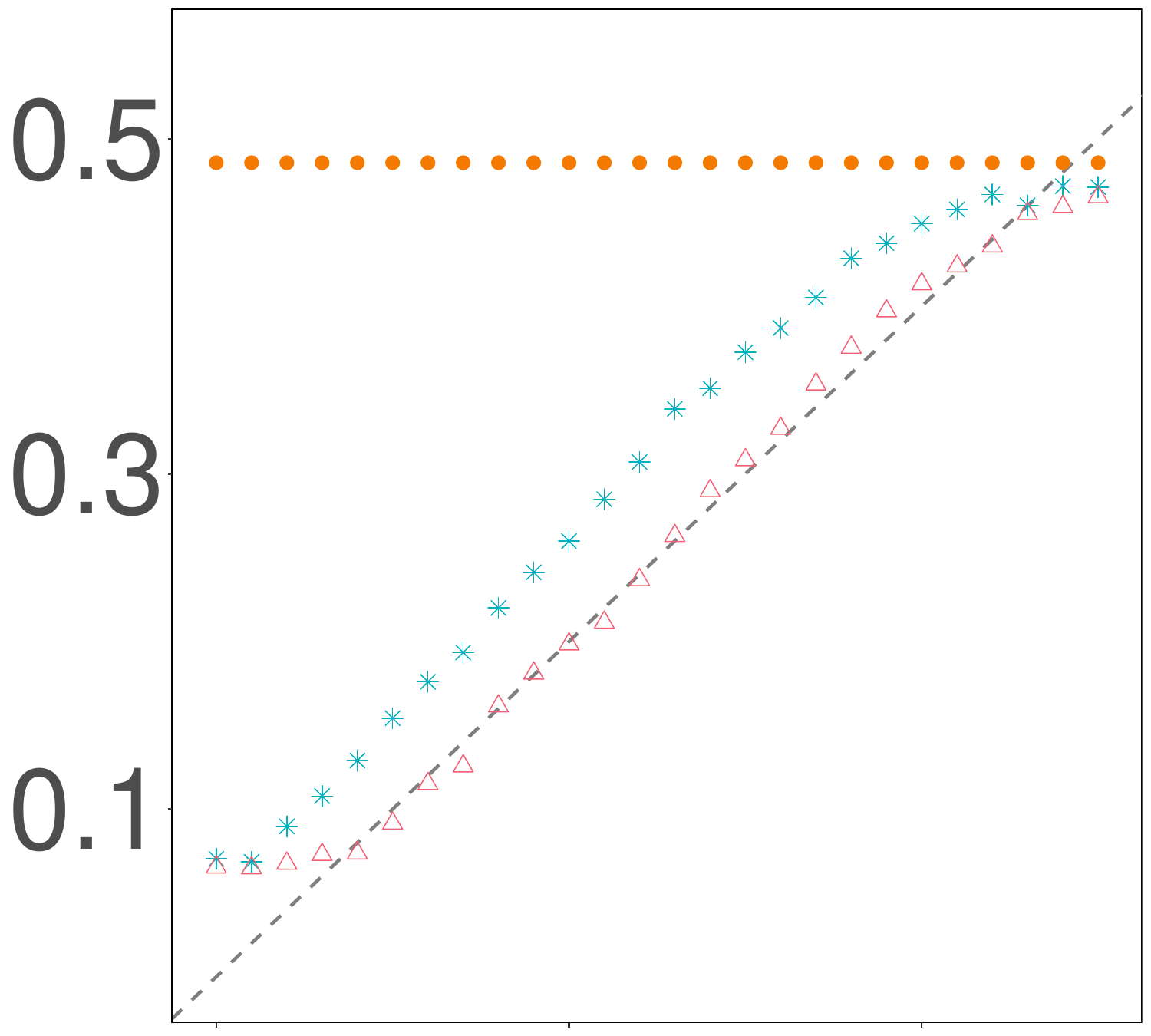} 
    
    \includegraphics[width=0.156\textwidth]{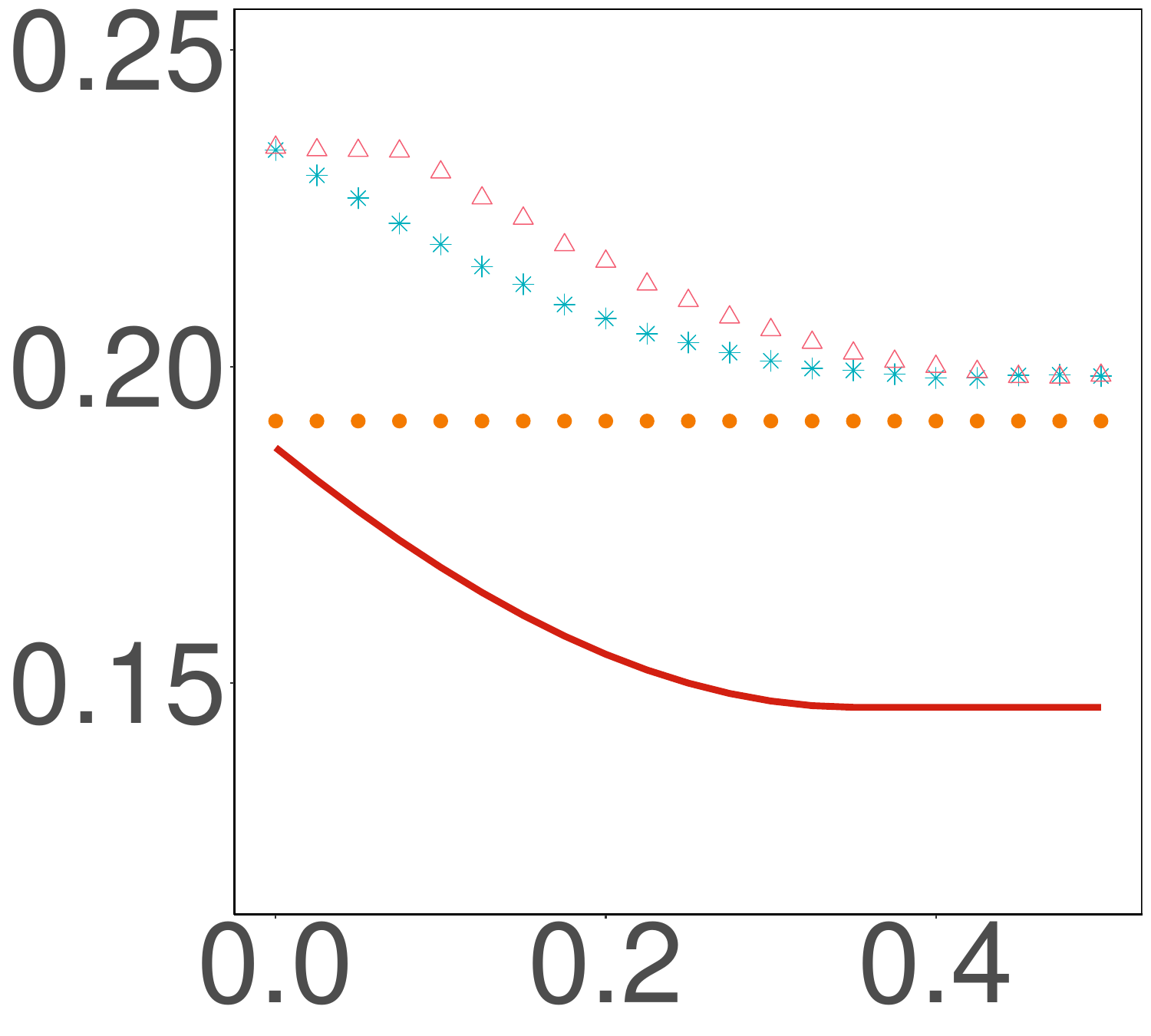}
	\includegraphics[width=0.156\textwidth]{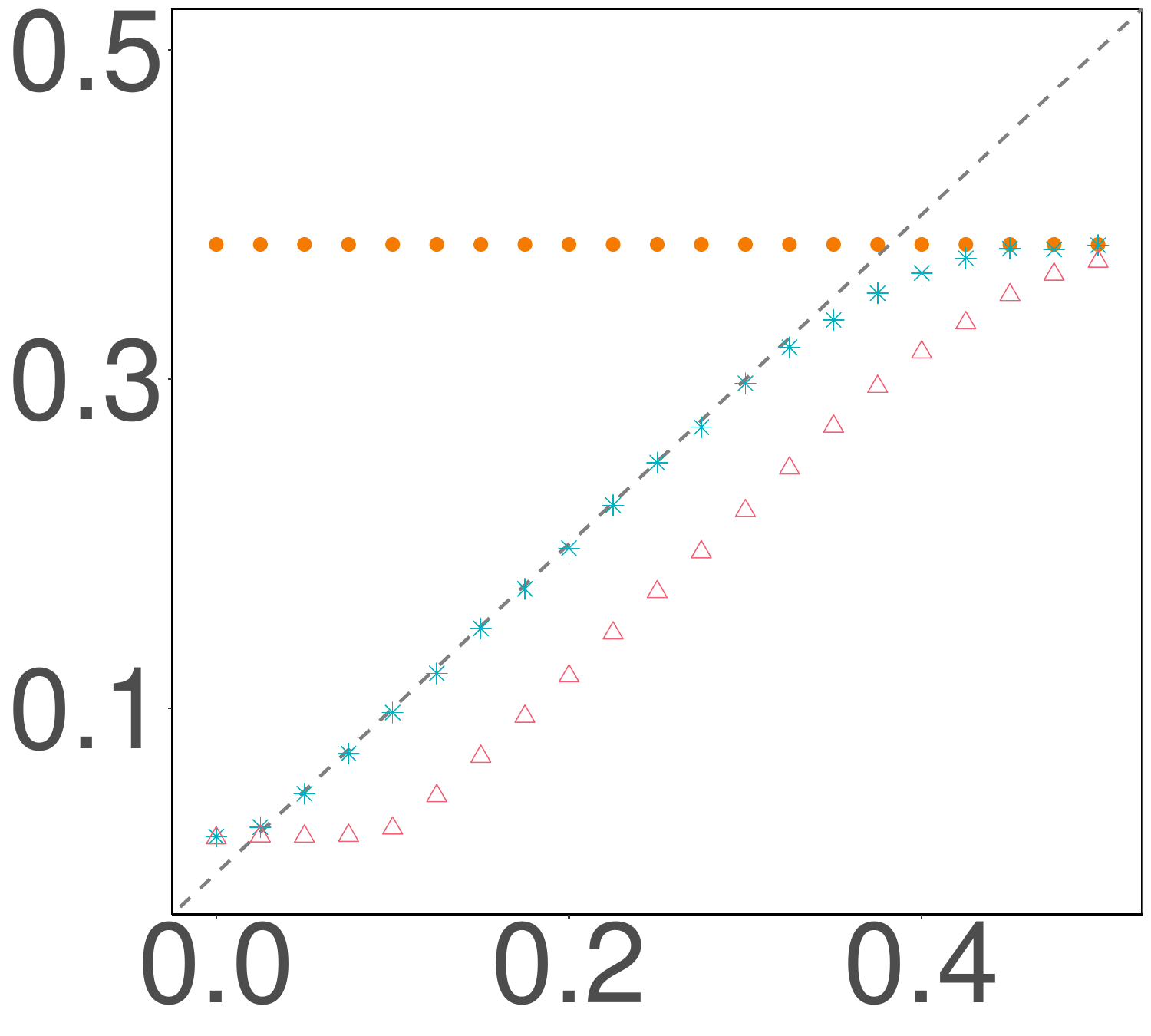}
	\includegraphics[width=0.156\textwidth]{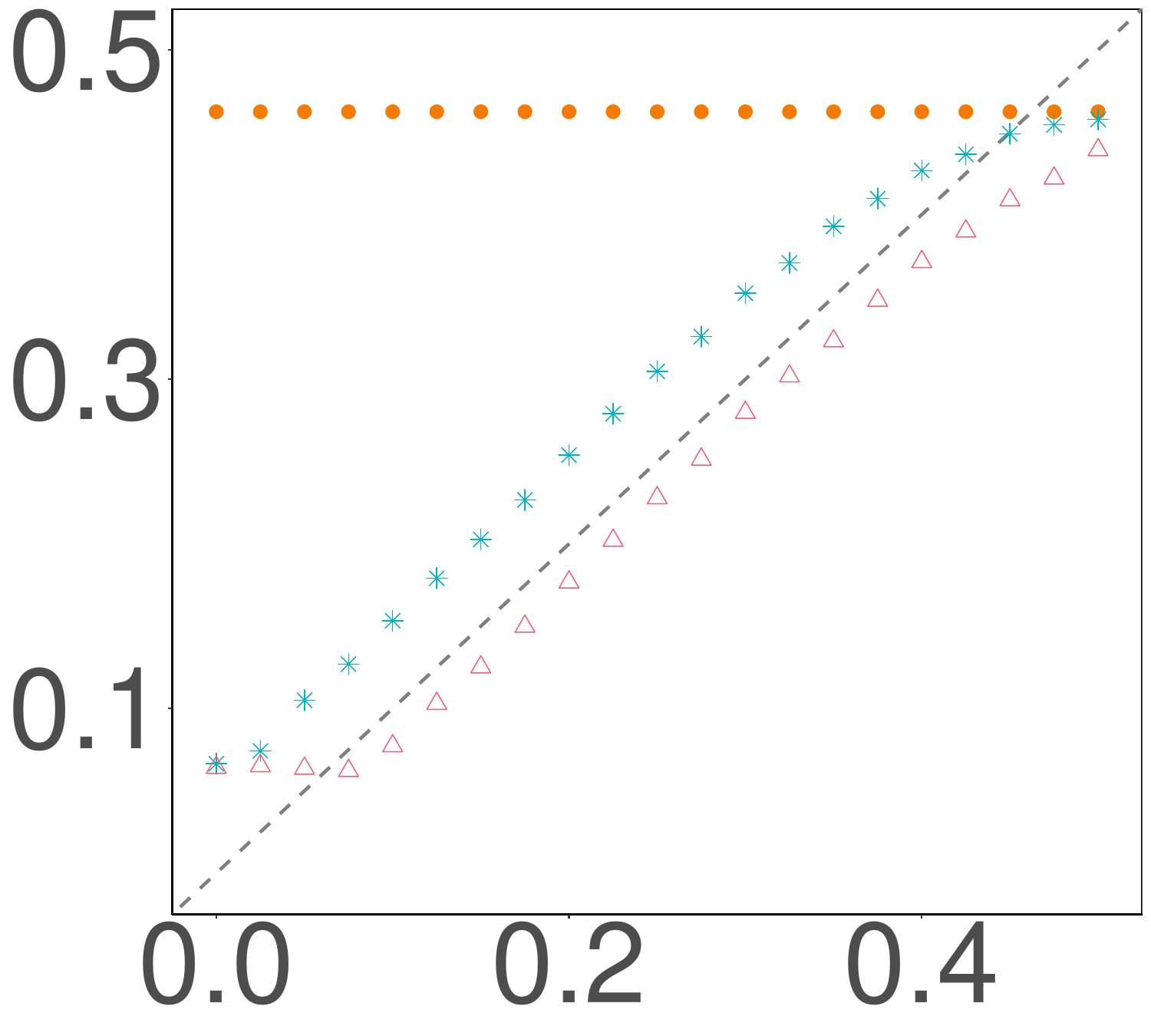}  
    \hspace{0.02\textwidth}
    \includegraphics[width=0.156\textwidth]{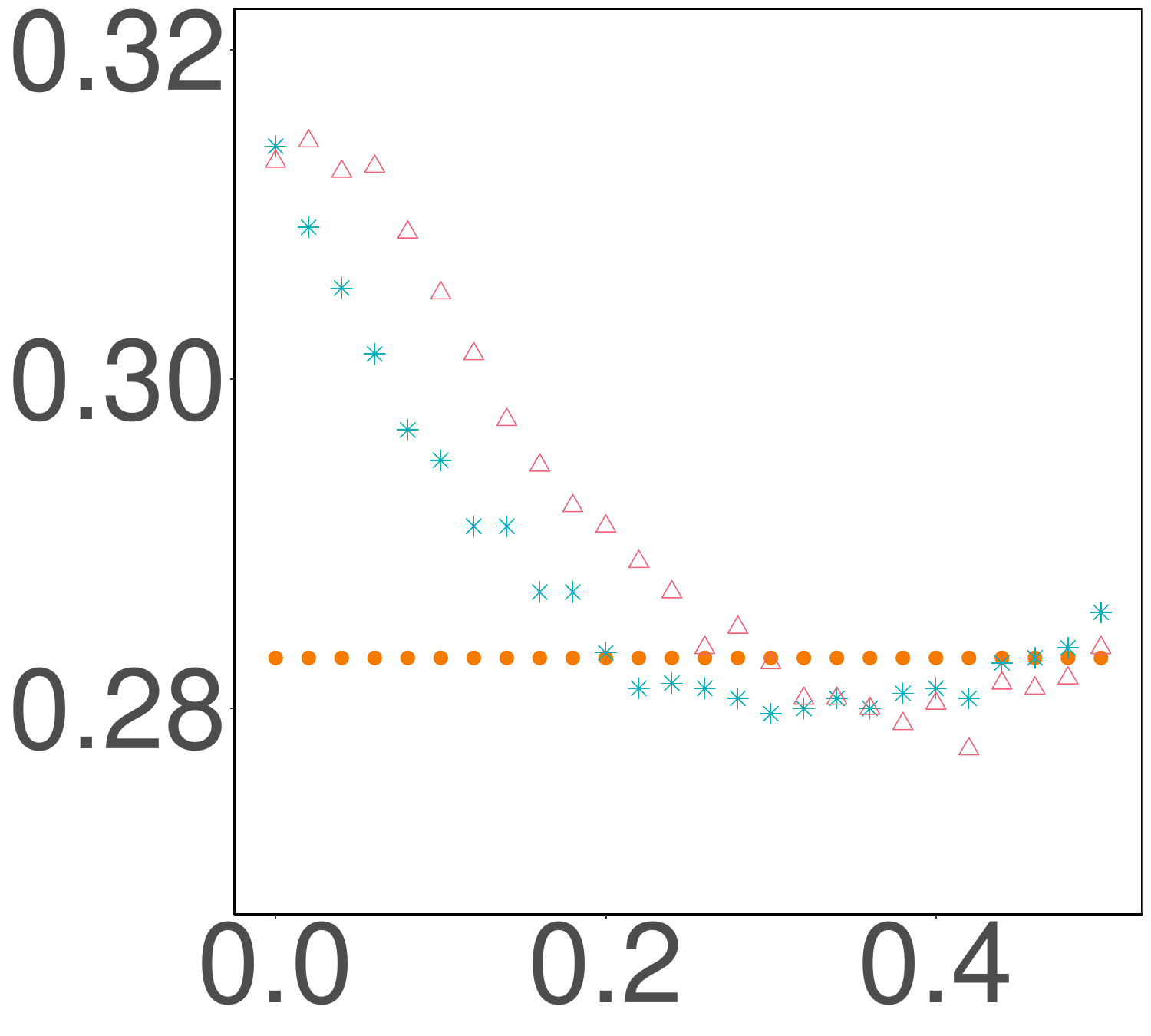}
	\includegraphics[width=0.156\textwidth]{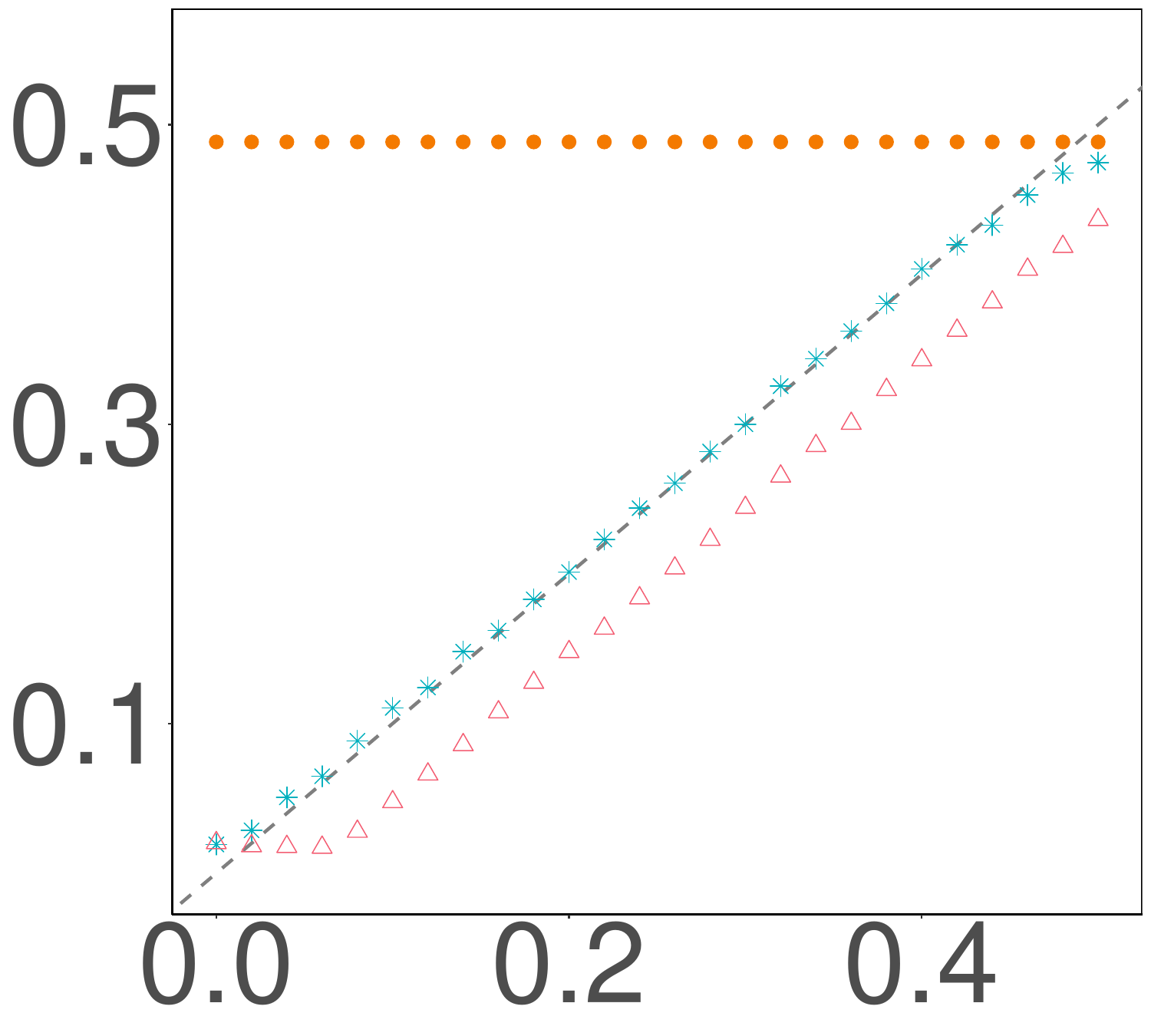} 
	\includegraphics[width=0.156\textwidth]{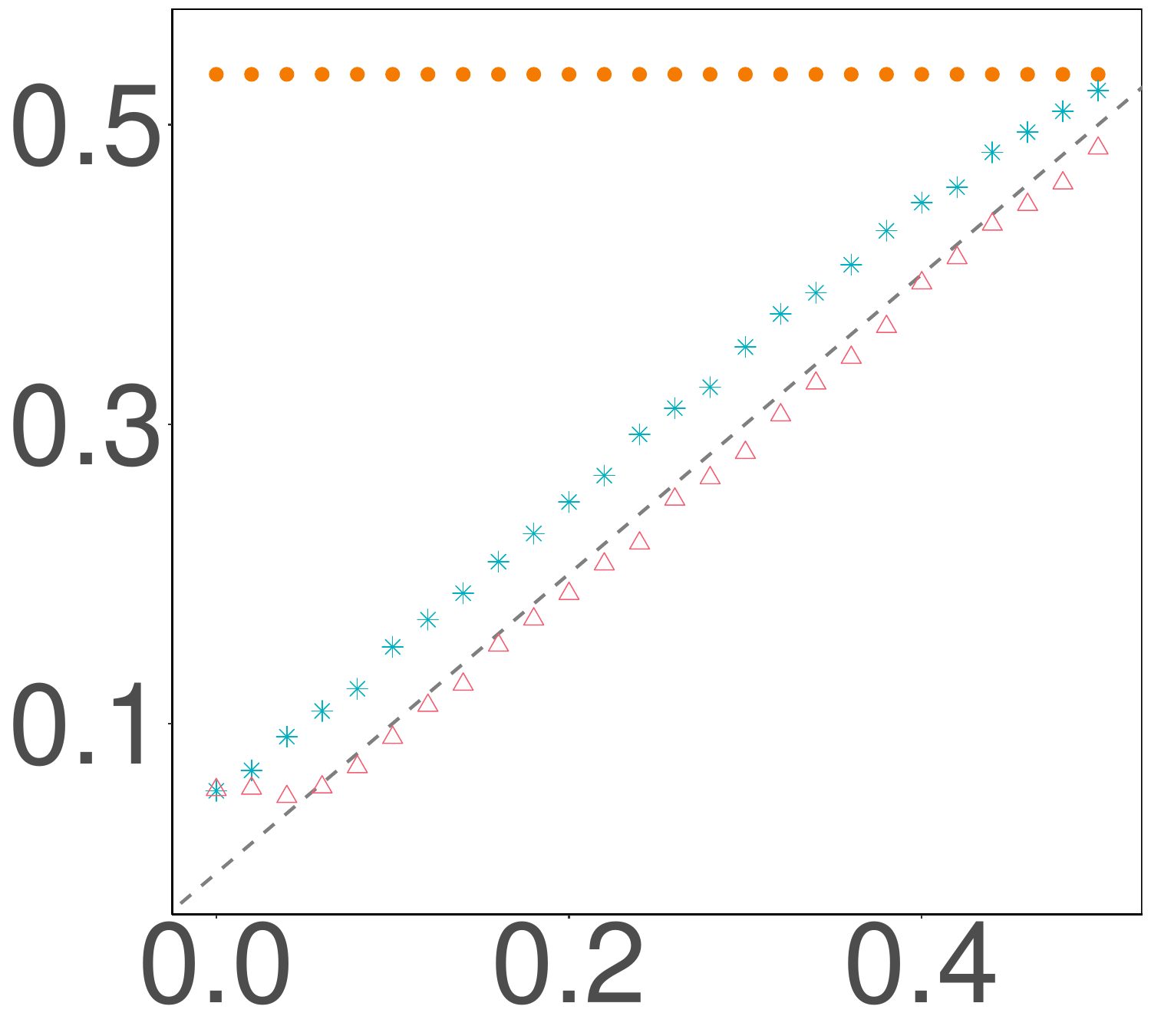} 
    \caption{All $x$-axis are the values of $\delta$. From left to right: medians of classification errors, medians and 95\% quantiles of the disparity measures, in the simulations (1st-3rd columns) and real data (4-6th columns). From top to bottom: DD, DO and PD. Orange dots: FLDA; blue stars: Fair-FLDA; pink triangles: Fair-$\mathrm{FLDA_c}$; red solid line: oracle Bayes classifier; grey dashed line: $y=x$.}\label{fig:results_main}
\end{figure}

\subsection{Real data analysis}\label{sec:real-data}

We apply the proposed method to data from the 2005-2006 National Health and Nutrition Examination Survey \citep{NHANES}. Further details about this dataset can be found in  \citet{lin2023causal}. Following the preprocessing steps in \citet{lin2023causal}, we exclude observations with questionable data reliability according to NHANES protocol, remove observations with intensity values higher than 1000 or equal to 0, and retain subjects with at least 100 remaining observations. 
The response variable is whether an individual is under 20 or over 50 years old, with the quantile function of intensity values as the functional covariate. The sensitive attribute refers to race, categorised as non-Hispanic white and non-Hispanic black. The final dataset consists of 3252 instances, which we randomly split into equal-sized training and test subsets.  

The methods are implemented following the same procedures described in Section \ref{sec:sim}.
As shown in the 4th-6th columns in Figure \ref{fig:results_main}, the classical classifier FLDA exhibits substantial unfairness, whereas the Fair-FLDA effectively controls the median disparity. In terms of probabilistic disparity control, the 95\% empirical quantile of the disparity for Fair-$\mathrm{FLDA_c}$ slightly exceeds $\delta$ with the default choice of the calibration constant under DO. In complex real world scenarios, we recommend tuning the calibration parameter $\kappa$ to achieve more reliable probabilistic disparity control. Details about the tuning strategy and corresponding results are deferred to \Cref{sec:apx_realdata}.
Importantly, the classification errors of both Fair-FLDA and Fair-$\mathrm{FLDA_c}$ remain comparable to that of the unconstrained classifier FLDA. These results demonstrate that our fairness-aware classifiers effectively mitigate unfairness while maintaining competitive classification accuracy.

\section{Conclusion} \label{section_conclusion}
In this paper, we study fairness-aware Bayes optimal classification for functional data. To the best of our knowledge, this is the first
time seen in the literature. We propose a unified framework for functional classification under fairness constraints and design a post-processing Fair-FLDA algorithm for settings where the functional features are modelled as Gaussian processes. Under appropriate assumptions, theoretical guarantees for both fairness and excess risk controls are provided, which are further supported by extensive numerical experiments on both synthetic and real datasets.

We envisage several potential extensions. Firstly, our framework and the Fair-FLDA algorithm depend on the availability of the sensitive features, which may be restricted in certain practical settings due to privacy concerns. When sensitive features $A$ are available during training but not available during prediction, one possible way to address this challenge is to predict $A$ from the functional features \citep[e.g.~Section 4.3.3 in][]{zeng2024bayes}. However, when $A$ is not available even during the training process, more refined methods for inferring the sensitive features would be necessary. Secondly, in many scenarios, the Radon--Nikodym derivatives $\mathrm{d}P_{a,1}/\mathrm{d}P_{a,0}$ are not explicitly known and easy to work with. To address this, a natural strategy during implementation is to approximate it using the density ratios of projection scores \citep[e.g.][]{dai2017optimal}. Finally, in reality, functions can only be discretely observed over sampling grids. Investigating the effect of sparsity on the excess risk under fairness constraints remains an intriguing area for further investigation.

\section*{Acknowledgements}

Yu's research is partially supported by the EPSRC (EP/Z531327/1) and the Leverhulme Trust (Philip Leverhulme Prize). Lin's research is partially supported by the Singapore MOE AcRF Tier 1 grant.

\bibliographystyle{apalike}
\bibliography{fairness}

\newpage
\appendix
\begin{appendices}
All technical details and additional numerical results are collected in the Appendices. Additional experimental results are collected in \Cref{sec:apx_numerical}. We present proofs and properties related to the Bayes optimal classifier $f^\star_{D,\delta}$ in \Cref{appendix_bayes_optimal}. The proof of \Cref{thm_fair_guarantee} is collected in \Cref{appendix_fair_guarantee}, with the proofs of Theorems \ref{thm_fairness_general} and \ref{thm_fairness} presented in \Cref{appendix_excess_risk}. All results related to class probability and eigenspace estimation can be found in Appendices \ref{appendix_empirical_prob} and \ref{appendix_func_estim}. For completeness, necessary technical lemmas are included in \Cref{appendix_technical}.

Throughout the appendix, with a slight abuse of notation, unless specifically stated otherwise, let $c_1, C_1, c_2, C_2, \ldots > 0$ denote absolute constants whose values may vary from place to place.

\section{Additional experimental results}\label{sec:apx_numerical}

The figure labels are consistent with those used in Figure \ref{fig:results_main} in the main text, unless otherwise specified.

\subsection{Additional simulation results under Gaussian models}\label{sec:apx_gaussian}

\paragraph{Results under varying sample sizes.} We evaluate the model from Section \ref{sec:sim} under varying sample sizes. As shown in Figures \ref{fig:DO_gauss_beta_1.5}-\ref{fig:DD_gauss_beta_1.5}, the excess risk of both Fair-FLDA and Fair-$\mathrm{FLDA_c}$ decreases with increasing $n$. Moreover, with larger $n$, the difference between Fair-FLDA and Fair-$\mathrm{FLDA_c}$ in disparity control becomes less significant. 

\begin{figure*}[!htbp]
	\begin{center}
		\newcommand{\thiswidth}{0.2\linewidth}
		\newcommand{\thisgap}{0mm}
		\begin{tabular}{ccc}
			\hspace{\thisgap}\includegraphics[width=\thiswidth]{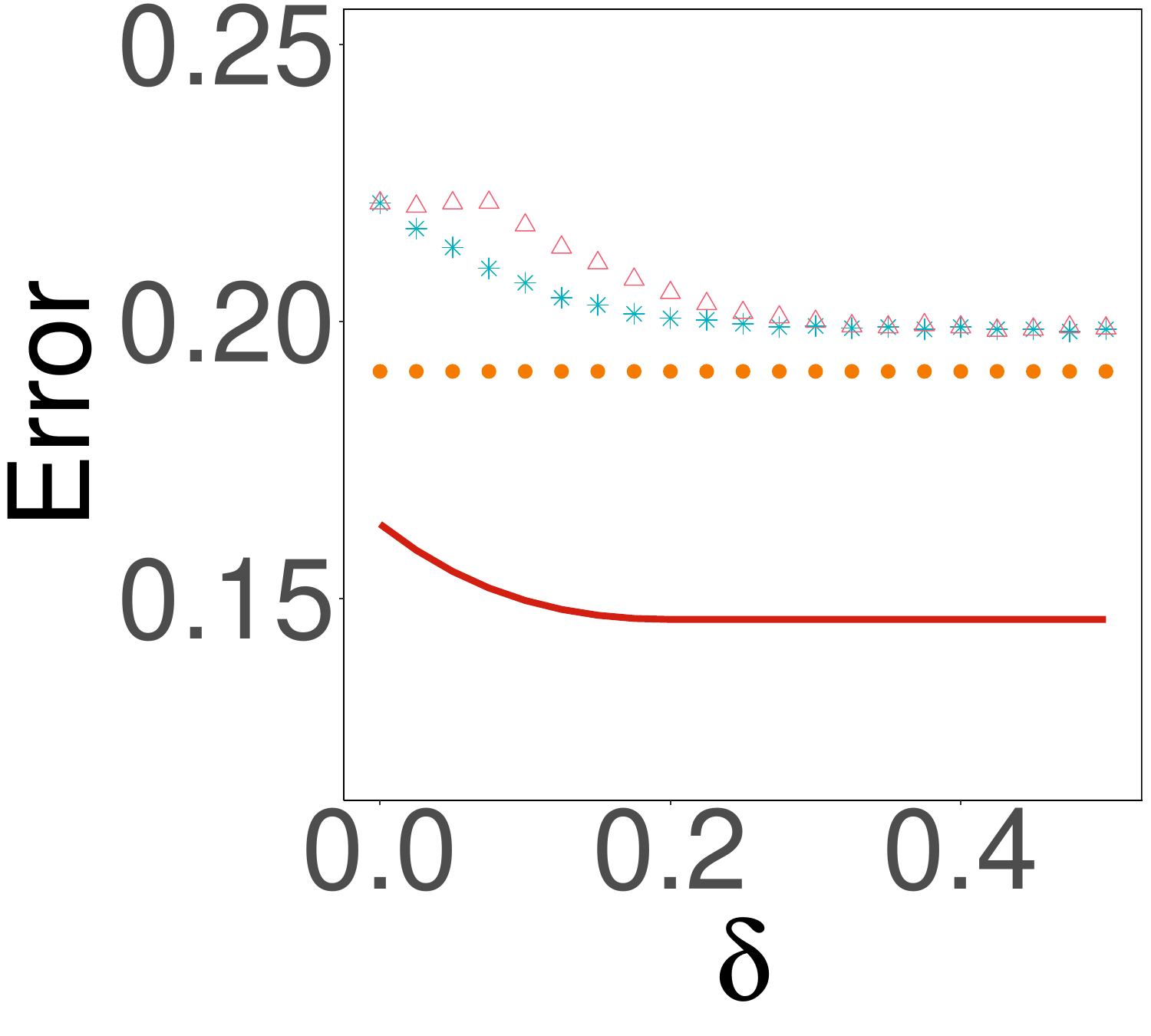} &
			\hspace{\thisgap}\includegraphics[width=\thiswidth]{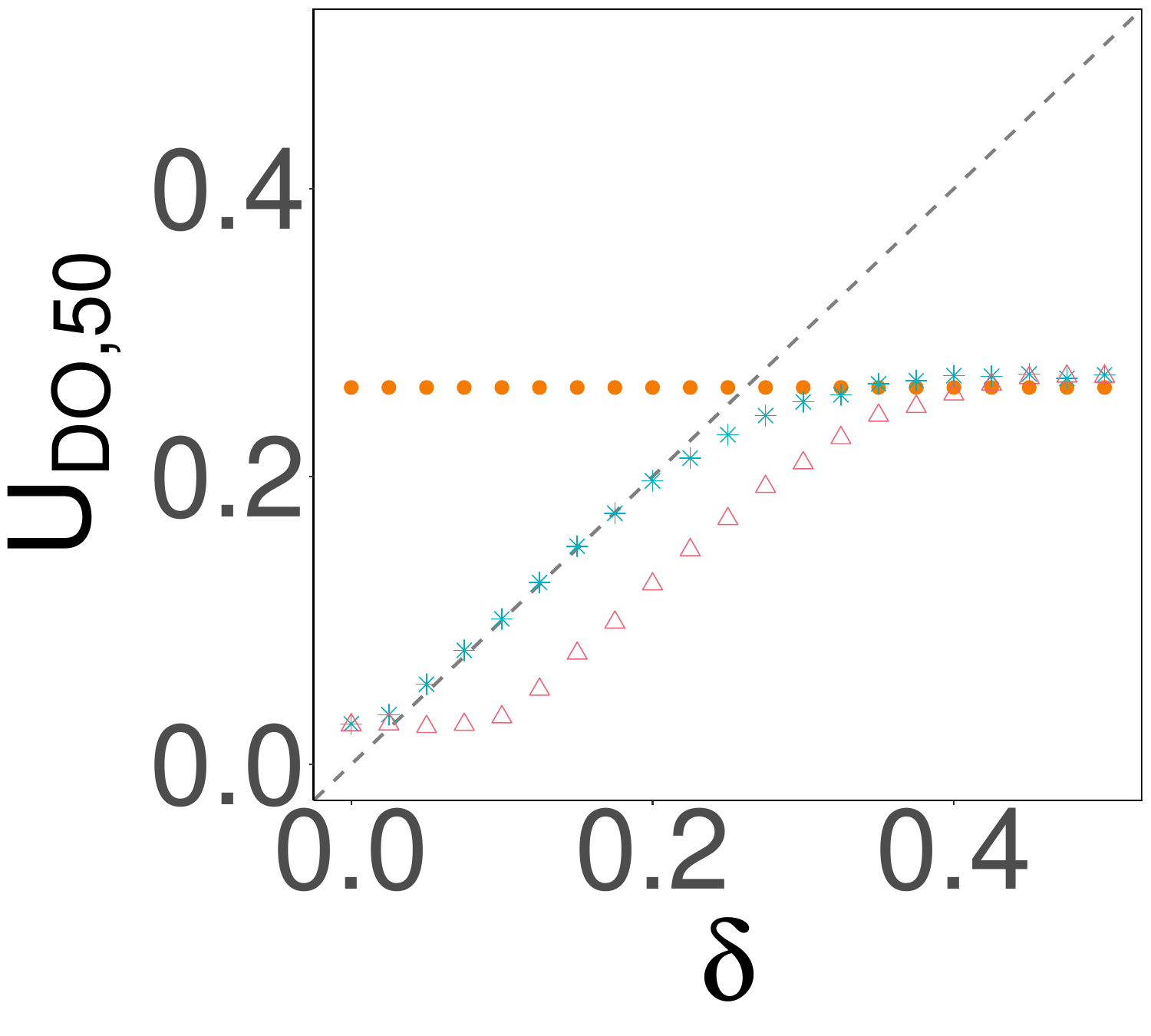} &
			\hspace{\thisgap}\includegraphics[width=\thiswidth]{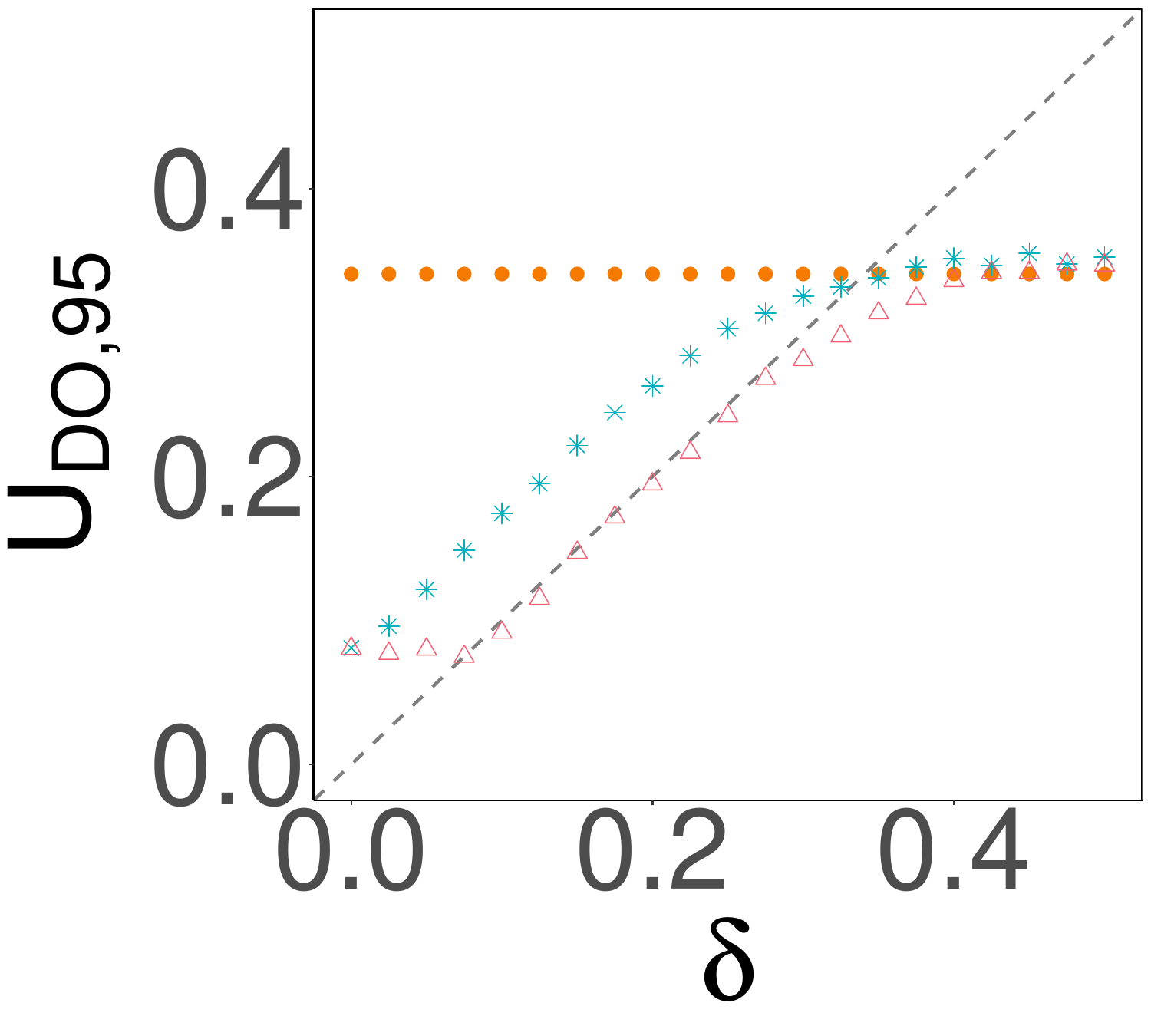} \\
			\hspace{\thisgap}\includegraphics[width=\thiswidth]{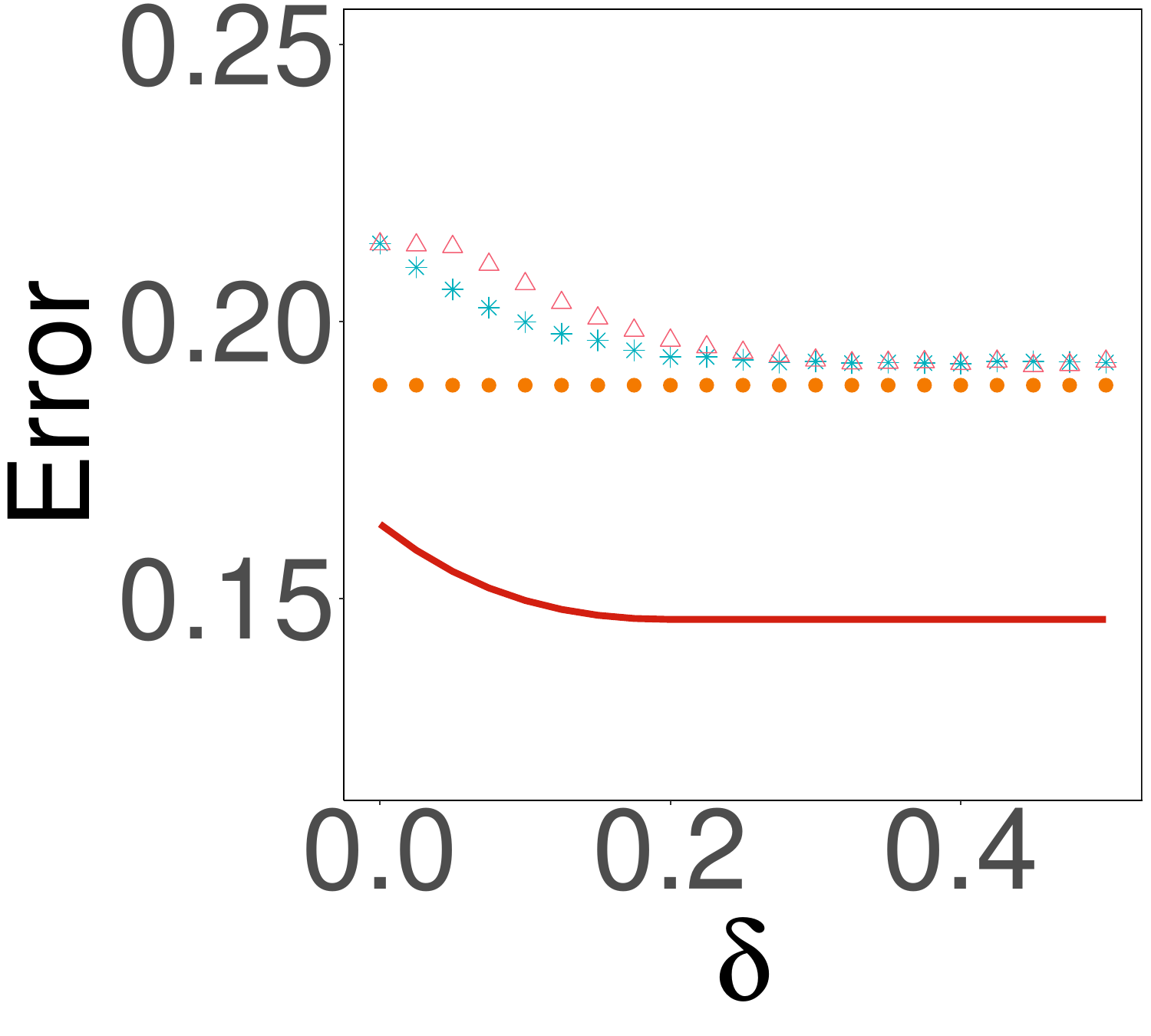} &
			\hspace{\thisgap}\includegraphics[width=\thiswidth]{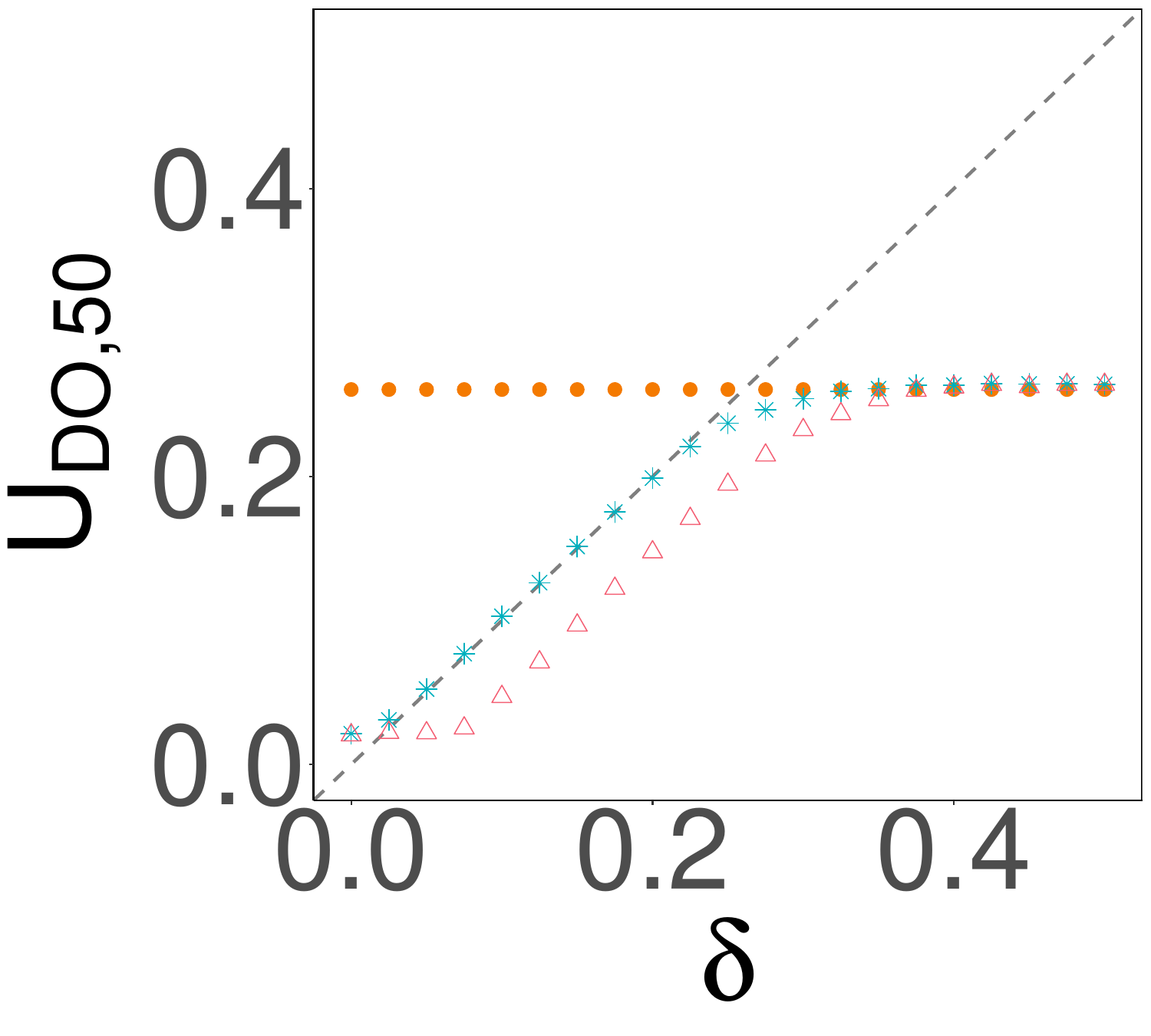} &
			\hspace{\thisgap}\includegraphics[width=\thiswidth]{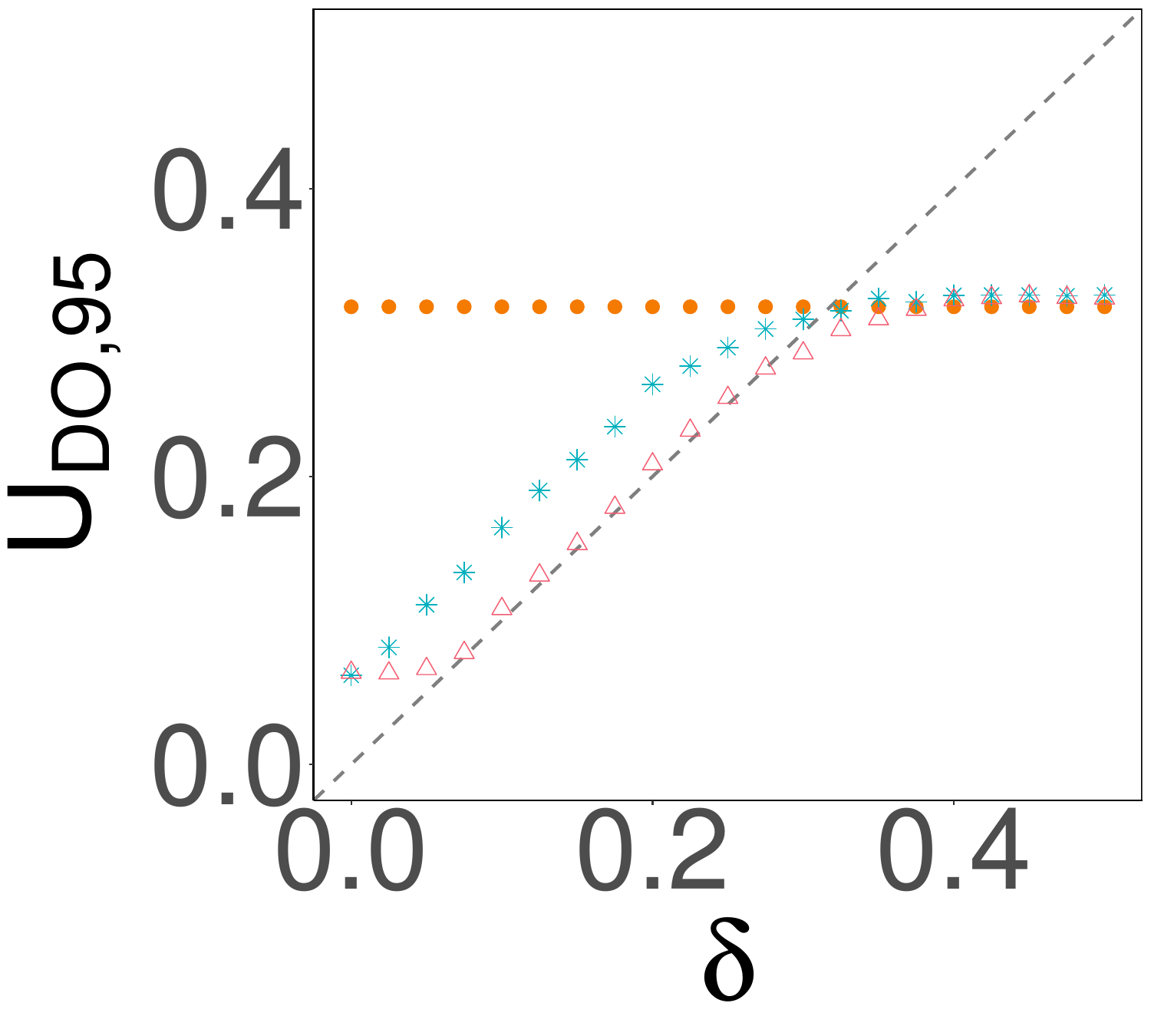} \\
			\hspace{\thisgap}\includegraphics[width=\thiswidth]{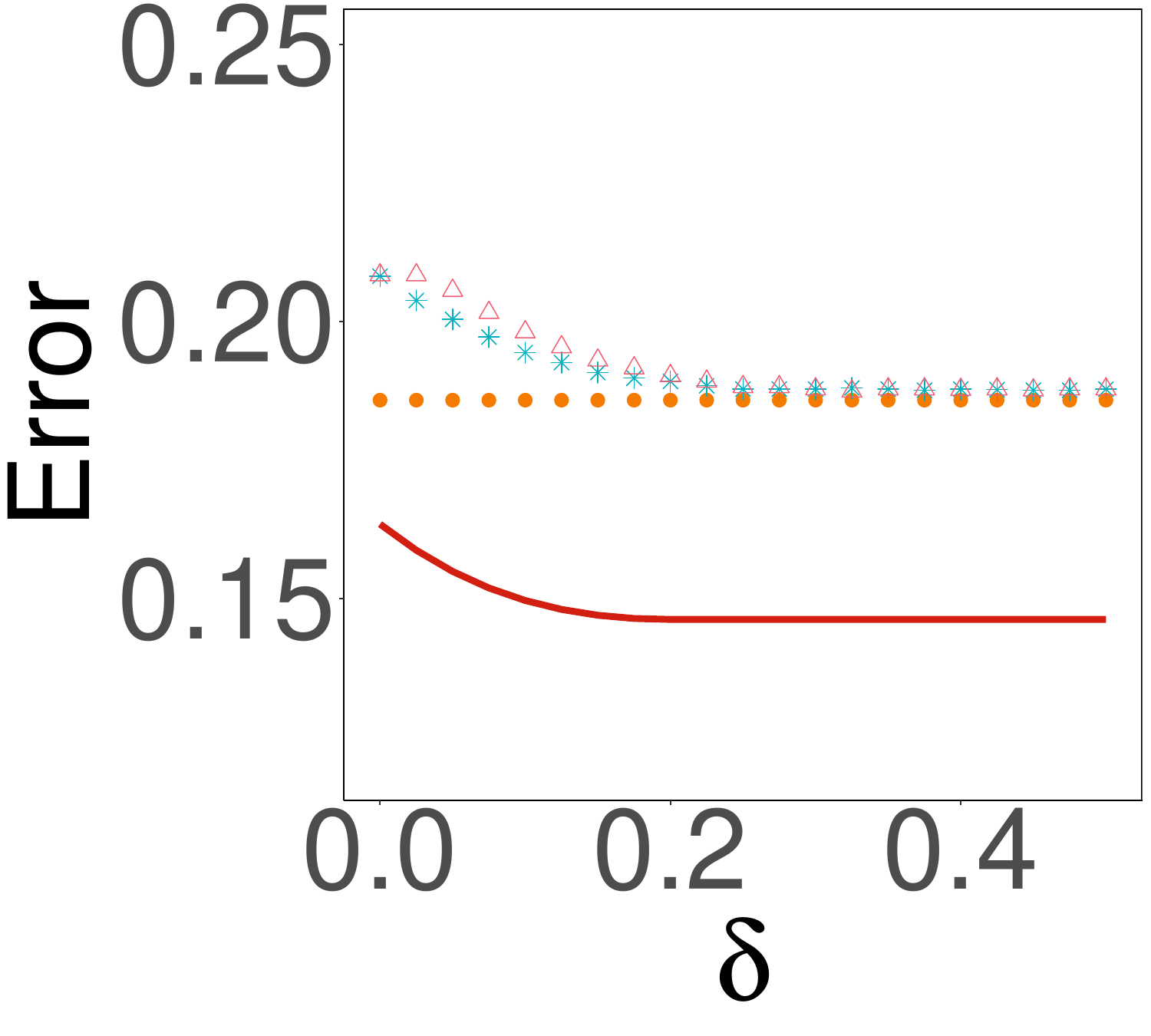} &
			\hspace{\thisgap}\includegraphics[width=\thiswidth]{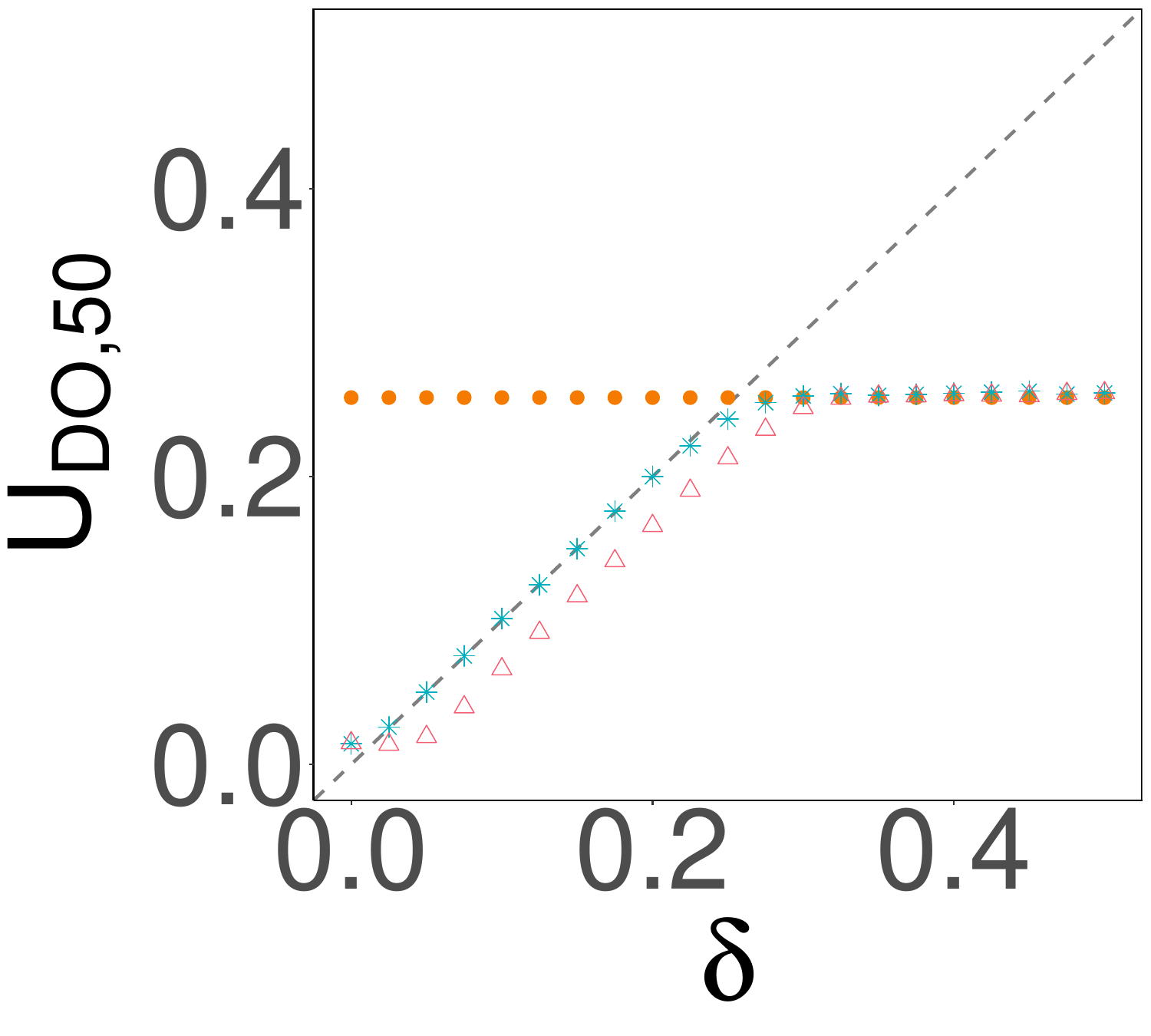} &
			\hspace{\thisgap}\includegraphics[width=\thiswidth]{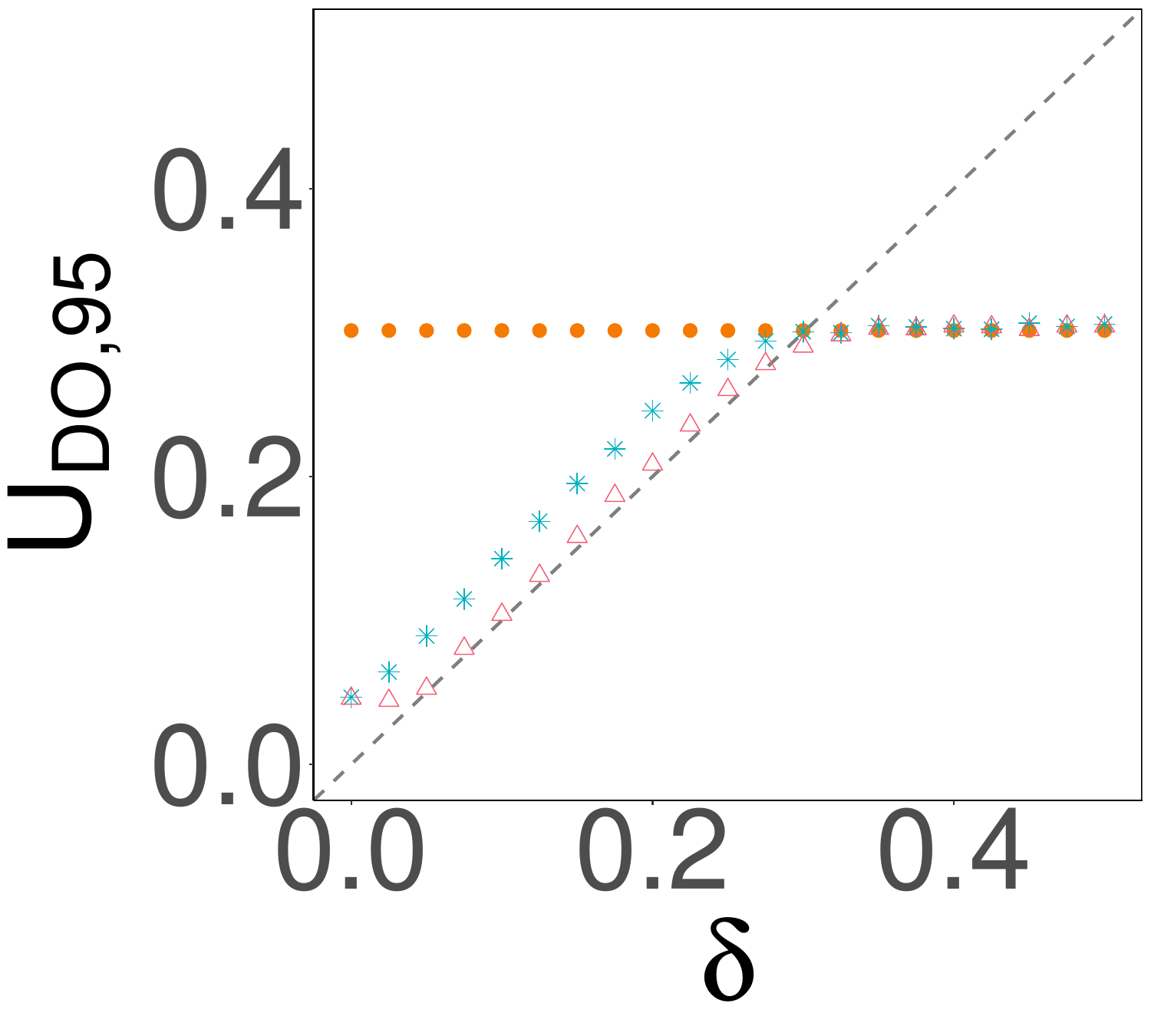}
		\end{tabular}
		\caption{Disparity DO results under the Gaussian model, $\beta=1.5$. Top: $n=1000$; middle: $n=2000$; bottom: $n=5000$. }
		\label{fig:DO_gauss_beta_1.5}
	\end{center}
\end{figure*}

\begin{figure*}[!htbp]
	\begin{center}
		\newcommand{\thiswidth}{0.2\linewidth}
		\newcommand{\thisgap}{0mm}
		\begin{tabular}{ccc}
				\hspace{\thisgap}\includegraphics[width=\thiswidth]{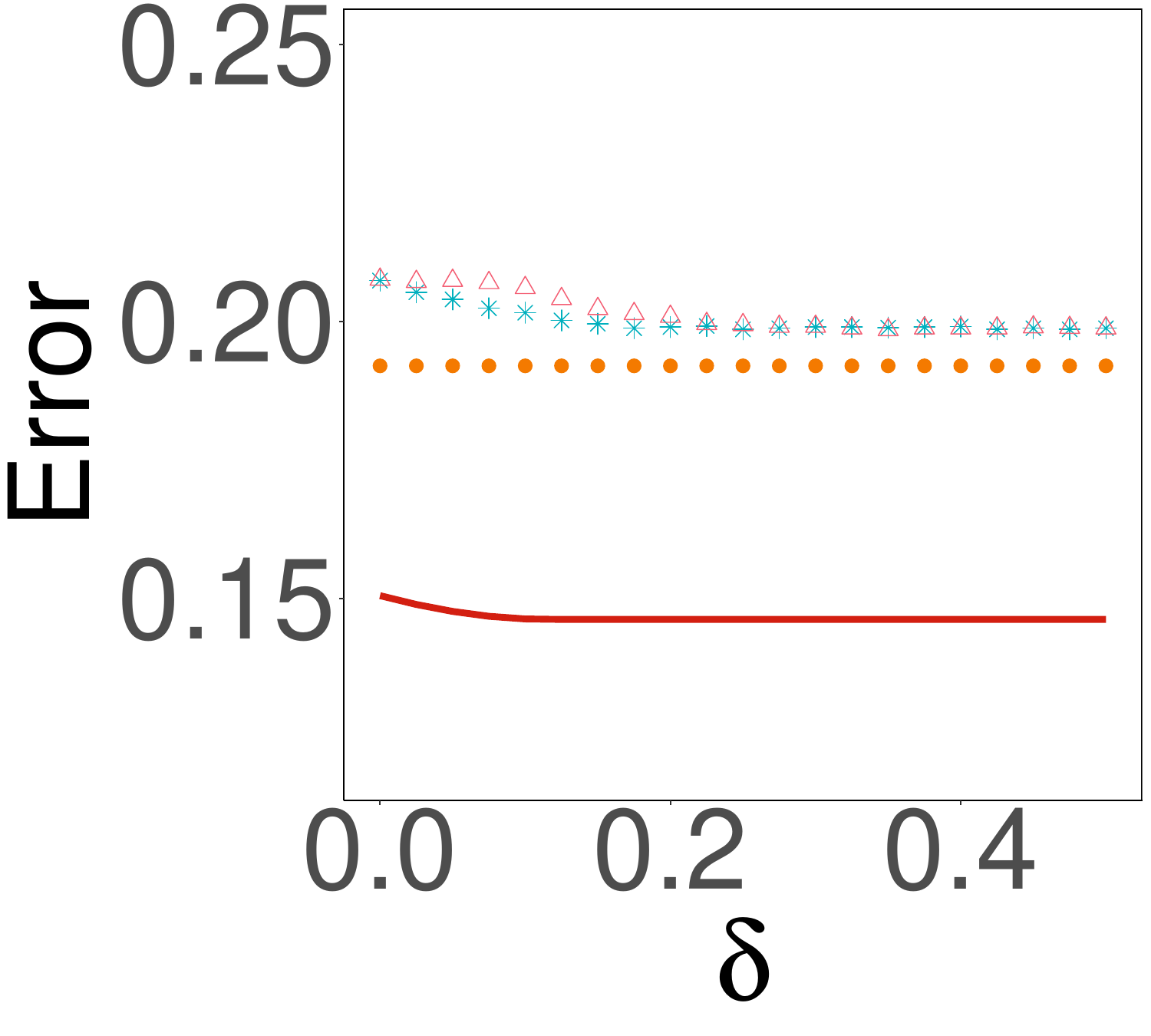} &
			\hspace{\thisgap}\includegraphics[width=\thiswidth]{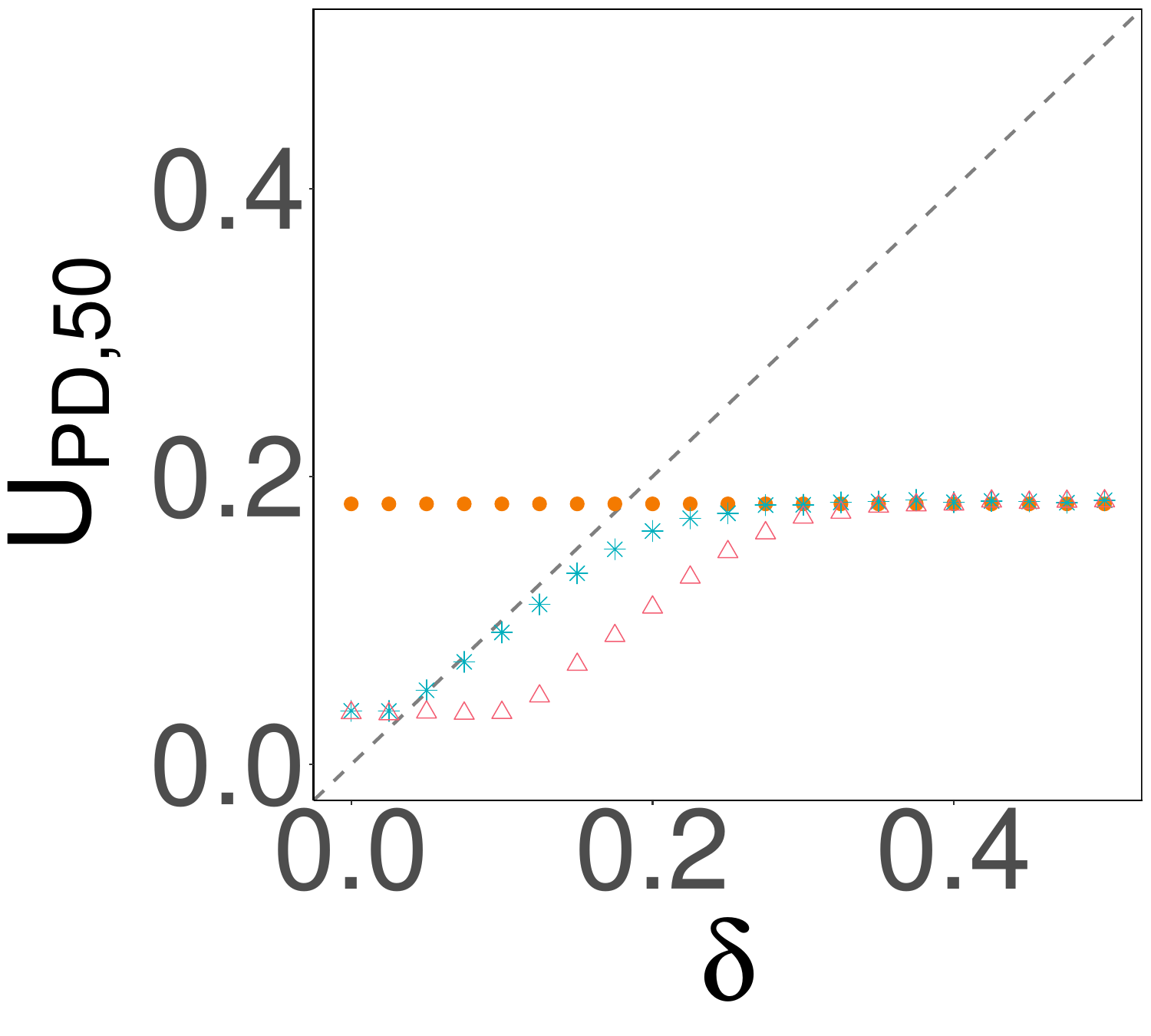} &
			\hspace{\thisgap}\includegraphics[width=\thiswidth]{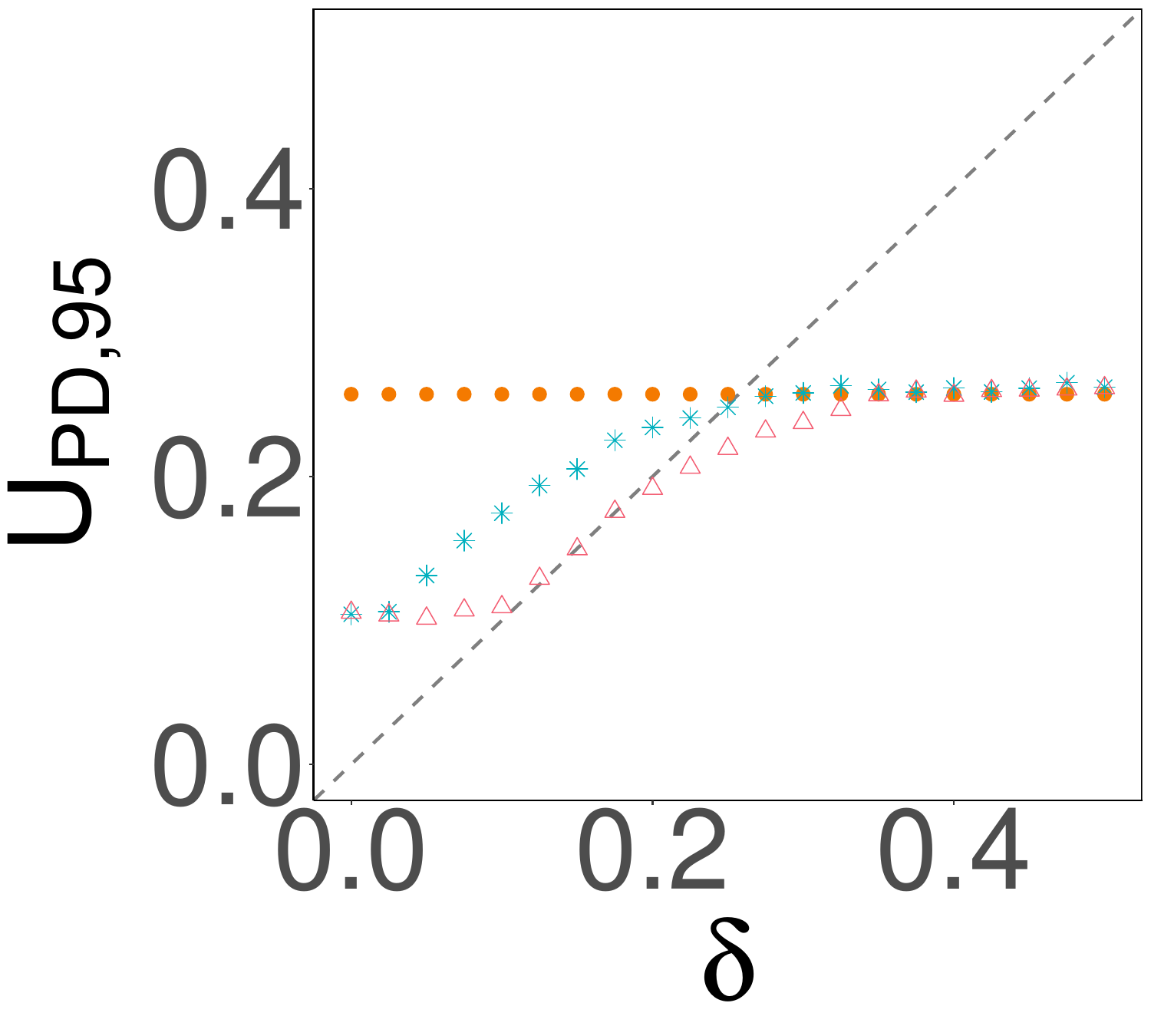} \\
			\hspace{\thisgap}\includegraphics[width=\thiswidth]{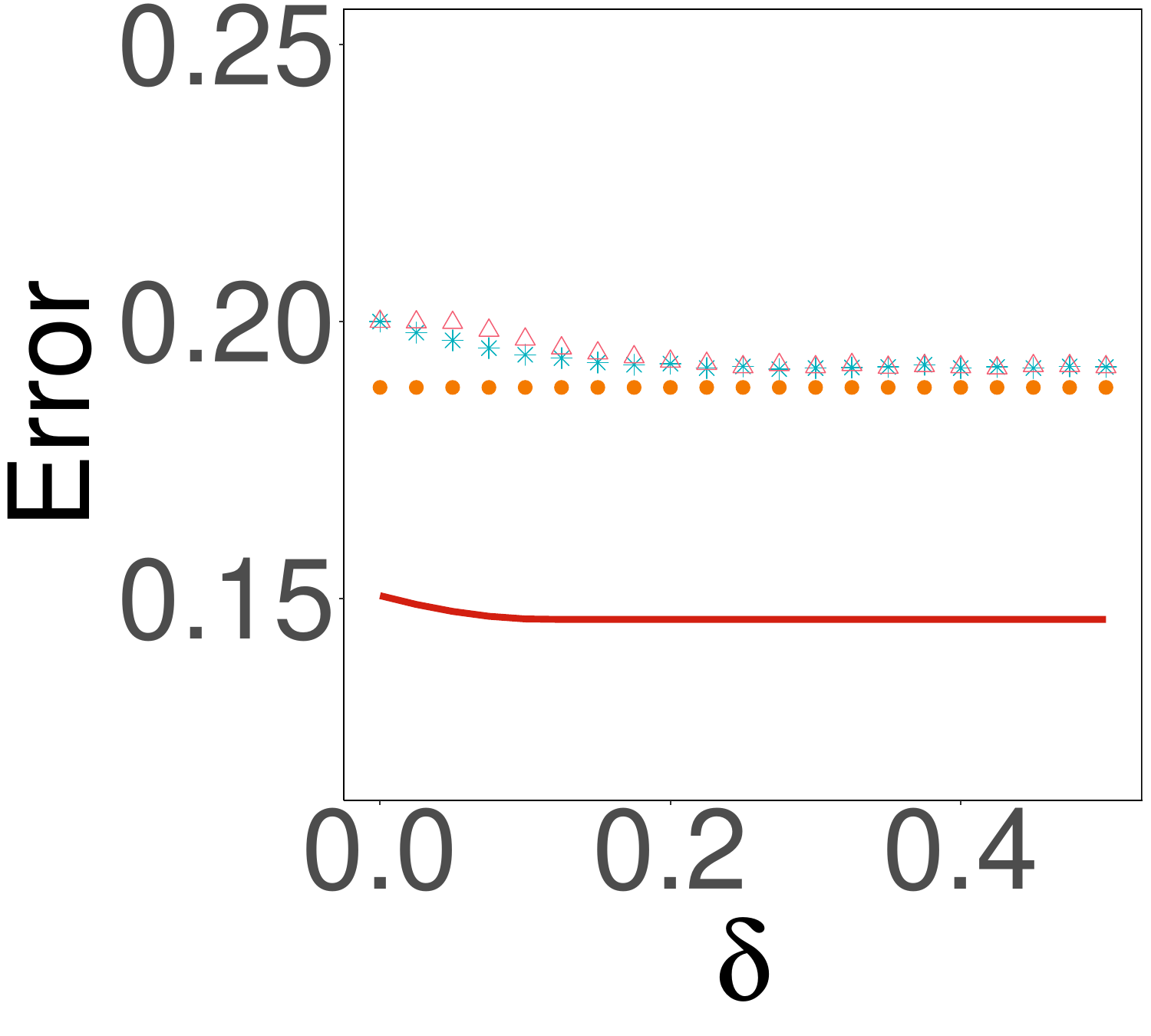} &
			\hspace{\thisgap}\includegraphics[width=\thiswidth]{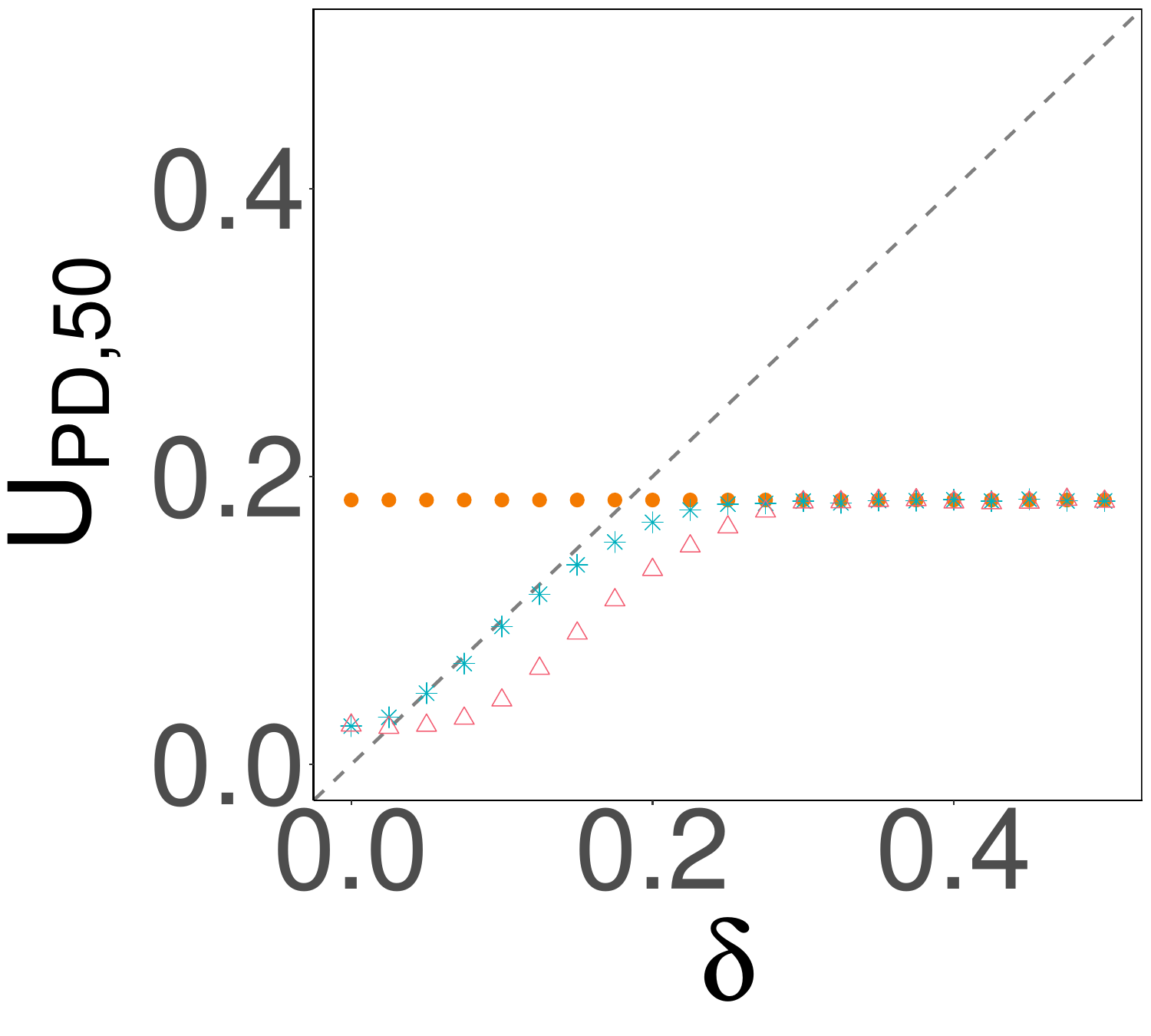} &
			\hspace{\thisgap}\includegraphics[width=\thiswidth]{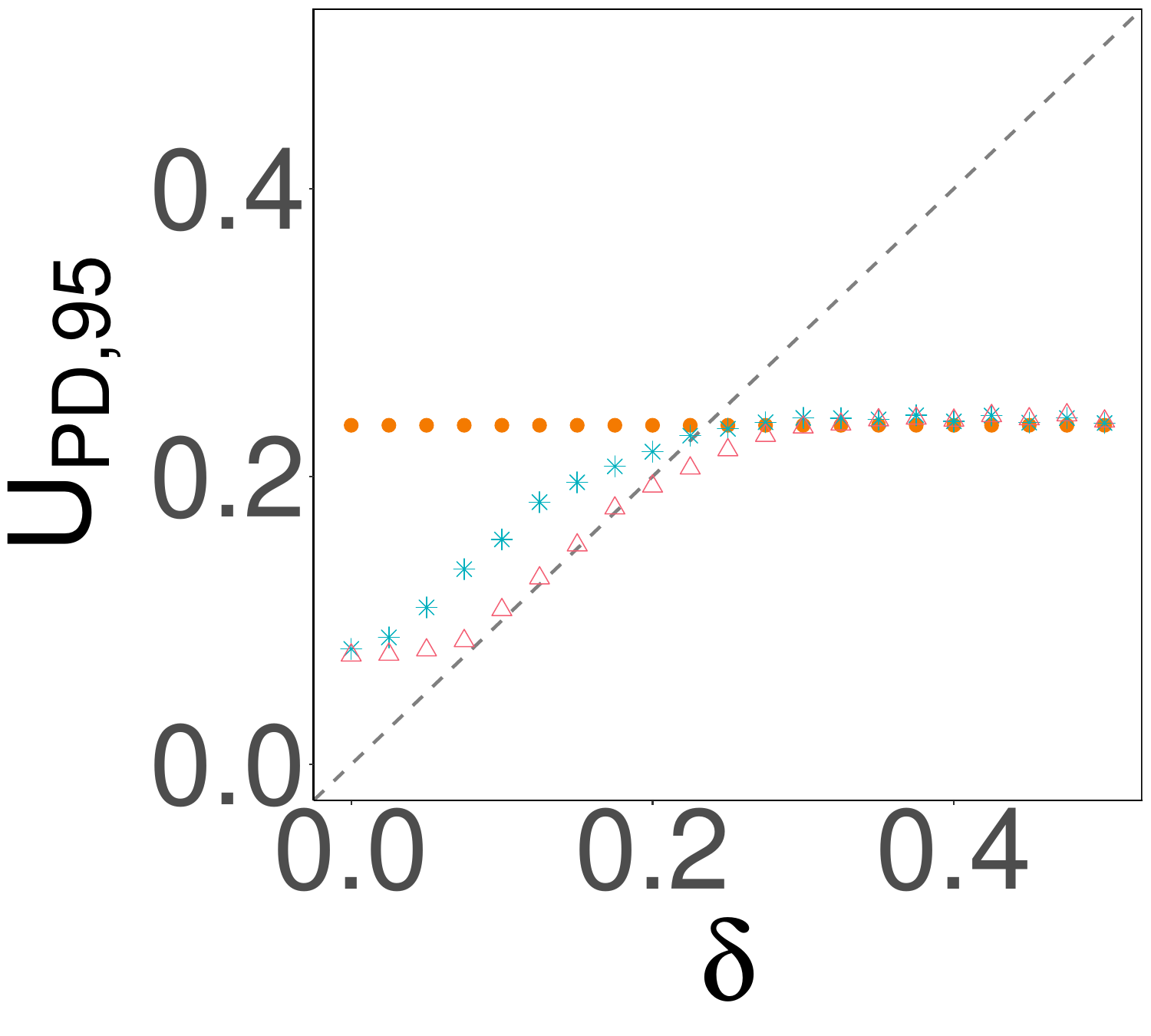} \\
			\hspace{\thisgap}\includegraphics[width=\thiswidth]{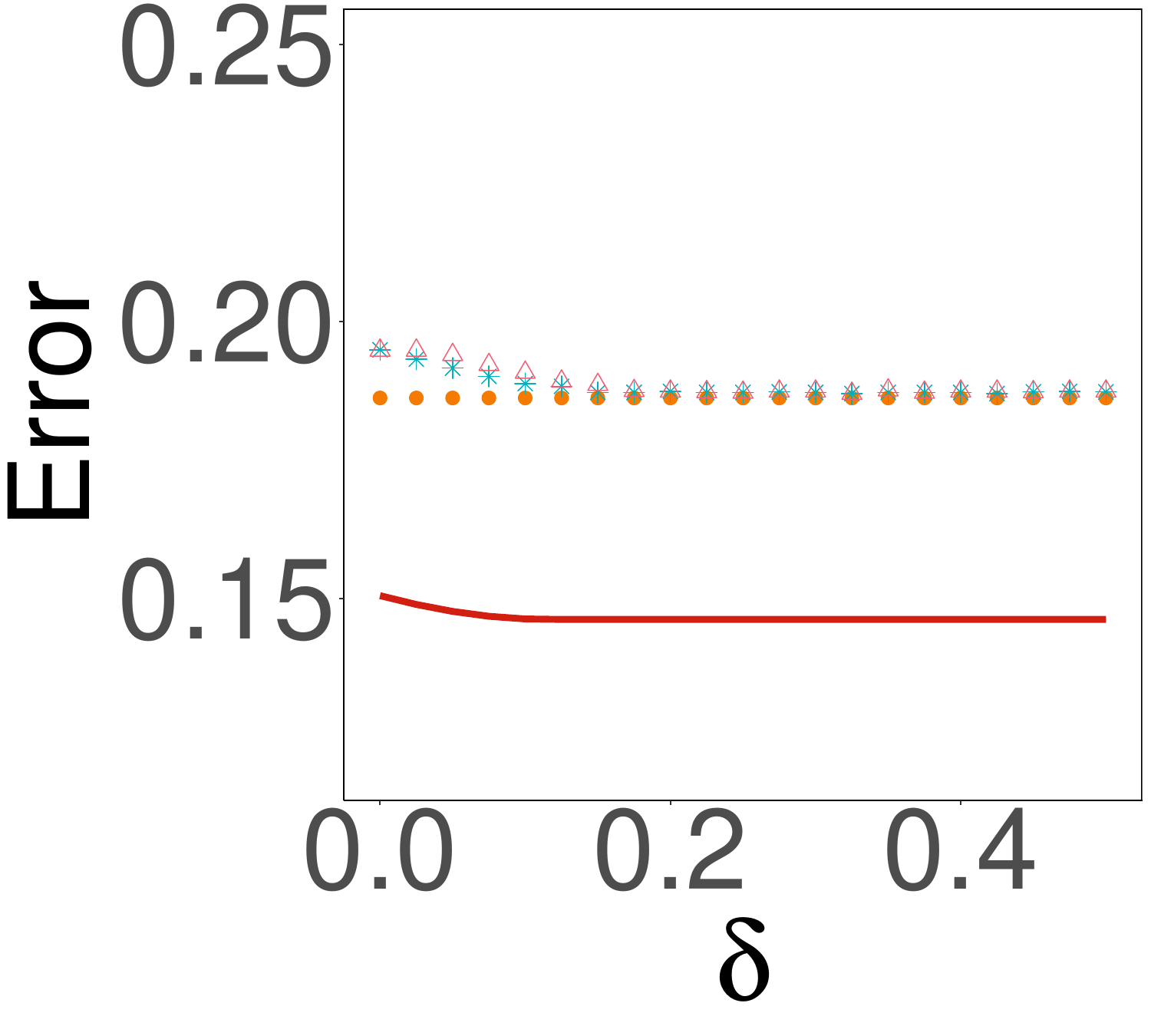} &
			\hspace{\thisgap}\includegraphics[width=\thiswidth]{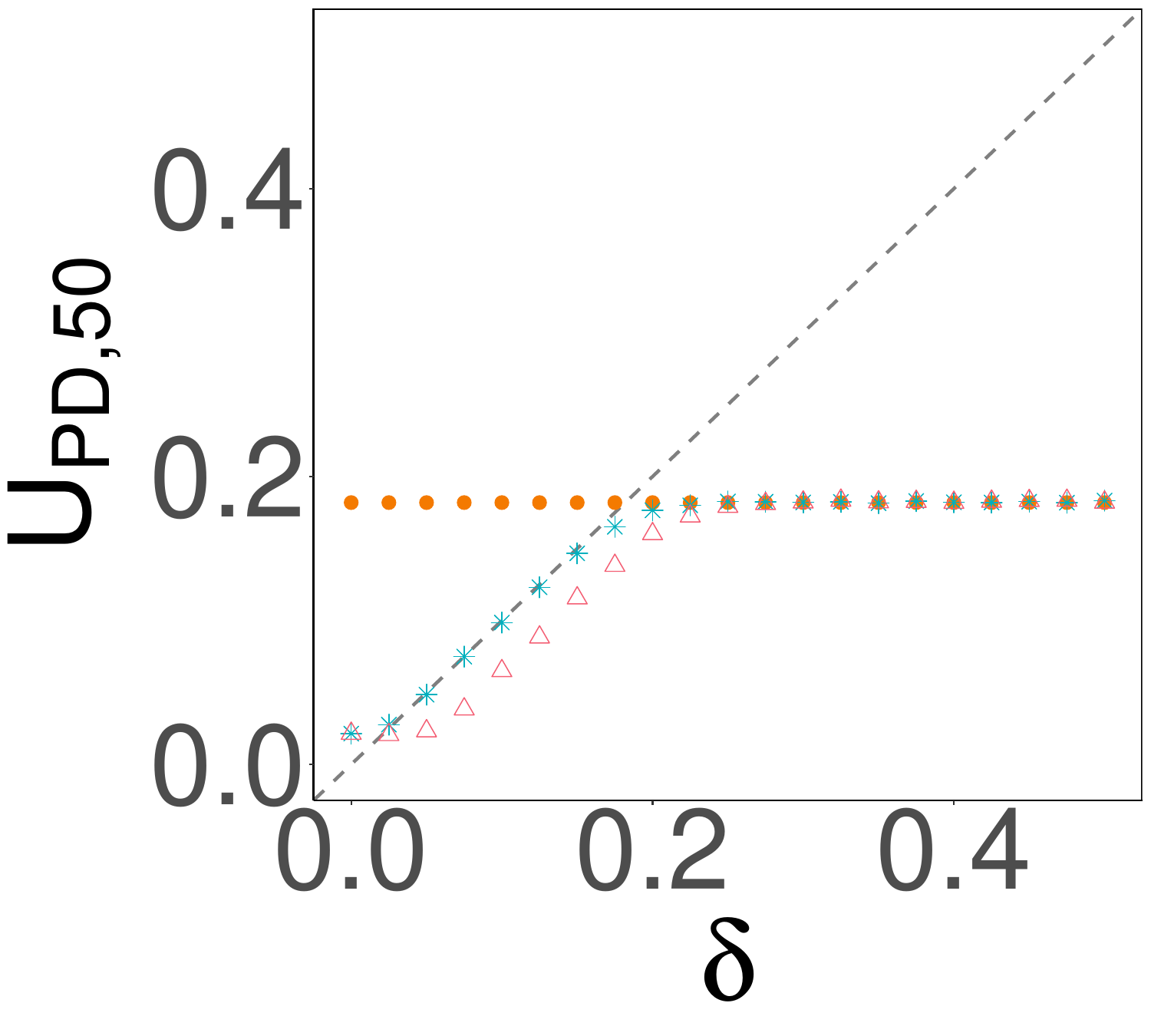} &
			\hspace{\thisgap}\includegraphics[width=\thiswidth]{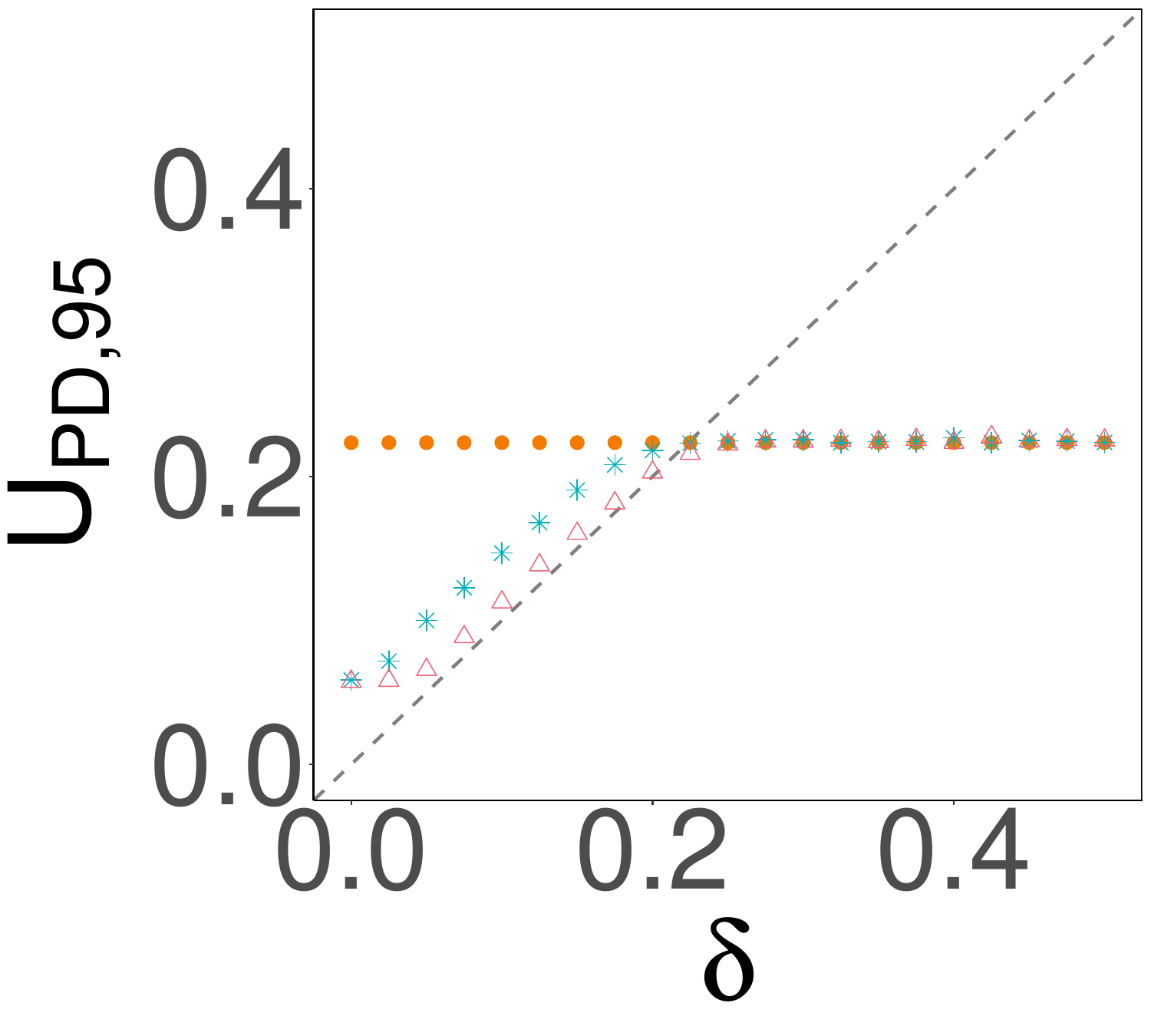}
		\end{tabular}
		\caption{Disparity PD results under the Gaussian model, $\beta=1.5$. Top: $n=1000$; middle: $n=2000$; bottom: $n=5000$. }
		\label{fig:PD_gauss_beta_1.5}
	\end{center}
\end{figure*}

\begin{figure*}[ht!]
	\begin{center}
		\newcommand{\thiswidth}{0.2\linewidth}
		\newcommand{\thisgap}{0mm}
		\begin{tabular}{ccc}
			\hspace{\thisgap}\includegraphics[width=\thiswidth]{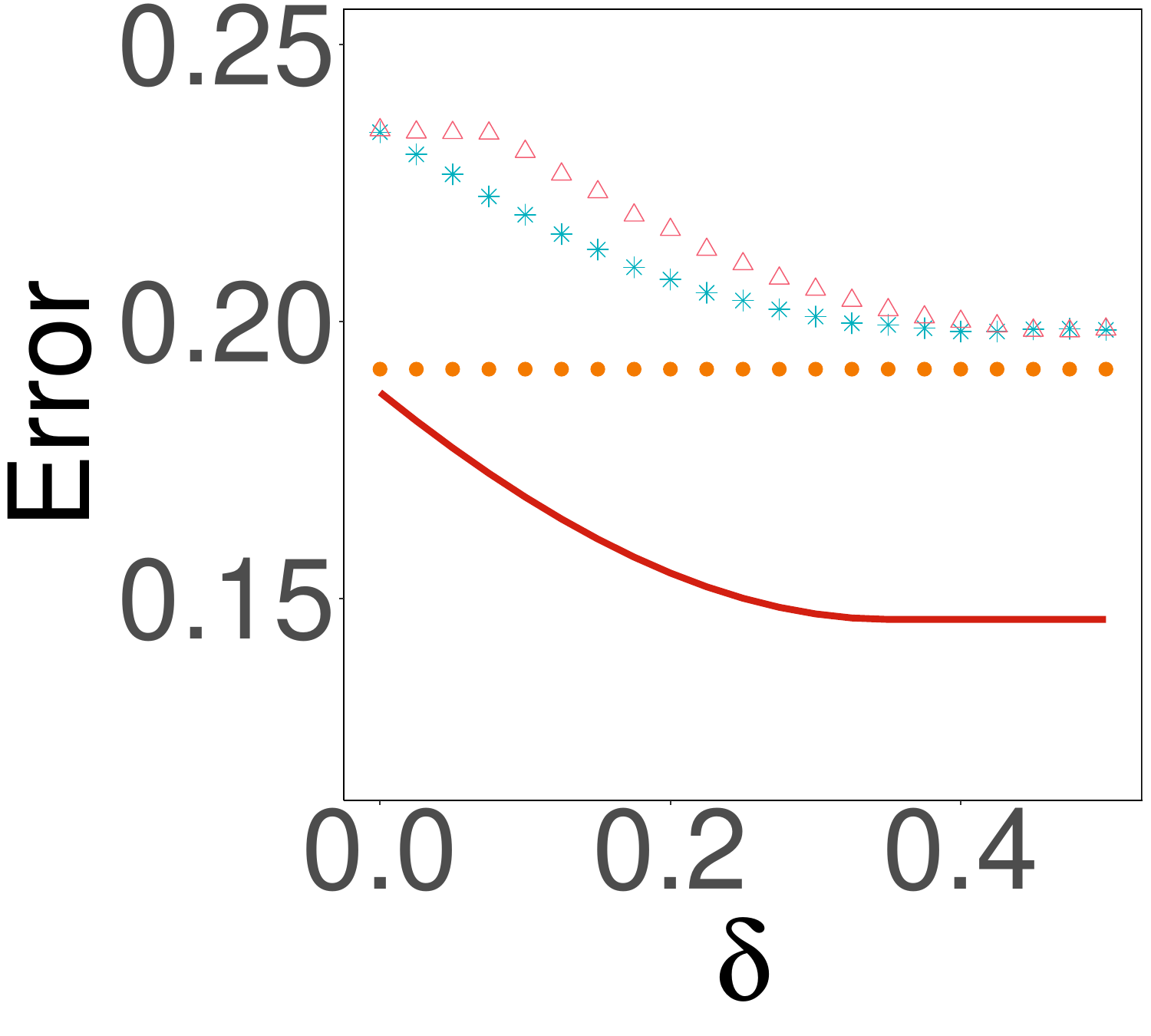} &
			\hspace{\thisgap}\includegraphics[width=\thiswidth]{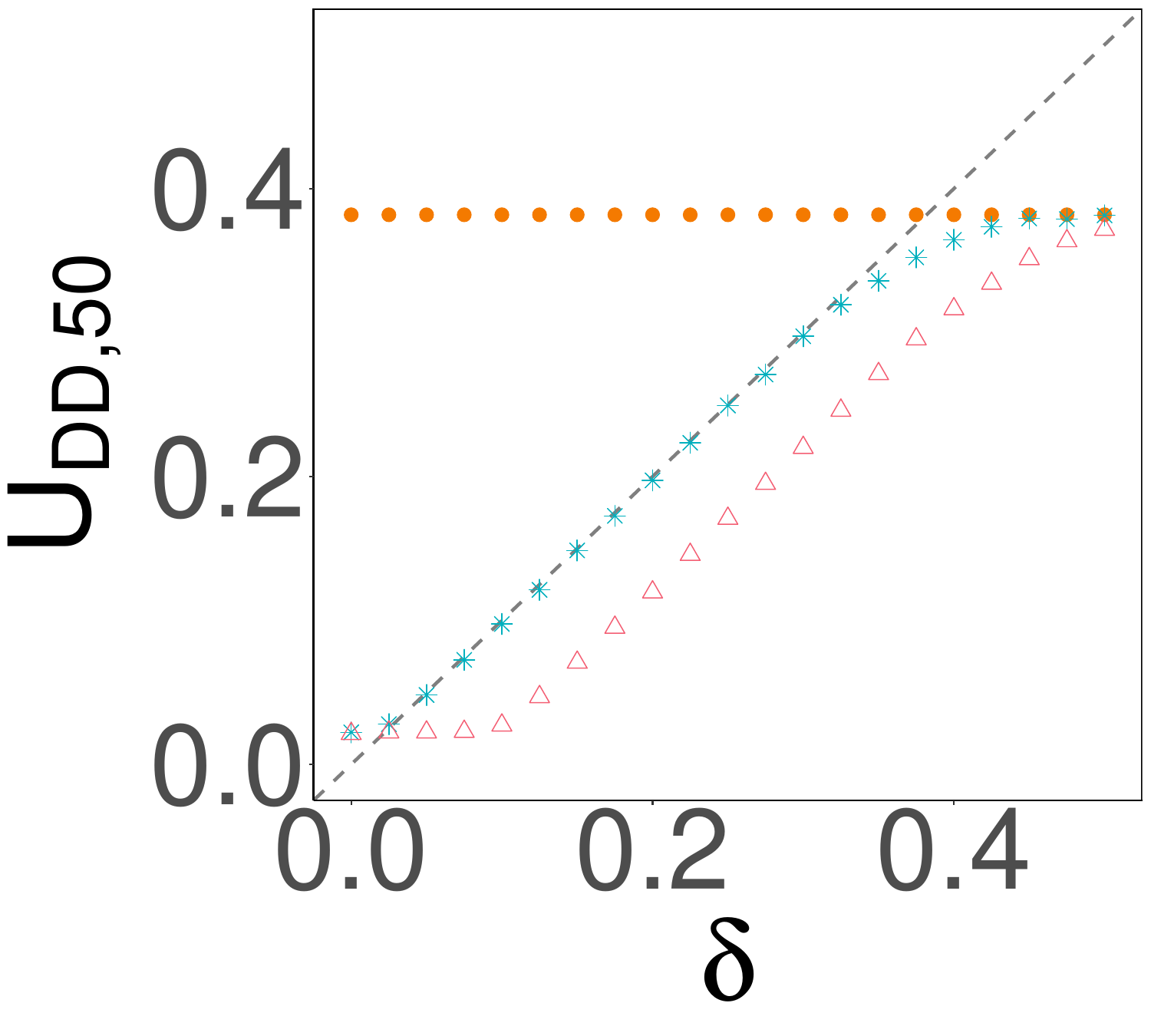} &
			\hspace{\thisgap}\includegraphics[width=\thiswidth]{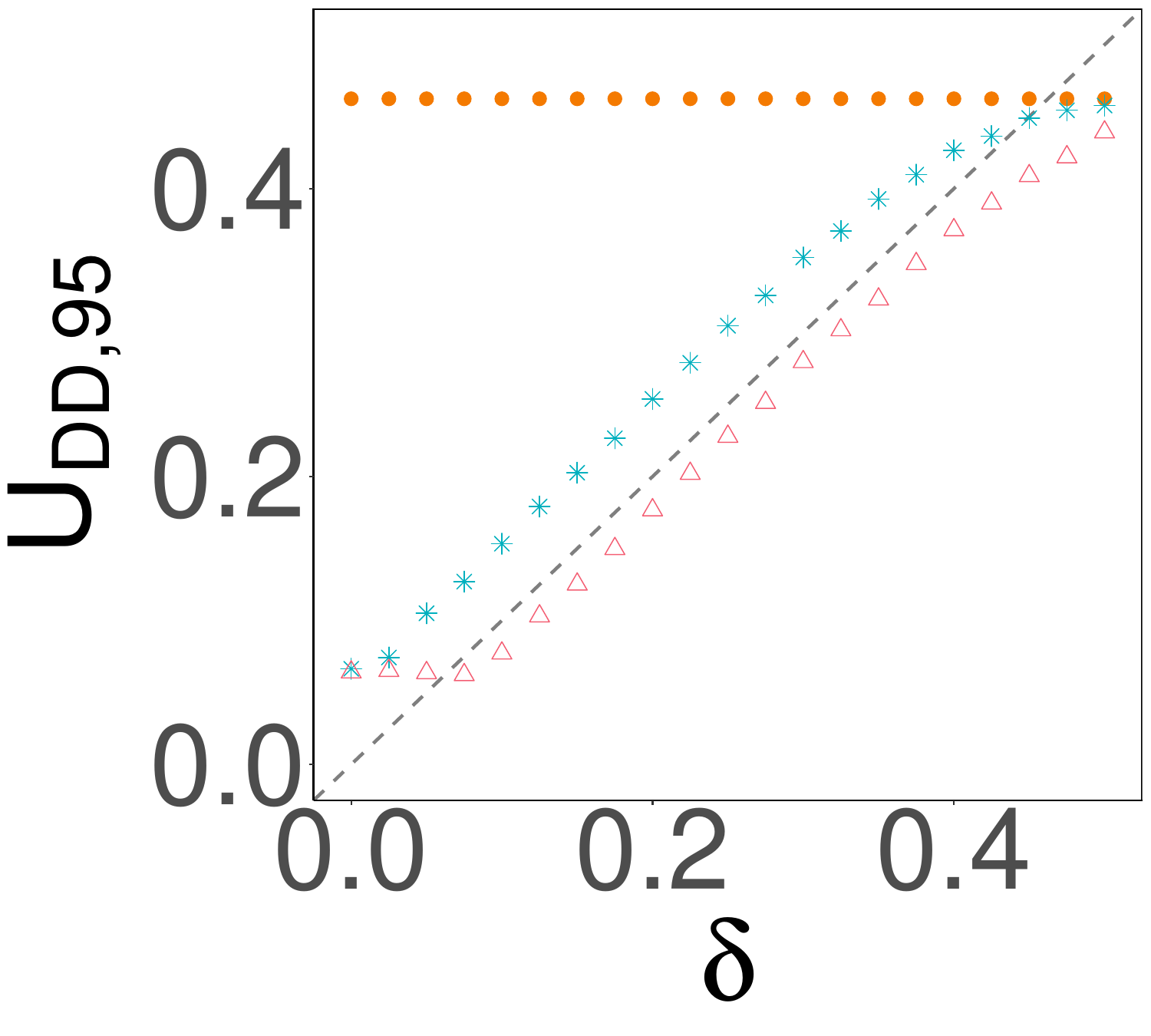} \\
			\hspace{\thisgap}\includegraphics[width=\thiswidth]{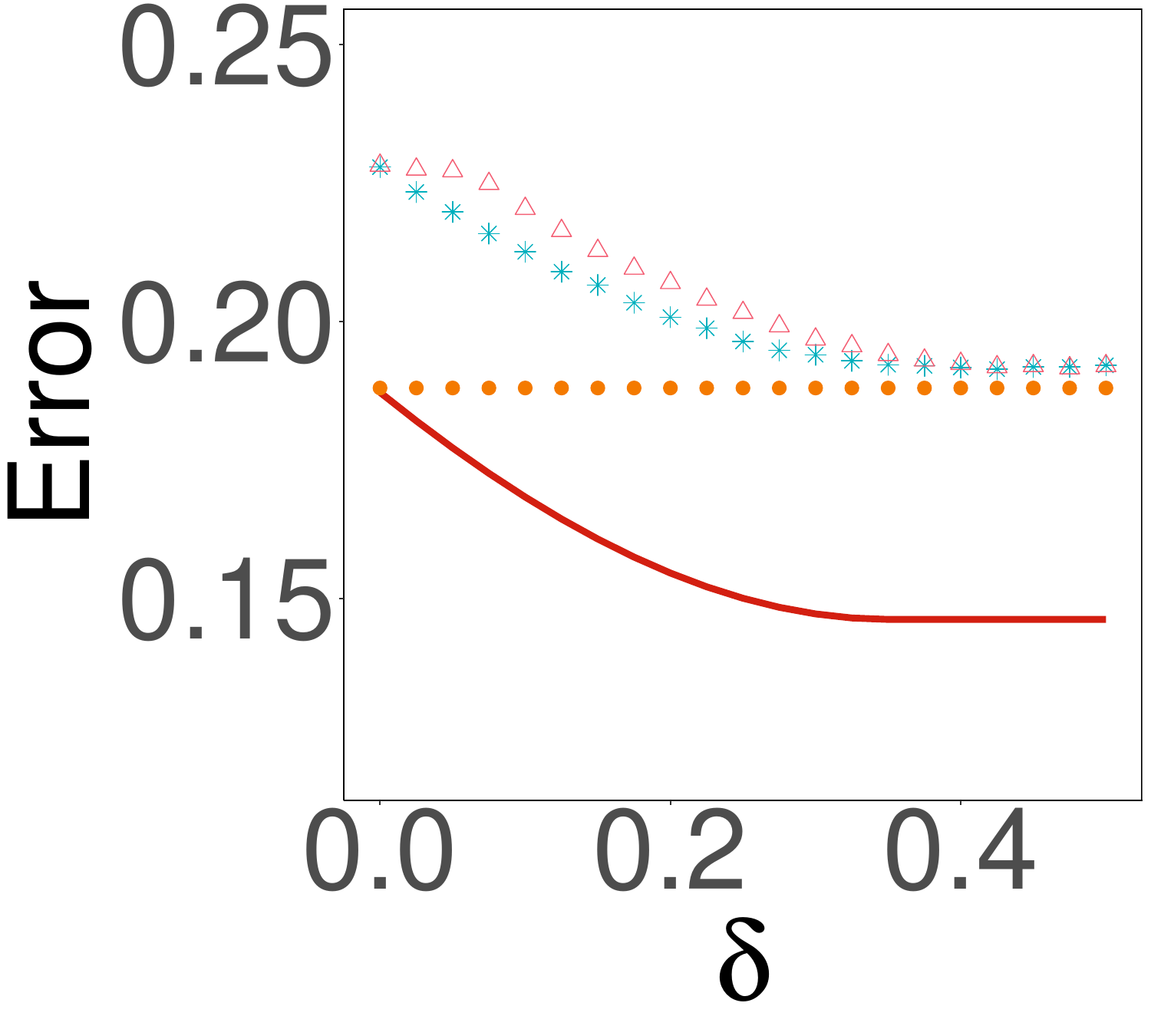} &
			\hspace{\thisgap}\includegraphics[width=\thiswidth]{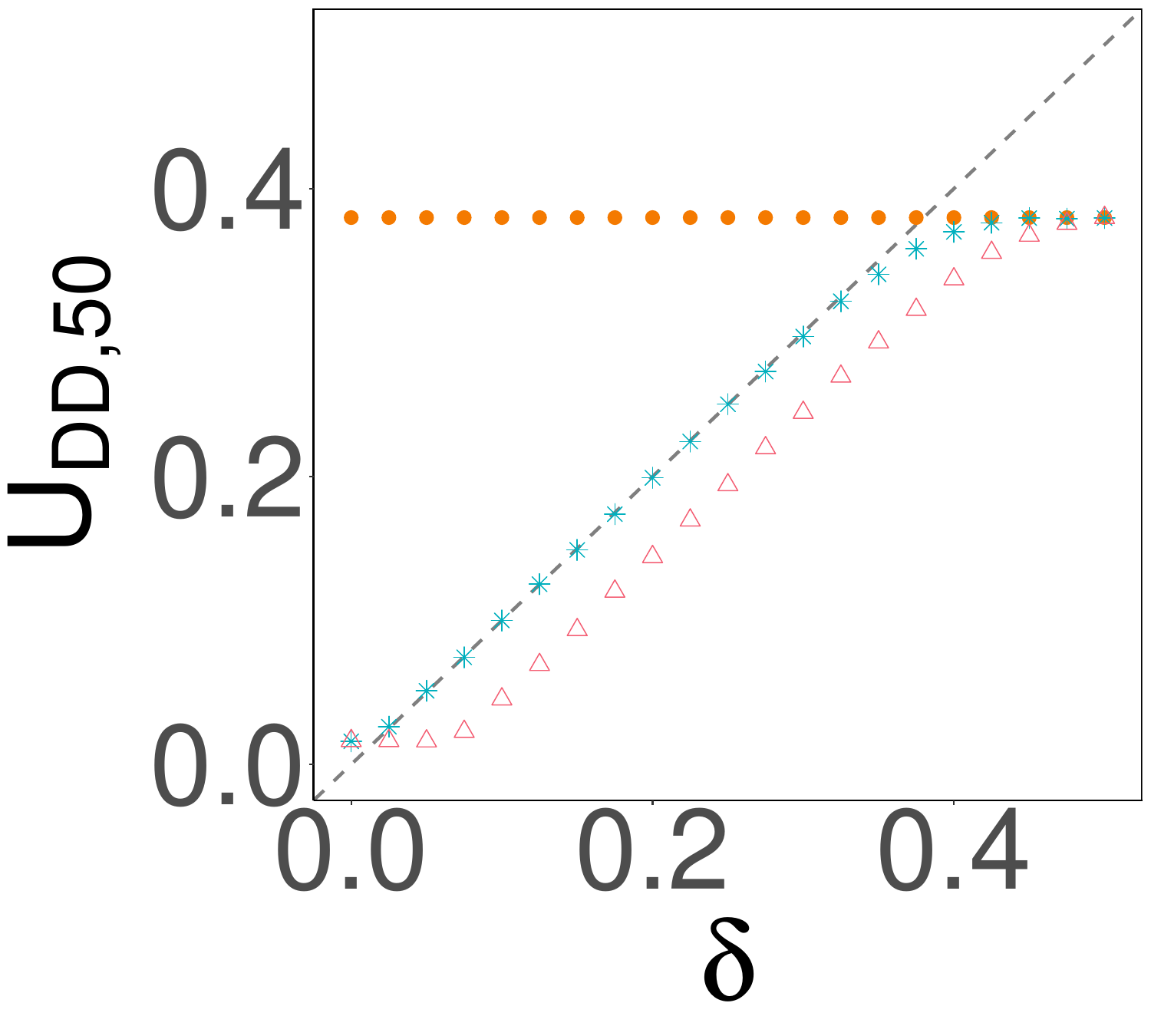} &
			\hspace{\thisgap}\includegraphics[width=\thiswidth]{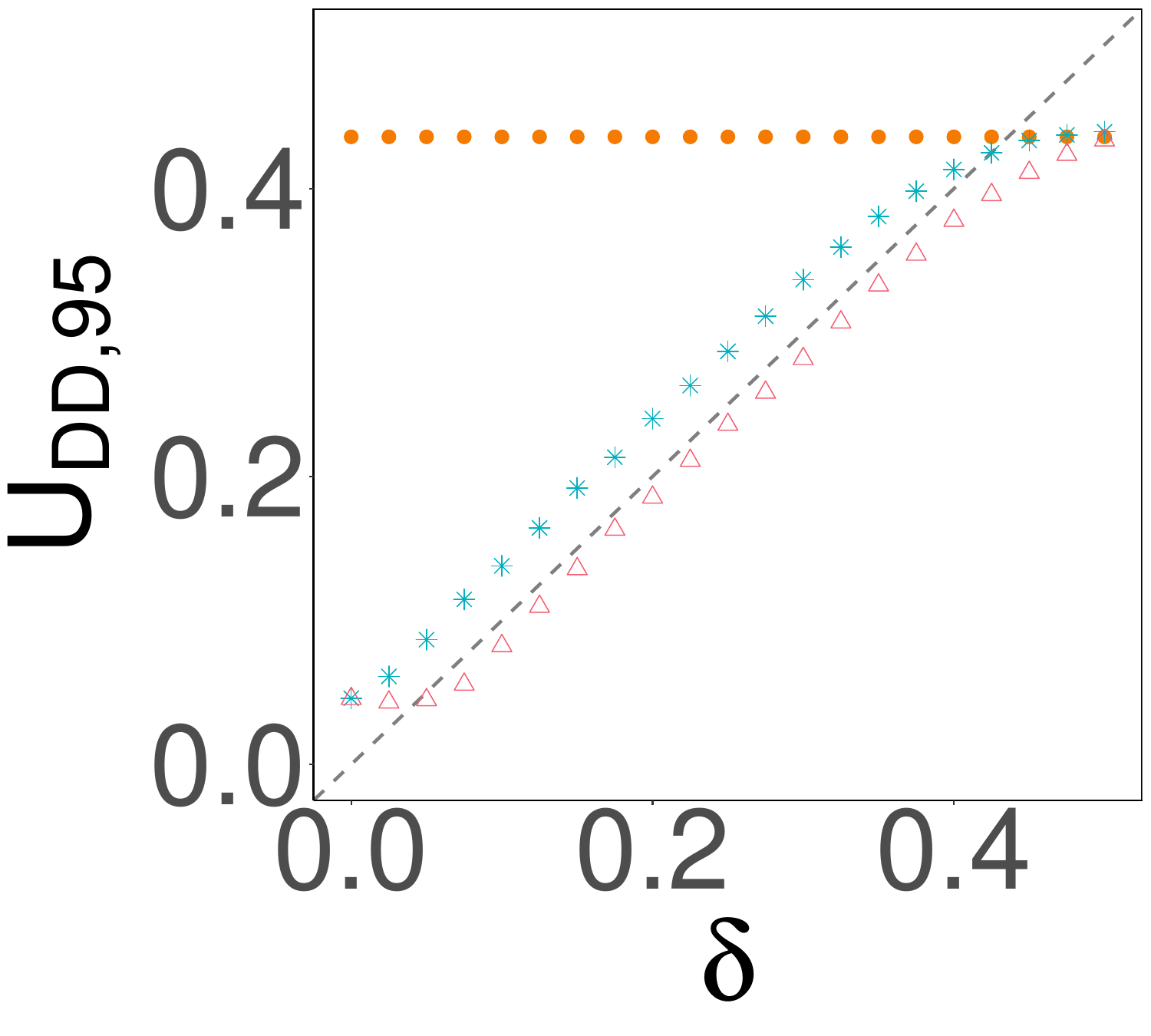} \\
			\hspace{\thisgap}\includegraphics[width=\thiswidth]{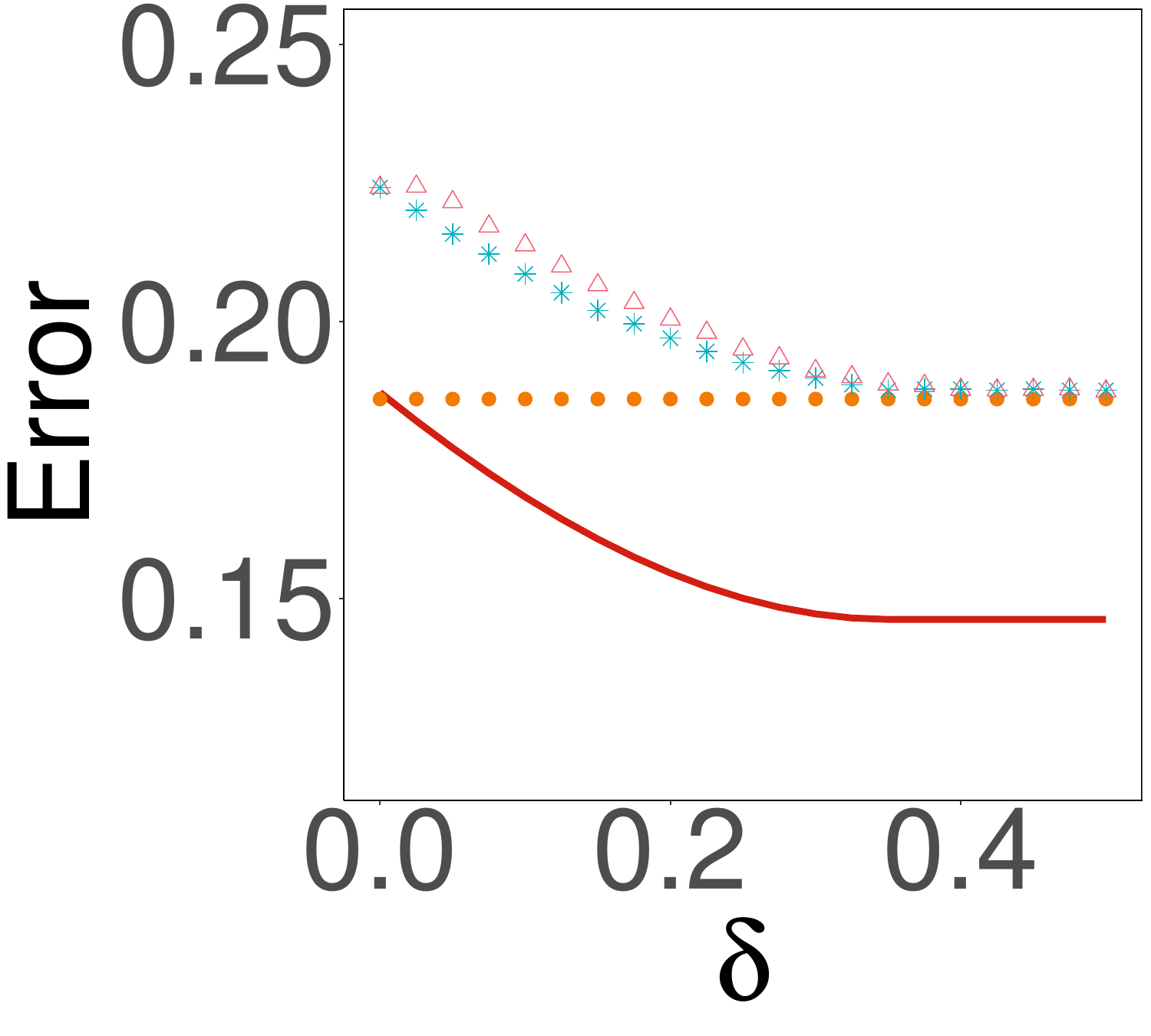} &
			\hspace{\thisgap}\includegraphics[width=\thiswidth]{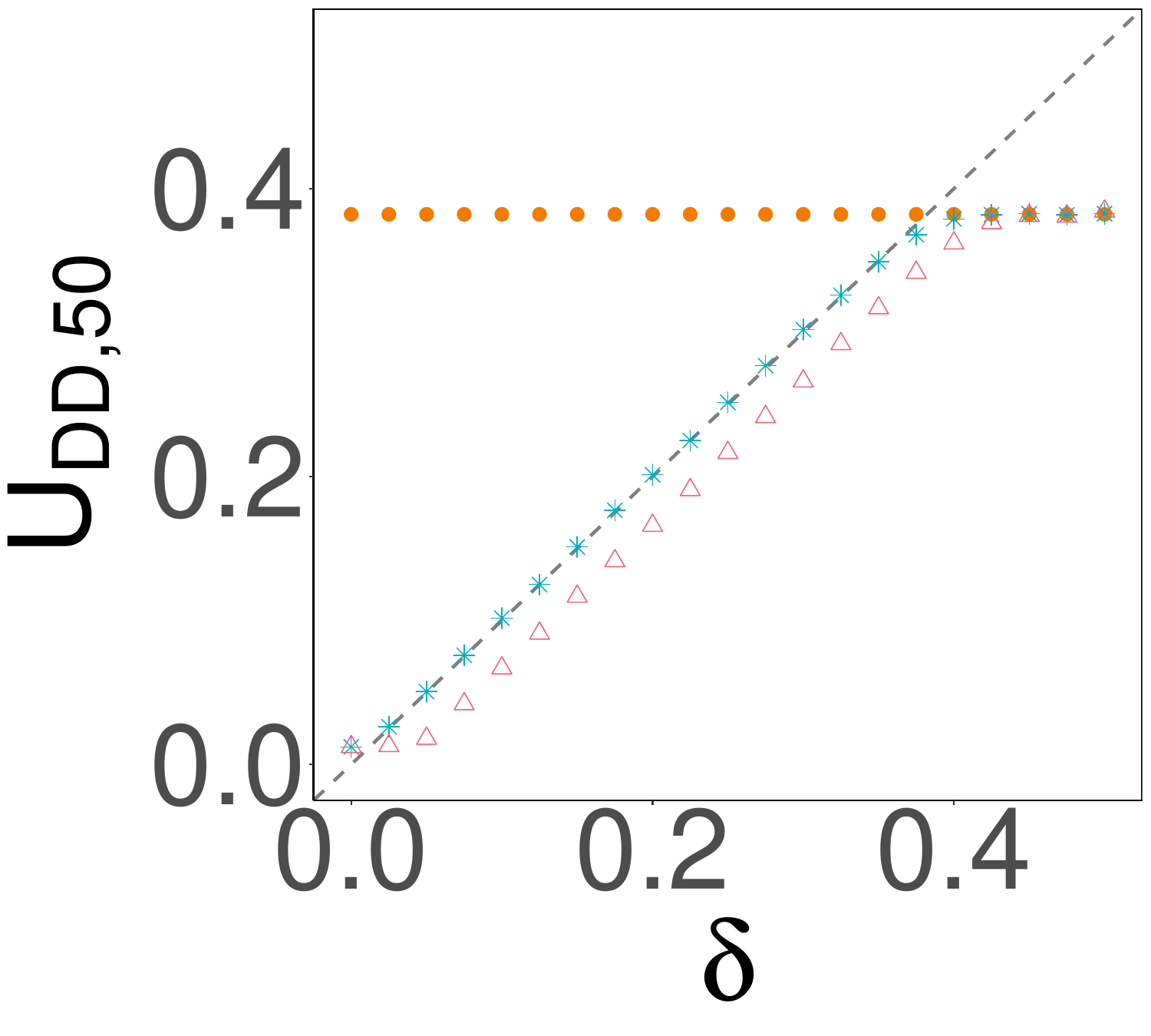} &
			\hspace{\thisgap}\includegraphics[width=\thiswidth]{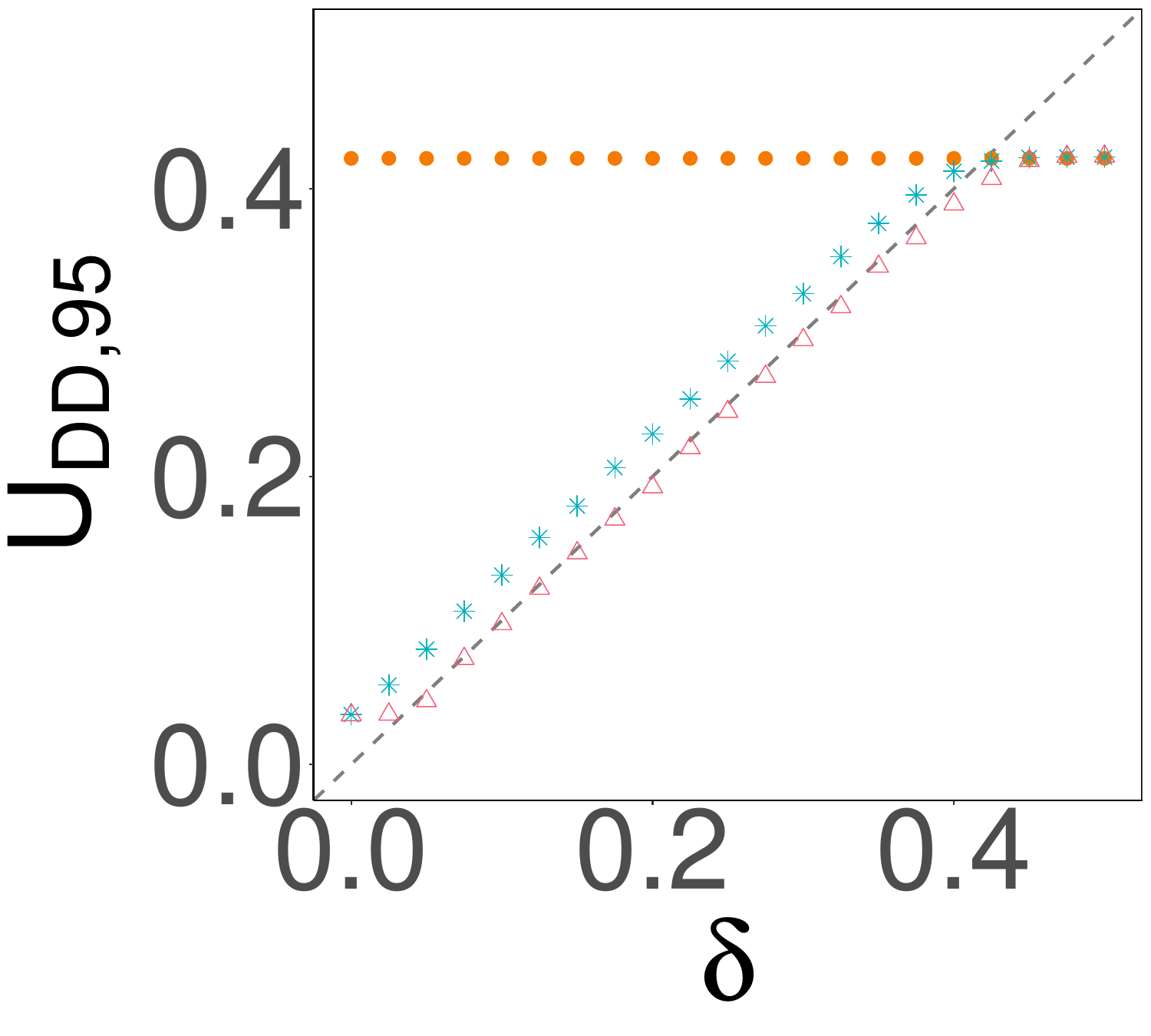}
		\end{tabular}
		\caption{Disparity DD results under the Gaussian model, $\beta=1.5$. Top: $n=1000$; middle: $n=2000$; bottom: $n=5000$.}
		\label{fig:DD_gauss_beta_1.5}
	\end{center}
\end{figure*}

\paragraph{Error-unfairness trade-off.}
We illustrate the error-unfairness trade-off in Figures \ref{fig:tradeoff_DO_gauss_beta_1.5}-\ref{fig:tradeoff_DD_gauss_beta_1.5}. There is only a single point in each figure for FLDA, as it does not incorporate any fairness correction. Both Fair-FDA and Fair-$\mathrm{FLDA_c}$ demonstrate comparable trade-offs between classification error and unfairness, since they are both derived from the Bayes optimal fair classifier. 

\begin{figure*}[!htbp]
	\begin{center}
		\newcommand{\thiswidth}{0.2\linewidth}
		\newcommand{\thisgap}{0mm}
		\begin{tabular}{ccc}
			\hspace{\thisgap}\includegraphics[width=\thiswidth]{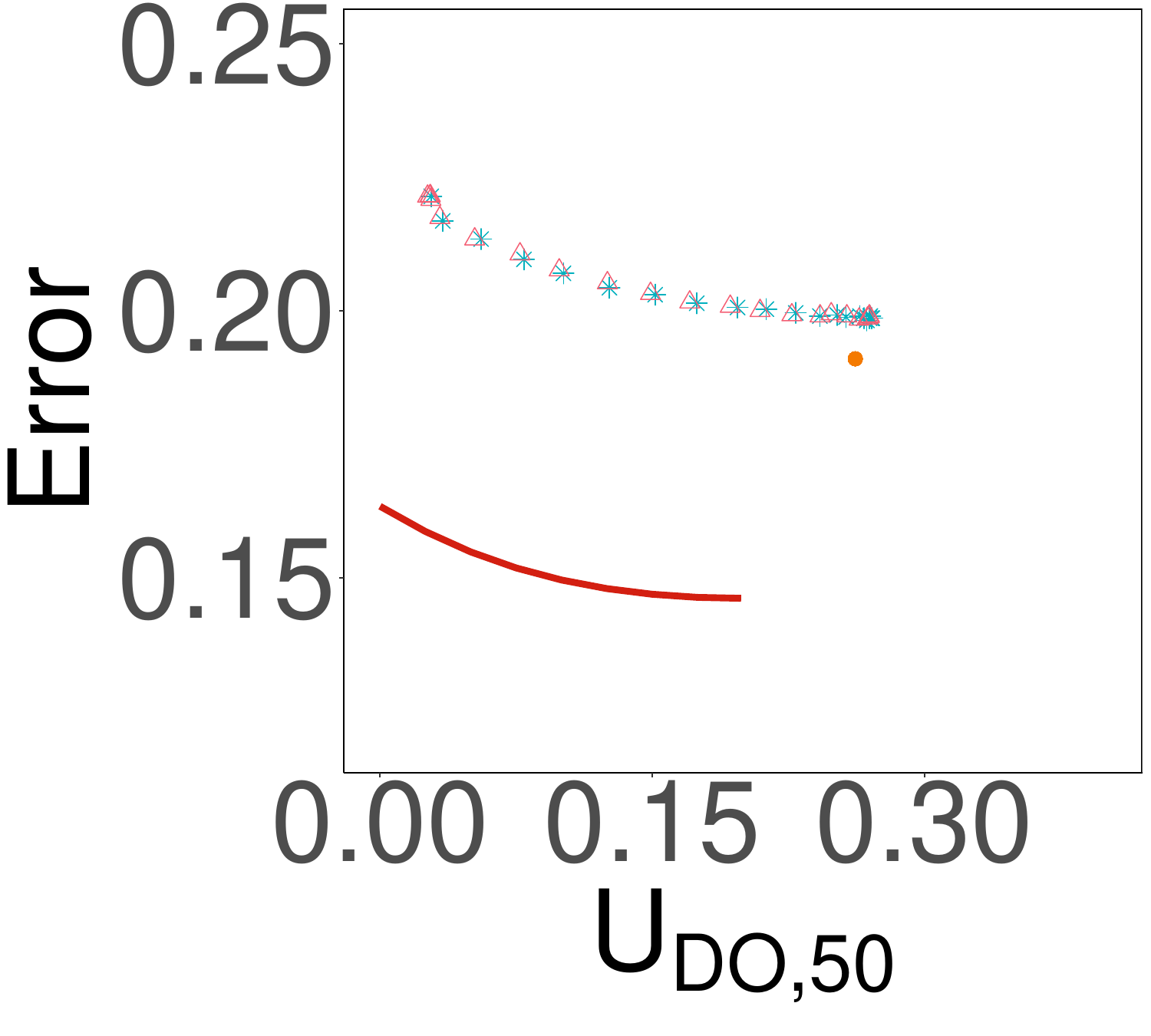} &
            \hspace{\thisgap}\includegraphics[width=\thiswidth]{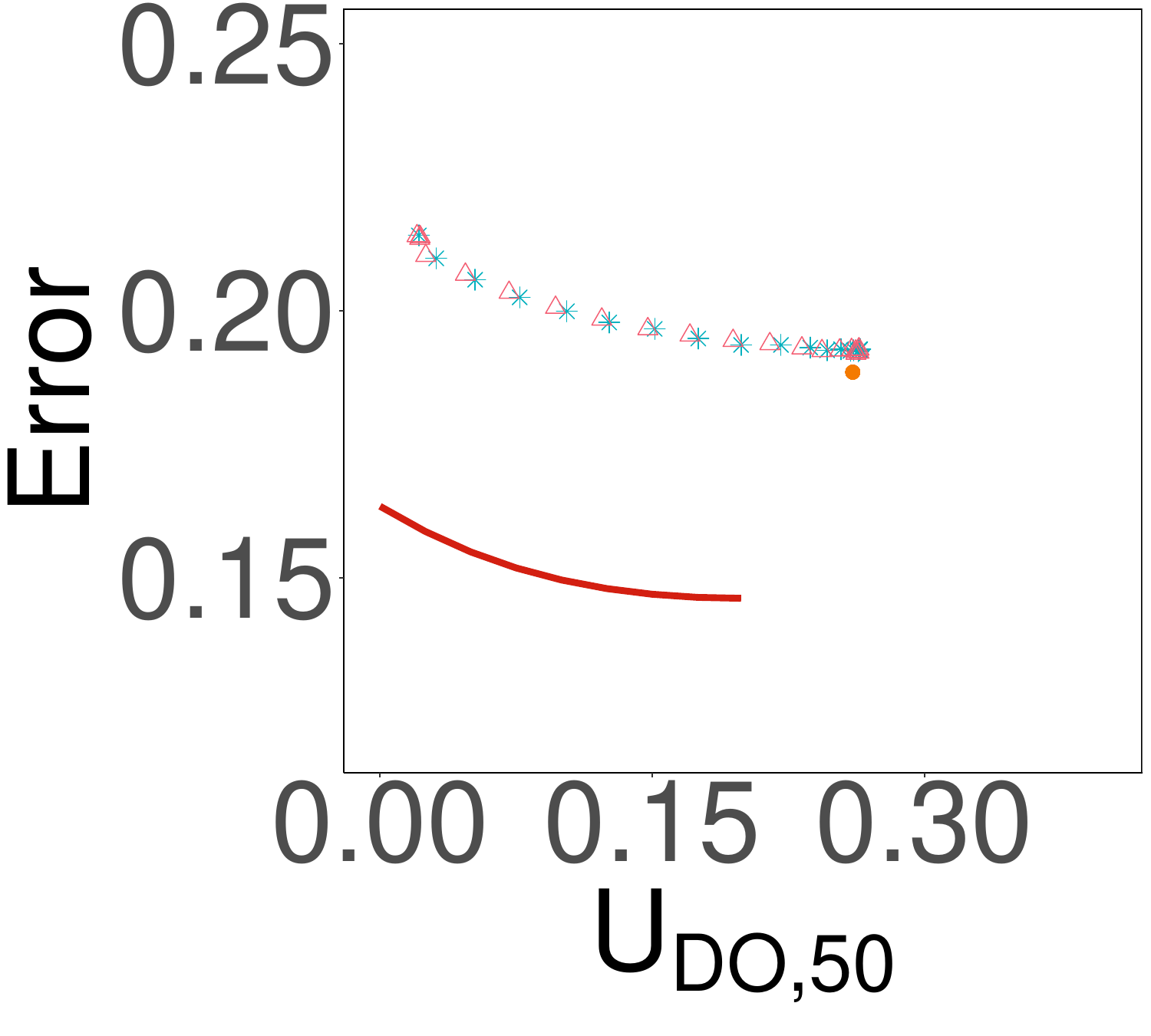} &
            \hspace{\thisgap}\includegraphics[width=\thiswidth]{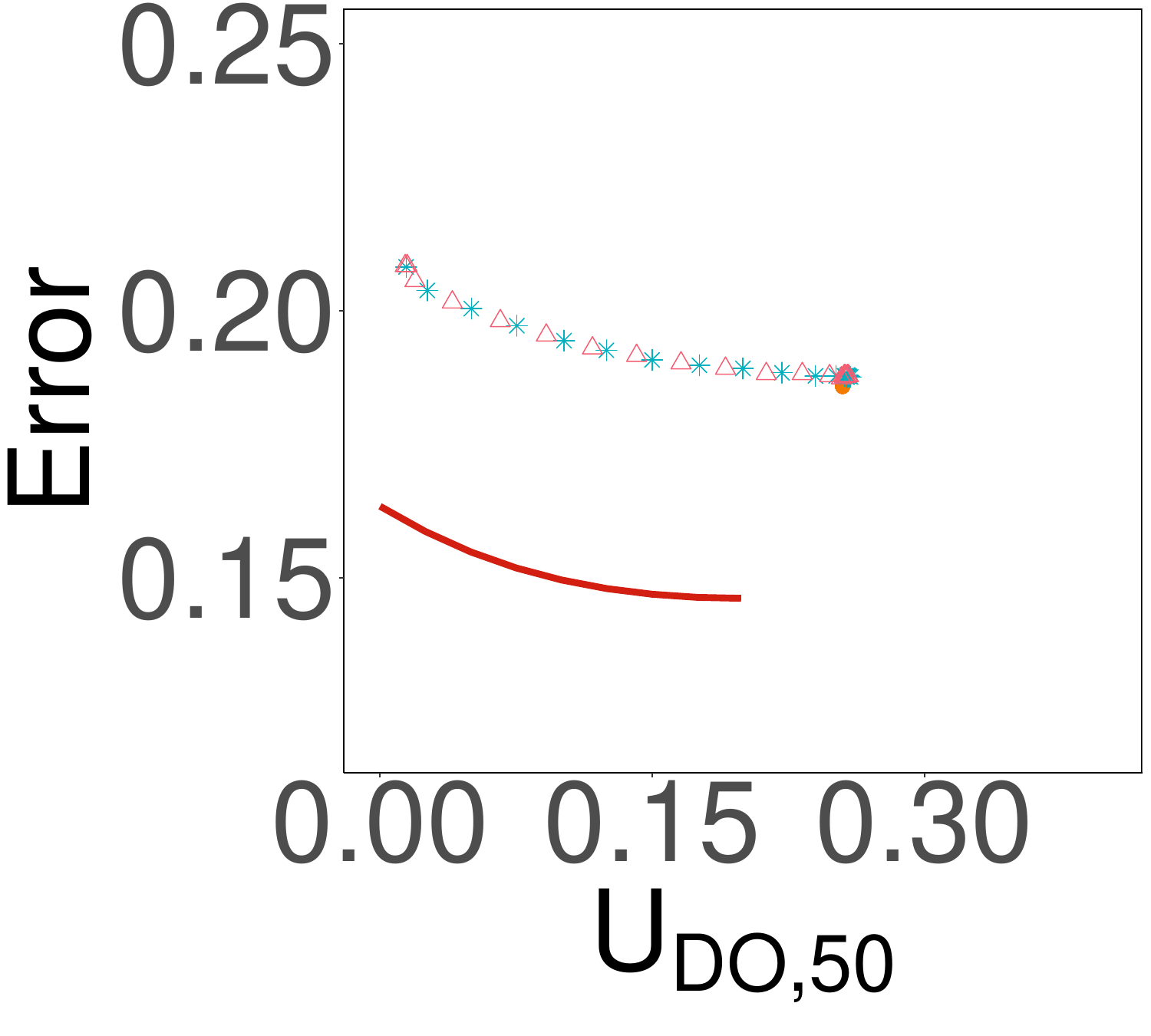} \\
		\hspace{\thisgap}\includegraphics[width=\thiswidth]{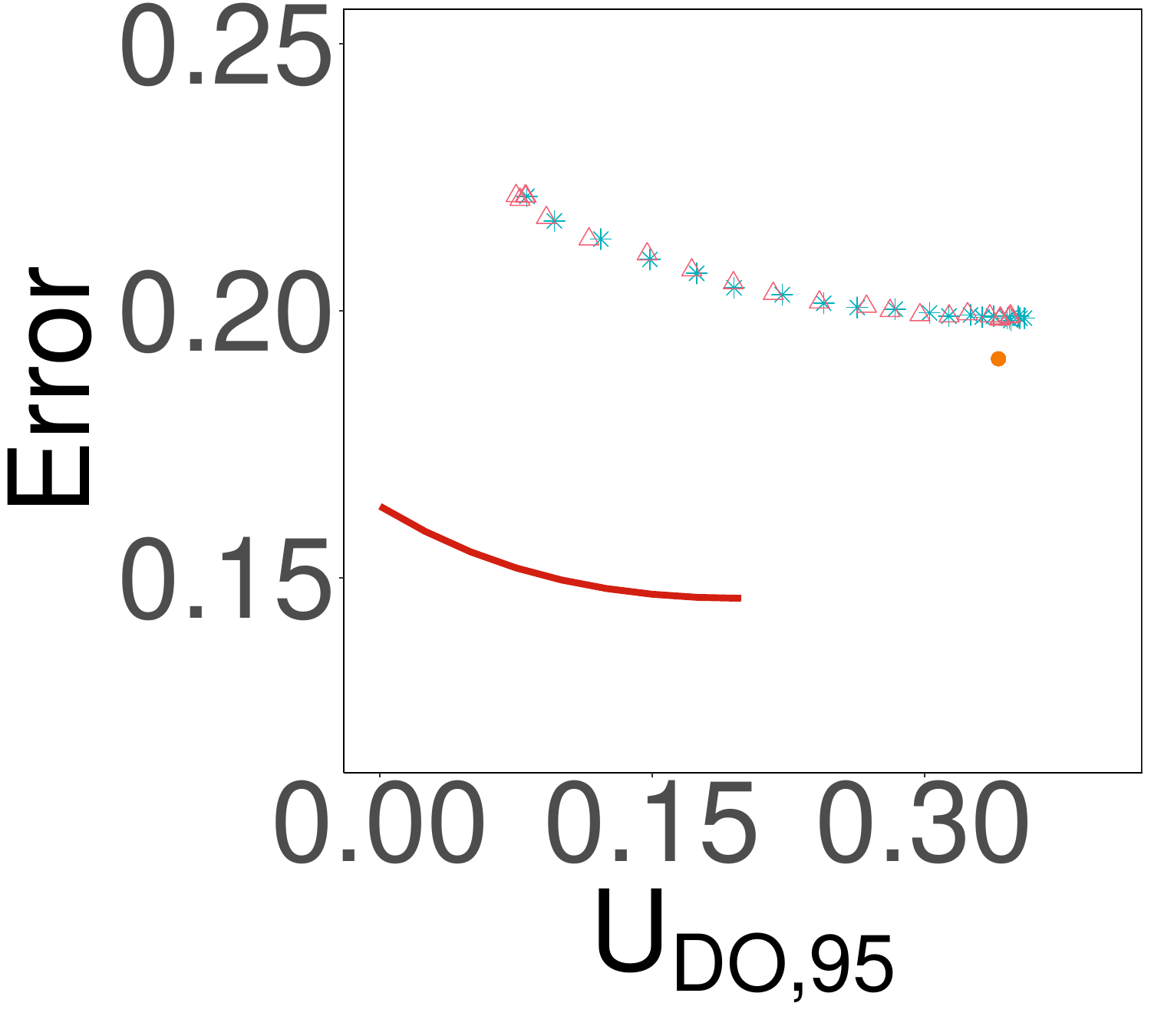} &
			\hspace{\thisgap}\includegraphics[width=\thiswidth]{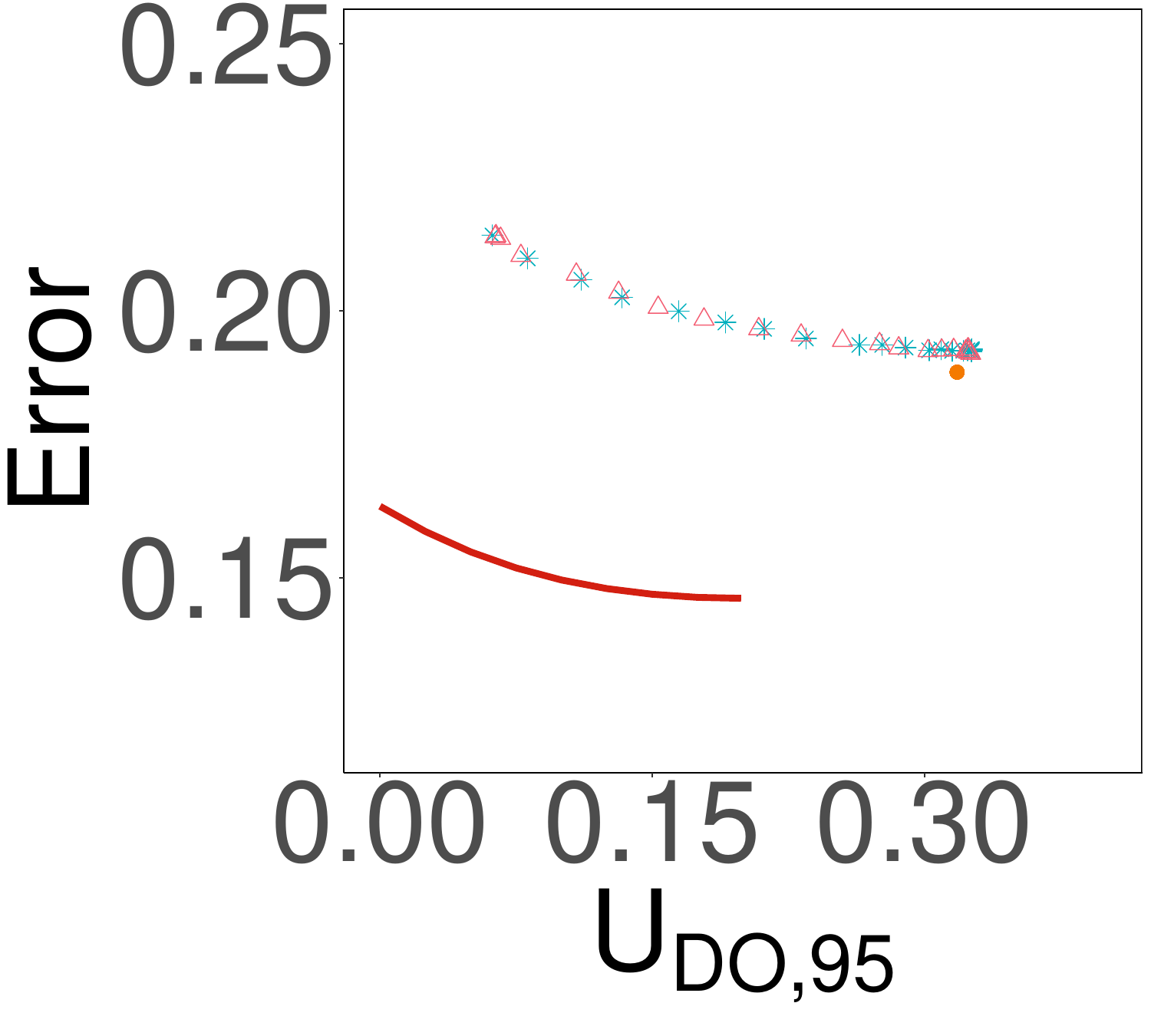} &
            \hspace{\thisgap}\includegraphics[width=\thiswidth]{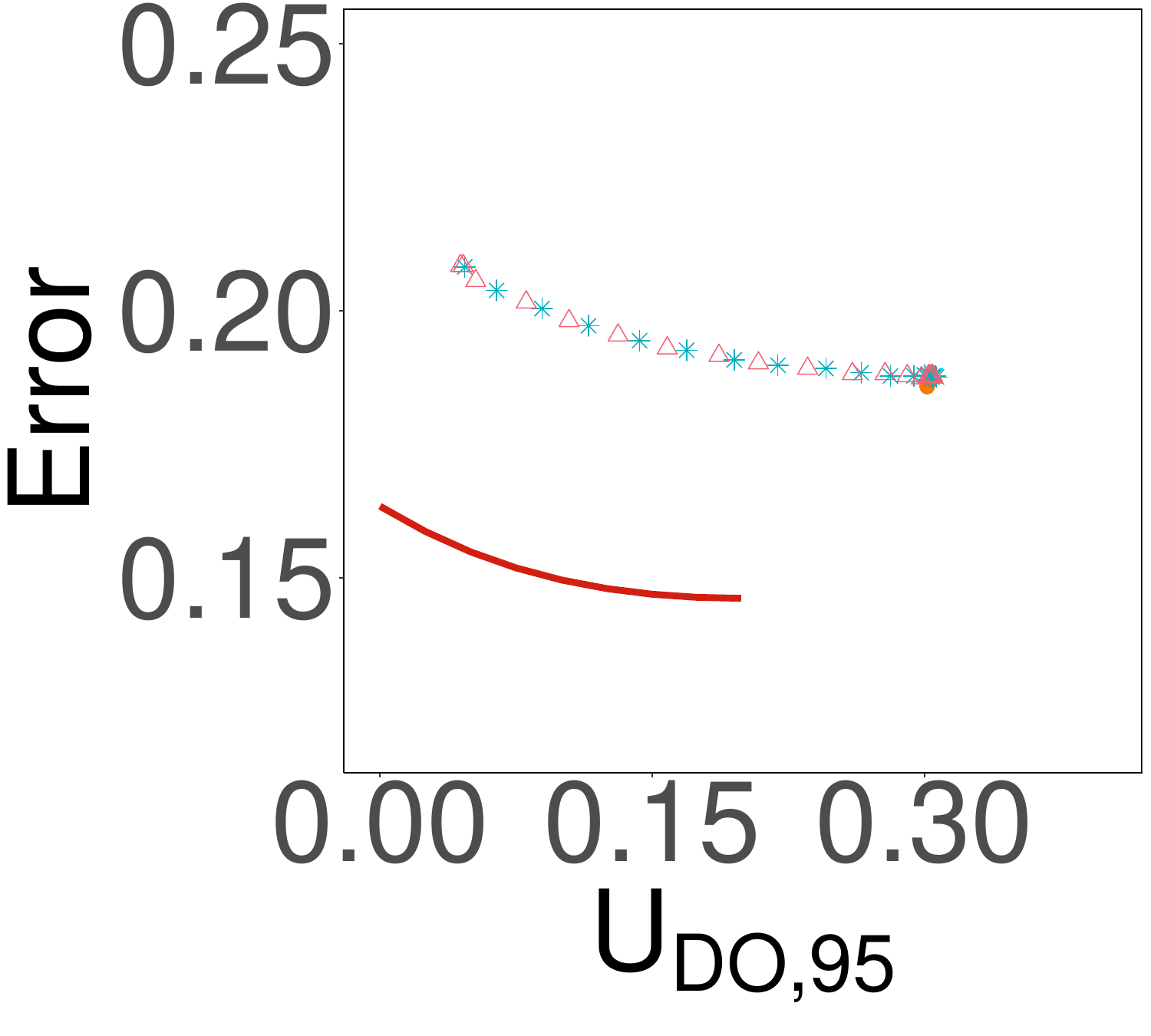} 
		\end{tabular}
		\caption{Error-unfairness trade-off for DO under the Gaussian model, $\beta=1.5$. Left: $n=1000$; middle: $n=2000$; right: $n=5000$.}
		\label{fig:tradeoff_DO_gauss_beta_1.5}
	\end{center}
\end{figure*}

\begin{figure*}[!htbp]
	\begin{center}
		\newcommand{\thiswidth}{0.2\linewidth}
		\newcommand{\thisgap}{0mm}
		\begin{tabular}{ccc}
			\hspace{\thisgap}\includegraphics[width=\thiswidth]{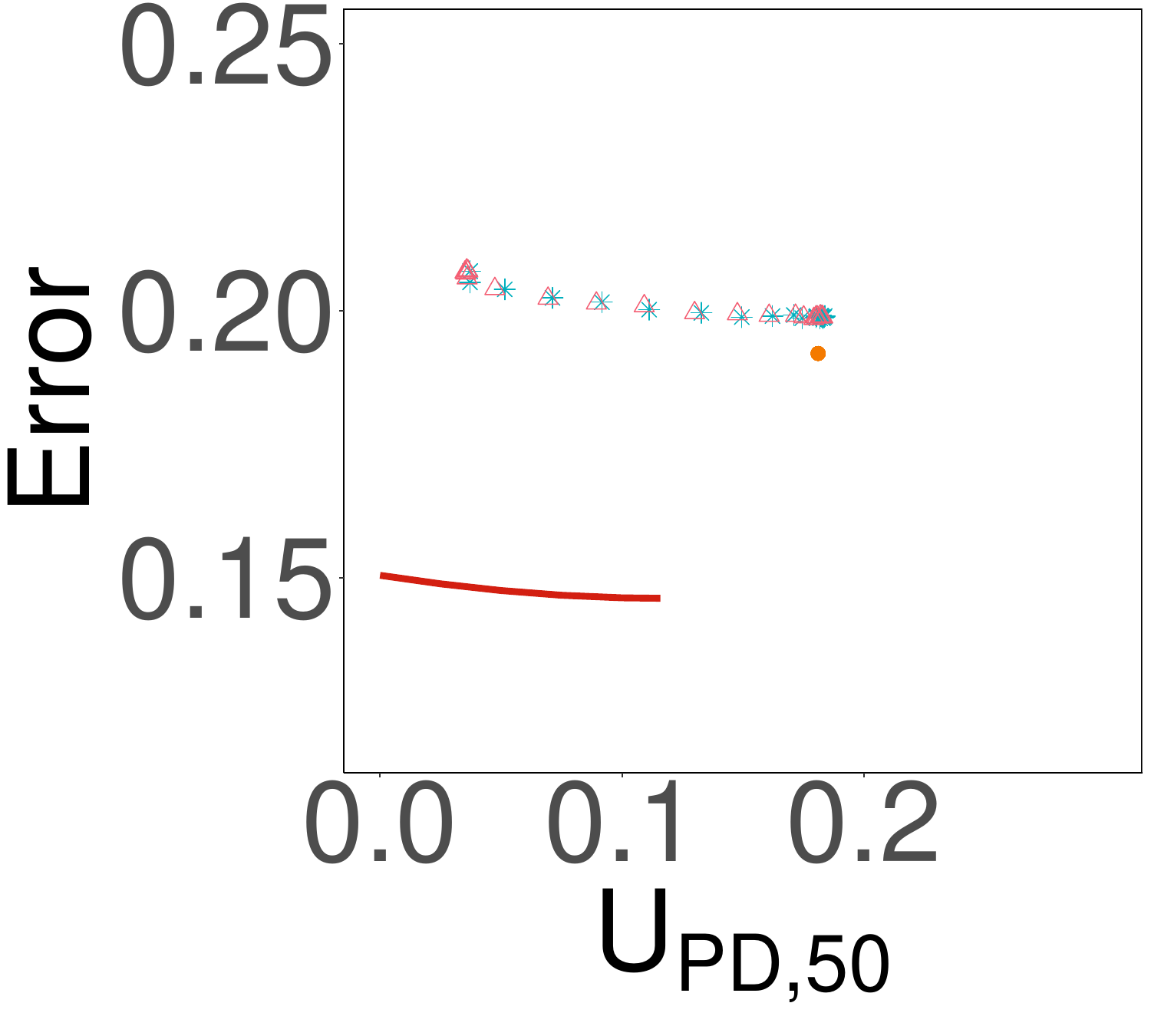} &
            \hspace{\thisgap}\includegraphics[width=\thiswidth]{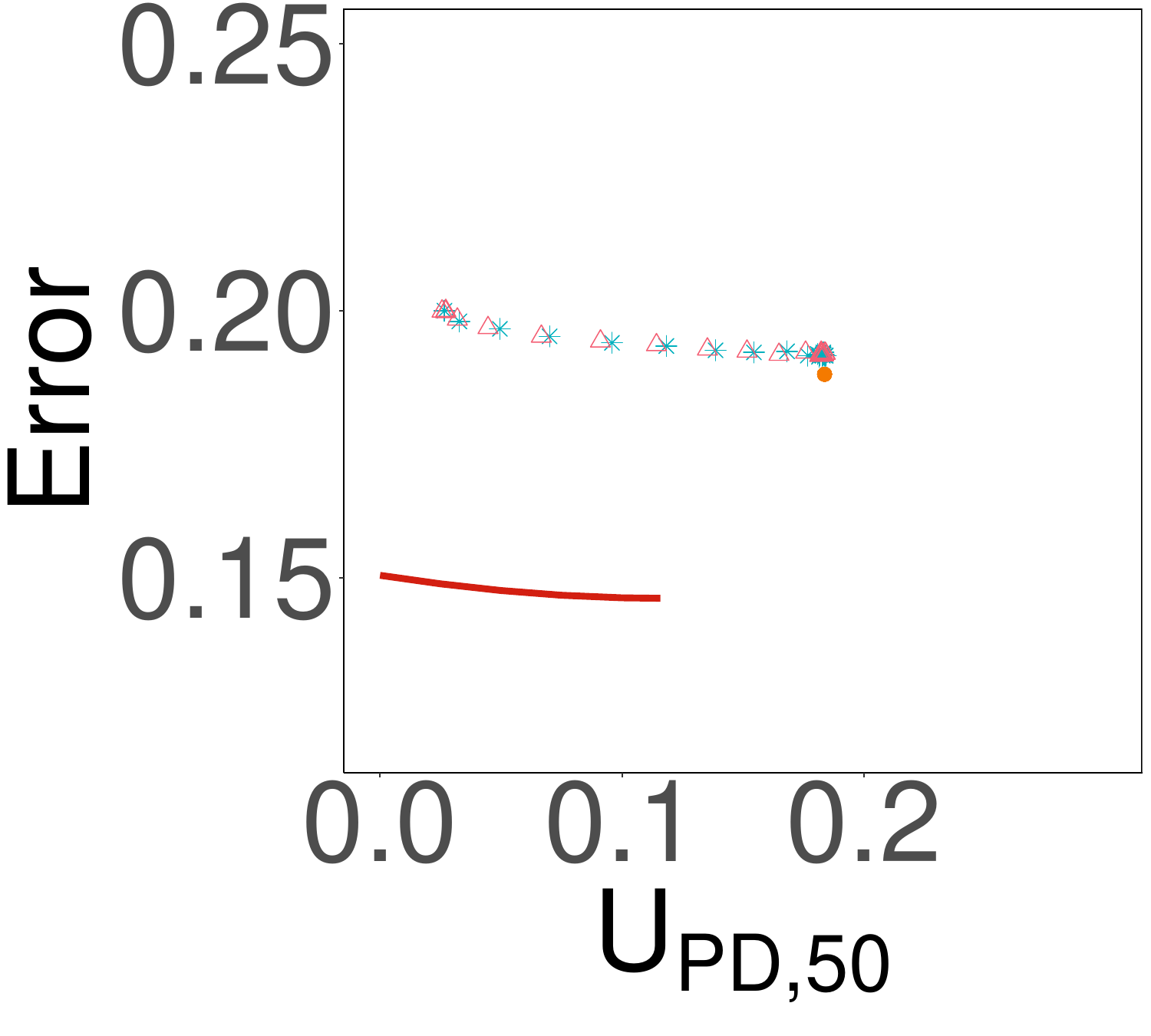} &
            \hspace{\thisgap}\includegraphics[width=\thiswidth]{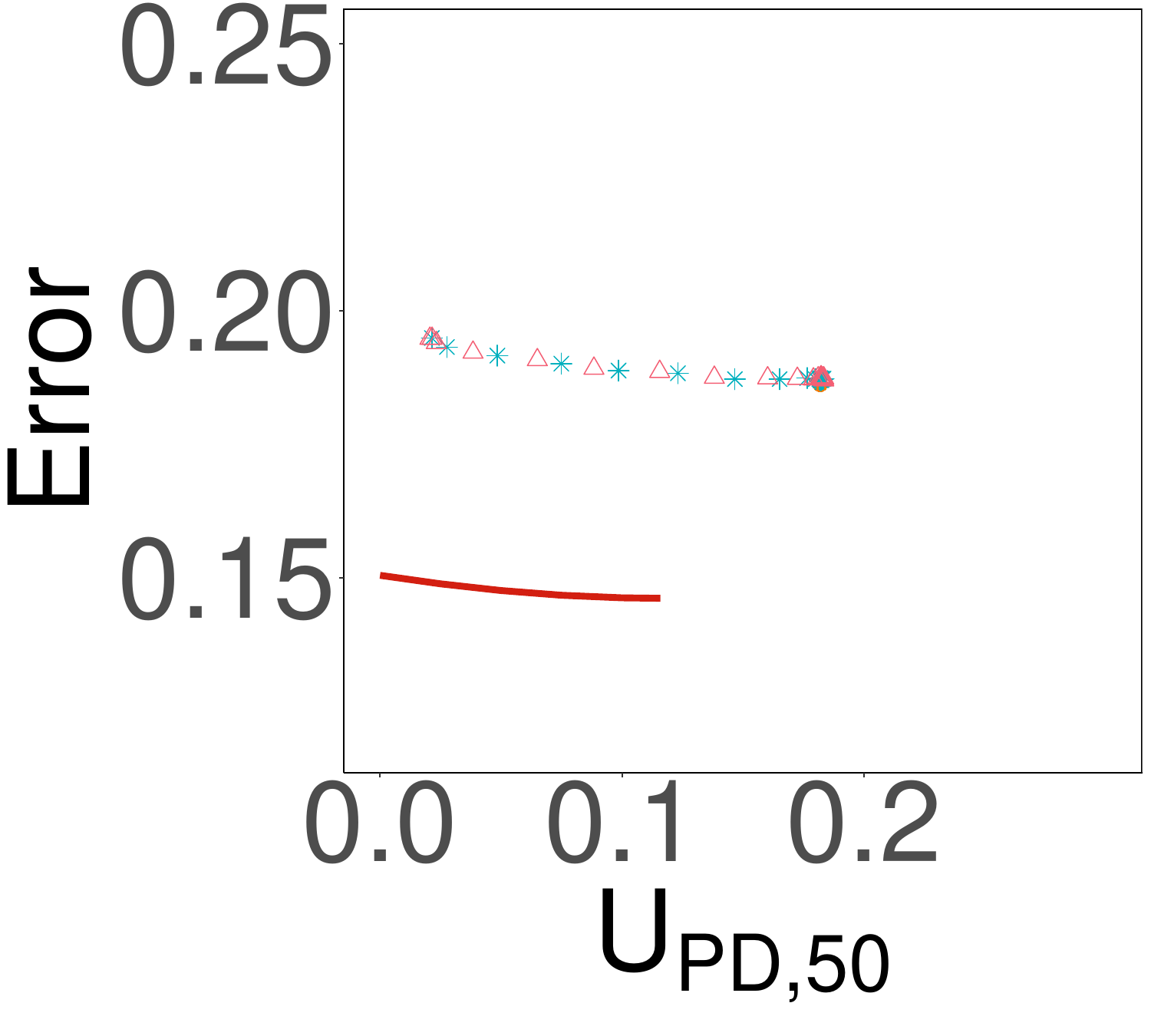} \\
		\hspace{\thisgap}\includegraphics[width=\thiswidth]{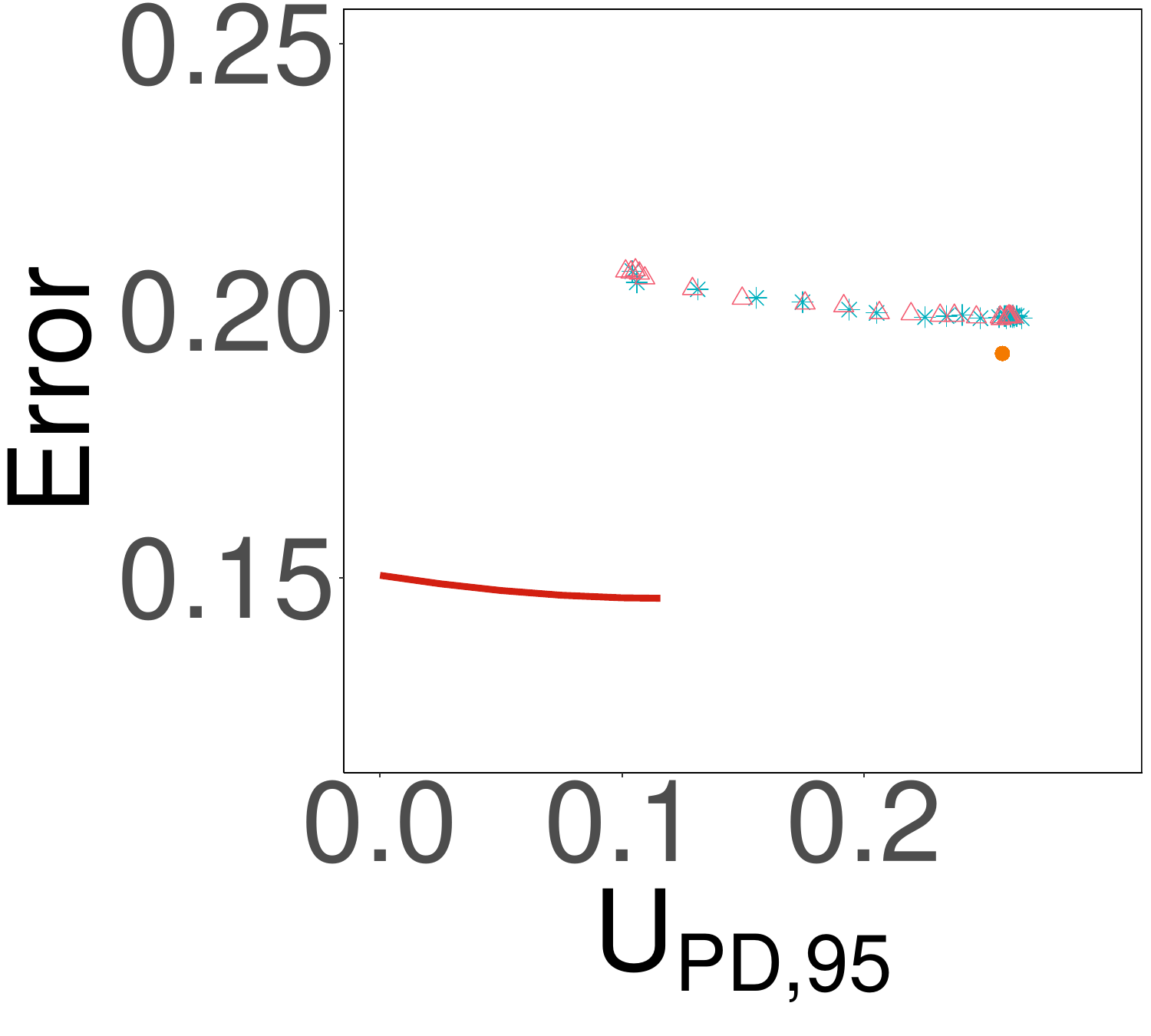} & 
        \hspace{\thisgap}\includegraphics[width=\thiswidth]{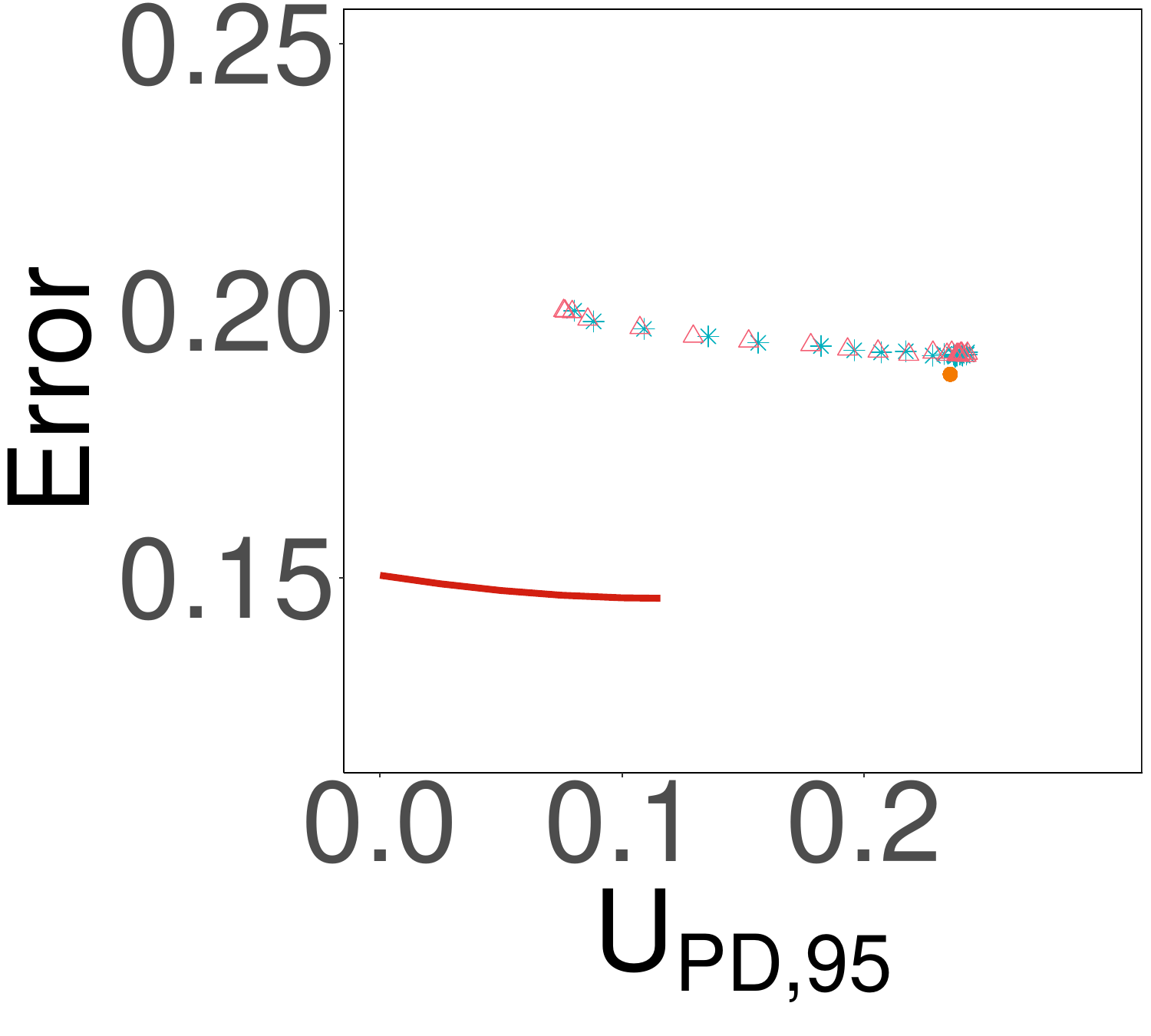} &
        \hspace{\thisgap}\includegraphics[width=\thiswidth]{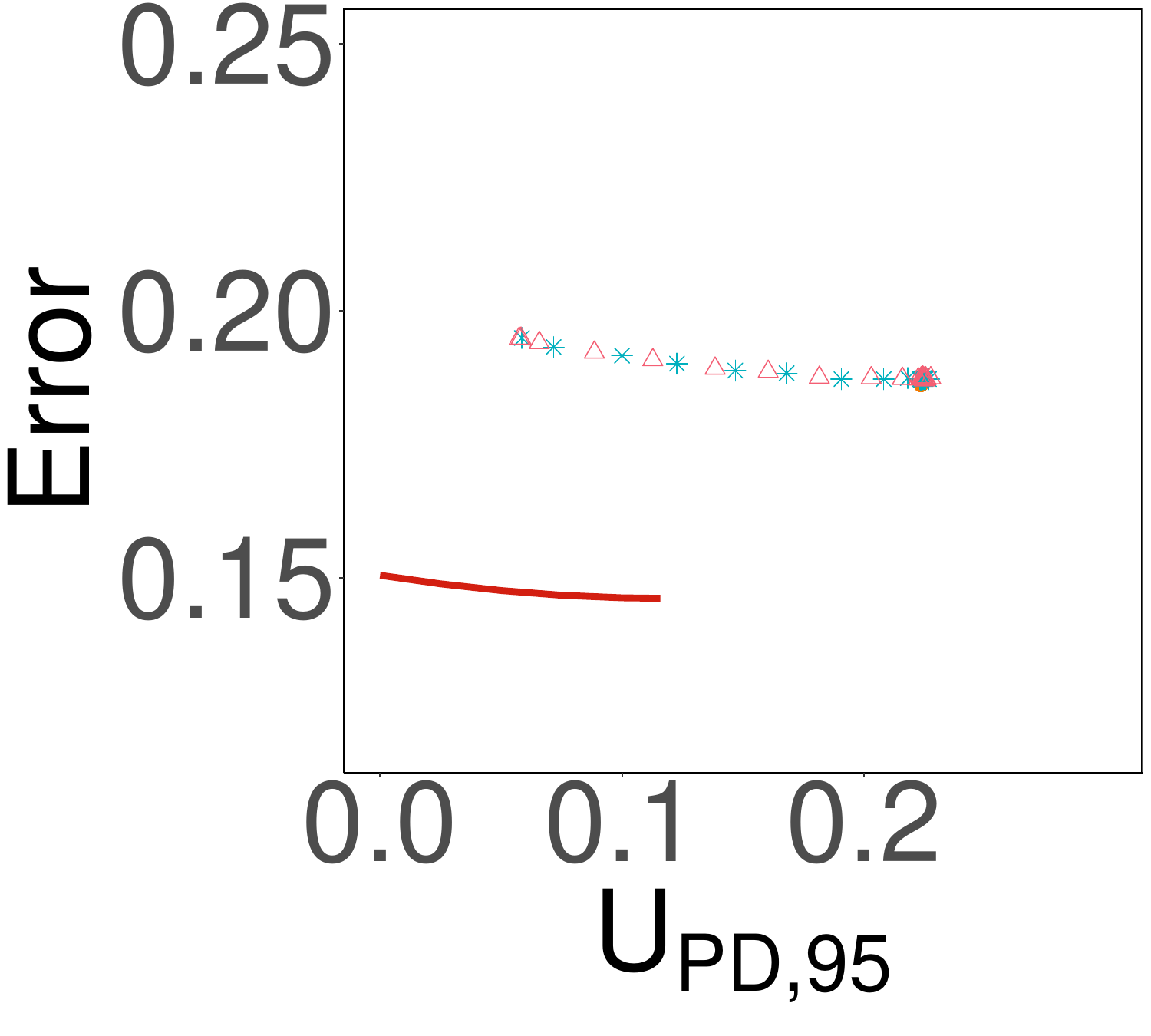}
		\end{tabular}
		\caption{Error-unfairness trade-off for PD under the Gaussian model, $\beta=1.5$. Left: $n=1000$; middle: $n=2000$; right: $n=5000$.}
		\label{fig:tradeoff_PD_gauss_beta_1.5}
	\end{center}
\end{figure*}

\begin{figure*}[!htbp]
	\begin{center}
		\newcommand{\thiswidth}{0.2\linewidth}
		\newcommand{\thisgap}{0mm}
		\begin{tabular}{ccc}
			\hspace{\thisgap}\includegraphics[width=\thiswidth]{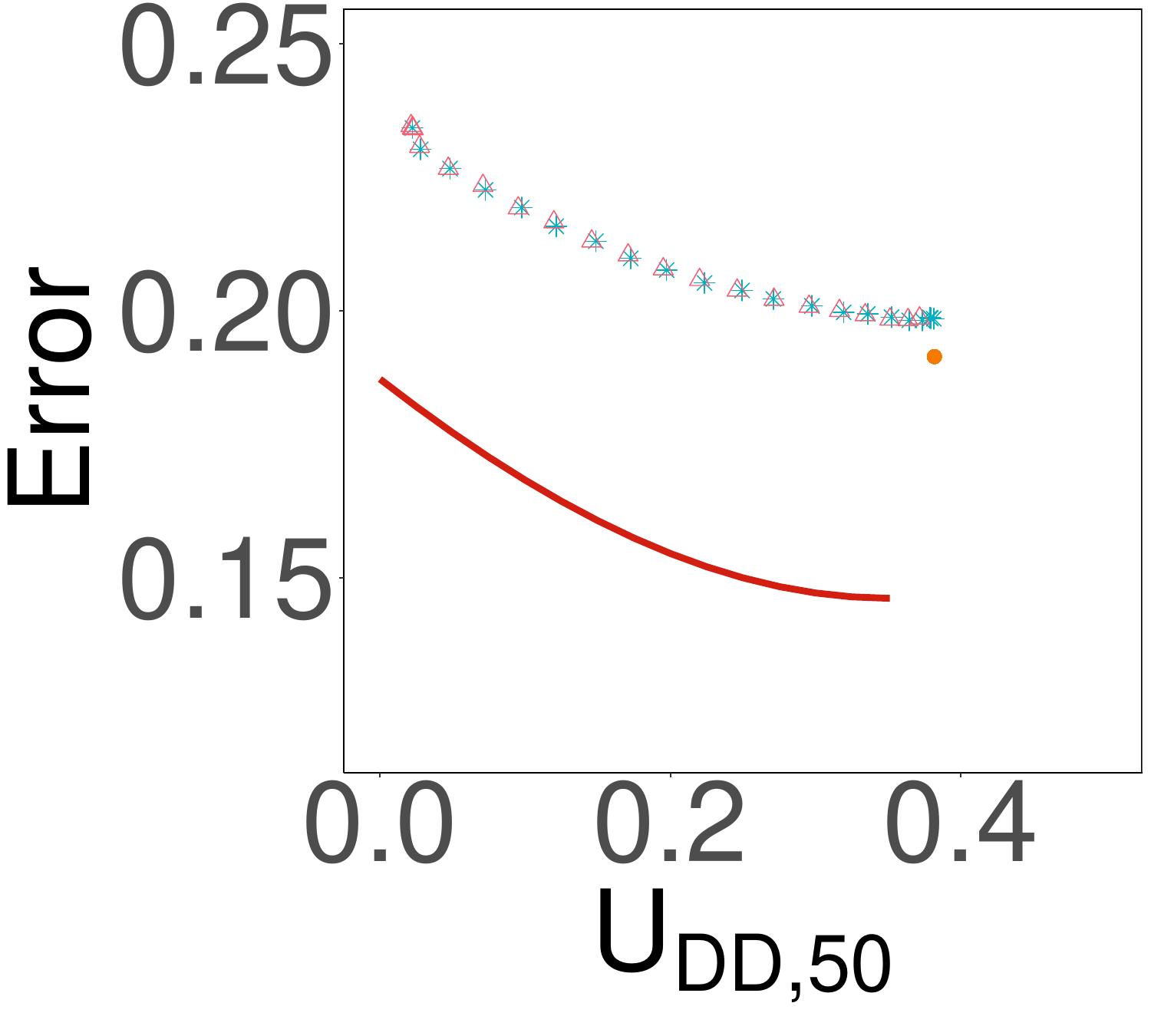} &
            \hspace{\thisgap}\includegraphics[width=\thiswidth]{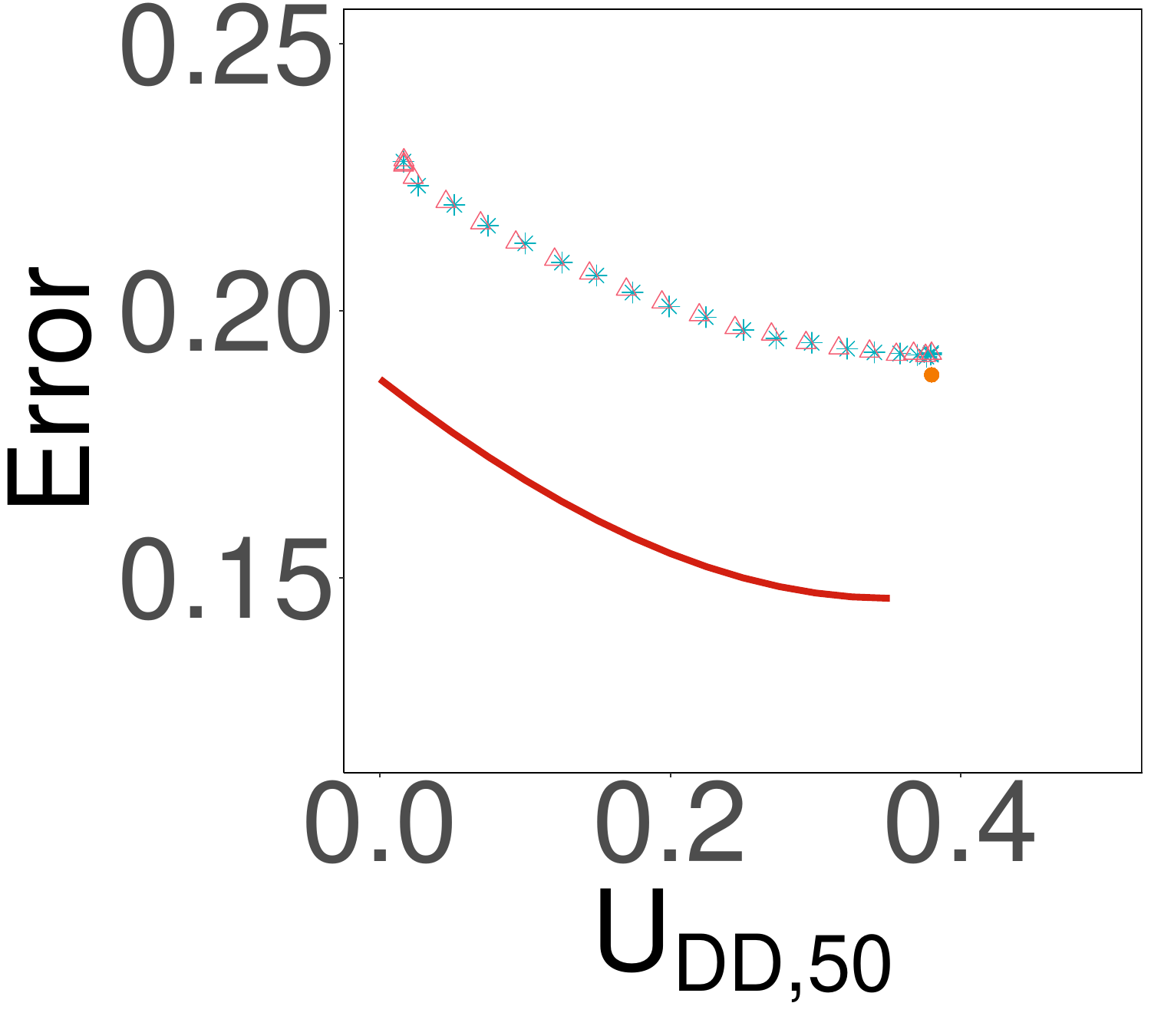} &
            \hspace{\thisgap}\includegraphics[width=\thiswidth]{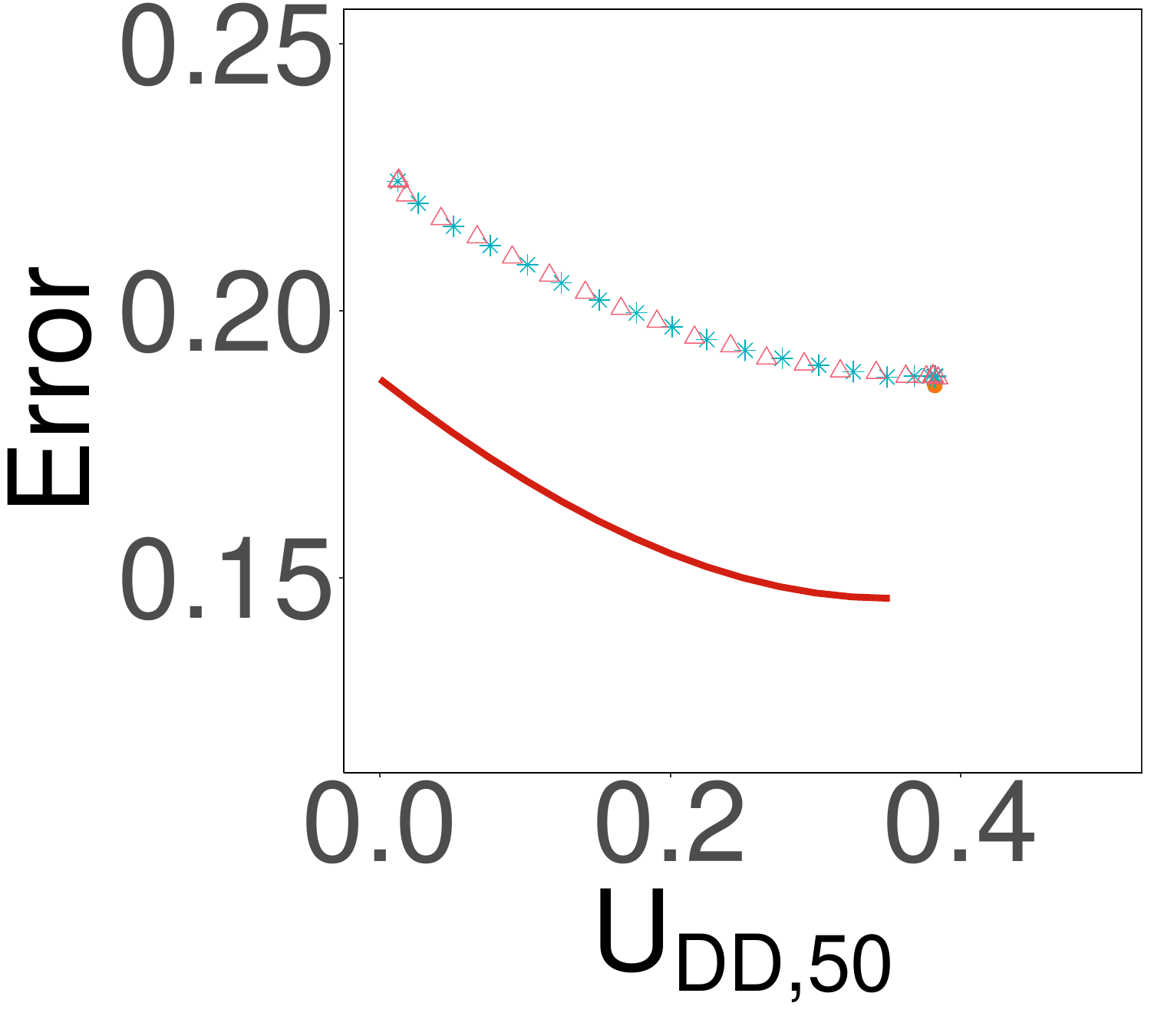} \\
		\hspace{\thisgap}\includegraphics[width=\thiswidth]{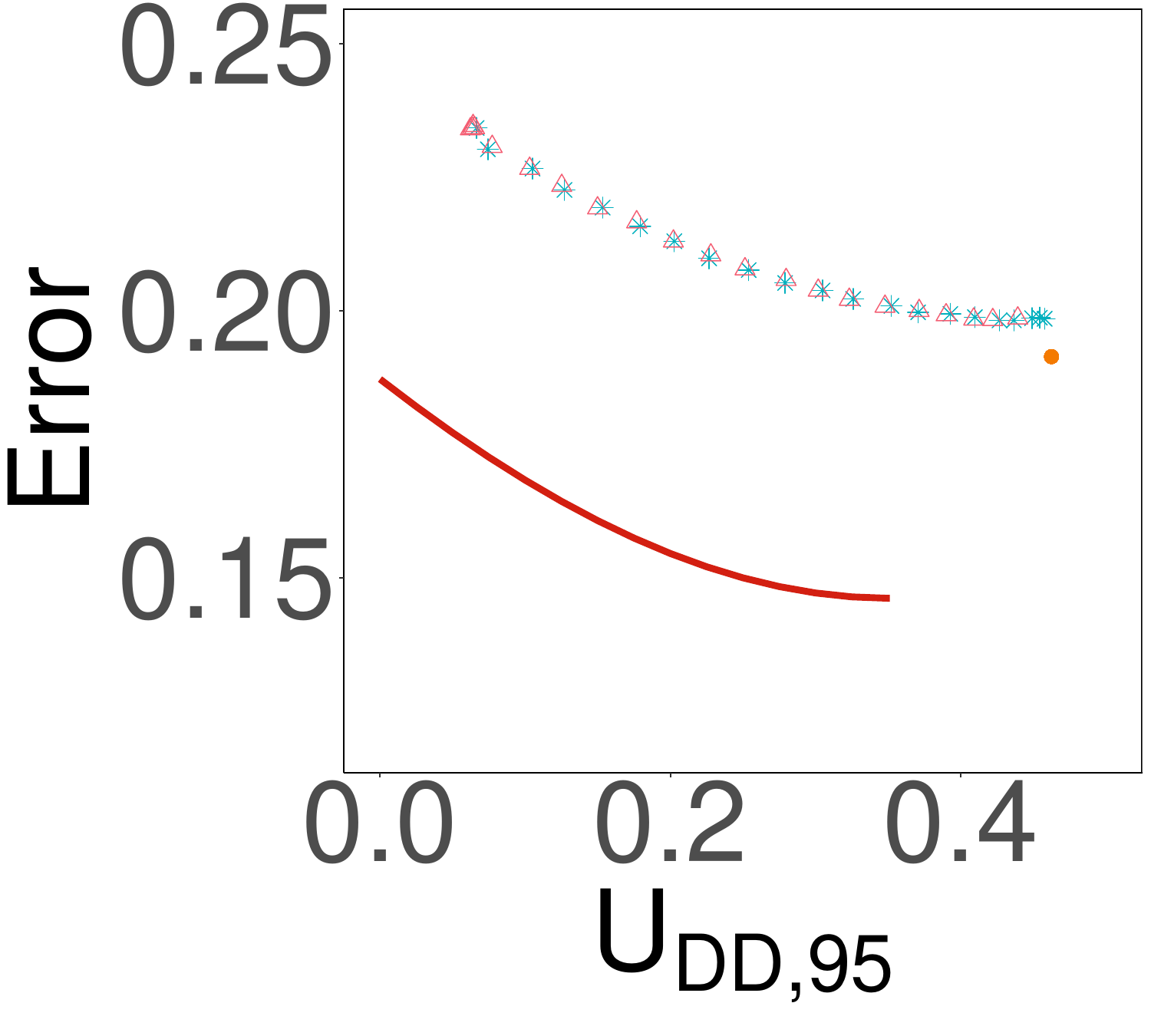} &
		\hspace{\thisgap}\includegraphics[width=\thiswidth]{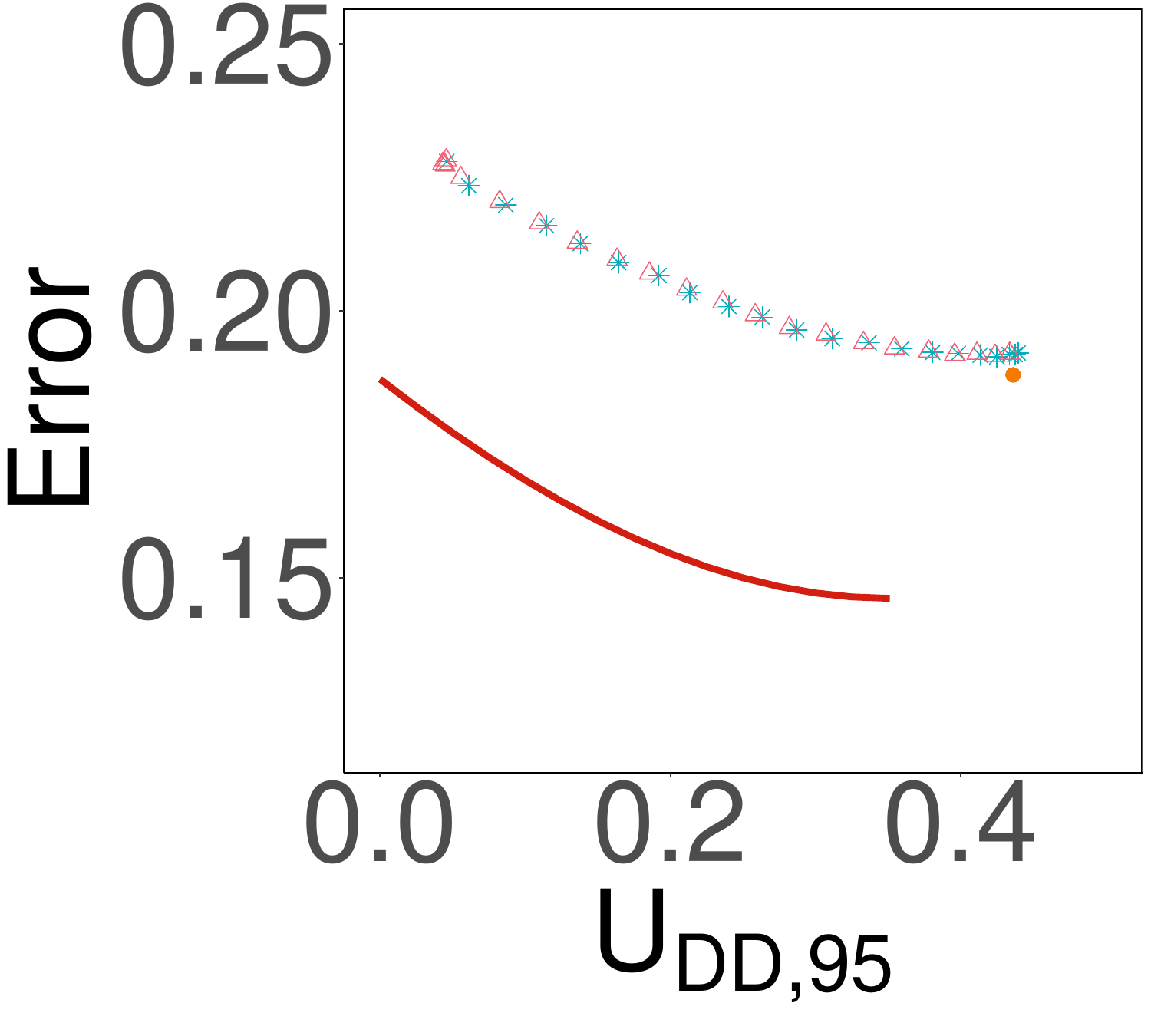} &
        \hspace{\thisgap}\includegraphics[width=\thiswidth]{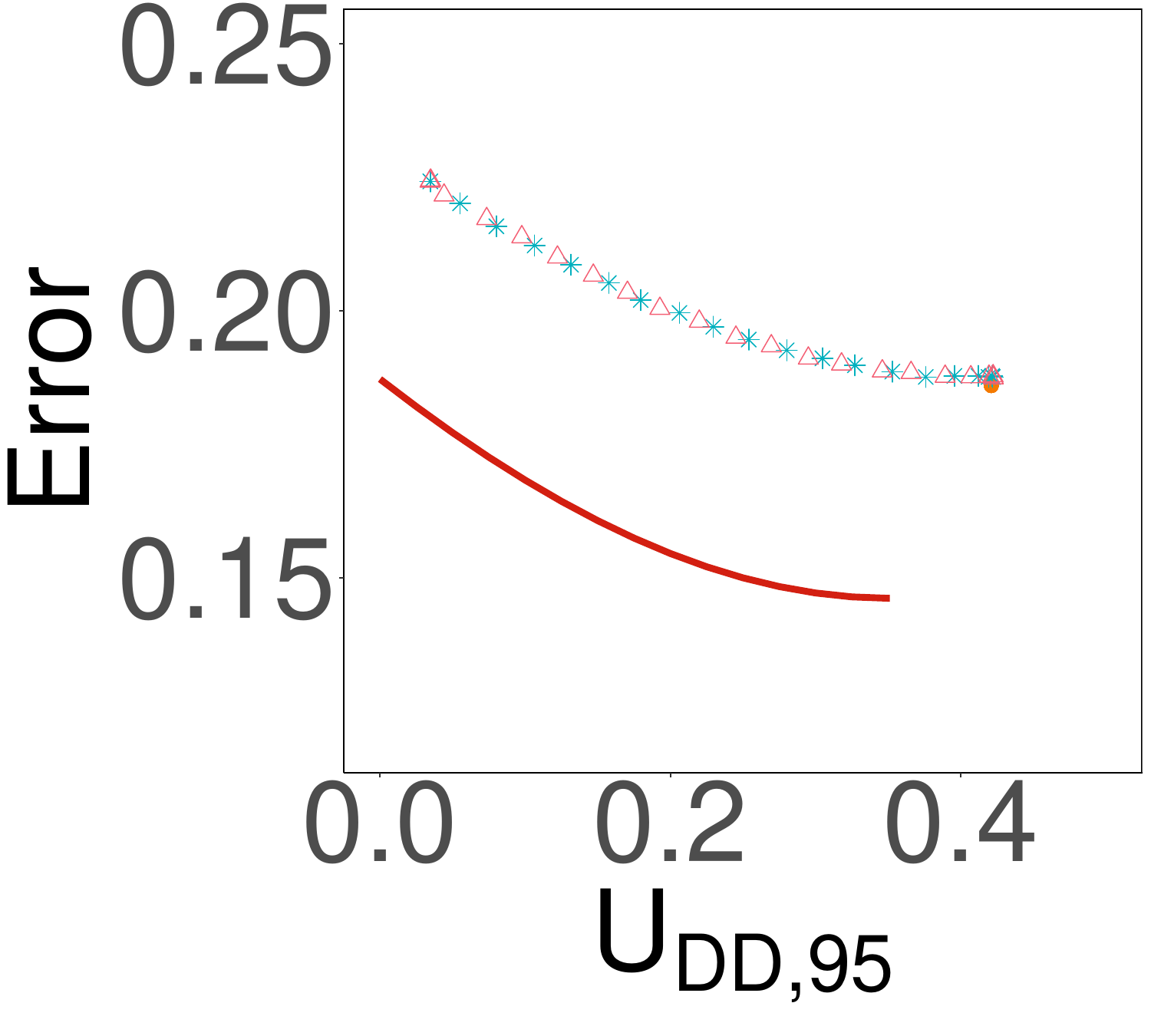}	
		\end{tabular}
		\caption{Error-unfairness trade-off for DD under the Gaussian model, $\beta=1.5$. Left: $n=1000$; middle: $n=2000$; right: $n=5000$.}
		\label{fig:tradeoff_DD_gauss_beta_1.5}
	\end{center}
\end{figure*}

\paragraph{Results under $\beta=2$.}
We evaluate the methods under the Gaussian model with $\beta=2$, while keeping other model parameters consistent with those in Section \ref{sec:sim}. 
The patterns of disparity control are similar to those observed under $\beta=1.5$. As illustrated in Figures \ref{fig:DO_gauss_beta_2}-\ref{fig:DD_gauss_beta_2},  the classification errors are generally higher compared to the case of $\beta=1.5$,  due to the lower signal-to-noise ratio. Nonetheless, the excess risk decreases more rapidly with increasing $n$ for larger values of $\beta$, aligning well with our theoretical results.

\begin{figure*}[!htbp]
	\begin{center}
		\newcommand{\thiswidth}{0.2\linewidth}
		\newcommand{\thisgap}{0mm}
		\begin{tabular}{ccc}
			\hspace{\thisgap}\includegraphics[width=\thiswidth]{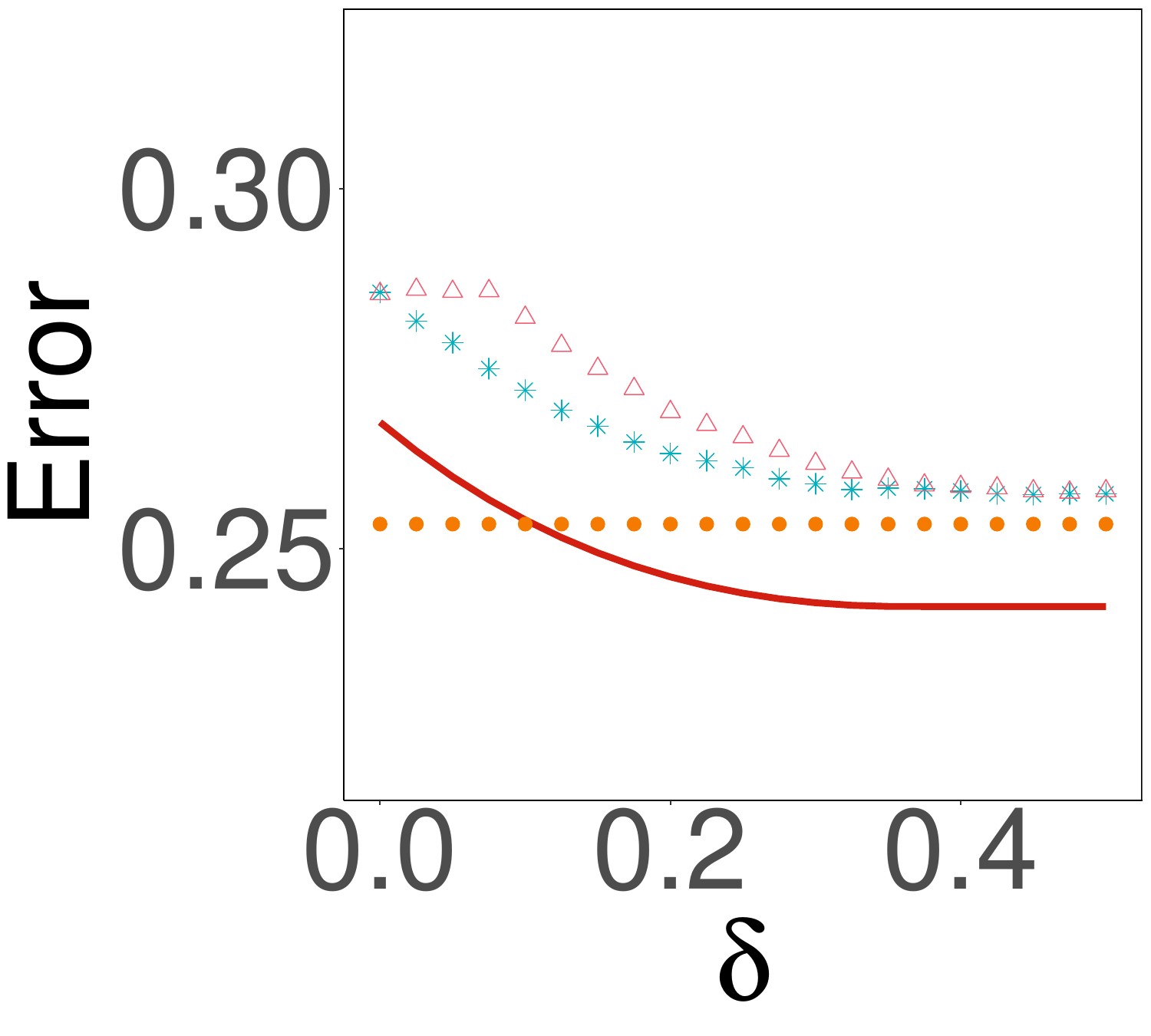} &
			\hspace{\thisgap}\includegraphics[width=\thiswidth]{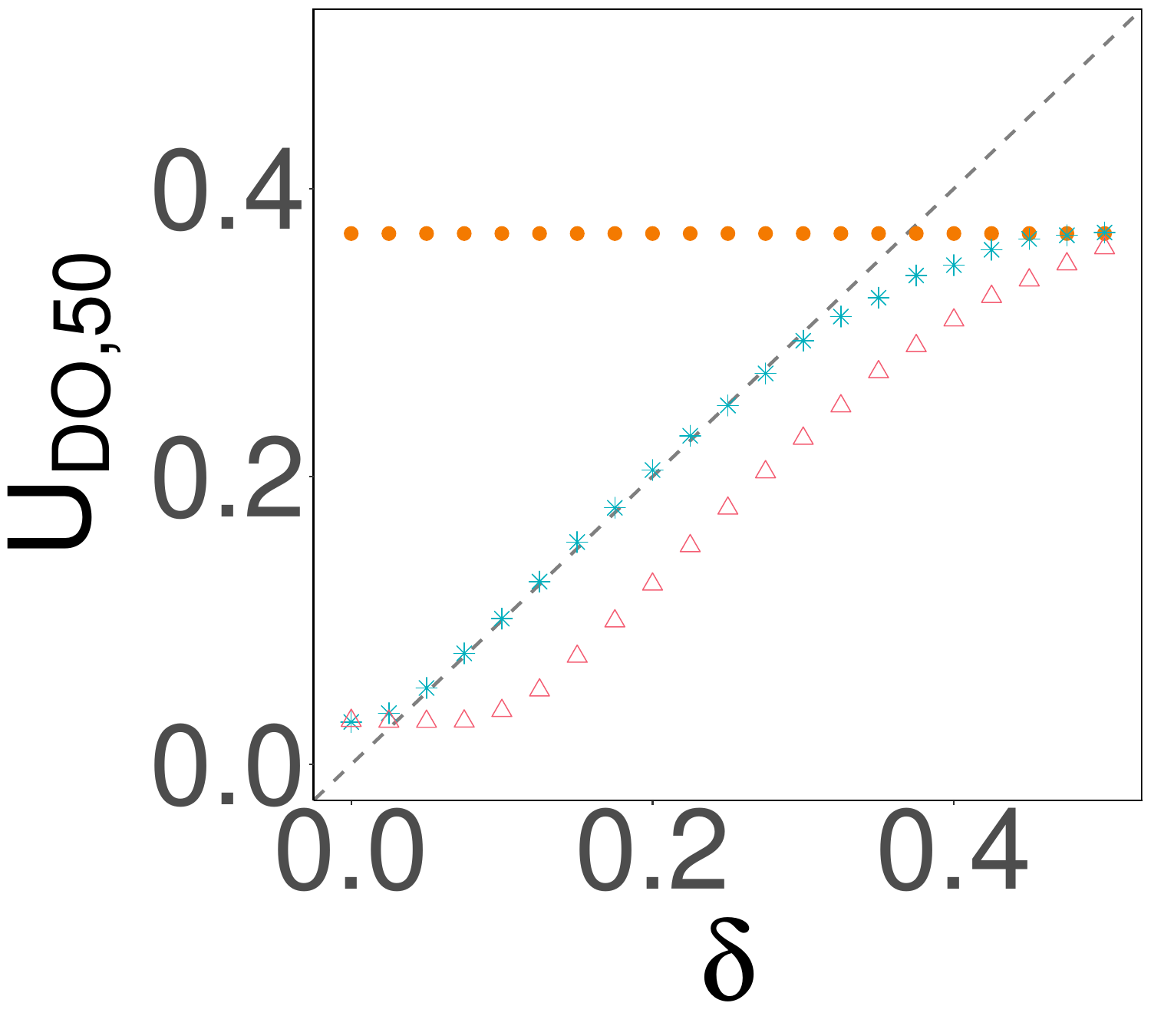} &
			\hspace{\thisgap}\includegraphics[width=\thiswidth]{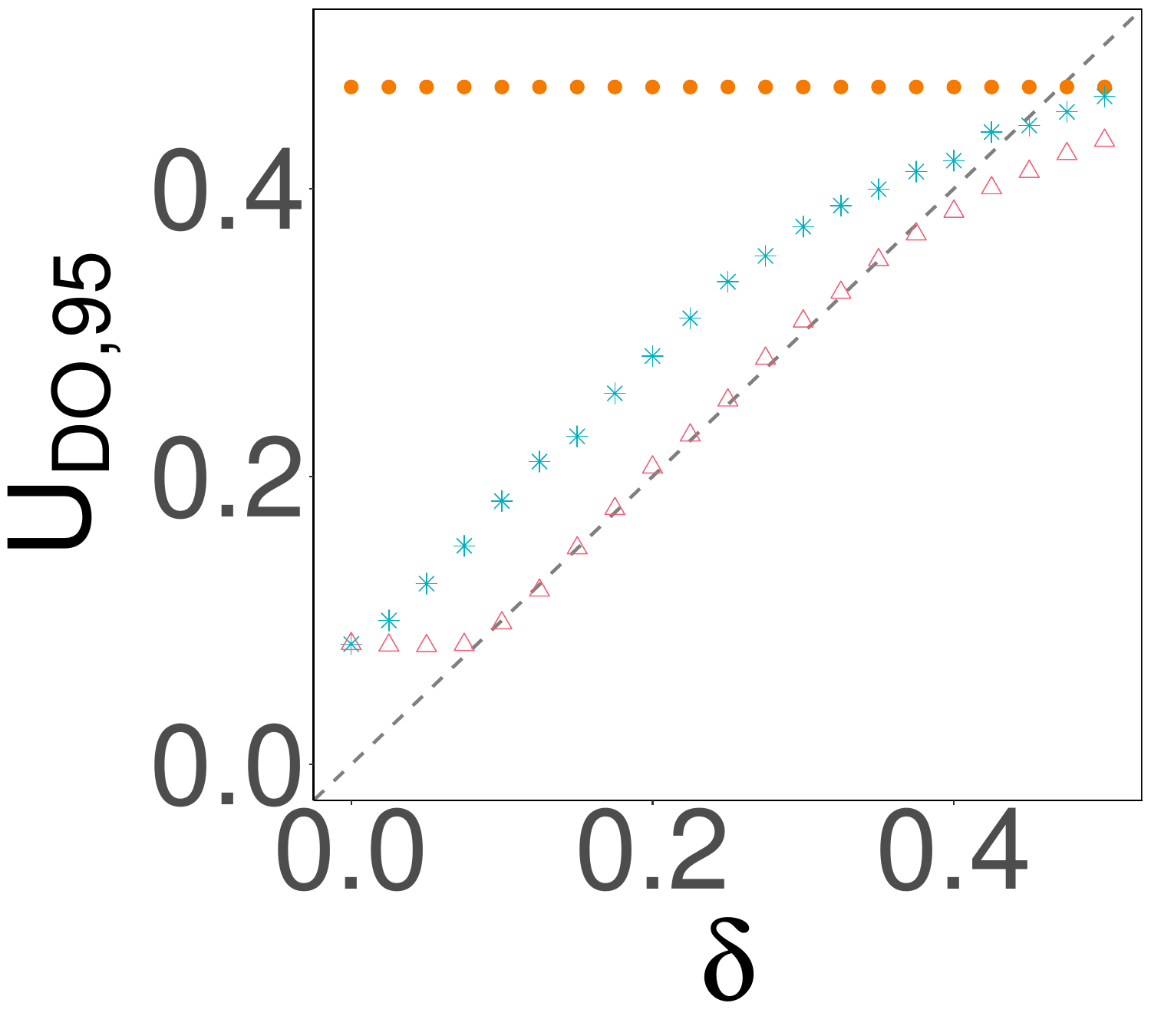} \\
			\hspace{\thisgap}\includegraphics[width=\thiswidth]{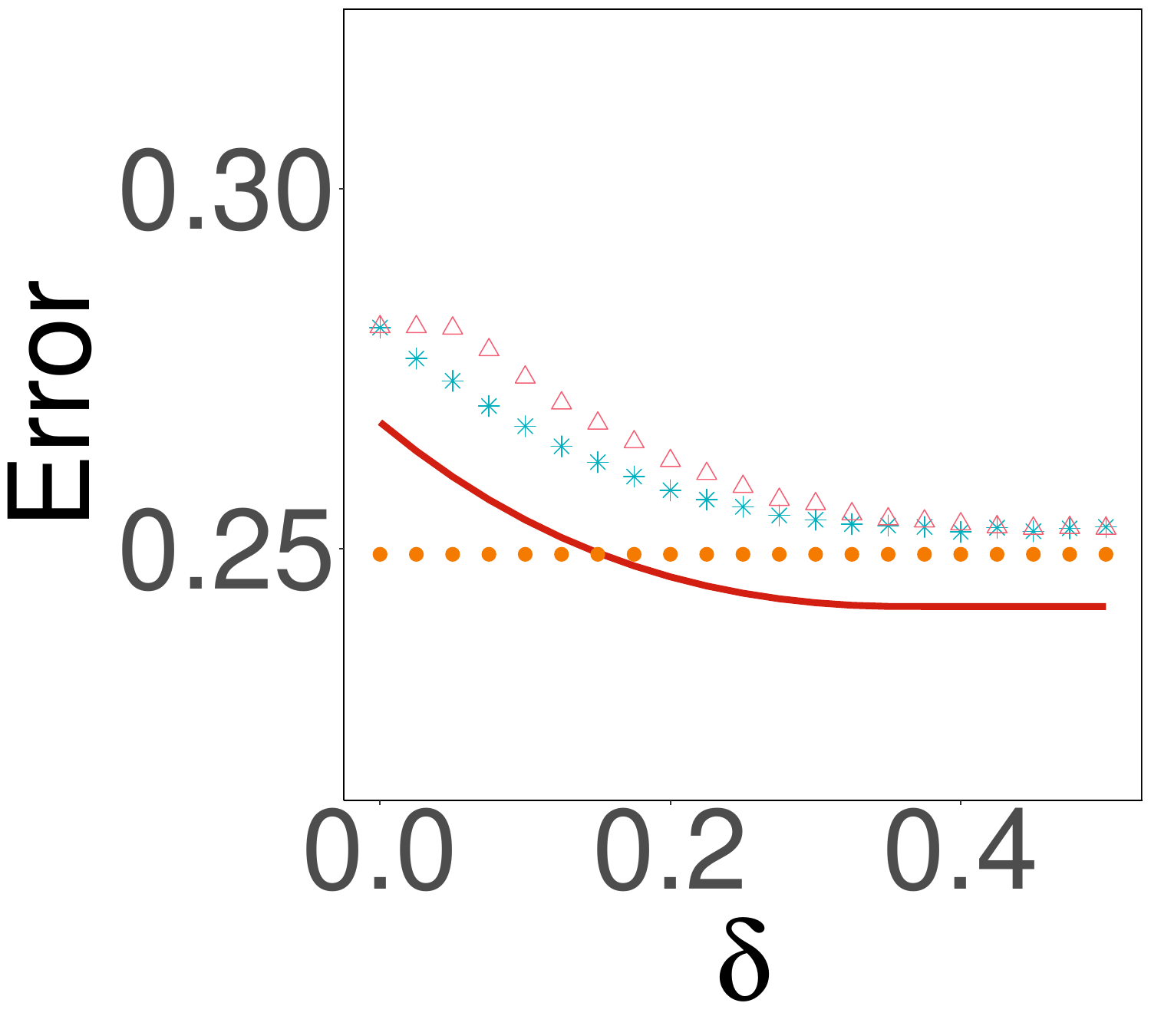} &
			\hspace{\thisgap}\includegraphics[width=\thiswidth]{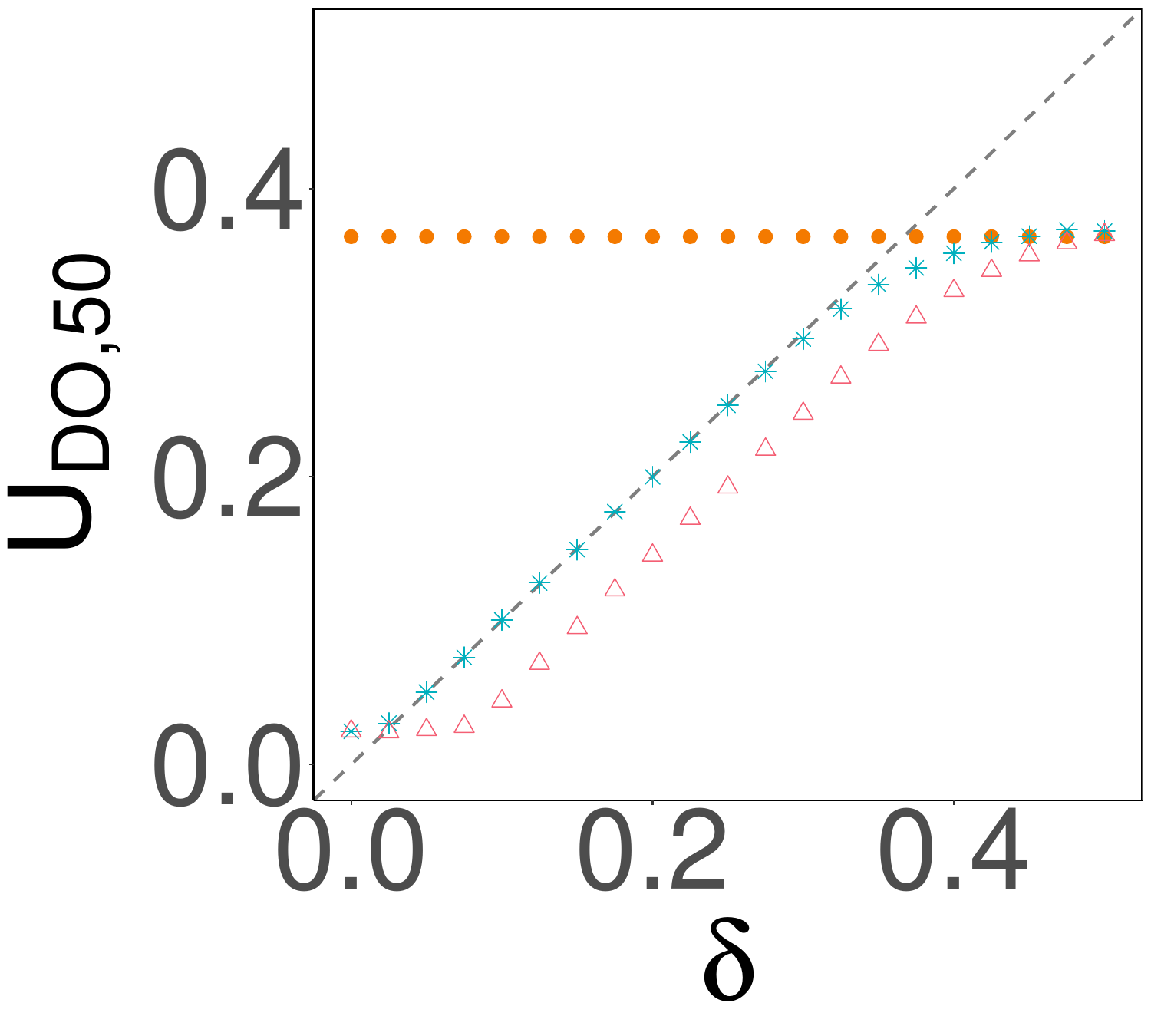} &
			\hspace{\thisgap}\includegraphics[width=\thiswidth]{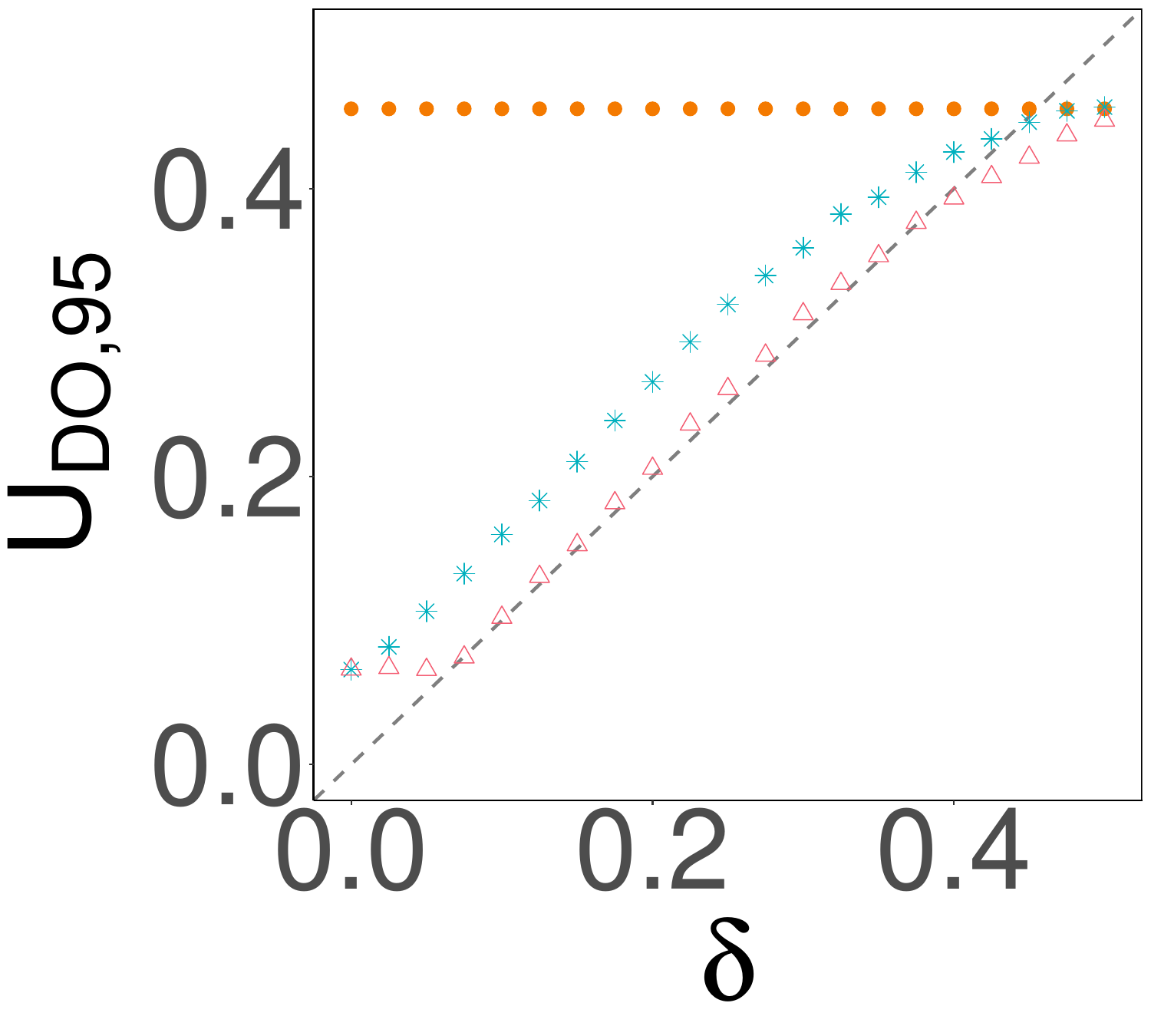} \\
			\hspace{\thisgap}\includegraphics[width=\thiswidth]{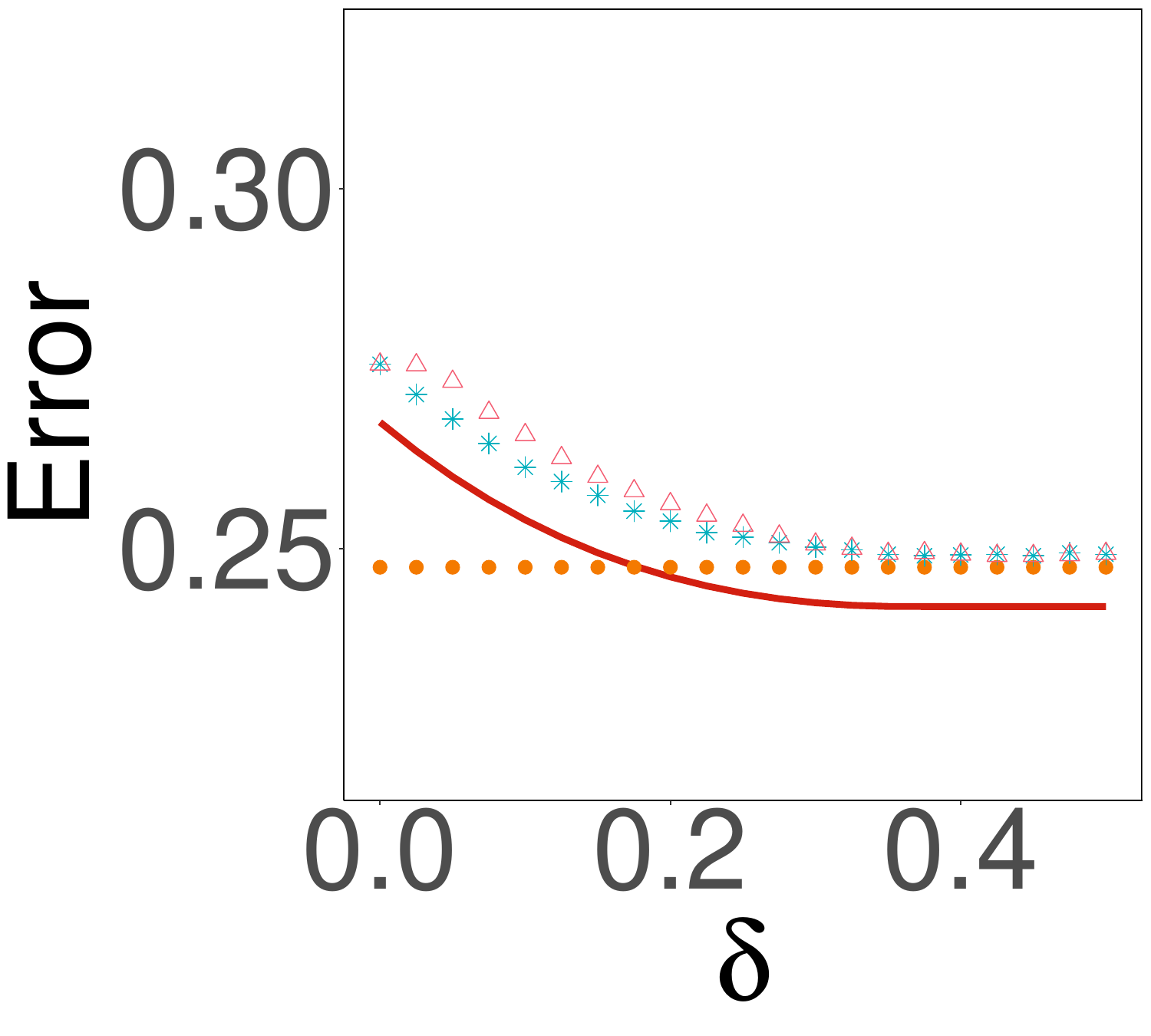} &
			\hspace{\thisgap}\includegraphics[width=\thiswidth]{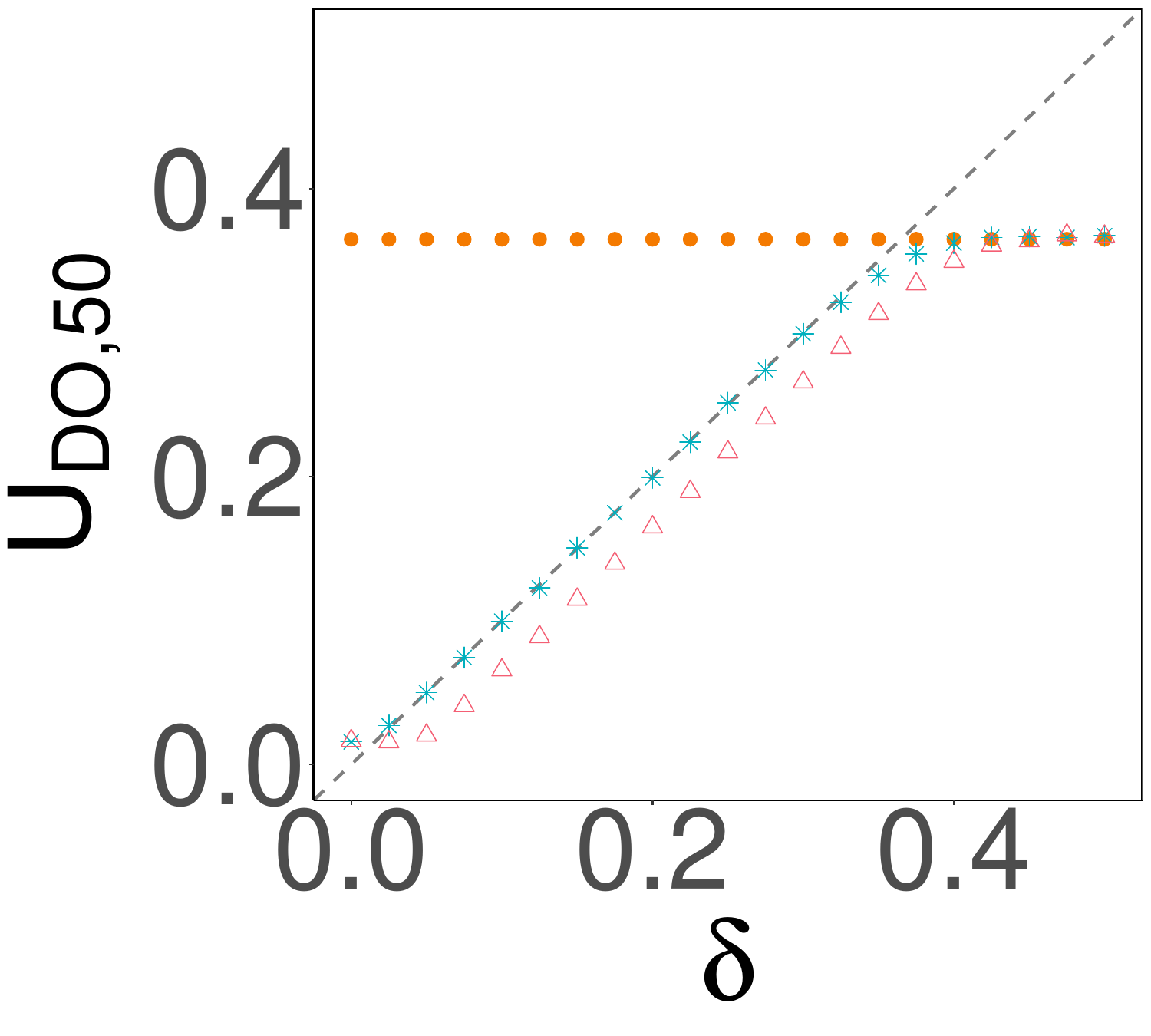} &
			\hspace{\thisgap}\includegraphics[width=\thiswidth]{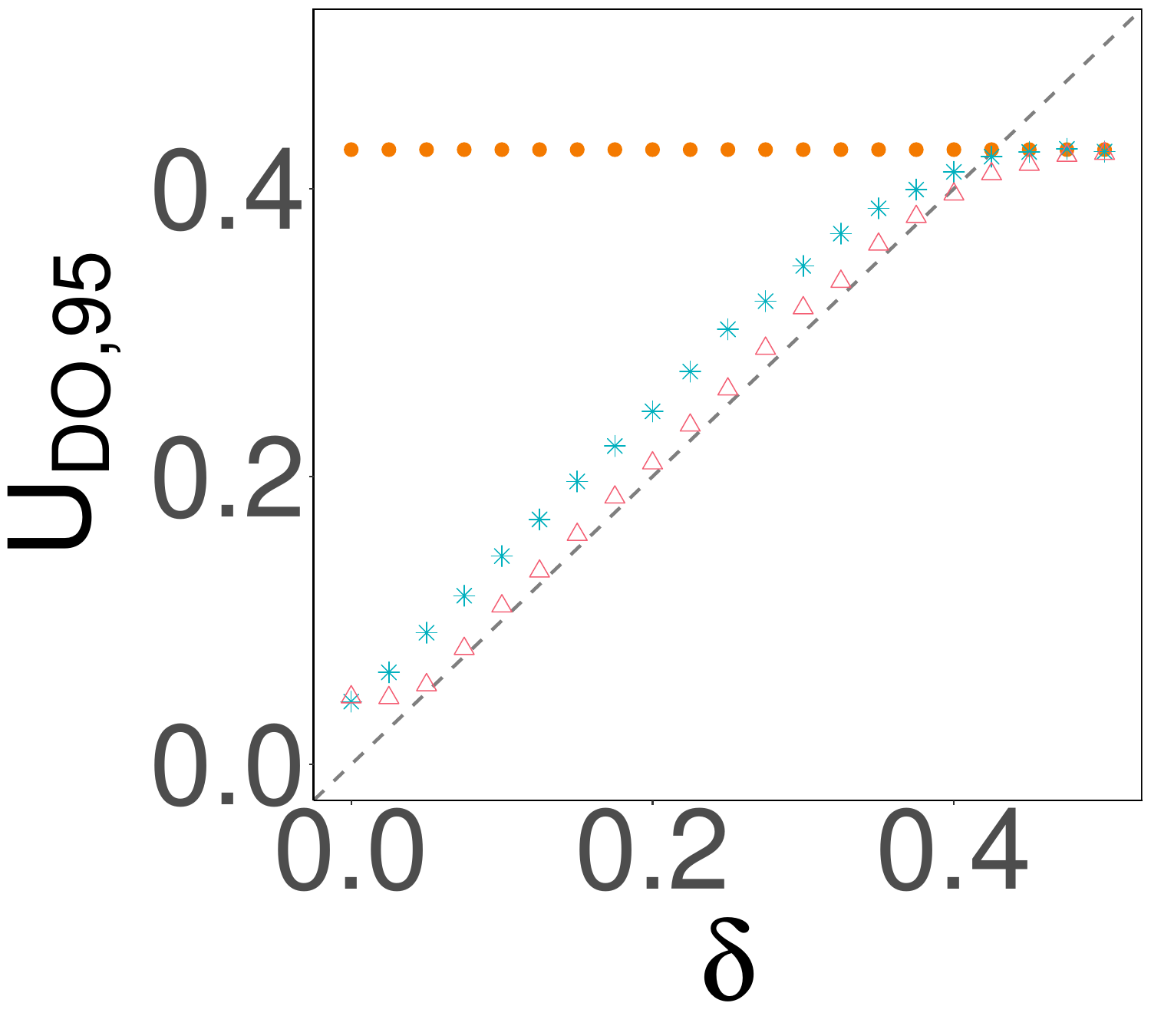}
		\end{tabular}
		\caption{Disparity DO results under the Gaussian model, $\beta=2$. Top: $n=1000$; middle: $n=2000$; bottom: $n=5000$}
		\label{fig:DO_gauss_beta_2}
	\end{center}
\end{figure*}

\begin{figure*}[!htbp]
	\begin{center}
		\newcommand{\thiswidth}{0.18\linewidth}
		\newcommand{\thisgap}{0mm}
		\begin{tabular}{ccc}
			\hspace{\thisgap}\includegraphics[width=\thiswidth]{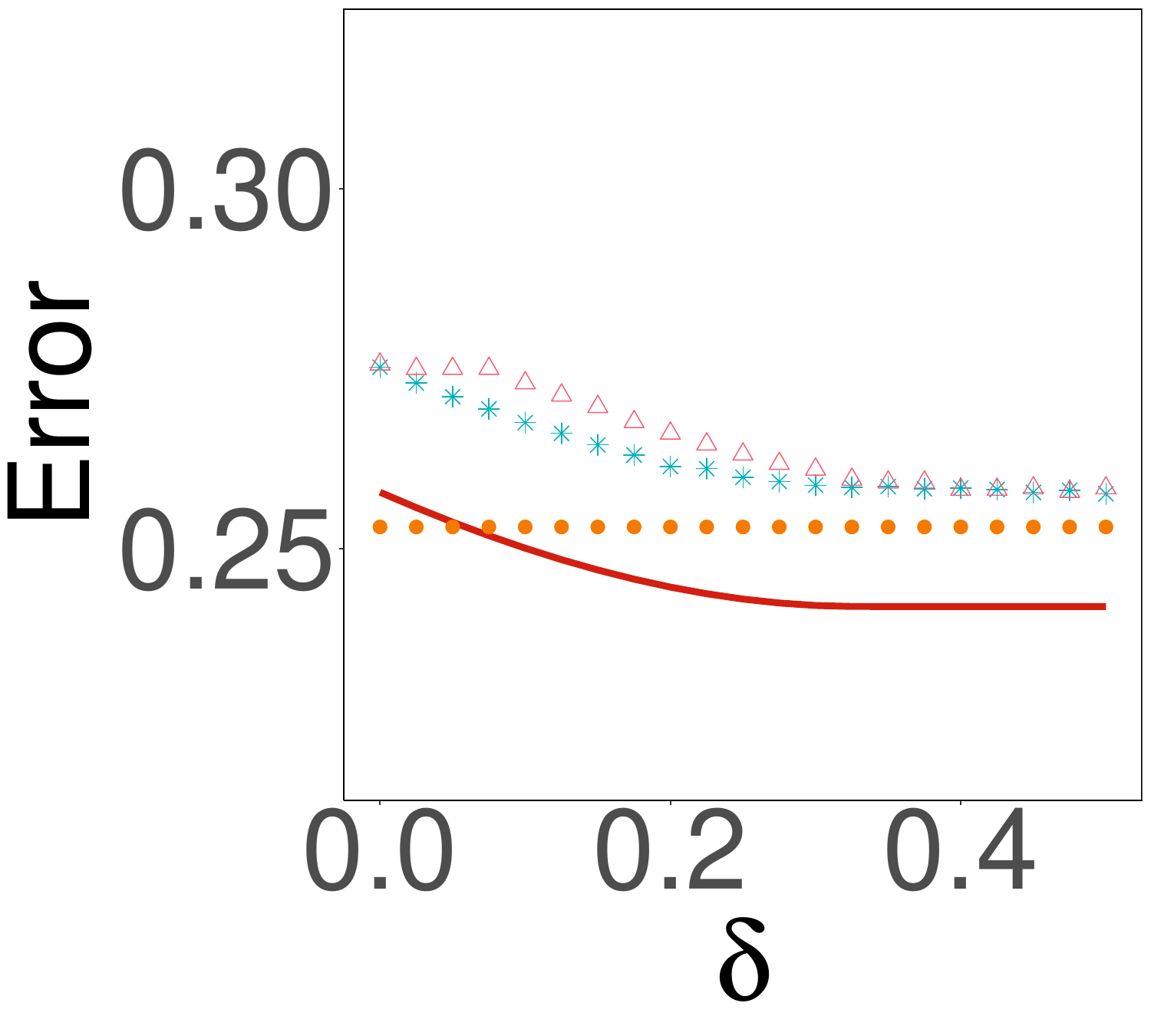} &
			\hspace{\thisgap}\includegraphics[width=\thiswidth]{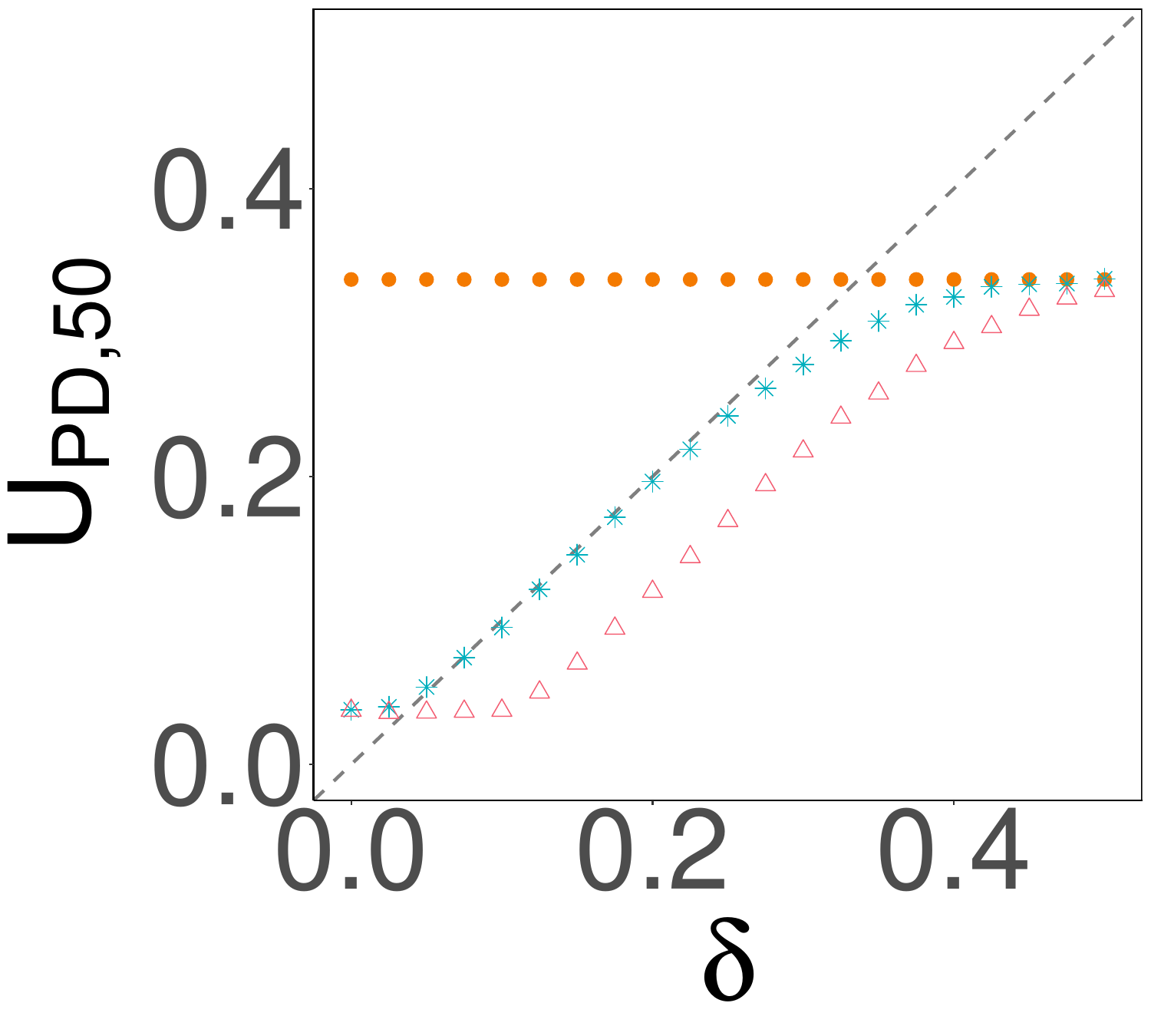} &
			\hspace{\thisgap}\includegraphics[width=\thiswidth]{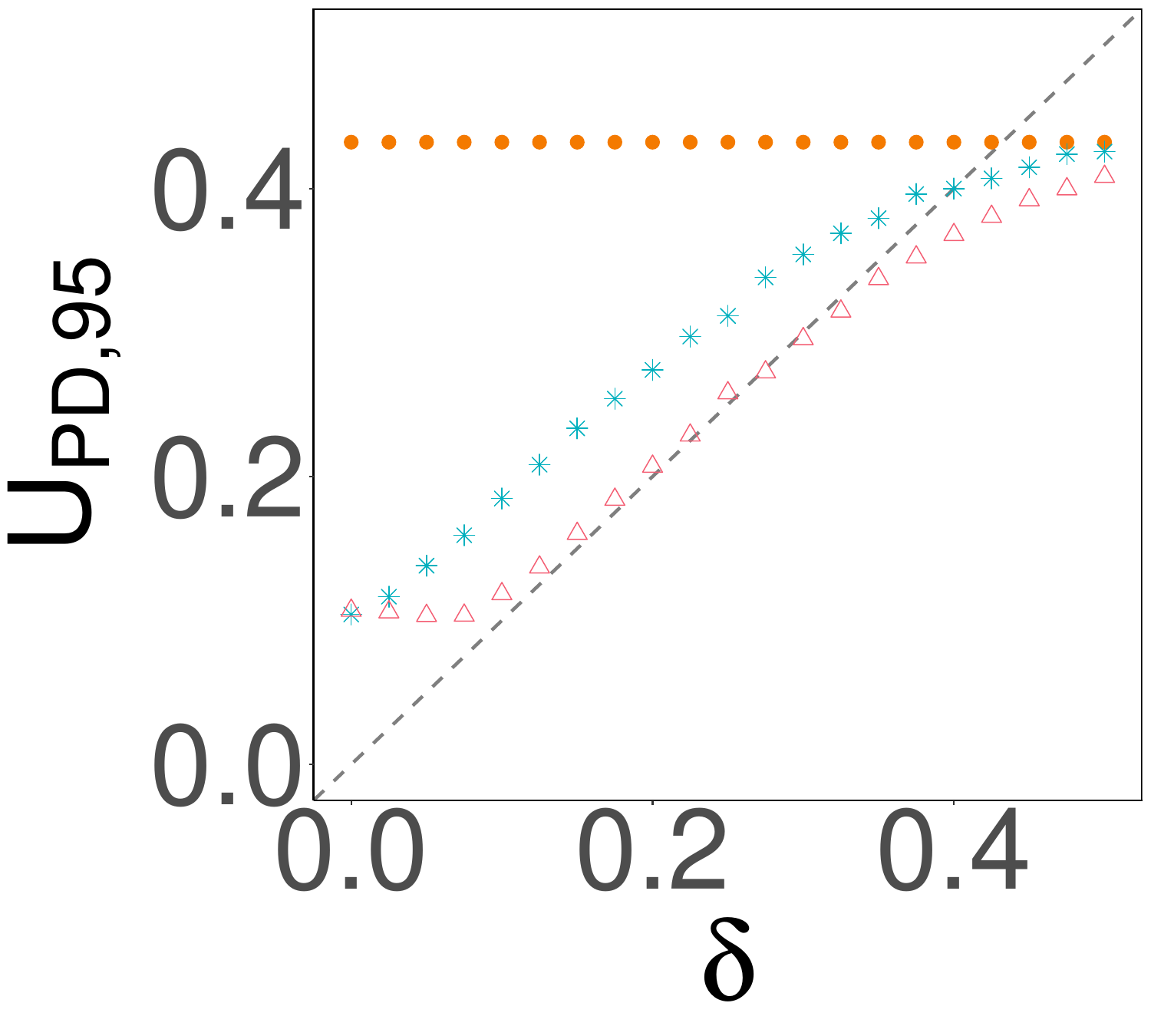} \\
			\hspace{\thisgap}\includegraphics[width=\thiswidth]{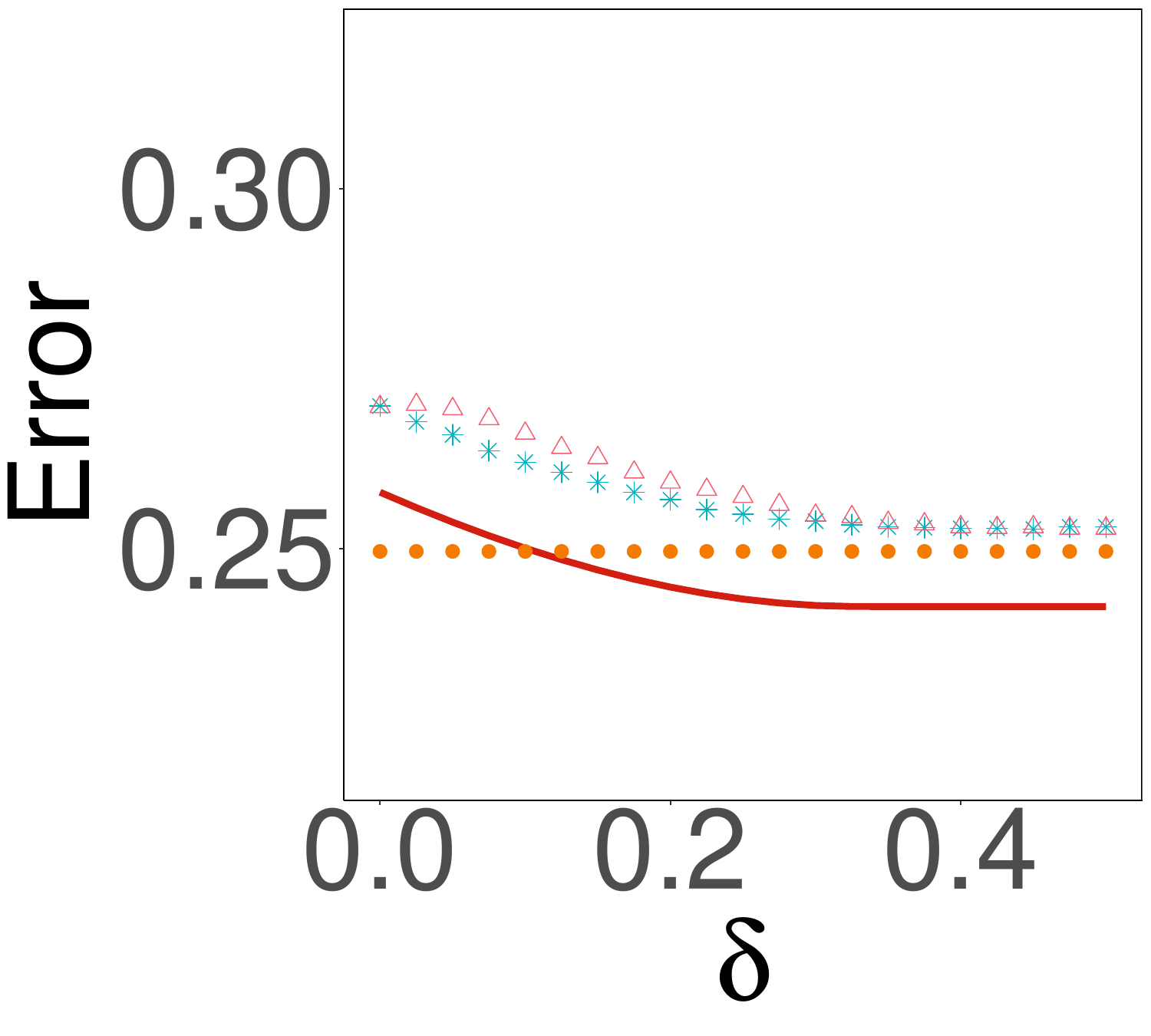} &
			\hspace{\thisgap}\includegraphics[width=\thiswidth]{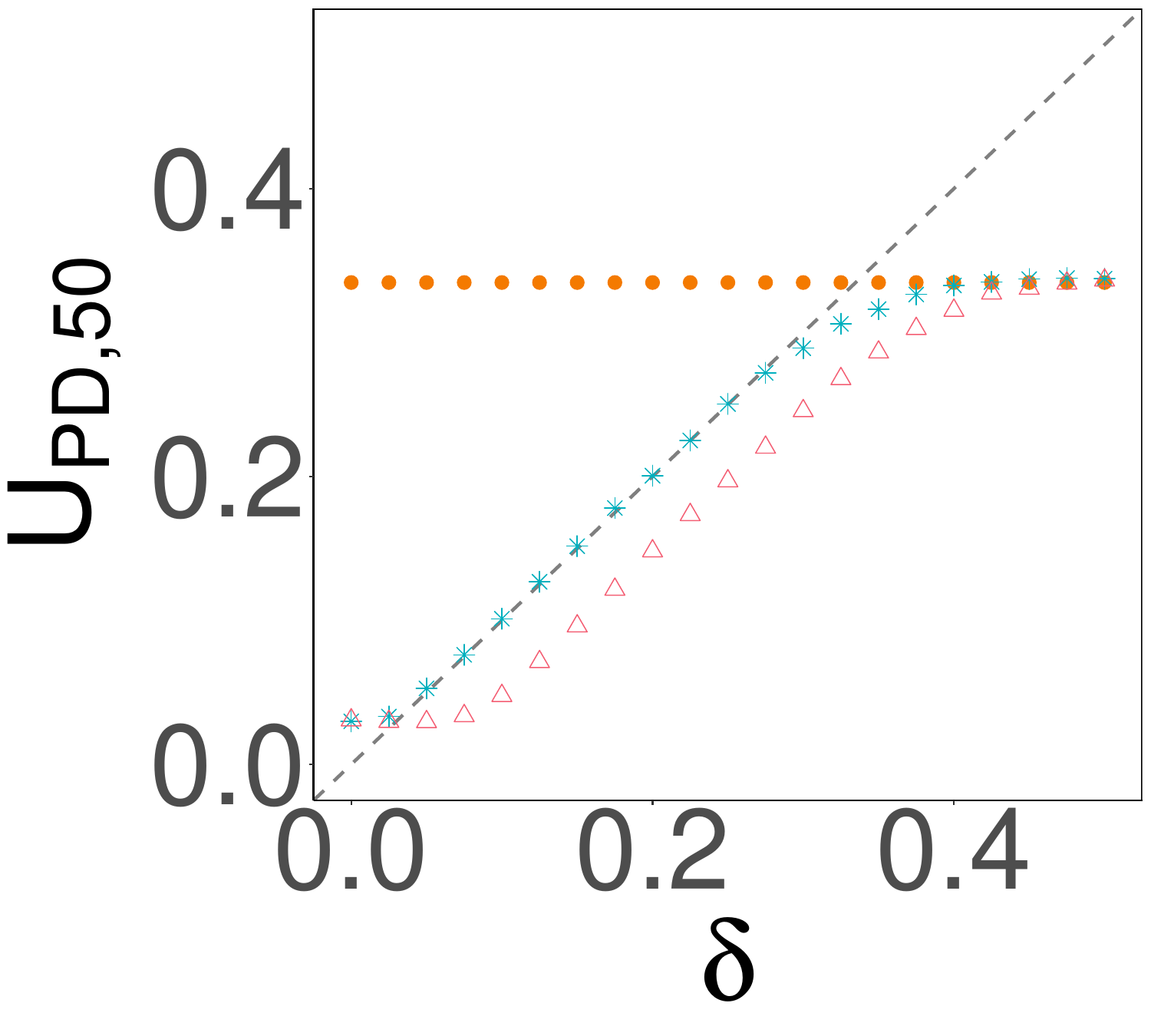} &
			\hspace{\thisgap}\includegraphics[width=\thiswidth]{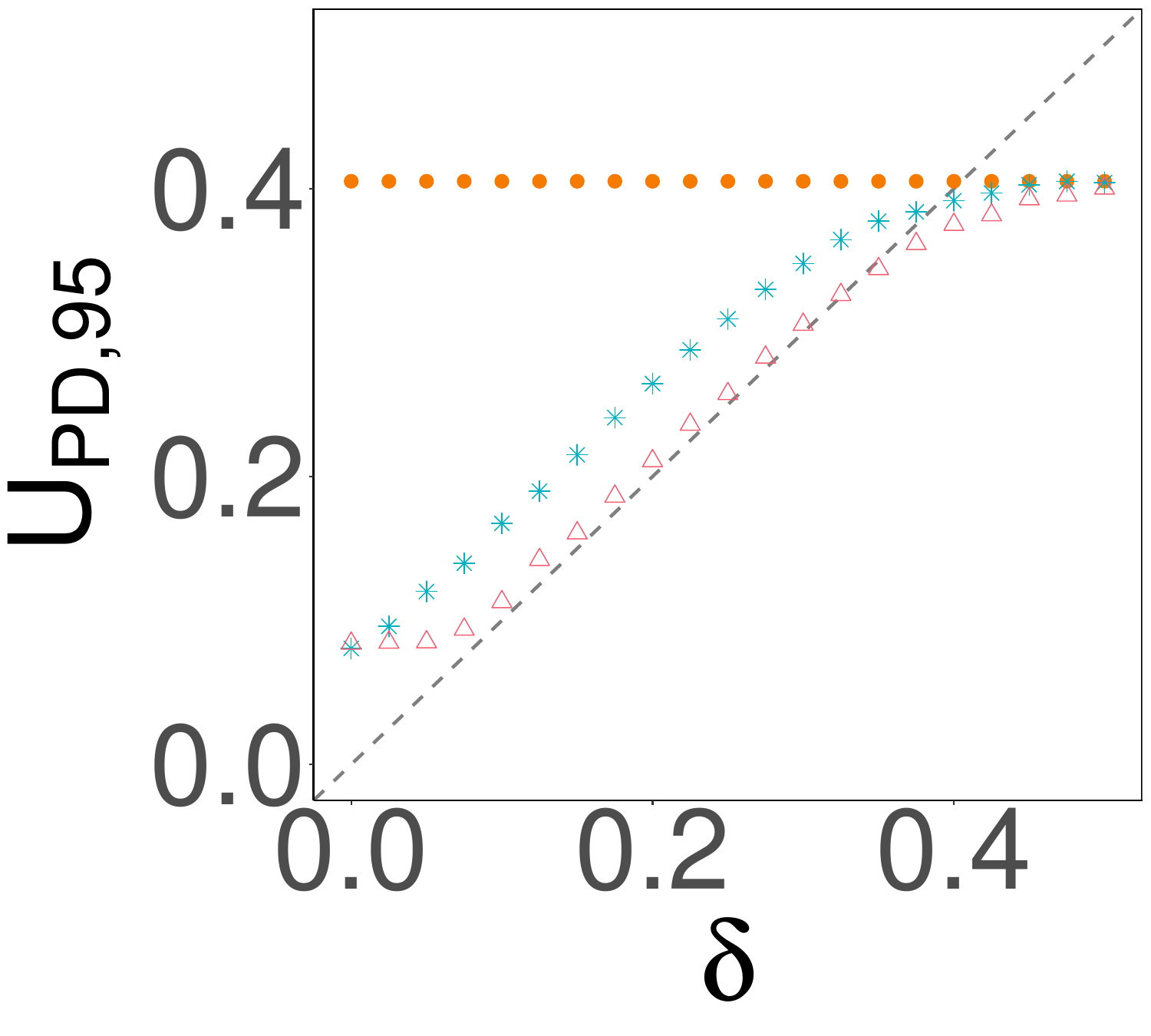} \\
			\hspace{\thisgap}\includegraphics[width=\thiswidth]{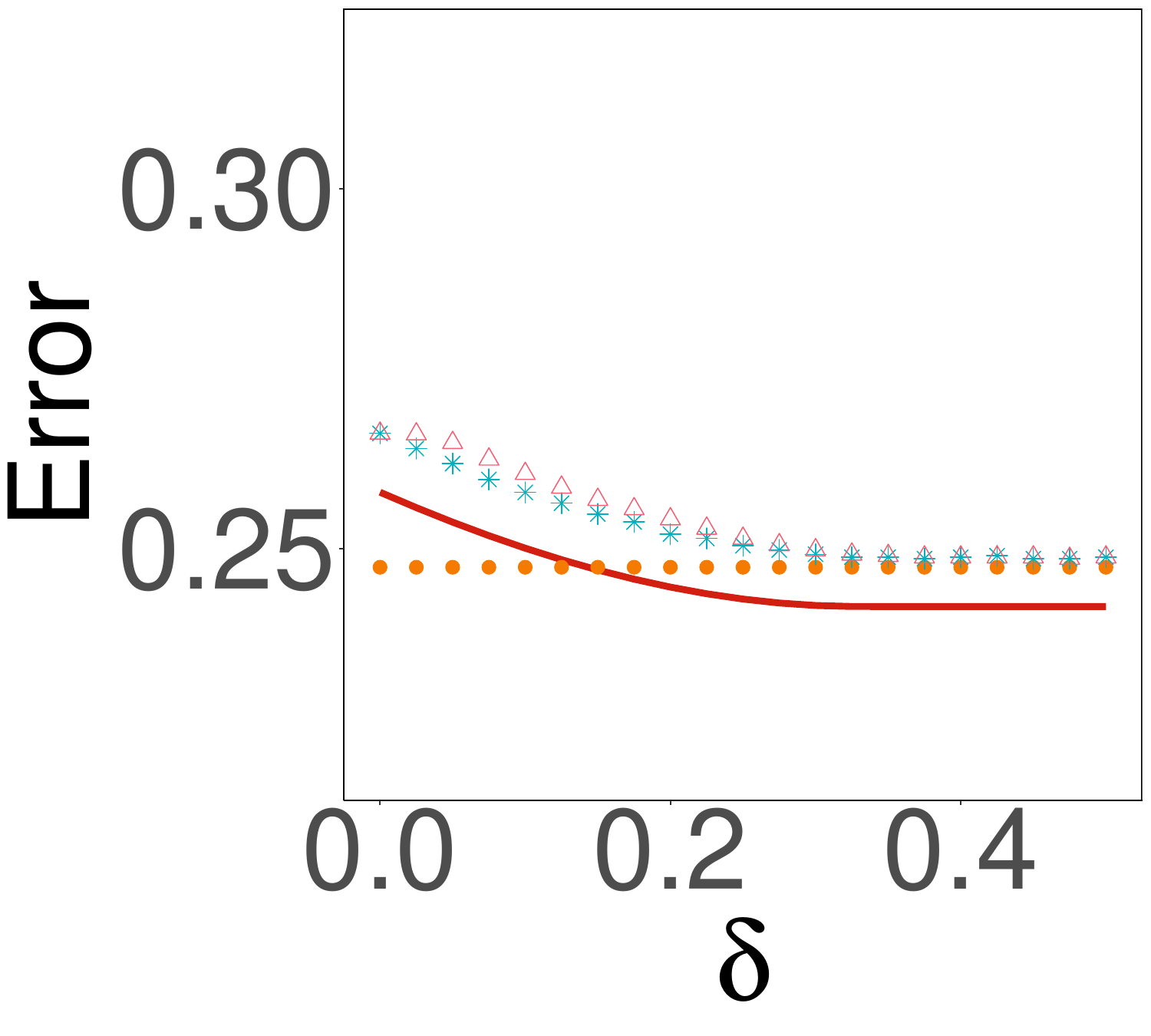} &
			\hspace{\thisgap}\includegraphics[width=\thiswidth]{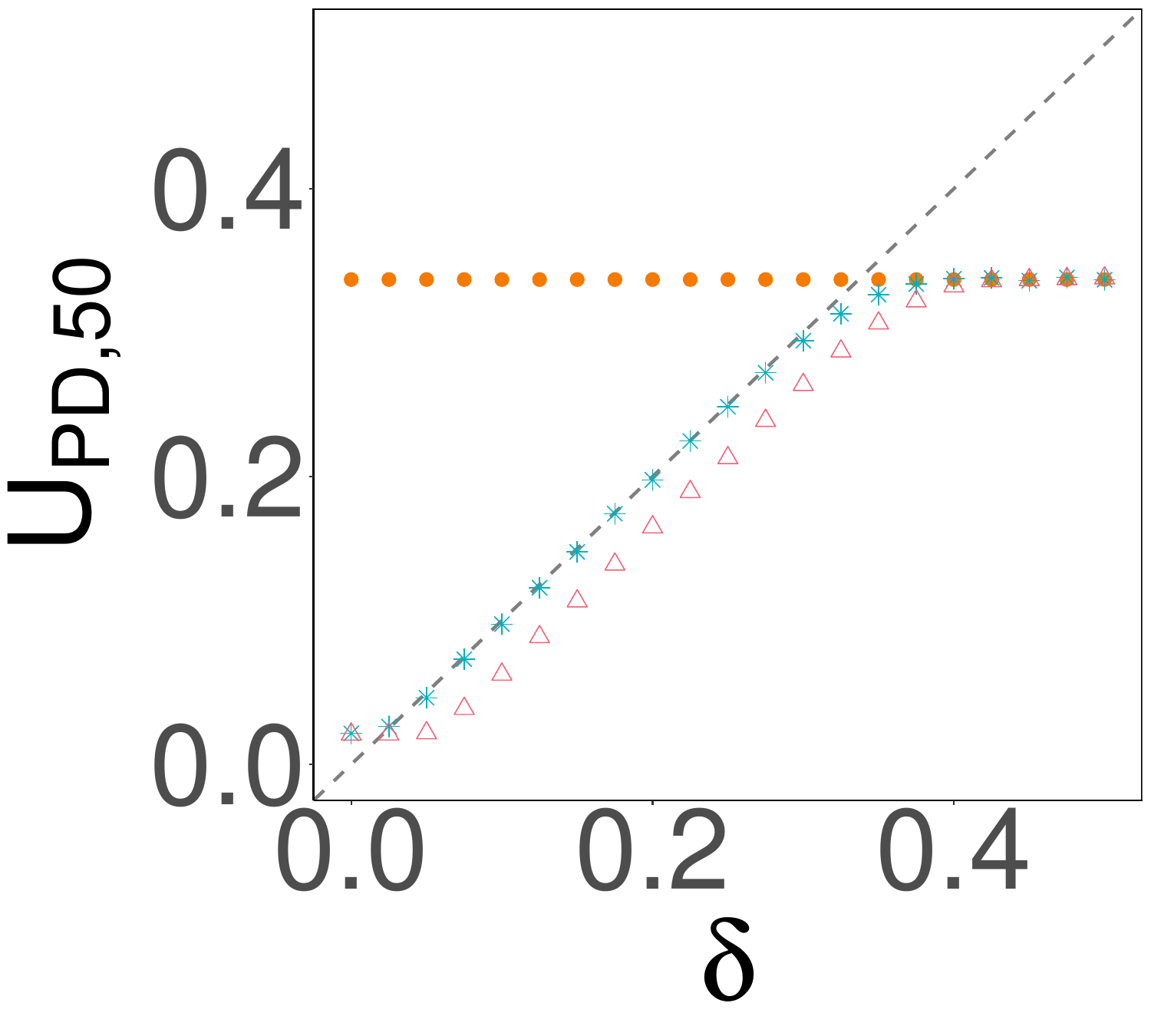} &
			\hspace{\thisgap}\includegraphics[width=\thiswidth]{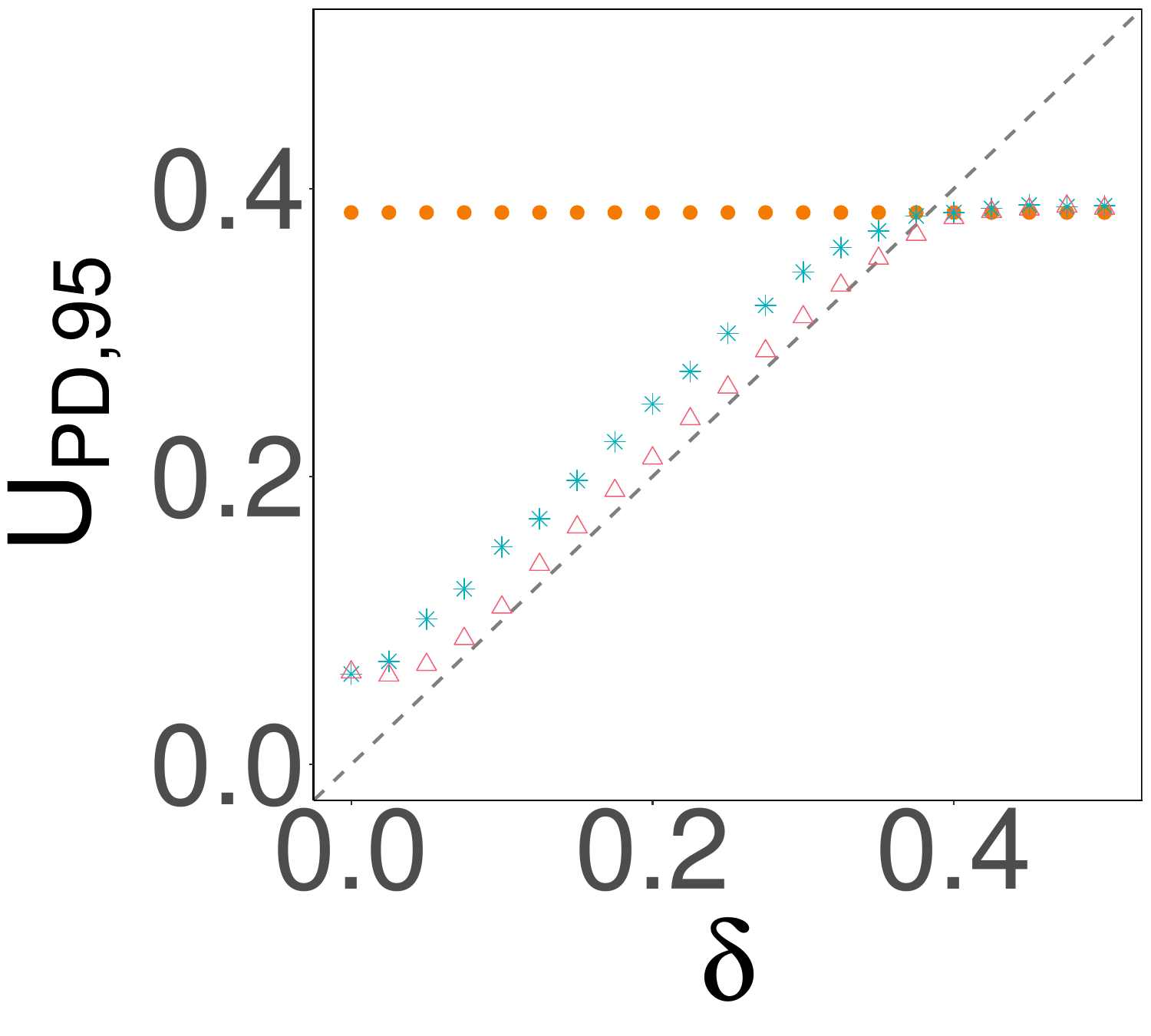}
		\end{tabular}
		\caption{Disparity PD results under the Gaussian model, $\beta=2$. Top: $n=1000$; middle: $n=2000$; bottom: $n=5000$.}
		\label{fig:PD_gauss_beta_2}
	\end{center}
\end{figure*}

\begin{figure*}[!htbp]
	\begin{center}
		\newcommand{\thiswidth}{0.18\linewidth}
		\newcommand{\thisgap}{0mm}
		\begin{tabular}{ccc}
			\hspace{\thisgap}\includegraphics[width=\thiswidth]{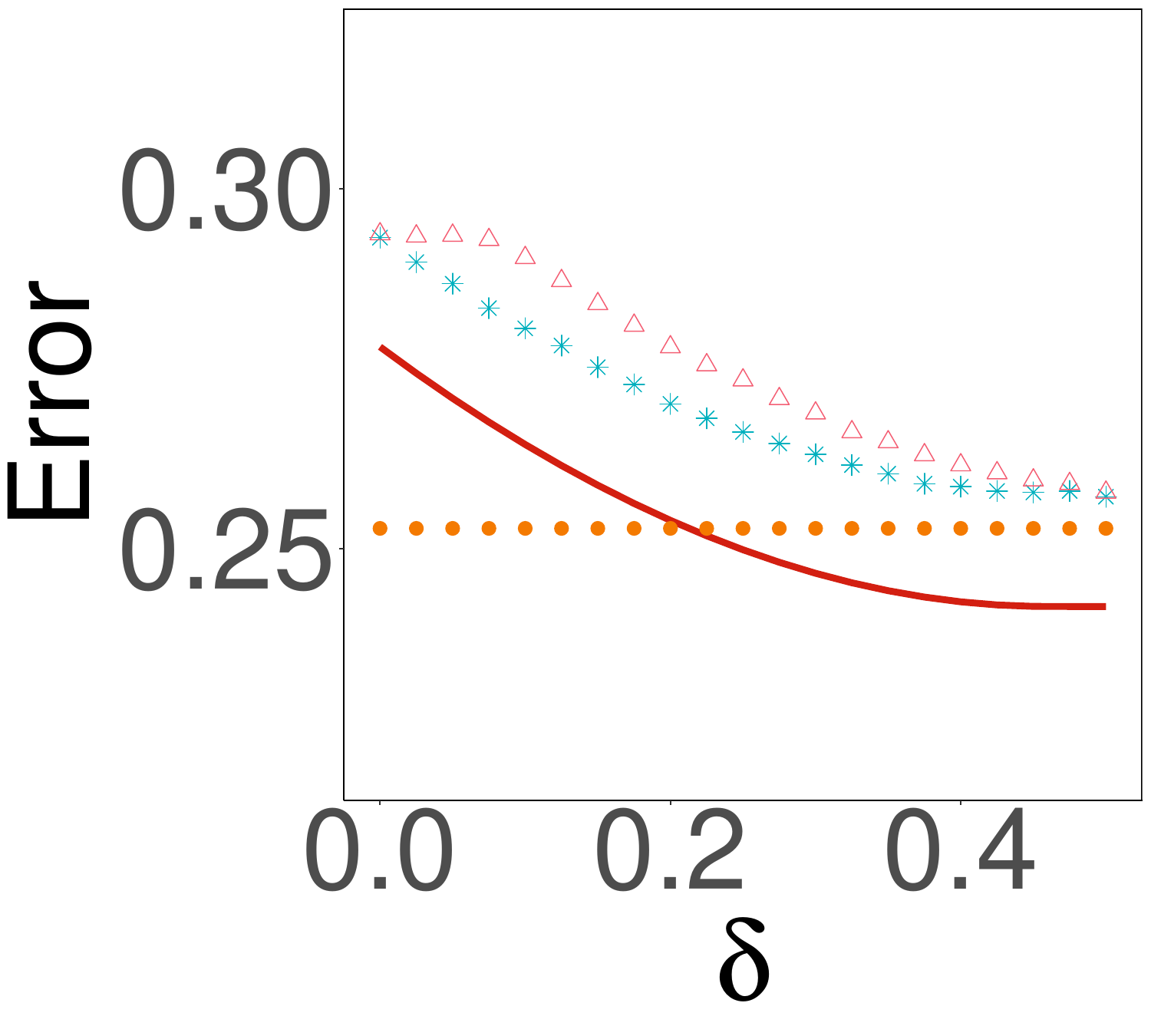} &
			\hspace{\thisgap}\includegraphics[width=\thiswidth]{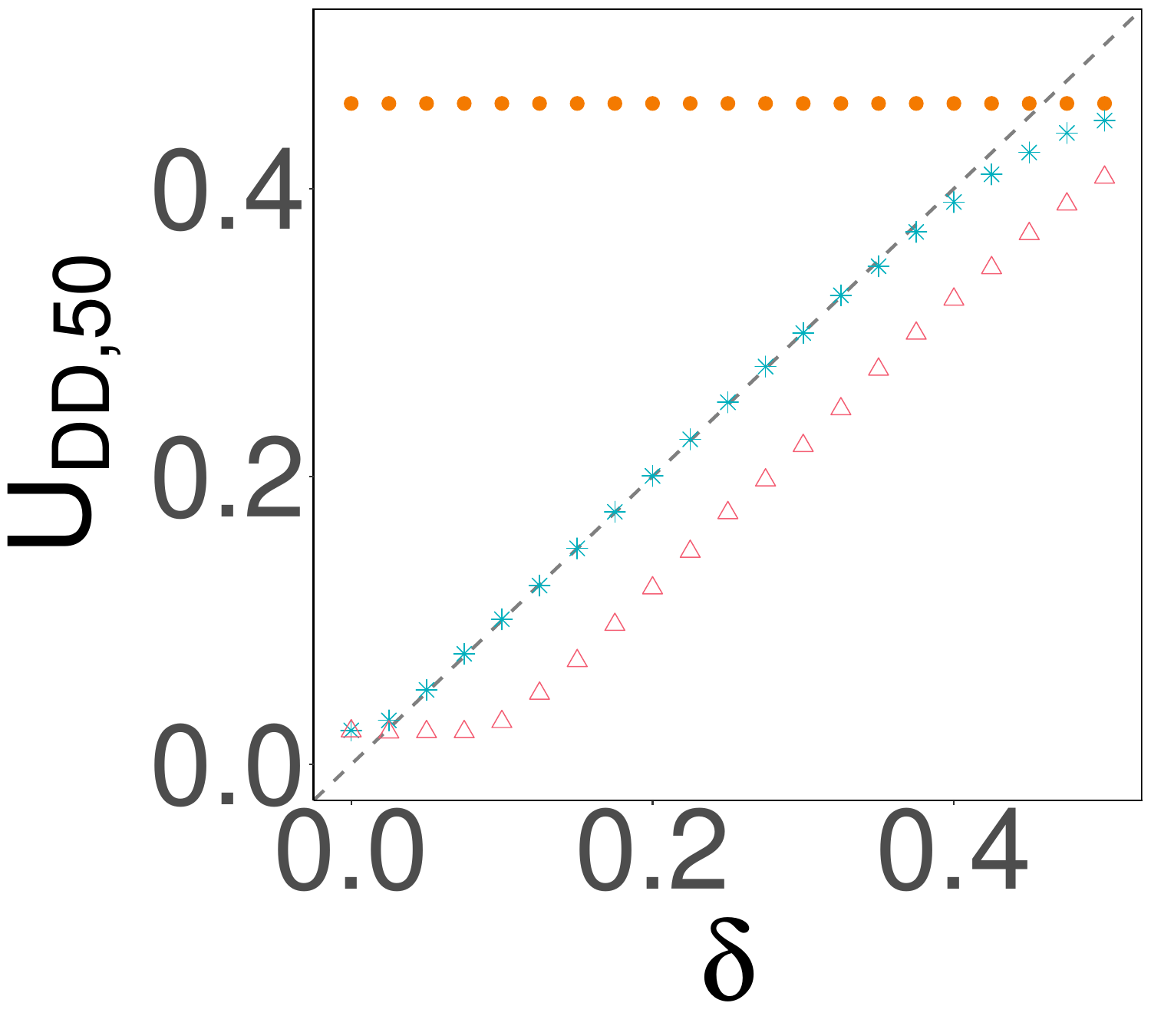} &
			\hspace{\thisgap}\includegraphics[width=\thiswidth]{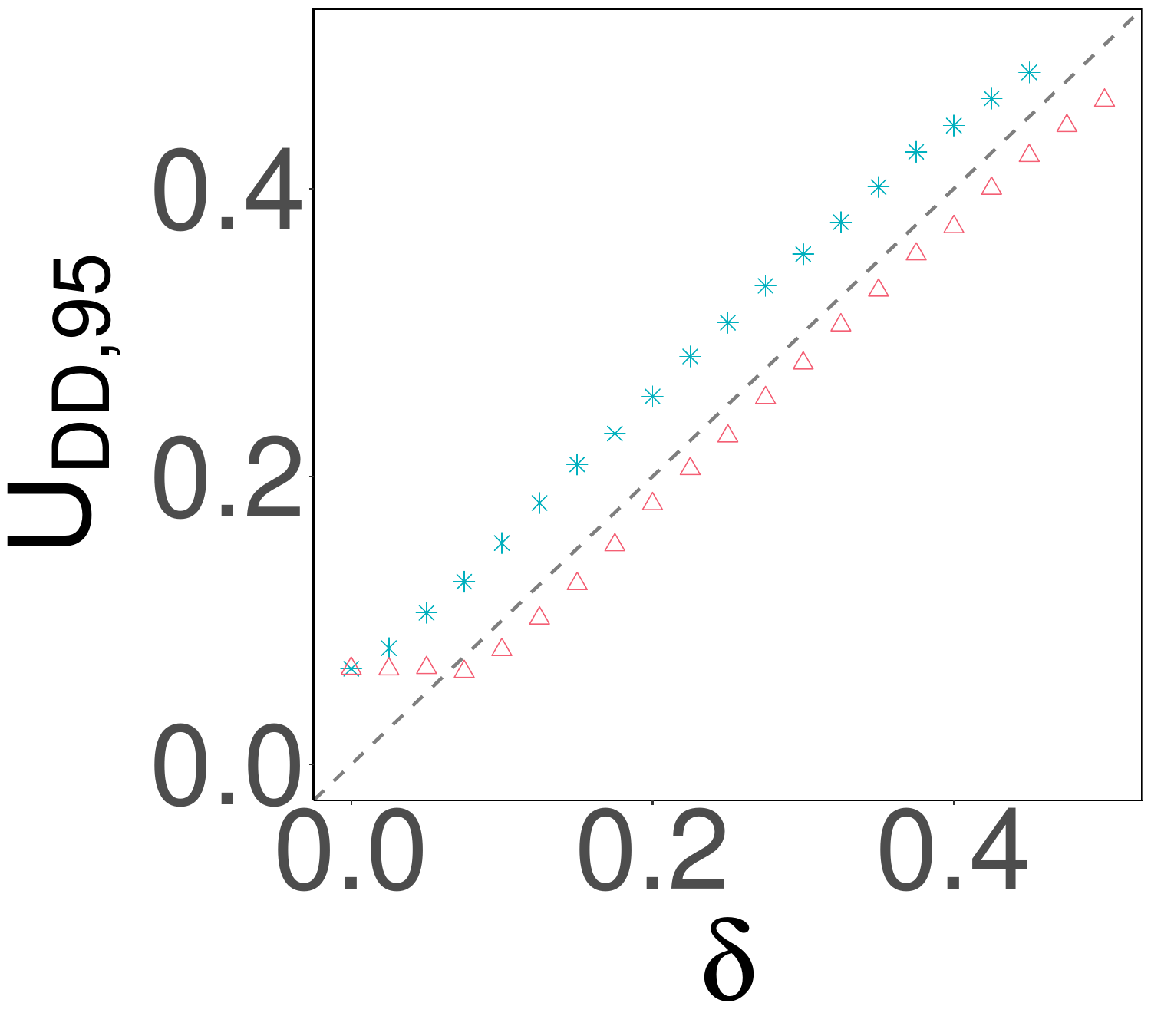} \\
			\hspace{\thisgap}\includegraphics[width=\thiswidth]{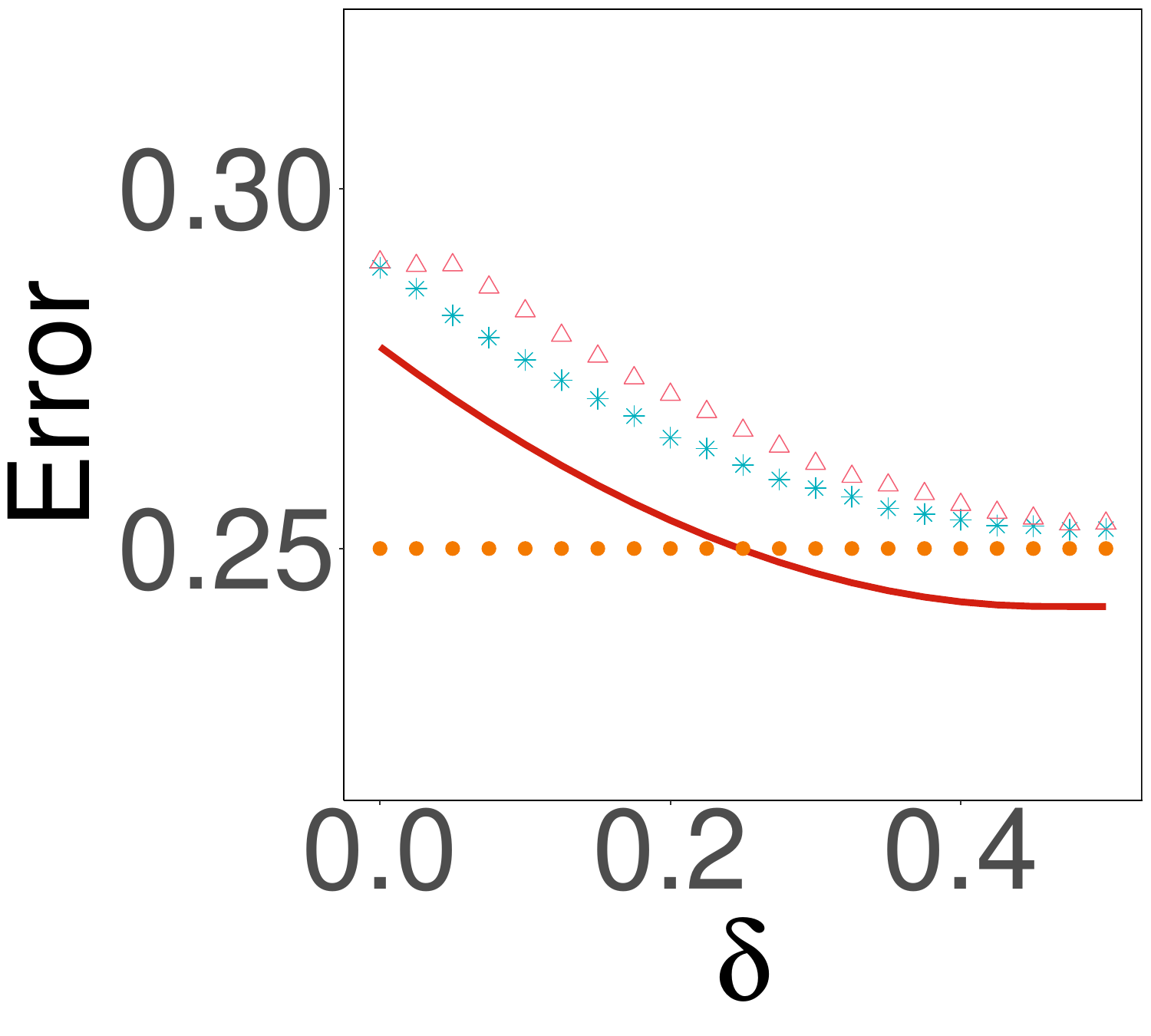} &
			\hspace{\thisgap}\includegraphics[width=\thiswidth]{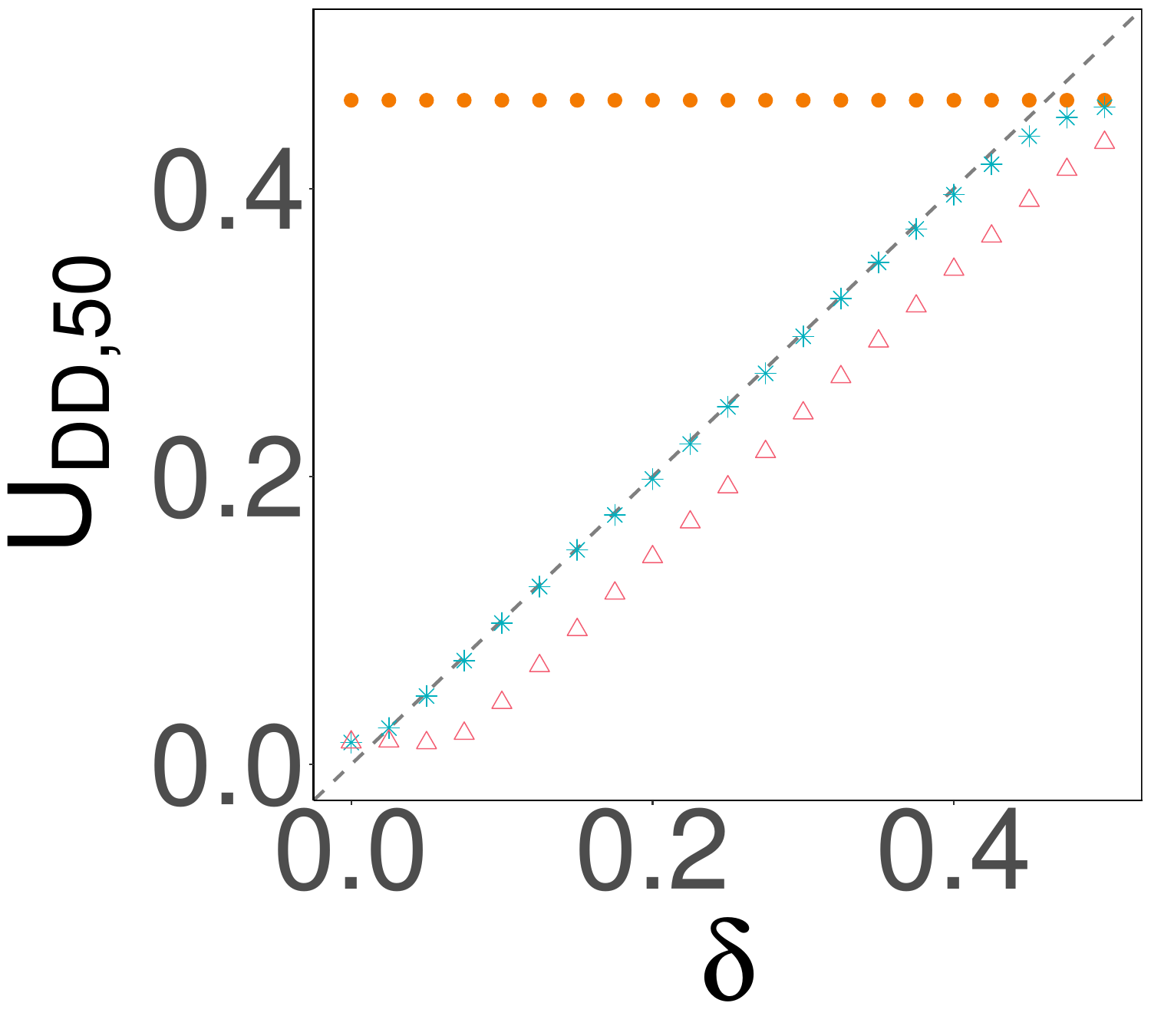} &
			\hspace{\thisgap}\includegraphics[width=\thiswidth]{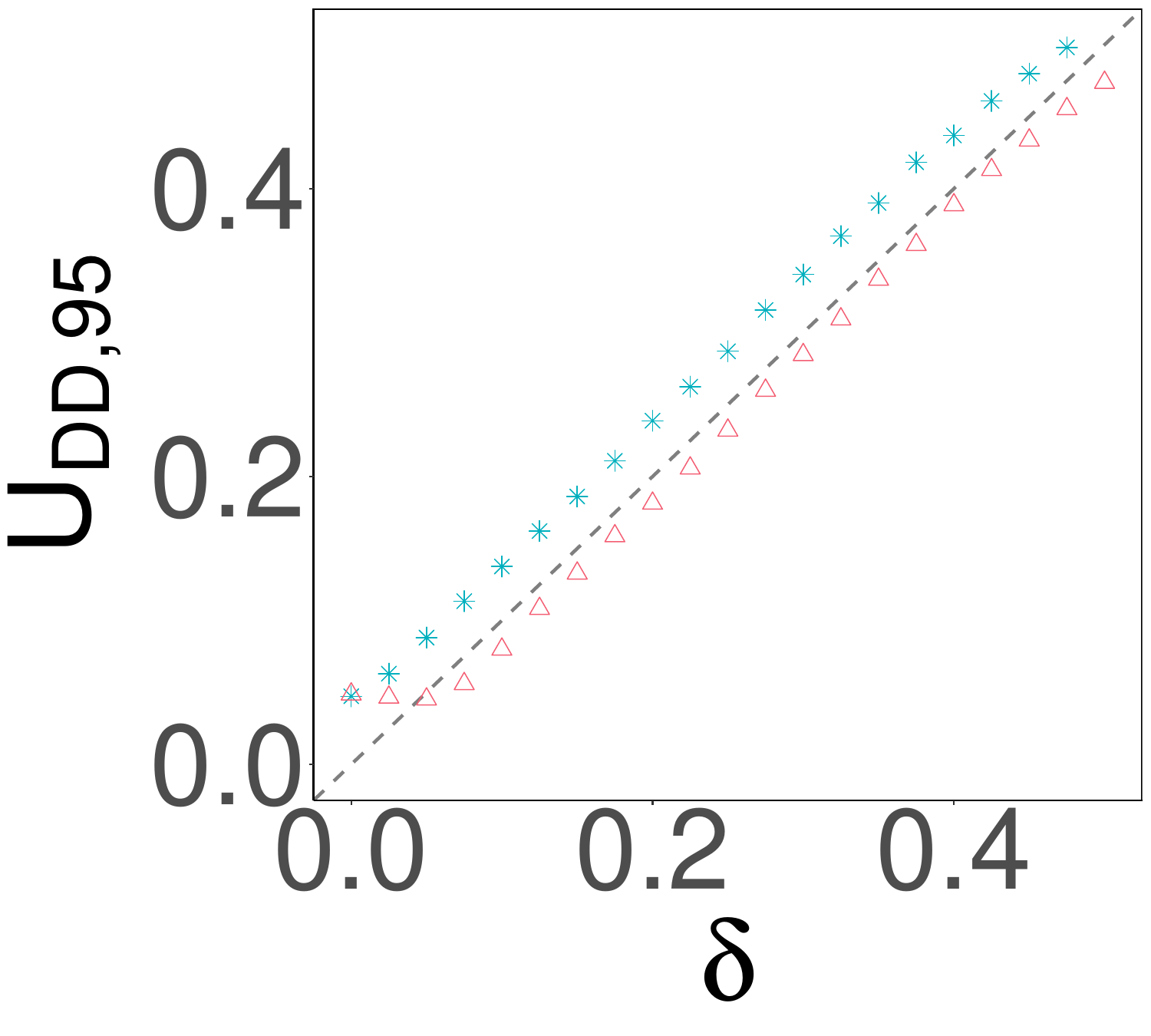} \\
			\hspace{\thisgap}\includegraphics[width=\thiswidth]{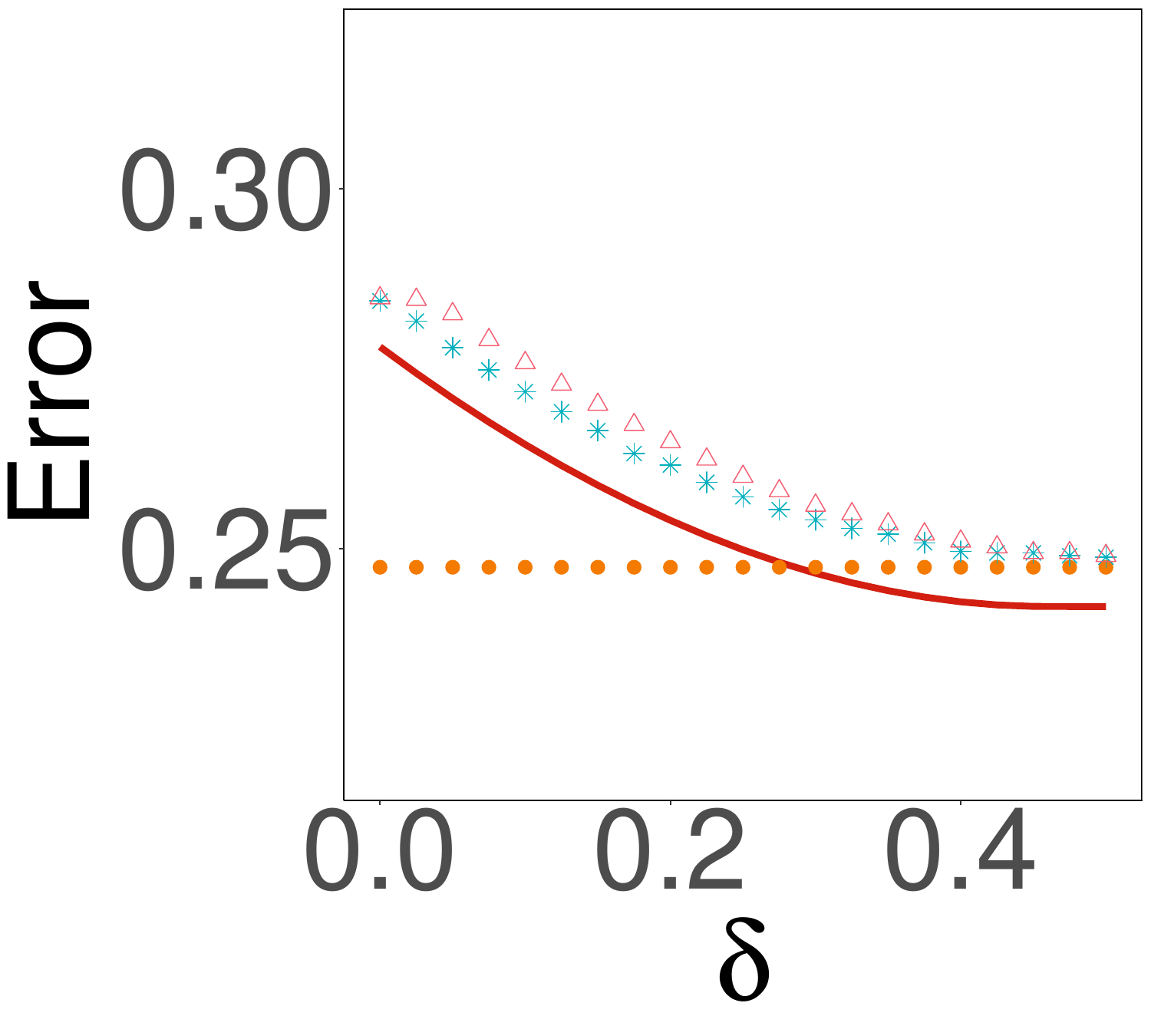} &
			\hspace{\thisgap}\includegraphics[width=\thiswidth]{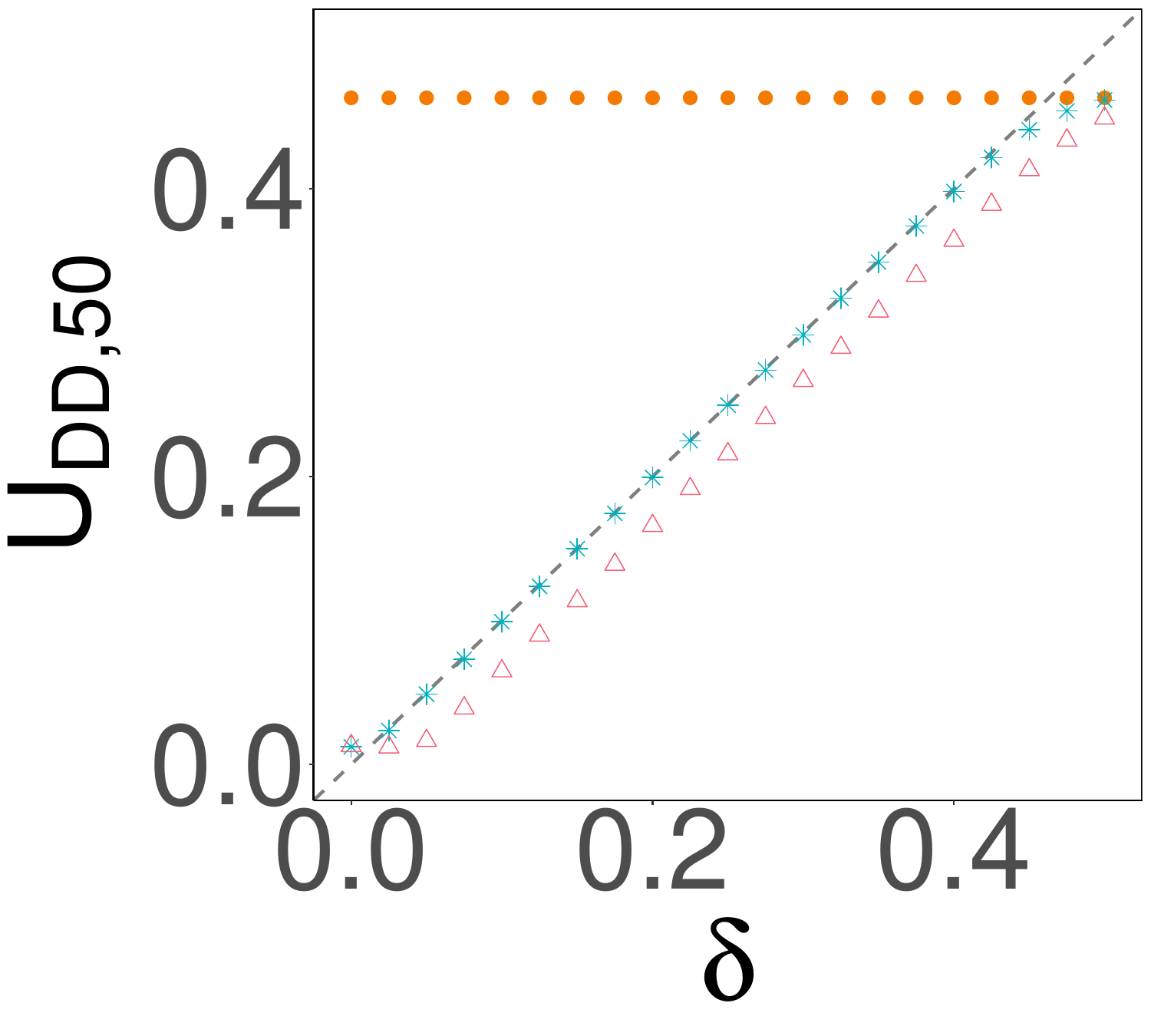} &
			\hspace{\thisgap}\includegraphics[width=\thiswidth]{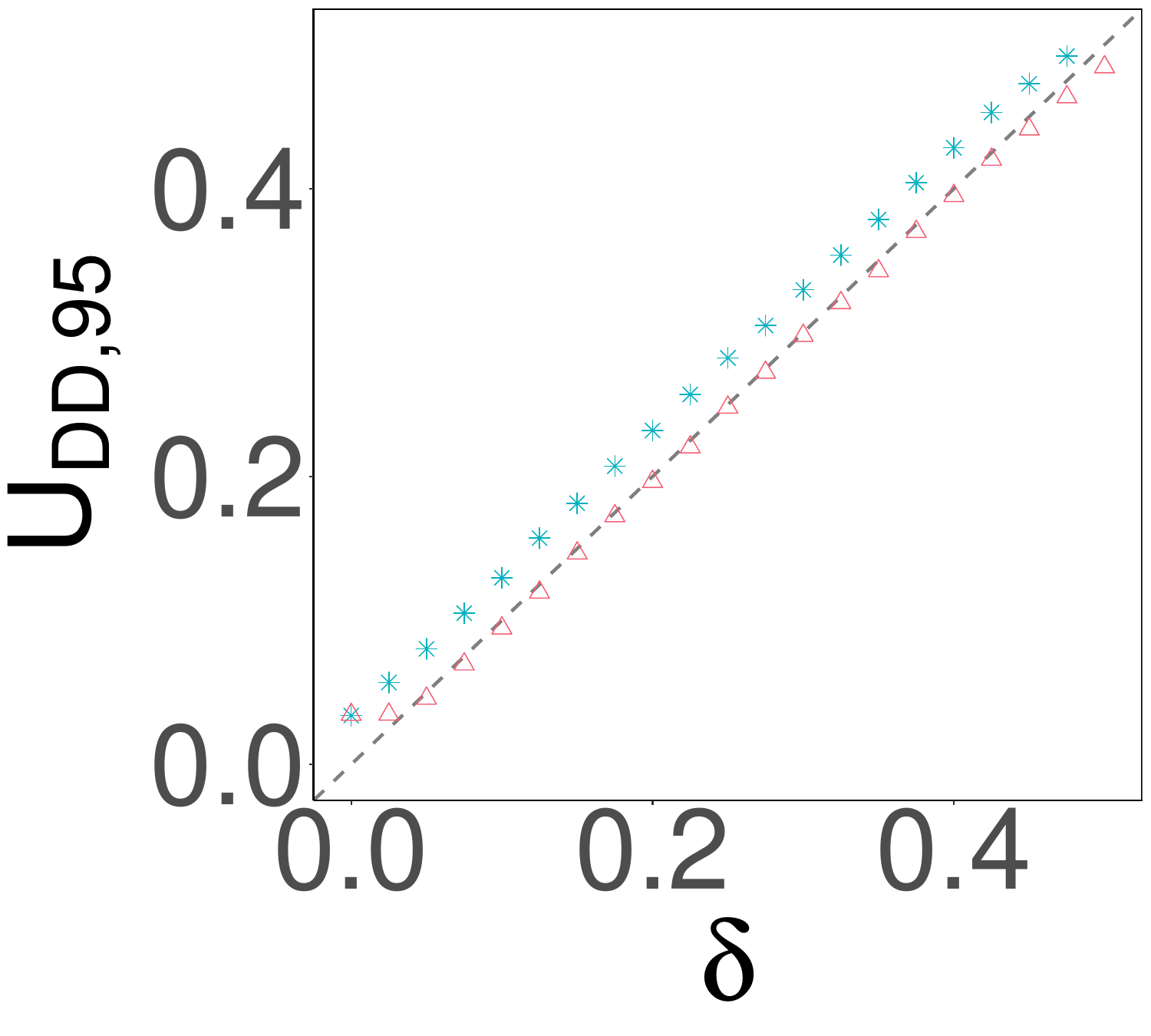}
		\end{tabular}
		\caption{Disparity DD results under the Gaussian model, $\beta=2$. Top: $n=1000$; middle: $n=2000$; bottom: $n=5000$.}
		\label{fig:DD_gauss_beta_2}
	\end{center}
\end{figure*}

\subsection{Simulation results under non-Gaussian models}\label{sec:apx_nongaussian}

Although the proposed fairness-aware classifier is established based on the explicit form of the Radon--Nikodym derivative under Gaussian assumptions, it remains applicable in more general scenarios. However, when the Gaussian assumption is violated, the proposed classifier may no longer be Bayes optimal.
To assess its performance beyond the Gaussian setting, we generate non-Gaussian stochastic processes by sampling $\zeta_{a,k} \sim \lambda_{a,k}^{1/2}\mathrm{Unif}(-\sqrt{3}, \sqrt{3})$, while keeping all other model parameters identical to those in Section \ref{sec:sim}.  

The results are presented in Figures \ref{fig:DO_unif_beta_1.5}-\ref{fig:DD_unif_beta_1.5}. 
Despite the lack of Bayes optimality guarantees in this setting, the Fair-FLDA and Fair-$\mathrm{FLDA_c}$ continue to exhibit effective disparity control and satisfactory classification accuracy. This robustness highlights the practical utility of our approach in more general, non-Gaussian scenarios.

\begin{figure*}[!htbp]
	\begin{center}
		\newcommand{\thiswidth}{0.18\linewidth}
		\newcommand{\thisgap}{0mm}
		\begin{tabular}{ccc}
			\hspace{\thisgap}\includegraphics[width=\thiswidth]{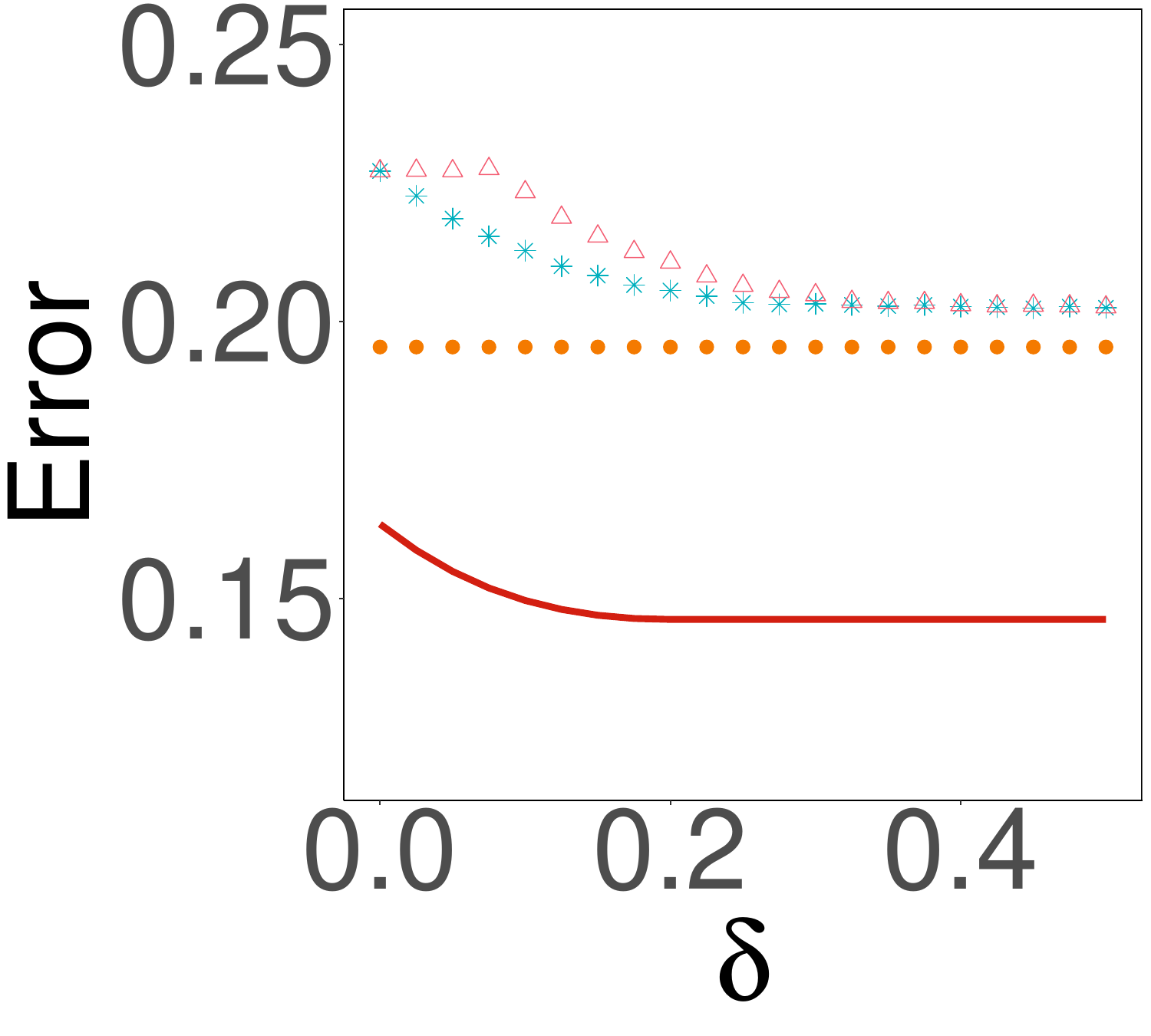} &
			\hspace{\thisgap}\includegraphics[width=\thiswidth]{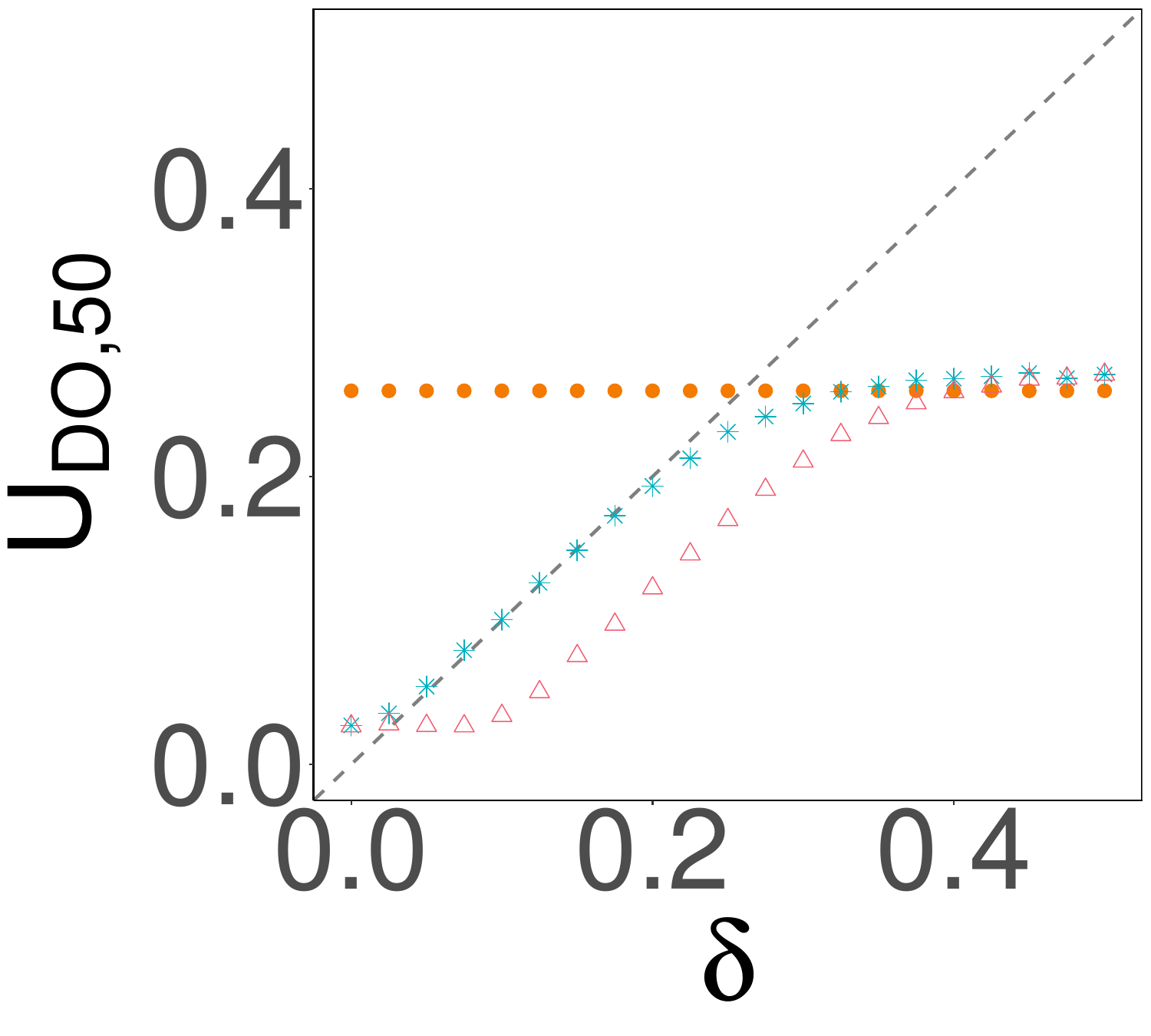} &
			\hspace{\thisgap}\includegraphics[width=\thiswidth]{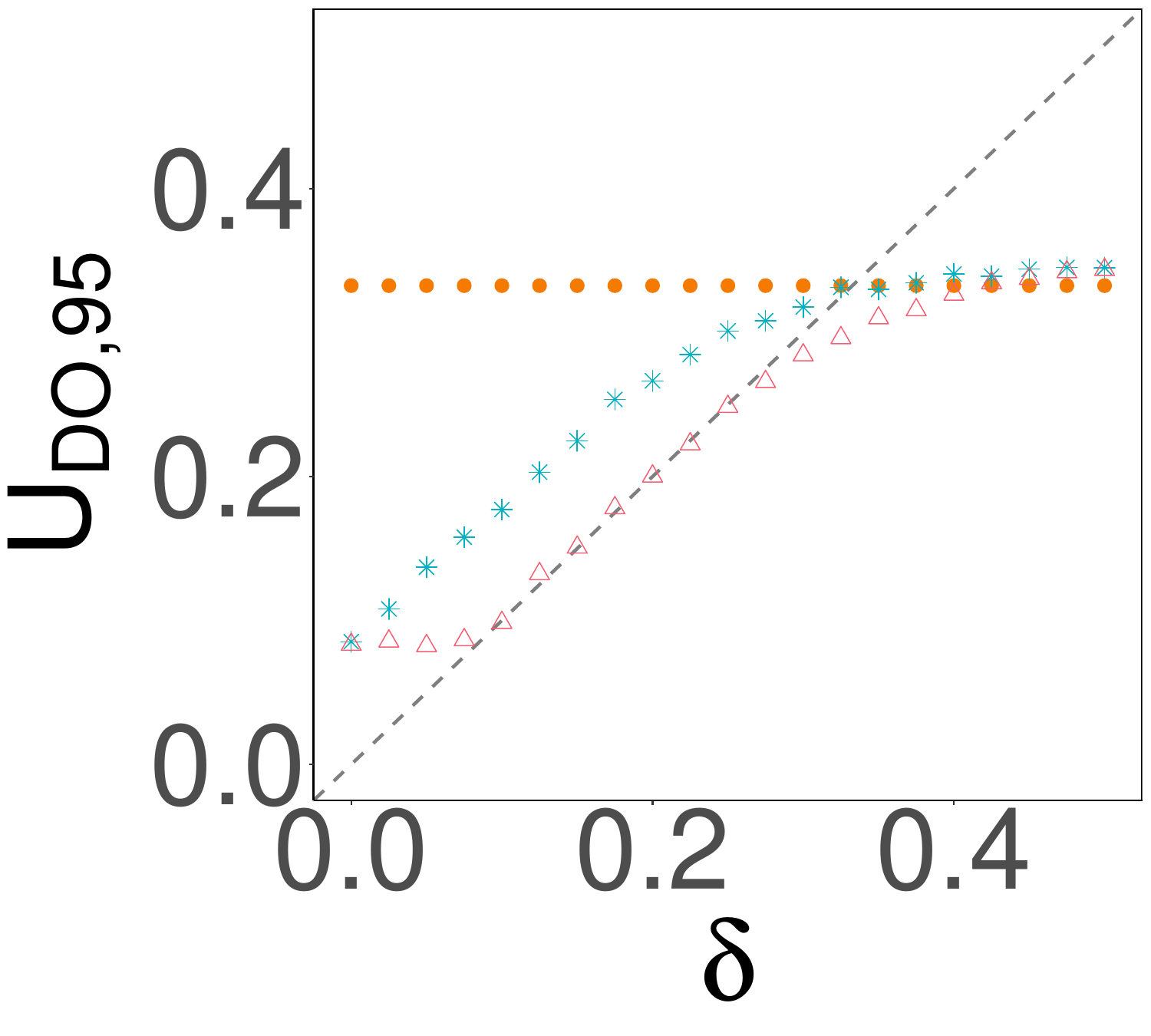} \\
			\hspace{\thisgap}\includegraphics[width=\thiswidth]{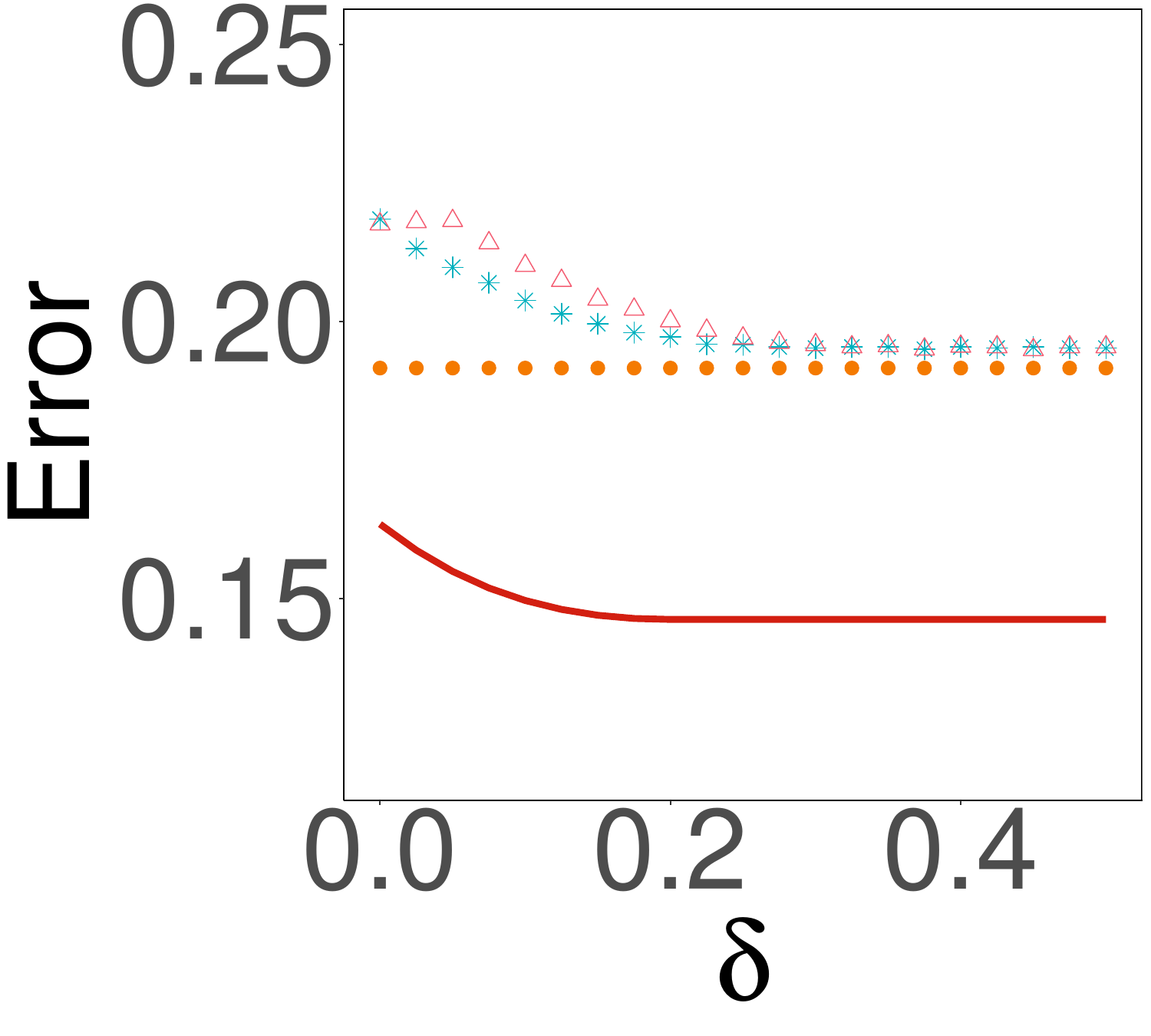} &
			\hspace{\thisgap}\includegraphics[width=\thiswidth]{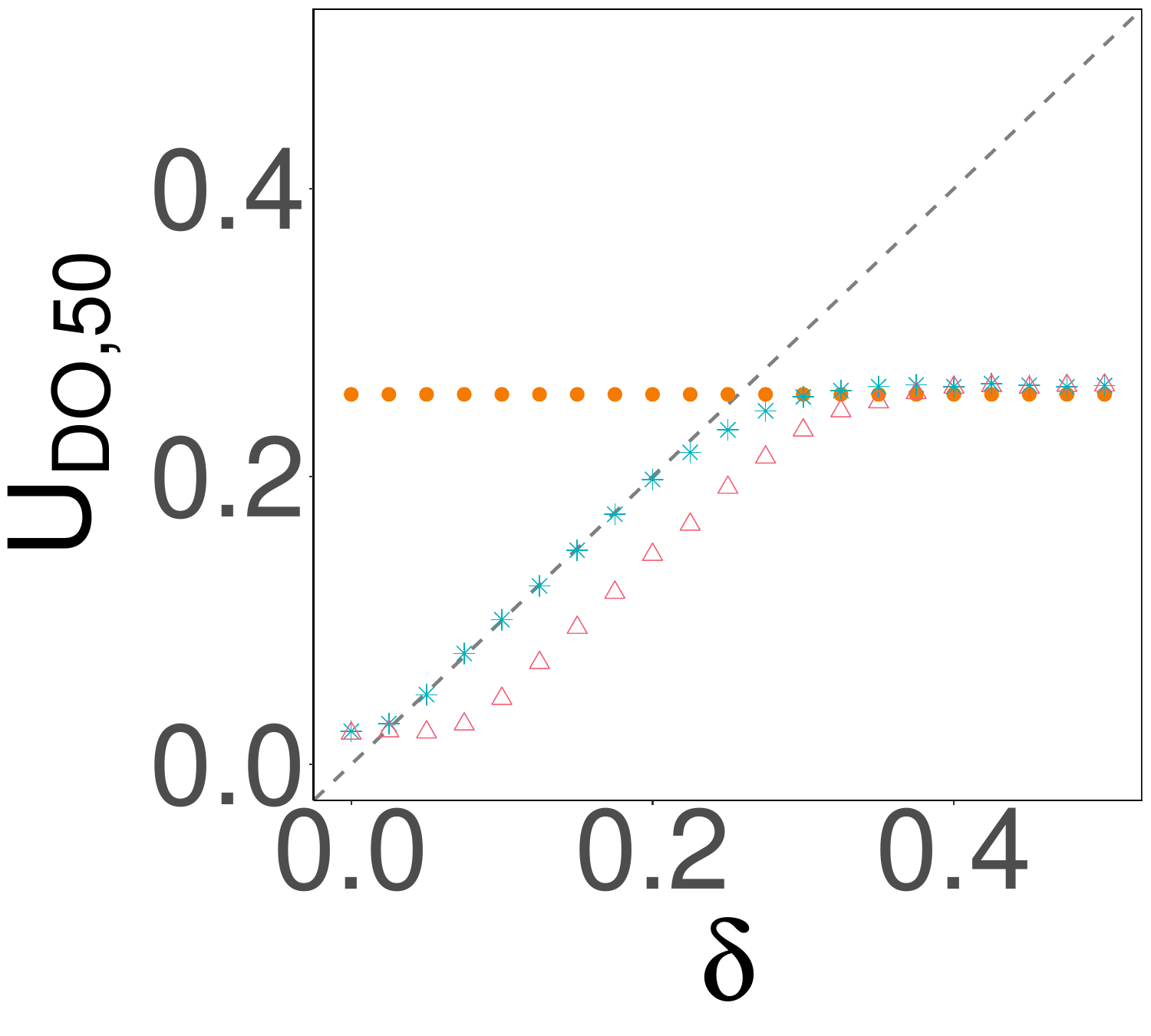} &
			\hspace{\thisgap}\includegraphics[width=\thiswidth]{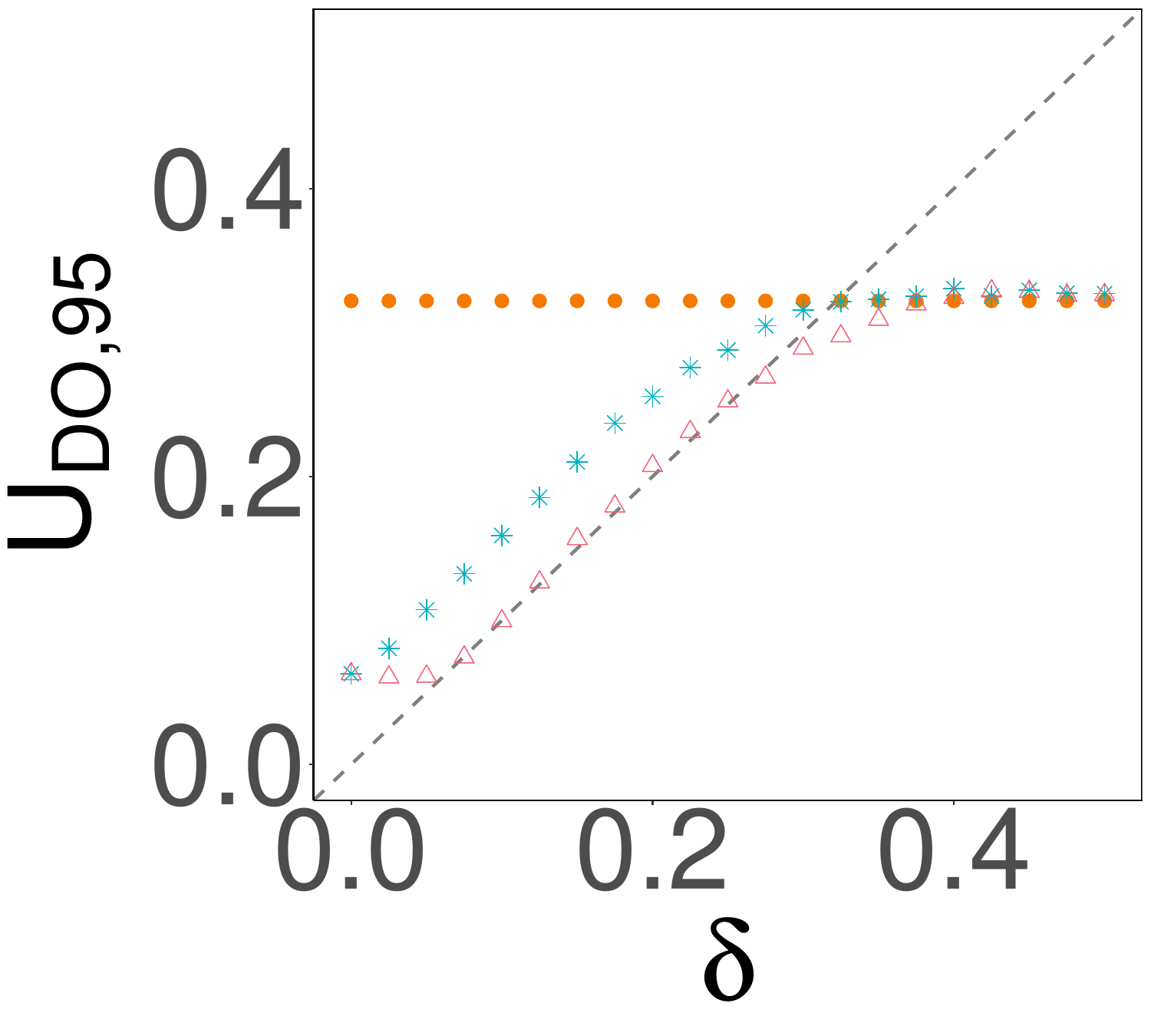} \\
			\hspace{\thisgap}\includegraphics[width=\thiswidth]{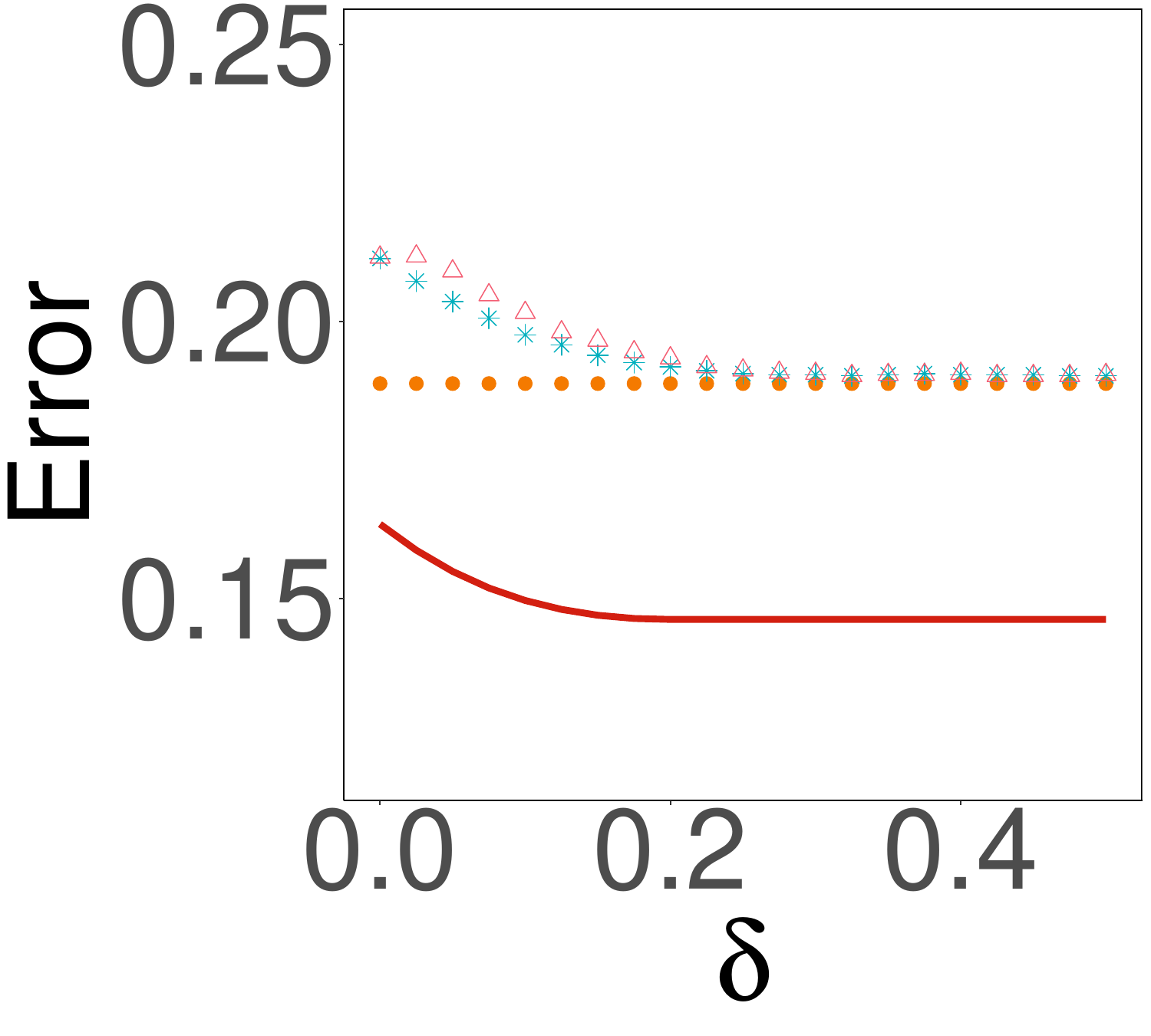} &
			\hspace{\thisgap}\includegraphics[width=\thiswidth]{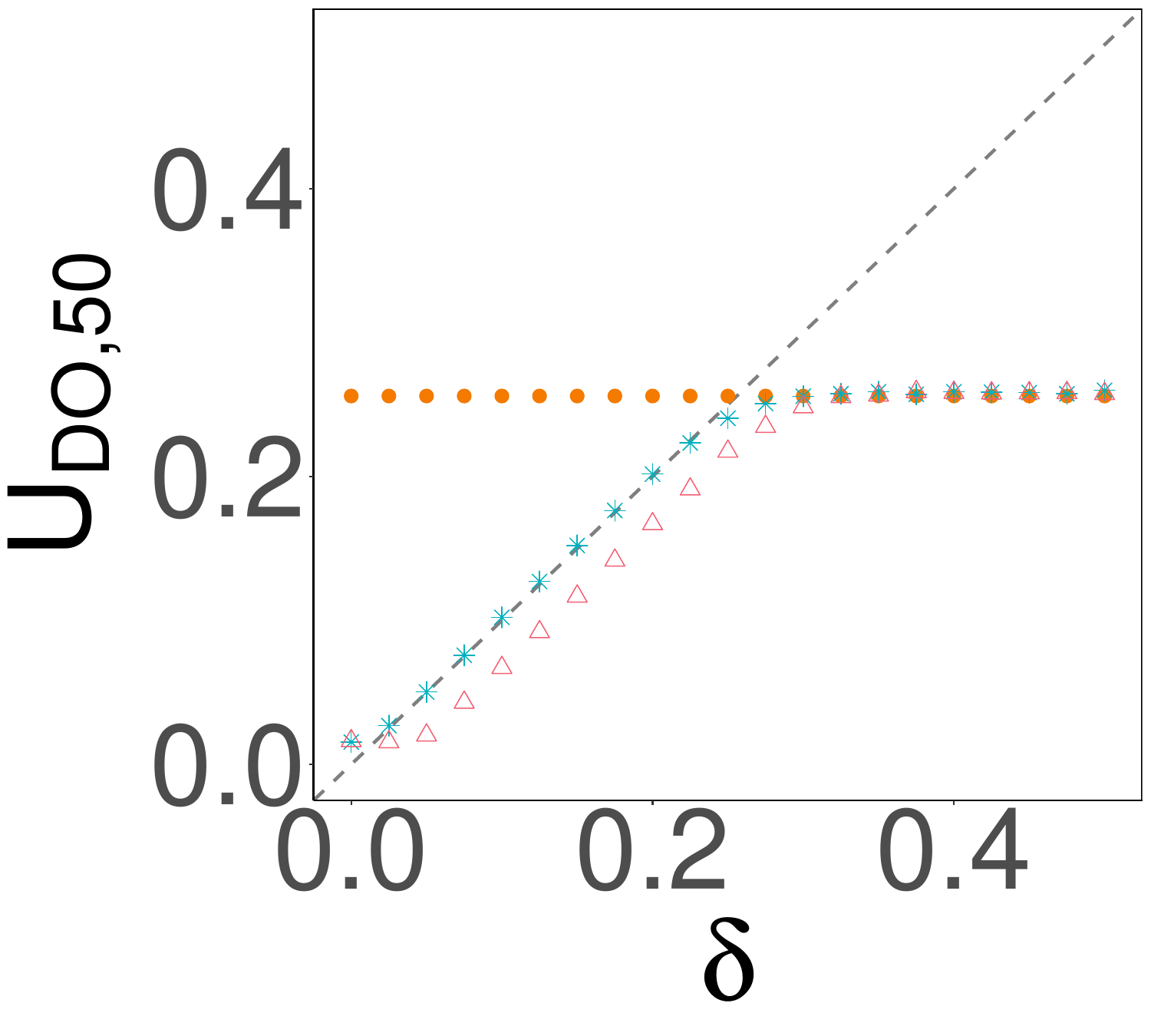} &
			\hspace{\thisgap}\includegraphics[width=\thiswidth]{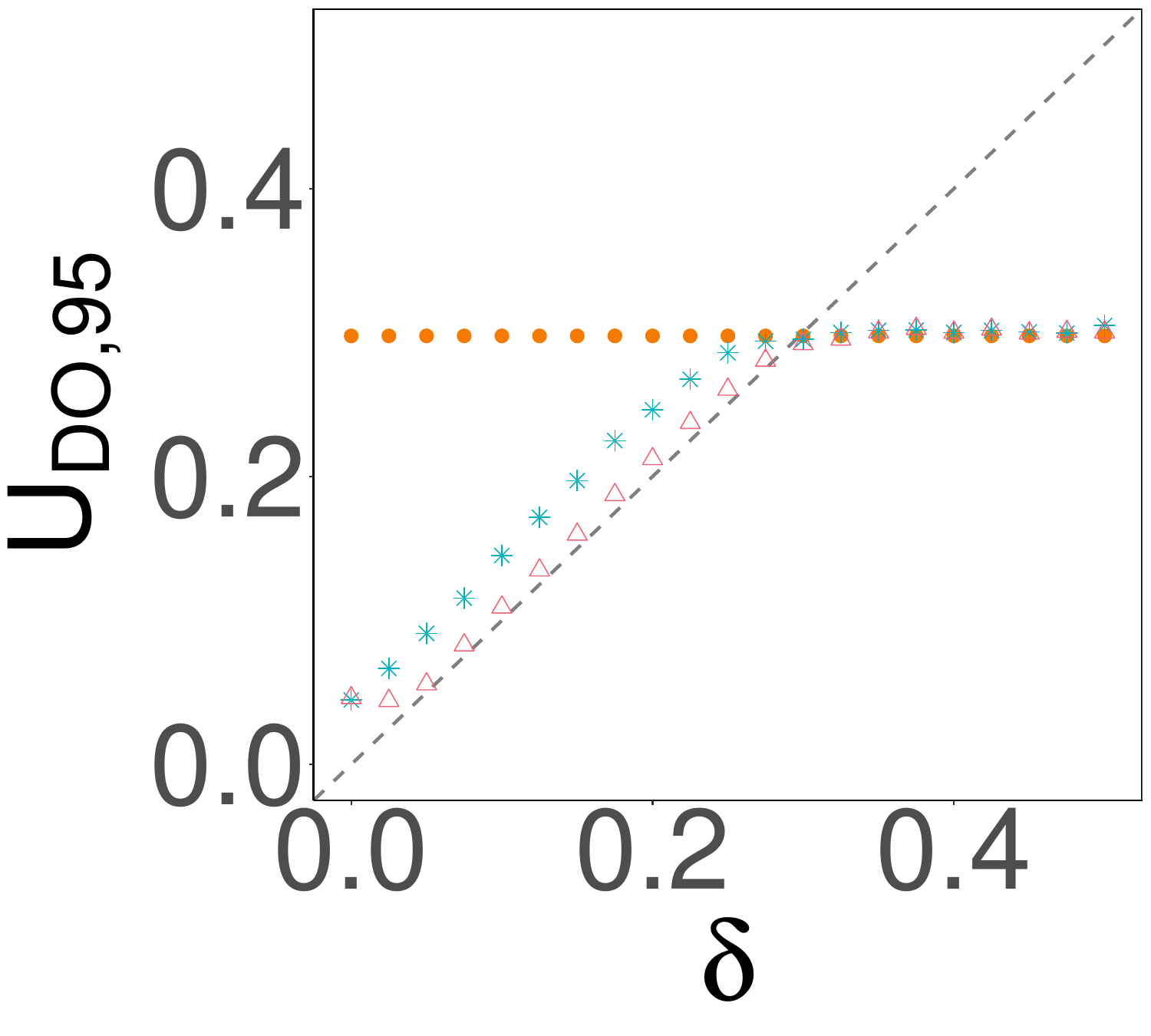}
		\end{tabular}
		\caption{Disparity DO results under the non-Gaussian model, $\beta=1.5$. Top: $n=1000$; middle: $n=2000$; bottom: $n=5000$.}
		\label{fig:DO_unif_beta_1.5}
	\end{center}
\end{figure*}

\begin{figure*}[!htbp]
	\begin{center}
		\newcommand{\thiswidth}{0.18\linewidth}
		\newcommand{\thisgap}{0mm}
		\begin{tabular}{ccc}
			\hspace{\thisgap}\includegraphics[width=\thiswidth]{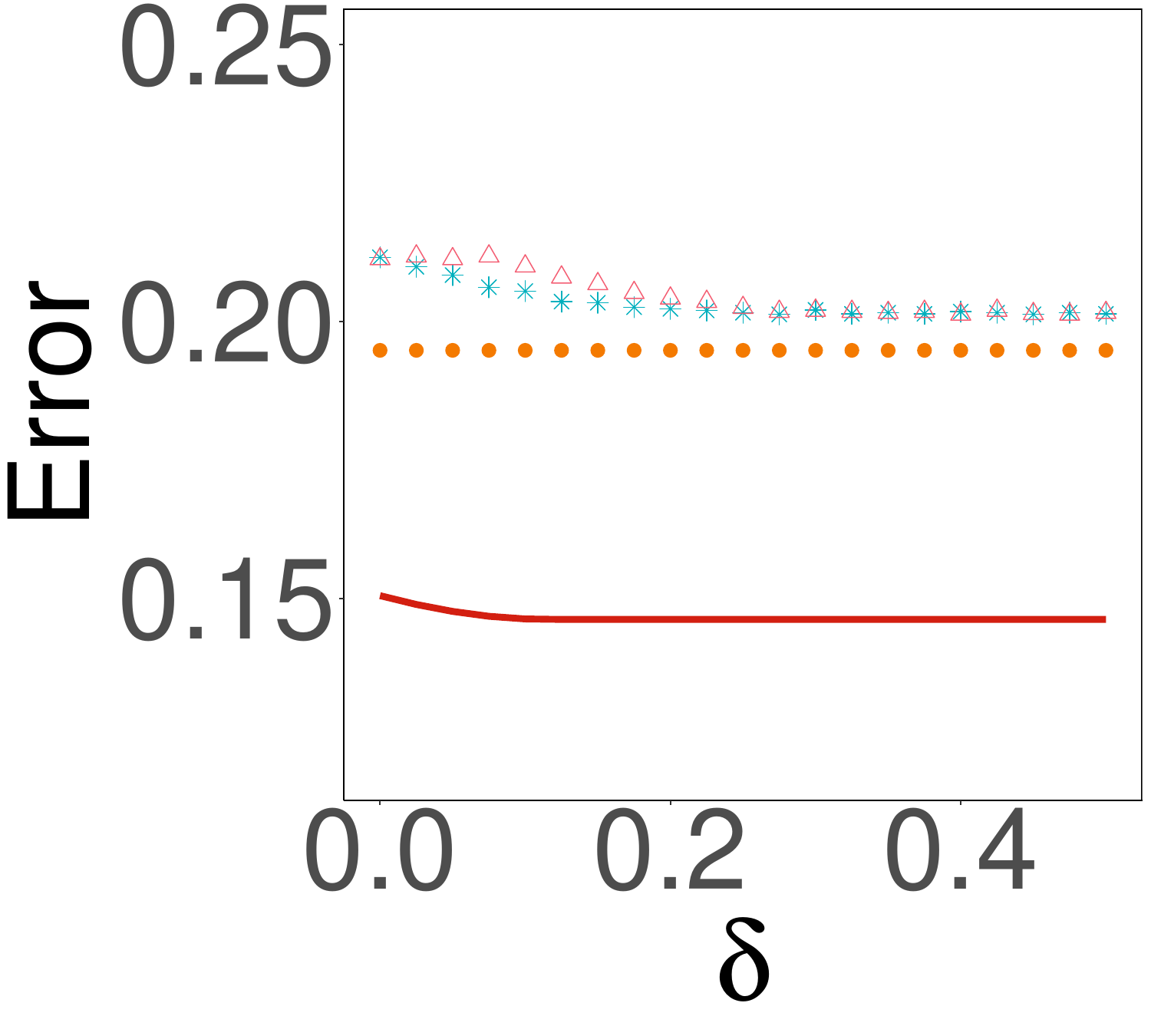} &
			\hspace{\thisgap}\includegraphics[width=\thiswidth]{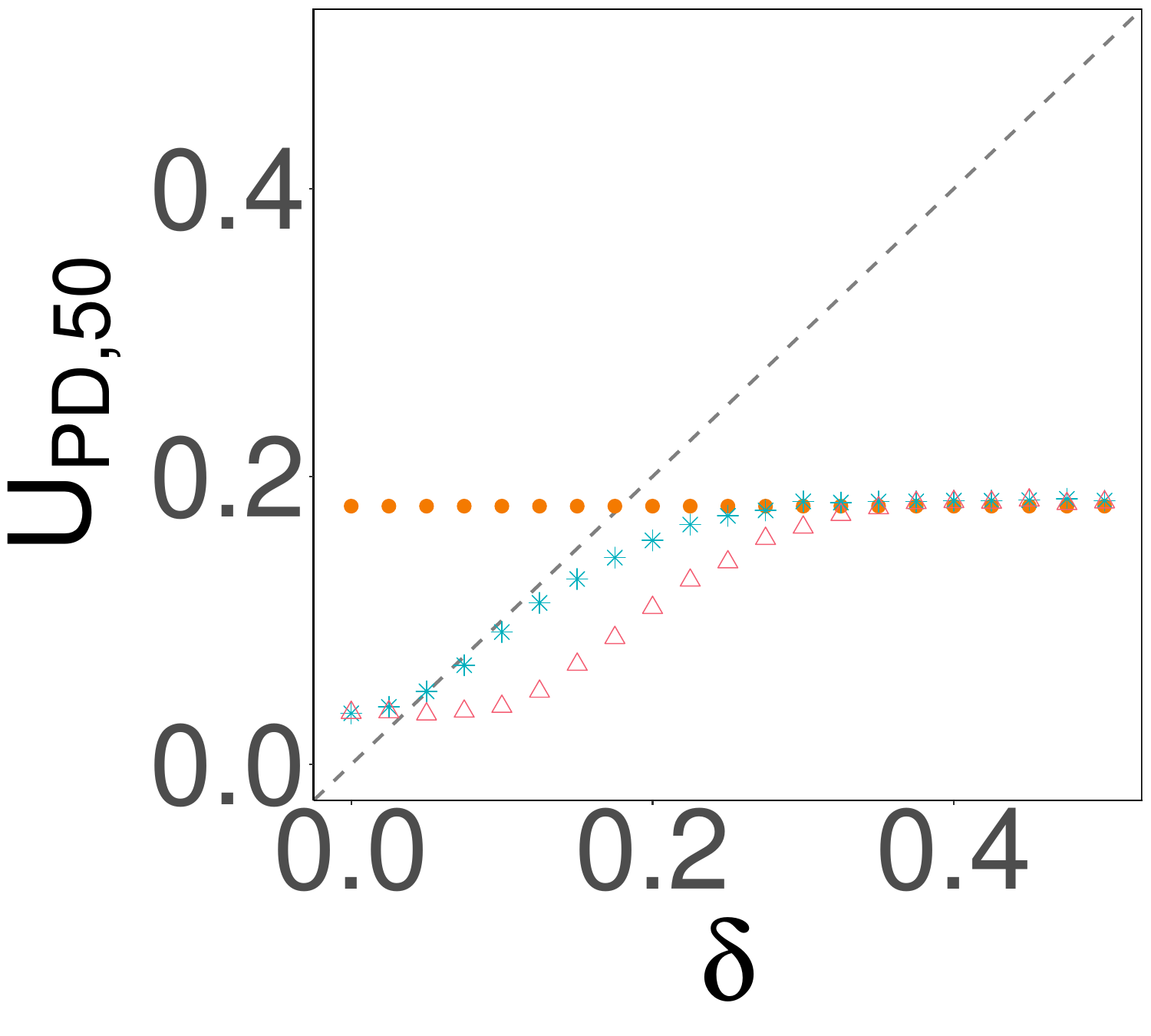} &
			\hspace{\thisgap}\includegraphics[width=\thiswidth]{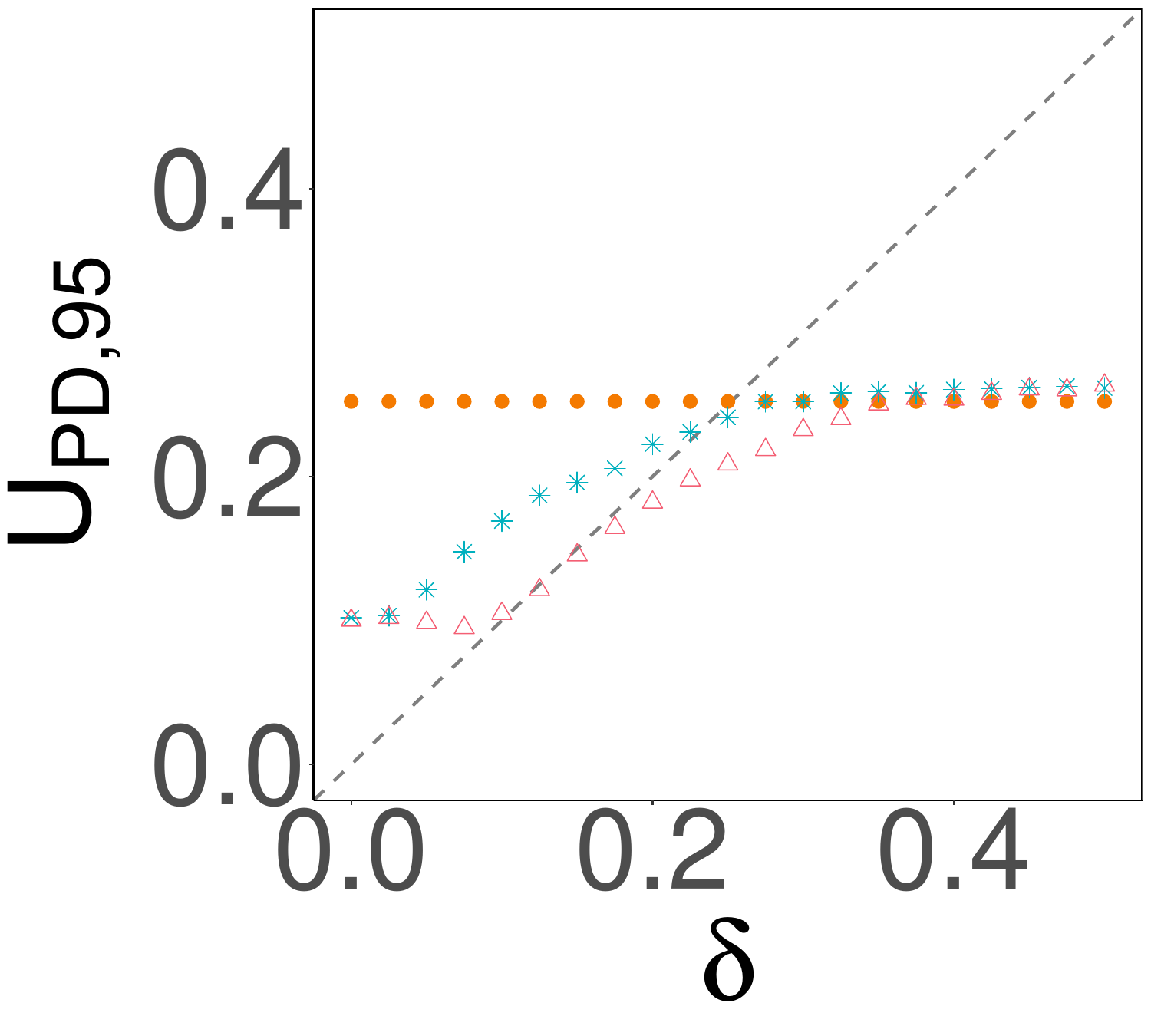} \\
			\hspace{\thisgap}\includegraphics[width=\thiswidth]{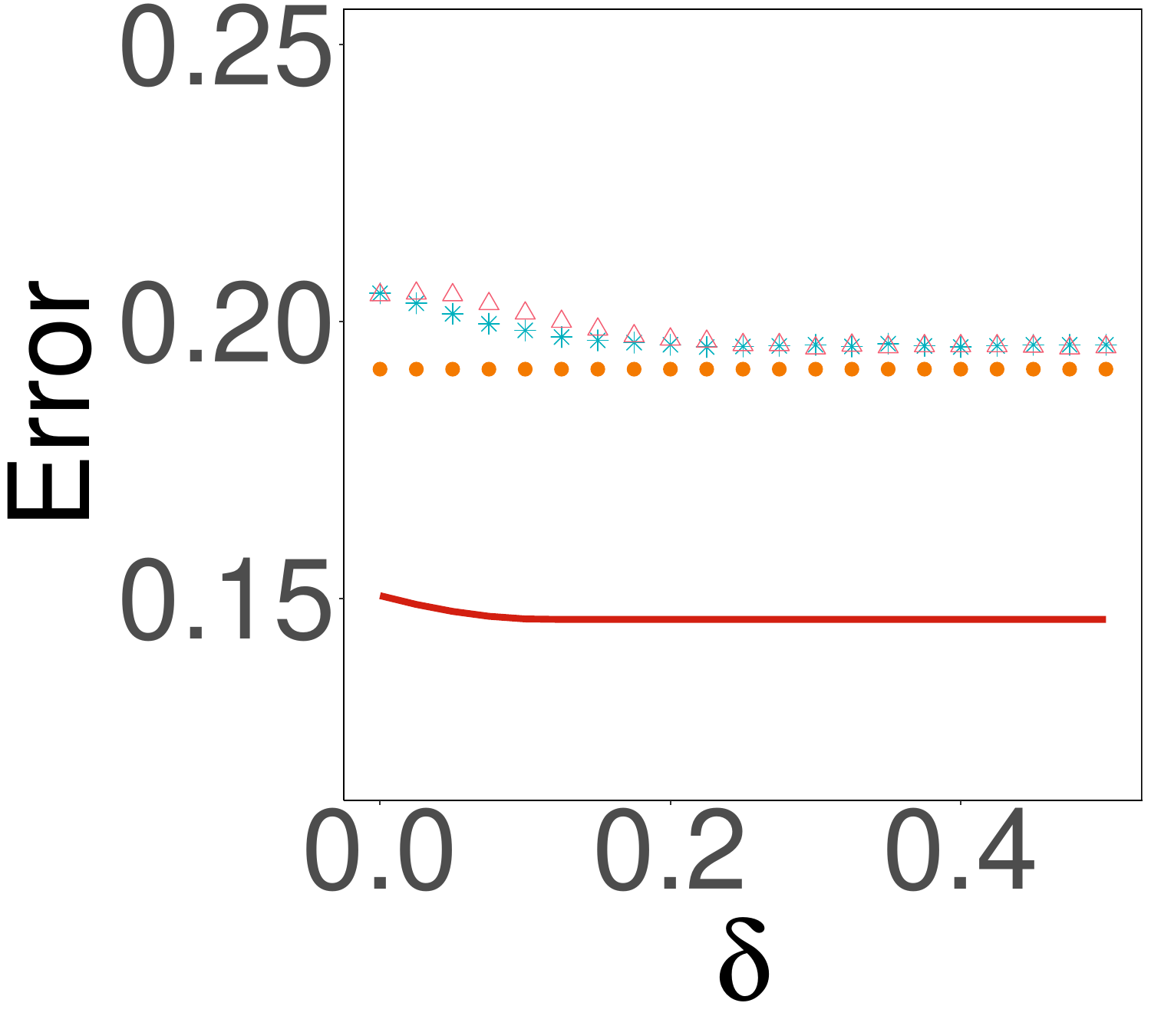} &
			\hspace{\thisgap}\includegraphics[width=\thiswidth]{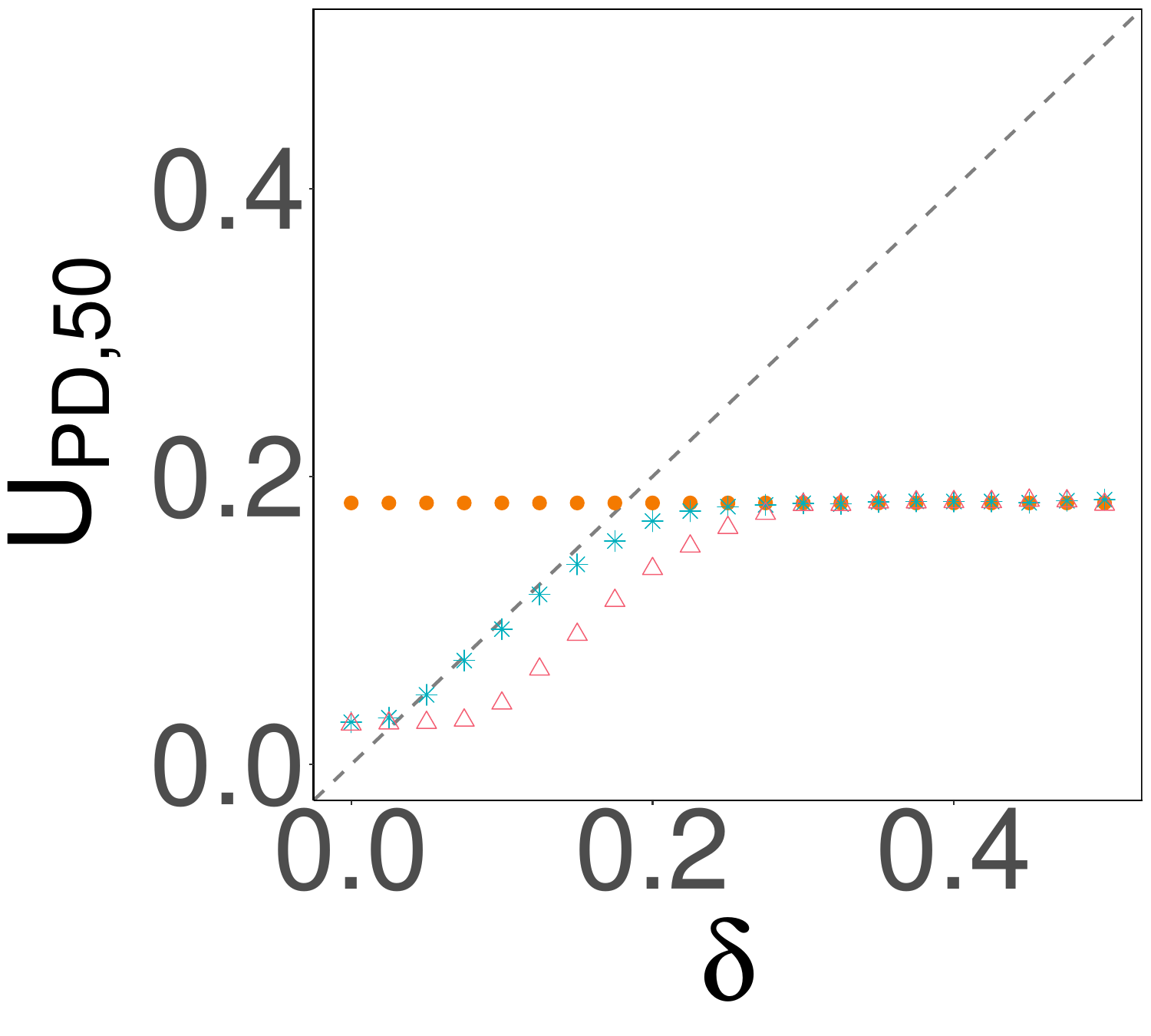} &
			\hspace{\thisgap}\includegraphics[width=\thiswidth]{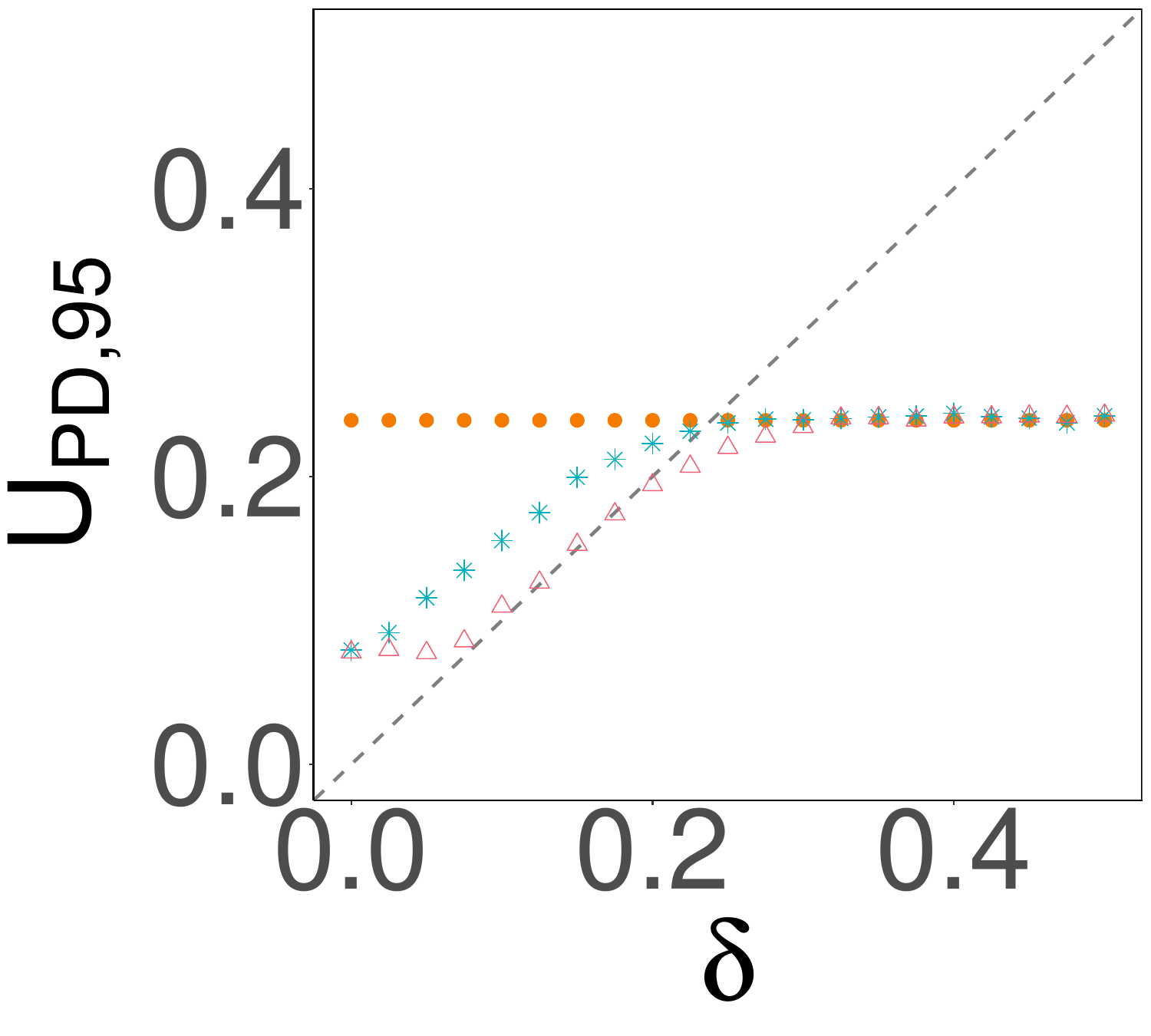} \\
			\hspace{\thisgap}\includegraphics[width=\thiswidth]{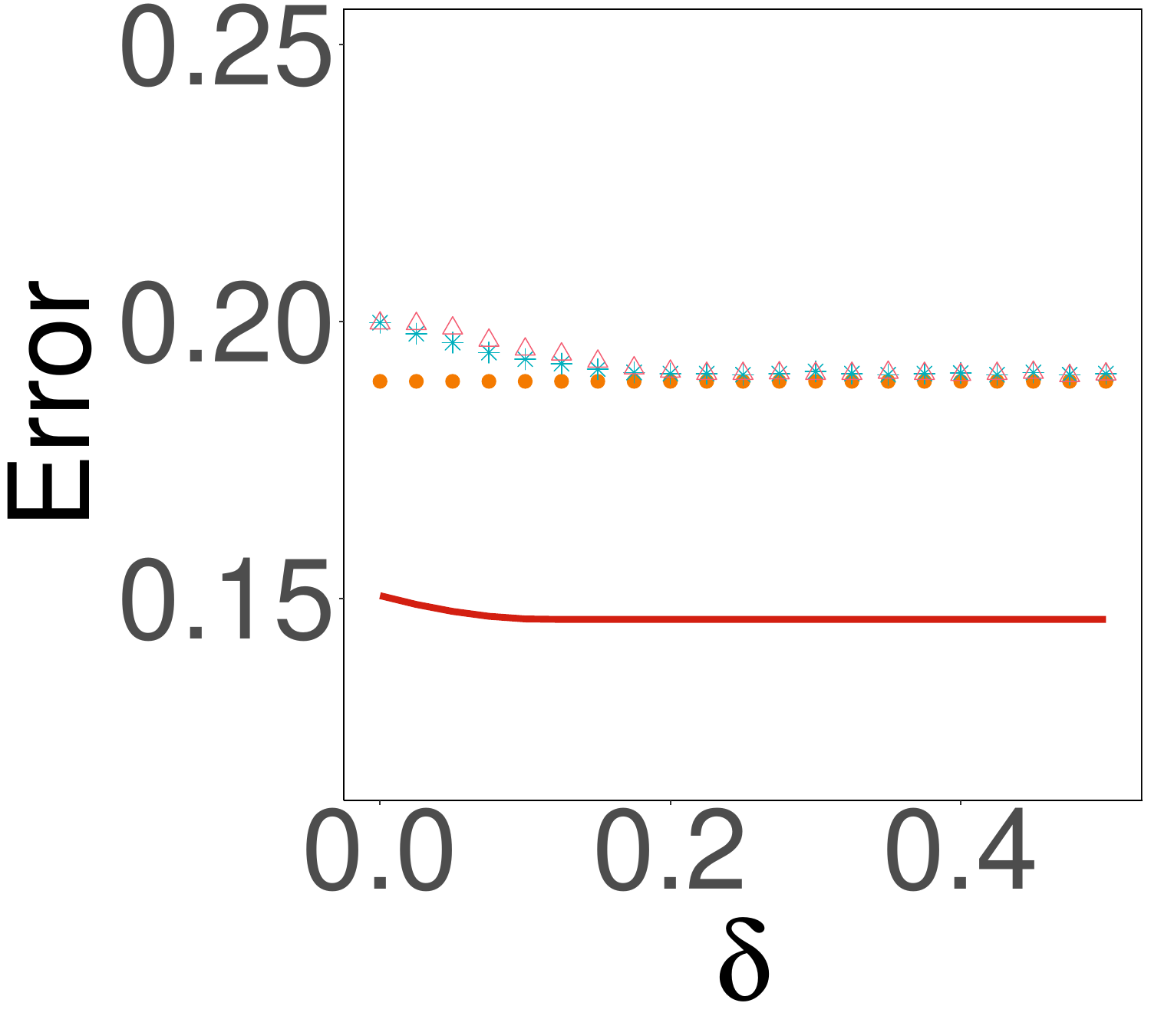} &
			\hspace{\thisgap}\includegraphics[width=\thiswidth]{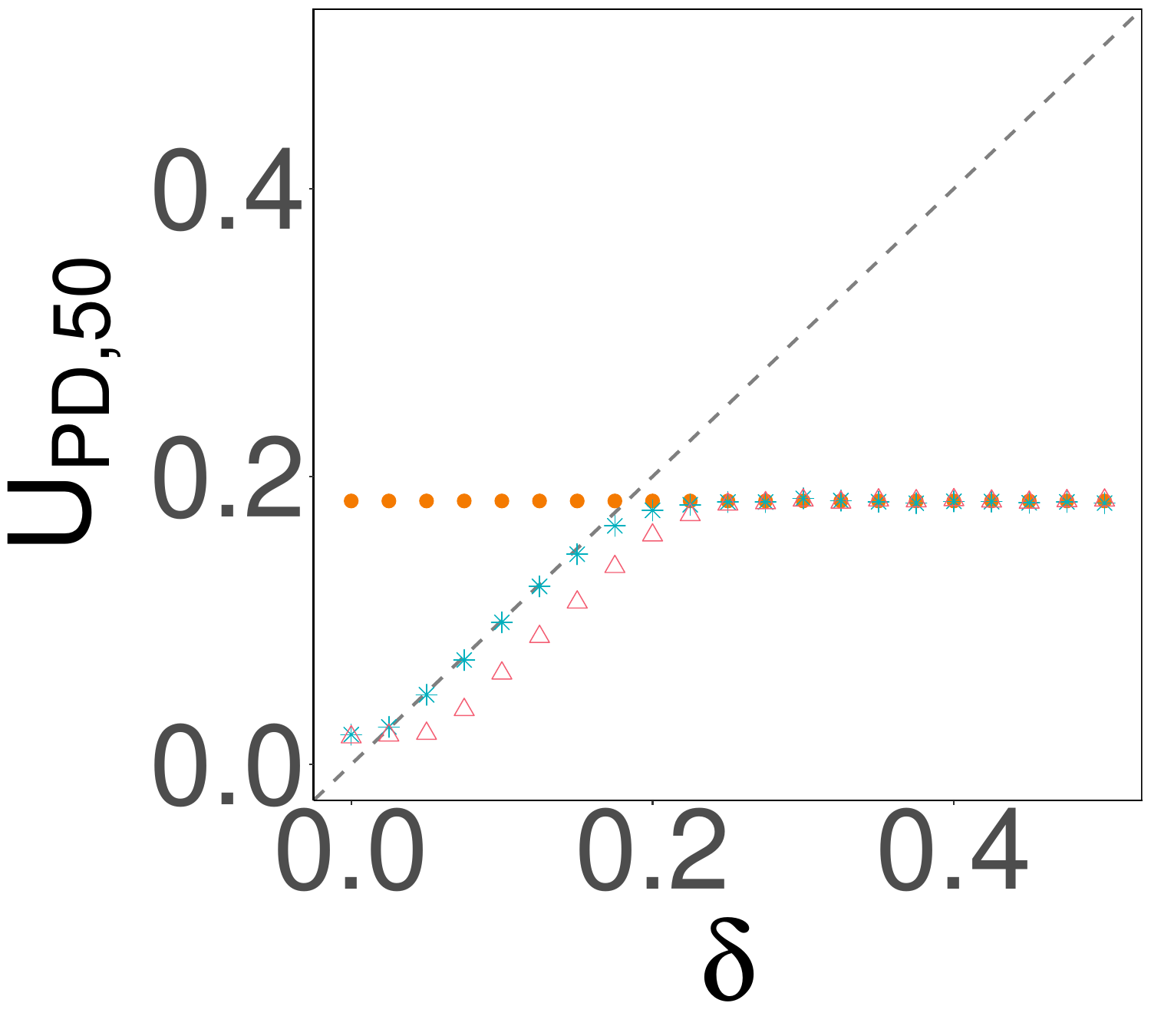} &
			\hspace{\thisgap}\includegraphics[width=\thiswidth]{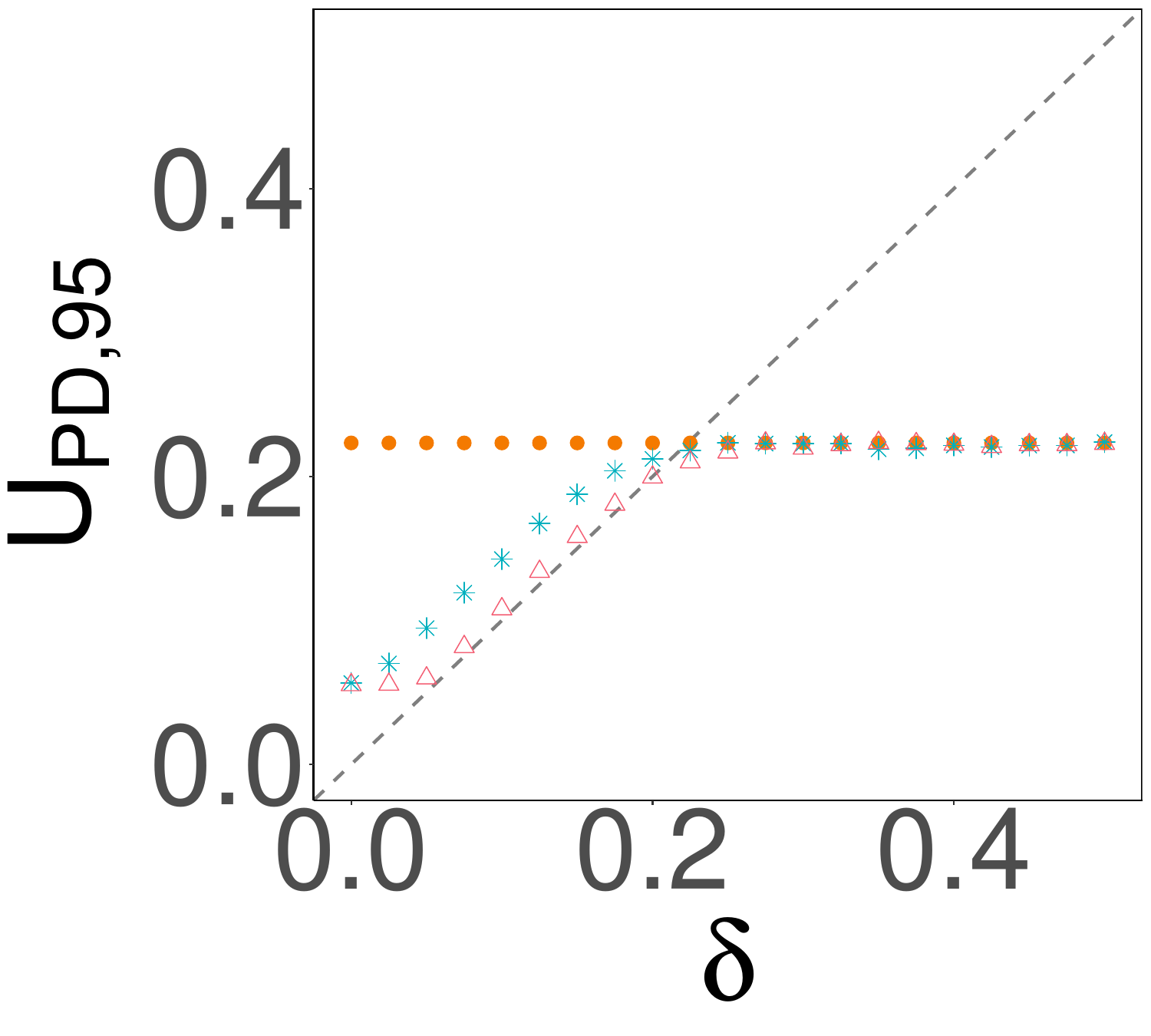}
		\end{tabular}
		\caption{Disparity PD results under the non-Gaussian model, $\beta=1.5$. Top: $n=1000$; middle: $n=2000$; bottom: $n=5000$.}
		\label{fig:PD_unif_beta_1.5}
	\end{center}
\end{figure*}

\begin{figure*}[!htbp]
	\begin{center}
		\newcommand{\thiswidth}{0.18\linewidth}
		\newcommand{\thisgap}{0mm}
		\begin{tabular}{ccc}
			\hspace{\thisgap}\includegraphics[width=\thiswidth]{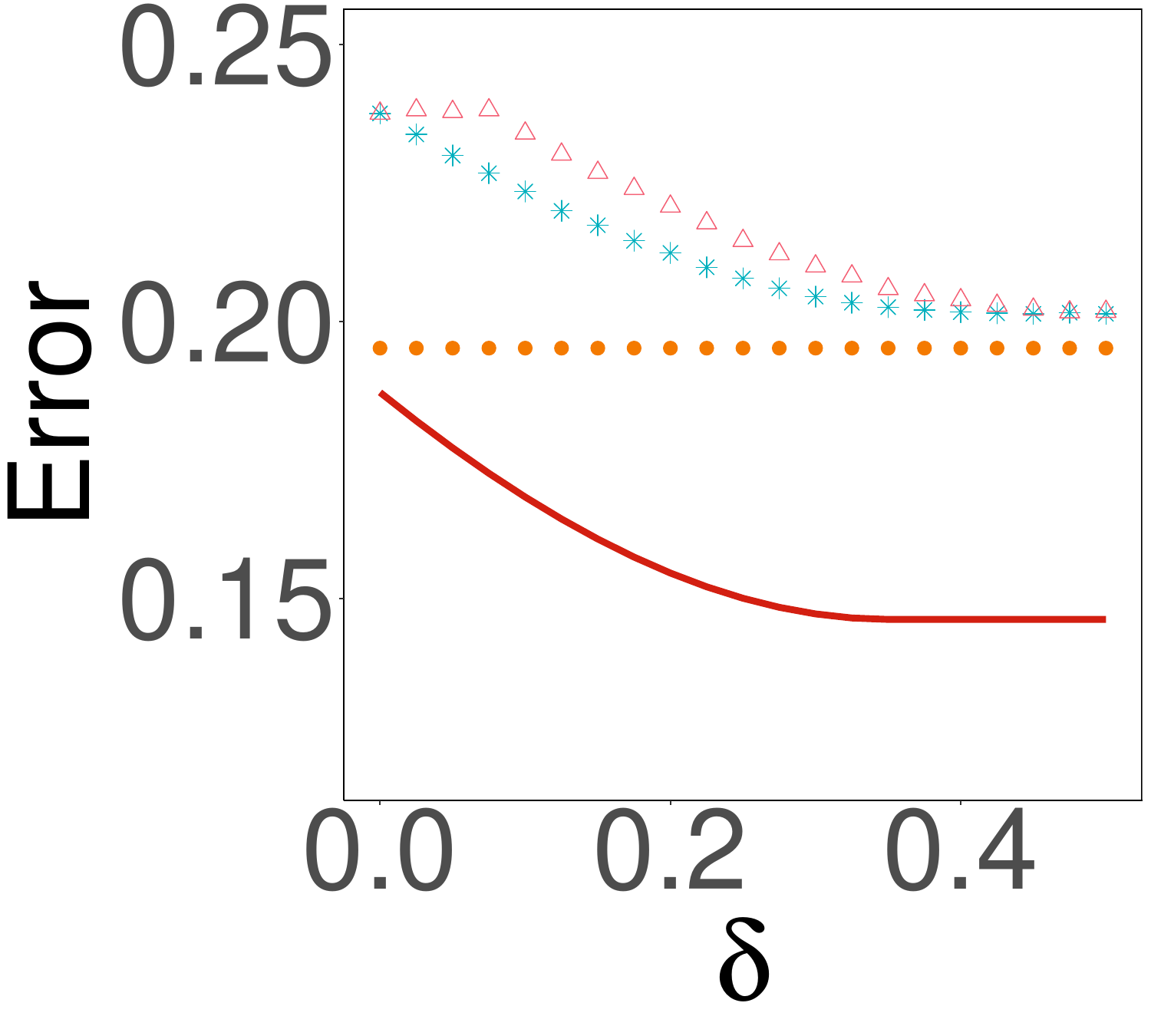} &
			\hspace{\thisgap}\includegraphics[width=\thiswidth]{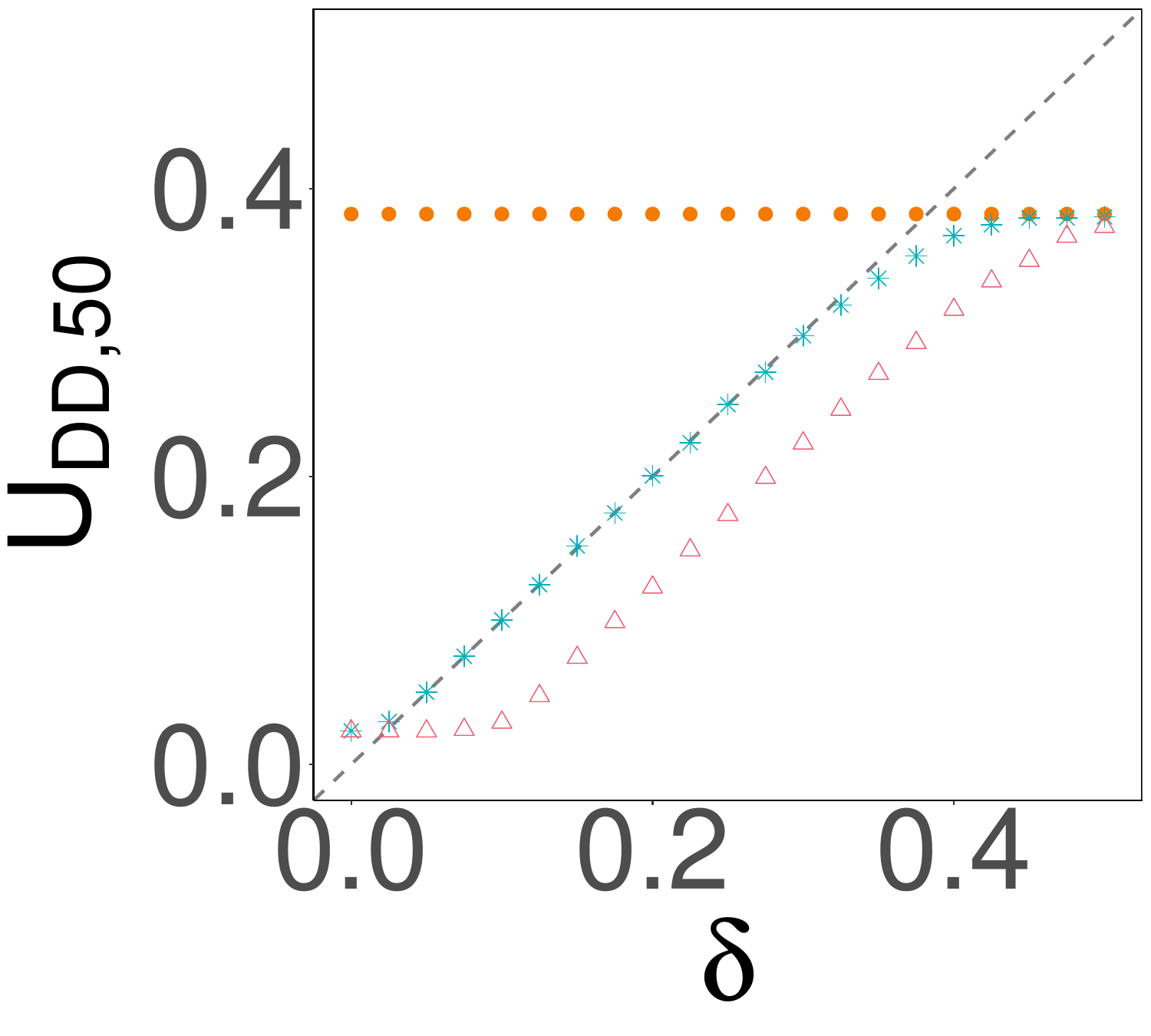} &
			\hspace{\thisgap}\includegraphics[width=\thiswidth]{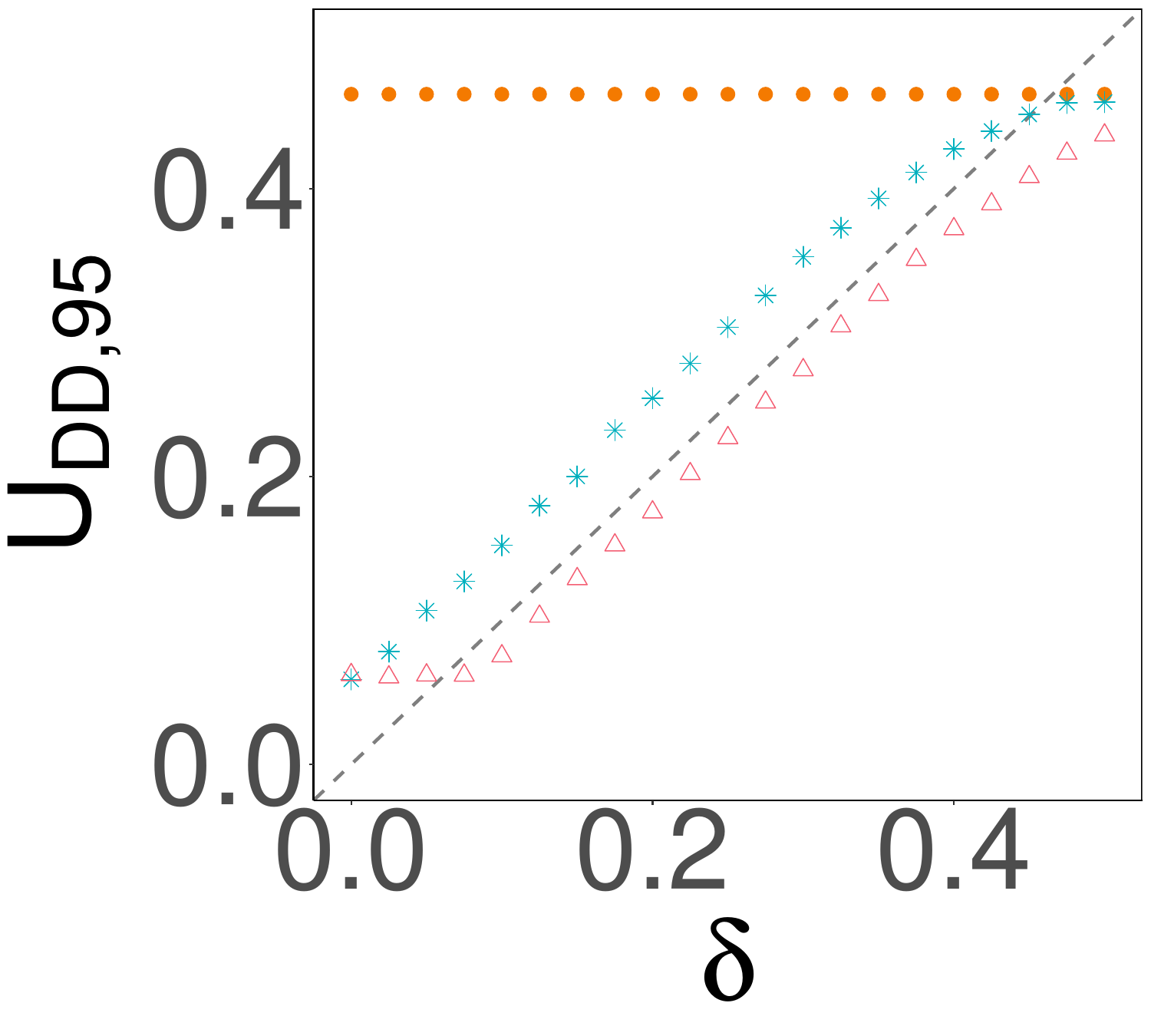} \\
			\hspace{\thisgap}\includegraphics[width=\thiswidth]{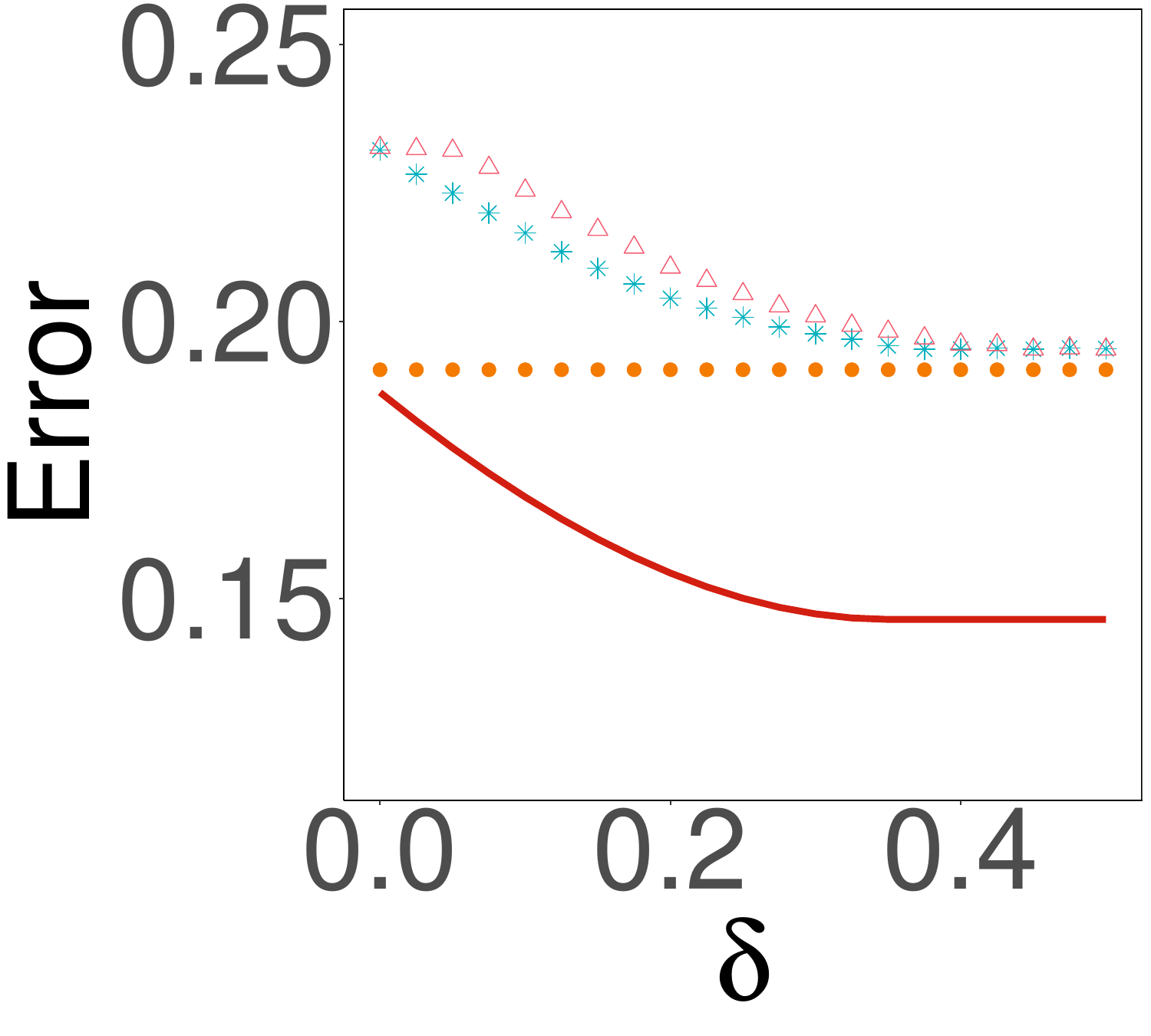} &
			\hspace{\thisgap}\includegraphics[width=\thiswidth]{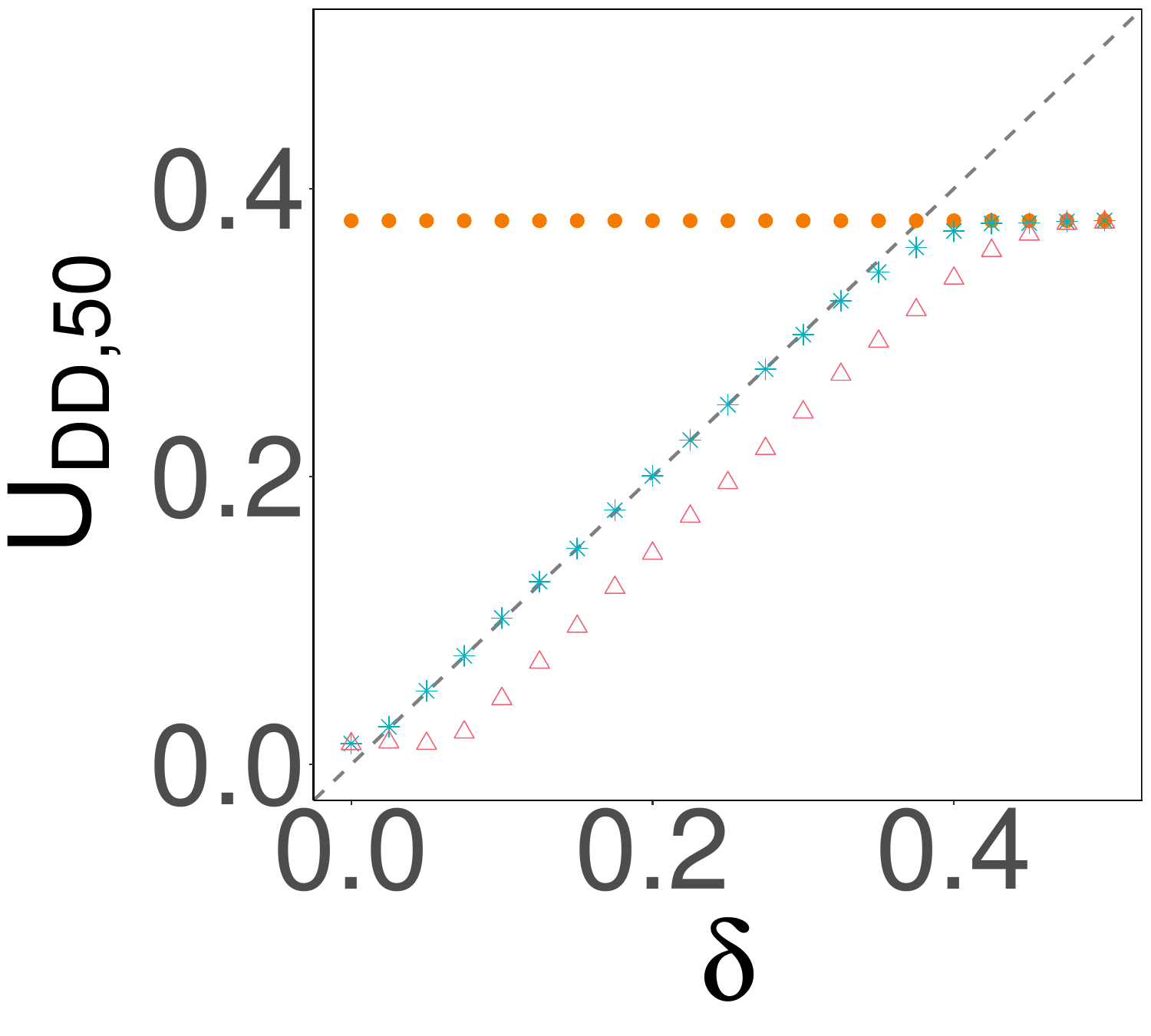} &
			\hspace{\thisgap}\includegraphics[width=\thiswidth]{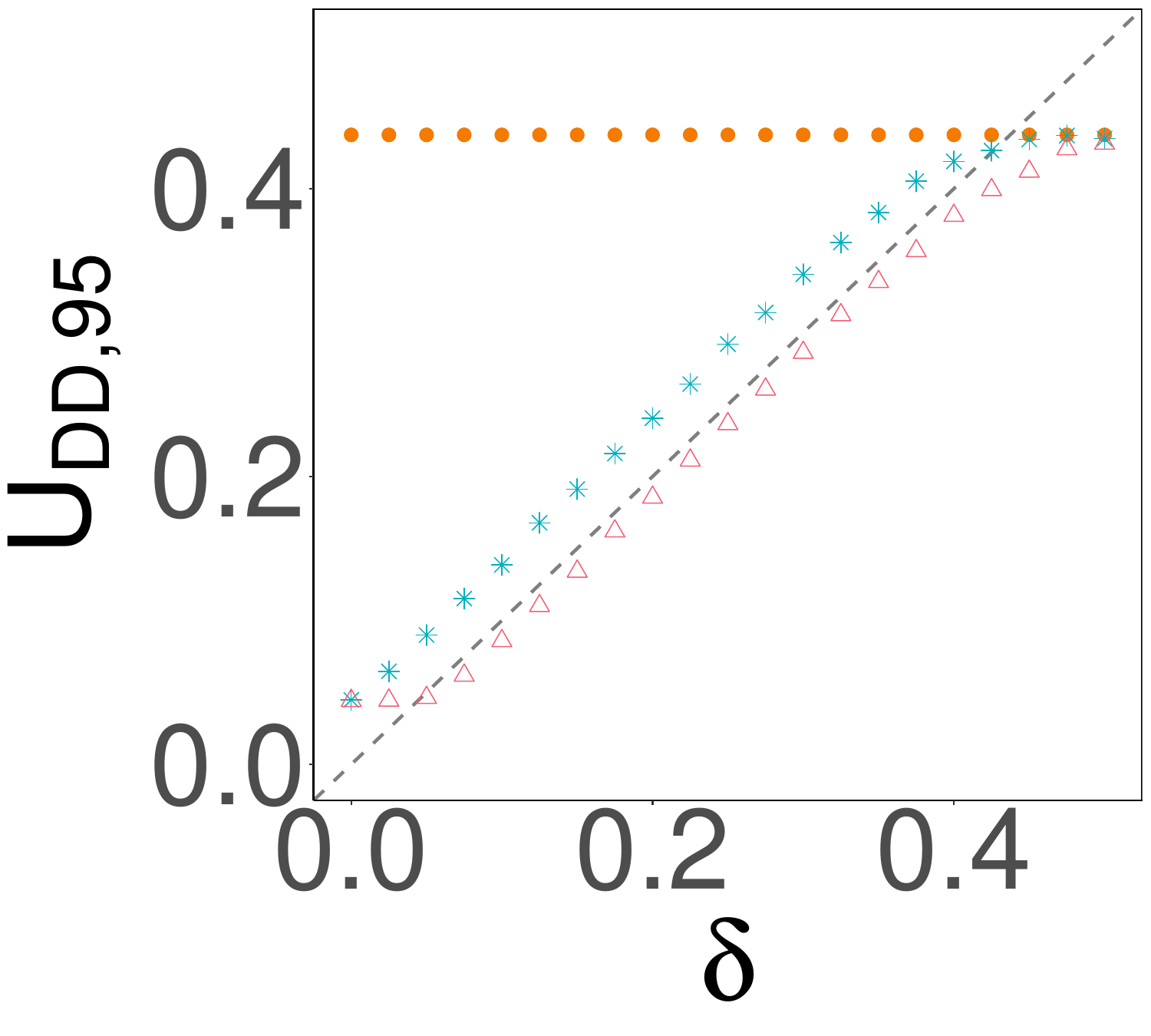} \\
			\hspace{\thisgap}\includegraphics[width=\thiswidth]{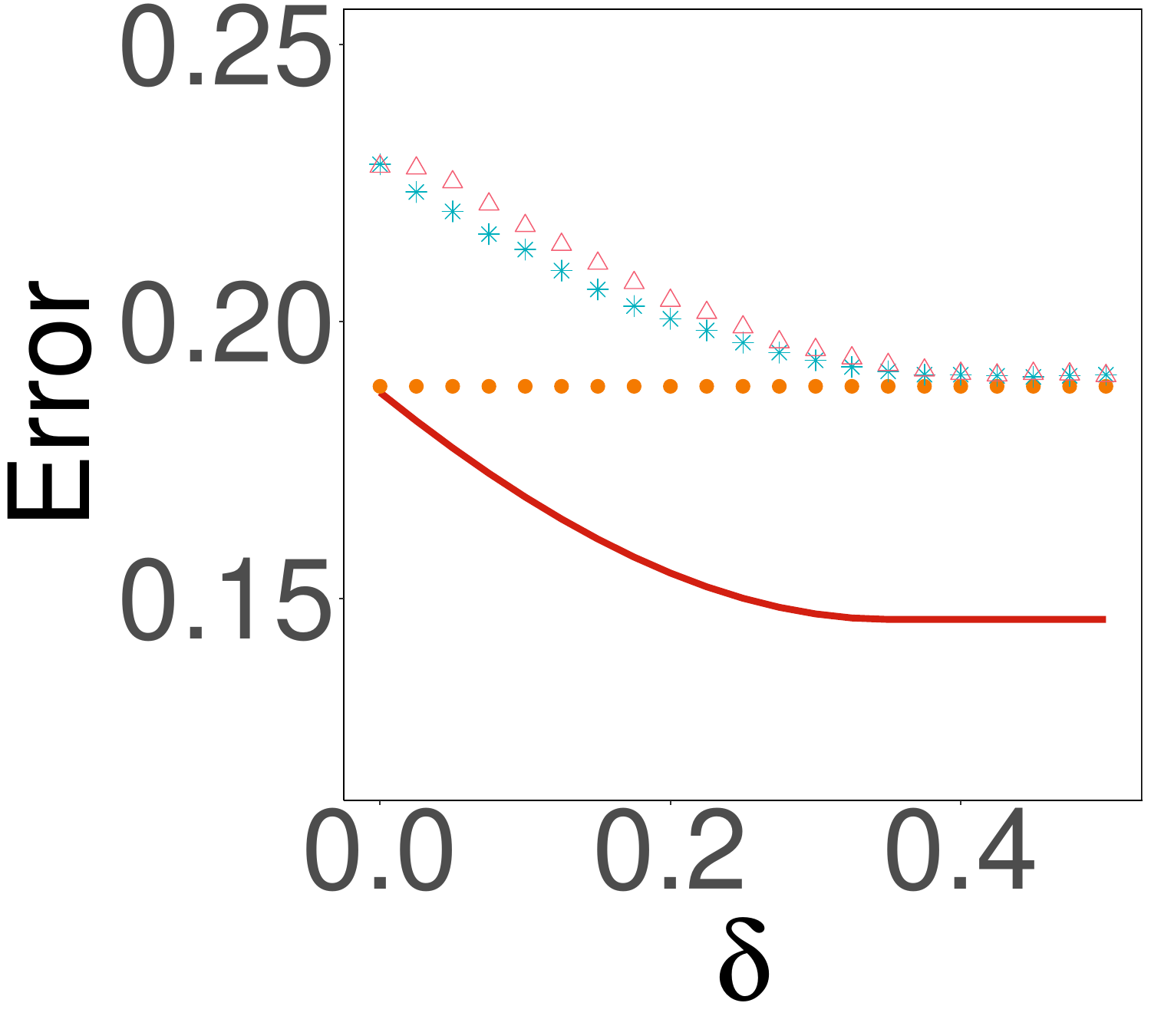} &
			\hspace{\thisgap}\includegraphics[width=\thiswidth]{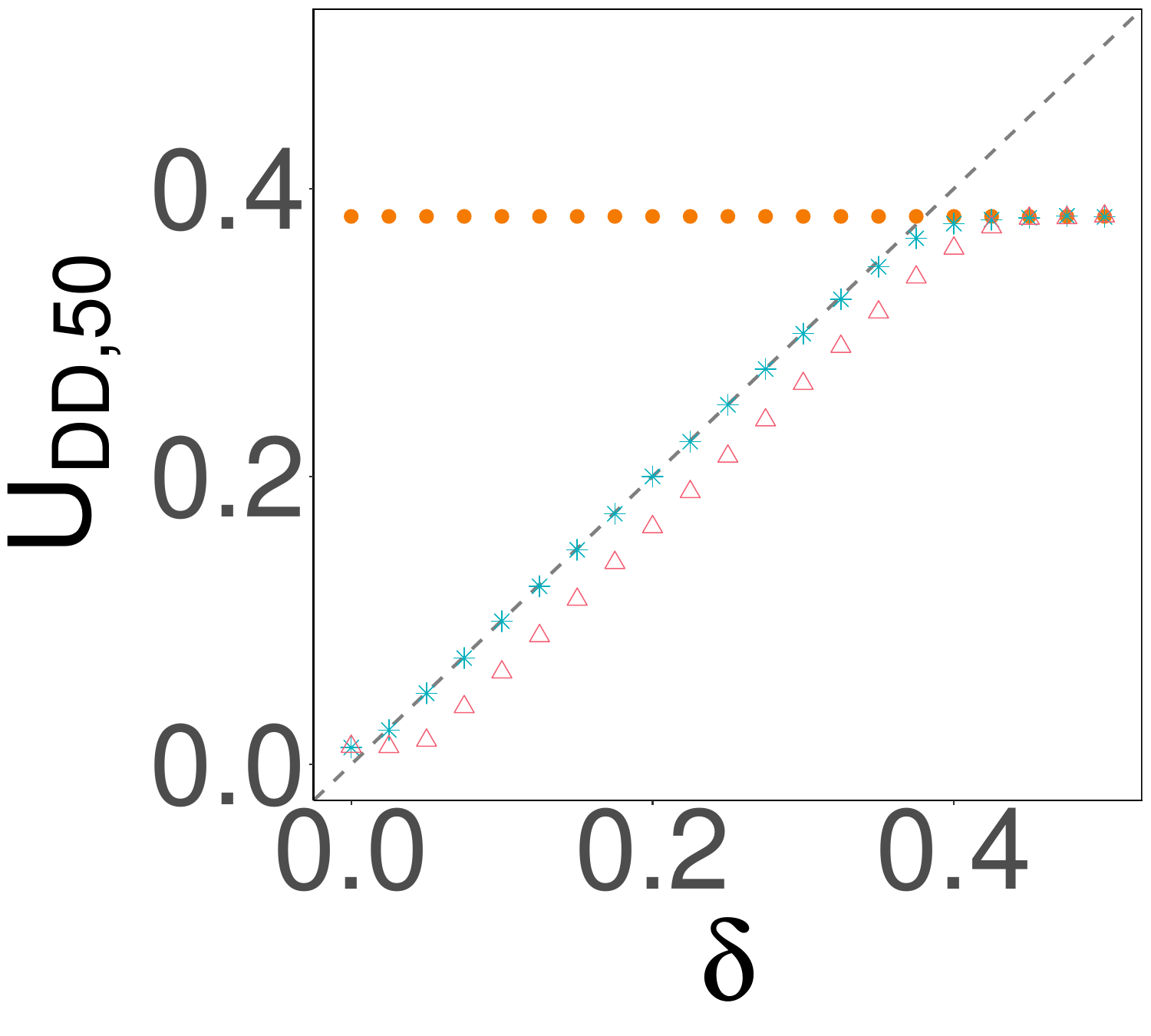} &
			\hspace{\thisgap}\includegraphics[width=\thiswidth]{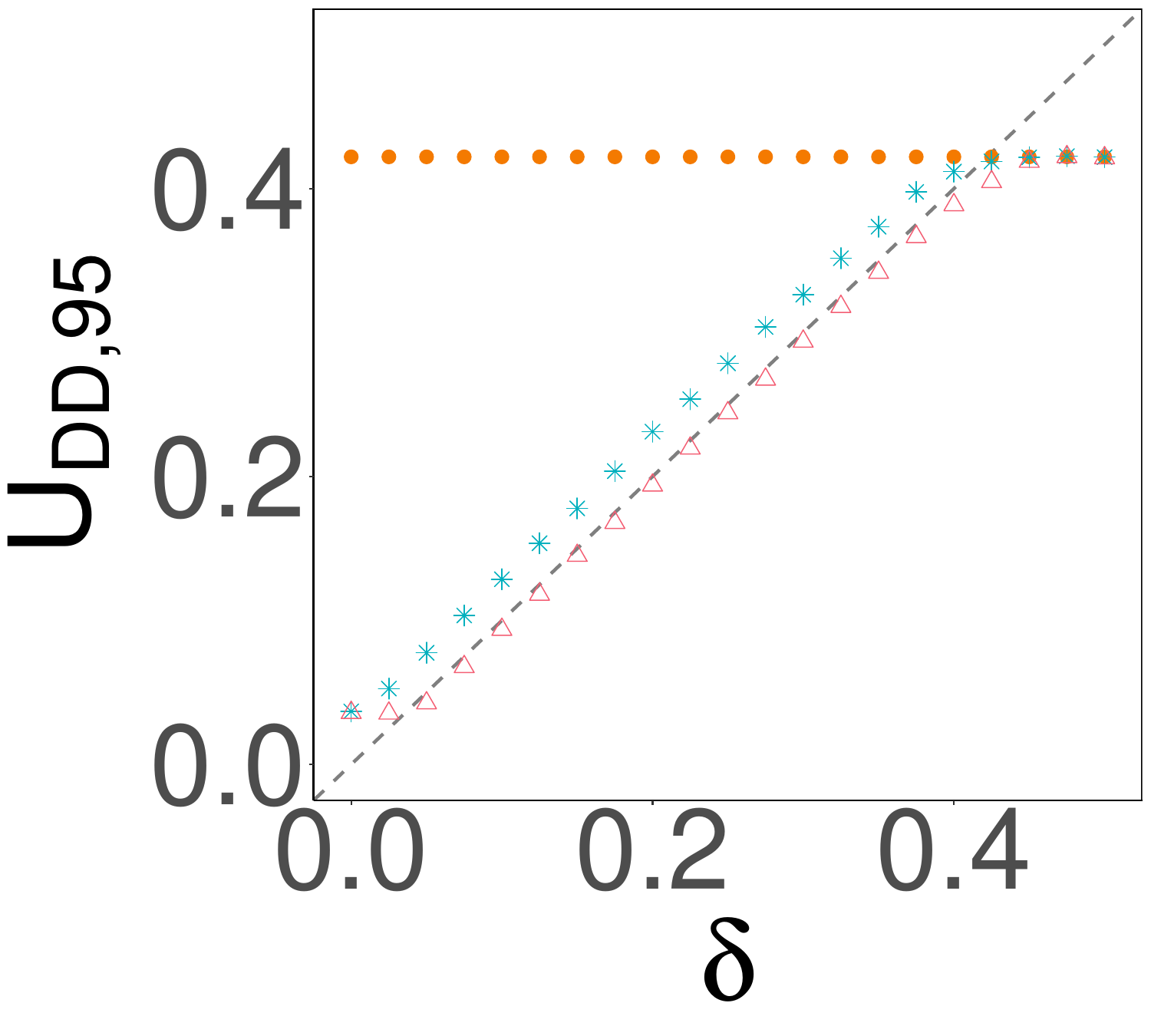}
		\end{tabular}
		\caption{Disparity DD results under the non-Gaussian model, $\beta=1.5$. Top: $n=1000$; middle: $n=2000$; bottom: $n=5000$. }
		\label{fig:DD_unif_beta_1.5}
	\end{center}
\end{figure*}

\newpage
\subsection{Results under perfect classification}\label{sec:perfect}

Functional data exhibit a unique property known as perfect classification, where the classification error can vanish, a phenomenon that does not typically arise in multivariate data.
To evaluate the performance of our algorithm in such scenarios, we consider the Gaussian model in Section \ref{sec:sim} with $\beta=0.5$, under which the signal-to-noise ratio is sufficiently high to mimic the perfect classification regime. For class probabilities, we examine two settings: (I) $\prob(Y=1|A=0)=0.4, \prob(Y=1|A=1)=0.7$; and (II) $\prob(Y=1|A=0)=\prob(Y=1|A=1)=0.5$, while keeping all other model parameters in Section \ref{sec:sim} unchanged. 

As discussed in Remark \ref{rmk:perfect}, under setting (I), the classical unconstrained Bayes classifier is automatically fair with respect to DO and PD. In setting (II),  it is automatically fair with respect to all the three disparity measures DO, PD and DD. To visualize this difference, we plot DD as a function of $\tau$ in Figure \ref{fig:DD_tau_beta_0.5}.

The results are presented in Figures \ref{fig:DO_gauss_beta_0.5}-\ref{fig:DD_gauss_beta_0.5_II}. Under the disparity measures DO and PD in setting (I), the classification errors of the fairness-aware classifiers are nearly zero, and the median disparity levels converge to zero across all values of $\delta$ as $n$ increases. This confirms that our approach naturally reduces to the classical FLDA classifier in such automatically fair cases. In contrast, under DD in setting (I), the fact that $|DD(\tau)| \equiv 0.3$ indicates that it is infeasible to achieve lower disparity levels in this setting. By comparison, in setting (II), the unconstrained Bayes classifier is automatically fair under DD, and the empirical results exhibit a similar pattern to those observed for DO and PD in setting (I).

Overall, in perfect classification cases, our proposed algorithm continues to perform comparably to the oracle fairness-aware Bayes optimal  classifier, further highlighting its effectiveness and adaptability.

\begin{figure*}[!htbp]
	\begin{center}
		\newcommand{\thiswidth}{0.2\linewidth}
		\newcommand{\thisgap}{0mm}
		\begin{tabular}{cc}
			\hspace{\thisgap}\includegraphics[width=\thiswidth]{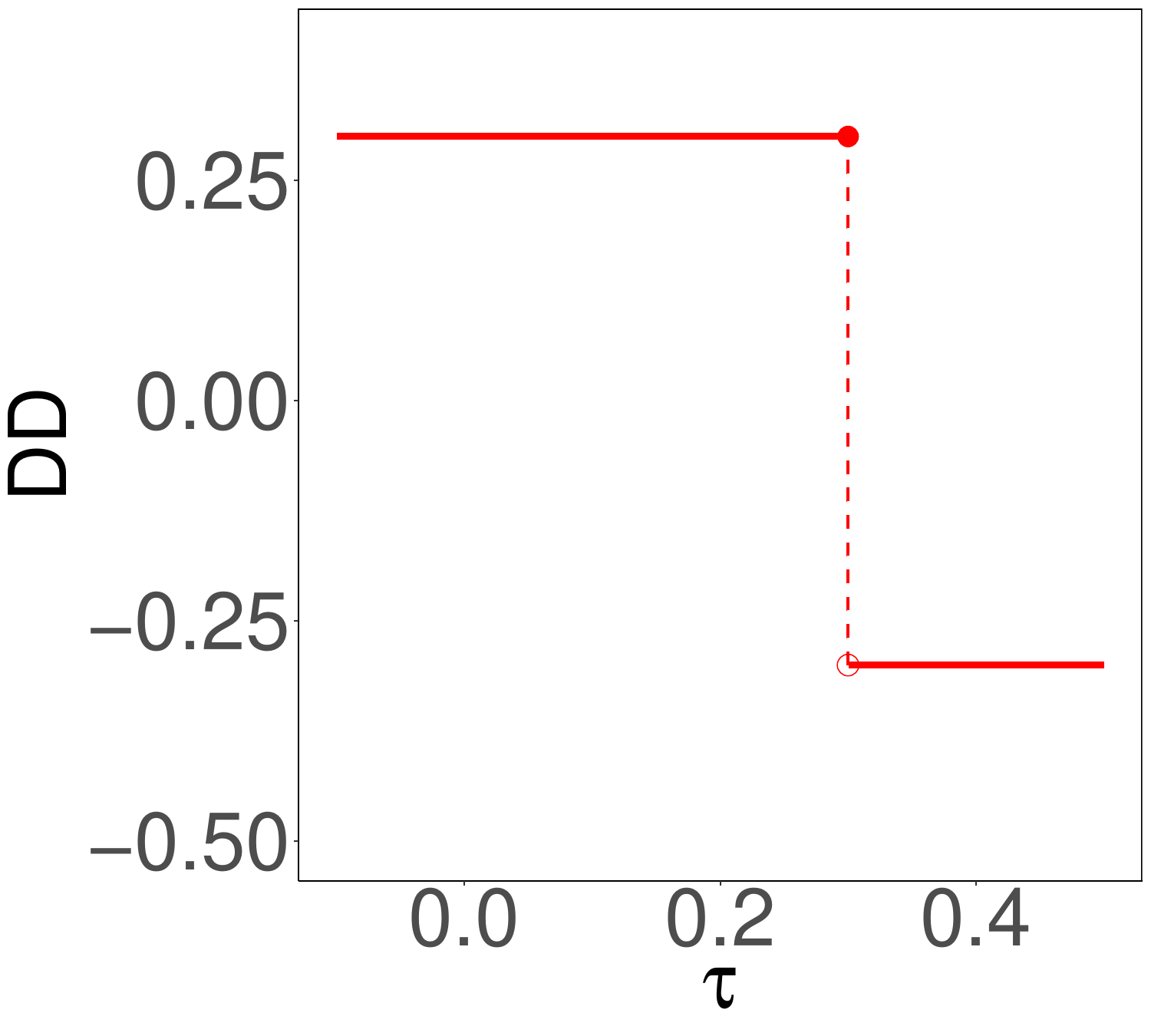} &
			\hspace{\thisgap}\includegraphics[width=\thiswidth]{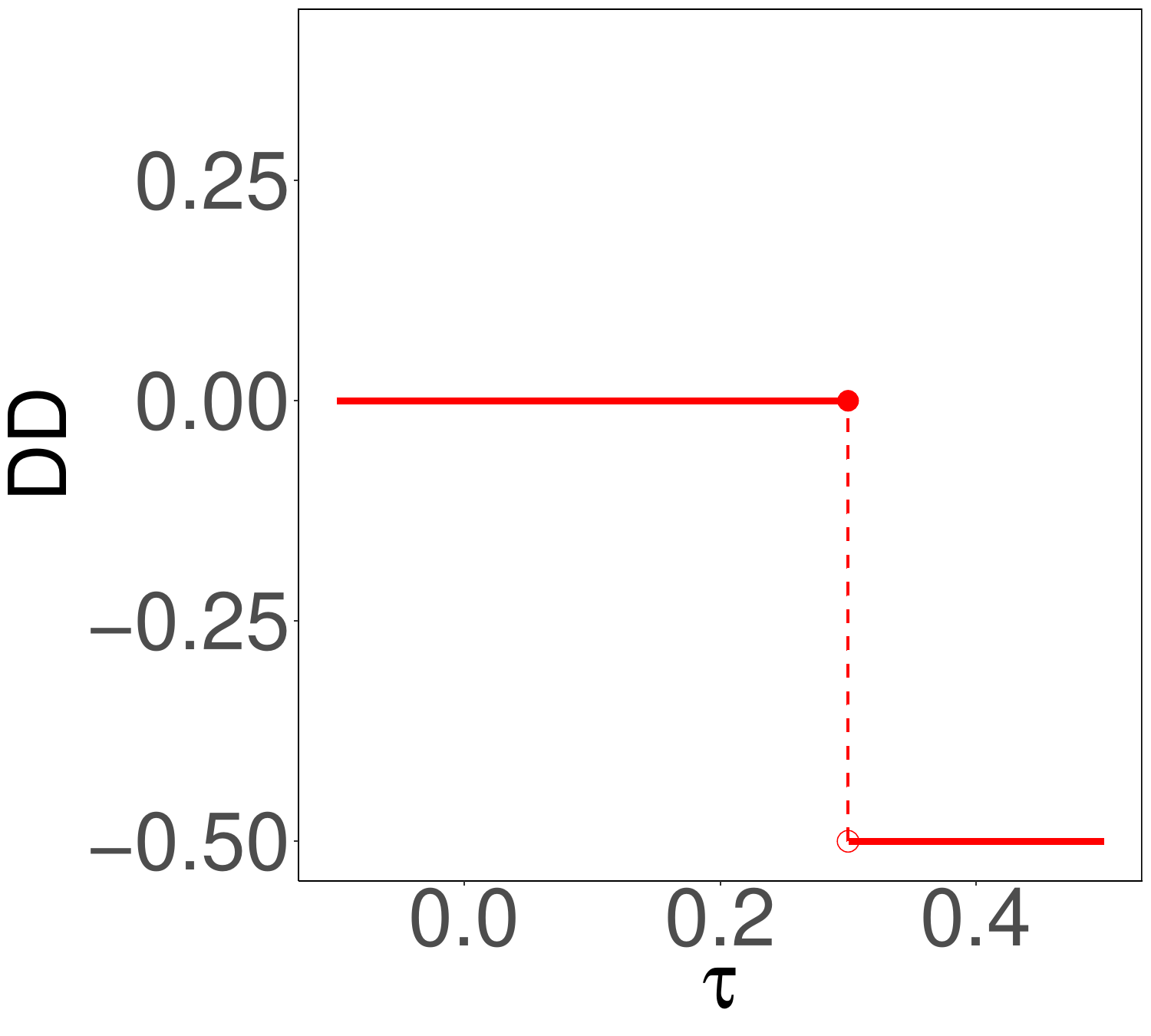} 
		\end{tabular}
		\caption{Oracle disparity DD versus $\tau$. Left: (I); right: (II).}
		\label{fig:DD_tau_beta_0.5}
	\end{center}
\end{figure*}

\begin{figure*}[!htbp]
	\begin{center}
		\newcommand{\thiswidth}{0.2\linewidth}
		\newcommand{\thisgap}{0mm}
		\begin{tabular}{ccc}
			\hspace{\thisgap}\includegraphics[width=\thiswidth]{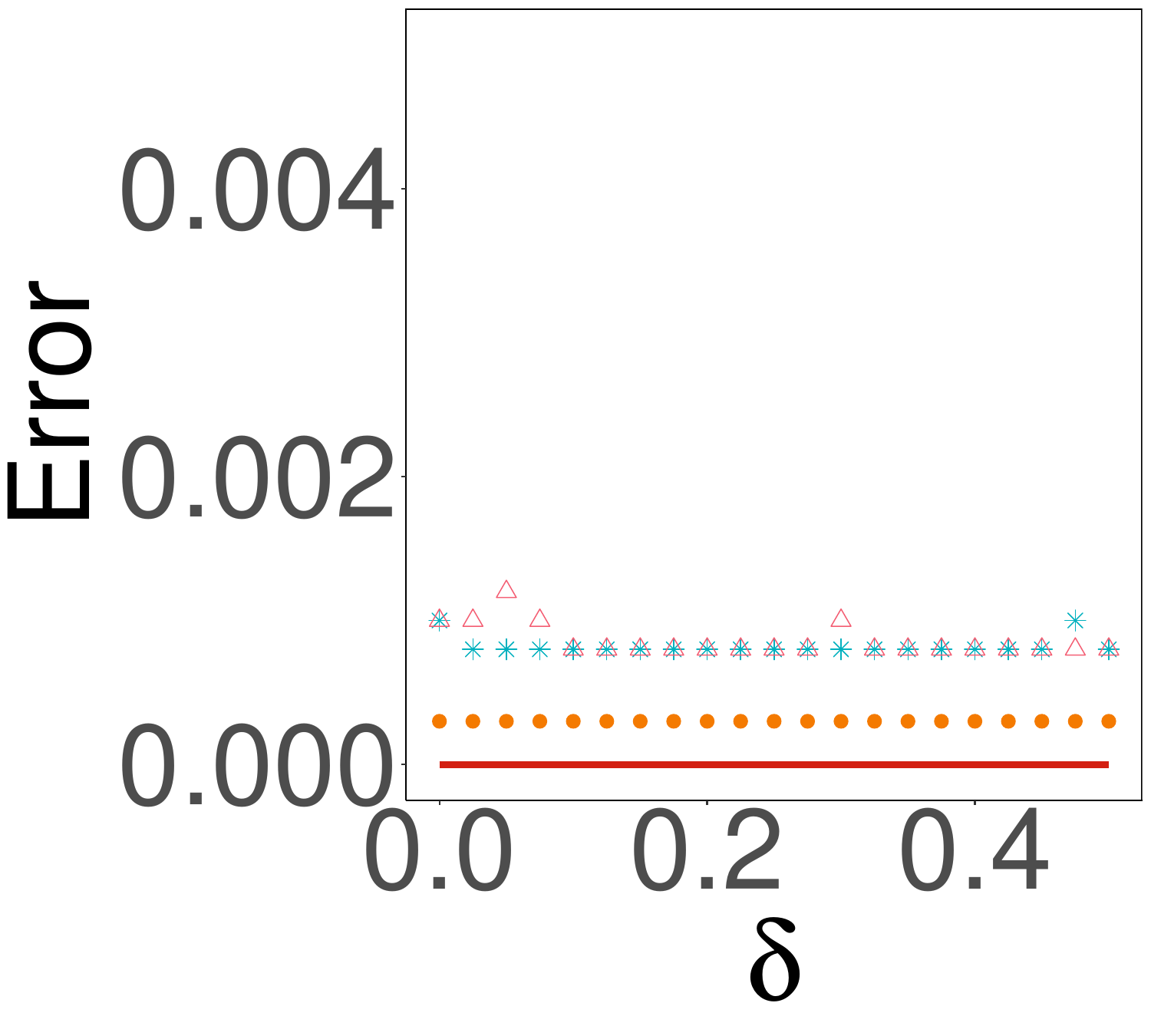} &
			\hspace{\thisgap}\includegraphics[width=\thiswidth]{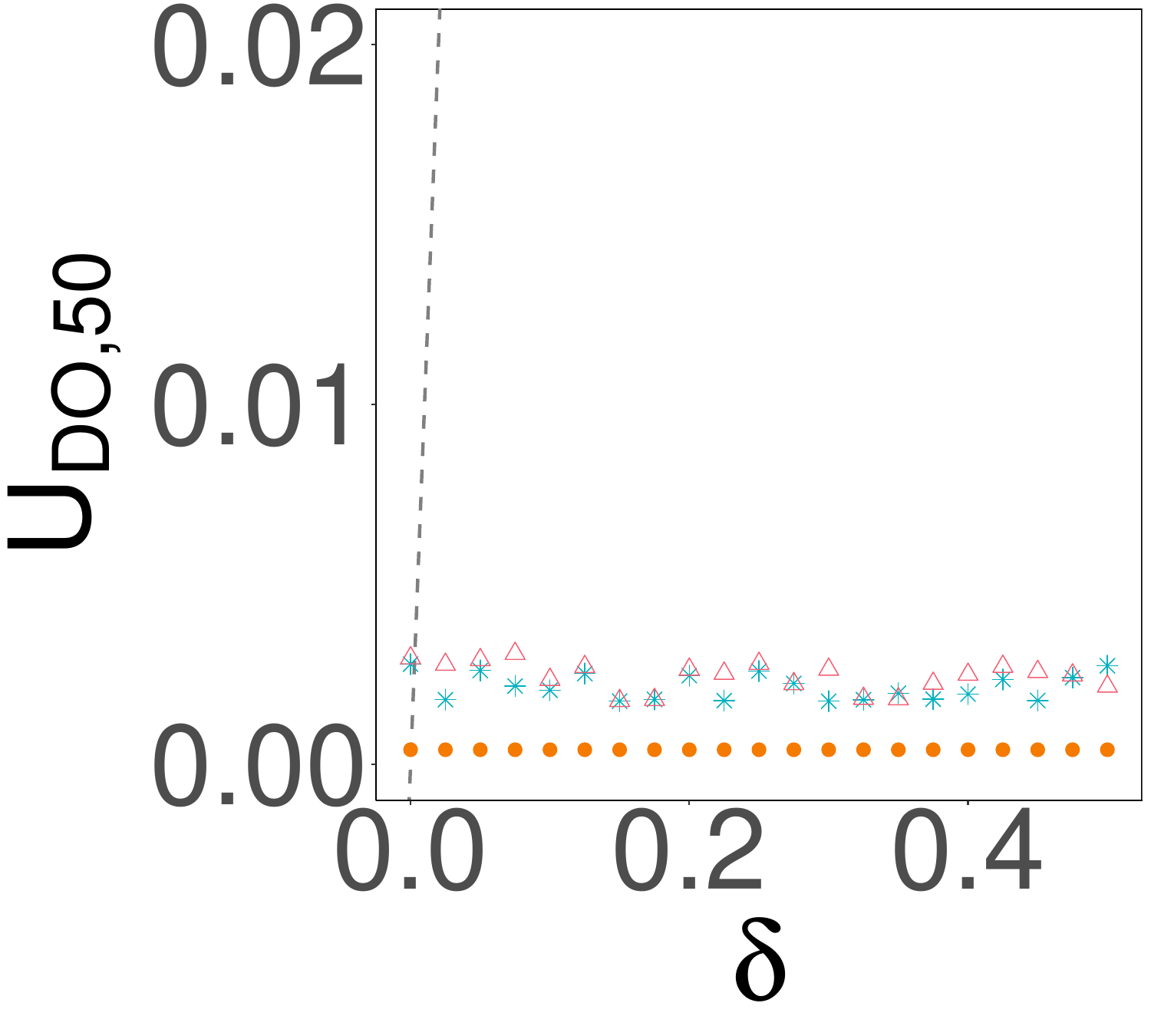} &
			\hspace{\thisgap}\includegraphics[width=\thiswidth]{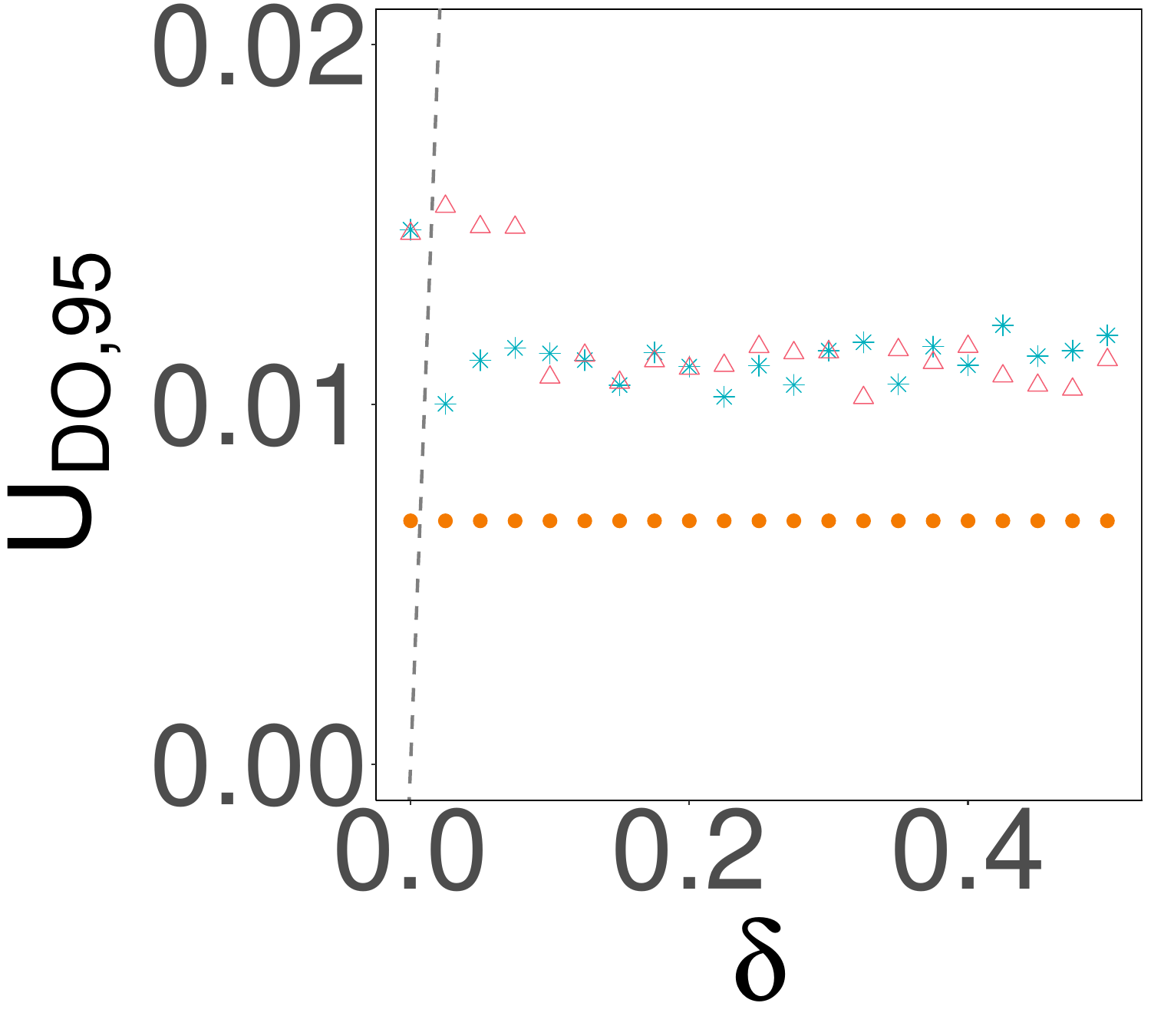} \\
			\hspace{\thisgap}\includegraphics[width=\thiswidth]{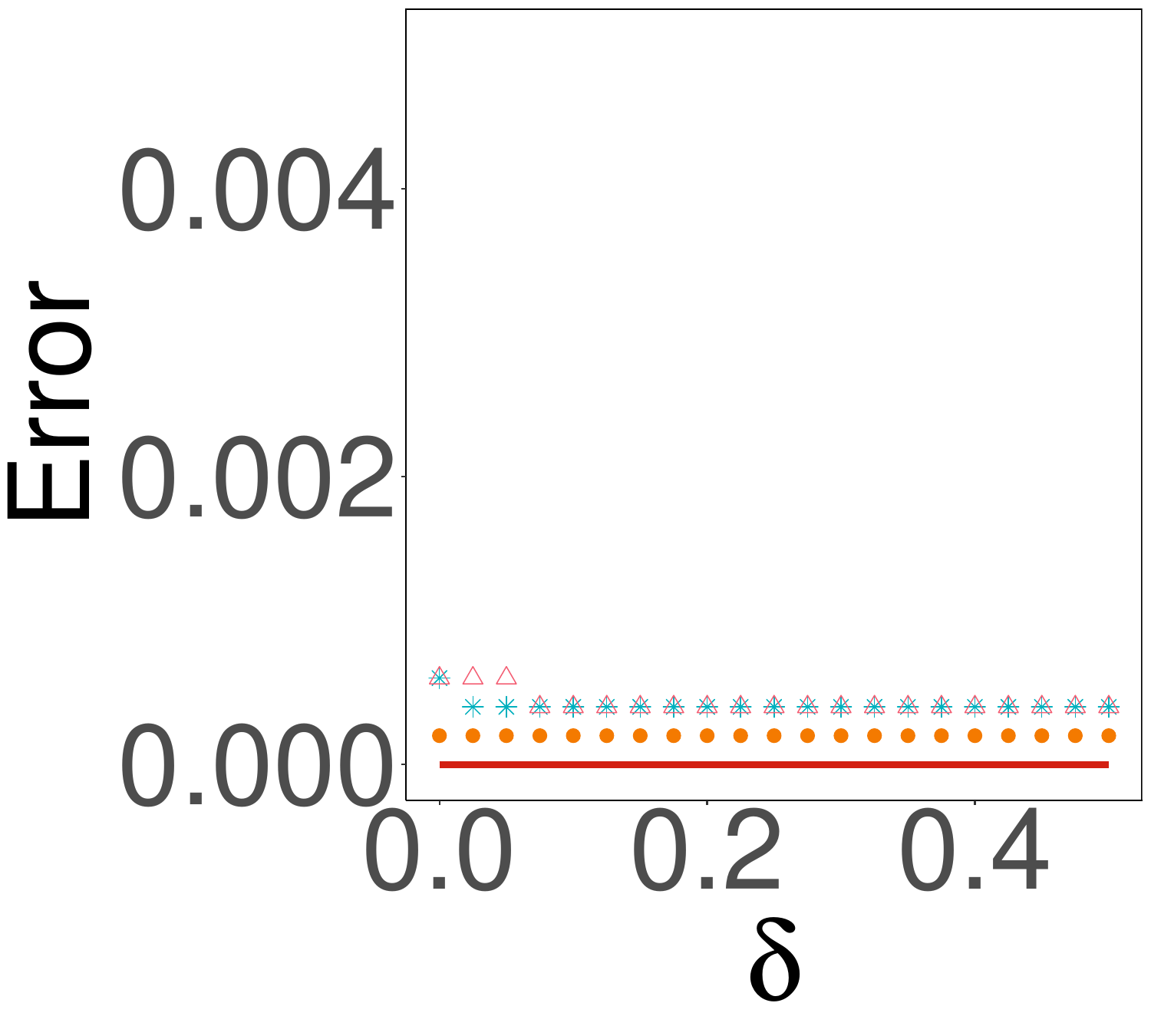} &
			\hspace{\thisgap}\includegraphics[width=\thiswidth]{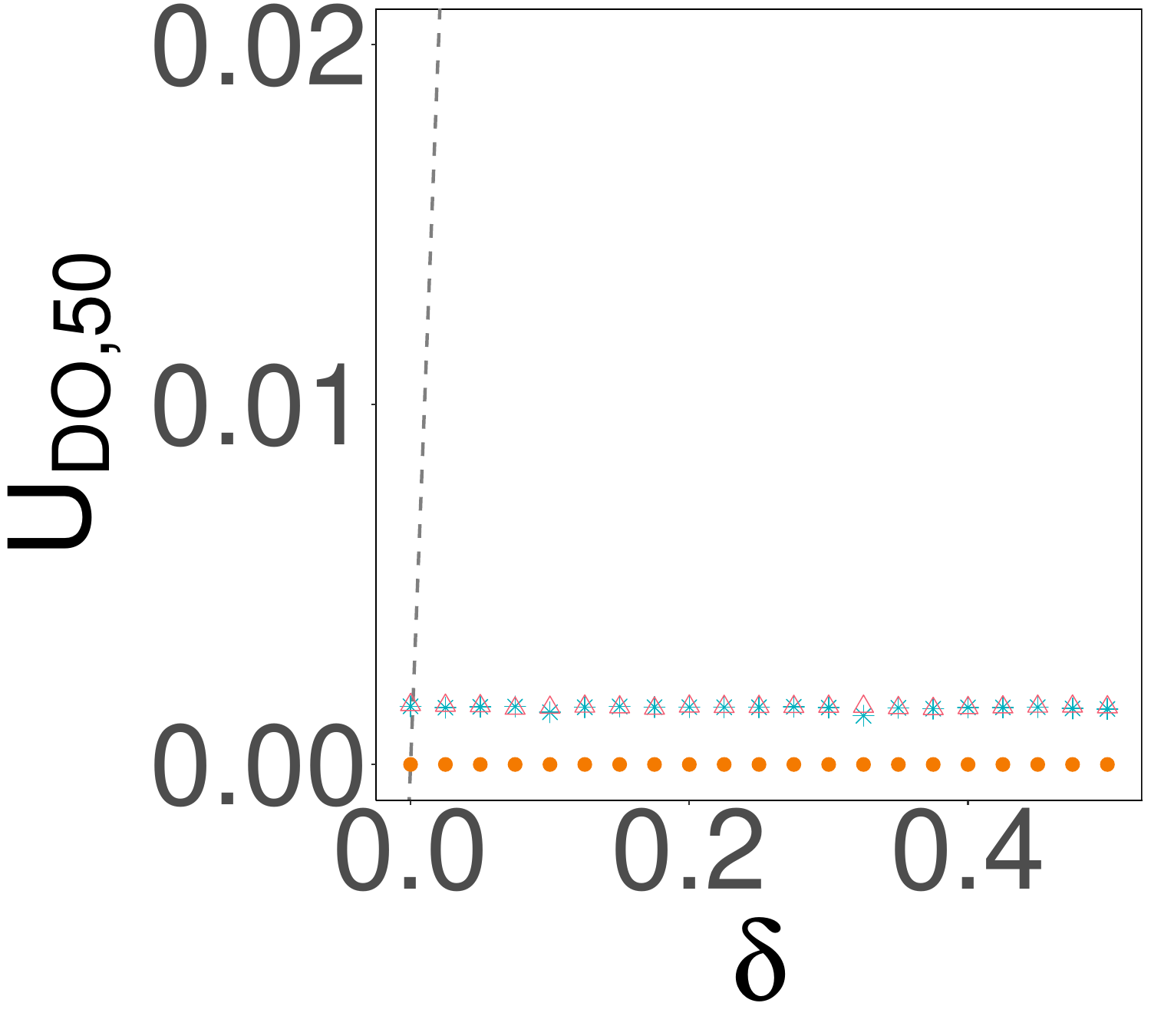} &
			\hspace{\thisgap}\includegraphics[width=\thiswidth]{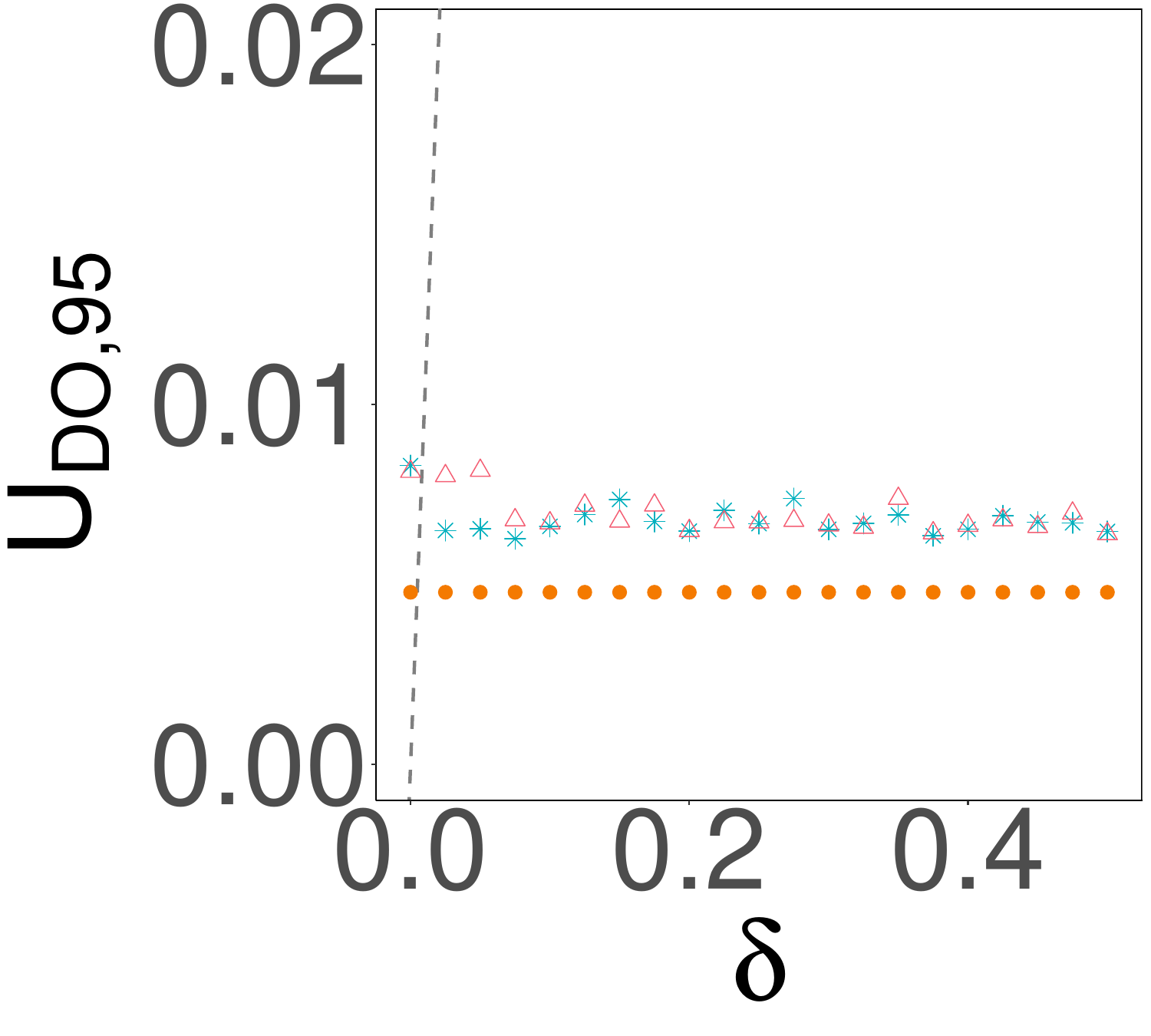} \\
			\hspace{\thisgap}\includegraphics[width=\thiswidth]{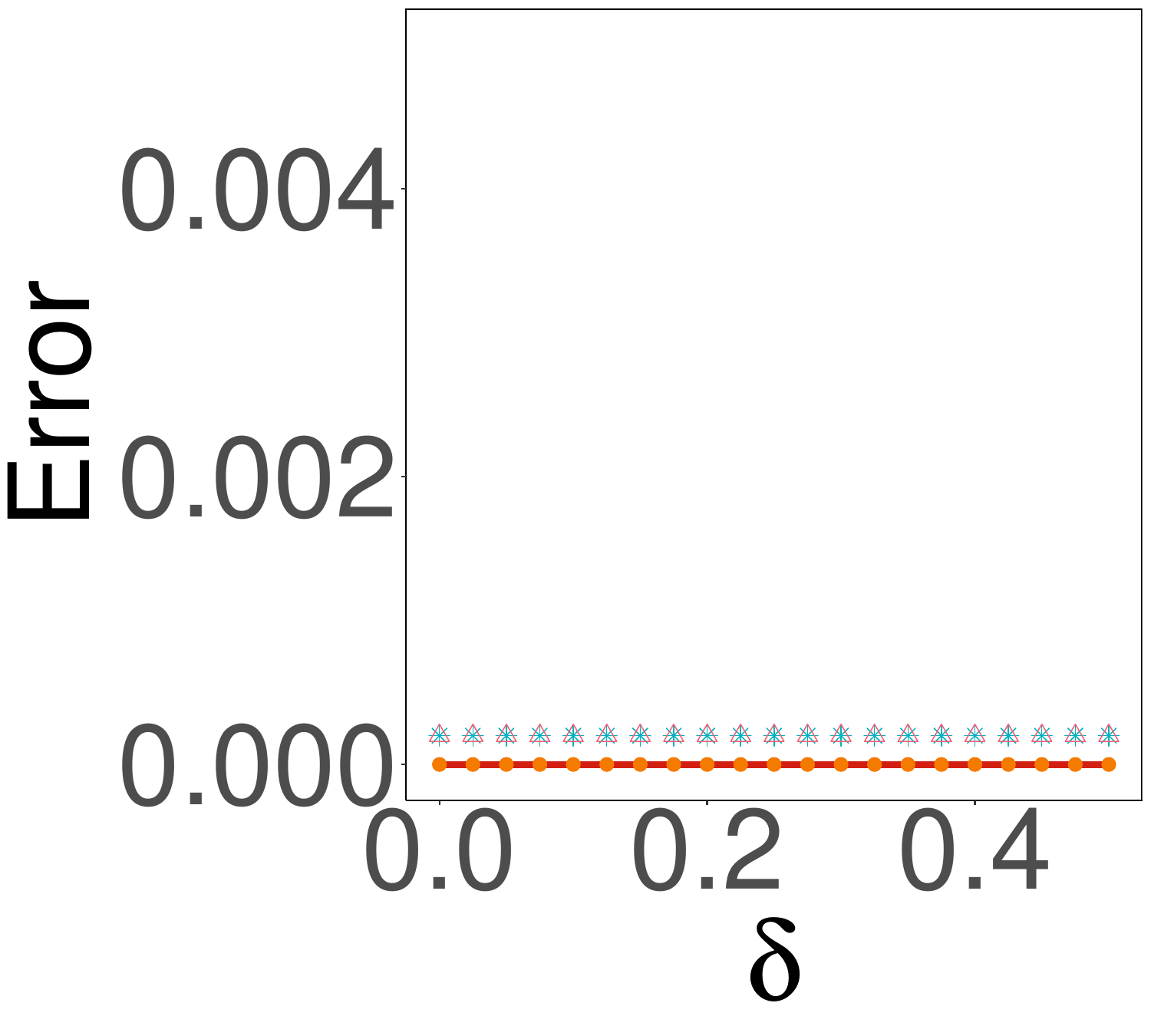} &
			\hspace{\thisgap}\includegraphics[width=\thiswidth]{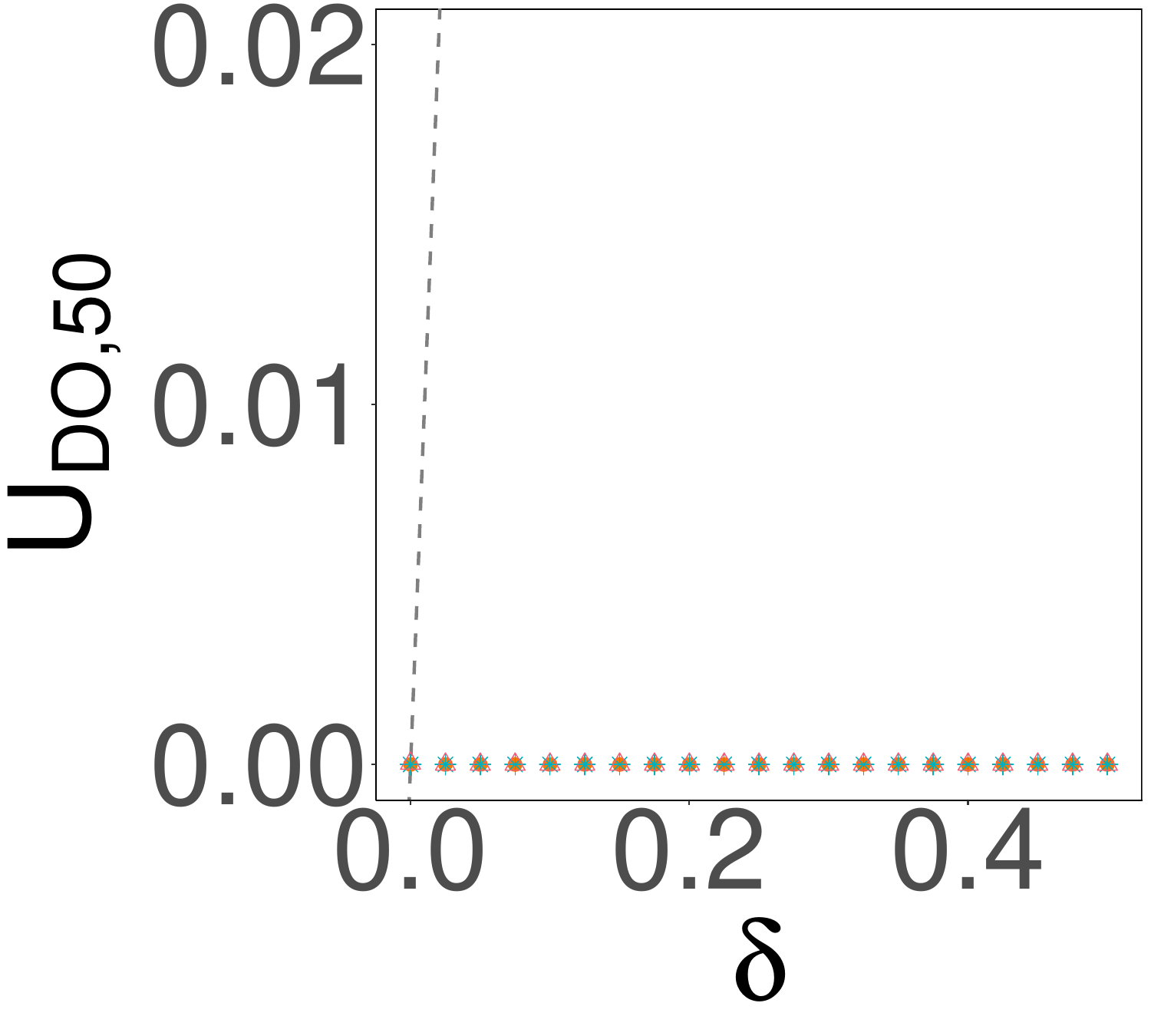} &
			\hspace{\thisgap}\includegraphics[width=\thiswidth]{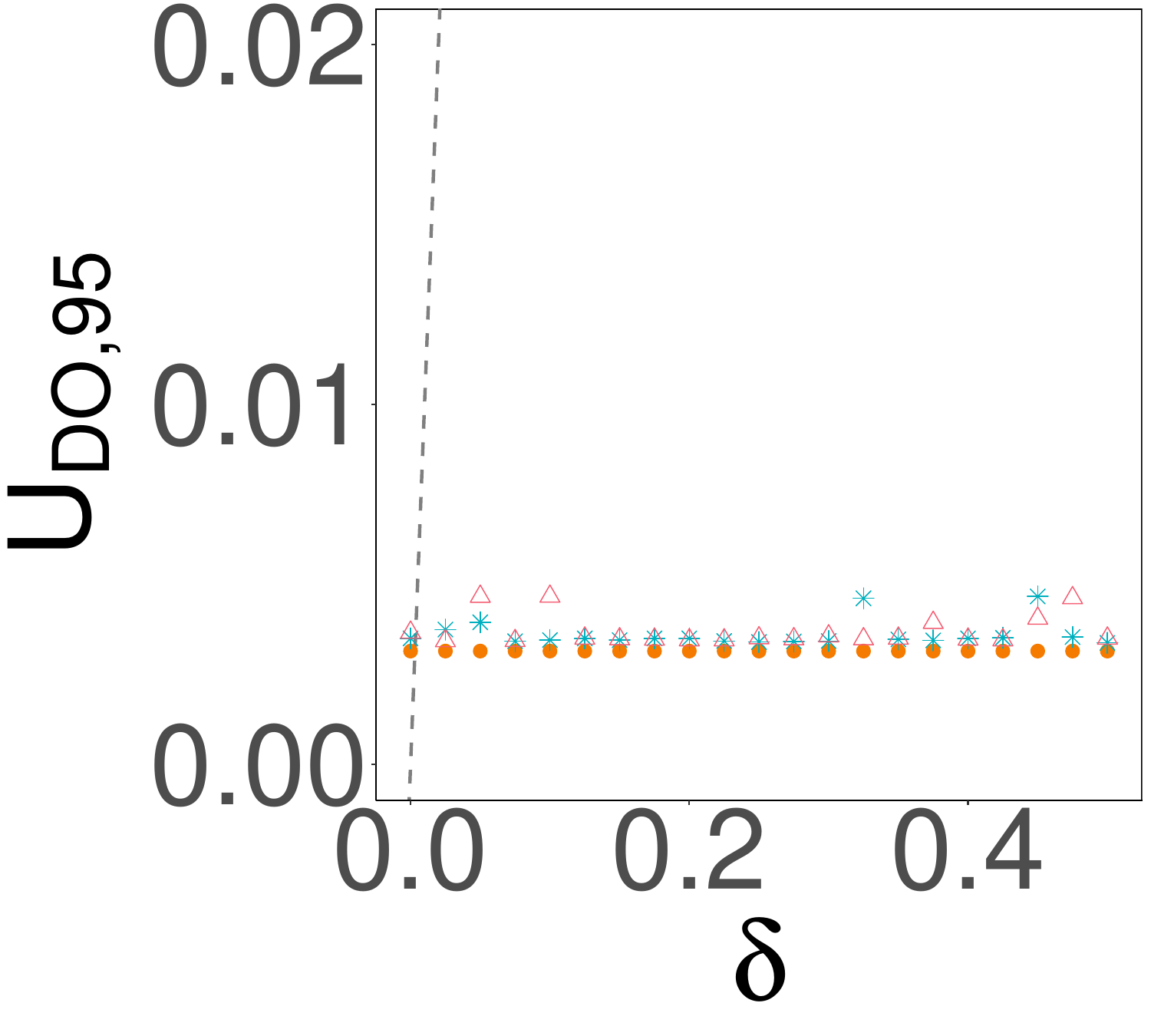}
		\end{tabular}
		\caption{Disparity DO results under (I). Top: $n=1000$; middle: $n=2000$; bottom: $n=5000$.}
		\label{fig:DO_gauss_beta_0.5}
	\end{center}
\end{figure*}

\begin{figure*}[!htbp]
	\begin{center}
		\newcommand{\thiswidth}{0.2\linewidth}
		\newcommand{\thisgap}{0mm}
		\begin{tabular}{ccc}
			\hspace{\thisgap}\includegraphics[width=\thiswidth]{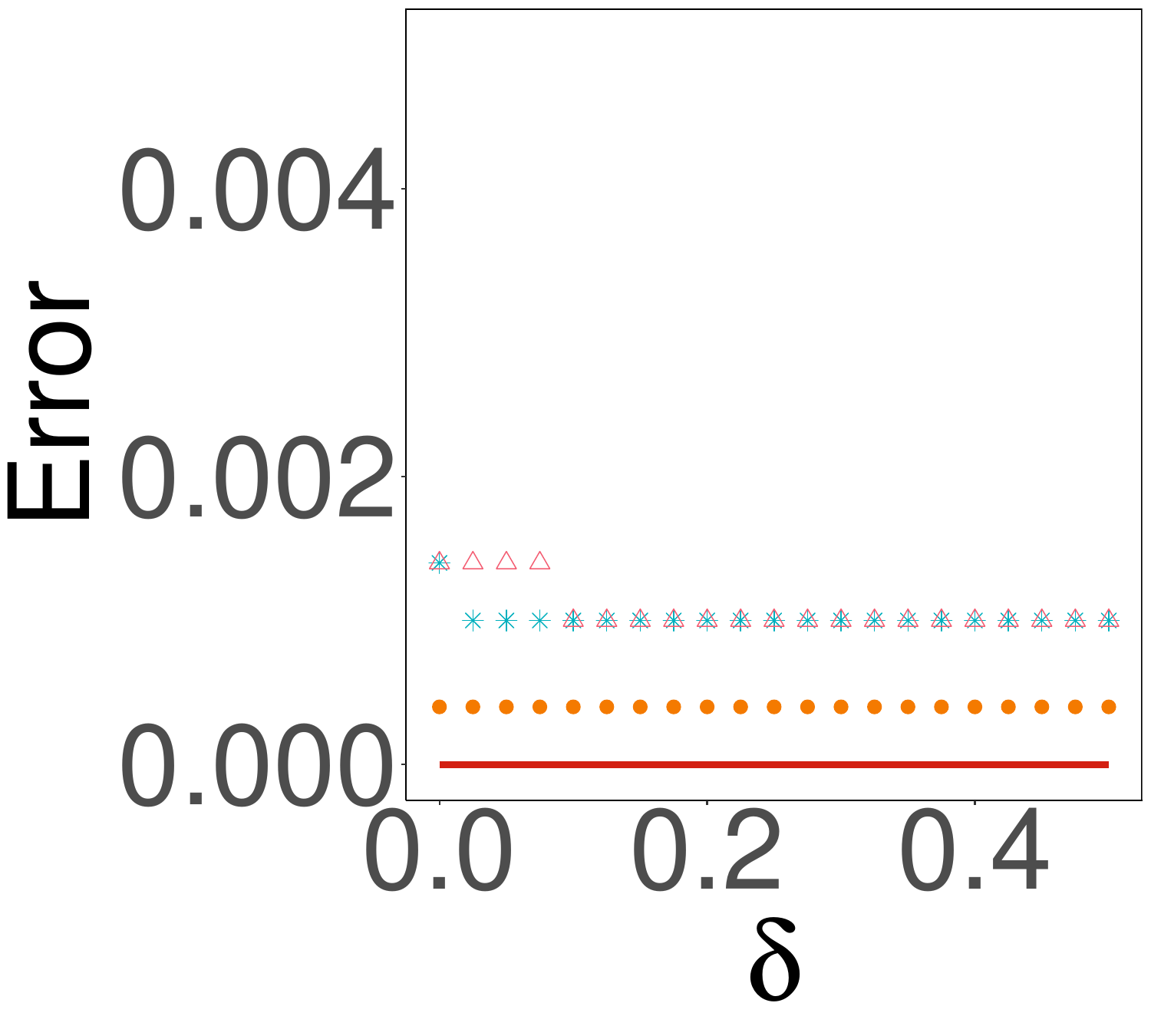} &
			\hspace{\thisgap}\includegraphics[width=\thiswidth]{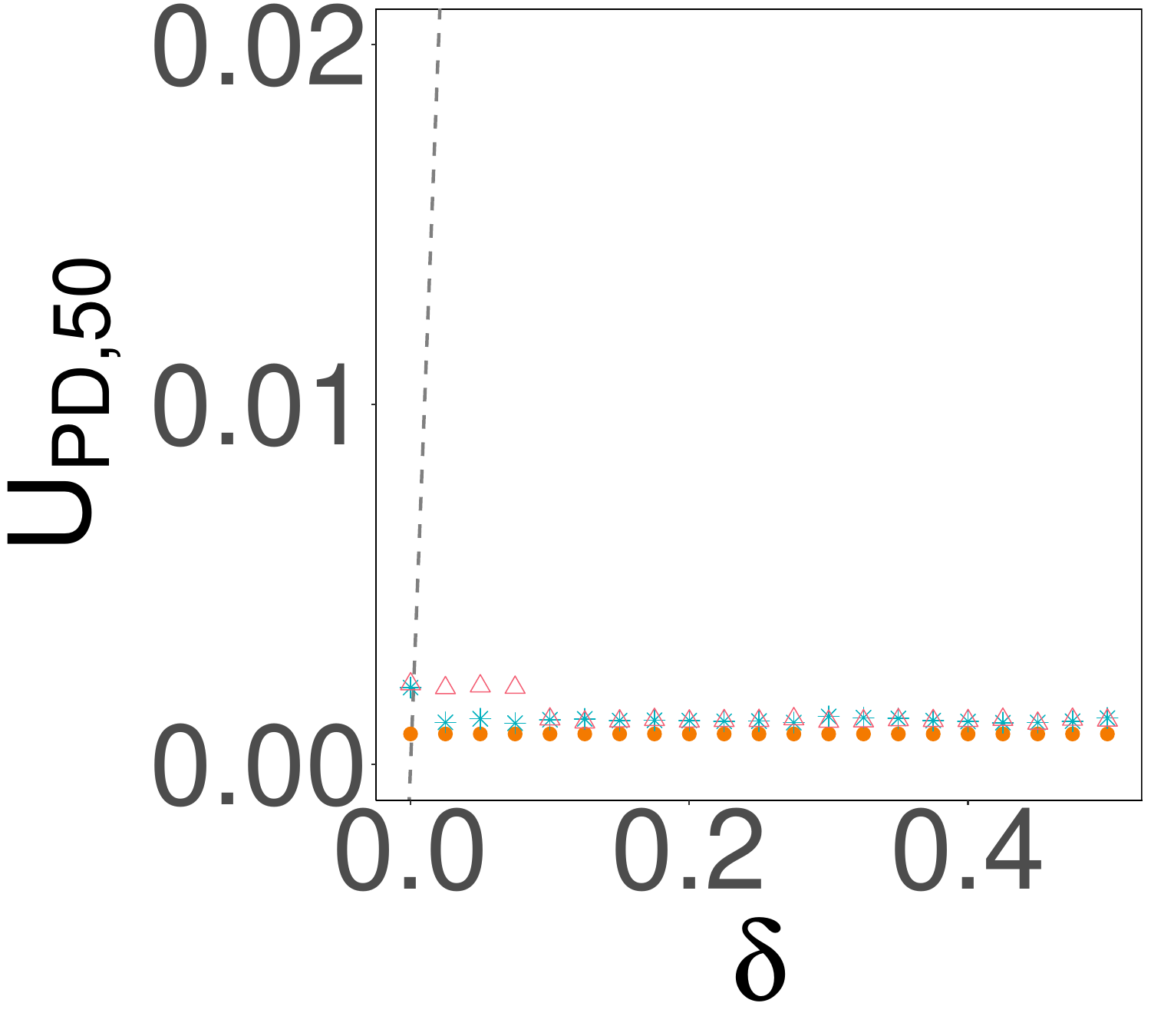} &
			\hspace{\thisgap}\includegraphics[width=\thiswidth]{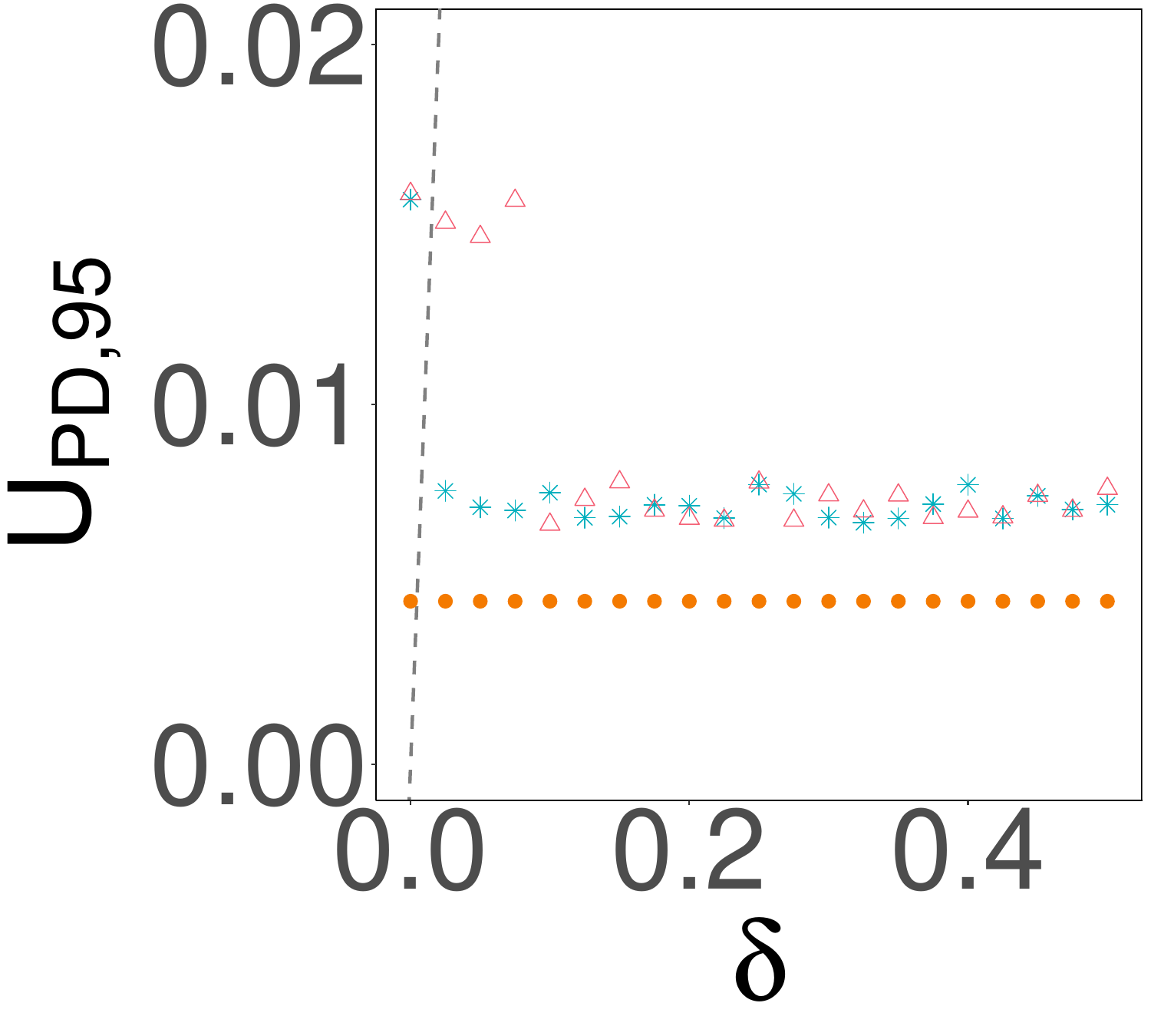} \\
			\hspace{\thisgap}\includegraphics[width=\thiswidth]{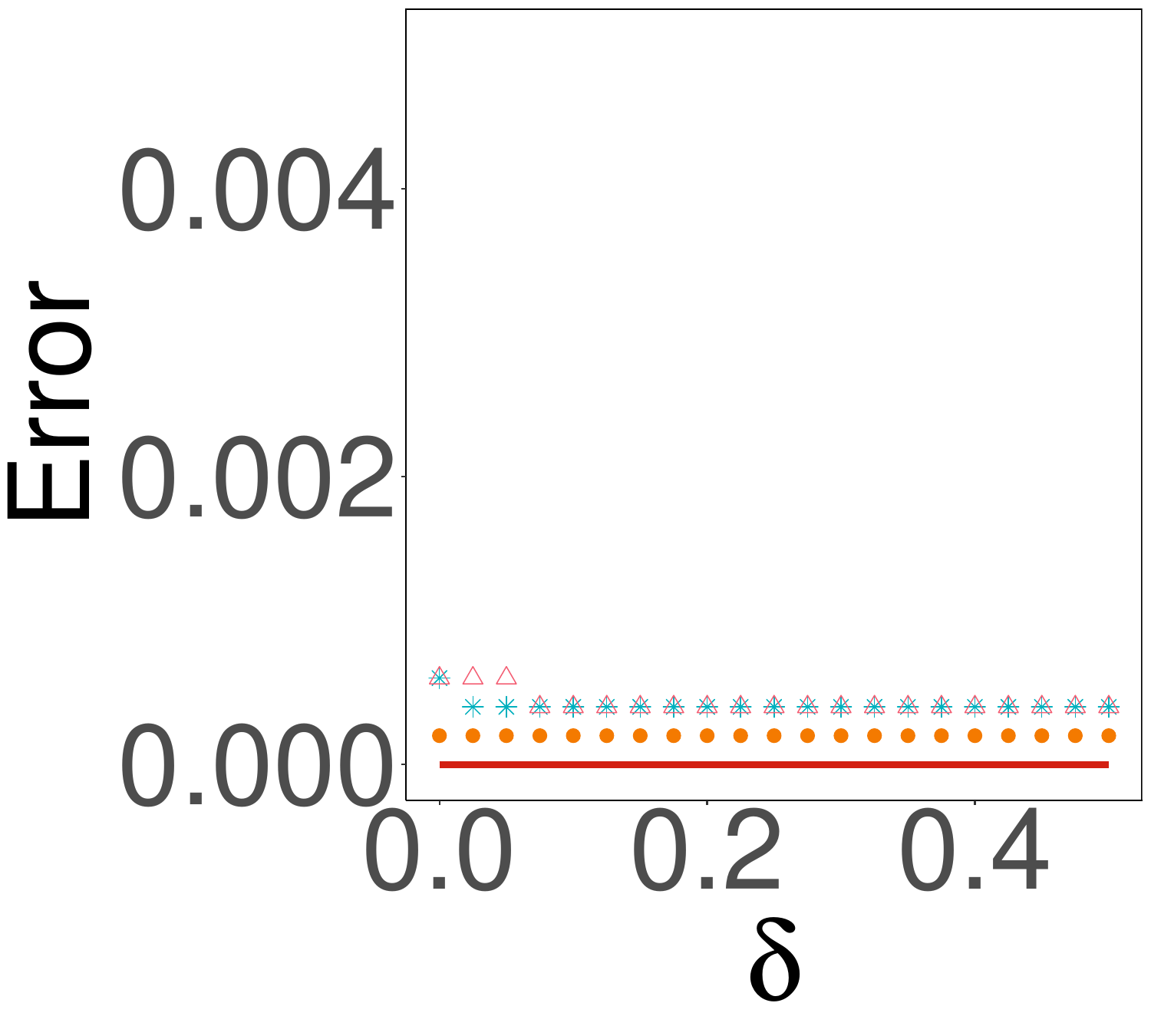} &
			\hspace{\thisgap}\includegraphics[width=\thiswidth]{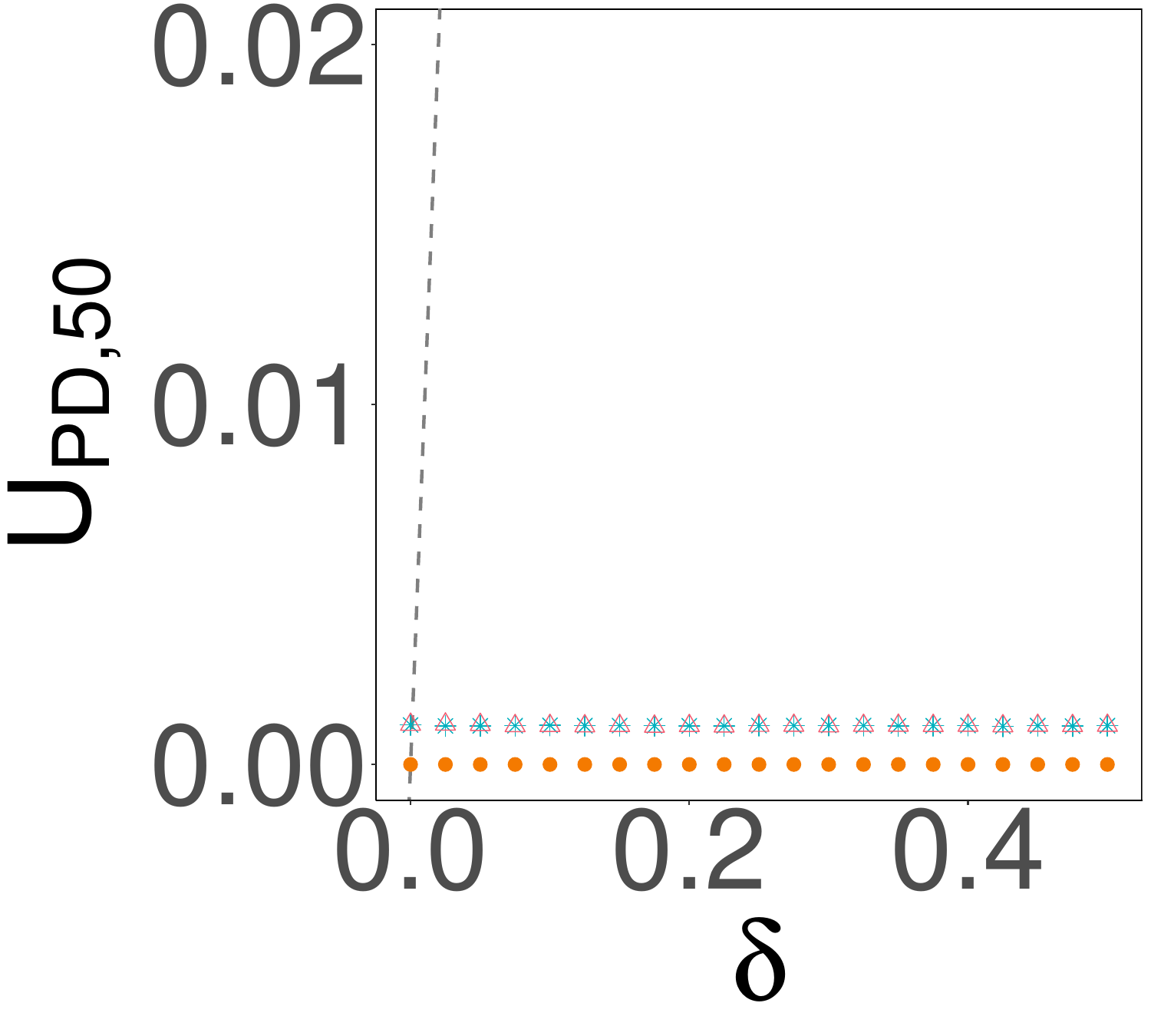} &
			\hspace{\thisgap}\includegraphics[width=\thiswidth]{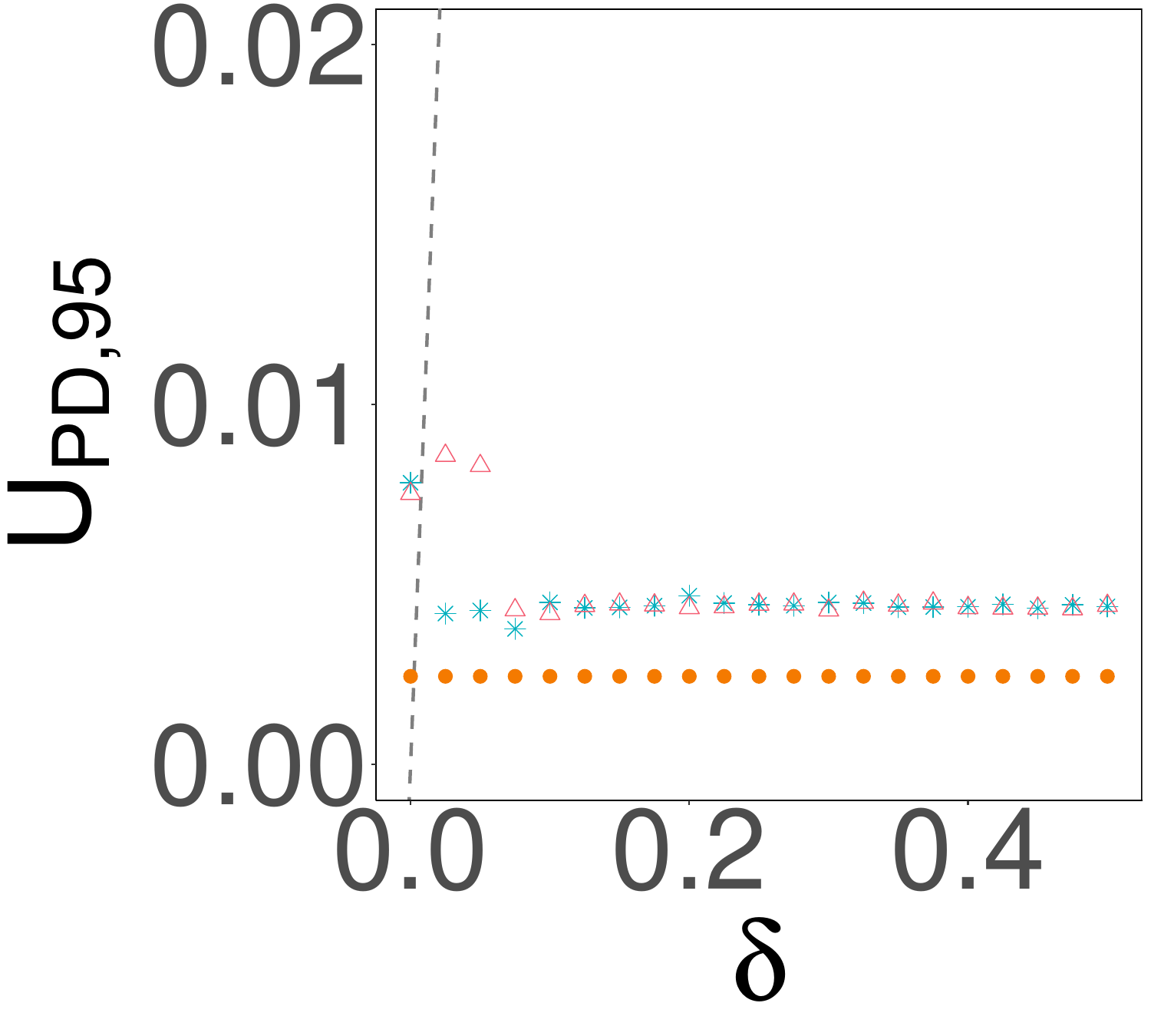} \\
			\hspace{\thisgap}\includegraphics[width=\thiswidth]{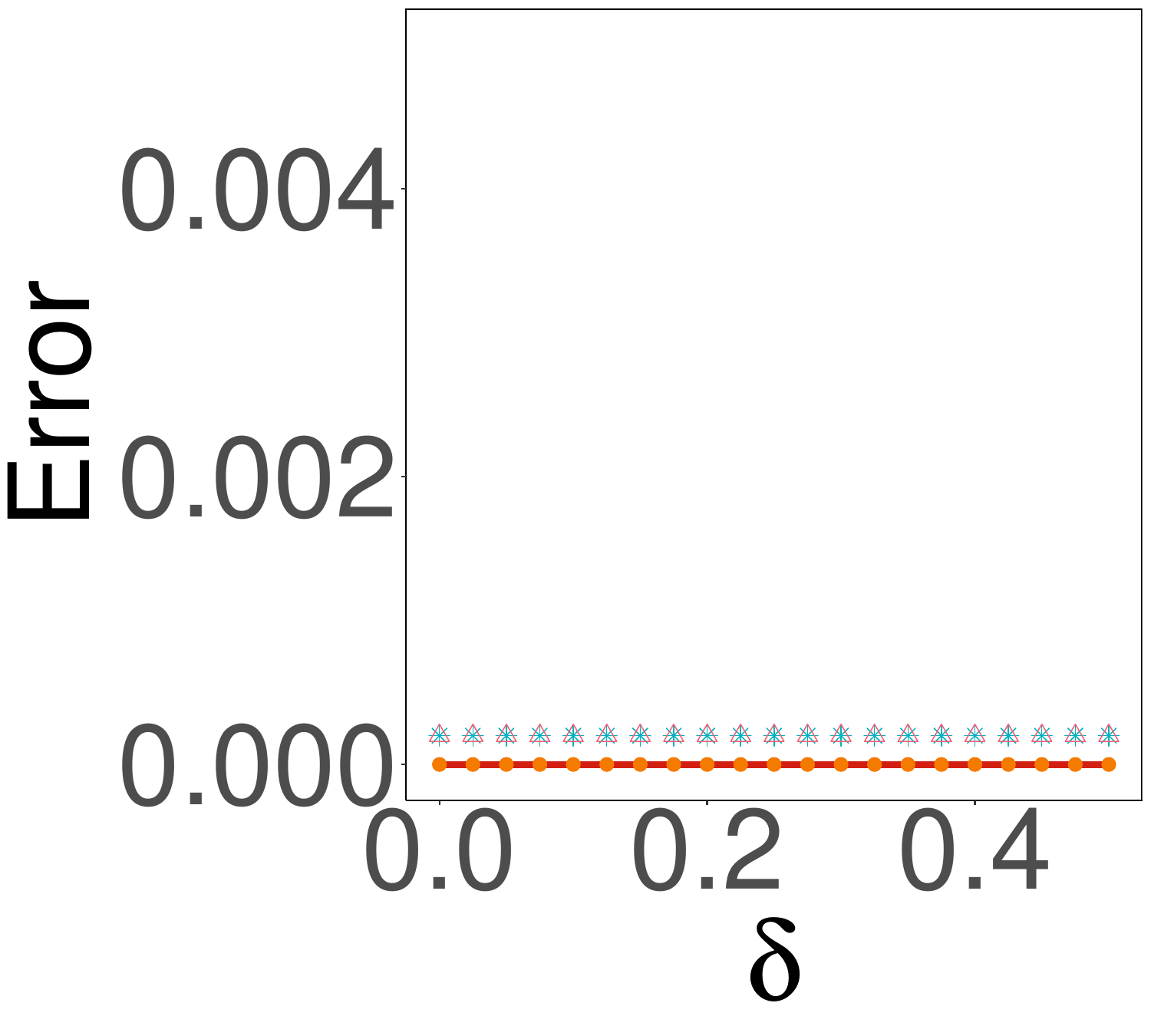} &
			\hspace{\thisgap}\includegraphics[width=\thiswidth]{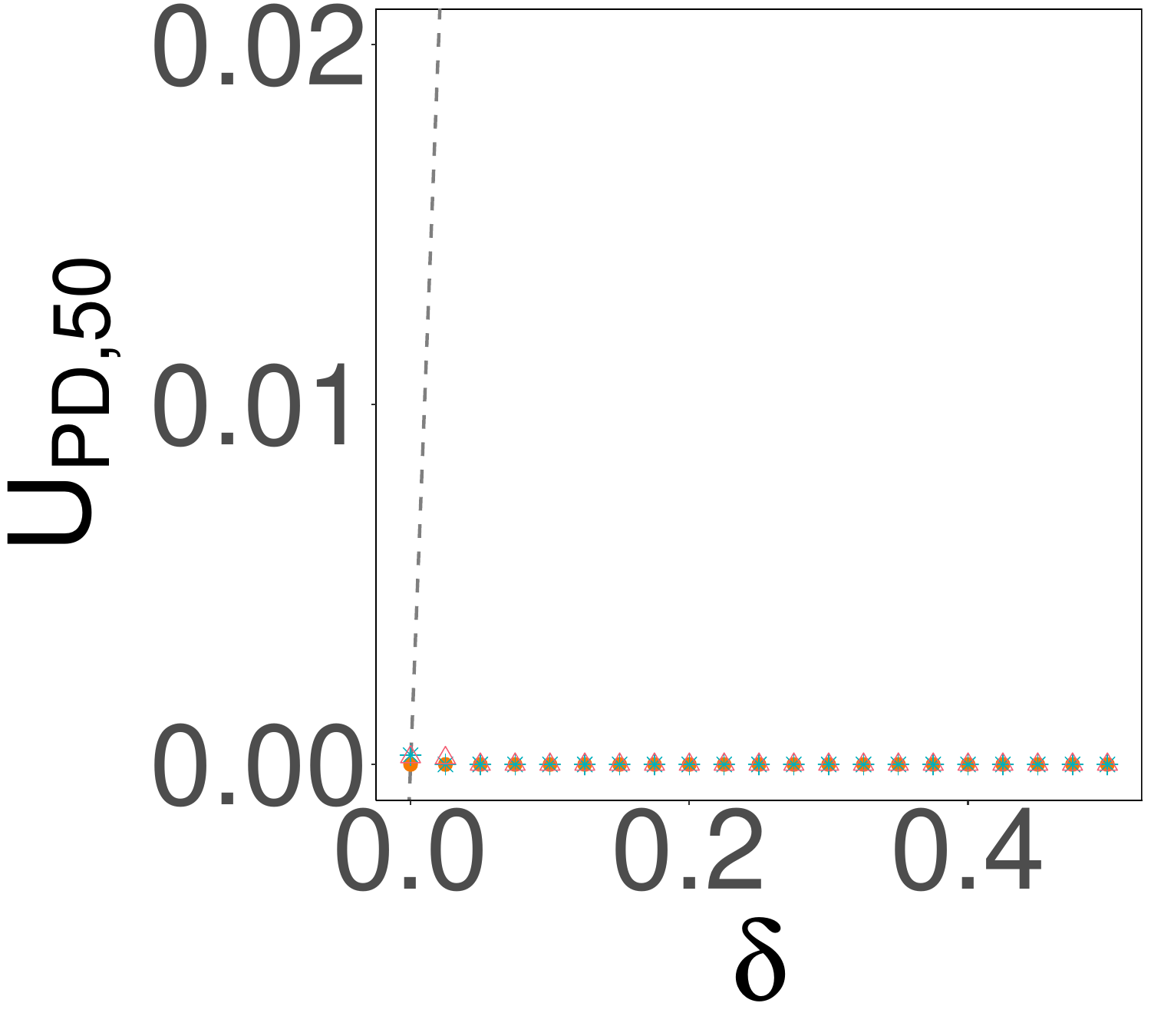} &
			\hspace{\thisgap}\includegraphics[width=\thiswidth]{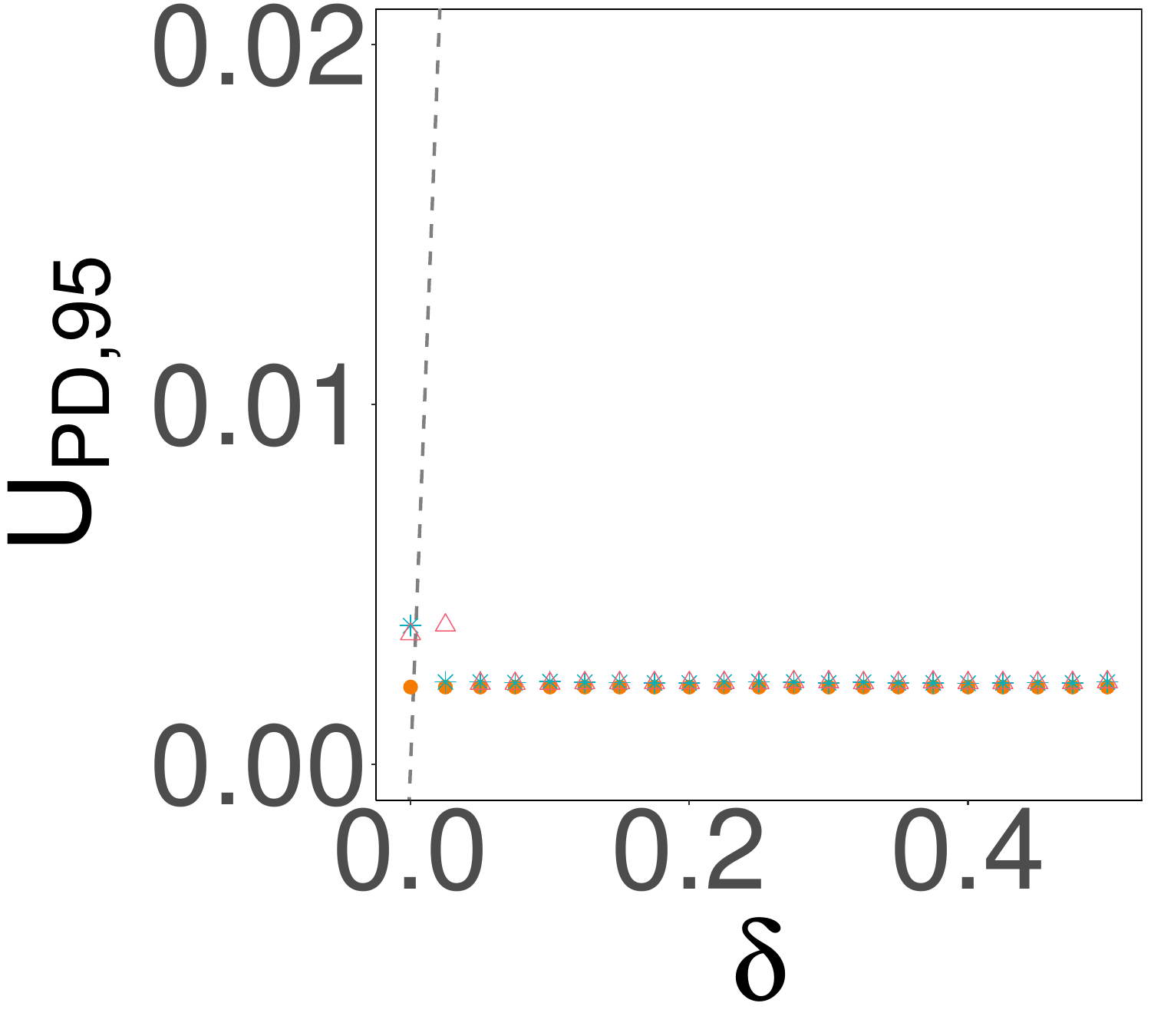}
		\end{tabular}
		\caption{Disparity PD results under (I). Top: $n=1000$; middle: $n=2000$; bottom: $n=5000$.}
		\label{fig:PD_gauss_beta_0.5}
	\end{center}
\end{figure*}

\begin{figure*}[!htbp]
	\begin{center}
		\newcommand{\thiswidth}{0.2\linewidth}
		\newcommand{\thisgap}{0mm}
		\begin{tabular}{ccc}
			\hspace{\thisgap}\includegraphics[width=\thiswidth]{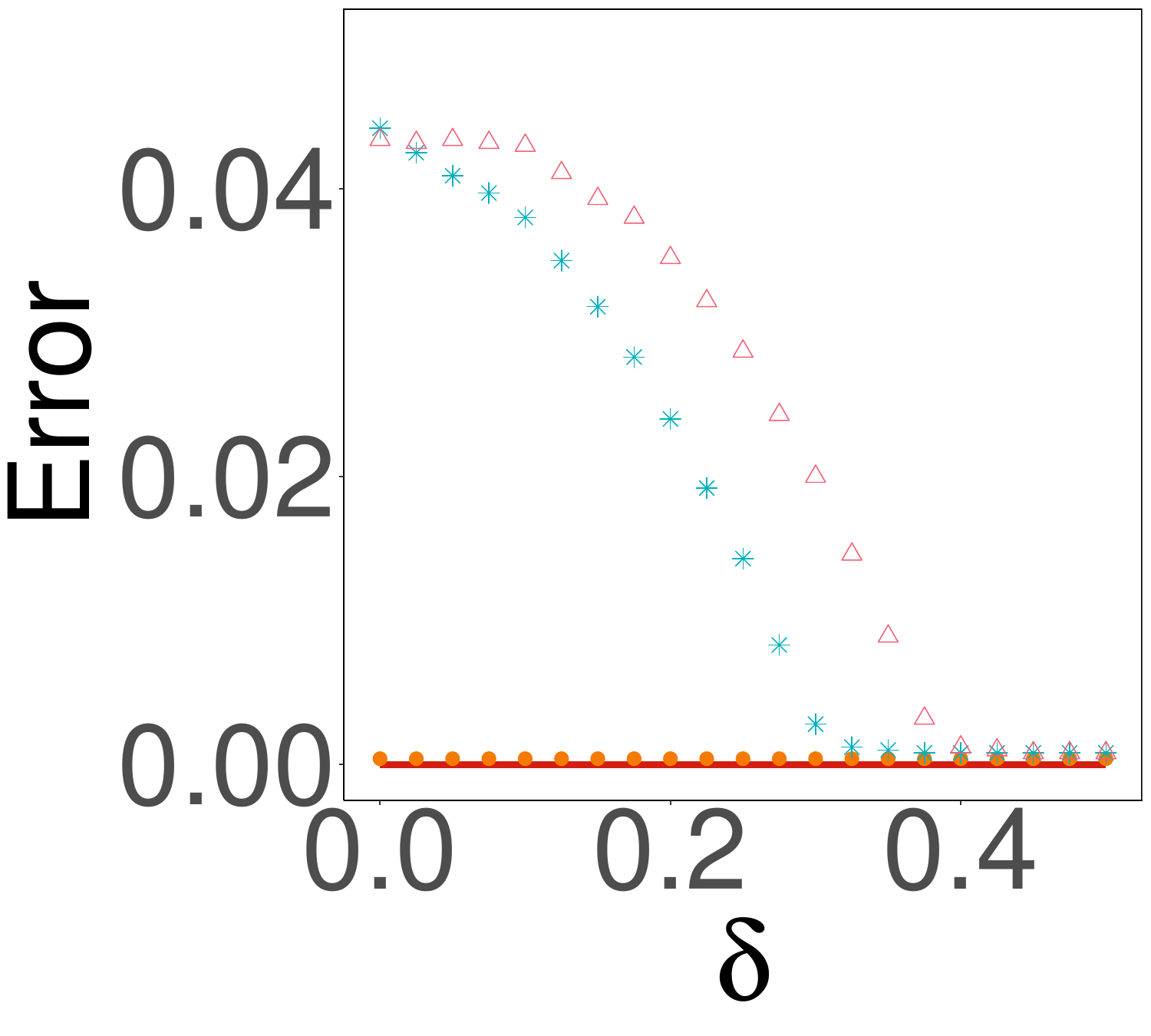} &
			\hspace{\thisgap}\includegraphics[width=\thiswidth]{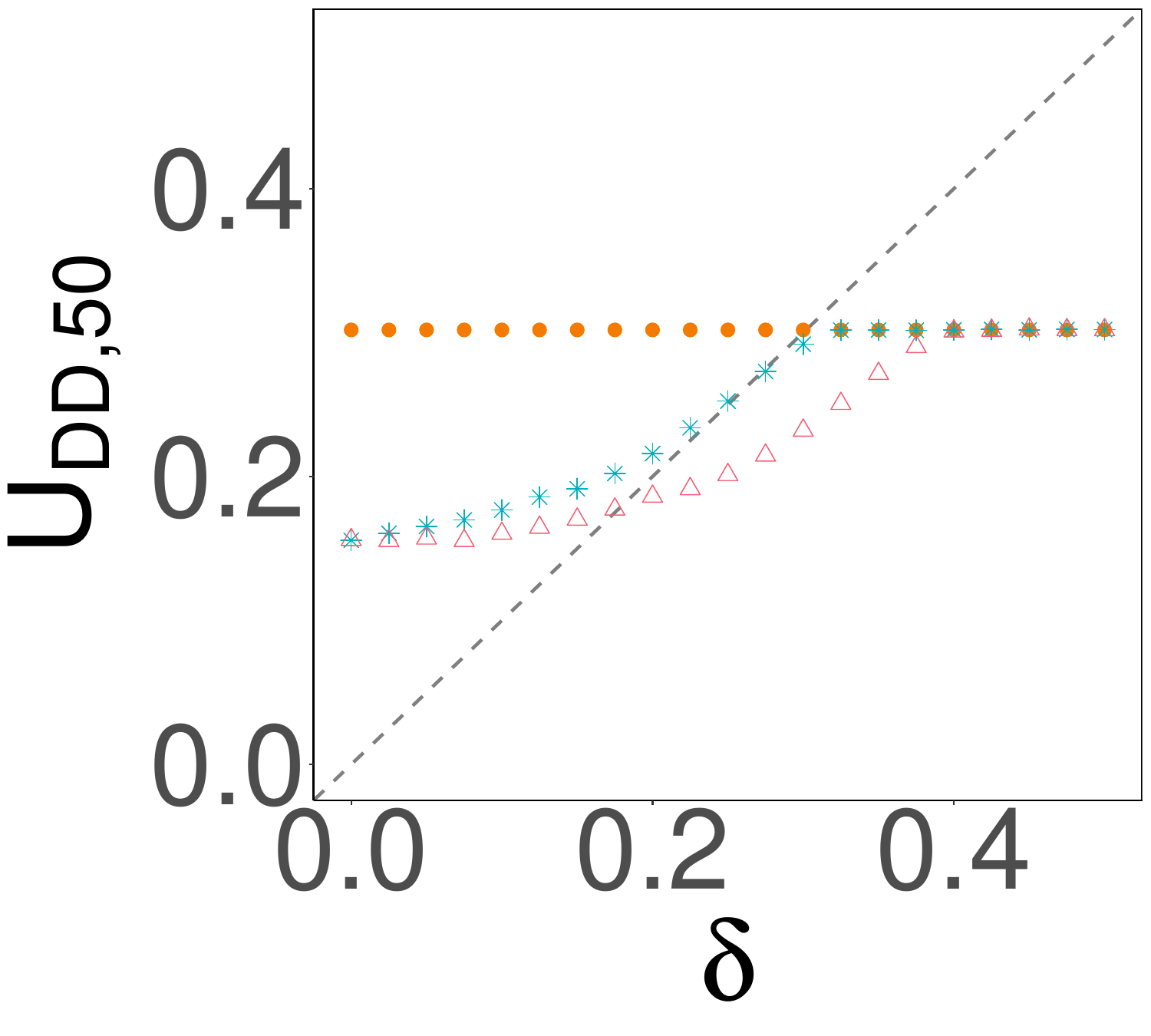} &
			\hspace{\thisgap}\includegraphics[width=\thiswidth]{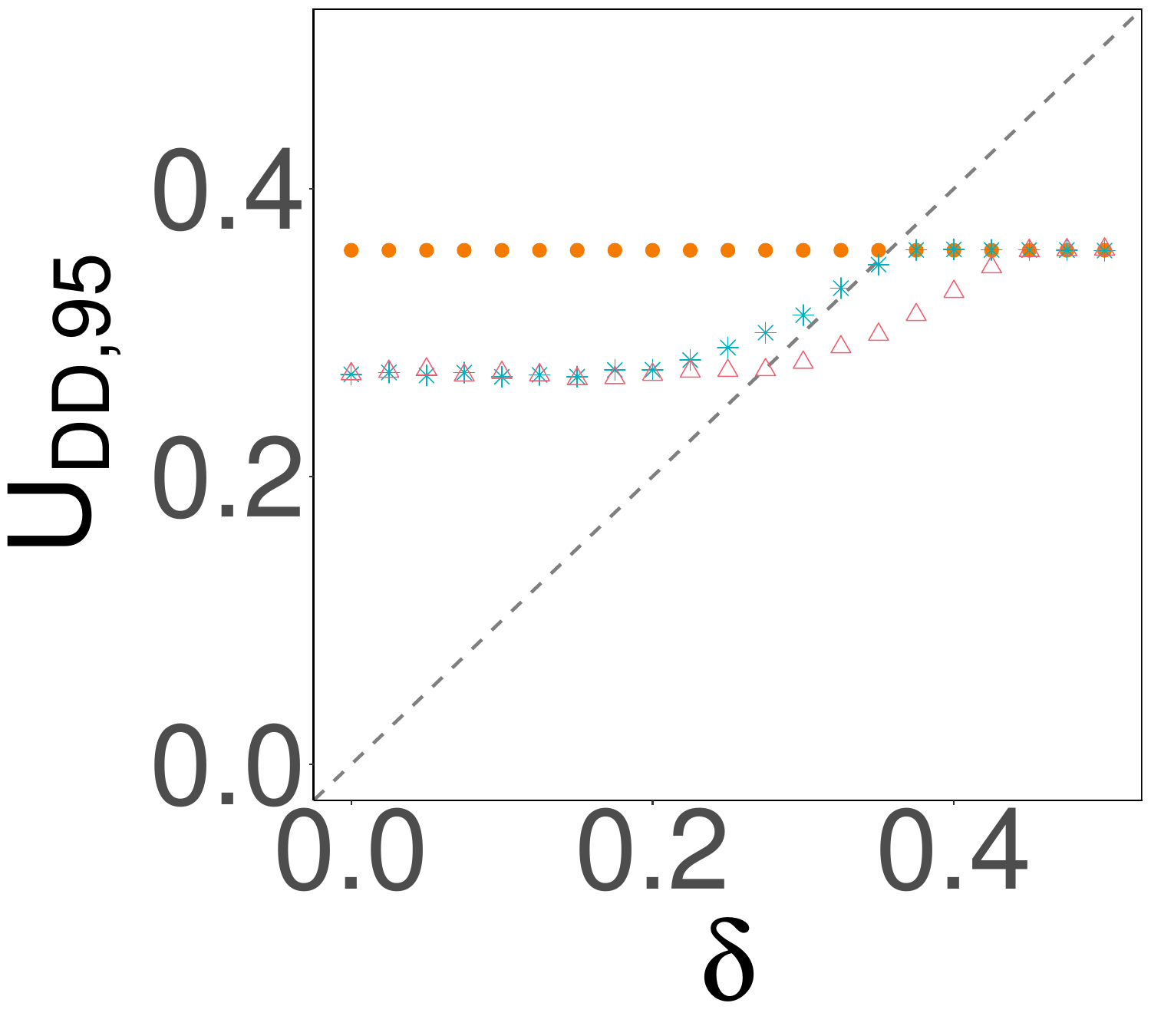} \\
			\hspace{\thisgap}\includegraphics[width=\thiswidth]{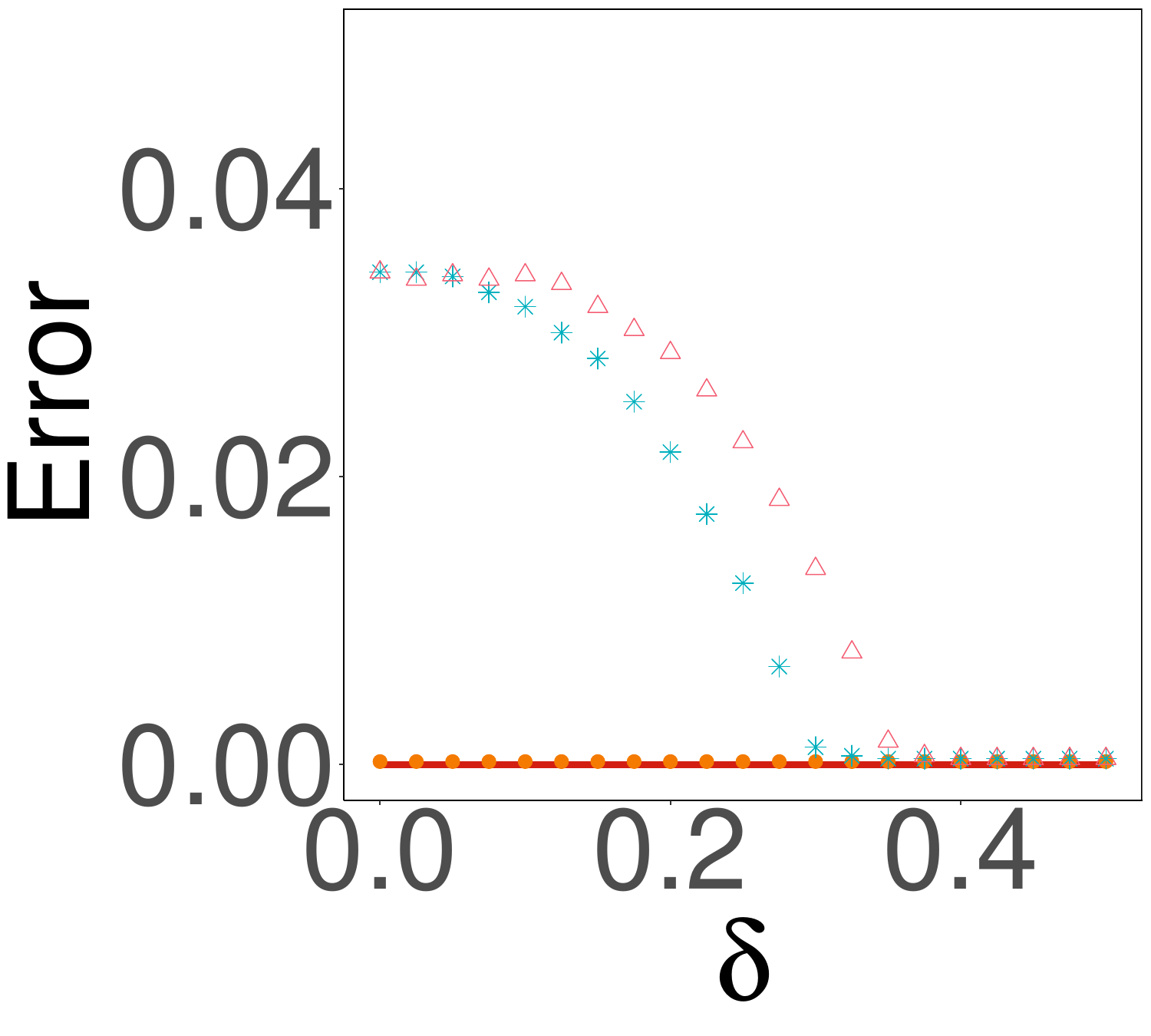} &
			\hspace{\thisgap}\includegraphics[width=\thiswidth]{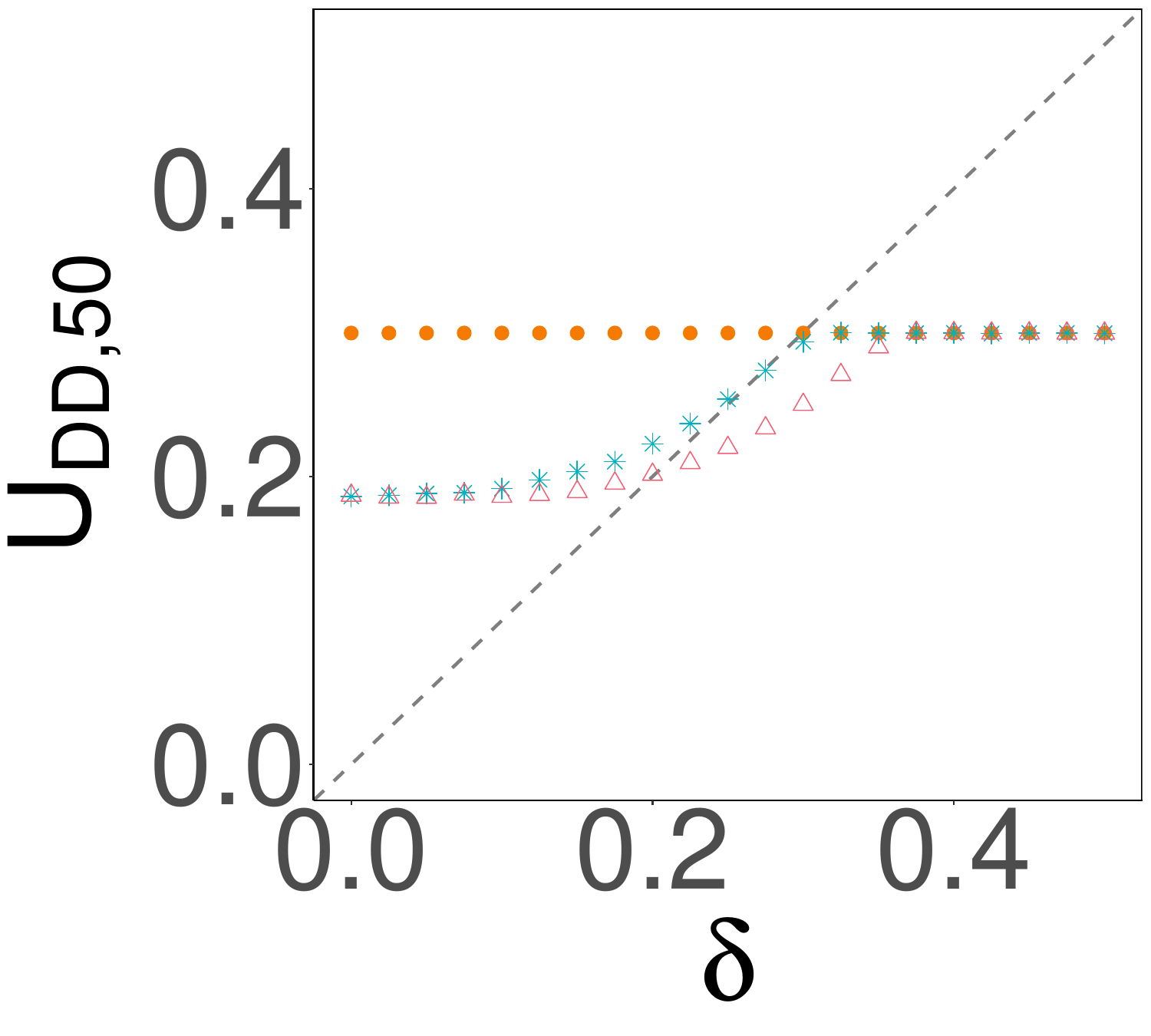} &
			\hspace{\thisgap}\includegraphics[width=\thiswidth]{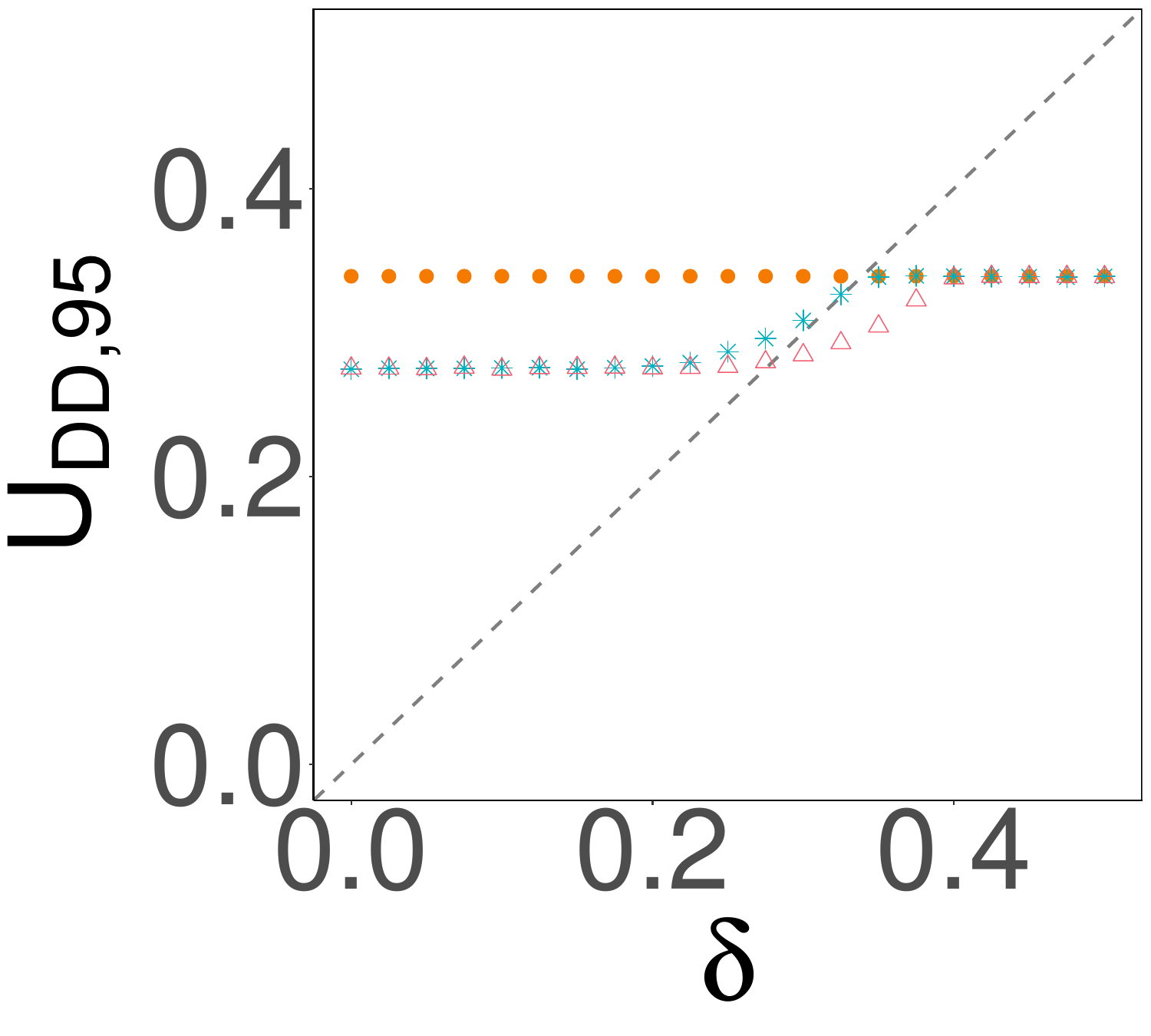} \\
			\hspace{\thisgap}\includegraphics[width=\thiswidth]{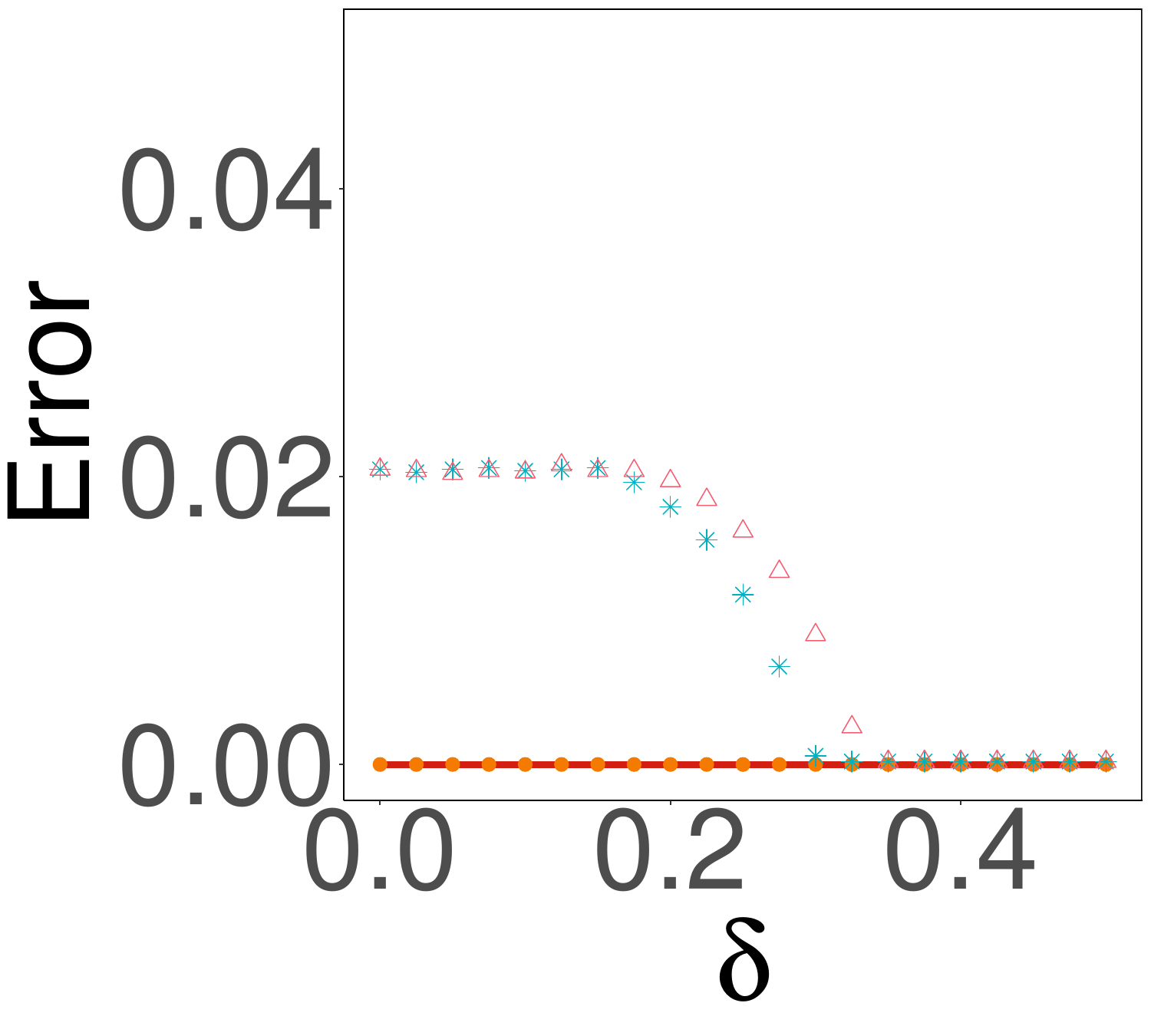} &
			\hspace{\thisgap}\includegraphics[width=\thiswidth]{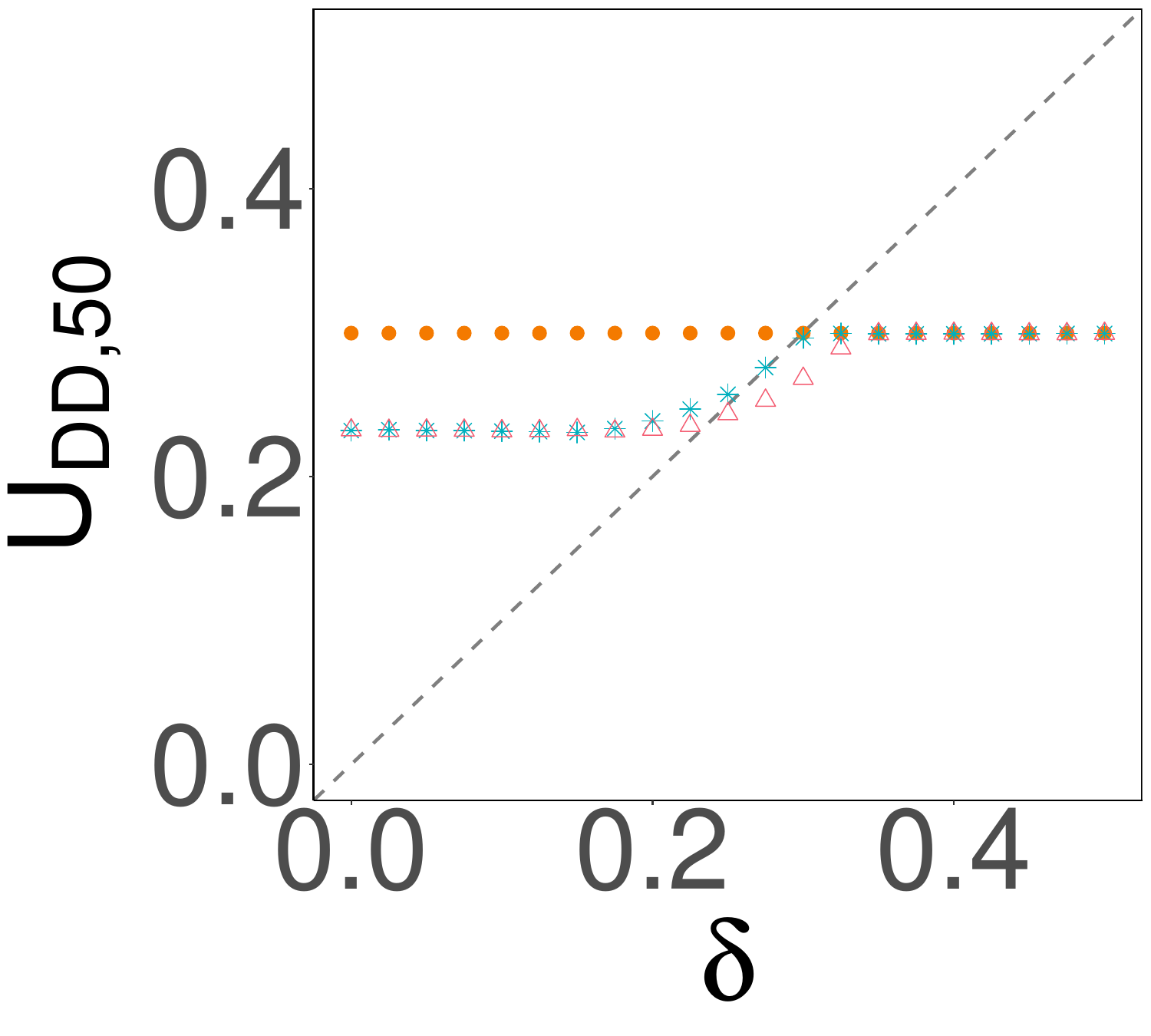} &
			\hspace{\thisgap}\includegraphics[width=\thiswidth]{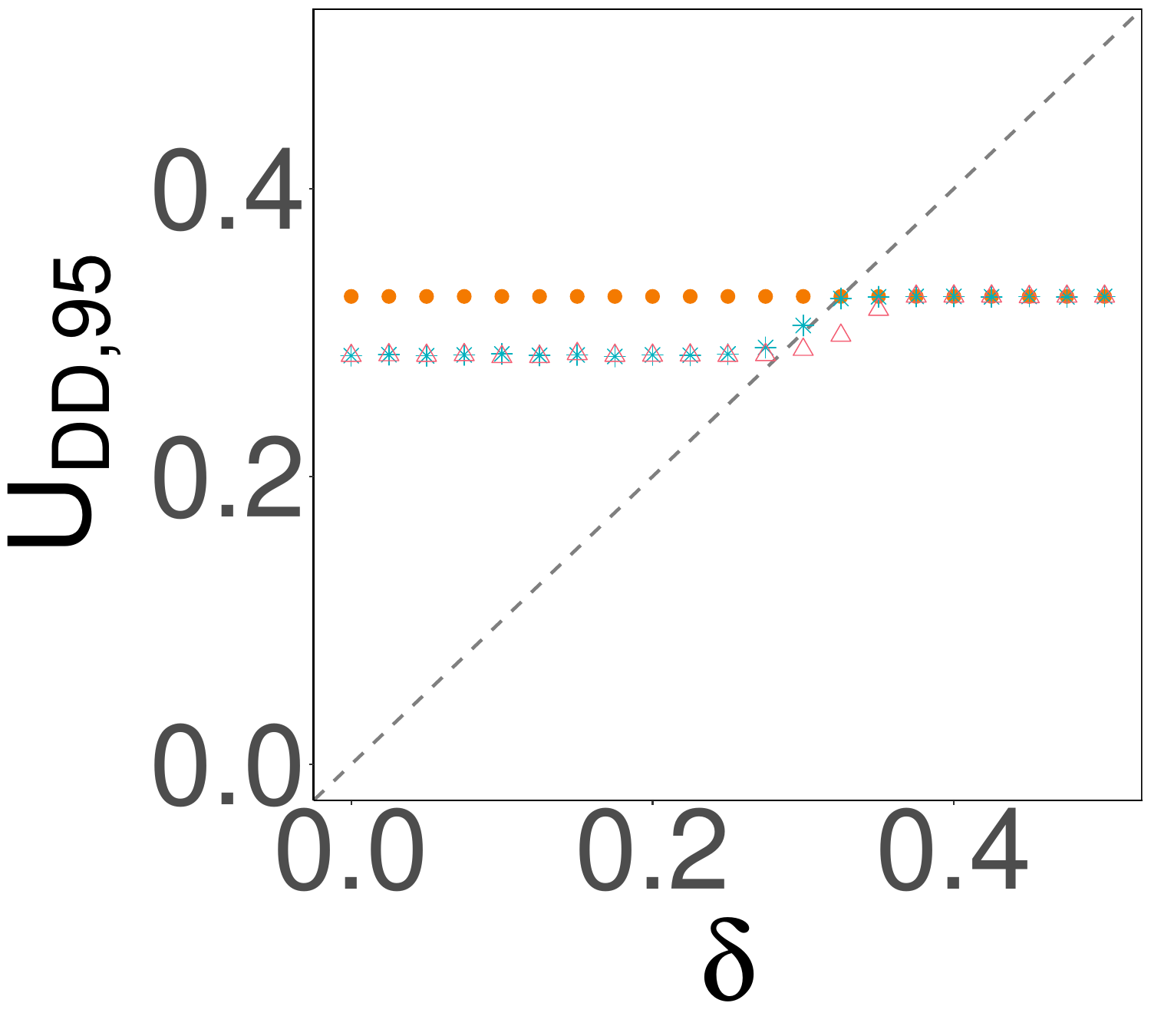}
		\end{tabular}
		\caption{Disparity DD results under (I). Top: $n=1000$; middle: $n=2000$; bottom: $n=5000$.}
		\label{fig:DD_gauss_beta_0.5}
	\end{center}
\end{figure*}

\begin{figure*}[!htbp]
	\begin{center}
		\newcommand{\thiswidth}{0.2\linewidth}
		\newcommand{\thisgap}{0mm}
		\begin{tabular}{ccc}
			\hspace{\thisgap}\includegraphics[width=\thiswidth]{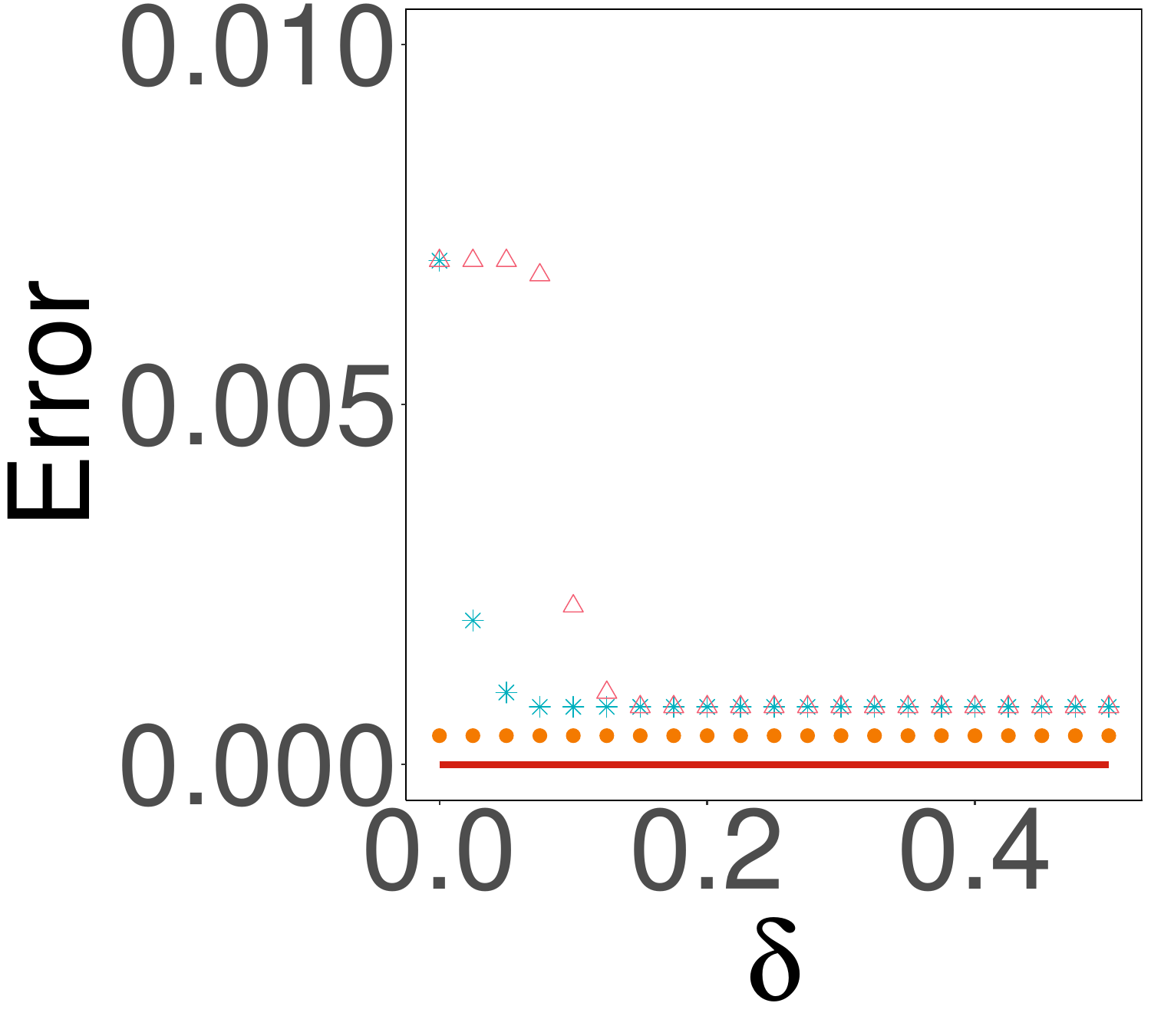} &
			\hspace{\thisgap}\includegraphics[width=\thiswidth]{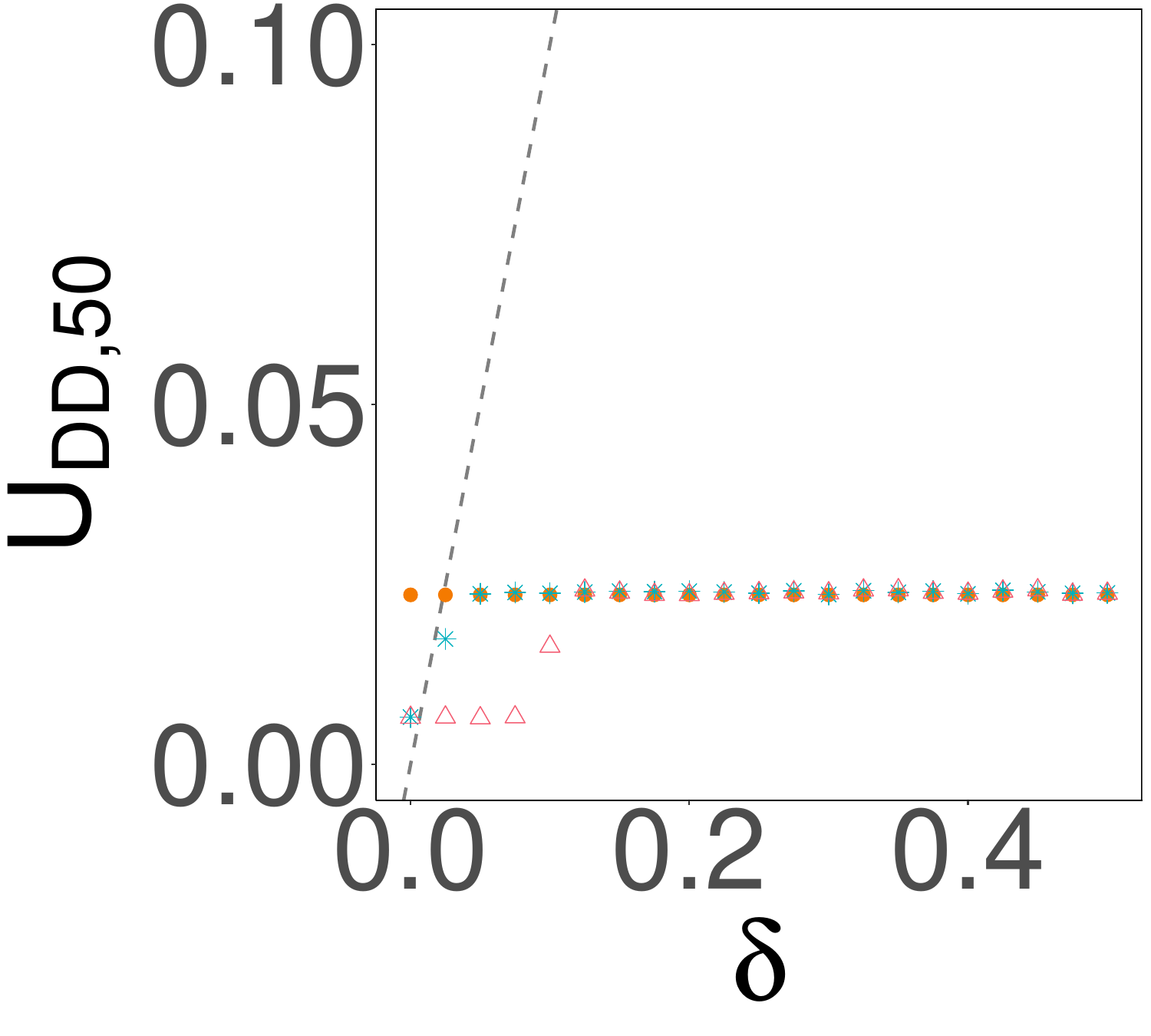} &
			\hspace{\thisgap}\includegraphics[width=\thiswidth]{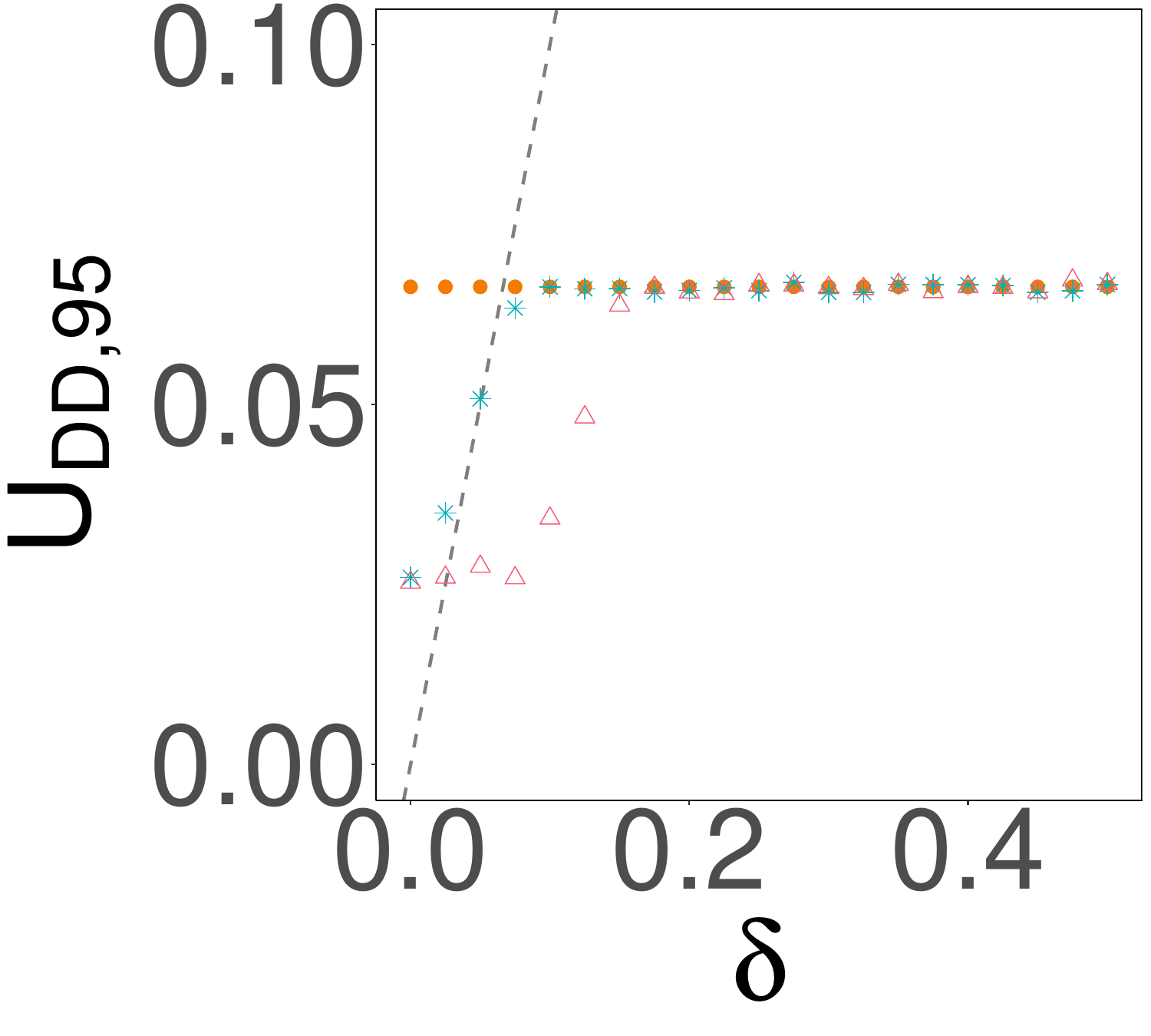} \\
			\hspace{\thisgap}\includegraphics[width=\thiswidth]{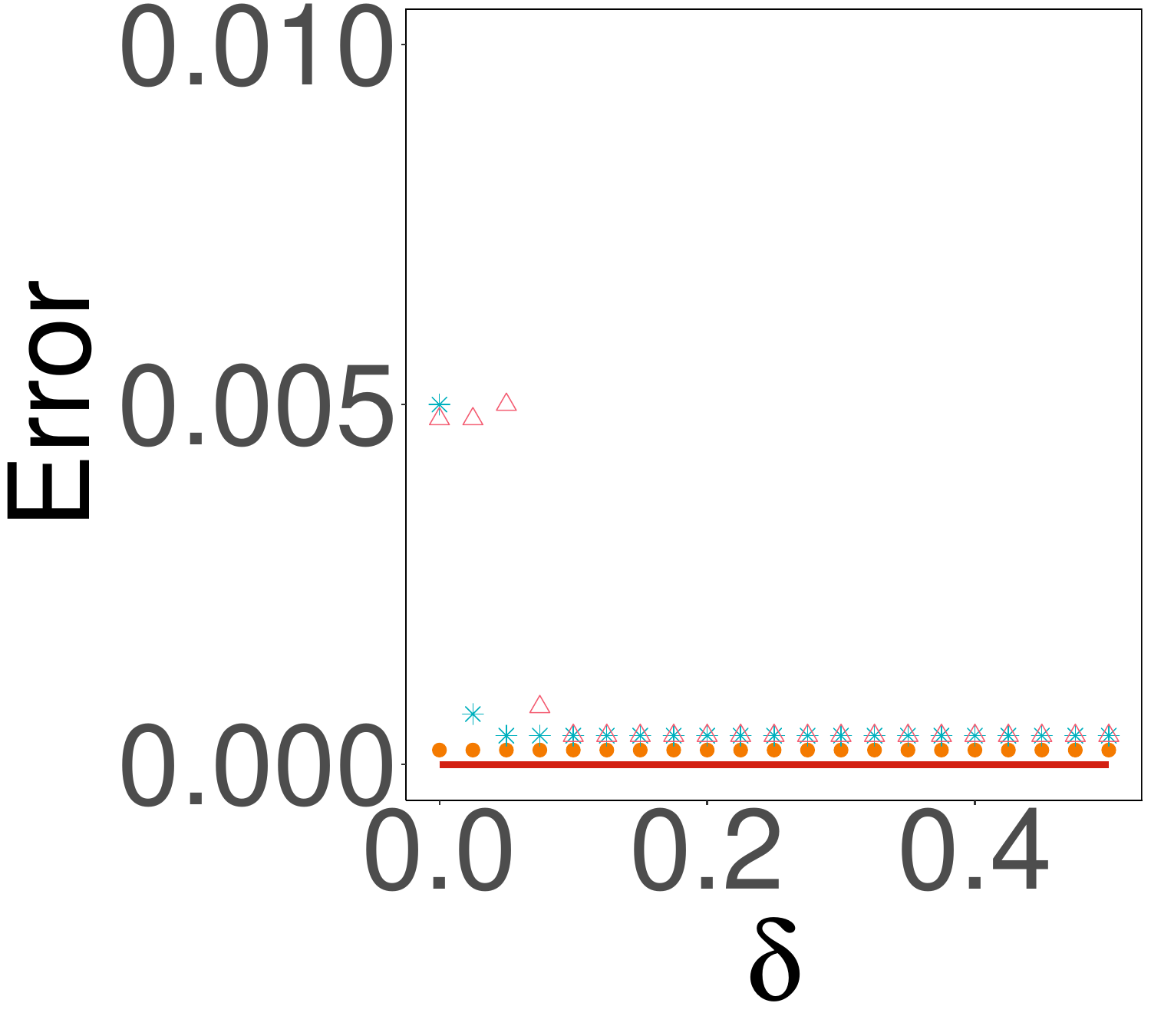} &
			\hspace{\thisgap}\includegraphics[width=\thiswidth]{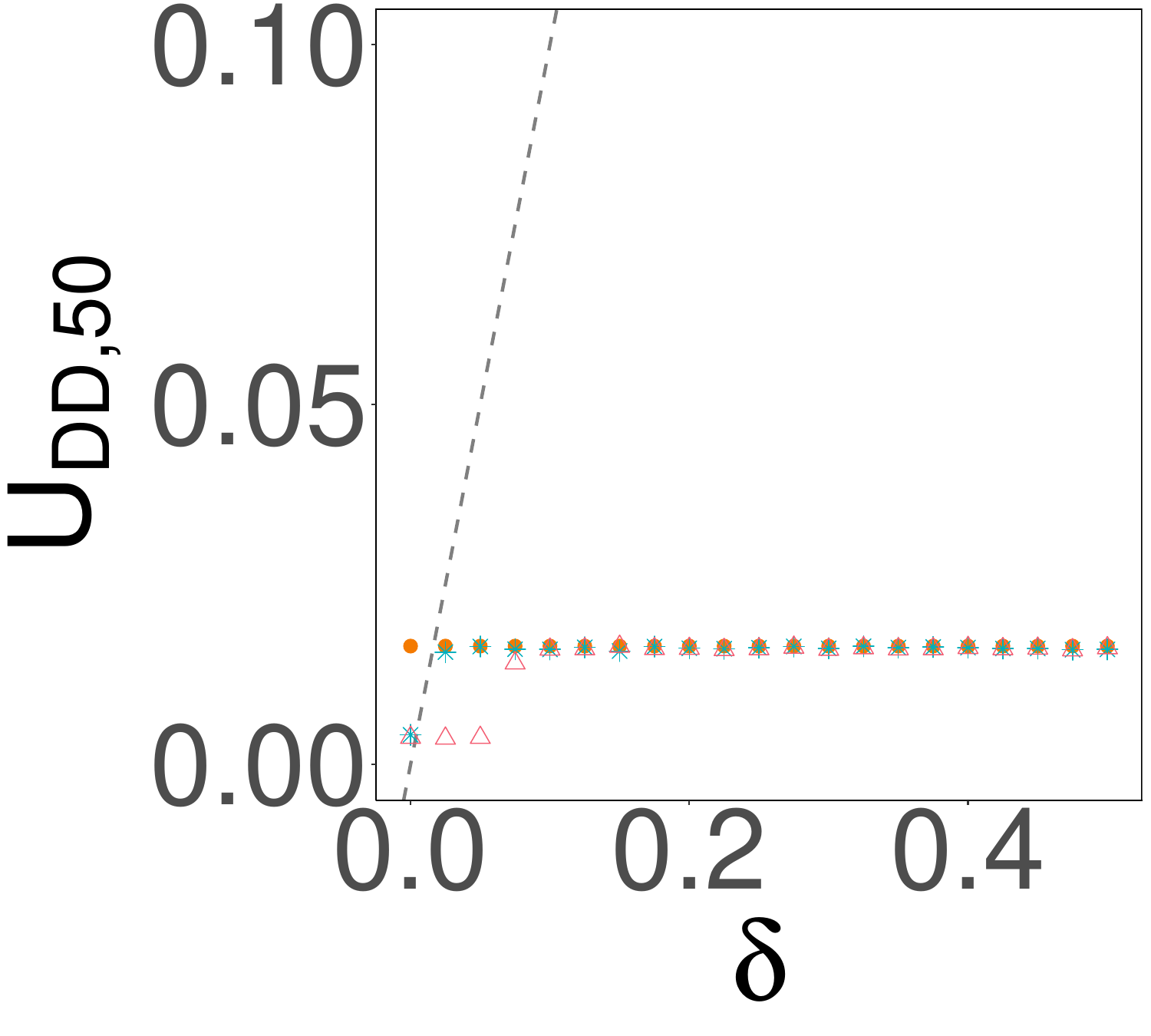} &
			\hspace{\thisgap}\includegraphics[width=\thiswidth]{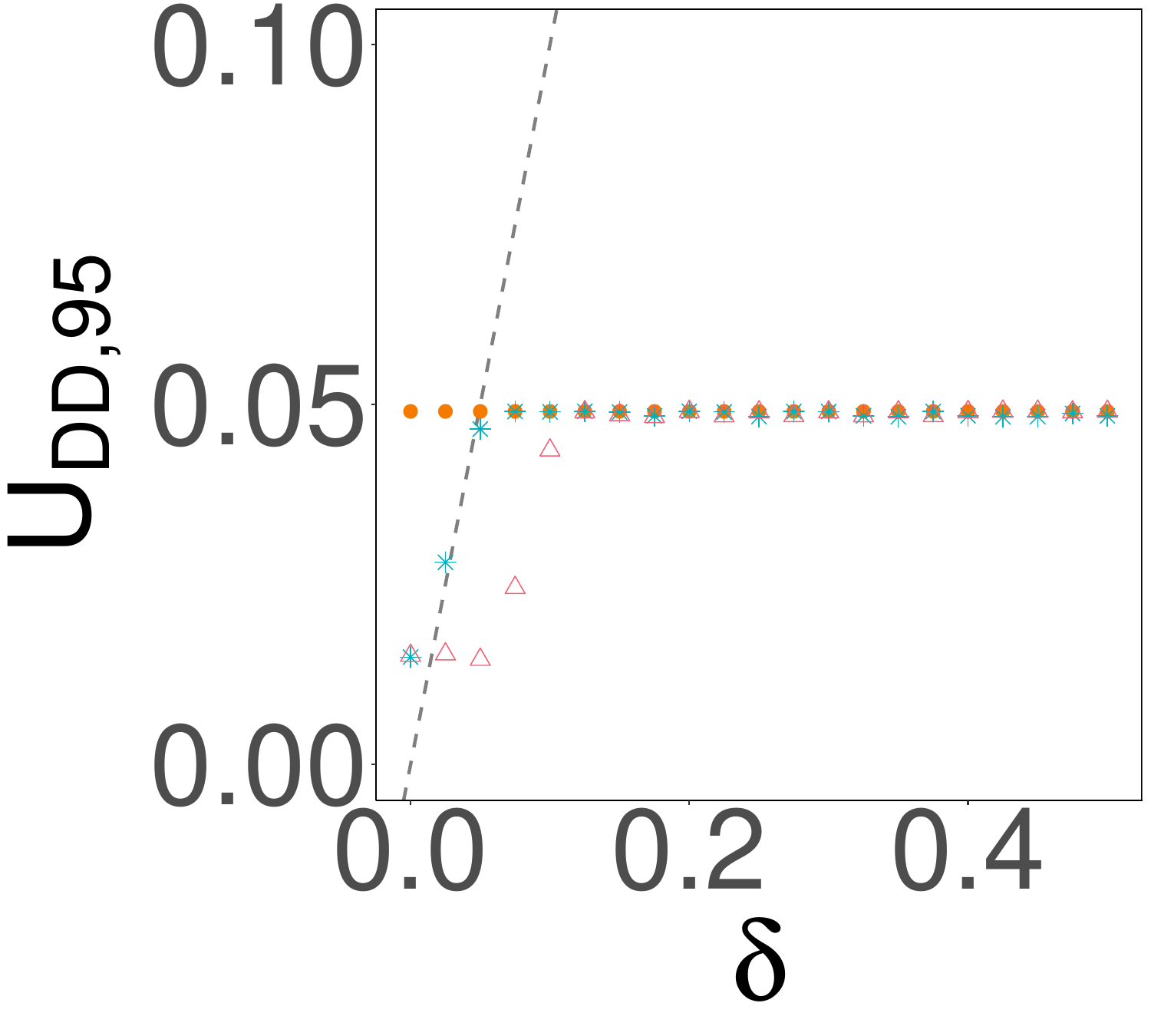} \\
			\hspace{\thisgap}\includegraphics[width=\thiswidth]{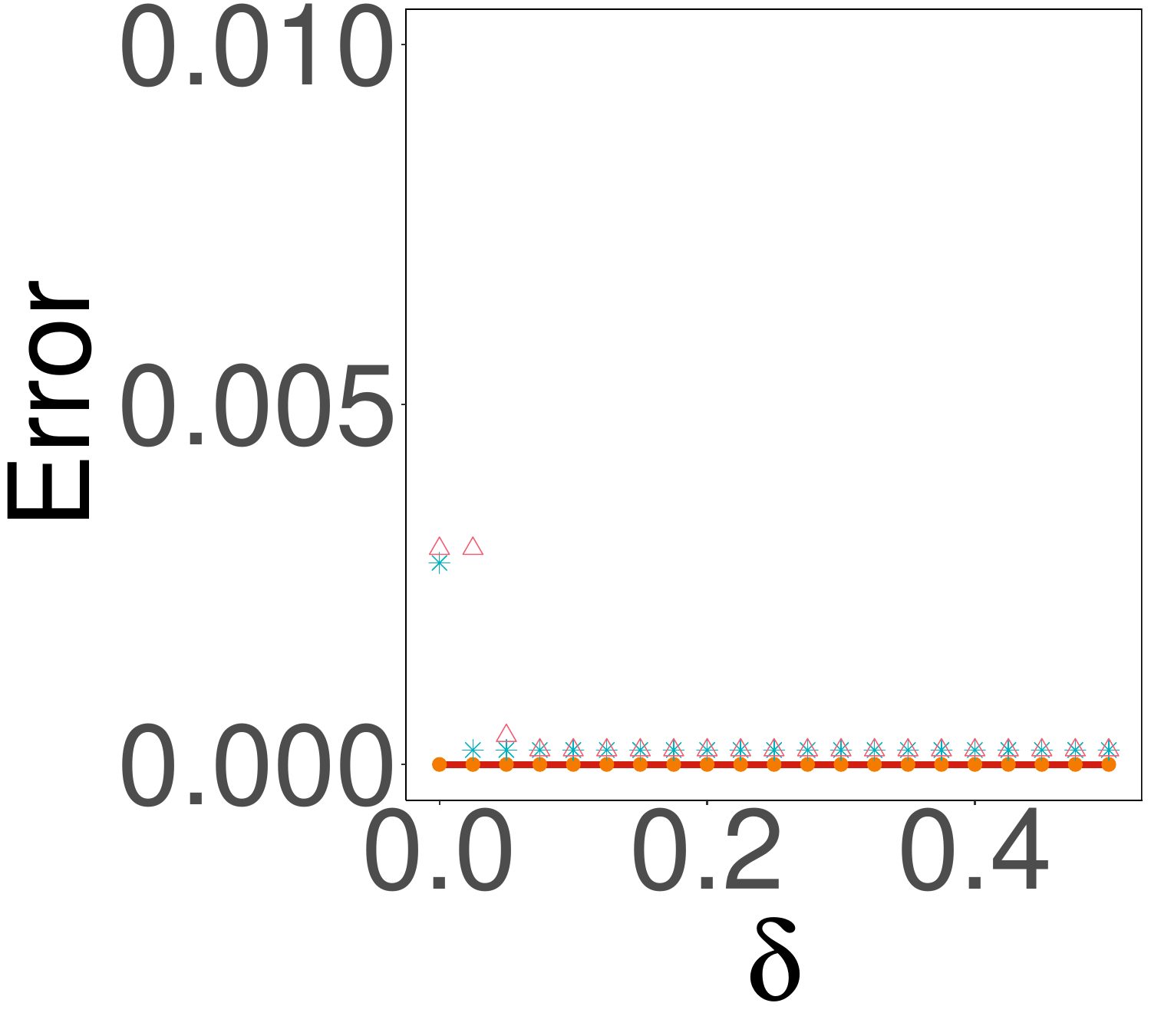} &
			\hspace{\thisgap}\includegraphics[width=\thiswidth]{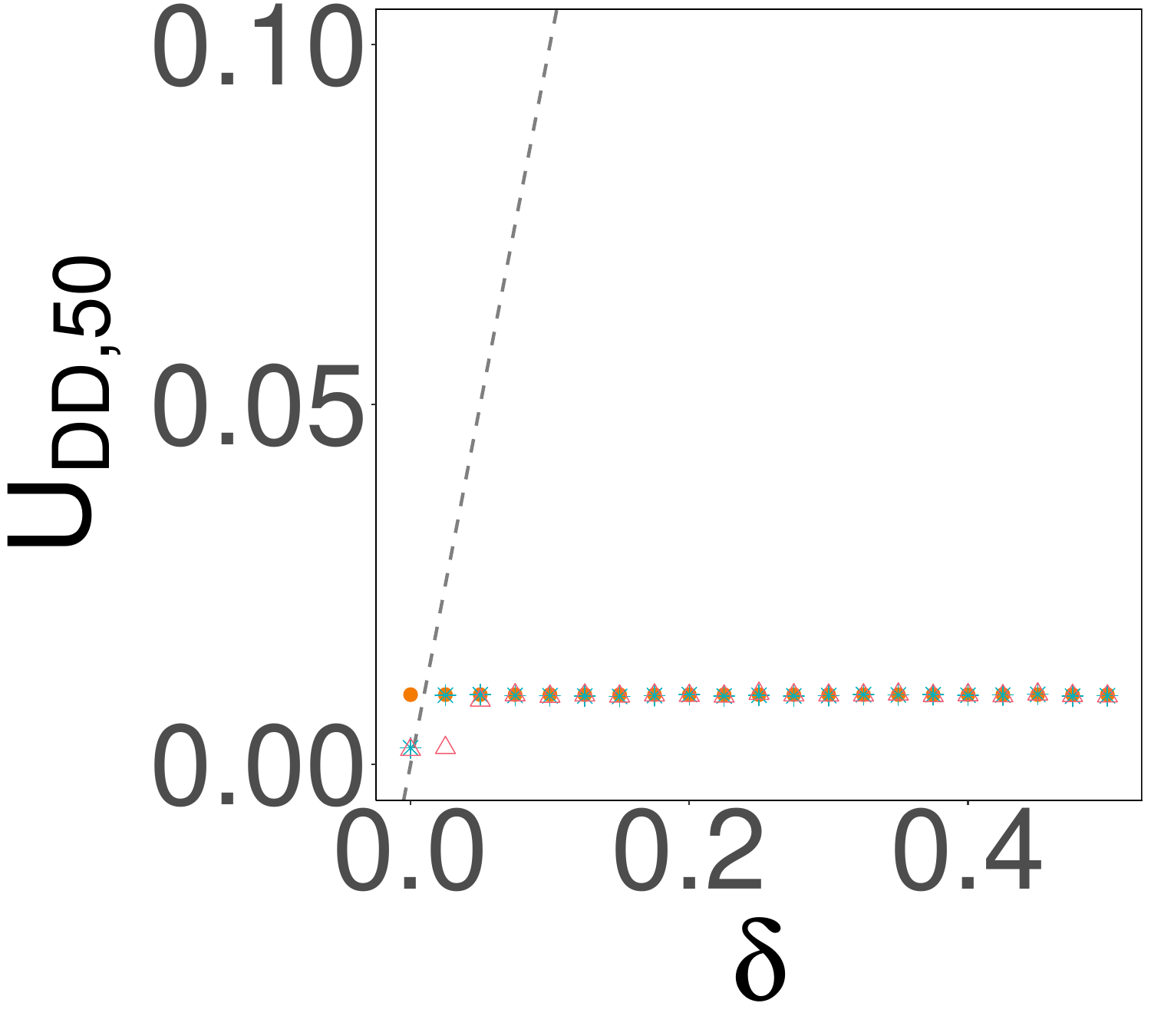} &
			\hspace{\thisgap}\includegraphics[width=\thiswidth]{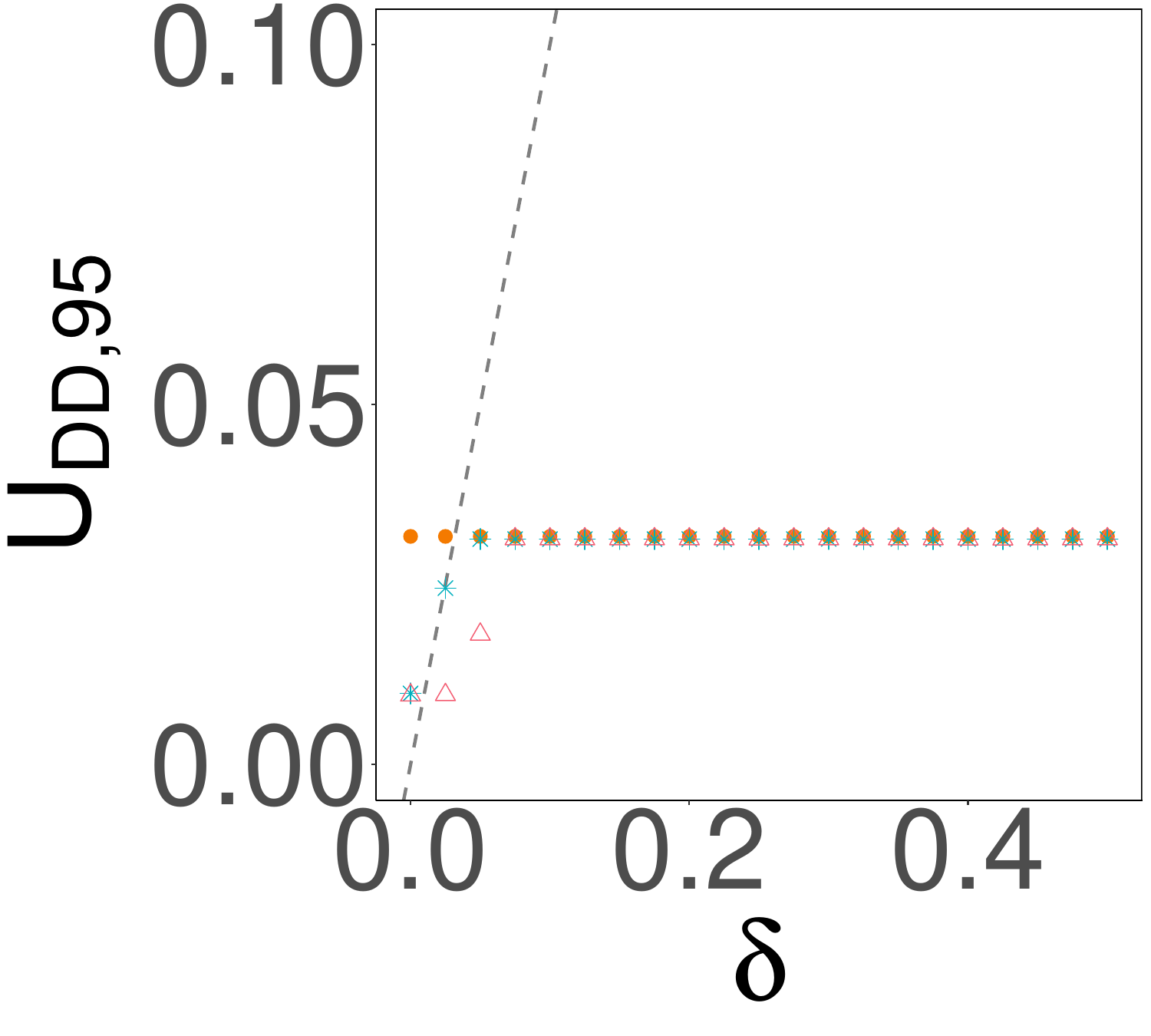}
		\end{tabular}
		\caption{Disparity DD results under (II). Top: $n=1000$; middle: $n=2000$; bottom: $n=5000$.}
		\label{fig:DD_gauss_beta_0.5_II}
	\end{center}
\end{figure*}

\subsection{Additional results for real data}\label{sec:apx_realdata}

In practice, we recommend tuning the calibration parameter $\kappa$ in Fair-$\mathrm{FLDA_c}$ to achieve more reliable probabilistic disparity control. Specifically, we select the smallest value of $\kappa$ such that the empirical $1-\rho$ quantile of the disparity remains below the pre-specified threshold $\delta$. To estimate this empirical quantile, we resort to random splitting. The data are randomly divided into two subsets, with one used to estimate the fairness-aware classifier, and the other to evaluate the resulting disparity. We repeat the process multiple times, e.g.~100 times, and the empirical $1-\rho$ quantile is then computed from the empirical distribution of observed disparities. 

The results obtained using the tuned calibration levels are reported in Figure \ref{fig:nhanes_tune}. As shown, the tuned Fair-$\mathrm{FLDA_c}$ consistently maintains disparity below $\delta$ with probability at least $1-\rho$, except for a slight violation under one small $\delta$ under DO. This demonstrates the overall effectiveness of the proposed tuning strategy.

\begin{figure*}[!htbp]
	\begin{center}
		\newcommand{\thiswidth}{0.18\linewidth}
		\newcommand{\thisgap}{0mm}
		\begin{tabular}{ccc}
			\hspace{\thisgap}\includegraphics[width=\thiswidth]{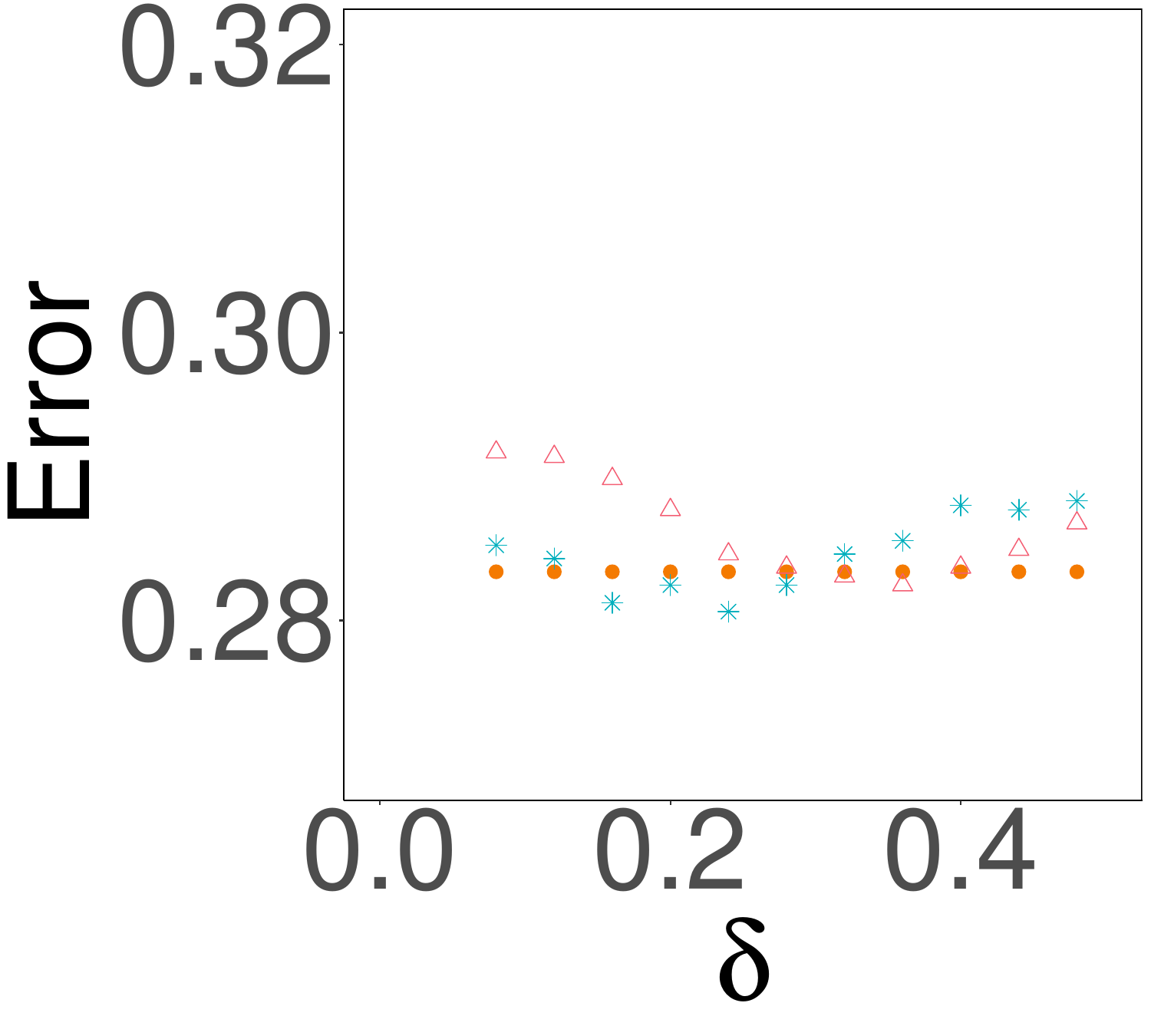} &
			\hspace{\thisgap}\includegraphics[width=\thiswidth]{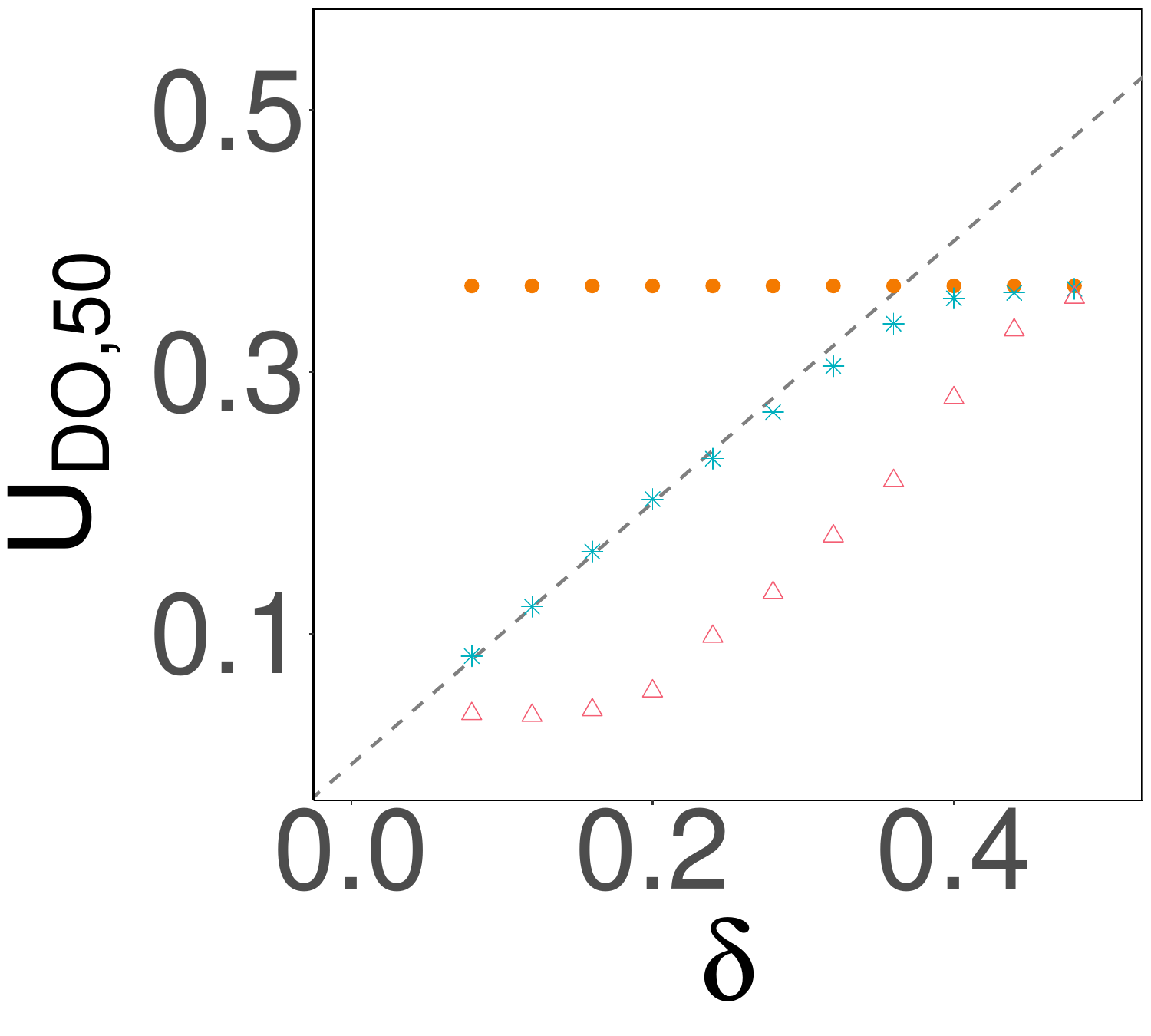} &
			\hspace{\thisgap}\includegraphics[width=\thiswidth]{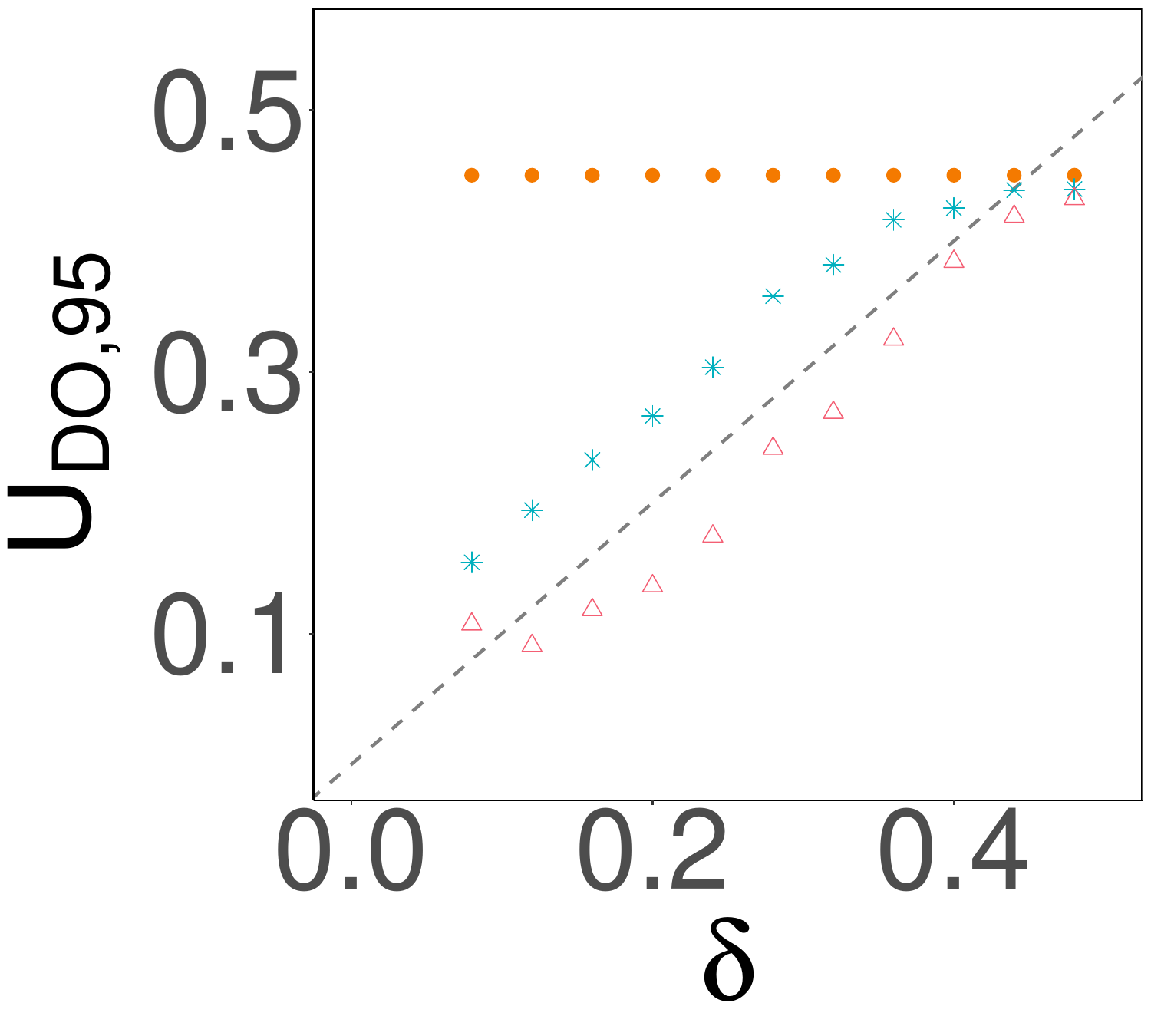} \\
			\hspace{\thisgap}\includegraphics[width=\thiswidth]{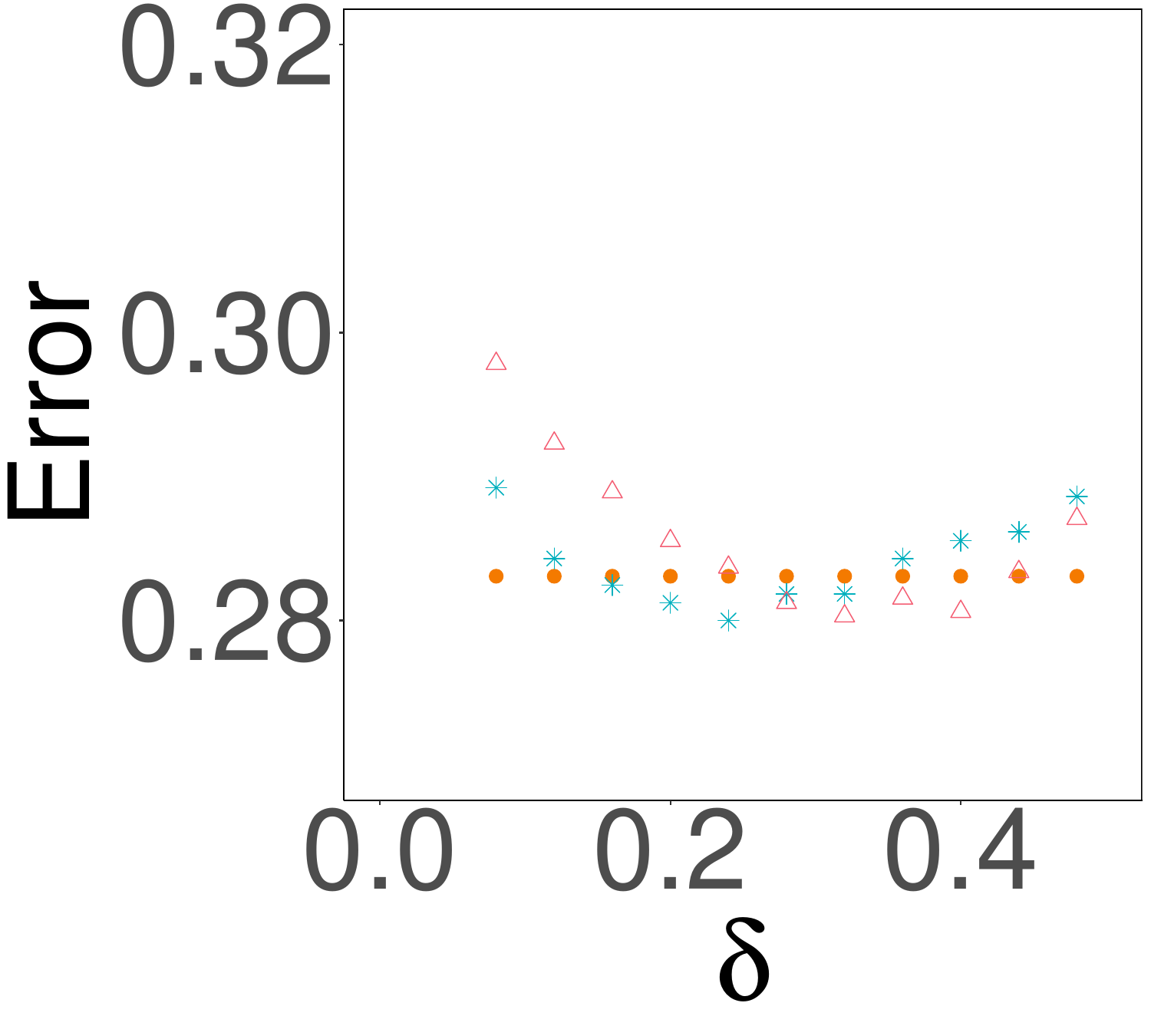} &
			\hspace{\thisgap}\includegraphics[width=\thiswidth]{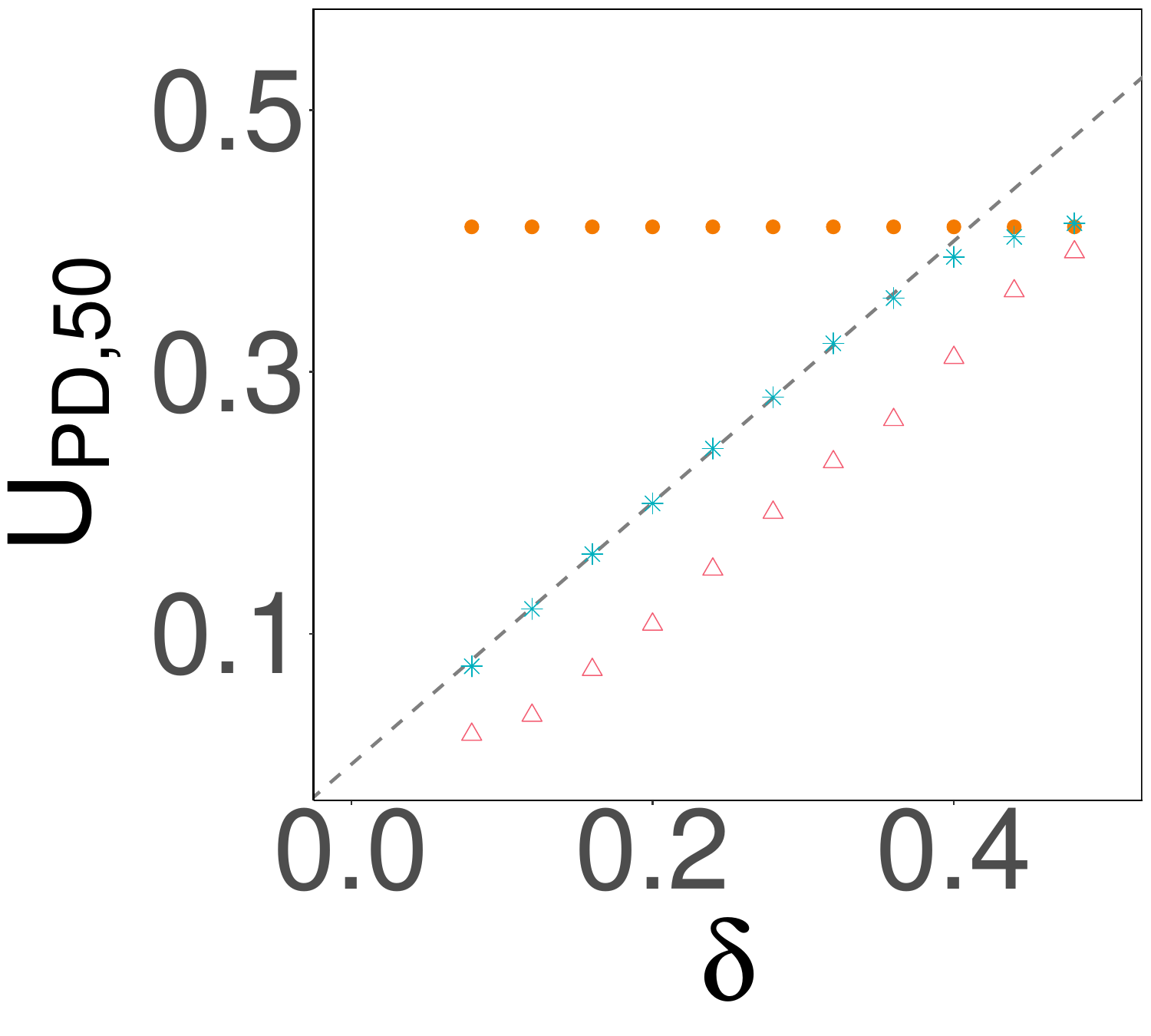} &
			\hspace{\thisgap}\includegraphics[width=\thiswidth]{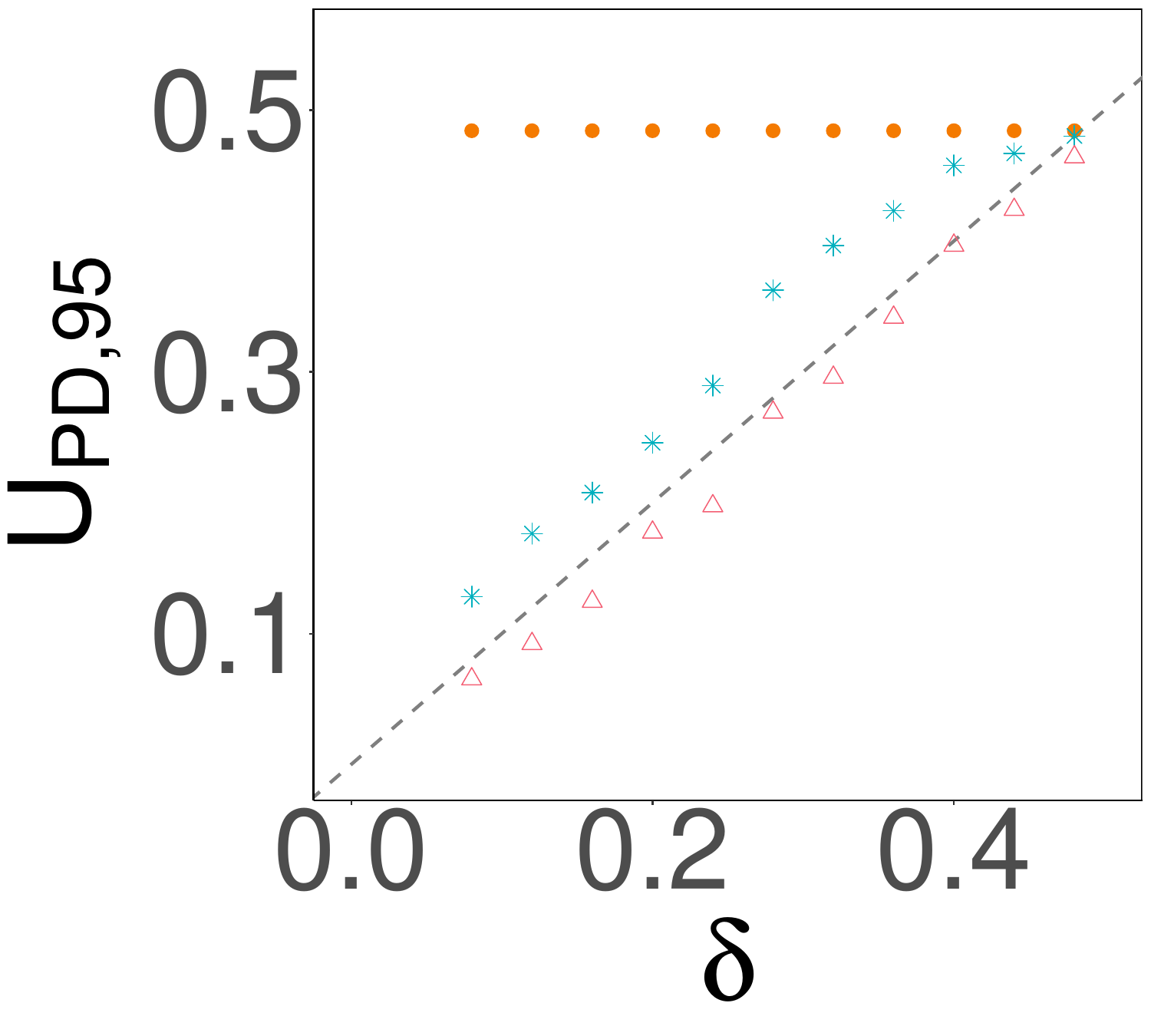} \\
			\hspace{\thisgap}\includegraphics[width=\thiswidth]{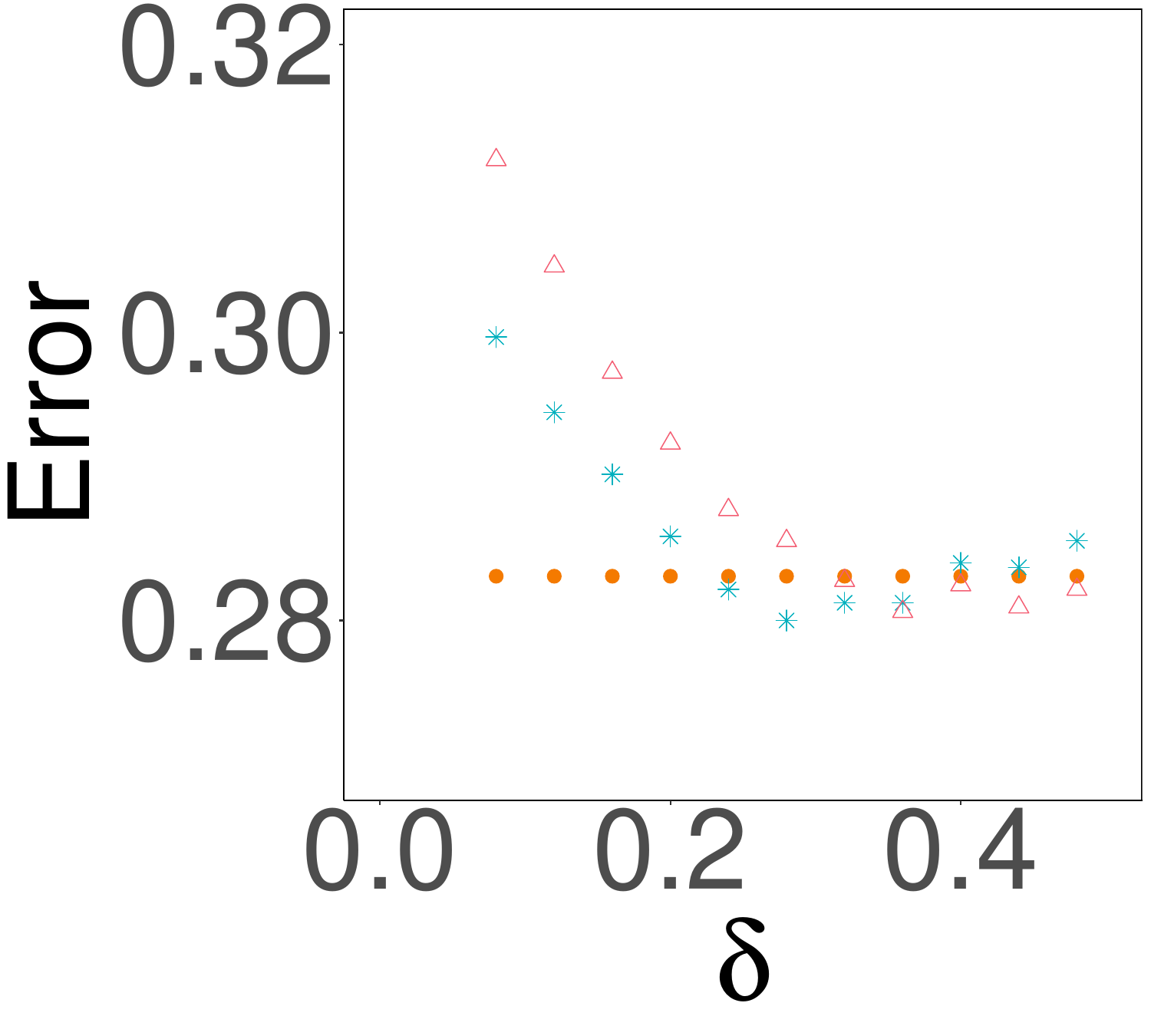} &
			\hspace{\thisgap}\includegraphics[width=\thiswidth]{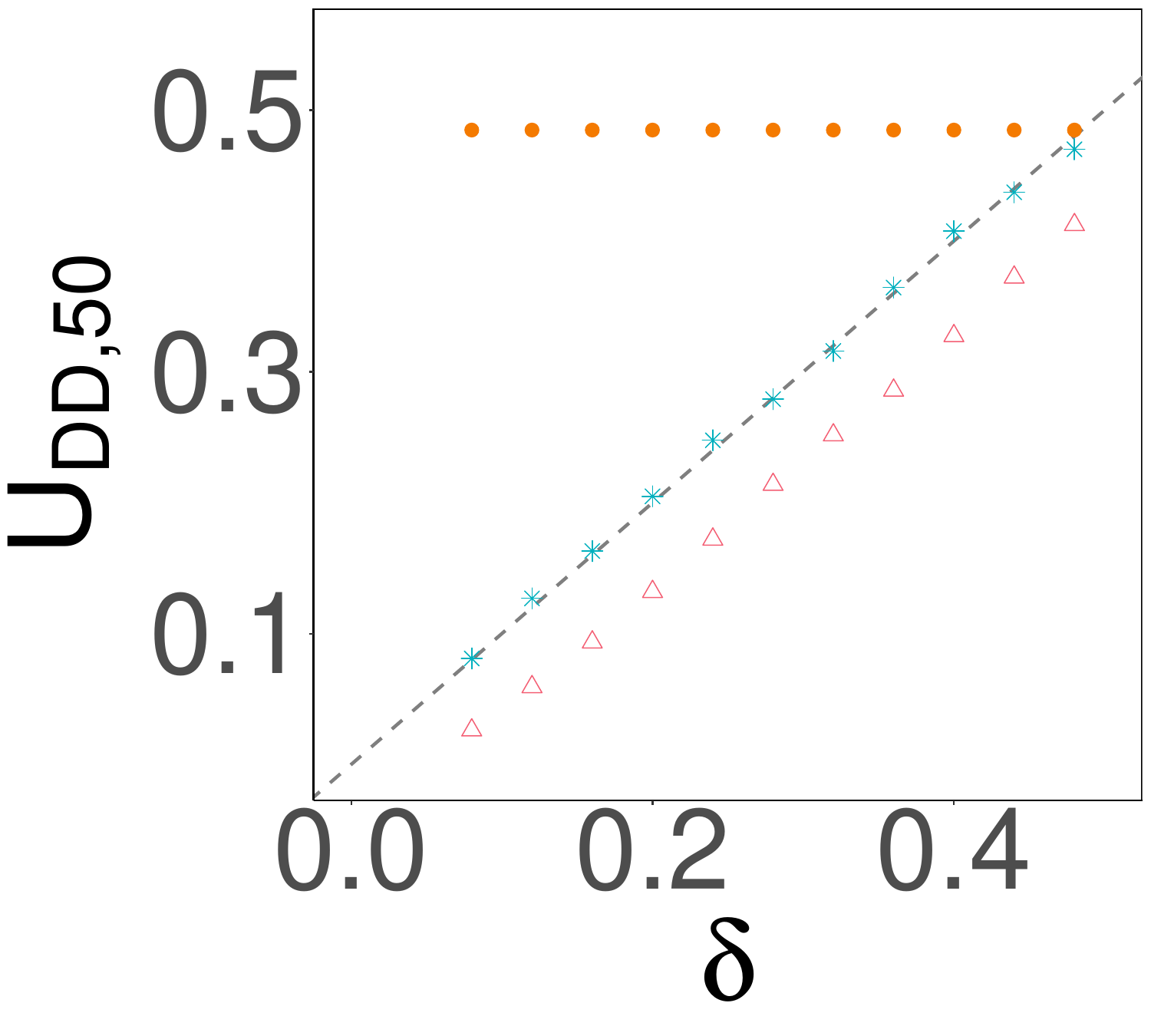} &
			\hspace{\thisgap}\includegraphics[width=\thiswidth]{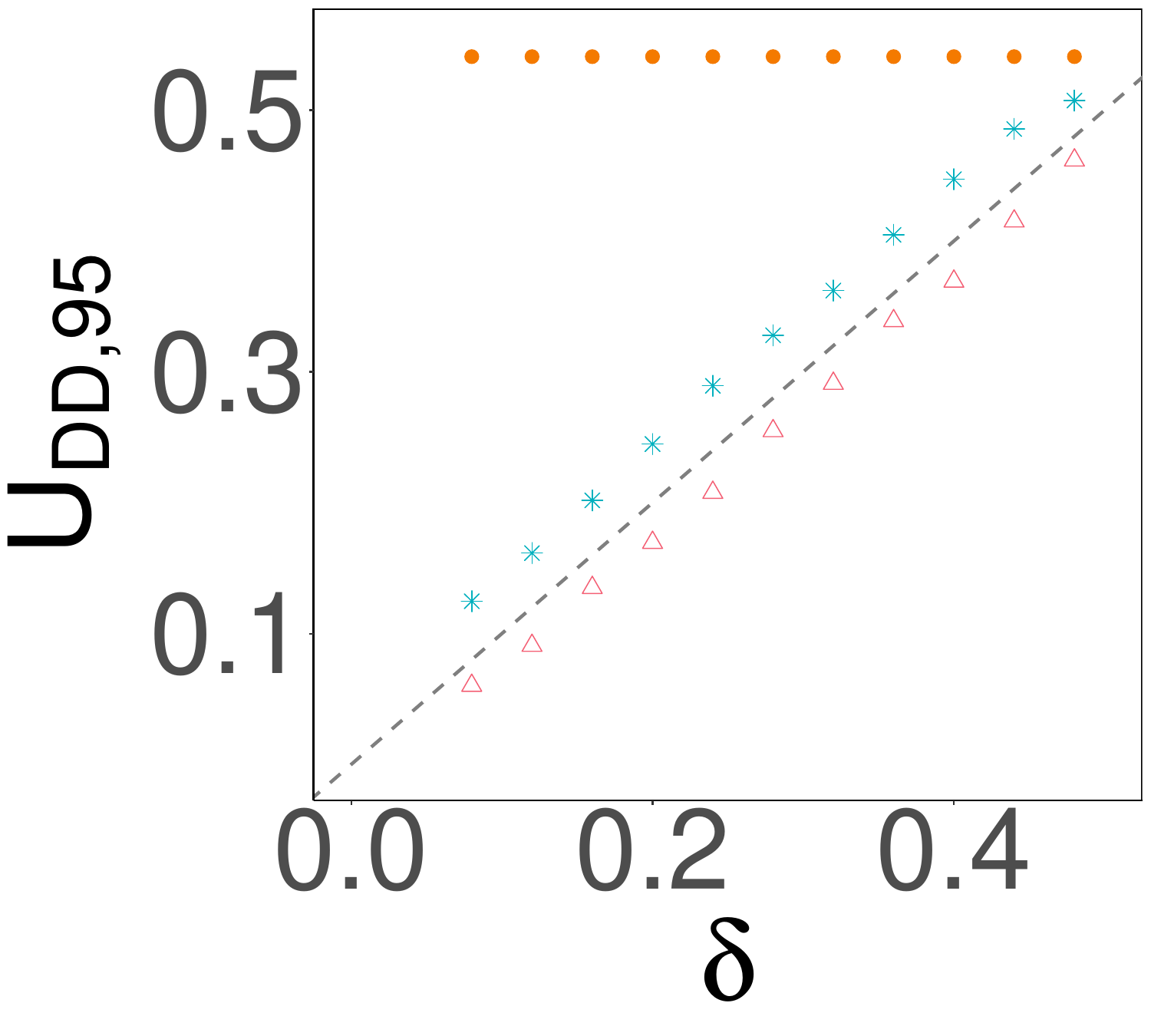} 
		\end{tabular}
		\caption{Results under NHANES with tuned calibration parameters over 100 Monte Carlo trials.}
		\label{fig:nhanes_tune}
	\end{center}
\end{figure*}

\newpage

\section{Proofs for Bayes optimal fairness-aware classifier} \label{appendix_bayes_optimal}
\subsection[]{Proof of \Cref{prop:bilinear_coef}} \label{appendix_proof_bilinear}
\begin{proof}[Proof of \Cref{prop:bilinear_coef}]\leavevmode 
    \begin{itemize}
        \item For DO, we are to show that  $s_{DO, a} = 2a-1$ and $b_{DO, a}=0$.  It can be seen from the following that
    	\begin{align*}
    		DO(f) & = \prob\{ \widehat Y_f(X, 1)=1| A=1, Y=1\} - \prob\{ \widehat Y_f(X, 0)=1 | A=0, Y=1\} \\
    		& = \int_{\Xspace} f(x, 1) \frac{\mathrm{d}P_{1,1}}{\mathrm{d}P_{1,0}}(x)\;\mathrm{d}P_{1,0}(x) - \int_{\Xspace} f(x, 0) \frac{\mathrm{d}P_{0,1}}{\mathrm{d}P_{0,0}}(x) \;\mathrm{d}P_{0,0}(x).
    	\end{align*}

        \item For PD, we are to show that $s_{PD, a}=0$ and $b_{PD, a} = 2a-1$.  It can be seen from the following that
	\begin{align*}
		PD(f) & = \prob\{\widehat Y_f(X, 1)=1|A=1, Y=0\} - \prob\{\widehat Y_f(X, 0)=1|A=0, Y=0\} \\
		& = \int_{\Xspace} f(x, 1) \;\mathrm{d}P_{1,0}(x) - \int_{\Xspace} f(x, 0)\; \mathrm{d}P_{0,0}(x).
	\end{align*}
        
	\item For DD, we are to show that $s_{DD, a} = (2a-1)\pi_{a,1}/\pi_a$ and $b_{DD, a} = (2a-1)\pi_{a,0}/\pi_a$.  It can be seen from the following that
	\begin{align*}
		DD(f)  =\;& \prob\{\widehat Y_f(X, 1)=1| A=1\} - \prob\{\widehat Y_f(X,0)=1 | A=0\} \\
		=\;& \int_{\Xspace} f(x,1)\frac{\pi_{1,1}}{\pi_1}\frac{\mathrm{d}P_{1,1}}{\mathrm{d}P_{1,0}}(x) \;\mathrm{d}P_{1,0}(x) +  \int_{\Xspace} f(x,1)\frac{\pi_{1,0}}{\pi_1}\;\mathrm{d}P_{1,0}(x)  \\
		& - \int_{\Xspace} f(x,0)\frac{\pi_{0,1}}{\pi_0}\frac{\mathrm{d}P_{0,1}}{\mathrm{d}P_{0,0}}(x) \;\mathrm{d}P_{0,0}(x) - \int_{\Xspace} f(x,0)\frac{\pi_{0,0}}{\pi_0}\;\mathrm{d}P_{0,0}(x) \\
		=\;& \int_{\Xspace} f(x,1)\bigg(\frac{\pi_{1,1}}{\pi_1}\frac{\mathrm{d}P_{1,1}}{\mathrm{d}P_{1,0}}(x) + \frac{\pi_{1,0}}{\pi_1} \bigg) \;\mathrm{d}P_{1,0}(x)\\
        &- \int_{\Xspace} f(x,0)\bigg(\frac{\pi_{0,1}}{\pi_0}\frac{\mathrm{d}P_{0,1}}{\mathrm{d}P_{0,0}}(x) + \frac{\pi_{0,0}}{\pi_0} \bigg) \;\mathrm{d}P_{0,0}(x).
	\end{align*}

\end{itemize}
\end{proof}

\subsection[]{Proof of \Cref{thm:fair_bayes_opt}} \label{appendix_proof_optimal_bayes}
\begin{proof}[Proof of \Cref{thm:fair_bayes_opt}]
    If $|D(0)| \le \delta$, the unconstrained Bayes optimal classifier satisfies the fairness constraint. Therefore, we have $\tau_{D, \delta}^\star=0$ and the $\delta$-fair Bayes optimal classifier is given by $f_{D, \delta}^\star = g_{D, 0}$.
    
    If $D(0)>\delta$, by Proposition \ref{prop:D_and_R}, we have $D(\tau_{D, \delta}^\star)=\delta$. Moreover, $\tau_{D, \delta}^\star>0$. By Lemma \ref{lem:fairnessbayesopt_gnp},
    \[ g_{D, \tau_{D, \delta}^\star} = \mathop{\argmin}_{f \in \mathcal F}\bigg\{ R(f) : |D(f)| \le \frac{\tau_{D, \delta}^\star D(\tau_{D, \delta}^\star)}{|\tau_{D, \delta}^\star|} \bigg\} =   \mathop{\argmin}_{f \in \mathcal F}\bigg\{ R(f) : |D(f)| \le \delta  \bigg\}. \]
    
    Analogously, we can establish the claim when $D(0)<-\delta$. This completes the proof.
\end{proof}

\subsection[]{Proof of \Cref{thm_misclassification}} \label{appendix_proof_misclassification_error}
\begin{proof}[Proof of \Cref{thm_misclassification}]
    Let 
    \[ \Lambda_a = \langle X-\mu_{a,0},  \mu_{a,1}-\mu_{a,0} \rangle_{K_a}  = \sum_{j=1}^\infty \frac{(\zeta_{a, j}-\theta_{a,0, j})(\theta_{a, 1, j} - \theta_{a,0, j})}{\lambda_{a,j}}.\]
    Then, by standard properties of Gaussina processes, we have that 
    \begin{align*}
    	\Lambda_a |\{ A=a, Y=0\} &  \sim N\bigg( 0, ~\sum_{j=1}^\infty \frac{ (\theta_{a,1,j} - \theta_{a,0, j})^2 }{\lambda_{a,j}} \bigg),\\
    	\Lambda_a |\{ A=a, Y=1\} & \sim N\bigg( \sum_{j=1}^\infty \frac{ (\theta_{a,1,j} - \theta_{a,0, j})^2 }{\lambda_{a,j}}, ~ \sum_{j=1}^\infty \frac{ (\theta_{a,1,j} - \theta_{a,0, j})^2 }{\lambda_{a,j}} \bigg).
    \end{align*}
    The proposition then follows by a similar argument as the one used in the proof of Theorem 2 in \citet{berrendero2018use} and the format of $f^\star_{D,\delta}$ in \eqref{eq_f_star}.

\end{proof}

\subsection{Auxiliary results} \label{appendix_bayes_optimal_auxiliary}
\begin{proposition}\label{prop:D_and_R}
	Recall that $D(\tau) = D(g_{D, \tau})$, where $g_{D, \tau}$ is defined in \eqref{eq:f_bayesform}. Then, under the assumptions in \Cref{thm:fair_bayes_opt}, the following properties hold.
	\begin{enumerate}
		\item[(i)] The disparity $D(\tau)$ is continuous and non-increasing.
		\item[(ii)]  The misclassification $R(g_{D, \tau})$ is non-increasing on $(-\infty, 0)$ and non-decreasing on $(0, +\infty)$.
	\end{enumerate}
\end{proposition}
\begin{proof}[Proof of Proposition \ref{prop:D_and_R}.]\leavevmode 
    \begin{enumerate}
	\item[(i)] Note that by Definition \ref{def:linear_disparity}, 
	\begin{align*}
		D(\tau) & = \sum_{a \in \{0,1\}} \int_{\Xspace} g_{D, \tau}(x,a) \bigg\{s_{D, a} \frac{\mathrm{d}P_{a,1}}{\mathrm{d}P_{a,0}}(x) + b_{D,a}\bigg\} \mathrm{d}P_{a,0}(x).		
	\end{align*}
	Since $\mathrm{d}P_{a,1}/\mathrm{d}P_{a,0}(x)$ is a continuous random variable given $A=a \in \{0,1\}$ and $Y=y \in \{0, 1\}$, we have that the function $\tau \mapsto \prob_{X|A=a, Y=y}\big( (\pi_{a,1} - \tau s_{D,a})\mathrm{d}P_{a,1}/\mathrm{d}P_{a,0}(x) > \pi_{a,0} + \tau b_{D,a} \big) $ is continuous for $a\in \{0,1\}$ and $y \in \{0, 1\}$. Thus, the function $\tau \mapsto D(\tau)$ is continuous.
	
	Define 
	\[ \mathcal E_{a, +} = \bigg\{x\in \Xspace: s_{D, a} \frac{\mathrm{d}P_{a,1}}{\mathrm{d}P_{a,0}}(x) + b_{D,a} >0  \bigg\}, ~~ \mathcal E_{a, -} = \bigg\{ x \in \Xspace: s_{D, a} \frac{\mathrm{d}P_{a,1}}{\mathrm{d}P_{a,0}}(x) + b_{D,a} <0  \bigg\}. \]
	Let $\tau_1 < \tau_2$. For $a \in \{0, 1\}$ and $x \in \Xspace$,
	\begin{align*}
		g_{D, \tau_1}(x, a) - g_{D, \tau_2}(x, a) = \left\{ 
		\begin{array}{cc}
			\indc\bigg\{\tau_1 < \frac{\pi_{a,1} \frac{\mathrm{d}P_{a,1}}{\mathrm{d}P_{a,0}}(x) - \pi_{a,0} }{s_{D, a} \frac{\mathrm{d}P_{a,1}}{\mathrm{d}P_{a,0}}(x) + b_{D,a} } \le \tau_2 \bigg\}, & x \in \mathcal E_{a, +} ; \\
			-\indc\bigg\{\tau_1 \le \frac{\pi_{a,1} \frac{\mathrm{d}P_{a,1}}{\mathrm{d}P_{a,0}}(x) - \pi_{a,0} }{s_{D, a} \frac{\mathrm{d}P_{a,1}}{\mathrm{d}P_{a,0}}(x) + b_{D,a} } < \tau_2 \bigg\}, & x \in \mathcal E_{a, -} ; \\
			0, & \text{otherwise}.
		\end{array}
		\right.
	\end{align*}
	We then have 
	\begin{align*}
		& \qquad D(\tau_1) - D(\tau_2) \\
		& =  \sum_{a \in \{0,1\}} \int_{\Xspace}\{ g_{D, \tau_1}(x, a) - g_{D, \tau_2}(x, a)  \} \bigg\{s_{D, a} \frac{\mathrm{d}P_{a,1}}{\mathrm{d}P_{a,0}}(x) + b_{D,a}\bigg\} \mathrm{d}P_{a,0}(x) \\
		& =  \sum_{a \in \{0,1\}} \int_{x \in \mathcal E_{a, +}} \indc\bigg\{\tau_1 < \frac{\pi_{a,1} \frac{\mathrm{d}P_{a,1}}{\mathrm{d}P_{a,0}}(x) - \pi_{a,0} }{s_{D, a} \frac{\mathrm{d}P_{a,1}}{\mathrm{d}P_{a,0}}(x) + b_{D,a} } \le \tau_2 \bigg\} \bigg\{s_{D, a} \frac{\mathrm{d}P_{a,1}}{\mathrm{d}P_{a,0}}(x) + b_{D,a}\bigg\} \mathrm{d}P_{a,0}(x) \\
		&  \quad -  \sum_{a \in \{0,1\}} \int_{x \in \mathcal E_{a, -}} \indc\bigg\{\tau_1 \le \frac{\pi_{a,1} \frac{\mathrm{d}P_{a,1}}{\mathrm{d}P_{a,0}}(x) - \pi_{a,0} }{s_{D, a} \frac{\mathrm{d}P_{a,1}}{\mathrm{d}P_{a,0}}(x) + b_{D,a} } < \tau_2 \bigg\} \bigg\{s_{D, a} \frac{\mathrm{d}P_{a,1}}{\mathrm{d}P_{a,0}}(x) + b_{D,a}\bigg\} \mathrm{d}P_{a,0}(x) \\
		& \ge 0.
	\end{align*}
	Consequently, the function $\tau \mapsto D(\tau)$ is non-increasing.
	
	\item[(ii)] We first consider $\tau_1 < \tau_2 <0$.  
	If $x \in \mathcal E_{a, +}$, 
	\begin{align*}
		& \indc\bigg\{\tau_1 < \frac{\pi_{a,1} \frac{\mathrm{d}P_{a,1}}{\mathrm{d}P_{a,0}}(x) - \pi_{a,0} }{s_{D, a} \frac{\mathrm{d}P_{a,1}}{\mathrm{d}P_{a,0}}(x) + b_{D,a} } \le \tau_2 \bigg\} \bigg\{ \pi_{a,0} - \pi_{a,1} \frac{\mathrm{d}P_{a,1}}{\mathrm{d}P_{a,0}}(x) \bigg\} \\
		& \ge -\tau_2 \bigg\{ s_{D, a} \frac{\mathrm{d}P_{a,1}}{\mathrm{d}P_{a,0}}(x) + b_{D,a}  \bigg\} \indc\bigg\{\tau_1 < \frac{\pi_{a,1} \frac{\mathrm{d}P_{a,1}}{\mathrm{d}P_{a,0}}(x) - \pi_{a,0} }{s_{D, a} \frac{\mathrm{d}P_{a,1}}{\mathrm{d}P_{a,0}}(x) + b_{D,a} }  \le \tau_2 \bigg\}  \ge 0. 
	\end{align*}
	
	If $x \in \mathcal E_{a, -}$,
	\begin{align*}
		& \indc\bigg\{\tau_1 \le \frac{\pi_{a,1} \frac{\mathrm{d}P_{a,1}}{\mathrm{d}P_{a,0}}(x) - \pi_{a,0} }{s_{D, a} \frac{\mathrm{d}P_{a,1}}{\mathrm{d}P_{a,0}}(x) + b_{D,a} } < \tau_2 \bigg\} \bigg\{ \pi_{a,0} - \pi_{a,1} \frac{\mathrm{d}P_{a,1}}{\mathrm{d}P_{a,0}}(x) \bigg\} \\
		& \le -\tau_2 \bigg\{s_{D, a} \frac{\mathrm{d}P_{a,1}}{\mathrm{d}P_{a,0}}(x) + b_{D,a}  \bigg\}\indc\bigg\{\tau_1 \le \frac{\pi_{a,1} \frac{\mathrm{d}P_{a,1}}{\mathrm{d}P_{a,0}}(x) - \pi_{a,0} }{s_{D, a} \frac{\mathrm{d}P_{a,1}}{\mathrm{d}P_{a,0}}(x) + b_{D,a} } < \tau_2 \bigg\} \le 0.  
	\end{align*}
	
	Then, by \Cref{l:risk_densityratio}, it holds that
	\begin{align*}
		& \qquad R(g_{D, \tau_1}) - R(g_{D, \tau_2}) \\
		& = \sum_{a \in \{0,1\}} \int_{\Xspace}\{ g_{D, \tau_1}(x, a) - g_{D, \tau_2}(x, a)  \} \bigg\{ \pi_{a,0} - \pi_{a,1} \frac{\mathrm{d}P_{a,1}}{\mathrm{d}P_{a,0}}(x) \bigg\} \mathrm{d}P_{a,0}(x) \\
		& =  \sum_{a \in \{0,1\}} \int_{x \in \mathcal E_{a, +}} \indc\bigg\{\tau_1 < \frac{\pi_{a,1} \frac{\mathrm{d}P_{a,1}}{\mathrm{d}P_{a,0}}(x) - \pi_{a,0} }{s_{D, a} \frac{\mathrm{d}P_{a,1}}{\mathrm{d}P_{a,0}}(x) + b_{D,a} } \le \tau_2 \bigg\} \bigg\{ \pi_{a,0} - \pi_{a,1} \frac{\mathrm{d}P_{a,1}}{\mathrm{d}P_{a,0}}(x) \bigg\} \mathrm{d}P_{a,0}(x) \\
		&  \quad -  \sum_{a \in \{0,1\}} \int_{x \in \mathcal E_{a, -}} \indc\bigg\{\tau_1 \le \frac{\pi_{a,1} \frac{\mathrm{d}P_{a,1}}{\mathrm{d}P_{a,0}}(x) - \pi_{a,0} }{s_{D, a} \frac{\mathrm{d}P_{a,1}}{\mathrm{d}P_{a,0}}(x) + b_{D,a} } < \tau_2 \bigg\} \bigg\{ \pi_{a,0} - \pi_{a,1} \frac{\mathrm{d}P_{a,1}}{\mathrm{d}P_{a,0}}(x) \bigg\} \mathrm{d}P_{a,0}(x) \\
		& \ge 0.
	\end{align*}
	Therefore, $\tau \mapsto R(g_{D, \tau})$ is non-increasing on $(-\infty, 0)$.
	
	Consider $0 \le \tau_1 < \tau_2$.
	If $x \in \mathcal E_{a, +}$, 
	\begin{align*}
		& \indc\bigg\{\tau_1 < \frac{\pi_{a,1} \frac{\mathrm{d}P_{a,1}}{\mathrm{d}P_{a,0}}(x) - \pi_{a,0} }{s_{D, a} \frac{\mathrm{d}P_{a,1}}{\mathrm{d}P_{a,0}}(x) + b_{D,a} } \le \tau_2 \bigg\} \bigg\{ \pi_{a,0} - \pi_{a,1} \frac{\mathrm{d}P_{a,1}}{\mathrm{d}P_{a,0}}(x) \bigg\} \\
		& \le -\tau_1 \bigg\{ s_{D, a} \frac{\mathrm{d}P_{a,1}}{\mathrm{d}P_{a,0}}(x) + b_{D,a}  \bigg\} \indc\bigg\{\tau_1 < \frac{\pi_{a,1} \frac{\mathrm{d}P_{a,1}}{\mathrm{d}P_{a,0}}(x) - \pi_{a,0} }{s_{D, a} \frac{\mathrm{d}P_{a,1}}{\mathrm{d}P_{a,0}}(x) + b_{D,a} }  \le \tau_2 \bigg\}  \le 0. 
	\end{align*}
	
	If $x \in \mathcal E_{a, -}$,
	\begin{align*}
		& \indc\bigg\{\tau_1 \le \frac{\pi_{a,1} \frac{\mathrm{d}P_{a,1}}{\mathrm{d}P_{a,0}}(x) - \pi_{a,0} }{s_{D, a} \frac{\mathrm{d}P_{a,1}}{\mathrm{d}P_{a,0}}(x) + b_{D,a} } < \tau_2 \bigg\} \bigg\{ \pi_{a,0} - \pi_{a,1} \frac{\mathrm{d}P_{a,1}}{\mathrm{d}P_{a,0}}(x) \bigg\} \\
		& \ge -\tau_1 \bigg\{s_{D, a} \frac{\mathrm{d}P_{a,1}}{\mathrm{d}P_{a,0}}(x) + b_{D,a}  \bigg\}\indc\bigg\{\tau_1 \le \frac{\pi_{a,1} \frac{\mathrm{d}P_{a,1}}{\mathrm{d}P_{a,0}}(x) - \pi_{a,0} }{s_{D, a} \frac{\mathrm{d}P_{a,1}}{\mathrm{d}P_{a,0}}(x) + b_{D,a} } < \tau_2 \bigg\} \ge 0.  
	\end{align*}
	Then, we have
	\begin{align*}
		& \qquad R(g_{D, \tau_1}) - R(g_{D, \tau_2}) \\
		& = \sum_{a \in \{0,1\}} \int_{\Xspace}\{ g_{D, \tau_1}(x, a) - g_{D, \tau_2}(x, a)  \} \bigg\{ \pi_{a,0} - \pi_{a,1} \frac{\mathrm{d}P_{a,1}}{\mathrm{d}P_{a,0}}(x) \bigg\} \mathrm{d}P_{a,0}(x) \\
		& =  \sum_{a \in \{0,1\}} \int_{x \in \mathcal E_{a, +}} \indc\bigg\{\tau_1 < \frac{\pi_{a,1} \frac{\mathrm{d}P_{a,1}}{\mathrm{d}P_{a,0}}(x) - \pi_{a,0} }{s_{D, a} \frac{\mathrm{d}P_{a,1}}{\mathrm{d}P_{a,0}}(x) + b_{D,a} } \le \tau_2 \bigg\} \bigg\{ \pi_{a,0} - \pi_{a,1} \frac{\mathrm{d}P_{a,1}}{\mathrm{d}P_{a,0}}(x) \bigg\} \mathrm{d}P_{a,0}(x) \\
		&  \quad -  \sum_{a \in \{0,1\}} \int_{x \in \mathcal E_{a, -}} \indc\bigg\{\tau_1 \le \frac{\pi_{a,1} \frac{\mathrm{d}P_{a,1}}{\mathrm{d}P_{a,0}}(x) - \pi_{a,0} }{s_{D, a} \frac{\mathrm{d}P_{a,1}}{\mathrm{d}P_{a,0}}(x) + b_{D,a} } < \tau_2 \bigg\} \bigg\{ \pi_{a,0} - \pi_{a,1} \frac{\mathrm{d}P_{a,1}}{\mathrm{d}P_{a,0}}(x) \bigg\} \mathrm{d}P_{a,0}(x) \\
		& \le 0.
	\end{align*}
	Therefore, $\tau \mapsto R(g_{D, \tau})$ is non-decreasing on $[0, +\infty)$.
	
\end{enumerate}
\end{proof}

\begin{lemma}\label{l:risk_densityratio}
    For any classifier $f: \Xspace\times \Aspace \rightarrow [0,1]$, we have
    \[R(f) = \sum_{a \in \{0,1\}} \int_{\Xspace} f(x,a) \bigg\{ \pi_{a,0} - \pi_{a,1} \frac{\mathrm{d}P_{a,1}}{\mathrm{d}P_{a,0}}(x) \bigg\} \mathrm{d}P_{a,0}(x) +  \prob(Y=1).\]
\end{lemma}

\begin{proof}
    By definition, we have that 
    \begin{align*}
	R(f) =\;& \sum_{a, y \in \{0,1\}} \prob\{\widehat Y_f(X, a) \neq y | A=a, Y = y\}P(A=a, Y=y) \nonumber \\
	=\;& \sum_{a \in \{0,1\}} \big[ \mean\{1 - f(X, a) | A=a, Y = 1\}P(A=a, Y=1) \\
    &+  \mean\{f(X, a) | A=a, Y = 0\}P(A=a, Y=0)  \big] \nonumber\\
	=\;& \sum_{a \in \{0,1\}} \bigg\{ \int_{\Xspace} f(x,a)\mathrm{d}P_{a,0}(x) \pi_{a,0} - \int_{\Xspace} f(x,a)\mathrm{d}P_{a,1}(x)\pi_{a,1} \bigg\} +  \sum_{a \in \{0,1\}} \prob(A=a, Y=1) \nonumber\\
	=\;&  \sum_{a \in \{0,1\}} \int_{\Xspace} f(x,a) \bigg\{ \pi_{a,0} - \pi_{a,1} \frac{\mathrm{d}P_{a,1}}{\mathrm{d}P_{a,0}}(x) \bigg\} \mathrm{d}P_{a,0}(x) +  \prob(Y=1).
\end{align*}
\end{proof}

\begin{lemma} \label{l_range_tau}
    For the bilinear disparity measures DO, PD and DD, it holds that 
    \begin{align*}
        \pi_{a,1} - \tau_{D, \delta}^\star s_{D,a}>0, \;\; \text{and} \;\;\pi_{a,0}+\tau_{D, \delta}^\star b_{D,a}>0,
    \end{align*}
    with $s_{D,a}$ and $b_{D,a}$ defined in \Cref{def:linear_disparity}.
\end{lemma}
\begin{proof} \leavevmode 
    \begin{enumerate}
        \item \textbf{DO:} In this case, it suffices to show that $- \pi_{0,1} < \tau_{DO, \delta}^\star < \pi_{1,1}$.
        If $\tau=\pi_{1,1}$, then $g_{DO, \pi_{1,1}}(x, 1) = 0$ for all $x \in \Xspace$, and,
        \[ DO(\pi_{1,1}) = DO(g_{DO, \pi_{1,1}}) = -\prob\Big\{ \big(\pi_{0,1} + \pi_{1,1}\big) \frac{\mathrm{d}P_{0,1}}{\mathrm{d}P_{0,0}}(x) > \pi_{0,0} \Big| A=0, Y=1\Big\} \le 0. \]
        If $\tau = -\pi_{0,1}$, then $g_{DO, -\pi_{0,1}}(x, 0) = 0$ for all  $x \in \Xspace$. Then 
        \[  DO(g_{DO, -\pi_{0,1}}) = \prob\Big\{ \big(\pi_{1,1} + \pi_{0,1}\big) \frac{\mathrm{d}P_{1,1}}{\mathrm{d}P_{1,0}}(x) > \pi_{1,0} \Big| A=1, Y=1\Big\} \ge 0.  \]
        Note that if $|DO(0)| \le \delta$, then $\tau_{DO, \delta}^\star =0$. By Proposition \ref{prop:D_and_R}(i), if $DO(0)>\delta$, then $0~<~\tau_{DO, \delta}^\star$ $<\pi_{1,1} $. If $DO(0) < -\delta$, then $-\pi_{0,1} < \tau_{DO, \delta}^\star <0$. Therefore, we conclude that $- \pi_{0,1} < \tau_{DO, \delta}^\star < \pi_{1,1}$.

        \item \textbf{PD} In this case, it suffices to show that $- \pi_{1,0} < \tau_{PD, \delta}^\star < \pi_{0,0}$. Note that if $\tau = -\pi_{1,0}$, we have $g_{PD, -\pi_{1,0}}(x, 1) = 1$ for all $x \in \Xspace$. Then,
        \[ PD(g_{PD, -\pi_{1,0}}) = 1 - \prob\Big\{ \pi_{0,1} \frac{\mathrm{d}P_{0,1}}{\mathrm{d}P_{0,0}}(x) > \pi_{0,0} + \pi_{1,0} \Big| A=0, Y=0 \Big\} \ge 0. \]
        If $\tau = \pi_{0,0}$, we have $g_{PD, \pi_{0,0}}(x, 0) = 1$ for all $x \in \Xspace$. Then,
        \[ PD(g_{PD, \pi_{0,0}}) = \prob\Big\{ \pi_{1,1} \frac{\mathrm{d}P_{1,1}}{\mathrm{d}P_{1,0}}(x) > \pi_{1,0} + \pi_{0,0} \Big| A=1, Y=0 \Big\} - 1 \le 0. \]
        
        Note that if $|PD(0)| \le \delta$, then $\tau_{PD, \delta}^\star =0$. By Proposition \ref{prop:D_and_R} (i), if $PD(0)>\delta$, then $0< \tau_{PD, \delta}^\star < \pi_{0,0} $. If $PD(0) < -\delta$, then $-\pi_{1,0} < \tau_{PD, \delta}^\star <0$. Therefore, we conclude that $- \pi_{1,0} < \tau_{PD, \delta}^\star < \pi_{0,0}$. 

         \item \textbf{DD:}  In this case, it suffices to show that $|\tau_{DD, \delta}^\star| < \min\{\pi_0, \pi_1 \}$. Note that if $\tau = \pi_1$ then $g_{DD, \pi_1}(x, 1) = 0$ for all $x \in \Xspace$. Also, 
        \begin{align*}
        DD(\pi_1) & = - \prob\Big\{ \big(\pi_{0,1} + \frac{\pi_1 \pi_{0,1}}{\pi_0}\big) \frac{\mathrm{d}P_{0,1}}{\mathrm{d}P_{0,0}}(x) > \pi_{0,0} - \frac{\pi_1 \pi_{0,0}}{\pi_0}  \Big| A=0\Big\}  \\
        &  \left\{ \begin{array}{cc}
        	= -1, & \pi_1 \ge \pi_0, \\
        	\le 0, & \pi_1 < \pi_0.
        \end{array} \right.
        \end{align*}
        If $\tau = \pi_0$, then $g_{DD, \pi_0}(x, 0) = 1$ for all $x \in \Xspace$. Then
        \[ DD(\pi_0) =  \prob\Big\{ \Big(\pi_{1,1} - \frac{\pi_0 \pi_{1,1}}{\pi_1}\Big) \frac{\mathrm{d}P_{1,1}}{\mathrm{d}P_{1,0}}(x) > \pi_{1,0} + \frac{\pi_0 \pi_{1,0}}{\pi_1}  \Big| A=1\Big\} - 1 \le 0. \]
        
        Moreover, if $\tau=-\pi_0$, then $g_{DD, -\pi_0}(x, 0)=0$ for all $x \in \Xspace$. And,
        \begin{align*}
        	DD(-\pi_0) & = \prob\Big\{ \Big(\pi_{1,1} + \frac{\pi_0 \pi_{1,1}}{\pi_1}\Big) \frac{\mathrm{d}P_{1,1}}{\mathrm{d}P_{1,0}}(x) > \pi_{1,0} - \frac{\pi_0 \pi_{1,0}}{\pi_1}  \Big| A=1\Big\}   \\
        	& \left\{ \begin{array}{cc}
        		\ge 0, & \pi_1 \ge \pi_0, \\
        		= 1, & \pi_1 < \pi_0.
        	\end{array} \right.
        \end{align*}
        If $\tau=-\pi_1$, then $g_{DD, -\pi_1}(x, 1)= 1$ for all $x \in \Xspace$. Then,
        \[ DD(-\pi_1) = 1 -  \prob\Big\{ \Big(\pi_{0,1} - \frac{\pi_1 \pi_{0,1}}{\pi_0}\Big) \frac{\mathrm{d}P_{0,1}}{\mathrm{d}P_{0,0}}(x) > \pi_{0,0} + \frac{\pi_1 \pi_{0,0}}{\pi_0}  \Big| A=0\Big\} \ge 0. \]
        If $|DD(0)| \le \delta$, then $\tau_{DD, \delta}^\star =0$.
        By Proposition \ref{prop:D_and_R} (i), if $DD(0) > \delta$, then $0 < \tau_{DD, \delta}^\star < \min\{\pi_0, \pi_1\}$. If $DD(0) < -\delta$, then  $\max\{-\pi_0, -\pi_1\} < \tau_{DD, \delta}^\star <0$. Therefore, we have $|\tau_{DD, \delta}^\star| < \min\{\pi_0, \pi_1 \}$.        
    \end{enumerate}
\end{proof}

\begin{lemma}\label{lem:fairnessbayesopt_gnp}
	Recall that $g_{D, \tau}$ is defined in \eqref{eq:f_bayesform} and $D(\tau) = D(g_{D, \tau})$. For any fixed $\tau \in \mathbb{R}$, 
	\[ g_{D, \tau} = \mathop{\argmin}_{f \in \mathcal F} \bigg\{ R(f) : \frac{\tau D(f)}{|\tau|} \le \frac{\tau D(\tau)}{|\tau|} \bigg\}. \]
	Moreover, for all classifiers $f^\prime \in  \mathop{\argmin}_{f \in \mathcal F} \big\{ R(f) : \tau D(f)/|\tau| \le\tau D(\tau)/|\tau| \big\}$, $f'=g_{D, \tau}$ almost surely with respect to $P_{X, A}$. In addition, if $\tau \in [\min(0, \tau_{D, 0}^\star), \max(0, \tau_{D, 0}^\star)]$,
	\[ g_{D, \tau} = \mathop{\argmin}_{f \in \mathcal F} \bigg\{ R(f): |D(f)| \le \frac{\tau D(\tau) }{|\tau|} \bigg\} . \]
\end{lemma}
\begin{proof}
    If $\tau=0$, then the result follows because $g_{D, 0}$ is the unconstrained Bayes optimal classifier.
    
    If $\tau \neq 0$, take $\phi_0(x, a) =  \pi_{a,1} \mathrm{d}P_{a,1}/\mathrm{d}P_{a,0}(x) - \pi_{a,0}$ and $\phi_1(x, a)=s_{D, a} \mathrm{d}P_{a,1}/\mathrm{d}P_{a,0}(x) + b_{D,a}$ in Lemma \ref{lem:GNP}.
    
    Write  
    \[ g_{D, \tau}(x, a) = \indc\bigg\{ \phi_0(x, a) > |\tau| \frac{\tau \phi_1(x, a) }{|\tau|} \bigg\} . \]
    
    Define
    \begin{align*}
    	Acc(f) & = 1-R(f) =  \sum_{a \in \{0,1\}} \int_{\Xspace} f(x,a) \bigg\{  \pi_{a,1} \frac{\mathrm{d}P_{a,1}}{\mathrm{d}P_{a,0}}(x) - \pi_{a,0} \bigg\} \mathrm{d}P_{a,0}(x) + \prob(Y=0), \\
    	\tilde D_\tau(f) & = \sum_{a \in \{0,1\}} \int_{\Xspace} f(x,a) \frac{\tau}{|\tau|} \phi_1(x, a) dP_{a, 0}(x).
    \end{align*}
    Let 
    \begin{align*}
    	& \mathcal F_{\tau, =} = \bigg\{ f \in \mathcal F: \tilde D_\tau(f) = \frac{\tau D(\tau)}{|\tau|} \bigg\}; ~~   \mathcal F_{\tau, |\cdot|, \le} = \bigg\{ f \in \mathcal F: |\tilde D_\tau(f)| \le \frac{\tau D(\tau)}{|\tau|} \bigg\}; \\
    	&	 \mathcal F_{\tau, \le} = \bigg\{ f \in \mathcal F: \tilde D_\tau(f) \le \frac{\tau D(\tau)}{|\tau|} \bigg\}.  
    \end{align*}
    
    Since $|\tau| \ge 0$, by Lemma \ref{lem:GNP},
    \[ g_{D, \tau} \in \mathop{\arg\max}_{f \in \mathcal F_{\tau, \le}} Acc(f). \]
    Moreover, since $\mathrm{d}P_{a,1}/\mathrm{d}P_{a,0}(x)$ is a continuous random variable given $A=a \in \{0,1\}$ and $Y=y \in \{0, 1\}$, we have $\prob_{X|A=a, Y=y} \big(  \pi_{a,1} \mathrm{d}P_{a,1}/\mathrm{d}P_{a,0}(x) - \pi_{a,0} = \tau \{s_{D, a} \mathrm{d}P_{a,1}/\mathrm{d}P_{a,0}(x) + b_{D,a} \}  \big) = 0$. Thus, for all $f^\prime \in \mathop{\arg\max}_{f \in \mathcal F_{\tau, \le}} Acc(f)$, $f^\prime = g_{D, \tau}$ almost surely with respect to $P_{X, A}$.
    
    By Lemma \ref{lem:GNP}, we have $g_{D, \tau} \in \mathop{\arg\max}_{f \in \mathcal F_{\tau, =}} Acc(f)$. 
    By result (i) of Proposition \ref{prop:D_and_R}, if $\tau_{D, 0}^\star\ge0$, then we have $D(\tau)\ge 0$ for $\tau \in [0, \tau_{D, 0}^\star]$. If $\tau_{D, 0}^\star \le 0$, then $D(\tau) \le 0$ for $\tau \in [\tau_{D, 0}^\star, 0]$. Therefore, when $\tau \in [\min(0, \tau_{D, 0}^\star), \max(0, \tau_{D, 0}^\star)]$, we have $\tau D(\tau) \ge 0$.
    Consequently,  $g_{D, \tau} \in  \mathcal F_{\tau, =}  \subseteq \mathcal F_{\tau, |\cdot|, \le} \subseteq \mathcal F_{\tau, \le}$, we have 
    \[ \max_{f \in \mathcal F_{\tau, \le}} Acc(f) = Acc(g_{D, \tau}) = \max_{f \in \mathcal F_{\tau, =}} Acc(f) \le \max_{f \in \mathcal F_{\tau, |\cdot|, \le }} Acc(f) \le \max_{f \in \mathcal F_{\tau, \le }} Acc(f) .  \]
    Thus, we conclude that
    \[ g_{D, \tau} = \mathop{\arg\max}_{f \in \mathcal F_{\tau, |\cdot|, \le }} Acc(f) = \mathop{\argmin}_{f \in \mathcal F} \bigg\{ R(f): |D(f)| \le \frac{\tau D(\tau)}{|\tau|} \bigg\}. \]
\end{proof}

\section[]{Proof of \Cref{thm_fair_guarantee}} \label{appendix_fair_guarantee}
\begin{proof}[Proof of \Cref{thm_fair_guarantee}]
    Conditioning on the training data $\widetilde{\mathcal{D}}$, by the Dvoretzky--Kiefer--Wolfowitz inequality \citep{dvoretzky1956asymptotic,massart1990tight}, we have, for any $a \in \{0,1\}$, that
    \begin{align*}
        &\mathbb{P}\Bigg[\sup_{\tau}\Big|\int_{\Xspace} \widehat{g}_{D, \tau}(x,a) \Big\{s_{D, a} \frac{\mathrm{d}\widehat{P}_{a,1}}{\mathrm{d}\widehat{P}_{a,0}}(x) + b_{D,a}\Big\}\; \mathrm{d}\widehat{P}_{a,0}(x)\\
        &\hspace{3cm} -\int_{\Xspace} \widehat{g}_{D, \tau} \Big\{s_{D, a} \frac{\mathrm{d}P_{a,1}}{\mathrm{d}P_{a,0}}(x) + b_{D,a}\Big\} \;\mathrm{d}P_{a,0}(x)\Big| \geq \epsilon\Bigg] \lesssim \exp\{-(n_{a,1}\wedge n_{a,0})\epsilon^2\},
    \end{align*}
    where $\widehat{g}_{D, \tau}$ is given in \Cref{alg:plug-in}. Thus, by taking $\epsilon \asymp \sqrt{\log(1/\eta)/(n_{a,1}\wedge n_{a,0})}$ and applying a union bound argument over $a \in \{0,1\}$ and the event in \Cref{l_n_ay}, the theorem holds by taking another expectation with respect to $\widetilde{\mathcal{D}}$.
\end{proof}

\section{Proofs for excess risk control} \label{appendix_excess_risk}
To simplify the notation, we write $\widehat{\tau} =\widehat{\tau}_{D, \delta}$ and $\tau^\star = \tau_{D, \delta}^\star$ in this section. For $a \in \{0,1\}$, denote
\begin{equation}\label{eq_T}
    T_a(X) = \sum_{j=1}^\infty \Big\{\frac{(\zeta_{a,j}-\theta_{a,0,j})(\theta_{a,1,j}-\theta_{a,0,j})}{\lambda_{a,j}} - \frac{(\theta_{a,1,j}-\theta_{a,0,j})^2}{2\lambda_{a,j}}\Big\} - \log\Big\{\frac{\pi_{a,0}}{\pi_{a,1} - \tau^\star(2a-1)}\Big\},
\end{equation}
and 
\[\widehat{T}_a(X) = \sum_{j=1}^\infty \Big\{\frac{(\widehat{\zeta}_{a,j}-\widehat{\theta}_{a,0,j})(\widehat{\theta}_{a,1,j}-\widehat{\theta}_{a,0,j})}{\widehat{\lambda}_{a,j}} - \frac{(\widehat{\theta}_{a,1,j}-\widehat{\theta}_{a,0,j})^2}{2\widehat{\lambda}_{a,j}}\Big\} - \log\Big\{\frac{\widehat{\pi}_{a,0}}{\widehat{\pi}_{a,1} - \widehat{\tau}(2a-1)}\Big\}.\] 
We further let
\begin{equation} \label{eq_H}
    H_a(X) = \sum_{j=1}^\infty \Big\{\frac{(\zeta_{a,j}-\theta_{a,0,j})(\theta_{a,1,j}-\theta_{a,0,j})}{\lambda_{a,j}} - \frac{(\theta_{a,1,j}-\theta_{a,0,j})^2}{2\lambda_{a,j}}\Big\},
\end{equation}
and 
\begin{equation}\label{eq_H_hat}
    \widehat{H}_a(X) = \sum_{j=1}^J \Big\{\frac{(\widehat{\zeta}_{a,j}-\widehat{\theta}_{a,0,j})(\widehat{\theta}_{a,1,j}-\widehat{\theta}_{a,0,j})}{\widehat{\lambda}_{a,j}} - \frac{(\widehat{\theta}_{a,1,j}-\widehat{\theta}_{a,0,j})^2}{2\widehat{\lambda}_{a,j}}\Big\}.
\end{equation}
With the above notation, we can rewrite $T_a(X)$ and $\widehat{T}_a(X)$ as
\[T_a(X)= H_a(X) -\log\Big\{\frac{\pi_{a,0}}{\pi_{a,1} - \tau^\star(2a-1)}\Big\},\quad \text{and} \;\; \widehat{T}_a(X) = \widehat{H}_a(X)- \log\Big\{\frac{\widehat{\pi}_{a,0}}{\widehat{\pi}_{a,1} - \widehat{\tau}(2a-1)}\Big\}.\]

\subsection[]{Proof of \Cref{thm_fairness_general}}
\begin{proof}[Proof of \Cref{thm_fairness_general}]
    \Cref{thm_fairness_general} is a general version of \Cref{thm_fairness}. Most of the proof follows from a similar argument to the one used in the proof of \Cref{thm_fairness}. We only include the difference here. 

    \noindent \textbf{Upper bound on $|\widehat{\tau} - \tau^\star|$.} Consider the following event, $\mathcal{E}_{D} = \{ \sup_{\tau \in \mathbb{R}}|\widehat{D}(\tau) - D(\tau)| \le \epsilon_{D}\}$. Then, condition on $\mathcal{E}_{D}$ happening, by the argument in the proof of \Cref{l_tau_diff_const}, we have that with probability at least $1 - \eta$ that
    \begin{align*}
        C_D|\widehat{\tau} - \tau^\star|^{\frac{1}{\gamma}} \lesssim \epsilon_D \indc\{\tau^\star \neq 0\}.
    \end{align*}
    Thus, it holds that $|\widehat{\tau} - \tau^\star| \lesssim  \epsilon^\gamma_D \indc\{\tau^\star_D \neq 0\}$.

    \noindent \textbf{Upper bound on $d_E(\widehat{f}_{D, \delta}, f_{D, \delta}^\star)$.} The proof follows from a similar argument leading to \eqref{pf_thm_fairness_eq1} and it suffices to verify $\mathbb{P}\{\mathcal{E}_{\tau}\} \wedge \mathbb{P}(\mathcal{E}_{T_0} \cap  \mathcal{E}_{T_1}) \geq 1-\eta$, with $\mathcal{E}_{\tau}, \mathcal{E}_{T_0}$ and $\mathcal{E}_{T_1}$ defined in the proof of \Cref{thm_fairness}. To control $\mathcal{E}_\tau$, since by assumption, it holds that $\pi_{a,1} - \tau^\star s_{D,a} \geq c$ and $\pi_{a,0}+\tau^\star b_{D,a} \geq c$, we have 
    \begin{align*}
        \widehat{\pi}_{a,1} - \widehat{\tau}s_{D,a} \geq \pi_{a,1} -\epsilon_\pi - \tau^\star s_{D,a} - |s_{D,a}|\epsilon_\tau \geq \frac{c}{2},
    \end{align*}
    where the first inequality follows from the fact that $\tau^\star$ and $\widehat{\tau}$ share the same sign, and the last inequality follows from the fact that $|s_{D,a}| \asymp 1$ and $\epsilon_\pi, \epsilon_\tau \ll 1$. Similarly, we can verify that $\widehat{\pi}_{a,0}+\widehat{\tau} b_{D,a}>c/2$. To control $\mathcal{E}_{T_a}$ for $a \in \{0,1\}$, by a similar argument as the one in \Cref{l_T_diff_const}, pick $\epsilon_{T_a} \lesssim \epsilon_\eta +\epsilon_\pi + \epsilon_\tau$, then by a union bound argument, we have that $\mathbb{P}\{\mathcal{E}_{\tau}\}\wedge \mathbb{P}\{\mathcal{E}_{T_0} \cap  \mathcal{E}_{T_1}\} \geq 1-\eta$. \eqref{pf_thm_fairness_eq1} thus leads to $d_E(\widehat{f}_{D, \delta}, f_{D, \delta}^\star) \lesssim (\epsilon_\eta +\epsilon_\pi + \epsilon_\tau)^2$.

    \noindent \textbf{Upper bound on $|R(\widehat{f}_{D, \delta}) -R(f_{D, \delta}^\star)|$.} This follows directly from the fact that 
    \begin{align*}
    |R(f) -R(f_{D, \delta}^\star)| &= d_E(f, f_{D, \delta}^\star)  + |\tau^\star \{D(f_{D, \delta}^\star) - D(f)\}|\\
    & \lesssim d_E(f, f_{D, \delta}^\star) + |\tau^\star|\sqrt{\frac{\log(1/\eta)}{n}}.
\end{align*}
\end{proof}

\subsection[]{Proof of \Cref{thm_fairness}} \label{appendix_proof_thm_fairness}
\begin{proof}[Proof of \Cref{thm_fairness}] \label{pf_thm_fairness}
    For any classifier $f: \Xspace \times \{0,1\} \rightarrow [0,1]$, by \Cref{prop:bilinear_coef}, the fairness-aware excess risk under DO is defined as
        \begin{align*}
    	d^{DO}_E(f, f_{D, \delta}^\star) = \sum_{a \in \{0,1\}}  \int_\Xspace\big\{ f(x, a) - f_{D, \delta}^\star(x, a)\big \} \Big[\pi_{a, 0} + \big\{\tau^\star (2a-1)-\pi_{a,1}\big\}\frac{\mathrm{d}P_{a, 1}}{\mathrm{d}P_{a,0}}(x)  \Big] \mathrm{d}P_{a, 0}(x).
    \end{align*}
    With the notation in \eqref{eq_T}, we have that 
    \[\frac{\mathrm{d}P_{a,1}}{\mathrm{d}P_{a,0}} = \frac{\pi_{a,0}}{\pi_{a,1} - \tau^\star(2a-1)}\exp\big\{T_a(X)\big\}.\]
    In addition, consider the following event $ \mathcal{E}_{\widehat{\tau}} = \{\widehat{\tau} \in (-\widehat{\pi}_{0,1}, \widehat{\pi}_{1,1})\}$. When $\mathcal{E}_{\widehat{\tau}}$ holds, we can then control the fairness-aware excess risk $d^{DO}_E(\widehat{f},f^\star)$ by
    \begin{align*}
        &d^{DO}_E(\widehat{f},f^\star)\\
        =\;& \int_{\widehat{T}_1(x) \geq 0} \pi_{1,0}+\big(\tau^\star-\pi_{1,1}\big)\frac{\mathrm{d}P_{1, 1}}{\mathrm{d}P_{1,0}}(x)\;\mathrm{d}P_{1,0}(x) \\
        &-\int_{T_1(x) \geq 0} \pi_{1,0}+\big(\tau^\star-\pi_{1,1}\big)\frac{\mathrm{d}P_{1, 1}}{\mathrm{d}P_{1,0}}(x)\;\mathrm{d}P_{1,0}(x) \\
        &+ \int_{\widehat{T}_0(x) \geq 0} \pi_{0,0}+\big(-\tau^\star-\pi_{0,1}\big)\frac{\mathrm{d}P_{0, 1}}{\mathrm{d}P_{0,0}}(x)\;\mathrm{d}P_{0,0}(x) \\
        &-\int_{T_0(x) \geq 0} \pi_{0,0}+\big(-\tau^\star-\pi_{0,1}\big)\frac{\mathrm{d}P_{0, 1}}{\mathrm{d}P_{0,0}}(x)\;\mathrm{d}P_{0,0}(x) \\
        \leq \;&  \int_{\widehat{T}_1(x) \geq 0, T_1(x)<0} \pi_{1,0}+\big(\tau^\star-\pi_{1,1}\big)\frac{\mathrm{d}P_{1, 1}}{\mathrm{d}P_{1,0}}(x)\;\mathrm{d}P_{1,0}(x)\\
        & +  \int_{\widehat{T}_0(x) \geq 0, T_0(X) <0} \pi_{0,0}+\big(-\tau^\star-\pi_{0,1}\big)\frac{\mathrm{d}P_{0, 1}}{\mathrm{d}P_{0,0}}(x)\;\mathrm{d}P_{0,0}(x)\\
        =\;& \;\pi_{1,0}\cdot\int_{T_1(x)- \widehat{T}_1(x) \leq T_1(x)<0}1- e^{T_1(x)}\;\mathrm{d}P_{1,0}(x) \\
        &+  \pi_{0,0}\cdot\int_{T_0(x)- \widehat{T}_0(x) \leq T_0(x)<0} 1- e^{T_0(x)}\;\mathrm{d}P_{0,0}(x)\\
        =\;&\pi_{1,0}\cdot \mathbb{E}_{P_{1,0}}\Big[\big\{1- e^{T_1(X)}\big\}\indc\Big\{T_1(X) - \widehat{T}_1(X)\leq T_1(X)<0\Big\}\Big]\\
        & +\pi_{0,0} \cdot \mathbb{E}_{P_{0,0}}\Big[\big\{1- e^{T_0(X)}\big\}\indc\Big\{T_0(X) - \widehat{T}_0(X)\leq T_0(X)<0\Big\}\Big]\\
        = \;& \pi_{1,0}\cdot \mathbb{E}_{P_{1,0}}\Bigg[\big\{1- e^{T_1(X)}\big\} \indc\Big\{T_1(X) - \widehat{T}_1(X)\leq T_1(X)<0\Big\}\indc\Big\{|T_1(X) -\widehat{T}_1(X)| \leq \epsilon_{T_1}\Big\}\Bigg]\\
        & +\pi_{1,0}\cdot\mathbb{E}_{P_{1,0}}\Bigg[\big\{1- e^{T_1(X)}\big\} \indc\Big\{T_1(X) - \widehat{T}_1(X)\leq T_1(X)<0\Big\} \indc\Big\{|T_1(X) -\widehat{T}_1(X)| > \epsilon_{T_1}\Big\}\Bigg]\\
        &+ \pi_{0,0} \cdot\mathbb{E}_{P_{0,0}}\Bigg[\big\{1- e^{T_0(X)}\big\} \indc\Big\{T_0(X) - \widehat{T}_0(X)\leq T_0(X)<0\Big\}\indc\Big\{|T_0(X) -\widehat{T}_0(X)| \leq \epsilon_{T_0}\Big\}\Bigg]\\
        & + \pi_{0,0} \cdot\mathbb{E}_{P_{0,0}}\Bigg[\big\{1- e^{T_0(X)}\big\} \indc\Big\{T_0(X) - \widehat{T}_0(X)\leq T_0(X)<0\Big\}\indc\Big\{|T_0(X) -\widehat{T}_0(X)| >\epsilon_{T_0}\Big\}\Bigg].
    \end{align*}
    For any $a\in \{0,1\}$, write $X_a \sim \mathcal{GP}(\mu_{a,0},K_a)$. We further denote 
    \begin{equation} \label{pf_thm_fairness_eq2}
        \mathcal{E}_{T_a} = \Big\{|\widehat{T}_a(X_a)-T_a(X_a)| \leq \epsilon_{T}\Big\}.
    \end{equation}
    Consequently, we can further upper bound the fairness-aware excess risk by 
    \begin{align} \notag
        d^{DO}_E(\widehat{f},f^\star)  \leq \;&\pi_{1,0}\cdot\epsilon_{T}\cdot\mathbb{E}_{P_{1,0}}\Big[\indc\big\{-\epsilon_{T}\leq T_1(X) < 0\big\}\Big] + \pi_{0,0}\cdot \epsilon_{T} \cdot \mathbb{E}_{P_{0,0}}\Big[\indc\big\{-\epsilon_{T}\leq T_0(X) < 0\big\}\Big]\\ \notag
        & + \pi_{0,0}\mathbb{P}(\mathcal{E}_{T_0}^c) +\pi_{1,0}\mathbb{P}(\mathcal{E}_{T_1}^c) \\ \notag
        \leq \;& \epsilon_{T}^2\cdot\Big(\sup_{-\epsilon_{T}\leq t < 0} f_{0,T_1}(t)\Big)+ \epsilon_{T}^2\cdot\Big(\sup_{-\epsilon_{T}\leq t < 0} f_{0,T_0}(t)\Big) + \mathbb{P}(\mathcal{E}_{T_0}^c) + \mathbb{P}(\mathcal{E}_{T_1}^c)\\ \notag
        \lesssim \;& \epsilon_{T}^2\cdot\frac{1}{\|\mu_{1,1} - \mu_{1,0}\|_{K_1}}\exp\Big\{-\frac{\|\mu_{1,1} - \mu_{1,0}\|_{K_1}^4 \vee 1}{\|\mu_{1,1} - \mu_{1,0}\|_{K_1}^2}\Big\}\\ \notag
        & + \epsilon_{T}^2\cdot\frac{1}{\|\mu_{0,1} - \mu_{0,0}\|_{K_0}}\exp\Big\{-\frac{\|\mu_{0,1} - \mu_{0,0}\|_{K_0}^4 \vee 1}{\|\mu_{0,1} - \mu_{0,0}\|_{K_0}^2}\Big\} + \mathbb{P}(\mathcal{E}_{T_0}^c) + \mathbb{P}(\mathcal{E}_{T_1}^c)\\\label{pf_thm_fairness_eq1}
        \lesssim \;& \epsilon_{T}^2 + \mathbb{P}(\mathcal{E}_{T_0}^c \cup \mathcal{E}_{T_1}^c).
    \end{align}
    where for $a\in \{0,1\}$, $f_{0,T_a}$ denotes the density for $T_a$ given $Y=0$, the first inequality holds as $1- \exp\{T_a(x)\} \leq -T_a(X) \leq \widehat{T}_a(X) - T_a(X) \leq \epsilon_{T_a}$ and the last inequality follows from \Cref{l_tau_boundary} and the fact that
    \[T_{a}(X_a)|Y=0 \sim N\Big(-\frac{\|\mu_{a,1} - \mu_{a,0}\|_{K_a}^2}{2}-\log\Big\{\frac{\pi_{a,0}}{\pi_{a,1} - \tau^\star(2a-1)}\Big\}, \; \|\mu_{a,1} - \mu_{a,0}\|_{K_a}^2\Big).\] 
 
    In the rest of the proof, it suffices to control the probability for $\mathcal{E}_{\tau}, \mathcal{E}_{T_1}, \mathcal{E}_{T_0}$ happening. By Lemmas  \ref{l_tau_diff_const}, \ref{l_tau_boundary}, \ref{l_emperical_prob} and a union bound argument, it holds with probability at least $1-\eta/4$ that 
    \begin{align*}
        \widehat{\pi}_{1,1} - \widehat{\tau} \geq \pi_{1,1} - \tau^\star - \Big\{\sqrt{\frac{\log(1/\eta)}{\widetilde{n}}}+\epsilon_H +\sqrt{\frac{\log(1/\eta)}{n}}\Big\} \geq c-\frac{c}{2} = \frac{c}{2}.
    \end{align*}
    Analogously, we can also show that with probability at least $1-\eta/4$ that, when $\widehat{\tau} < 0$,  
    \begin{align*}
        \widehat{\pi}_{0,1} + \widehat{\tau} \geq \pi_{0,1} + \tau^\star -\Big\{\sqrt{\frac{\log(1/\eta)}{\widetilde{n}}}+\epsilon_H +\sqrt{\frac{\log(1/\eta)}{n}}\Big\} \geq \frac{c}{2}.
    \end{align*}
    Moreover, taking $\epsilon_T = \epsilon_H + \sqrt{\log(1/\eta)/n} \cdot \indc\{|DO(0)| > \delta - \epsilon_H -\sqrt{\log(1/\eta)/n}\}$ in \eqref{pf_thm_fairness_eq2}, by \Cref{l_T_diff_const} and an additional union bound argument, it holds that $\mathbb{P}\{\mathcal{E}_{\tau}\}\wedge \mathbb{P}(\mathcal{E}_{T_0} \cap  \mathcal{E}_{T_1})\geq 1-\eta$. Thus, conditioning on $\mathcal{E}_{\tau}$ happening, we have with probability at least $1-\eta$ that
    \begin{align*}
        d^{DO}_E(\widehat{f},f^\star) &\lesssim \Bigg\{\epsilon_H +\sqrt{\frac{\log(1/\eta)}{n}}\cdot \indc\Big\{|DO(0)| > \delta - \epsilon_H -\sqrt{\frac{\log(1/\eta)}{n}}\Big\}\Bigg\}^2+ \eta\\
        &\lesssim \Bigg\{\epsilon_H +\sqrt{\frac{\log(1/\eta)}{n}}\cdot \indc\Big\{|DO(0)| > \delta - \epsilon_H -\sqrt{\frac{\log(1/\eta)}{n}}\Big\}\Bigg\}^2,
    \end{align*}
    whenever $\eta \in (0,n^{-1/2} \wedge \widetilde{n}^{(\alpha-2\beta+1)/(2\beta-\alpha)})$. 
\end{proof}

\subsection{Control of \texorpdfstring{$|\widehat{T}_a(X) - T_a(X)|$}{}}
\begin{lemma}\label{l_T_diff_const}
    Suppose the training and calibration data $\mathcal{D} \cup \widetilde{\mathcal{D}}$ are generated under Assumptions \ref{a_class_prob} and~\ref{a_data}. Then for any $a \in \{0,1\}$, $J\in \mathbb{N}_+$ such that 
        \[J \gtrsim \log^2(J) \;\; \text{and} \;\; J^{2\alpha+2}\log^2(J)\log(\widetilde{n}/\widetilde{\eta}) \lesssim \widetilde{n},\]
    and any constant $\eta \in (0,1/2)$, it holds that 
    \begin{align*}
        \mathbb{P}_{a,0}\Bigg[|\widehat{T}_a(X) - T_a(X)| \lesssim \epsilon_H +\sqrt{\frac{\log(1/\eta)}{n}} \cdot \indc\Big\{|DO(0)| > \delta - \epsilon_H -\sqrt{\frac{\log(1/\eta)}{n}}\Big\}\Bigg] \geq 1-\eta,
    \end{align*}
    where $\epsilon_H$ is defined in \eqref{l_H_diff_eq1}.
\end{lemma}

\begin{proof}
    It follows from the triangle inequality that
    \begin{align*}
        |\widehat{T}_a(X) - T_a(X)| \leq \;& |\widehat{H}_a(X) - H_a(X)| + |\log(\widehat{\pi}_{a,0}) - \log(\pi_{a,0})| \\
        &+ |\log\{\widehat{\pi}_{a,1}-\widehat{\tau}(2a-1)\} - \log\{\pi_{a,1} - \tau^\star(2a-1)\}|.
    \end{align*}
    For $X\sim\mathcal{GP}(\mu_{a,0},K_a)$, consider the following events,
    \[\mathcal{E}_\pi = \Big\{|\widehat{\pi}_{a,y}-\pi_{a,y}| \lesssim \sqrt{\frac{\log(1/\eta)}{\widetilde{n}}},\quad a,y \in \{0,1\}\Big\},\]
    \[\mathcal{E}_H = \Big\{|\widehat{H}_a(X) - H_a(X)| \lesssim \epsilon_H, \quad a \in \{0,1\}\;\Big\},\]
    and
    \[\mathcal{E}_\tau = \Bigg\{|\widehat{\tau} - \tau^\star| \lesssim  \Big\{\epsilon_H +\sqrt{\frac{\log(1/\eta)}{n}}\Big\}\cdot \indc\Big\{|DO(0)| > \delta - \epsilon_H -\sqrt{\frac{\log(1/\eta)}{n}}\Big\}\Bigg\}.\]
    By Lemmas \ref{l_H_diff_const}, \ref{l_tau_diff_const}, \ref{l_emperical_prob} and a union bound argument, we have that for any $a\in \{0,1\}$, $\mathbb{P}_{a,0}(\mathcal{E}_\pi \cap\mathcal{E}_H \cap \mathcal{E}_\tau) \geq 1-\eta$. The rest of the proof is constructed conditioning on $\mathcal{E}_\pi \cap\mathcal{E}_H \cap \mathcal{E}_\tau$ happening. Therefore, we have that 
    \begin{align*}
        |\widehat{T}_a(X) - T_a(X)| &\lesssim |\widehat{H}_a(X) - H_a(X)| + |\widehat{\pi}_{a,0} -\pi_{a,0}| + |\widehat{\pi}_{a,1} -\pi_{a,1}| + |\widehat{\tau} -\tau^\star|\\
        & \lesssim \epsilon_H +\sqrt{\frac{\log(1/\eta)}{n}} \cdot \indc\Big\{|DO(0)| > \delta - \epsilon_H -\sqrt{\frac{\log(1/\eta)}{n}}\Big\}.
    \end{align*}
\end{proof}

\subsection{Control of \texorpdfstring{$|\widehat{H}_a(X) - H_a(X)|$}{}}
\begin{lemma}\label{l_H_diff_const}
    Under the same condition of \Cref{l_T_diff_const}, for any $a, y \in \{0,1\}$ and small constant $\eta \in (0,1/2)$, it holds that 
    \begin{align*}
        \mathbb{P}_{a,y}\Big\{|\widehat{H}_a(X) - H_a(X)| \lesssim \epsilon_H \Big\}  \geq 1-\eta,
    \end{align*}
    where 
    \begin{align} \label{l_H_diff_eq1}
        \epsilon_H = \begin{cases}\vspace{0.5em}
               \sqrt{\frac{J^{\alpha-2\beta+4}\log(\widetilde{n}/\eta)\log(1/\eta)}{\widetilde{n}}} + \sqrt{J^{\alpha-2\beta+1}\log(1/\eta)} \quad & \text{when}\;\; \frac{\alpha+1}{2}< \beta\leq \frac{\alpha+2}{2},\\ \vspace{0.5em}
               \sqrt{\frac{J^2\log(\widetilde{n}/\eta)\log(1/\eta)}{\widetilde{n}}}+ \sqrt{J^{\alpha-2\beta+1}\log(1/\eta)}, & \text{when}\;\; \frac{\alpha+2}{2}< \beta\leq \frac{\alpha+3}{2},\\
                \sqrt{\frac{J\log(\widetilde{n}/\eta)\log(1/\eta)}{\widetilde{n}}}+ \sqrt{J^{\alpha-2\beta+1}\log(1/\eta)},& \text{when}\;\;  \beta > \frac{\alpha+3}{2}.
            \end{cases}
    \end{align}
\end{lemma}

\begin{proof}
    Note that 
    \begin{align*}
        &|\widehat{H}_a(X) - H_a(X)|\\
        \leq\;& \Big|\sum_{j=1}^J \frac{(\zeta_{a,j}-\theta_{a,0,j})(\theta_{a,1,j}-\theta_{a,0,j})}{\lambda_{a,j}} -  \sum_{j=1}^J \frac{(\widehat{\zeta}_j-\widehat{\theta}_{a,0,j})(\widehat{\theta}_{a,1,j}-\widehat{\theta}_{a,0,j})}{\widehat{\lambda}_{a,j}} \Big|\\
        & + \Big|\sum_{j=J+1}^\infty \frac{(\zeta_{a,j}-\theta_{a,0,j})(\theta_{a,1,j}-\theta_{a,0,j})}{\lambda_{a,j}}\Big|\\
        & +\Big|\sum_{j=1}^J\frac{(\theta_{a,1,j}-\theta_{a,0,j})^2}{2\lambda_{a,j}} - \frac{(\widehat{\theta}_{a,1,j}-\widehat{\theta}_{a,0,j})^2}{2\widehat{\lambda}_{a,j}}\Big|\\
        & +\Big|\sum_{j=J+1}^\infty\frac{(\theta_{a,1,j}-\theta_{a,0,j})^2}{2\lambda_{a,j}}\Big|.
    \end{align*}
    In the case when $X \sim \mathcal{GP}(\mu_{a,0},K_a)$, $a \in \{0,1\}$, the lemma thus follows by applying a union bound argument to the results in Lemmas \ref{l_I_const}, \ref{l_II} and \ref{l_rkhs_approximation_const}, together with the fact that under Assumptions \ref{a_data}\ref{a_data_cov} and \ref{a_data}\ref{a_data_mean_decay}, we have 
    \[\Big|\sum_{j=J+1}^\infty\frac{(\theta_{a,1,j}-\theta_{a,0,j})^2}{2\lambda_{a,j}}\Big| \lesssim J^{\alpha-2\beta+1}.\]
    The case when $X \sim \mathcal{GP}(\mu_{a,1},K_a)$, $a \in \{0,1\}$ can be justified similarly, hence is omitted here.
\end{proof}

\subsection{Control of \texorpdfstring{$|\widehat{\tau} - \tau^\star|$}{}}
\begin{lemma} \label{l_tau_diff_const}
    Under the same condition as \Cref{l_T_diff_const}, for any small constant $\eta \in (0,1/2)$, it holds with probability at least $1-\eta$ that 
    \begin{align*}
        |\widehat{\tau} - \tau^\star| \lesssim \Big\{\epsilon_H +\sqrt{\frac{\log(1/\eta)}{n}}\Big\}\cdot \indc\Big\{|DO(0)| > \delta - \epsilon_H -\sqrt{\frac{\log(1/\eta)}{n}}\Big\},
    \end{align*}
    where $\epsilon_H$ is explicitly defined in \eqref{l_H_diff_eq1}.
    
\end{lemma}

\begin{proof}
    Denote $\mathcal{E}_{DO} = \{ \sup_{\tau \in \mathbb{R}}|\widehat{DO}(\tau) - DO(\tau)| \le \epsilon_{DO}\}$. Taking 
    \[\epsilon_{DO} \asymp \epsilon_H +\sqrt{\frac{\log(1/\eta)}{n}},\]
    it holds from \Cref{l_DO_estimation_const} that $\mathbb{P}(\mathcal{E}_{DO}^c) \leq \eta$. The rest of the proof is divided into four cases.

    \noindent \textbf{Case 1: When $|DO(0)| \le \delta - \epsilon_{DO}$.} In this case, we have that $\tau^\star = 0$. Moreover, under the event $\mathcal{E}_{DO}$, $|\widehat{DO}(0)| \le \delta$, then $\widehat{\tau}=0$. Consequently, $|\widehat{\tau} - \tau^\star|=0$.

    \noindent \textbf{Case 2 : When $\tau^\star=0$ and $ \delta - \epsilon_{DO} < |DO(0)| \le \delta$.} In this case, we have that 
    \begin{align*}
	\prob(\widehat{\tau} - \tau^\star > \varepsilon) & = \prob(\widehat{\tau}> \varepsilon) \\
	& \le \prob(\widehat{\tau} > \varepsilon, \mathcal{E}_{DO}) + \prob(\mathcal{E}_{DO}^c) \\
	& \le \prob(\widehat{\tau} > \varepsilon, \widehat{DO}(0) > \delta, \mathcal{E}_{DO}) + \prob(\mathcal{E}_{DO}^c) \\
	& \le \prob \big( \widehat{DO}(\varepsilon) > \delta, \mathcal{E}_{DO} \big) + \prob(\mathcal{E}_{DO}^c) \\
	& \le \prob \big( \widehat{DO}(\varepsilon) - DO(\varepsilon) > \delta - DO(0) + c_1 \varepsilon, \mathcal{E}_{DO} \big) + \prob(\mathcal{E}_{DO}^c) \\
	& \le \prob(\mathcal{E}_{DO}^c),
\end{align*}
where the second inequality follows as $\widehat{\tau} >0$, the third inequality follows from \Cref{l_est_DO_decrease}, the fourth inequality follows as $DO(\epsilon) \le DO(0) - c_1\varepsilon$ by \Cref{l_D_lower_upper_const} and the last inequality follows by taking  $\varepsilon > \epsilon_{DO}/c_1$. Analogously,
\begin{align*}
	\prob(\widehat{\tau} < \tau^\star -\varepsilon) & = \prob(\widehat{\tau} < -\varepsilon)  \\
	& \le  \prob(\widehat{\tau} < -\varepsilon, \mathcal{E}_{DO}) + \prob(\mathcal{E}_{DO}^c) \\
	& = \prob(\widehat{\tau} < -\varepsilon, \widehat{DO}(0) < -\delta, \mathcal{E}_{DO}) + \prob(\mathcal{E}_{DO}^c) \\
	& \le  \prob\big( \widehat{DO}(-\varepsilon) < -\delta, \mathcal{E}_{DO} \big) + \prob(\mathcal{E}_{DO}^c) \\
	& \le \prob\big( \widehat{DO}(-\varepsilon) - DO(-\varepsilon) < -\delta - DO(0) - c_1\varepsilon, \mathcal{E}_{DO} \big) + \prob(\mathcal{E}_{DO}^c) \\
	& \le \prob(\mathcal{E}_{DO}^c),
\end{align*}
where the second inequality follows from \Cref{l_est_DO_decrease}, the third inequality follows from the fact that $DO(-\varepsilon) \ge DO(0) + c_1\varepsilon$ by \Cref{l_D_lower_upper_const} and the last inequality follows by taking  $\varepsilon > \epsilon_{DO}/c_1$.

\noindent \textbf{Case 3: When $DO(\tau^\star)=\delta$ and $\tau^\star >0$.} In this case, by \Cref{prop:D_and_R}, it holds that $DO(0)>\delta$ . Hence,
\begin{align*}
	\prob(\widehat{\tau} > \tau^\star + \varepsilon) & \le \prob\big( \widehat{\tau} > \tau^\star + \varepsilon, \mathcal{E}_{DO} \big) +\prob(\mathcal{E}_{DO}^c)  \\
	& = \prob\big( \widehat{\tau} > \tau^\star + \varepsilon, \widehat{DO}(0) > \delta, \mathcal{E}_{DO} \big)  +\prob(\mathcal{E}_{DO}^c) \\
	& \le \prob\big( \widehat{DO}(\tau^\star + \varepsilon) > \delta, \mathcal{E}_{DO} \big) + \prob(\mathcal{E}_{DO}^c) \\
	& = \prob\big( \widehat{DO}(\tau^\star + \varepsilon) - DO(\tau^\star+\varepsilon) > DO(\tau^\star) - DO(\tau^\star + \varepsilon) , \mathcal{E}_{DO} \big) + \prob(\mathcal{E}_{DO}^c) \\
	& \le \prob( \epsilon_{DO} > c_1 \varepsilon ) + \prob(\mathcal{E}_{DO}^c) =  \prob(\mathcal{E}_{DO}^c),
\end{align*}
where the first equality follows as $\prob( \widehat{\tau} > \tau^\star + \varepsilon, \widehat{DO}(0) \le \delta) = 0$, the second inequality follows from \Cref{l_est_DO_decrease}, the third inequality follows from \Cref{l_D_lower_upper_const} and the last equality follows by taking $\varepsilon > \epsilon_{DO}/c_1$. Similarly, by taking $\varepsilon > \epsilon_{DO}/c_1$, it also holds that
\begin{align*}
	& \quad \prob(\widehat{\tau} < \tau^\star - \varepsilon) \\
	& \le \prob(\widehat{\tau} < \tau^\star - \varepsilon, \mathcal{E}_{DO}) + \prob(\mathcal{E}_{DO}^c) \\
	& \le  \prob\big( \widehat{DO}(\tau^\star - \varepsilon) \le \delta,  \mathcal{E}_{DO} \big) + \prob(\mathcal{E}_{DO}^c) \\
	& = \prob\big(  DO(\tau^\star - \varepsilon) - \widehat{DO}(\tau^\star - \varepsilon) \ge DO(\tau^\star - \varepsilon) -DO(\tau^\star),  \mathcal{E}_{DO} \big) +  \prob(\mathcal{E}_{DO}^c) \\
	& \le  \prob( \epsilon_{DO} > c_1 \varepsilon )  + \prob(\mathcal{E}_{DO}^c)  \\
	& \le \prob(\mathcal{E}_{DO}^c).
\end{align*}

\noindent \textbf{Case 4: When $DO(\tau^\star)=-\delta$ and $\tau^\star <0$.} In this case, it holds that $DO(0)<-\delta$. The rest of the proof follows similarly to the proof for \textbf{Case 3}, we include them for completeness. 
For $\varepsilon > \epsilon_{DO}/c_1$,
\begin{align*}
	& \quad \prob(\widehat{\tau} > \tau^\star+\varepsilon) \\
	& \le \prob( \widehat{\tau} > \tau^\star+\varepsilon, \mathcal{E}_{DO} ) + \prob(\mathcal{E}_{DO}^c) \\
	& \le \prob( \widehat{DO}(\tau^\star + \varepsilon) \ge -\delta, \mathcal{E}_{DO} ) +  \prob(\mathcal{E}_{DO}^c)  \\
	& = \prob( \widehat{DO}(\tau^\star + \varepsilon) - DO(\tau^\star+\varepsilon) \ge DO(\tau^\star) - DO(\tau^\star+\varepsilon), \mathcal{E}_{DO} ) +  \prob(\mathcal{E}_{DO}^c) \\ 
	& \le  \prob(\mathcal{E}_{DO}^c). 
\end{align*}
Since $\prob(\widehat{\tau} < \tau^\star-\varepsilon,  \widehat{DO}(0) \ge -\delta )=0$, for $\varepsilon > \epsilon_{DO}/c_1$,
\begin{align*}
	\prob(\widehat{\tau} < \tau^\star-\varepsilon) & \le \prob(\widehat{\tau} < \tau^\star-\varepsilon, \mathcal{E}_{DO}) + \prob( \mathcal{E}_{DO}^c ) \\
	& =  \prob(\widehat{\tau} < \tau^\star-\varepsilon, \widehat{DO}(0) < -\delta, \mathcal{E}_{DO}) + \prob( \mathcal{E}_{DO}^c ) \\
	& \le \prob( \widehat{DO}(\tau^\star - \varepsilon) < -\delta,  \mathcal{E}_{DO}) + \prob( \mathcal{E}_{DO}^c ) \\
	& = \prob( \widehat{DO}(\tau^\star - \varepsilon) - DO(\tau^\star - \varepsilon) < -\delta - DO(\tau^\star - \varepsilon),  \mathcal{E}_{DO}) + \prob( \mathcal{E}_{DO}^c ) \\
	& \le \prob( \mathcal{E}_{DO}^c ).
\end{align*}

The lemma thus follows by combining results from all cases together.

\end{proof}

\subsection{Control of \texorpdfstring{$|\widehat{DO}(\tau) - DO(\tau)|$}{}}
\begin{lemma}\label{l_DO_estimation_const}
    Under the same condition of \Cref{l_T_diff_const}, for any small constant $\eta \in (0,1/2)$, it holds with probability at least $1-\eta$ that 
    \begin{align*}
        \sup_{\tau \in \mathbb{R}} |\widehat{DO}(\tau) - DO(\tau)| \lesssim \epsilon_H +\sqrt{\frac{\log(1/\eta)}{n}},
    \end{align*}
    where $\epsilon_H$ is explicitly defined in \eqref{l_H_diff_eq1}.
\end{lemma}
\begin{proof}
    Note that, by triangle inequality, 
    \begin{align} \notag
        &\sup_{\tau \in \mathbb{R}} |\widehat{DO}(\tau) - DO(\tau)|\\ \notag
        \leq\;&  \sup_{\tau \in \mathbb{R}}|\widehat{DO}(\tau) - \mathbb{E}\{\widehat{DO}(\tau) | \widetilde{\data}\}| +  \sup_{\tau \in \mathbb{R}} |\mathbb{E}\{\widehat{DO}(\tau) | \widetilde{\data}\} - DO(\tau)|\\\label{l_DO_estimation_const_eq1}
        =\;& (I)+(II).
    \end{align}
    \noindent \textbf{Step 1: Upper bound on $(I)$.} To control $(I)$, by the Dvoretzky--Kiefer--Wolfowitz inequality \citep{dvoretzky1956asymptotic,massart1990tight}, we have that 
    \begin{align*}
        &\prob\bigg[ \sup_{\tau \in \mathbb{R}} \bigg|   \frac{1}{n_{1,1}}\sum_{i=1}^{n_{1,1}}\indc\big\{ (\widehat{\pi}_{1,1} - \tau) \widehat \eta_1(X_{1,1,i}) > \widehat{\pi}_{1,0}\big\} - \prob\big\{ (\widehat{\pi}_{1,1} - \tau) \widehat \eta_1(X_{1,1,i}) > \widehat{\pi}_{1,0} | \widetilde{\data}  \big\}   \bigg| \ge \epsilon \bigg| \widetilde{\data}  \bigg] \\
        &\hspace{11cm} \le 2\exp\{-2n_{1,1}\epsilon^2\}, \\
        &\prob\bigg[ \sup_{\tau \in \mathbb{R}} \bigg|  \frac{1}{n_{0, 1}}\sum_{i=1}^{n_{0,1}} \indc\big\{ (\widehat{\pi}_{0,1} + \tau) \widehat \eta_0(X_{0,1,i}) > \widehat{\pi}_{0,0} \big\}- \prob\big\{ (\widehat{\pi}_{0,1} + \tau) \widehat \eta_0(X_{0,1,i}) > \widehat{\pi}_{0,0}   | \widetilde{\data}  \big\}   \bigg| \ge \epsilon \bigg| \widetilde{\data}  \bigg] \\
        & \hspace{11cm}\le 2\exp\{-2n_{0,1}\epsilon^2\}. 
    \end{align*}
    Thus, by a union bound argument, it holds with probability at least $1-\eta/2$ that 
    \begin{align}\label{eq:error_D_cond_train}
        \sup_{\tau \in \mathbb{R}} |\widehat{DO}(\tau) -  \mathbb{E}\{ \widehat{DO}(\tau) | \widetilde{\data} \}|  \lesssim \sqrt{\frac{\log(1/\eta)}{n_{1,1}}} + \sqrt{\frac{\log(1/\eta)}{n_{0,1}}} \lesssim \sqrt{\frac{\log(1/\eta)}{n}},
    \end{align}
    where the last inequality follows by further conditioning on the event in \Cref{l_n_ay} and a union bound argument.

    \noindent \textbf{Step 2: Upper bound on $(II)$.} Note that 
    \begin{align*}
        &\mathbb{E}\{ 	\widehat{DO}(\tau) | \widetilde{\data} \} \\
    =\;&\prob_{1,1}\big\{ (\widehat{\pi}_{1,1} - \tau) \widehat \eta_1(X) > \widehat{\pi}_{1,0} | \widetilde{\data} \big\}- \prob_{0,1}\big\{ (\widehat{\pi}_{0,1} + \tau) \widehat \eta_0(X) > \widehat{\pi}_{0,0}   | \widetilde{\data}  \big\}\\
    =\;& \left\{
    \begin{array}{cc}
    \prob_{1,1}\big[ \log\{\widehat{\eta}_1(X)\} >\log( \frac{\widehat{\pi}_{1,0}}{\widehat{\pi}_{1,1}-\tau})  \big| \widetilde{\data}\big] -  \prob_{0,1}\big[ \log\{\widehat{\eta}_0(X)\} >\log( \frac{\widehat{\pi}_{0,0}}{\widehat{\pi}_{0,1}+\tau})  \big| \widetilde{\data}\big]	 , & -\widehat{\pi}_{0,1} < \tau < \widehat{\pi}_{1,1},  \\
        \prob_{1,1}\big[ \log\{\widehat{\eta}_1(X)\} >\log( \frac{\widehat{\pi}_{1,0}}{\widehat{\pi}_{1,1}-\tau})  \big| \widetilde{\data} \big] , & \tau \le -\widehat{\pi}_{0,1}, \\
        -  \prob_{0,1}\big[ \log\{\widehat{\eta}_0(X)\} >\log( \frac{\widehat{\pi}_{0,0}}{\widehat{\pi}_{0,1}+\tau}) \big| \widetilde{\data} \big], & \tau \ge \widehat{\pi}_{1,1}. 
    \end{array}
     \right. \\
    \end{align*}
    We denote $\Pi = (-\min(\pi_{0,1},\widehat{\pi}_{0,1}), \min(\pi_{1,1}, \widehat{\pi}_{1,1}))$. In the rest of the proof, to control $(II)$, we will consider various cases.

    \noindent \textbf{Step 2-Case 1: $\tau \in \Pi$.} Note that for $\tau \in \Pi$, 
    \begin{align*}
        &\mathbb{E}\{ 	\widehat{DO}(\tau) | \widetilde{\data} \} \\
        =\;& \prob\big\{ (\widehat{\pi}_{1,1} - \tau) \widehat \eta_1(X_{1,1,i}) > \widehat{\pi}_{1,0} | \widetilde{\data} \big\}- \prob\big\{ (\widehat{\pi}_{0,1} + \tau) \widehat \eta_0(X_{0,1,i})> \widehat{\pi}_{0,0}   | \widetilde{\data}  \big\} \\
        =\;&\prob_{1,1}\big\{ (\widehat{\pi}_{1,1} - \tau) \widehat \eta_1(X) > \widehat{\pi}_{1,0} | \widetilde{\data} \big\}- \prob_{0,1}\big\{ (\widehat{\pi}_{0,1} + \tau) \widehat \eta_0(X) > \widehat{\pi}_{0,0}   | \widetilde{\data}  \big\}\\
        =\;& \int_{\widehat{H}_1(x) - \log\{\widehat{\pi}_{1,0}/(\widehat{\pi}_{1,1} - \tau)\} >0} \mathrm{d}P_{1,1}(x) - \int_{\widehat{H}_0(x) - \log\{\widehat{\pi}_{0,0}/(\widehat{\pi}_{0,1} + \tau)\} >0} \mathrm{d}P_{0,1}(x),
    \end{align*}
    with $\widehat{H}_a$ defined in \eqref{eq_H_hat}. Similarly, we can rewrite $DO(\tau)$ as
    \begin{align*}
        DO(\tau) = \int_{H_1(x) - \log\{\pi_{1,0}/(\pi_{1,1} - \tau)\} >0} \mathrm{d}P_{1,1}(x) - \int_{H_0(x) - \log\{\pi_{0,0}/(\pi_{0,1} + \tau)\} >0} \mathrm{d}P_{0,1}(x),
    \end{align*}
    with $H_a$ defined in \eqref{eq_H}. 
    For $a\in \{0,1\}$ and $X \sim \mathcal{GP}(\mu_{a,1},K_a)$, consider the following event,
    \[\mathcal{E}_{H_a} = \Big\{\sup_{\tau\in \Pi}\Big|\widehat{H}_a(X) -\log\Big(\frac{\widehat{\pi}_{a,0}}{\widehat{\pi}_{a,1} + (1-2a)\tau}\Big) - H_a(X)+\log\Big(\frac{\pi_{a,0}}{\pi_{a,1} + (1-2a)\tau}\Big)\Big| \leq \epsilon_H,\;\; a\in\{0,1\}\Big\}.\]
    Note that 
    \begin{align*}
        &\sup_{\tau\in \Pi}\Big|\widehat{H}_a(X) -\log\Big(\frac{\widehat{\pi}_{a,0}}{\widehat{\pi}_{a,1} + (1-2a)\tau}\Big) - H_a(X)+\log\Big(\frac{\pi_{a,0}}{\pi_{a,1} + (1-2a)\tau}\Big)\Big| \\
        =\;&\sup_{\tau\in \Pi}\Big\{|\widehat{H}_a(X) - H_a(X)|+ |\log(\widehat{\pi}_{a,0}) - \log(\pi_{a,0})| + |\log(\widehat{\pi}_{a,1} + (1-2a)\tau) - \log(\pi_{a,1} + (1-2a)\tau)|\Big\}\\
        \lesssim \;& |\widehat{H}_a(X) - H_a(X)| + |\widehat{\pi}_{a,0} -\pi_{a,0}| + |\widehat{\pi}_{a,1}-\pi_{a,1}|.
    \end{align*}
    Thus, picking $\epsilon_H$ as the one in \Cref{l_H_diff_eq1}, by Lemmas \ref{l_H_diff_const}, \ref{l_emperical_prob} and a union bound argument, we have that $\mathbb{P}_{0,1}(\mathcal{E}_{H_0})+\mathbb{P}_{1,1}(\mathcal{E}_{H_1})\geq 1-\eta/2$. 
    
    For any $\tau \in \Pi$, it holds that 
    \begin{align} \notag
        &\mathbb{E}\{ \widehat{DO}(\tau) | \widetilde{\data} \} -DO(\tau)\\ \notag
        =\;&\int_{\widehat{H}_1(x) - \log\{\widehat{\pi}_{1,0}/(\widehat{\pi}_{1,1} - \tau)\} >0} \mathrm{d}P_{1,1}(x) -\int_{H_1(x) - \log\{\pi_{1,0}/(\pi_{1,1} - \tau)\} >0} \mathrm{d}P_{1,1}(x)\\ \notag
        & +\int_{H_0(x) - \log\{\pi_{0,0}/(\pi_{0,1} + \tau)\} >0} \mathrm{d}P_{0,1}(x) - \int_{\widehat{H}_0(x) - \log\{\widehat{\pi}_{0,0}/(\widehat{\pi}_{0,1} + \tau)\} >0} \mathrm{d}P_{0,1}(x)\\ \notag
        \leq\;& \int_{0>H_1(x) - \log\{\pi_{1,0}/(\pi_{1,1} - \tau)\}> H_1(x) - \log\{\pi_{1,0}/(\pi_{1,1} - \tau)\}-\widehat{H}_1(x) + \log\{\widehat{\pi}_{1,0}/(\widehat{\pi}_{1,1} - \tau)\}} \mathrm{d}P_{1,1}(x)\\ \notag
        & +\int_{0<H_0(x) - \log\{\pi_{0,0}/(\pi_{0,1} + \tau)\} <H_0(x) - \log\{\pi_{0,0}/(\pi_{0,1} + \tau)\} -\widehat{H}_0(x) + \log\{\widehat{\pi}_{0,0}/(\widehat{\pi}_{0,1} + \tau)} \mathrm{d}P_{0,1}(x)\\ \label{l_DO_estimation_const_eq2}
        \leq\;& \epsilon_{H}\Bigg[\frac{1}{\|\mu_{1,1} - \mu_{1,0}\|_{K_1}}\exp\Big\{-\frac{\|\mu_{1,1} - \mu_{1,0}\|_{K_1}^4 \vee 1}{\|\mu_{1,1} - \mu_{1,0}\|_{K_1}^2}\Big\}+ \frac{1}{\|\mu_{0,1} - \mu_{0,0}\|_{K_0}}\exp\Big\{-\frac{\|\mu_{0,1} - \mu_{0,0}\|_{K_0}^4 \vee 1}{\|\mu_{0,1} - \mu_{0,0}\|_{K_0}^2}\Big\}\Bigg]\\ \label{l_DO_estimation_const_eq3}
        \asymp \;&\epsilon_{H},
    \end{align}
    where the last inequality follows from a similar argument as the one leading to \eqref{pf_thm_fairness_eq1}. Using a similar argument, we can also achieve a same (up to constant order) upper bound for $DO(\tau)-\mathbb{E}\{ \widehat{DO}(\tau) | \widetilde{\data}\}$ as the one in \eqref{l_DO_estimation_const_eq3}.

    \noindent \textbf{Step 2-Case 2: When $-\pi_{0,1} < \tau \leq -\widehat{\pi}_{0,1}$}. Note that in this case, by standard calculation, we have that 
    \begin{align*}
        DO(\tau) = \int_{H_1(x) - \log\{\pi_{1,0}/(\pi_{1,1} - \tau)\} >0} \mathrm{d}P_{1,1}(x) - \int_{H_0(x) - \log\{\pi_{0,0}/(\pi_{0,1} + \tau)\} >0} \mathrm{d}P_{0,1}(x),
    \end{align*}
    and
    \begin{align*}
        \mathbb{E}\{ 	\widehat{DO}(\tau) | \widetilde{\data} \} 
        =\int_{\widehat{H}_1(x) - \log\{\widehat{\pi}_{1,0}/(\widehat{\pi}_{1,1} - \tau)\} >0} \mathrm{d}P_{1,1}(x).
    \end{align*}
    Consequently,
    \begin{align*}
        &\sup_{-\pi_{0,1} < \tau \leq -\widehat{\pi}_{0,1}}\big|\mathbb{E}\{ 	\widehat{DO}(\tau) | \widetilde{\data} \} - DO(\tau)\big|\\
        \leq \;& \sup_{-\pi_{0,1} < \tau \leq -\widehat{\pi}_{0,1}}
        \Big|\int_{\widehat{H}_1(x) - \log\{\widehat{\pi}_{1,0}/(\widehat{\pi}_{1,1} - \tau)\} >0} \mathrm{d}P_{1,1}(x)-\int_{H_1(x) - \log\{\pi_{1,0}/(\pi_{1,1} - \tau)\} >0} \mathrm{d}P_{1,1}(x)\Big|\\
        &+ \sup_{-\pi_{0,1} < \tau \leq -\widehat{\pi}_{0,1}}\Big|\int_{H_0(x) - \log\{\pi_{0,0}/(\pi_{0,1} + \tau)\} >0} \mathrm{d}P_{0,1}(x)\Big|\\
        =\;& (A) + (B).
    \end{align*}
    Note that $(A)$ can be controlled using a similar argument as the one used in \textbf{Step 1} and we only present the upper bound on $(B)$ in the sequel. Consider the event $\mathcal{E}_\pi = \{|\pi_{0,1} - \widehat{\pi}_{0,1}| \lesssim \sqrt{\log(1/\eta)/\widetilde{n}}\}$. By \Cref{l_emperical_prob}, it holds that $\mathbb{P}(\mathcal{E}_\pi) \geq 1-\eta/2$. The rest of the proof is constructed conditioning on the event $\mathcal{E}_\pi$ happening. Therefore, we have that 
    \begin{align} \notag
        (B) &= \sup_{-\pi_{0,1} < \tau \leq -\widehat{\pi}_{0,1}} \mathbb{P}_{0,1}\Big\{H_0(X) >\log\Big(\frac{\pi_{0,0}}{\pi_{0,1}+\tau}\Big)\Big\}\\ \notag
        &\leq \mathbb{P}_{0,1}\Big\{H_0(X) >\log\Big(\frac{\pi_{0,0}}{\pi_{0,1}-\widehat{\pi}_{0,1}}\Big)\Big\}\\ \notag
        & \leq \mathbb{P}_{0,1}\Big\{H_0(X) >\log\Big(\sqrt{\frac{C_p^2\widetilde{n}}{\log(1/\eta)}}\Big)\Big\} \\ \notag
        &= \mathbb{P}_{0,1}\Big\{H_0(X) -\mathbb{E}_{P_{0,1}}\{H_0(X)\} >\log\Big(\sqrt{\frac{C_p^2\widetilde{n}}{\log(1/\eta)}}\Big)-\frac{\|\mu_{0,1} -\mu_{0,0}\|_{K_0}^2}{2}\Big\}\\ \label{l_DO_estimation_const_eq4}
        & \lesssim \exp\Big[-\frac{1}{\|\mu_{0,1} -\mu_{0,0}\|_{K_0}^2}\log^2\Big(\sqrt{\frac{\widetilde{n}}{\log(1/\eta)}}\Big)\Big] \leq \eqref{l_DO_estimation_const_eq4} \lesssim \epsilon_H,
    \end{align}
    where the third inequality follows from the standard property of Gaussian random variables. 

    The proof for the other cases follows a similar argument as the one used in \textbf{Step2-Case 2}, hence it is omitted here. The lemma thus follows by substituting the results in \eqref{eq:error_D_cond_train} and \eqref{l_DO_estimation_const_eq3} (or \eqref{l_DO_estimation_const_eq4}) into \eqref{l_DO_estimation_const_eq1} and a union bound argument.
\end{proof}

\subsection{Behaviour of \texorpdfstring{$DO$}{} around \texorpdfstring{$\tau^\star$}{} }
\begin{lemma} \label{l_D_lower_upper_const}
    Recall $DO$ given in \Cref{def_disparity_measure}. Under \Cref{a_data}, for any $\epsilon$ in a small neighbourhood of $0$, there exists absolute constants $c_1,c_2 >0$ such that
    \begin{align*}
	c_1\epsilon \le DO(\tau^\star) - DO(\tau^\star+\epsilon) \le c_2 \epsilon, ~~ c_1\epsilon \le DO(\tau^\star - \epsilon) - DO(\tau^\star) \le c_2 \epsilon.
\end{align*}
\end{lemma}

\begin{proof}
    Note that for any $\tau \in \mathbb{R}$, 
\begin{align*}
    &DO(\tau) \\
    =\;& \prob\big\{ \pi_{1,1}\eta_1(X) - \pi_{1,0} > \tau \eta_1(X) | A=1, Y=1  \big\} - \prob\big\{ \pi_{0,1} \eta_0(X) - \pi_{0,0} > - \tau \eta_0(X) | A=0, Y=1 \big\} \\
    =\;& \left\{
    \begin{array}{cc}
    \prob_{1,1}\big[ \log\{\eta_1(X)\} >\log( \frac{\pi_{1,0}}{\pi_{1,1}-\tau})  \big] -  \prob_{0,1}\big[ \log\{\eta_0(X)\} >\log( \frac{\pi_{0,0}}{\pi_{0,1}+\tau}) \big]	 , & -\pi_{0,1} < \tau < \pi_{1,1},  \\
        \prob_{1,1}\big[ \log\{\eta_1(X)\} >\log( \frac{\pi_{1,0}}{\pi_{1,1}-\tau})  \big] , & \tau \le -\pi_{0,1}, \\
        -  \prob_{0,1}\big[ \log\{\eta_0(X)\} >\log( \frac{\pi_{0,0}}{\pi_{0,1}+\tau}) \big], & \tau \ge \pi_{1,1}. 
    \end{array}
     \right. \\
     =\;&  \left\{
     \begin{array}{cc}
         \Phi\bigg[  \frac{\|\mu_{1,1} -\mu_{1,0}\|_{K_1}}{2} - \frac{\log\big\{ \pi_{1,0}/(\pi_{1,1} - \tau) \big\}}{ \|\mu_{1,1} -\mu_{1,0}\|_{K_1} }  \bigg] - \Phi\bigg[  \frac{\|\mu_{0,1} -\mu_{0,0}\|_{K_0}}{2} - \frac{\log\big\{ \pi_{0,0}/(\pi_{0,1} + \tau) \big\}}{ \|\mu_{0,1} -\mu_{0,0}\|_{K_0} }  \bigg], &  -\pi_{0,1} < \tau < \pi_{1,1}, \\
          \Phi\bigg[  \frac{\|\mu_{1,1} -\mu_{1,0}\|_{K_1}}{2} - \frac{\log\big\{ \pi_{1,0}/(\pi_{1,1} - \tau) \big\}}{ \|\mu_{1,1} -\mu_{1,0}\|_{K_1} }  \bigg], &   \tau \le -\pi_{0,1}, \\
          - \Phi\bigg[  \frac{\|\mu_{0,1} -\mu_{0,0}\|_{K_0}}{2} - \frac{\log\big\{ \pi_{0,0}/(\pi_{0,1} + \tau) \big\}}{ \|\mu_{0,1} -\mu_{0,0}\|_{K_0} }  \bigg], &   \tau \ge \pi_{1,1}.
        \end{array}
    \right.
         \end{align*}
     By standard calculation, the derivative of $DO$ with respect to $\tau$ is 
     \begin{align}  \notag
        DO'(\tau) =\;& -\frac{1}{\|\mu_{1,1} -\mu_{1,0}\|_{K_1}(\pi_{1,1}-\tau)}\cdot \phi\bigg[\frac{\|\mu_{1,1} -\mu_{1,0}\|_{K_1}}{2} - \frac{\log\big\{ \pi_{1,0}/(\pi_{1,1} - \tau) \big\}}{ \|\mu_{1,1} -\mu_{1,0}\|_{K_1} } \bigg]\\\label{l_D_lower_upper_const_eq1}
        & -\frac{1}{\|\mu_{0,1} -\mu_{0,0}\|_{K_0}(\pi_{0,1}+\tau)}\cdot\phi\bigg[  \frac{\|\mu_{0,1} -\mu_{0,0}\|_{K_0}}{2} - \frac{\log\big\{ \pi_{0,0}/(\pi_{0,1} + \tau) \big\}}{ \|\mu_{0,1} -\mu_{0,0}\|_{K_0} }  \bigg],
    \end{align}
    where $\phi$ is the standard normal pdf. For any $\epsilon$ in a small right neighborhood of 0, by \Cref{l_tau_boundary}, it holds that  $-\pi_{0,1}+c < \tau^\star-\epsilon, \tau^\star, \tau^\star+\epsilon < \pi_{1,1}-c$, where $c>0$ is a small constant.  Thus, by \eqref{l_D_lower_upper_const_eq1}, we have that $-c_2 \le DO'(\tau^\star-\epsilon), DO'(\tau^\star), DO'(\tau^\star +\epsilon) \le -c_1 $ for some universal constants $c_1, c_2>0$. Hence, by the mean value theorem, we have that 
    \begin{align*}
        c_2 \epsilon \ge DO(\tau^\star) - DO(\tau^\star+\epsilon) \ge c_1 \epsilon, ~~ c_2\epsilon \ge DO(\tau^\star - \epsilon) - DO(\tau^\star) \ge c_1 \epsilon.
    \end{align*}
\end{proof}

\subsection{Auxiliary results}
In this subsection, without loss of generality, we assume that $X \sim \mathcal{GP}(\mu_{a,0}, K)$ for $a \in \{0,1\}$. The case when $X \sim \mathcal{GP}(\mu_{a,1},K)$ can be justified similarly.

\begin{lemma} \label{l_tau_boundary}
    Recall the definition of $\tau^\star$ in \eqref{eq:taustar} with $D$ chosen as DO. Under Assumptions \ref{a_class_prob} and \ref{a_data}, there exists a small absolute constant $0<c<\min\{\pi_{0,1},\pi_{1,1}\}$ such that 
    \[-\pi_{0,1} + c \leq \tau^\star \leq \pi_{1,1} -c .\]
\end{lemma}

\begin{proof}[Proof of \Cref{l_tau_boundary}]
    By the definition of $DO$, we have that
   \begin{align*}
	&DO(\tau) \\
        =\;& \prob\big\{ \pi_{1,1}\eta_1(X) - \pi_{1,0} > \tau \eta_1(X) | A=1, Y=1  \big\} - \prob\big\{ \pi_{0,1} \eta_0(X) - \pi_{0,0} > - \tau \eta_0(X) | A=0, Y=1 \big\} \\
	=\;& \left\{
    	\begin{array}{cc}
    	\prob_{1,1}\big[ \log\{\eta_1(X)\} >\log( \frac{\pi_{1,0}}{\pi_{1,1}-\tau})  \big] -  \prob_{0,1}\big[ \log\{\eta_0(X)\} >\log( \frac{\pi_{0,0}}{\pi_{0,1}+\tau}) \big]	 , & -\pi_{0,1} < \tau < \pi_{1,1},  \\
    		\prob_{1,1}\big[ \log\{\eta_1(X)\} >\log( \frac{\pi_{1,0}}{\pi_{1,1}-\tau})  \big] , & \tau \le -\pi_{0,1}, \\
    		-  \prob_{0,1}\big[ \log\{\eta_0(X)\} >\log( \frac{\pi_{0,0}}{\pi_{0,1}+\tau}) \big], & \tau \ge \pi_{1,1}. 
    	\end{array}
    	 \right. \\
	  =\;&  \left\{
    	 \begin{array}{cc}
    	 	 \Phi\bigg[  \frac{\|\mu_{1,1} -\mu_{1,0}\|_{K_1}}{2} - \frac{\log\big\{ \pi_{1,0}/(\pi_{1,1} - \tau) \big\}}{ \|\mu_{1,1} -\mu_{1,0}\|_{K_1} }  \bigg] - \Phi\bigg[  \frac{\|\mu_{0,1} -\mu_{0,0}\|_{K_0}}{2} - \frac{\log\big\{ \pi_{0,0}/(\pi_{0,1} + \tau) \big\}}{ \|\mu_{0,1} -\mu_{0,0}\|_{K_0} }  \bigg], &  -\pi_{0,1} < \tau < \pi_{1,1}, \\
    	 	  \Phi\bigg[  \frac{\|\mu_{1,1} -\mu_{1,0}\|_{K_1}}{2} - \frac{\log\big\{ \pi_{1,0}/(\pi_{1,1} - \tau) \big\}}{ \|\mu_{1,1} -\mu_{1,0}\|_{K_1} }  \bigg], &   \tau \le -\pi_{0,1}, \\
    	 	  - \Phi\bigg[  \frac{\|\mu_{0,1} -\mu_{0,0}\|_{K_0}}{2} - \frac{\log\big\{ \pi_{0,0}/(\pi_{0,1} + \tau) \big\}}{ \|\mu_{0,1} -\mu_{0,0}\|_{K_0} }  \bigg], &   \tau \ge \pi_{1,1}.
    	 	\end{array}
     	\right.
     \end{align*}
     
    In the case when $|DO(0)| \leq \delta$, the lemma holds trivially as $\tau^\star = 0$, and it is automatically bounded away from the boundary by a small constant. We will divide the following proof by conditioning on different assumptions. 

    Let $\tau_0$ denote the value such that $DO(\tau_0) = 0$. Then in the case when $DO(0) >\delta$, it holds that $\tau^\star >0$ and $DO(\tau^\star) = \delta$. By \Cref{prop:D_and_R} and the fact that $DO(\pi_{1,1}) <0$, it must be the case that $0< \tau^\star < \tau_0 < \pi_{1,1}$. Consequently, we write $\tau_0 = \pi_{1,1}-\epsilon$ where $0 < \epsilon_1 < \pi_{1,1}$, and in the rest of the proof it suffices to prove that $\epsilon_1 \asymp 1$. By the fact that $DO(\tau_0) =0$ and $\Phi$ is a strictly increasing function, it holds that
    \begin{align*}
          \frac{\|\mu_{1,1} -\mu_{1,0}\|_{K_1}}{2} - \frac{\log\big\{ \pi_{1,0}/(\pi_{1,1} - \tau_0) \big\}}{ \|\mu_{1,1} -\mu_{1,0}\|_{K_1} }=  \frac{\|\mu_{0,1} -\mu_{0,0}\|_{K_0}}{2} - \frac{\log\big\{ \pi_{0,0}/(\pi_{0,1} + \tau_0) \big\}}{ \|\mu_{0,1} -\mu_{0,0}\|_{K_0}}, 
    \end{align*}
    hence implies
    \begin{equation} \label{l_tau_boundary_eq1}
        \frac{\|\mu_{1,1} -\mu_{1,0}\|_{K_1}}{2} - \frac{\log\big(\pi_{1,0}/\epsilon_1 \big)}{ \|\mu_{1,1} -\mu_{1,0}\|_{K_1} }=  \frac{\|\mu_{0,1} -\mu_{0,0}\|_{K_0}}{2} - \frac{\log\big\{ \pi_{0,0}/(\pi_{0,1} + \pi_{1,1}-\epsilon_1) \big\}}{ \|\mu_{0,1} -\mu_{0,0}\|_{K_0}}.
    \end{equation}
    Suppose that $\epsilon_1$ is a function of $n$ and $\widetilde{n}$ such that $\epsilon_1 \prec 1$. However, in this case $\log\big(\pi_{1,0}/\epsilon \big) \gg \log\big\{ \pi_{0,0}/(\pi_{0,1} + \pi_{1,1}-\epsilon) \big\} \asymp 1$, hence \eqref{l_tau_boundary_eq1} never holds when $\|\mu_{0,1} -\mu_{0,0}\|_{K_0}\asymp \|\mu_{1,1} -\mu_{1,0}\|_{K_1} \asymp 1$. Thus, we achieve a contradiction, and it must be the case that $\epsilon_1 =c_1$ where $c_1>0$ is an absolute constant. Therefore, we have that $\tau^\star >0$ and $\pi_{1,1} - \tau^\star > \pi_{1,1} - \tau_0 =\epsilon_1 =c_1$.

    Similarly, when $DO(0) < -\delta$, we have $\tau^\star <0$ and $DO(\tau^\star) = -\delta$. With the same notation as above, by \Cref{prop:D_and_R}, it must be the case that $-\pi_{0,1}<-\pi_{0,1} +\epsilon_2 =\tau_0 < \tau^\star<0$, where $0< \epsilon_2 < -\pi_{0,1}$. Using a similar argument as above, it holds that
    \begin{equation}\label{l_tau_boundary_eq2}
        \frac{\|\mu_{1,1} -\mu_{1,0}\|_{K_1}}{2} - \frac{\log\big\{ \pi_{1,0}/(\pi_{1,1}+\pi_{0,1}-\epsilon_2) \big\}}{ \|\mu_{1,1} -\mu_{1,0}\|_{K_1} }=  \frac{\|\mu_{0,1} -\mu_{0,0}\|_{K_0}}{2} - \frac{\log\big( \pi_{0,0}/\epsilon_2 \big)}{ \|\mu_{0,1} -\mu_{0,0}\|_{K_0}}.
    \end{equation}
    Suppose that $\epsilon_2$ is a function of $n$ and $\widetilde{n}$ such that $\epsilon_2 \prec 1$. However, in this case, we have that $\log\big( \pi_{0,0}/\epsilon_2 \big) \gg \log\big\{ \pi_{1,0}/(\pi_{1,1}+\pi_{0,1}-\epsilon_2)\}$. Thus \eqref{l_tau_boundary_eq2} never holds and we reach a contradiction. Hence, we conclude that $\epsilon_2 = c_2$ and $\tau^\star + \pi_{0,1} \geq \tau_0 + \pi_{0,1} =\epsilon_2 = c_2$. The lemma thus follows by combining the results in the three cases above.
\end{proof}

\begin{lemma} \label{l_est_DO_decrease}
    Recall the empirical estimator for Disparity of Opportunity 
    \[\widehat{DO}(\tau) = \frac{1}{n_{1,1}}\sum_{i=1}^{n_{1,1}}\indc\big\{ (\widehat{\pi}_{1,1} - \tau) \widehat \eta_1(X_{1,1,i}) > \widehat{\pi}_{1,0}\big\}  - \frac{1}{n_{0, 1}}\sum_{i=1}^{n_{0,1}} \indc\big\{ (\widehat{\pi}_{0,1} + \tau) \widehat \eta_0(X_{0,1,i}) > \widehat{\pi}_{0,0} \big\}.\]
    It holds that $\widehat{DO}$ is a non-increasing function. Moreover, $\widehat{DO}(\widehat{\pi}_{1,1}) \le 0$ and $\widehat{DO}(-\widehat{\pi}_{0,1}) \ge 0$.
\end{lemma}
\begin{proof}[Proof of \Cref{l_est_DO_decrease}]
    Recall that $s_{DO, a} = 2a-1$ and $b_{DO, a}=0$. Then $\{ \tau\in\real:  \widehat{\pi}_{a,0}+\tau b_{D,a}\ge 0,  \widehat{\pi}_{a,1} - \tau s_{D, a} \ge 0, \forall a \in \{0, 1\} \} = \{ \tau \in \real: -\widehat{\pi}_{0,1} \le \tau \le \widehat{\pi}_{1,1} \}$.
    The estimated disparity function is
    \begin{align*}
    	\widehat{DO}(\tau) & = \frac{1}{n_{1,1}}\sum_{i=1}^{n_{1,1}}\indc\big\{ (\widehat{\pi}_{1,1} - \tau) \widehat \eta_1(X_{1,1,i}) > \widehat{\pi}_{1,0}\big\}  - \frac{1}{n_{0, 1}}\sum_{i=1}^{n_{0,1}} \indc\big\{ (\widehat{\pi}_{0,1} + \tau) \widehat \eta_0(X_{0,1,i}) > \widehat{\pi}_{0,0} \big\} \\
    	& = \left\{ 
    	\begin{array}{cc}
    		 \frac{1}{n_{1,1}}\sum_{i=1}^{n_{1,1}}\indc\bigg\{ \widehat \eta_1(X_{1,1,i}) > \frac{\widehat{\pi}_{1,0}}{\widehat{\pi}_{1,1} - \tau}\bigg\}  - \frac{1}{n_{0, 1}}\sum_{i=1}^{n_{0,1}} \indc\bigg\{ \widehat \eta_0(X_{0,1,i}) > \frac{\widehat{\pi}_{0,0}}{\widehat{\pi}_{0,1} + \tau} \bigg\}, &  \tau \in (-\widehat{\pi}_{0,1}, \widehat{\pi}_{1,1}), \\
    		  \frac{1}{n_{1,1}}\sum_{i=1}^{n_{1,1}}\indc\bigg\{ \widehat \eta_1(X_{1,1,i}) > \frac{\widehat{\pi}_{1,0}}{\widehat{\pi}_{1,1} - \tau}\bigg\}, & \tau \in (-\infty, -\widehat{\pi}_{0,1}], \\
    		  - \frac{1}{n_{0, 1}}\sum_{i=1}^{n_{0,1}} \indc\bigg\{ \widehat \eta_0(X_{0,1,i}) > \frac{\widehat{\pi}_{0,0}}{\widehat{\pi}_{0,1} + \tau} \bigg\}, & \tau \in [\widehat{\pi}_{1,1}, +\infty).
    	\end{array}
    	\right.
    \end{align*}
    For $-\widehat{\pi}_{0,1} \le \tau_1 < \tau_2 \le \widehat{\pi}_{1,1}$, 
    \begin{align*}
    	& \quad \widehat{DO}(\tau_1) - \widehat{DO}(\tau_2) \\
    	& = \frac{1}{n_{1,1}}\sum_{i=1}^{n_{1,1}}\indc\bigg\{ \frac{\widehat{\pi}_{1,0}}{\widehat{\pi}_{1,1} - \tau_1} < \widehat{\eta}_1(X_{1,1,i}) \le \frac{\widehat{\pi}_{1,0}}{\widehat{\pi}_{1,1} - \tau_2} \bigg\} + \frac{1}{n_{0, 1}}\sum_{i=1}^{n_{0,1}} \indc\bigg\{ \frac{\widehat{\pi}_{0,0}}{\widehat{\pi}_{0,1} + \tau_2} < \widehat{\eta}_0(X_{0,1,i}) \le \frac{\widehat{\pi}_{0,0}}{\widehat{\pi}_{0,1} + \tau_1}   \bigg\}  \\
    	& \ge 0.
    \end{align*}
    For $\tau_1 < \tau_2 < -\widehat{\pi}_{0,1}$,
    \begin{align*}
    	& \quad \widehat{DO}(\tau_1) - \widehat{DO}(\tau_2) \\
    	& = \frac{1}{n_{1,1}}\sum_{i=1}^{n_{1,1}}\indc\bigg\{ \frac{\widehat{\pi}_{1,0}}{\widehat{\pi}_{1,1} - \tau_1} < \widehat{\eta}_1(X_{1,1,i}) \le \frac{\widehat{\pi}_{1,0}}{\widehat{\pi}_{1,1} - \tau_2} \bigg\} \ge 0.
    \end{align*}
    For $\widehat{\pi}_{1,1} < \tau_1 < \tau_2$,
    \begin{align*}
    	& \quad \widehat{DO}(\tau_1) - \widehat{DO}(\tau_2) \\
    	& = \frac{1}{n_{0, 1}}\sum_{i=1}^{n_{0,1}} \indc\bigg\{ \frac{\widehat{\pi}_{0,0}}{\widehat{\pi}_{0,1} + \tau_2} < \widehat{\eta}_0(X_{0,1,i}) \le \frac{\widehat{\pi}_{0,0}}{\widehat{\pi}_{0,1} + \tau_1}   \bigg\} \ge 0.
    \end{align*}
    Hence, $\widehat{DO}(\tau)$ is a non-increasing function. Therefore, if $|\widehat{DO}(0)| \le \delta$, then $\widehat{\tau} =0$. If $\widehat{DO}(0)>\delta$, then $\widehat{\tau} \ge 0$. If $\widehat{DO}(0)<-\delta$, $\widehat{\tau} \le 0$.
    Moreover, a straight calculation leads to
    \[ \widehat{DO}(\widehat{\pi}_{1,1}) \le 0 \quad \text{and} \quad  \widehat{DO}(-\widehat{\pi}_{0,1}) \ge 0.\]
\end{proof}

\begin{lemma}\label{l_I_const}
    Conditioning on the training data $\widetilde{\mathcal{D}}$, under the same condition as the one in \Cref{l_T_diff_const}, for any small constant $\eta \in (0,1/2)$, it holds with probability at least $1-\eta$ that
    \begin{align*}
            &\Big|\sum_{j=1}^J \frac{(\widehat{\zeta}_{a,j}-\widehat{\theta}_{a,0,j})(\widehat{\theta}_{a,1,j}-\widehat{\theta}_{a,0,j})}{\widehat{\lambda}_{a,j}} - \sum_{j=1}^J \frac{(\zeta_{a,j}-\theta_{a,0,j})(\theta_{a,1,j}-\theta_{a,0,j})}{\lambda_{a,j}} \Big|\\
            \lesssim \;&\begin{cases}\vspace{0.5em}
               \sqrt{\frac{J^{\alpha-2\beta+4}\log(\widetilde{n}/\eta)\log(1/\eta)}{\widetilde{n}}} \quad & \text{when}\;\; \frac{\alpha+1}{2}< \beta\leq \frac{\alpha+2}{2},\\ \vspace{0.5em}
               \sqrt{\frac{J^2\log(\widetilde{n}/\eta)\log(1/\eta)}{\widetilde{n}}}, & \text{when}\;\; \frac{\alpha+2}{2}< \beta\leq \frac{\alpha+3}{2},\\\vspace{0.5em}
                \sqrt{\frac{J\log(\widetilde{n}/\eta)\log(1/\eta)}{\widetilde{n}}},& \text{when}\;\;  \beta > \frac{\alpha+3}{2}.
            \end{cases}
        \end{align*}

\end{lemma}
\begin{proof}
    Note that
    \begin{align} \notag
        &\sum_{j=1}^J \frac{(\widehat{\zeta}_{a,j}-\widehat{\theta}_{a,0,j})(\widehat{\theta}_{a,1,j}-\widehat{\theta}_{a,0,j})}{\widehat{\lambda}_{a,j}}-\sum_{j=1}^\infty \frac{(\zeta_{a,j}-\theta_{a,0,j})(\theta_{a,1,j}-\theta_{a,0,j})}{\lambda_{a,j}} \\ \notag
        =\;& \sum_{j=1}^J \frac{(\widehat{\zeta}_{a,j}-\widehat{\theta}_{a,0,j})(\widehat{\theta}_{a,1,j}-\widehat{\theta}_{a,0,j})}{\widehat{\lambda}_{a,j}}-\sum_{j=1}^J \frac{(\zeta_{a,j}-\theta_{a,0,j})(\widehat{\theta}_{a,1,j}-\widehat{\theta}_{a,0,j})}{\widehat{\lambda}_{a,j}}\\ \notag
        &+ \sum_{j=1}^J \frac{(\zeta_{a,j}-\theta_{a,0,j})(\widehat{\theta}_{a,1,j}-\widehat{\theta}_{a,0,j})}{\widehat{\lambda}_{a,j}}-\sum_{j=1}^J \frac{(\zeta_{a,j}-\theta_{a,0,j})(\theta_{a,1,j}-\theta_{a,0,j})}{\lambda_{a,j}}\\ \notag
        =\;&  \sum_{j=1}^J \frac{(\widehat{\zeta}_{a,j}-\widehat{\theta}_{a,0,j}-\zeta_{a,j}+\theta_{a,0,j})(\widehat{\theta}_{a,1,j}-\widehat{\theta}_{a,0,j})}{\widehat{\lambda}_{a,j}}+ \sum_{j=1}^J(\zeta_{a,j}-\theta_{a,0,j})\Big(\frac{\widehat{\theta}_{a,1,j}-\widehat{\theta}_{a,0,j}}{\widehat{\lambda}_{a,j}} - \frac{\theta_{a,1,j}-\theta_{a,0,j}}{\lambda_{a,j}}\Big)\\ \label{l_I_const_eq1}
        =\;& (I) + (II).
    \end{align}
    When $X \sim \mathcal{GP}(\mu_{a,0},K)$, by standard properties of Gaussian process, it holds that $(\zeta_{a,1}, \ldots, \zeta_{a,j})^\top \sim N(0,\Lambda)$ where $\Lambda = \text{diag}(\lambda_1, \ldots, \lambda_{a,j})$ and for any $j, k \in [J]$ such that $j\neq k$, $\var(\zeta_{a,j}) = \lambda_{a,j}$,
    \begin{align*}
        \var\big(\widehat{\zeta}_{a,j}-\widehat{\theta}_{a,0,j}-\zeta_{a,j}+\theta_{a,0,j}\big)  &= \int\int K_a(s,t)\big\{\widehat{\phi}_{a,j}(s)-\phi_{a,j}(s)\big\}\big\{\widehat{\phi}_{a,j}(t)-\phi_{a,j}(t)\big\}\;\mathrm{d}s\;\mathrm{d}t,
    \end{align*}
    and
    \begin{align*}
        &\cov\big(\widehat{\zeta}_{a,j}-\widehat{\theta}_{a,0,j}-\zeta_{a,j}+\theta_{a,0,j}, \widehat{\zeta}_{a,k}-\widehat{\theta}_{a,0,k}-\zeta_k+\theta_{a,0,k}\big)\\
        =\;& \int\int K_a(s,t)\big\{\widehat{\phi}_{a,j}(s)-\phi_{a,j}(s)\big\}\big\{\widehat{\phi}_{a,k}(t)-\phi_{a,k}(t)\big\}\;\mathrm{d}s\;\mathrm{d}t.
    \end{align*}
    Therefore, it holds that
    \[\sum_{j=1}^J \frac{(\widehat{\zeta}_{a,j}-\widehat{\theta}_{a,0,j}-\zeta_{a,j}+\theta_{a,0,j})(\widehat{\theta}_{a,1,j}-\widehat{\theta}_{a,0,j})}{\widehat{\lambda}_{a,j}} \sim N\Big(Q_1, Q_2\Big),\]
    and
    \[\sum_{j=1}^J(\zeta_{a,j}-\theta_{a,0,j})\Big(\frac{\widehat{\theta}_{a,1,j}-\widehat{\theta}_{a,0,j}}{\widehat{\lambda}_{a,j}} - \frac{\theta_{a,1,j}-\theta_{a,0,j}}{\lambda_{a,j}}\Big)\sim N\Big(0, Q_3\Big),\]
    where 
    \begin{align*}
        Q_1 =  \sum_{j=1}^J \frac{(\widehat{\theta}_{a,1,j}-\widehat{\theta}_{a,0,j})}{\widehat{\lambda}_{a,j}}\int\big\{\mu_{a,0}(t)-\widehat{\mu}_{a,0}(t)\big\}\widehat{\phi}_{a,j}(t) \;\mathrm{d}t,
    \end{align*}
    \begin{align*}
        Q_2 =\;& \sum_{j=1}^J\frac{(\widehat{\theta}_{a,1,j}-\widehat{\theta}_{a,0,j})^2}{\widehat{\lambda}_{a,j}^2}\int\int K_a(s,t)\big\{\widehat{\phi}_{a,j}(s)-\phi_{a,j}(s)\big\}\big\{\widehat{\phi}_{a,j}(t)-\phi_{a,j}(t)\big\}\;\mathrm{d}s\;\mathrm{d}t\\
        &\hspace{-6mm} + 2\sum_{1\leq j<k\leq J}\frac{\widehat{\theta}_{a,1,j}-\widehat{\theta}_{a,0,j}}{\widehat{\lambda}_{a,j}}\cdot\frac{\widehat{\theta}_{a,1,k}-\widehat{\theta}_{a,0,k}}{\widehat{\lambda}_{a,k}}\int\int K_a(s,t)\big\{\widehat{\phi}_{a,j}(s)-\phi_{a,j}(s)\big\}\big\{\widehat{\phi}_{a,k}(t)-\phi_{a,k}(t)\big\}\;\mathrm{d}s\;\mathrm{d}t,
    \end{align*}
    and 
    \[Q_3 = \sum_{j=1}^J \lambda_{a,j}\Big(\frac{\widehat{\theta}_{a,1,j}-\widehat{\theta}_{a,0,j}}{\widehat{\lambda}_{a,j}} - \frac{\theta_{a,1,j}-\theta_{a,0,j}}{\lambda_{a,j}}\Big)^2.\]
    Consequently, by standard Gaussian tail properties \citep[e.g.~Proposition 2.1.2 in][]{vershynin2018high}, it holds with probability at least $1-\eta$ that
    \begin{align}\label{l_I_const_eq2}
        |(I)| \lesssim |Q_1| + \sqrt{Q_2\log(1/\eta)}, \quad \text{and} \quad |(II)| \lesssim \sqrt{Q_3\log(1/\eta)}.
    \end{align}
    Substituting \eqref{l_I_const_eq2} into \eqref{l_I_const_eq1}, the lemma thus follows by applying a union bound argument to the results in Lemmas \ref{l_Q1_const}, \ref{l_Q2_1}, \ref{l_Q2_2} and \ref{l_Q3}.
\end{proof}

\begin{lemma} \label{l_II}
   Under Assumptions \ref{a_class_prob} and \ref{a_data}, if we assume that $X~\sim~\mathcal{GP}(\mu_{a,0},K)$ for $a\in \{0,1\}$, then for any small constant $\eta \in (0,1/2)$, it holds with probability at least $1-\eta$ that 
    \begin{align*}
        \Big|\sum_{j=J+1}^\infty \frac{(\zeta_{a,j}-\theta_{a,0,j})(\theta_{a,1,j}-\theta_{a,0,j})}{\lambda_{a,j}}\Big| \lesssim \sqrt{J^{\alpha-2\beta+1}\log(1/\eta)}.
    \end{align*}
\end{lemma}

\begin{proof}
    When $X \sim \mathcal{GP}(\mu_{a,0},K)$, it holds that 
    \[\sum_{j=J+1}^\infty \frac{(\zeta_{a,j}-\theta_{a,0,j})(\theta_{a,1,j}-\theta_{a,0,j})}{\lambda_{a,j}} \sim N\Big(0,\sum_{j=J+1}^\infty \frac{(\theta_{a,1,j}-\theta_{a,0,j})^2}{\lambda_{a,j}}\Big).\]
    Consequently, by standard Gaussian tail properties \citep[e.g.~Proposition 2.1.2 in][]{vershynin2018high}, we have with probability at least $1-\eta$ that 
    \[\Big|\sum_{j=J+1}^\infty \frac{(\zeta_{a,j}-\theta_{a,0,j})(\theta_{a,1,j}-\theta_{a,0,j})}{\lambda_{a,j}}\Big| \lesssim \sqrt{\log(1/\eta)\sum_{j=J+1}^\infty \frac{(\theta_{a,1,j}-\theta_{a,0,j})^2}{\lambda_{a,j}}} \lesssim \sqrt{J^{\alpha-2\beta+1}\log(1/\eta)},\]
    where the last inequality follows from Assumptions \ref{a_data}\ref{a_data_cov} and \ref{a_data}\ref{a_data_mean_decay}. The lemma thus follows.
\end{proof}

\begin{lemma} \label{l_Q1_const}
    Under the same condition of \Cref{l_T_diff_const}, for any small constant $\eta \in (0,1/2)$, it holds with probability at least $1-\eta$ that
    \begin{align*}
        Q_1 = &\Big|\sum_{j=1}^J \frac{(\widehat{\theta}_{a,1,j}-\widehat{\theta}_{a,0,j})}{\widehat{\lambda}_{a,j}}\int\big\{\mu_{a,0}(t)-\widehat{\mu}_{a,0}(t)\big\}\widehat{\phi}_{a,j}(t) \;\mathrm{d}t\Big|\\
        \lesssim \;& \begin{cases}
            \sqrt{\frac{J\log(\widetilde{n}/\eta)}{\widetilde{n}}}, \quad & \text{when} \;\;\frac{\alpha+1}{2} < \beta \leq \frac{\alpha+2}{2},\\
            \sqrt{\frac{\log(\widetilde{n}/\eta)}{\widetilde{n}}}, & \text{when} \;\; \beta > \frac{\alpha+2}{2}.
        \end{cases}
    \end{align*}
\end{lemma}
\begin{proof}
    By Lemmas \ref{l_rkhs_approximation_const} and \ref{l_Q1_1}, the event in \eqref{l_rkhs_approximation_const_eq3} and a union bound argument, we have with probability at least $1-\eta/2$ that
    \begin{align} \notag
            &\Big|\sum_{j=1}^J \frac{(\widehat{\theta}_{a,1,j}-\widehat{\theta}_{a,0,j})}{\widehat{\lambda}_{a,j}}\int\big\{\mu_{a,0}(t)-\widehat{\mu}_{a,0}(t)\big\}\widehat{\phi}_{a,j}(t) \;\mathrm{d}t\Big|\\ \notag 
            \leq\;& \sum_{j=1}^J\Big|\frac{\widehat{\theta}_{a,1,j}-\widehat{\theta}_{a,0,j}}{\widehat{\lambda}_{a,j}}\Big|\cdot\Big|\int\big\{\mu_{a,0}(t)-\widehat{\mu}_{a,0}(t)\big\}\widehat{\phi}_{a,j}(t) \;\mathrm{d}t\Big|\\ \notag
            \lesssim \;& \sum_{j=1}^J \Bigg|\frac{\widehat{\theta}_{a,1,j}-\widehat{\theta}_{a,0,j}}{\sqrt{\lambda_{a,j}}}\Bigg|\frac{1}{\sqrt{\lambda_{a,j}}}\sqrt{\frac{j^{-\alpha}\log(\widetilde{n}/\eta)}{\widetilde{n}}} \\ \label{l_Q1_const_eq1}
            \lesssim \;& \sum_{j=1}^J \frac{|\widehat{\theta}_{a,1,j}-\widehat{\theta}_{a,0,j}|}{\sqrt{\lambda_{a,j}}}\sqrt{\frac{\log(\widetilde{n}/\eta)}{\widetilde{n}}}.
    \end{align}
    To further control \eqref{l_Q1_const_eq1}, we will consider three different cases.

    \noindent \textbf{Case 1: When $(\alpha+1)/2 < \beta \leq (\alpha+2)/2$.} In this scenario, by a union bound argument, we have with probability at least $1-\eta$ that 
    \begin{align*}
        \eqref{l_Q1_const_eq1} &\lesssim \sqrt{\sum_{j=1}^J \frac{(\widehat{\theta}_{a,1,j}-\widehat{\theta}_{a,0,j})^2}{\widehat{\lambda}_{a,j}}} \sqrt{\sum_{j=1}^J \frac{\log(\widetilde{n}/\eta)}{\widetilde{n}}}\\
        & \lesssim \sqrt{\frac{J\log(\widetilde{n}/\eta)}{\widetilde{n}}}\Big\{1 \vee \Big(\frac{J^{2}\log^2(J)\log(\widetilde{n}/\eta)}{\widetilde{n}}\Big)^{\frac{1}{4}}\Big\}\\ 
        & \lesssim \sqrt{\frac{J\log(\widetilde{n}/\eta)}{\widetilde{n}}},
    \end{align*}
    where the first inequality follows from Cauchy--Schwarz inequality, the second inequality follows from \Cref{l_rkhs_approximation_const} and the last inequality follows from the fact that 
    \[\sqrt{\frac{J^{2}\log^2(J)\log(\widetilde{n}/\eta)}{\widetilde{n}}} \lesssim 1.\]

    \noindent \textbf{Case 2: When $(\alpha+2)/2 < \beta \leq (3\alpha+2)/2$.} In this case, by \Cref{l_score_const}, we have with probability at least $1-\eta/2$ that, for any $j \in [J]$, $|\widehat{\theta}_{a,1,j}-\widehat{\theta}_{a,0,j}| \lesssim j^{-\beta}$. Therefore, we have that 
    \[\eqref{l_Q1_const_eq1} \lesssim \sqrt{\frac{\log(\widetilde{n}/\eta)}{\widetilde{n}}}\sum_{j=1}^J j^{\frac{\alpha}{2} -\beta} \lesssim \sqrt{\frac{\log(\widetilde{n}/\eta)}{\widetilde{n}}},\]
    where the last inequality follows as $\alpha/2 -\beta < -1$.

    \noindent \textbf{Case 3: When $\beta > (3\alpha+2)/2$.} In this case, by \Cref{l_score_const}, we have that with probability at least $1-\eta/2$ that, for any $j \in [J]$, 
    \[|\widehat{\theta}_{a,1,j}-\widehat{\theta}_{a,0,j}| \lesssim j^{-\beta} + \sqrt{\frac{j^{-\alpha}\log(\widetilde{n}/\eta)}{\widetilde{n}}}.\]  
    Therefore, we have that 
    \begin{align*}
        \eqref{l_Q1_const_eq1} &\lesssim \sqrt{\frac{\log(\widetilde{n}/\eta)}{\widetilde{n}}}\Big\{ \sum_{j=1}^Jj^{\frac{\alpha}{2} -\beta} + \sum_{j=1}^J \sqrt{\frac{\log(\widetilde{n}/\eta)}{\widetilde{n}}}\Big\}\\
        & \lesssim \sqrt{\frac{\log(\widetilde{n}/\eta)}{\widetilde{n}}} + \frac{J\log(\widetilde{n}/\eta)}{\widetilde{n}}\\
        & \lesssim \sqrt{\frac{\log(\widetilde{n}/\eta)}{\widetilde{n}}},
    \end{align*}
    where the second inequality follows as $\alpha/2 -\beta < -1$ and the last inequality follows as $J^2\log(\widetilde{n}/\eta) \lesssim  \widetilde{n}$.

    The lemma thus follows by combining results in three cases together. 
\end{proof}

\begin{lemma} \label{l_Q1_1}
   Under Assumptions \ref{a_class_prob} and \ref{a_data}\ref{a_data_cov}, for any small constant $\eta \in (0,1/2)$, it holds with probability at least $1-\eta$ that, for any $a, y \in \{0,1\}$ and $j \in [J]$ such that $J^{2\alpha+2}\log(\widetilde{n}/\eta) \lesssim \widetilde{n}$,
    \begin{align*}
        \Big|\int\big\{\mu_{a,y}(t)-\widehat{\mu}_{a,y}(t)\big\}\widehat{\phi}_{a,j}(t) \;\mathrm{d}t\Big| \lesssim \sqrt{\frac{j^{-\alpha}\log(\widetilde{n}/\eta)}{\widetilde{n}}}.
    \end{align*}
\end{lemma}

\begin{proof}
    Consider the following events,
    \[\mathcal{E}_1 = \Big\{\|\widehat{\mu}_{a,y}(t)-\mu_{a,y}(t)\|_{L^2} \lesssim \sqrt{\frac{\log(1/\eta)}{\widetilde{n}}},\; \text{for}\; a,y \in \{0,1\}\Big\},\] 
    \[\mathcal{E}_2  = \Big\{\|\widehat{\phi}_{a,j} - \phi_{a,j}\|_{L^2} \lesssim \sqrt{\frac{j^2\log(\widetilde{n}/\eta)}{\widetilde{n}}}, \quad \text{for}\;a \in\{0,1\}, j \in [J]\Big\},\]
    and 
    \[\mathcal{E}_3 = \Big\{\Big|\frac{1}{\widetilde{n}_{a,y}}\sum_{i=1}^{\widetilde{n}_{a,y}} \int \big\{\widetilde{X}_{a,y}^i(t)-\mu_{a,y}(t)\big\}\phi_{a,j}(t)\;\mathrm{d}t\Big|\lesssim \sqrt{\frac{j^{-\alpha}\log(\widetilde{n}/\eta)}{\widetilde{n}}},\quad \text{for}\;a \in\{0,1\}, j \in [J]\Big\}.\]
    By Lemmas \ref{l_n_ay}, \ref{l_mean_cov_estimation}, \ref{l_eigenfunc_estimation} and \ref{l_mean_diff_score_1}, it holds from a union-bound argument that $\mathbb{P}(\mathcal{E}_1 \cap \mathcal{E}_2 \cap \mathcal{E}_3) \geq 1- \eta$. The rest of the proof is constructed conditioning on the events happening. Note that 
    \begin{align} \notag
        &\Big|\int\big\{\mu_{a,y}(t)-\widehat{\mu}_{a,y}(t)\big\}\widehat{\phi}_{a,j}(t) \;\mathrm{d}t\Big|\\ \notag
        =\;& \Big|\int\big\{\widehat{\mu}_{a,y}(t)-\mu_{a,y}(t)\big\}\big\{\widehat{\phi}_{a,j}(t) - \phi_{a,j}(t)\big\}\;\mathrm{d}t \Big|+ \Big|\int\big\{\widehat{\mu}_{a,y}(t)-\mu_{a,y}(t)\big\}\phi_{a,j}(t)\;\mathrm{d}t\Big|\\ \label{l_Q1_1_eq1}
        =\;& (I) + (II),
    \end{align}
    and in the rest of the proof, we will control $(I)$ and $(II)$ individually.

    To control $(I)$, it holds from Cauchy--Schwarz inequality that 
    \begin{align} \notag
        |(I)| &\leq \sqrt{\int \big\{\mu_{a,y}(t)-\widehat{\mu}_{a,y}(t)\big\}^2\;\mathrm{d}t \int \big\{\widehat{\phi}_{a,j}(t) - \phi_{a,j}(t)\big\}^2\;\mathrm{d}t}\\ \label{l_Q1_1_eq2}
        & \lesssim \sqrt{\frac{\log(1/\eta)}{\widetilde{n}}\frac{j^2\log(\widetilde{n}/\eta)}{\widetilde{n}}} = \sqrt{\frac{j^2\log(\widetilde{n}/\eta)\log(1/\eta)}{\widetilde{n}^2}},
    \end{align}
    where the last inequality follows from $\mathcal{E}_1$ and $\mathcal{E}_2$.
    
    To control $(II)$, it holds that 
    \begin{align} \notag
        |(II)| & = \Big|\frac{1}{\widetilde{n}_{a,y}}\sum_{i=1}^{\widetilde{n}_{a,y}} \int \big\{\widetilde{X}_{a,y}^i(t) - \mu_{a,y}(t)\big\}\phi_{a,j}(t)\;\mathrm{d}t\Big|\\ \label{l_Q1_1_eq3}
        & \lesssim \sqrt{\frac{j^{-\alpha}\log(\widetilde{n}/\eta)}{\widetilde{n}}},
    \end{align}
    where the inequality holds from $\mathcal{E}_3$. Substituting the results in \eqref{l_Q1_1_eq2} and \eqref{l_Q1_1_eq3} into \eqref{l_Q1_1_eq1}, we have that for any $j \in [J]$,
    \begin{align*}
        \Big|\int\big\{\mu_{a,y}(t)-\widehat{\mu}_{a,y}(t)\big\}\widehat{\phi}_{a,j}(t) \;\mathrm{d}t\Big| &\lesssim \sqrt{\frac{j^2\log(\widetilde{n}/\eta)\log(1/\eta)}{\widetilde{n}^2}}+\sqrt{\frac{j^{-\alpha}\log(\widetilde{n}/\eta)}{\widetilde{n}}}\\
        & \lesssim \sqrt{\frac{j^{-\alpha}\log(\widetilde{n}/\eta)}{\widetilde{n}}},
    \end{align*}
    whenever $J^{\alpha+2}\log(1/\eta) \lesssim \widetilde{n}$. Thus, the lemma follows.
\end{proof}

\begin{lemma} \label{l_Q2_1}
     Under the same condition of \Cref{l_T_diff_const}, for any small constant $\eta \in (0,1/2)$, it holds with probability at least $1-\eta$ that
    \begin{align*}
        &\Big|\sum_{j=1}^J\frac{(\widehat{\theta}_{a,1,j}-\widehat{\theta}_{a,0,j})^2}{\widehat{\lambda}_{a,j}^2}\int\int K_a(s,t)\big\{\widehat{\phi}_{a,j}(s)-\phi_{a,j}(s)\big\}\big\{\widehat{\phi}_{a,j}(t)-\phi_{a,j}(t)\big\}\;\mathrm{d}s\;\mathrm{d}t\Big|\\
        \lesssim \;& \begin{cases}\vspace{0.5em}
            \frac{J^{\alpha-2\beta+3}\log(\widetilde{n}/\eta)}{\widetilde{n}}, \quad &\text{when}\;\;  \frac{\alpha+1}{2}< \beta \leq \frac{\alpha+2}{2},\\ \vspace{0.5em}
            \frac{J\log(\widetilde{n}/\eta)}{\widetilde{n}}, & \text{when}\;\; \frac{\alpha+2}{2}< \beta \leq \frac{\alpha+3}{2},\\
            \frac{\log(\widetilde{n}/\eta)}{\widetilde{n}}, & \text{when}\;\; \beta >\frac{\alpha+3}{2}.
        \end{cases}
    \end{align*}
\end{lemma}
\begin{proof}
    Consider the following event
    \begin{equation} \label{l_Q2_1_eq2}
        \mathcal{E}_1= \Big\{ \lambda_{a,j}/2 \leq \widehat{\lambda}_{a,j} \leq  3\lambda_{a,j}/2 ,\quad \text{for} \; j \in [J], a \in\{0,1\}\Big\}.
    \end{equation}
    By a similar argument as the one used to control \eqref{l_rkhs_approximation_const_eq3}, we have that $\mathbb{P}(\mathcal{E}_1) \geq 1-\eta$. The rest of the proof is constructed conditioning on $\mathcal{E}_1$. Note that,  by triangle inequality, we have that 
    \begin{align}\notag
        &\Big|\sum_{j=1}^J\frac{(\widehat{\theta}_{a,1,j}-\widehat{\theta}_{a,0,j})^2}{\widehat{\lambda}_{a,j}^2}\int\int K_a(s,t)\big\{\widehat{\phi}_{a,j}(s)-\phi_{a,j}(s)\big\}\big\{\widehat{\phi}_{a,j}(t)-\phi_{a,j}(t)\big\}\;\mathrm{d}s\;\mathrm{d}t\Big| \\ \label{l_Q2_1_eq1}
        \leq \;& \sum_{j=1}^J\frac{(\widehat{\theta}_{a,1,j}-\widehat{\theta}_{a,0,j})^2}{\widehat{\lambda}_{a,j}^2}\Big|\int\int K_a(s,t)\big\{\widehat{\phi}_{a,j}(s)-\phi_{a,j}(s)\big\}\big\{\widehat{\phi}_{a,j}(t)-\phi_{a,j}(t)\big\}\;\mathrm{d}s\;\mathrm{d}t\Big|.
    \end{align}
    We divide the following proof into three cases depending on the value of $\beta$. For any $\alpha>1$, by \Cref{l_Q2_1_int}, we have with probability at least $1-\eta$ that 
    \[\Big|\int\int K_a(s,t)\big\{\widehat{\phi}_{a,j}(s)-\phi_{a,j}(s)\big\}\big\{\widehat{\phi}_{a,j}(t)-\phi_{a,j}(t)\big\}\;\mathrm{d}s\;\mathrm{d}t\Big| \lesssim \frac{j^{2-\alpha}\log(\widetilde{n}/\eta)}{\widetilde{n}}.\]

    \noindent \textbf{Case 1: When $(\alpha+1)/2 <\beta \leq (\alpha+2)/2$.} In this scenario, by \Cref{l_score_const}, we have with probability at least $1-\eta/2$ that, for any $j \in [J]$, $|\widehat{\theta}_{a,1,j}-\widehat{\theta}_{a,0,j}| \lesssim j^{-\beta}$. Therefore, we have that 
    \begin{align*}
        \eqref{l_Q2_1_eq1} \lesssim  \frac{\log(\widetilde{n}/\eta)}{\widetilde{n}}\sum_{j=1}^J j^{-2\beta+2\alpha+2-\alpha} = \frac{\log(\widetilde{n}/\eta)}{\widetilde{n}}\sum_{j=1}^J j^{-2\beta+\alpha+2} \leq  \frac{J^{\alpha-2\beta+3}\log(\widetilde{n}/\eta)}{\widetilde{n}},
    \end{align*}
    where the last inequality follows from the fact that $\alpha-2\beta+2 \geq 0$.

    \noindent \textbf{Case 2: When $(\alpha+2)/2 <\beta \leq (\alpha+3)/2$.} In this scenario, by \Cref{l_score_const}, we still have with probability at least $1-\eta/2$ that, for any $j \in [J]$, $|\widehat{\theta}_{a,1,j}-\widehat{\theta}_{a,0,j}| \lesssim j^{-\beta}$. Therefore, we have that 
    \begin{align*}
        \eqref{l_Q2_1_eq1} \lesssim \frac{\log(\widetilde{n}/\eta)}{\widetilde{n}}\sum_{j=1}^J j^{-2\beta+\alpha+2} \leq  \frac{J\log(\widetilde{n}/\eta)}{\widetilde{n}},
    \end{align*}
    where the last inequality follows from the fact that $-1 \leq -2\beta+\alpha+2<0$.

    \noindent \textbf{Case 3: When $\beta > (\alpha+3)/2$.} In this case, by \Cref{l_score_const}, we have that with probability at least $1-\eta/2$ that, for any $j \in [J]$, 
    \[|\widehat{\theta}_{a,1,j}-\widehat{\theta}_{a,0,j}| \lesssim j^{-\beta} + \sqrt{\frac{j^{-\alpha}\log(\widetilde{n}/\eta)}{\widetilde{n}}}.\]
    Therefore, we have that 
    \begin{align*}
         \eqref{l_Q2_1_eq1} &\lesssim \frac{\log(\widetilde{n}/\eta)}{\widetilde{n}}\sum_{j=1}^J \{j^{-2\beta} \vee \frac{j^{-\alpha}\log(\widetilde{n}/\eta)}{\widetilde{n}}\} j^{2+\alpha} \leq \frac{\log(\widetilde{n}/\eta)}{\widetilde{n}}\Big\{\sum_{j=1}^J j^{\alpha-2\beta+2} +\sum_{j=1}^J \frac{j^2}{\widetilde{n}} \Big\}\\
         &\leq \frac{\log(\widetilde{n}/\eta)}{\widetilde{n}} \Big\{1 \vee \frac{J^3\log(\widetilde{n})}{\widetilde{n}}\Big\} \leq \frac{\log(\widetilde{n}/\eta)}{\widetilde{n}},
    \end{align*}
    where the third inequality follows from the fact that $\alpha-2\beta+2 <-1$ in this case and the last inequality follows from the assumption on $J$. 

    The lemma thus follows by combining the results for all six cases above. 
\end{proof}

\begin{lemma} \label{l_Q2_2}
   Under the same condition of \Cref{l_T_diff_const}, for any small constant $\eta \in (0,1/2)$, it holds with probability at least $1-\eta$ that
   \begin{align*}
       &\Big|\sum_{1\leq j<k\leq J}\frac{\widehat{\theta}_{a,1,j}-\widehat{\theta}_{a,0,j}}{\widehat{\lambda}_{a,j}}\cdot\frac{\widehat{\theta}_{a,1,k}-\widehat{\theta}_{a,0,k}}{\widehat{\lambda}_{a,k}}\int\int K_a(s,t)\big\{\widehat{\phi}_{a,j}(s)-\phi_{a,j}(s)\big\}\big\{\widehat{\phi}_{a,k}(t)-\phi_{a,k}(t)\big\}\;\mathrm{d}s\;\mathrm{d}t\Big|\\
       \lesssim\; &\begin{cases}\vspace{0.5em}
            \frac{J^{\alpha-2\beta+4}\log(\widetilde{n}/\eta)}{\widetilde{n}}, \quad &\text{when}\;\;\frac{\alpha+1}{2}< \beta \leq \frac{\alpha+2}{2},\\\vspace{0.5em}
            \frac{J^2\log(\widetilde{n}/\eta)}{\widetilde{n}}, &\text{when}\;\; \frac{\alpha+2}{2}< \beta \leq \frac{\alpha+3}{2},\\\vspace{0.5em}
            \frac{J\log(\widetilde{n}/\eta)}{\widetilde{n}}, &\text{when}\;\;\frac{\alpha+3}{2}< \beta \leq \frac{\alpha+4}{2},\\\vspace{0.5em}
            \frac{\log(\widetilde{n}/\eta)}{\widetilde{n}}, &\text{when}\;\; \beta > \frac{\alpha+4}{2}.
       \end{cases}
   \end{align*}
\end{lemma}
\begin{proof}
    Note that by the triangle inequality, we have that 
    \begin{align} \notag
        &\Big|\sum_{1\leq j<k\leq J}\frac{\widehat{\theta}_{a,1,j}-\widehat{\theta}_{a,0,j}}{\widehat{\lambda}_{a,j}}\cdot\frac{\widehat{\theta}_{a,1,k}-\widehat{\theta}_{a,0,k}}{\widehat{\lambda}_{a,k}}\int\int K_a(s,t)\big\{\widehat{\phi}_{a,j}(s)-\phi_{a,j}(s)\big\}\big\{\widehat{\phi}_{a,k}(t)-\phi_{a,k}(t)\big\}\;\mathrm{d}s\;\mathrm{d}t\Big|\\ \notag
        \leq \;& \sum_{1\leq j<k\leq J}\frac{|\widehat{\theta}_{a,1,j}-\widehat{\theta}_{a,0,j}|}{\widehat{\lambda}_{a,j}}\cdot\frac{|\widehat{\theta}_{a,1,k}-\widehat{\theta}_{a,0,k}|}{\widehat{\lambda}_{a,k}}\cdot\Big|\int\int K_a(s,t)\big\{\widehat{\phi}_{a,j}(s)-\phi_{a,j}(s)\big\}\big\{\widehat{\phi}_{a,k}(t)-\phi_{a,k}(t)\big\}\;\mathrm{d}s\;\mathrm{d}t\Big|\\  \notag
        = \;& \sum_{j=1}^{J-1}\sum_{k=j+1}^J\frac{|\widehat{\theta}_{a,1,j}-\widehat{\theta}_{a,0,j}|}{\widehat{\lambda}_{a,j}}\cdot\frac{|\widehat{\theta}_{a,1,k}-\widehat{\theta}_{a,0,k}|}{\widehat{\lambda}_{a,k}}\\\label{l_Q2_2_eq1}
        & \hspace{5cm} \cdot\Big|\int\int K_a(s,t)\big\{\widehat{\phi}_{a,j}(s)-\phi_{a,j}(s)\big\}\big\{\widehat{\phi}_{a,k}(t)-\phi_{a,k}(t)\big\}\;\mathrm{d}s\;\mathrm{d}t\Big|
    \end{align}
    The rest of the proof then follows a similar argument as the one used in the proof of \Cref{l_Q2_1}. The proof below is constructed conditioning on the event in \eqref{l_Q2_1_eq2} and we will divide the proof into various cases.
    For any $\alpha>1$, by \Cref{l_Q2_2_int}, we have with probability at least $1-\eta$ that
    \begin{align*}
        \Big|\int\int K_a(s,t)\big\{\widehat{\phi}_{a,j}(s)-\phi_{a,j}(s)\big\}\big\{\widehat{\phi}_{a,k}(t)-\phi_{a,k}(t)\big\}\;\mathrm{d}s\;\mathrm{d}t\Big| \lesssim \sqrt{\frac{j^{2-\alpha}k^{2-\alpha}\log^2(\widetilde{n}/\eta)}{\widetilde{n}^2}}.
    \end{align*}

    \noindent \textbf{Case 1: When $(\alpha+1)/2 < \beta \leq (\alpha+2)/2$.} In this scenario,
    by \Cref{l_score_const}, we have with probability at least $1-\eta/2$ that, for any $j \in [J]$, $|\widehat{\theta}_{a,1,j}-\widehat{\theta}_{a,0,j}| \lesssim j^{-\beta}$. Therefore, we have that 
    \begin{align*}
        \eqref{l_Q2_2_eq1} &\lesssim \frac{\log(\widetilde{n}/\eta)}{\widetilde{n}}\sum_{j=1}^{J-1}j^{1-\frac{\alpha}{2}}\frac{|\widehat{\theta}_{a,1,j}-\widehat{\theta}_{a,0,j}|}{\widehat{\lambda}_{a,j}} \sum_{k=j+1}^J k^{1-\frac{\alpha}{2}}\frac{|\widehat{\theta}_{a,1,k}-\widehat{\theta}_{a,0,k}|}{\widehat{\lambda}_{a,k}}\\
        & \lesssim \frac{\log(\widetilde{n}/\eta)}{\widetilde{n}}\sum_{j=1}^{J-1}j^{\frac{\alpha}{2}-\beta+1}  \sum_{k=j+1}^J k^{\frac{\alpha}{2}-\beta+1} \leq \frac{\log(\widetilde{n}/\eta)}{\widetilde{n}}\sum_{j=1}^{J-1}j^{\frac{\alpha}{2}-\beta+1}\cdot J^{\frac{\alpha}{2}-\beta+2}\\
        & \leq  \frac{J^{\alpha-2\beta+4}\log(\widetilde{n}/\eta)}{\widetilde{n}},
    \end{align*}
    where the third inequality follows as $\alpha/2 -\beta +1 >0$.

    \noindent \textbf{Case 2: When $(\alpha+2)/2 < \beta \leq (\alpha+3)/2$.} In this scenario, by \Cref{l_score_const}, we still have with probability at least $1-\eta/2$ that, for any $j \in [J]$, $|\widehat{\theta}_{a,1,j}-\widehat{\theta}_{a,0,j}| \lesssim j^{-\beta}$. Therefore, we have that 
    \begin{align*}
        \eqref{l_Q2_2_eq1} &\lesssim \frac{\log(\widetilde{n}/\eta)}{\widetilde{n}}\sum_{j=1}^{J-1}j^{\frac{\alpha}{2}-\beta+1}  \sum_{k=j+1}^J k^{\frac{\alpha}{2}-\beta+1} \leq \frac{J\log(\widetilde{n}/\eta)}{\widetilde{n}}\sum_{j=1}^{J-1}j^{\alpha-2\beta+2} \leq \frac{J^{2}\log(\widetilde{n}/\eta)}{\widetilde{n}},
    \end{align*}
    where the second inequality follows as $ \sum_{k=j+1}^J k^{\alpha/2-\beta+1} \leq Jj^{\alpha/2-\beta+1}$ as $-1< \alpha/2-\beta+1< 0$ and the last inequality follows as $-1 \leq \alpha -2\beta +2 <0$.

     \noindent \textbf{Case 3: When $(\alpha+3)/2 < \beta \leq (\alpha+4)/2$.}  In this scenario, by \Cref{l_score_const}, we still have with probability at least $1-\eta/2$ that, for any $j \in [J]$, $|\widehat{\theta}_{a,1,j}-\widehat{\theta}_{a,0,j}| \lesssim j^{-\beta}$. Therefore, we have that
     \begin{align*}
         \eqref{l_Q2_2_eq1} &\lesssim \frac{\log(\widetilde{n}/\eta)}{\widetilde{n}}\sum_{j=1}^{J-1}j^{\frac{\alpha}{2}-\beta+1}  \sum_{k=j+1}^J k^{\frac{\alpha}{2}-\beta+1} \leq \frac{J\log(\widetilde{n}/\eta)}{\widetilde{n}}\sum_{j=1}^{J-1}j^{\alpha-2\beta+2} \leq \frac{J\log(\widetilde{n}/\eta)}{\widetilde{n}},
     \end{align*}
    where the second inequality follows as $-1 \leq \alpha/2-\beta+1< 0$ and the last inequality follows as $\alpha-2\beta+2 < -1$.

    \noindent \textbf{Case 4: When $\beta > (\alpha+4)/2$.} In this case, we have that with probability at least $1-\eta/2$ that, for any $j \in [J]$, 
    \[|\widehat{\theta}_{a,1,j}-\widehat{\theta}_{a,0,j}| \lesssim j^{-\beta} + \sqrt{\frac{j^{-\alpha}\log(\widetilde{n}/\eta)}{\widetilde{n}}}.\]
    Therefore, we have that 
    \begin{align*}
        \eqref{l_Q2_2_eq1} &\lesssim \frac{\log(\widetilde{n}/\eta)}{\widetilde{n}}\sum_{j=1}^{J-1}\Big\{j^{\frac{\alpha}{2}-\beta+1}\vee \sqrt{\frac{j^{2}\log(\widetilde{n}/\eta)}{\widetilde{n}}} \Big\}  \sum_{k=j+1}^J \Big\{k^{\frac{\alpha}{2}-\beta+1} \vee \sqrt{\frac{k^2\log(\widetilde{n}/\eta)}{\widetilde{n}}}\Big\}\\
        & \leq \frac{\log(\widetilde{n}/\eta)}{\widetilde{n}}\sum_{j=1}^{J-1}\Big\{j^{\frac{\alpha}{2}-\beta+1}\vee \sqrt{\frac{j^{2}\log(\widetilde{n}/\eta)}{\widetilde{n}}} \Big\} \Big\{1 \vee \sqrt{\frac{J^4\log(\widetilde{n}/\eta)}{\widetilde{n}}}\Big\}\\
        & \leq \frac{\log(\widetilde{n}/\eta)}{\widetilde{n}}\Big\{\sum_{j=1}^{J-1} j^{\frac{\alpha}{2}-\beta+1}\vee \sum_{j=1}^{J-1}\sqrt{\frac{j^{2}\log(\widetilde{n}/\eta)}{\widetilde{n}}} \Big\} \leq \frac{\log(\widetilde{n}/\eta)}{\widetilde{n}},
    \end{align*}
    where the second inequality follows as $\alpha/2-\beta+1 < -1$ and the third inequality follows from the assumption of $J$.
    
    The lemma thus follows by combining results for all cases together.
\end{proof}

\begin{lemma} \label{l_Q2_1_int}
    Under Assumptions \ref{a_class_prob} and \ref{a_data}\ref{a_data_cov}, for $a\in \{0,1\}$ and any small constant $\eta \in (0,1/2)$, it holds with probability at least $1-\eta$ that, for any $j \in [J]$ such that $J^{2\alpha+2}\log(1/\eta) \lesssim \widetilde{n}$,
    \begin{align*}
        \Big|\int\int K_a(s,t)\big\{\widehat{\phi}_{a,j}(s)-\phi_{a,j}(s)\big\}\big\{\widehat{\phi}_{a,j}(t)-\phi_{a,j}(t)\big\}\;\mathrm{d}s\;\mathrm{d}t\Big| \lesssim \frac{j^{2-\alpha}\log(\widetilde{n}/\eta)}{\widetilde{n}}.
    \end{align*}
\end{lemma}

\begin{proof}
    By \Cref{l_eigen_split}, we have that 
    \begin{align} \notag
        &\Big|\int\int K_a(s,t)\big\{\widehat{\phi}_{a,j}(t)-\phi_{a,j}(t)\big\}\big\{\widehat{\phi}_{a,j}(s)-\phi_{a,j}(s)\big\}\;\mathrm{d}s\;\mathrm{d}t\Big| \\ \notag
        \leq \;& \Big|\int\big\{\widehat{\phi}_{a,j}(m) - \phi_{a,j}(m)\big\}\phi_{a,j}(m)\;\mathrm{d}m \cdot \int\int K_a(s,t)\phi_{a,j}(t)\big\{\widehat{\phi}_{a,j}(s)-\phi_{a,j}(s)\big\}\;\mathrm{d}s\;\mathrm{d}t\Big|\\ \notag
        & +\Big|\sum_{l:l \neq j}(\widehat{\lambda}_{a,j} - \lambda_{a,l})^{-1}\int\int K_a(s,t)\phi_{a,l}(t)\big\{\widehat{\phi}_{a,j}(s)-\phi_{a,j}(s)\big\}\;\mathrm{d}s\;\mathrm{d}t \\ \label{l_Q2_1_int_eq1}
        & \hspace{5cm} \cdot \int\int\big\{\widehat{K}_a(m,\ell)-K_a(m,\ell)\big\}\widehat{\phi}_{a,j}(m)\phi_{a,l}(\ell)\;\mathrm{d}m\;\mathrm{d}\ell\Big|,
    \end{align}
    and for any $l \in \mathbb{N}+$, it holds that
    \begin{align} \notag
       &\int\int K_a(s,t)\phi_{a,l}(t)\big\{\widehat{\phi}_{a,j}(s)-\phi_{a,j}(s)\big\}\;\mathrm{d}s\;\mathrm{d}t\\ \notag
       =\; &  \int\int K_a(s,t)\phi_{a,l}(t)\phi_{a,j}(s)\;\mathrm{d}s\;\mathrm{d}t\int\big\{\widehat{\phi}_{a,j}(m) - \phi_{a,j}(m)\big\}\phi_{a,j}(m)\;\mathrm{d}m\\ \notag
       & + \sum_{r:r \neq j} (\widehat{\lambda}_{a,j} - \lambda_{a,r})^{-1}\int\int K_a(s,t)\phi_{a,l}(t)\phi_{a,r}(s)\;\mathrm{d}s\;\mathrm{d}t\\ \notag
       & \hspace{6cm}\cdot \int\int\big\{\widehat{K}_a(m,\ell)-K_a(m,\ell)\big\}\widehat{\phi}_{a,j}(m)\phi_{a,r}(\ell)\;\mathrm{d}m\;\mathrm{d}\ell\\\label{l_Q2_1_int_eq2}
       =\;&\begin{cases}
           (\widehat{\lambda}_{a,j} - \lambda_{a,l})^{-1}\lambda_{a,l} \int\int\big\{\widehat{K}_a(m,\ell)-K_a(m,\ell)\big\}\widehat{\phi}_{a,j}(m)\phi_{a,l}(\ell)\;\mathrm{d}m\;\mathrm{d}\ell, \quad &\text{when}\;\; l \neq j,\\\\
           \lambda_{a,j}\int\big\{\widehat{\phi}_{a,j}(m) - \phi_{a,j}(m)\big\}\phi_{a,j}(m)\;\mathrm{d}m &\text{when}\;\; l =j.
           \end{cases} .
    \end{align}
    Substituting the results in \eqref{l_Q2_1_int_eq2} into \eqref{l_Q2_1_int_eq1}, we have that 
    \begin{align} \notag
        &\Big|\int\int K_a(s,t)\big\{\widehat{\phi}_{a,j}(t)-\phi_{a,j}(t)\big\}\big\{\widehat{\phi}_{a,j}(s)-\phi_{a,j}(s)\big\}\;\mathrm{d}s\;\mathrm{d}t\Big|\\ \notag
        \leq\;& \Big|\lambda_{a,j}\Big\{\int\big\{\widehat{\phi}_{a,j}(m) - \phi_{a,j}(m)\big\}\phi_{a,j}(m)\;\mathrm{d}m \Big\}^2\Big|\\ \notag
        & +\Big|\sum_{l:l \neq j}\lambda_{a,j}(\widehat{\lambda}_{a,j} - \lambda_{a,l})^{-2} \Big\{\int\int\big\{\widehat{K}_a(m,\ell)-K_a(m,\ell)\big\}\widehat{\phi}_{a,j}(m)\phi_{a,l}(\ell)\;\mathrm{d}m\;\mathrm{d}\ell\Big\}^2\Big|\\ \label{l_Q2_1_int_eq3}
        & = (I)+ (II).
    \end{align}
    Consider the following events,
    \begin{align*}
        \mathcal{E}_1 = \big\{(\widehat{\lambda}_{a,j} - \lambda_{a,k})^{-2} \leq 2 (\lambda_{a,j} - \lambda_{a,k})^{-2}, \quad \text{for}\; k,j \in\mathbb{N}_+, a\in\{0,1\} \big\},
    \end{align*}
    \begin{align*}
        \mathcal{E}_2 = \Big\{\|\widehat{K}_a - K_a\|_{L^2} \lesssim \sqrt{\frac{\log(1/\eta)}{\widetilde{n}}},\;\;a\in\{0,1\}\Big\},
    \end{align*}
    \begin{align*}
        \mathcal{E}_3  = \Big\{\|\widehat{\phi}_{a,j} - \phi_{a,j}\|_{L^2} \lesssim \sqrt{\frac{j^2\log(\widetilde{n}/\eta)}{\widetilde{n}}}, \quad \text{for}\; j \in [J],a\in\{0,1\}\Big\},
    \end{align*}
    and
    \begin{align*}
        \mathcal{E}_4 = \Big\{\Big|\int\int\big\{\widehat{K}_a(m,\ell) -K_a(m,\ell)\big\}\phi_{a,j}(m)\phi_{a,k}(\ell)\;\mathrm{d}m\;\mathrm{d}\ell\Big| \lesssim \sqrt{\frac{j^{-\alpha}k^{-\alpha}\log(\widetilde{n}/\eta)}{\widetilde{n}}}, \; k,j \in [J],a \in\{0,1\}\Big\}.
    \end{align*}
    By Lemmas \ref{l_n_ay}, \ref{l_average_cov}, \ref{l_eigenfunc_estimation}, \ref{l_cov_proj_1} and a similar argument as the one leads to \eqref{eq_l_eigenfunc_estimation_7}, it holds from a union-bound argument that $\mathbb{P}(\mathcal{E}_1 \cap \mathcal{E}_2 \cap \mathcal{E}_3 \cap \mathcal{E}_4) \geq 1- \eta$. The rest of the proof is constructed conditioning on these events happening. To control $(I)$, we have that 
    \begin{align} \label{l_Q2_1_int_eq4}
        (I) \leq j^{-\alpha}\int\big\{\widehat{\phi}_{a,j}(m) - \phi_{a,j}(m)\big\}^2\;\mathrm{d}m \cdot \int \phi_{a,j}^2(m)\;\mathrm{d}m  \lesssim \frac{j^{2-\alpha}\log(\widetilde{n}/\eta)}{\widetilde{n}},
    \end{align}
    where the first inequality follows from Cauchy--Schwarz inequality. To control $(II)$, note that by the triangle inequality, we have that 
    \begin{align*}
        (II) \leq\;& \Big|\sum_{l:l \neq j}\lambda_{a,l}(\widehat{\lambda}_{a,j} - \lambda_{a,l})^{-2} \Big\{\int\int\big\{\widehat{K}_a(m,\ell)-K_a(m,\ell)\big\}\phi_{a,j}(m)\phi_{a,l}(\ell)\;\mathrm{d}m\;\mathrm{d}\ell\Big\}^2\Big|\\
        & + \Big|\sum_{l:l \neq j}\lambda_{a,l}(\widehat{\lambda}_{a,j} - \lambda_{a,l})^{-2} \Big\{\int\int\big\{\widehat{K}_a(m,\ell)-K_a(m,\ell)\big\}\big\{\widehat{\phi}_{a,j}(m)-\phi_{a,j}(m)\big\}\phi_{a,l}(\ell)\;\mathrm{d}m\;\mathrm{d}\ell\Big\}^2\Big|\\
        =\; & (II)_1 + (II)_2. 
    \end{align*}
    To control $(II)_1$, we have that 
    \begin{align} \notag
        (II)_1 &\lesssim \sum_{l:l \neq j}\lambda_{a,l}(\widehat{\lambda}_{a,j} - \lambda_{a,l})^{-2} \cdot \frac{j^{-\alpha}l^{-\alpha}\log(\widetilde{n}/\eta)}{\widetilde{n}} \\\label{l_Q2_1_int_eq5}
        &\leq \frac{j^{-\alpha}\log(\widetilde{n}/\eta)}{\widetilde{n}}\sum_{l:l \neq j}(\lambda_{a,j} - \lambda_{a,l})^{-2} l^{-2\alpha}\lesssim \frac{j^{2-\alpha}\log(\widetilde{n}/\eta)}{\widetilde{n}},
    \end{align}
    where the last inequality follows from \Cref{l_eigenfunc_sum}. Similarly, to control $(II)_2$, we have that 
    \begin{align}\notag
        (II)_2 &\lesssim \sum_{l:l \neq j}\lambda_{a,l}(\lambda_{a,j} - \lambda_{a,l})^{-2}\|\widehat{K}_a-K_a\|_{L^2}^2\|\widehat{\phi}_{a,j} -\phi_{a,j}\|^2_{L^2}\|\phi_{a,l}\|^2_{L^2}\\\label{l_Q2_1_int_eq6}
        & \lesssim \frac{\log(\widetilde{n}/\eta)\log(1/\eta)}{\widetilde{n}^2}\sum_{l:l \neq j}(\lambda_{a,j} - \lambda_{a,l})^{-2}j^{2}l^{-\alpha} \lesssim \frac{j^{4+\alpha}\log(\widetilde{n}/\eta)\log(1/\eta)}{\widetilde{n}^2} \lesssim \frac{j^{2-\alpha}\log(\widetilde{n}/\eta)}{\widetilde{n}},
    \end{align}
    where the first inequality follows from Cauchy--Schwarz inequality, the third inequality follows from \Cref{l_eigenfunc_sum} and the last inequality follows from the assumption that $J^{2\alpha+2}\lesssim \widetilde{n}$. The lemma thus follows by substituting the results in \eqref{l_Q2_1_int_eq4}, \eqref{l_Q2_1_int_eq5} and \eqref{l_Q2_1_int_eq6} into \eqref{l_Q2_1_int_eq3}. 
\end{proof}

\begin{lemma} \label{l_Q2_2_int}
    Under Assumptions \ref{a_class_prob} and \ref{a_data}\ref{a_data_cov}, for any $a\in \{0,1\}$ and small constant $\eta \in (0,1/2)$, it holds with probability at least $1-\eta$ that, for any $j,k \in [J]$ such that $1\leq j< k\leq J$, $J^{2\alpha+2}\log(1/\eta) \lesssim \widetilde{n}$,
    \begin{align*}
        \Big|\int\int K_a(s,t)\big\{\widehat{\phi}_{a,j}(s)-\phi_{a,j}(s)\big\}\big\{\widehat{\phi}_{a,k}(t)-\phi_{a,k}(t)\big\}\;\mathrm{d}s\;\mathrm{d}t\Big| \lesssim \sqrt{\frac{j^{2-\alpha}k^{2-\alpha}\log^2(\widetilde{n}/\eta)}{\widetilde{n}^2}}.
    \end{align*}
\end{lemma}
\begin{proof}
    By \Cref{l_eigen_split}, we have that 
    \begin{align} \notag
        &\Big|\int\int K_a(s,t)\big\{\widehat{\phi}_{a,j}(t)-\phi_{a,j}(t)\big\}\big\{\widehat{\phi}_{a,k}(s)-\phi_{a,k}(s)\big\}\;\mathrm{d}s\;\mathrm{d}t\Big| \\ \notag
         = \;& \Big|\int\big\{\widehat{\phi}_{a,j}(m) - \phi_{a,j}(m)\big\}\phi_{a,j}(m)\;\mathrm{d}m \cdot \int\int K_a(s,t)\phi_{a,j}(t)\big\{\widehat{\phi}_{a,k}(s)-\phi_{a,k}(s)\big\}\;\mathrm{d}s\;\mathrm{d}t\Big|\\ \notag
        & +\Big|\sum_{l:l \neq j}(\widehat{\lambda}_{a,j} - \lambda_{a,l})^{-1}\int\int K_a(s,t)\phi_{a,l}(t)\big\{\widehat{\phi}_{a,k}(s)-\phi_{a,k}(s)\big\}\;\mathrm{d}s\;\mathrm{d}t \\ \label{l_Q2_2_int_eq1}
        & \hspace{5cm} \cdot \int\int\big\{\widehat{K}_a(m,\ell)-K_a(m,\ell)\big\}\widehat{\phi}_{a,j}(m)\phi_{a,l}(\ell)\;\mathrm{d}m\;\mathrm{d}\ell\Big|,
    \end{align}
    and for any $l \in \mathbb{N}+$,
    \begin{align} \notag
       &\int\int K_a(s,t)\phi_{a,l}(t)\big\{\widehat{\phi}_{a,k}(s)-\phi_{a,k}(s)\big\}\;\mathrm{d}s\;\mathrm{d}t\\ \notag
       =\; &  \int\int K_a(s,t)\phi_{a,l}(t)\phi_{a,k}(s)\;\mathrm{d}s\;\mathrm{d}t\int\big\{\widehat{\phi}_{a,k}(m) - \phi_{a,k}(m)\big\}\phi_{a,k}(m)\;\mathrm{d}m\\ \notag
       & + \sum_{r:r \neq k} (\widehat{\lambda}_{a,k} - \lambda_{a,r})^{-1}\int\int K_a(s,t)\phi_{a,l}(t)\phi_{a,r}(s)\;\mathrm{d}s\;\mathrm{d}t\\ \label{l_Q2_2_int_eq2}
       & \hspace{5cm}\cdot \int\int\big\{\widehat{K}_a(m,\ell)-K_a(m,\ell)\big\}\widehat{\phi}_{a,k}(m)\phi_{a,r}(\ell)\;\mathrm{d}m\;\mathrm{d}\ell.
    \end{align}
    Substituting the results in \eqref{l_Q2_2_int_eq2} into \eqref{l_Q2_2_int_eq1} and using the property of eigenfunctions, we have that 
    \begin{align} \notag
        &\Big|\int\int K_a(s,t)\big\{\widehat{\phi}_{a,j}(t)-\phi_{a,j}(t)\big\}\big\{\widehat{\phi}_{a,k}(s)-\phi_{a,k}(s)\big\}\;\mathrm{d}s\;\mathrm{d}t\Big|\\ \notag
        = \;& \Big|\int\big\{\widehat{\phi}_{a,j}(m) - \phi_{a,j}(m)\big\}\phi_{a,j}(m)\;\mathrm{d}m \cdot (\widehat{\lambda}_{a,k} - \lambda_{a,j})^{-1}\lambda_{a,j}\\ \notag
        & \hspace{7cm} \cdot\int\int\big\{\widehat{K}_a(m,\ell)-K_a(m,\ell)\big\}\widehat{\phi}_{a,k}(m)\phi_{a,j}(\ell)\;\mathrm{d}m\;\mathrm{d}\ell\Big|\\ \notag
        & +\Big|(\widehat{\lambda}_{a,j} - \lambda_{a,k})^{-1}\lambda_{a,k}\int\big\{\widehat{\phi}_{a,k}(m) - \phi_{a,k}(m)\big\}\phi_{a,k}(m)\;\mathrm{d}m\\\notag
        & \hspace{7cm} \cdot\int\int\big\{\widehat{K}_a(m,\ell)-K_a(m,\ell)\big\}\widehat{\phi}_{a,j}(m)\phi_{a,k}(\ell)\;\mathrm{d}m\;\mathrm{d}\ell\Big|\\\notag
         &+\Big|\sum_{l:l \neq j,k}(\widehat{\lambda}_{a,j} - \lambda_{a,l})^{-1}(\widehat{\lambda}_{a,k} - \lambda_{a,l})^{-1}\lambda_{a,l}\int\int\big\{\widehat{K}_a(m,\ell)-K_a(m,\ell)\big\}\widehat{\phi}_{a,k}(m)\phi_{a,l}(\ell)\;\mathrm{d}m\;\mathrm{d}\ell \\ \notag
        &\hspace{7cm} \cdot\int\int\big\{\widehat{K}_a(m,\ell)-K_a(m,\ell)\big\}\widehat{\phi}_{a,j}(m)\phi_{a,l}(\ell)\;\mathrm{d}m\;\mathrm{d}\ell \Big|\\ \label{l_Q2_2_int_eq3}
        =\;&  (I)+ (II) + (III).
    \end{align}
    Consider the following events,
    \begin{align*}
        \mathcal{E}_1 = \big\{|\widehat{\lambda}_{a,j} - \lambda_{a,k}|^{-1} \leq \sqrt{2} |\lambda_{a,j} - \lambda_{a,k}|^{-1}, \quad \text{for}\; k,j \in\mathbb{N}_+,a \in\{0,1\}\big\},
    \end{align*}
    \begin{align*}
        \mathcal{E}_2 = \Big\{\|\widehat{K}_a - K_a\|_{L^2} \lesssim \sqrt{\frac{\log(1/\eta)}{\widetilde{n}}}, \; \text{for}\;a \in\{0,1\}\Big\},
    \end{align*}
    \begin{align*}
        \mathcal{E}_3  = \Big\{\|\widehat{\phi}_{a,j} - \phi_{a,j}\|_{L^2} \lesssim \sqrt{\frac{j^2\log(\widetilde{n}/\eta)}{\widetilde{n}}}, \quad \text{for}\; j \in [J],a \in\{0,1\}\Big\},
    \end{align*}
    and
    \begin{align*}
        \mathcal{E}_4 = \Big\{\Big|\int\int\big\{\widehat{K}_a(m,\ell) -K_a(m,\ell)\big\}\phi_{a,j}(m)\phi_{a,k}(\ell)\;\mathrm{d}m\;\mathrm{d}\ell\Big| \lesssim \sqrt{\frac{j^{-\alpha}k^{-\alpha}\log(\widetilde{n}/\eta)}{\widetilde{n}}}, \; k,j \in [J],a \in\{0,1\}\Big\}.
    \end{align*}
    By Lemmas \ref{l_n_ay}, \ref{l_average_cov}, \ref{l_eigenfunc_estimation}, \ref{l_cov_proj_1} and a similar argument as the one leads to \eqref{eq_l_eigenfunc_estimation_7}, it holds from a union-bound argument that $\mathbb{P}(\mathcal{E}_1 \cap \mathcal{E}_2 \cap \mathcal{E}_3 \cap \mathcal{E}_4) \geq 1- \eta$. The rest of the proof is constructed conditioning on these events happening and we will control the three terms in \ref{l_Q2_2_int_eq3} above separately in the following three steps. 

    \noindent \textbf{Step 1: Upper bound on $(I)$.} Note that by triangle inequality, we have that 
    \begin{align*}
        (I) \leq\;&\Big|\int\big\{\widehat{\phi}_{a,j}(m) - \phi_{a,j}(m)\big\}\phi_{a,j}(m)\;\mathrm{d}m \cdot (\widehat{\lambda}_{a,k} - \lambda_{a,j})^{-1}\lambda_{a,j}\\
        & \hspace{4cm} \cdot\int\int\big\{\widehat{K}_a(m,\ell)-K_a(m,\ell)\big\}\big\{\widehat{\phi}_{a,k}(m)-\phi_{a,k}(m)\big\}\phi_{a,j}(\ell)\;\mathrm{d}m\;\mathrm{d}\ell\Big|\\
        &+\Big|\int\big\{\widehat{\phi}_{a,j}(m) - \phi_{a,j}(m)\big\}\phi_{a,j}(m)\;\mathrm{d}m \cdot (\widehat{\lambda}_{a,k} - \lambda_{a,j})^{-1}\lambda_{a,j}\\
        & \hspace{4cm} \cdot \int\int\big\{\widehat{K}_a(m,\ell)-K_a(m,\ell)\big\}\phi_{a,k}(m)\phi_{a,j}(\ell)\;\mathrm{d}m\;\mathrm{d}\ell\Big|\\
        =\;& (I)_1 + (I)_2.
    \end{align*}
    Note that the control of $(I)_1$ is similar to the argument in \eqref{l_Q2_1_int_eq6} and will be dominated by the upper bound we give on $(I)_2$. We omit the proof here. To control $(I)_2$, we have that 
    \begin{align*}
        (I)_2 &\leq \lambda_{a,j} |\lambda_{a,k} - \lambda_{a,j}|^{-1}\Big|\int\big\{\widehat{\phi}_{a,j}(m) - \phi_{a,j}(m)\big\}\phi_{a,j}(m)\;\mathrm{d}m\Big|\\
        & \hspace{5cm}\cdot \Big|\int\int\big\{\widehat{K}_a(m,\ell)-K_a(m,\ell)\big\}\phi_{a,k}(m)\phi_{a,j}(\ell)\;\mathrm{d}m\;\mathrm{d}\ell\Big|\\
        & \leq j^{-\alpha}|\lambda_{a,k} - \lambda_{a,j}|^{-1}\|\widehat{\phi}_{a,j} - \phi_{a,j}\|_{L^2}\|\phi_{a,j}\|_{L^2}\sqrt{\frac{j^{-\alpha}k^{-\alpha}\log(\widetilde{n}/\eta)}{\widetilde{n}}}\\
        & \leq |\lambda_{a,k} - \lambda_{a,j}|^{-1} \sqrt{\frac{j^{2-3\alpha}k^{-\alpha}\log^2(\widetilde{n}/\eta)}{\widetilde{n}^2}},
    \end{align*}
    where the second inequality follows from Cauchy--Schwarz inequality. Therefore, we have that 
    \begin{align} \label{l_Q2_2_int_eq4}
        (I) \lesssim |\lambda_{a,k} - \lambda_{a,j}|^{-1} \sqrt{\frac{j^{2-3\alpha}k^{-\alpha}\log^2(\widetilde{n}/\eta)}{\widetilde{n}^2}}.
    \end{align}

    \noindent \textbf{Step 2: Upper bound on $(II)$.} The argument to control $(II)$ is similar to the one used in \textbf{Step 1}. We omit the proof here and we have that 
    \begin{equation}\label{l_Q2_2_int_eq5}
        (II) \lesssim |\lambda_{a,k} - \lambda_{a,j}|^{-1} \sqrt{\frac{k^{2-3\alpha}j^{-\alpha}\log^2(\widetilde{n}/\eta)}{\widetilde{n}^2}}.
    \end{equation}

    \noindent \textbf{Step 3: Upper bound on $(III)$.} To control $(III)$, by Cauchy--Schwarz inequality, we have that 
    \begin{align*}
        (III) \leq \;& \sqrt{\sum_{l:l \neq j,k}\lambda_{a,l}(\widehat{\lambda}_{a,j} - \lambda_{a,l})^{-2}\Big\{\int\int\big\{\widehat{K}_a(m,\ell)-K_a(m,\ell)\big\}\widehat{\phi}_{a,j}(m)\phi_{a,l}(\ell)\;\mathrm{d}m\;\mathrm{d}\ell\Big\}^2}\\
        & \cdot \sqrt{\sum_{l:l \neq j,k}\lambda_{a,l}(\widehat{\lambda}_{a,k} - \lambda_{a,l})^{-2}\Big\{\int\int\big\{\widehat{K}_a(m,\ell)-K_a(m,\ell)\big\}\widehat{\phi}_{a,k}(m)\phi_{a,l}(\ell)\;\mathrm{d}m\;\mathrm{d}\ell\Big\}^2}\\
        =\;& \sqrt{(III)_1}\cdot \sqrt{(III)_2}.
    \end{align*}
    To control $(III)_1$, note that 
    \begin{align*}
        (III)_1 \lesssim\;& \sum_{l:l \neq j,k}\lambda_{a,l}(\lambda_{a,j} - \lambda_{a,l})^{-2}\Big\{\int\int\big\{\widehat{K}_a(m,\ell)-K_a(m,\ell)\big\}\big\{\widehat{\phi}_{a,j}(m) - \phi_{a,j}(m)\big\}\phi_{a,l}(\ell)\;\mathrm{d}m\;\mathrm{d}\ell\Big\}^2\\
        & + \sum_{l:l \neq j,k}\lambda_{a,l}(\lambda_{a,j} - \lambda_{a,l})^{-2}\Big\{\int\int\big\{\widehat{K}_a(m,\ell)- K_a(m,\ell)\big\}\phi_{a,j}(m)\phi_{a,l}(\ell)\;\mathrm{d}m\;\mathrm{d}\ell\Big\}^2\\
        = \;& (A) + (B).
    \end{align*}
    Similarly, the upper bound we provide on $(A)$ will be masked off by the upper bound we provide on $(B)$, and we focus on the upper bound on $(B)$ below. We have that 
    \begin{align*}
        (B) &\leq \sum_{l:l \neq j,k}\lambda_{a,l}(\lambda_{a,j} - \lambda_{a,l})^{-2} \cdot \frac{j^{-\alpha}l^{-\alpha}\log(\widetilde{n}/\eta)}{\widetilde{n}}\\
        &= \frac{j^{-\alpha}\log(\widetilde{n}/\eta)}{\widetilde{n}} \sum_{l:l \neq j,k}(\lambda_{a,j} - \lambda_{a,l})^{-2}l^{-2\alpha}\lesssim \frac{j^{2-\alpha}\log(\widetilde{n}/\eta)}{\widetilde{n}},
    \end{align*}
    where the last inequality follows from \Cref{l_eigenfunc_sum}. Similarly, we can show that 
    \begin{align*}
        (III)_2 \lesssim \frac{k^{2-\alpha}\log(\widetilde{n}/\eta)}{\widetilde{n}}.
    \end{align*}
    Therefore, we have that 
    \begin{align}\label{l_Q2_2_int_eq6}
        (III) \lesssim \sqrt{\frac{j^{2-\alpha}\log(\widetilde{n}/\eta)}{\widetilde{n}}\cdot\frac{k^{2-\alpha}\log(\widetilde{n}/\eta)}{\widetilde{n}}} \lesssim \sqrt{\frac{j^{2-\alpha}k^{2-\alpha}\log^2(\widetilde{n}/\eta)}{\widetilde{n}^2}}.
    \end{align}

    \noindent \textbf{Step 4: Combine results together.} Substituting results in \eqref{l_Q2_2_int_eq4}, \eqref{l_Q2_2_int_eq5} and \eqref{l_Q2_2_int_eq6} into \eqref{l_Q2_2_int_eq3}, we have that for any $1\leq j< k\leq J$
    \begin{align*}
        \eqref{l_Q2_2_int_eq6}\lesssim\;& |\lambda_{a,k} - \lambda_{a,j}|^{-1} \sqrt{\frac{j^{2-3\alpha}k^{-\alpha}\log^2(\widetilde{n}/\eta)}{\widetilde{n}^2}} + |\lambda_{a,k} - \lambda_{a,j}|^{-1} \sqrt{\frac{k^{2-3\alpha}j^{-\alpha}\log^2(\widetilde{n}/\eta)}{\widetilde{n}^2}}\\
        &+ \sqrt{\frac{j^{2-\alpha}k^{2-\alpha}\log^2(\widetilde{n}/\eta)}{\widetilde{n}^2}}\\
        \lesssim \;& |\lambda_{a,k} - \lambda_{a,j}|^{-1} \sqrt{\frac{j^{2-3\alpha}k^{-\alpha}\log^2(\widetilde{n}/\eta)}{\widetilde{n}^2}}+  \sqrt{\frac{j^{2-\alpha}k^{2-\alpha}\log^2(\widetilde{n}/\eta)}{\widetilde{n}^2}}\\
        \lesssim \;& \begin{cases}
             \sqrt{\frac{j^{4-\alpha}k^{-\alpha}\log^2(\widetilde{n}/\eta)}{\widetilde{n}^2}}+  \sqrt{\frac{j^{2-\alpha}k^{2-\alpha}\log^2(\widetilde{n}/\eta)}{\widetilde{n}^2}}\quad &\text{when}\;\;j<k\leq \{2j \wedge J \}\\\\
             \sqrt{\frac{j^{2-\alpha}k^{-\alpha}\log^2(\widetilde{n}/\eta)}{\widetilde{n}^2}}+  \sqrt{\frac{j^{2-\alpha}k^{2-\alpha}\log^2(\widetilde{n}/\eta)}{\widetilde{n}^2}} & \text{when}\;\; k> \{2j \wedge J\}
        \end{cases}\\
        \lesssim \;& \sqrt{\frac{j^{2-\alpha}k^{2-\alpha}\log^2(\widetilde{n}/\eta)}{\widetilde{n}^2}},
    \end{align*}
    where the second and the last inequality follows from the fact that $j<k$, the third inequality follows from the fact that
    \[|\lambda_{a,j}- \lambda_{a,k}| \gtrsim \begin{cases}
        |j-k|j^{-\alpha-1}, \quad &\text{if} \; j/2 \leq k\leq 2j,
        \\
        j^{-\alpha}, \quad & \text{if}\; k>2j.
    \end{cases}\]
    The lemma thus follows. 

\end{proof}

\begin{lemma}\label{l_Q3}
    Under the same condition of \Cref{l_T_diff_const}, for any small constant $\eta \in (0,1/2)$, it holds with probability at least $1-\eta$ that
    \begin{align*}
        Q_3 = \sum_{j=1}^J \lambda_{a,j}\Big(\frac{\widehat{\theta}_{a,1,j}-\widehat{\theta}_{a,0,j}}{\widehat{\lambda}_{a,j}} - \frac{\theta_{a,1,j}-\theta_{a,0,j}}{\lambda_{a,j}}\Big)^2 \lesssim \begin{cases}
            \frac{J^{\alpha-2\beta+3}\log^2(J)\log(\widetilde{n}/\eta)}{\widetilde{n}},\quad & \text{when}\;\; \frac{\alpha+1}{2} < \beta \leq \frac{\alpha+2}{2},\\
            \frac{J\log(\widetilde{n}/\eta)}{\widetilde{n}}, & \text{when}\;\;  \beta > \frac{\alpha+2}{2}.
        \end{cases}
    \end{align*}
\end{lemma}
\begin{proof}
    To control $Q_3$, we have that 
    \begin{align*}
        &\sum_{j=1}^J \lambda_{a,j}\Big(\frac{\widehat{\theta}_{a,1,j}-\widehat{\theta}_{a,0,j}}{\widehat{\lambda}_{a,j}} - \frac{\theta_{a,1,j}-\theta_{a,0,j}}{\lambda_{a,j}}\Big)^2\\
        =\;& \sum_{j=1}^J \lambda_{a,j}\Big(\frac{\widehat{\theta}_{a,1,j}-\widehat{\theta}_{a,0,j}}{\widehat{\lambda}_{a,j}}-\frac{\theta_{a,1,j}-\theta_{a,0,j}}{\widehat{\lambda}_{a,j}}+\frac{\theta_{a,1,j}-\theta_{a,0,j}}{\widehat{\lambda}_{a,j}} - \frac{\theta_{a,1,j}-\theta_{a,0,j}}{\lambda_{a,j}}\Big)^2\\
        \lesssim \;& \sum_{j=1}^J \frac{\lambda_{a,j}}{\widehat{\lambda}_{a,j}^2}(\widehat{\theta}_{a,1,j}-\widehat{\theta}_{a,0,j}-\theta_{a,1,j}+\theta_{a,0,j})^2 + \sum_{j=1}^J(\theta_{a,1,j}-\theta_{a,0,j})^2 \lambda_{a,j}\Big|\frac{1}{\widehat{\lambda}_{a,j}} - \frac{1}{\lambda_{a,j}}\Big|^2\\
        \lesssim \;& \sum_{j=1}^J \frac{\lambda_{a,j}}{\widehat{\lambda}_{a,j}^2}(\widehat{\theta}_{a,1,j}-\widehat{\theta}_{a,0,j}-\theta_{a,1,j}+\theta_{a,0,j})^2 + \sum_{j=1}^J\frac{(\theta_{a,1,j}-\theta_{a,0,j})^2}{\lambda_{a,j}}\cdot \frac{|\widehat{\lambda}_{a,j}- \lambda_{a,j}|^2}{\widehat{\lambda}_{a,j}^2}.
    \end{align*}
    The lemma thus follows by applying a union bound argument to results in Lemmas \ref{l_Q3_1_const} and \ref{l_Q3_2}.
\end{proof}

\begin{lemma} \label{l_Q3_1_const}
    Under the same condition of \Cref{l_T_diff_const}, for any small constant $\eta \in (0,1/2)$, it holds with probability at least $1-\eta$ that
    \begin{align*}
        \sum_{j=1}^J \frac{\lambda_{a,j}}{\widehat{\lambda}_{a,j}^2}(\widehat{\theta}_{a,1,j}-\widehat{\theta}_{a,0,j}-\theta_{a,1,j}+\theta_{a,0,j})^2 \lesssim  \begin{cases}
            \frac{J^{\alpha-2\beta+3}\log^2(J)\log(\widetilde{n}/\eta)}{\widetilde{n}},\quad & \text{when}\;\; \frac{\alpha+1}{2} < \beta \leq \frac{\alpha+2}{2},\\
            \frac{J\log(\widetilde{n}/\eta)}{\widetilde{n}}, & \text{when}\;\;  \beta > \frac{\alpha+2}{2}.
        \end{cases}
    \end{align*}
\end{lemma}
\begin{proof}
    Consider the following events
    \[\mathcal{E}_1= \Big\{ \lambda_{a,j}/2 \leq \widehat{\lambda}_{a,j} \leq  3\lambda_{a,j}/2 \;\; \text{and} \;\; \Big|\frac{1}{\widehat{\lambda}_{a,j}} - \frac{1}{\lambda_{a,j}}\Big| \lesssim \frac{1}{\lambda_{a,j}^2}\sqrt{\frac{j^{-2\alpha}\log(\widetilde{n}/\eta)}{\widetilde{n}}}, \; j \in [J],a \in\{0,1\}\Big\},\]
    \[\mathcal{E}_2=\Bigg\{|\widehat{\theta}_{a,1,j} - \theta_{a,1,j}+\widehat{\theta}_{a,0,j} - \theta_{a,0,j}| \lesssim \sqrt{\frac{j^{2-2\beta}\log^2(j)\log(\widetilde{n}/\eta)}{\widetilde{n}}},\;j \in [J],a \in\{0,1\},\frac{\alpha+1}{2} < \beta \leq \frac{\alpha+2}{2}\Bigg\},\]
    and
    \[\mathcal{E}_3 = \Bigg\{|\widehat{\theta}_{a,1,j} - \theta_{a,1,j}+\widehat{\theta}_{a,0,j} - \theta_{a,0,j}| \lesssim \sqrt{\frac{j^{-\alpha}\log(\widetilde{n}/\eta)}{\widetilde{n}}}, \; j \in [J],a \in\{0,1\}, \beta >\frac{\alpha+2}{2}\Bigg\}.\]
    By \Cref{l_score_const}, a similar argument used to control \eqref{l_rkhs_approximation_const_eq3} and a union bound argument, it holds that $\mathbb{P}(\mathcal{E}_1 \cap \mathcal{E}_2 \cap \mathcal{E}_3) \geq 1-\eta$. The rest of the proof is constructed conditioning on $\mathcal{E}_1 \cap \mathcal{E}_2 \cap \mathcal{E}_3$ happening and we have that
    \begin{equation} \label{l_Q3_1_const_eq1}
        \sum_{j=1}^J \frac{\lambda_{a,j}}{\widehat{\lambda}_{a,j}^2}(\widehat{\theta}_{a,1,j}-\widehat{\theta}_{a,0,j}-\theta_{a,1,j}+\theta_{a,0,j})^2 \leq 4 \sum_{j=1}^J \frac{1}{\lambda_{a,j}}(\widehat{\theta}_{a,1,j}-\widehat{\theta}_{a,0,j}-\theta_{a,1,j}+\theta_{a,0,j})^2.
    \end{equation}
    In the rest of the proof, we will present upper bounds on \eqref{l_Q3_1_const_eq1} in several cases depending on the relationship between $\alpha$ and $\beta$.
    
    \noindent \textbf{Case 1: When $(\alpha+1)/2 < \beta \leq (\alpha+2)/2$.} In this case we have that
    \begin{align*}
        \eqref{l_Q3_1_const_eq1} \lesssim \sum_{j=1}^J j^{\alpha}\cdot\frac{j^{2-2\beta}\log^2(j)\log(\widetilde{n}/\eta)}{\widetilde{n}} \leq \frac{\log^2(J)\log(\widetilde{n}/\eta)}{\widetilde{n}}\sum_{j=1}^Jj^{\alpha-2\beta+2} \leq \frac{J^{\alpha-2\beta+3}\log^2(J)\log(\widetilde{n}/\eta)}{\widetilde{n}},
    \end{align*}
    where the last inequality follows as $\alpha-2\beta+2 \geq 0$.

    \noindent \textbf{Case 2: When $\beta > (\alpha+2)/2$.} In this case we have that
    \begin{align*}
        \eqref{l_Q3_1_const_eq1} \lesssim \sum_{j=1}^J j^{\alpha}\cdot \frac{j^{-\alpha}\log(\widetilde{n}/\eta)}{\widetilde{n}} \lesssim \frac{J\log(\widetilde{n}/\eta)}{\widetilde{n}}.
    \end{align*}
\end{proof}

\begin{lemma} \label{l_Q3_2}
    Under Assumptions \ref{a_class_prob} and \ref{a_data}, for any small constant $\eta \in (0,1/2)$ and $J \in \mathbb{N}_+$ such that $J^{2\alpha+2}\log^2(J)\log(\widetilde{n}/\eta) \lesssim \widetilde{n}$, it holds with probability at least $1-\eta$ that
    \begin{align*}
        \sum_{j=1}^J(\theta_{a,1,j}-\theta_{a,0,j})^2 \lambda_{a,j}\Big|\frac{1}{\widehat{\lambda}_{a,j}} - \frac{1}{\lambda_{a,j}}\Big|^2 \lesssim \frac{\log(\widetilde{n}/\eta)}{\widetilde{n}}.
    \end{align*}
\end{lemma}
\begin{proof}
    Consider the following event,
    \[\mathcal{E}_1= \Big\{ \lambda_{a,j}/2 \leq \widehat{\lambda}_{a,j} \leq  3\lambda_{a,j}/2 \;\; \text{and} \;\; \Big|\frac{1}{\widehat{\lambda}_{a,j}} - \frac{1}{\lambda_{a,j}}\Big| \lesssim \frac{1}{\lambda_{a,j}^2}\sqrt{\frac{j^{-2\alpha}\log(\widetilde{n}/\eta)}{\widetilde{n}}},\quad \text{for} \; j \in [J],a \in\{0,1\}\Big\}.\]
    By a similar argument used to control \eqref{l_rkhs_approximation_const_eq3}, we have that $\mathbb{P}(\mathcal{E}_1) \geq 1-\eta$. The rest of the proof is constructed conditioning on $\mathcal{E}_1$ happening. We thus have that 
    \begin{align*}
        \sum_{j=1}^J(\theta_{a,1,j}-\theta_{a,0,j})^2 \lambda_{a,j}\Big|\frac{1}{\widehat{\lambda}_{a,j}} - \frac{1}{\lambda_{a,j}}\Big|^2 &\lesssim \sum_{j=1}^J(\theta_{a,1,j}-\theta_{a,0,j})^2 \lambda_{a,j} \cdot \frac{1}{\lambda_{a,j}^4}\frac{j^{-2\alpha}\log(\widetilde{n}/\eta)}{\widetilde{n}}\\
        &\lesssim \frac{\log(\widetilde{n}/\eta)}{\widetilde{n}} \sum_{j=1}^J\frac{(\theta_{a,1,j}-\theta_{a,0,j})^2 }{\lambda_{a,j}}\asymp \frac{\log(\widetilde{n}/\eta)}{\widetilde{n}}.
    \end{align*}
    The lemma thus follows.
\end{proof}

\begin{lemma} \label{l_rkhs_approximation_const}
    Under the same condition of \Cref{l_T_diff_const}, for any small constant $\eta \in (0,1/2)$, it holds with probability at least $1-\eta$ that
    \begin{equation*}
        \Big|\sum_{j=1}^J \frac{(\widehat{\theta}_{a,1,j}-\widehat{\theta}_{a,0,j})^2}{\widehat{\lambda}_{a,j}} -\sum_{j=1}^J \frac{(\theta_{a,1,j}-\theta_{a,0,j})^2}{\lambda_{a,j}}\Big|
        \lesssim \begin{cases}
            \sqrt{\frac{J^{2}\log^2(J)\log(\widetilde{n}/\eta)}{\widetilde{n}}}, \quad & \text{when}\;\;\frac{\alpha+1}{2}<\beta\leq \frac{\alpha+2}{2},\\
            \sqrt{\frac{\log(\widetilde{n}/\eta)}{\widetilde{n}}}, & \text{when}\;\; \beta > \frac{\alpha+2}{2}.
        \end{cases}
    \end{equation*}
\end{lemma}

\begin{proof}
    By triangle inequality, we have that
    \begin{align} \notag
        &\Big|\sum_{j=1}^J \frac{(\widehat{\theta}_{a,1,j}-\widehat{\theta}_{a,0,j})^2}{\widehat{\lambda}_{a,j}} -\sum_{j=1}^J \frac{(\theta_{a,1,j}-\theta_{a,0,j})^2}{\lambda_{a,j}}\Big|\\ \notag
        \lesssim \;& \Big|\sum_{j=1}^J \frac{(\widehat{\theta}_{a,1,j}-\widehat{\theta}_{a,0,j})^2-(\theta_{a,1,j}-\theta_{a,0,j})^2}{\widehat{\lambda}_{a,j}}\Big| +\Big|\sum_{j=1}^J\frac{(\theta_{a,1,j}-\theta_{a,0,j})^2}{\widehat{\lambda}_{a,j}} -\sum_{j=1}^J \frac{(\theta_{a,1,j}-\theta_{a,0,j})^2}{\lambda_{a,j}}\Big|\\ \notag
        =\;&\Big|\sum_{j=1}^J \frac{(\widehat{\theta}_{a,1,j}-\widehat{\theta}_{a,0,j}+\theta_{a,1,j}-\theta_{a,0,j})(\widehat{\theta}_{a,1,j}-\widehat{\theta}_{a,0,j}-\theta_{a,1,j}+\theta_{a,0,j})}{\widehat{\lambda}_{a,j}}\Big| \\\notag
        &+\Big|\sum_{j=1}^J\frac{(\theta_{a,1,j}-\theta_{a,0,j})^2}{\widehat{\lambda}_{a,j}} -\sum_{j=1}^J \frac{(\theta_{a,1,j}-\theta_{a,0,j})^2}{\lambda_{a,j}}\Big|\\ \notag
        \lesssim \;&\sum_{j=1}^J \frac{(|\widehat{\theta}_{a,1,j}-\widehat{\theta}_{a,0,j}|+|\theta_{a,1,j}-\theta_{a,0,j}|)\cdot|\widehat{\theta}_{a,1,j}-\widehat{\theta}_{a,0,j}-\theta_{a,1,j}+\theta_{a,0,j}|}{\widehat{\lambda}_{a,j}}\\ \label{l_rkhs_approximation_const_eq1}
        &+\Big|\sum_{j=1}^J\frac{(\theta_{a,1,j}-\theta_{a,0,j})^2}{\widehat{\lambda}_{a,j}} -\sum_{j=1}^J \frac{(\theta_{a,1,j}-\theta_{a,0,j})^2}{\lambda_{a,j}}\Big|.
    \end{align}
    
    Consider the events 
    \begin{equation}\label{l_rkhs_approximation_const_eq3}
        \mathcal{E}_1= \Big\{ \lambda_{a,j}/2 \leq \widehat{\lambda}_{a,j} \leq  3\lambda_{a,j}/2 \;\; \text{and} \;\; \Big|\frac{1}{\widehat{\lambda}_{a,j}} - \frac{1}{\lambda_{a,j}}\Big| \lesssim \frac{1}{\lambda_{a,j}^2}\sqrt{\frac{j^{-2\alpha}\log(\widetilde{n}/\eta)}{\widetilde{n}}},\quad \text{for} \; j \in [J],a \in \{0,1\}\Big\},
    \end{equation} 
    \[\mathcal{E}_2 = \Big\{|\widehat{\theta}_{a,1,j}-\widehat{\theta}_{a,0,j}| \leq j^{-\beta},\quad \text{for} \; j \in [J],a \in \{0,1\}, \frac{\alpha+1}{2} < \beta \leq \frac{3\alpha+2}{2}\Big\},\]
    and
    \[\mathcal{E}_3 = \Big\{|\widehat{\theta}_{a,1,j}-\widehat{\theta}_{a,0,j}| \lesssim  \sqrt{\frac{j^{-\alpha}\log(\widetilde{n}/\eta)}{\widetilde{n}}},\quad \text{for} \; j \in [J],a \in \{0,1\},\beta > \frac{3\alpha+2}{2}\Big\}.\]
    Note that by \eqref{eq_eigenvalue_error}, it holds with probability at least $1-\eta/6$ that 
    \begin{align*} 
        \frac{\lambda_{a,j}}{2}\lesssim \lambda_{a,j} -  \sqrt{\frac{j^{-2\alpha}\log(\widetilde{n}/\eta)}{\widetilde{n}}}\leq \widehat{\lambda}_{a,j} \lesssim \lambda_{a,j} + \sqrt{\frac{j^{-2\alpha}\log(\widetilde{n}/\eta)}{\widetilde{n}}} \leq \frac{3\lambda_{a,j}}{2},
    \end{align*}
    where the first and the fourth inequality follow from the fact that 
    \begin{align*}
         \sqrt{\frac{j^{-2\alpha}\log(\widetilde{n}/\eta)}{\widetilde{n}}} \lesssim j^{-\alpha} \leq \frac{\lambda_{a,j}}{2}.
    \end{align*}
    By \Cref{l_eigenvalue}, we also have with probability at least $1-\eta/6$ that for all $j \in [J]$,
    \begin{equation*} 
        \Big|\frac{1}{\widehat{\lambda}_{a,j}} - \frac{1}{\lambda_{a,j}}\Big| = \frac{|\lambda_{a,j}-\widehat{\lambda}_{a,j}|}{\lambda_{a,j}\widehat{\lambda}_{a,j}} \lesssim \frac{1}{\lambda_{a,j}^2}\sqrt{\frac{j^{-2\alpha}\log(\widetilde{n}/\eta)}{\widetilde{n}}} \lesssim \frac{1}{\lambda_{a,j}}\sqrt{\frac{\log(\widetilde{n}/\eta)}{\widetilde{n}}}.
    \end{equation*}
    Additionally, by \Cref{l_score_const}, we have with probability at least $1-\eta/3$ that for all $j \in [J]$, 
    \begin{align*}
        |\widehat{\theta}_{a,1,j}-\widehat{\theta}_{a,0,j}| &\lesssim \begin{cases}
        |\theta_{a,1,j}-\theta_{a,0,j}| + \sqrt{\frac{j^{2-2\beta}\log^2(j)\log(\widetilde{n}/\eta)}{\widetilde{n}}}, \quad &\text{when}\; \frac{\alpha+1}{2} < \beta \leq \frac{\alpha+2}{2},\\
        |\theta_{a,1,j}-\theta_{a,0,j}|+ \sqrt{\frac{j^{-\alpha}\log(\widetilde{n}/\eta)}{\widetilde{n}}},  \quad &\text{when}\;  \beta > \frac{\alpha+2}{2},
    \end{cases} \\
        & \lesssim \begin{cases}
         j^{-\beta}, \quad &\text{when}\; \frac{\alpha+1}{2} < \beta \leq \frac{3\alpha+2}{2},\\
         j^{-\beta} \vee \sqrt{\frac{j^{-\alpha}\log(\widetilde{n}/\eta)}{\widetilde{n}}},  \quad &\text{when}\;  \beta > \frac{3\alpha+2}{2},
    \end{cases}
    \end{align*}
    where the second inequality follows from the fact that 
    \begin{itemize}
        \item When $\frac{\alpha+1}{2} < \beta \leq \frac{\alpha+2}{2}$, it holds that 
        \begin{align} \label{l_rkhs_approximation_const_eq5}
            \sqrt{\frac{j^{2-2\beta}\log^2(j)\log(\widetilde{n}/\eta)}{\widetilde{n}}} \asymp j^{-\beta}\sqrt{\frac{j^{2}\log^2(j)\log(\widetilde{n}/\eta)}{\widetilde{n}}} \lesssim j^{-\beta}.
        \end{align}

        \item  When $\frac{\alpha+2}{2} < \beta \leq \frac{3\alpha+2}{2}$, it holds that
        \begin{align*}
            \sqrt{\frac{j^{-\alpha}\log(\widetilde{n}/\eta)}{\widetilde{n}}} \asymp j^{-\beta}\sqrt{\frac{j^{2\beta-\alpha}\log(\widetilde{n}/\eta)}{\widetilde{n}}} \lesssim j^{-\beta}\sqrt{\frac{j^{3\alpha+2-\alpha}\log(\widetilde{n}/\eta)}{\widetilde{n}}} \lesssim j^{-\beta}.
        \end{align*}

        \item When $\beta > \frac{3\alpha+2}{2}$, different term dominates depending on the value of $j \in [J]$.
    \end{itemize}

    Therefore, by a union bound argument, we have that $\mathbb{P}(\mathcal{E}_k) \geq 1-\eta/3$ for any $k \in \{1,2,3\}$. 
    In addition, consider the following disjoint events,
    \[\mathcal{E}_4=\Bigg\{|\widehat{\theta}_{a,1,j} - \theta_{a,1,j}+\widehat{\theta}_{a,0,j} - \theta_{a,0,j}| \lesssim \sqrt{\frac{j^{2-2\beta}\log^2(j)\log(\widetilde{n}/\eta)}{\widetilde{n}}},\;  j \in [J],a \in \{0,1\}, \frac{\alpha+1}{2} < \beta \leq \frac{\alpha+2}{2}\Bigg\},\]
    \[\mathcal{E}_5 = \Bigg\{|\widehat{\theta}_{a,1,j} - \theta_{a,1,j}+\widehat{\theta}_{a,0,j} - \theta_{a,0,j}| \lesssim \sqrt{\frac{j^{-\alpha}\log(\widetilde{n}/\eta)}{\widetilde{n}}},\; j \in [J], a \in \{0,1\}, \beta >\frac{\alpha+2}{2}\Bigg\},\]
    and by \Cref{l_score_const}, we have that for each $k \in \{4,5\}$, $\mathbb{P}(\mathcal{E}_k) \geq 1-\eta/3$. In the rest of the proof, we will consider $3$ different cases conditioning on various events happening based on the range of $\beta$.

    \noindent \textbf{Case 1: When $(\alpha+1)/2<\beta\leq (\alpha+2)/2$.} In this case, the proof is constructed conditioning on $\mathcal{E}_1 \cap \mathcal{E}_2 \cap \mathcal{E}_4$ happening and by a union bound argument we have that $\mathbb{P}(\mathcal{E}_1 \cap \mathcal{E}_2 \cap \mathcal{E}_4) \geq 1-\eta$. Consequently, it holds that 
    \begin{align*}
        \eqref{l_rkhs_approximation_const_eq1} & \lesssim \sum_{j=1}^J \frac{j^{-\beta}}{\lambda_{a,j}}\sqrt{\frac{j^{2-2\beta}\log^2(j)\log(\widetilde{n}/\eta)}{\widetilde{n}}}+\sum_{j=1}^J\Big|\frac{1}{\widehat{\lambda}_{a,j}} - \frac{1}{\lambda_{a,j}}\Big|(\theta_{a,1,j}-\theta_{a,0,j})^2\\
        & \lesssim \sqrt{\frac{\log^2(J)\log(\widetilde{n}/\eta)}{\widetilde{n}}}\sum_{j=1}^J j^{\alpha-2\beta+1} + \sqrt{\frac{\log(\widetilde{n}/\eta)}{\widetilde{n}}}\sum_{j=1}^J \frac{(\theta_{a,1,j}-\theta_{a,0,j})^2}{\lambda_{a,j}}\\
        & \lesssim \sqrt{\frac{J^{2}\log^2(J)\log(\widetilde{n}/\eta)}{\widetilde{n}}},
    \end{align*}
    where the last inequality follows from the fact that $-1 \leq \alpha-2\beta+1 <0$.

    \noindent\textbf{Case 2: When $(\alpha+2)/2 < \beta \leq (3\alpha+2)/2$.} In this case, the proof is constructed conditioning on $\mathcal{E}_1 \cap \mathcal{E}_2 \cap \mathcal{E}_5$ happening and by a union bound argument we have that $\mathbb{P}(\mathcal{E}_1 \cap \mathcal{E}_2 \cap \mathcal{E}_5) \geq 1-\eta$. Consequently, it holds that 
    \begin{align*}
        \eqref{l_rkhs_approximation_const_eq1} & \lesssim  \sum_{j=1}^J \frac{j^{-\beta}}{\lambda_{a,j}}\sqrt{\frac{j^{-\alpha}\log(\widetilde{n}/\eta)}{\widetilde{n}}} + \sqrt{\frac{\log(\widetilde{n}/\eta)}{\widetilde{n}}}\sum_{j=1}^J \frac{(\theta_{a,1,j}-\theta_{a,0,j})^2}{\lambda_{a,j}}\\
        & = \sqrt{\frac{\log(\widetilde{n}/\eta)}{\widetilde{n}}}\sum_{j=1}^J j^{\frac{\alpha}{2}-\beta} + \sqrt{\frac{\log(\widetilde{n}/\eta)}{\widetilde{n}}}\\
        & \lesssim \sqrt{\frac{\log(\widetilde{n}/\eta)}{\widetilde{n}}},
    \end{align*}
    where the last inequality follows from the fact that $\alpha/2-\beta <-1$.

    \noindent\textbf{Case 3: When $\beta > (3\alpha+2)/2$.} In this case, the proof is constructed conditioning on $\mathcal{E}_1 \cap \mathcal{E}_3 \cap \mathcal{E}_5$ happening and by a union bound argument we have that $\mathbb{P}(\mathcal{E}_1 \cap \mathcal{E}_3 \cap \mathcal{E}_5) \geq 1-\eta$. Consequently, it holds that 
    \begin{align*}
        \eqref{l_rkhs_approximation_const_eq1} & \lesssim \sum_{j=1}^J \frac{1}{\lambda_{a,j}}\Big\{j^{-\beta} \vee \sqrt{\frac{j^{-\alpha}\log(\widetilde{n}/\eta)}{\widetilde{n}}}\Big\}\sqrt{\frac{j^{-\alpha}\log(\widetilde{n}/\eta)}{\widetilde{n}}} + \sqrt{\frac{\log(\widetilde{n}/\eta)}{\widetilde{n}}}\\
        & =\sum_{j=1}^J \frac{1}{\lambda_{a,j}}\Big\{\sqrt{\frac{j^{-2\beta -\alpha}\log(\widetilde{n}/\eta)}{\widetilde{n}}}\vee \frac{j^{-\alpha}\log(\widetilde{n}/\eta)}{\widetilde{n}}\Big\} + \sqrt{\frac{\log(\widetilde{n}/\eta)}{\widetilde{n}}}\\
        & \lesssim \sum_{j=1}^J j^{\frac{\alpha}{2}-\beta}\sqrt{\frac{\log(\widetilde{n}/\eta)}{\widetilde{n}}} + \sum_{j=1}^J \frac{\log(\widetilde{n}/\eta)}{\widetilde{n}} + \sqrt{\frac{\log(\widetilde{n}/\eta)}{\widetilde{n}}}\\
        & \lesssim \frac{J\log(\widetilde{n}/\eta)}{\widetilde{n}} + \sqrt{\frac{\log(\widetilde{n}/\eta)}{\widetilde{n}}}\lesssim \sqrt{\frac{\log(\widetilde{n}/\eta)}{\widetilde{n}}},
    \end{align*}
    where the third inequality follows from the fact that $\alpha/2-\beta < -1$ and the last inequality follows whenever $J^{2\alpha+2}\log^2(J)\log(\widetilde{n}/\eta)\lesssim \widetilde{n}$.The lemma thus follows by combining results from three cases. 
\end{proof}

\section{Proofs for class probability estimation} \label{appendix_empirical_prob}
In this section, we present auxiliary lemmas related to class probabilities. Results below holds for any $a,y \in \{0,1\}$.

\begin{lemma}\label{l_emperical_prob}
    Under Assumption \ref{a_class_prob}, for any small $\epsilon >0$ and $a,y \in \{0,1\}$, it holds that 
    \[\mathbb{P}\Big(|\widehat{\pi}_{a,y} - \pi_{a,y}| \geq \epsilon\Big) \lesssim \exp\big(-\widetilde{n}\epsilon^2\big).\]
\end{lemma}
\begin{proof}
    Consider the sequence of bounded random variables $\{\indc\{Y_j=y, A_j =a\}\}_{j=1}^{\widetilde{n}}$, then it holds~that
    \begin{equation*} 
        \widehat{\pi}_{a,y}= \frac{\widetilde{n}_{a,y}}{\widetilde{n}} = \frac{1}{\widetilde{n}}\sum_{i=1}^{\widetilde{n}}  \indc\{Y_j=y, A_j =a\}.
    \end{equation*}
    Therefore, the lemma follows by applying Hoeffding’s inequality for general bounded random variables \citep[e.g.~Theorem 2.2.6 in][]{vershynin2018high}
\end{proof}

Note that in the lemma below, the constant $1/5$ is arbitrary. 
\begin{lemma}\label{l_n_ay}
    With $0< C_p,C_p'<1/5$ being the absolute constants in \Cref{a_class_prob}, consider the following events,
    \begin{equation*}
        \mathcal{E}_1 = \Big\{\Big(\frac{C_p}{2} \wedge \frac{(1-C'_p)}{2}\Big)\widetilde{n} \leq \widetilde{n}_{a,y}\leq \widetilde{n}, \;\; \text{for all}\;\; a, y \in \{0,1\}\Big\},
    \end{equation*}
    and
    \[\mathcal{E}_2 = \Big\{\Big(\frac{C_p}{2} \wedge \frac{(1-C'_p)}{2}\Big)n \leq n_{a,y}\leq n, \;\; \text{for all}\;\; a, y \in \{0,1\}\Big\}.\]
    We have that $\mathbb{P}(\mathcal{E}_1\cap \mathcal{E}_2) \geq 1-\eta$ for some small $\eta \in (0,1/2)$ whenever $\{\widetilde{n}\wedge n\} \gtrsim \log(1/\eta)$.
\end{lemma}

\begin{proof}
    For $y \in \{0,1\}$ and $a\in \{0,1\}$, consider the sequence of bounded random variables $\{\indc\{Y_i=y, A_i =a\}\}_{i=1}^{\widetilde{n}}$. By Hoeffding’s inequality for general bounded random variables \citep[e.g.~Theorem 2.2.6 in][]{vershynin2018high}, we have that for any $\epsilon_1,\epsilon_2 >0$,
    \[\mathbb{P}\Big\{\Big|\sum_{i=1}^{\widetilde{n}}\indc\{Y_i=1, A_i =1\} -\pi_{1,1}\widetilde{n}\Big| \geq \epsilon_1\Big\} \leq \exp\Big(-\frac{\epsilon_1^2}{\widetilde{n}}\Big),\]
    and
    \[\mathbb{P}\Big\{\Big|\sum_{i=1}^{\widetilde{n}}\indc\{Y_i=0, A_i=1\} -(1-\pi_{1,1})\widetilde{n} \geq \epsilon_2\Big|\Big\}\leq \exp\Big(-\frac{\epsilon_2^2}{\widetilde{n}}\Big).\]
    Therefore, we have that with probability at least $1-\eta/4$ that 
    \begin{align*}
        n_{1,1} = \sum_{i=1}^{\widetilde{n}}\indc\{Y_i=1, A=1\} \geq \pi_{1,1} \widetilde{n} -\sqrt{\widetilde{n}\log(1/\eta)} \geq C_p \widetilde{n} - \frac{C_p}{2}\widetilde{n} \geq \frac{C_p\widetilde{n}}{2},
    \end{align*}
    whenever $\widetilde{n} \geq 4\log(1/\eta)/C_p^2$. Similarly, we also have with probability at least $1-\eta/4$ that
    \begin{align*}
        n_{1,0} &= \sum_{i=1}^{\widetilde{n}}\indc\{Y_i=0, A=1\} \geq (1-\pi_0- \pi_{1,1}) \widetilde{n} -\sqrt{\widetilde{n}\log(1/\eta)} \\
        &\geq (1-3C'_p) \widetilde{n} - \frac{(1-5C'_p)}{2}\widetilde{n} \geq \frac{(1-C'_p)\widetilde{n}}{2},
    \end{align*}
    whenever $\widetilde{n} \geq 4\log(1/\eta)/(1-5C'_p)^2$. The other cases can be justified similarly. The lemma thus follows from a union bound argument.
\end{proof}

\section{Proofs for functional data estimation} \label{appendix_func_estim}
In this section, we present auxiliary lemmas related to mean, covariance, eigenvalue, eigenfunction and score estimation for the training data $\widetilde{\mathcal{D}}$. Denote $\widetilde{n}_a = \widetilde{n}_{a,0} + \widetilde{n}_{a,1}$ and the group-wise mean and covariance function by 
\begin{equation*}
    \widehat{\mu}_{a, y}(t) = \frac{1}{\widetilde n_{a, y}}\sum_{i=1}^{\widetilde n_{a, y}} \widetilde X_{a, y}^i(t),
\end{equation*}
and
\begin{equation*}
    \widehat{K}_{a}(s,t) = \sum_{y \in \{0,1\}} \frac{\widetilde n_{a,y}}{\widetilde n_{a,0}+\widetilde n_{a,1}} \frac{1}{\widetilde n_{a, y} -1} \sum_{i=1}^{\widetilde n_{a,y}}\{\widetilde X^i_{a,y}(s) - \widehat{\mu}_{a,y}(s)\}\{\widetilde X^i_{a,y}(t) - \widehat{\mu}_{a,y}(t)\}. 
\end{equation*}
We further let $\{\widehat{\lambda}_{a,j}\}_{j \geq 1}$ and $\{\widehat{\phi}_{a,j}\}_{j\geq 1}$ denote the eigenvalues and eigenfunctions of $\widehat{K}_a$ obtained by spectral expansion. The lemmas in the rest of the section holds for all $a\in \{0,1\}$ and $y \in \{0,1\}$.

\subsection{Mean and covariance function}
\begin{lemma}[Lemma 1 in \citealp{zapata2022partial}] \label{l_mean_cov_estimation}
    Assume Assumptions \ref{a_class_prob} and \ref{a_data}\ref{a_data_cov}~hold. For any small $\epsilon_1,\epsilon_2 >0$, it holds that 
    \begin{align*}
        \mathbb{P}\Big(\|\widehat{\mu}_{a,y} - \mu_{a,y}\|_{L^2} \geq \epsilon_1\Big) \lesssim \exp(-\widetilde{n}_{a,y}\epsilon_1^2) \quad \text{and} \quad \mathbb{P}\Big(\|\widehat{K}_{a,y} - K_a\|_{L^2} \geq \epsilon_2\Big) \lesssim \exp(-\widetilde{n}_{a,y}\epsilon_2^2),
    \end{align*}
    where $\widehat{K}_{a,y} = 1/(\widetilde n_{a, y} -1) \sum_{i=1}^{\widetilde n_{a,y}}\{\widetilde X^i_{a,y}(s) - \widehat{\mu}_{a,y}(s)\}\{\widetilde X^i_{a,y}(t) - \widehat{\mu}_{a,y}(t)\}$.
\end{lemma}

\begin{proof}
    By standard properties of Gaussian processes and \Cref{a_data}\ref{a_data_cov}, Assumption 2 in \citet{zapata2022partial} is automatically satisfied. Hence, the lemma follows. 

\end{proof}
\begin{lemma} \label{l_average_cov}
    Assume Assumptions \ref{a_class_prob} and \ref{a_data}\ref{a_data_cov}~hold. It holds for any small $0 <\epsilon \lesssim 1$ that 
    \[\mathbb{P}\Big(\|\widehat{K}_a - K_a\|_{L^2} \geq \epsilon \Big) \lesssim \exp(-\widetilde{n}_a\epsilon^2).\]
\end{lemma}

\begin{proof}
    Note that 
    \begin{align} \notag
        &\widehat{K}_a (s,t) - K_a(s,t) \\ \notag
        \asymp\;& \frac{1}{\widetilde{n}_a}\sum_{y \in \{0,1\}} \sum_{i=1}^{n_{a,y}}\Big[\big\{\widetilde{X}^i_{a,y}(s) - \mu_{a,y}(s)\big\}-\big\{\widehat{\mu}_{a,y}(s)-\mu_{a,y}(s) \big\}\Big] \\ \notag
        & \cdot\Big[\big\{\widetilde{X}^i_{a,y}(t) - \mu_{a,y}(t)\big\}-\big\{\widehat{\mu}_{a,y}(t)-\mu_{a,y}(t)\big\}\Big]- \mathbb{E}\Big\{\big\{\widetilde{X}^i_{a,y}(s) - \mu_{a,y}(s)\big\}\big\{\widetilde{X}^i_{a,y}(t) - \mu_{a,y}(t)\big\}\Big\}\\ \notag
        \asymp \;& \frac{1}{\widetilde{n}_a}\sum_{y \in \{0,1\}} \sum_{i=1}^{n_{a,y}}\Big[\big\{\widetilde{X}^i_{a,y}(s) - \mu_{a,y}(s)\big\}\big\{\widetilde{X}^i_{a,y}(t) - \mu_{a,y}(t)\big\}\\  \notag
        & \hspace{5cm}- \mathbb{E}\Big\{\big\{\widetilde{X}^i_{a,y}(s) - \mu_{a,y}(s)\big\}\big\{\widetilde{X}^i_{a,y}(t) - \mu_{a,y}(t)\big\}\Big\}\Big]\\\label{l_average_cov_4}
        & \hspace{5cm}- \frac{1}{\widetilde{n}_a}\sum_{y \in \{0,1\}}\widetilde{n}_{a,y}\big\{\widehat{\mu}_{a,y}(s)-\mu_{a,y}(s)\big\}\big\{\widehat{\mu}_{a,y}(t)-\mu_{a,y}(t)\big\}.
    \end{align}
    Therefore, triangle inequality implies that 
    \begin{align} \notag
        &\|\widehat{K}_a(s,t) - K_a(s,t)\|_{L^2} \\ \notag
        \lesssim\;& \Big\|\frac{1}{\widetilde{n}_a}\sum_{y \in \{0,1\}} \sum_{i=1}^{n_{a,y}}\Big[\big\{\widetilde{X}^i_{a,y}(s) - \mu_{a,y}(s)\big\}\big\{\widetilde{X}^i_{a,y}(t) - \mu_{a,y}(t)\big\}\\ \notag
        &\hspace{5cm} - \mathbb{E}\Big\{\big\{\widetilde{X}^i_{a,y}(s) - \mu_{a,y}(s)\big\}\big\{\widetilde{X}^i_{a,y}(t) - \mu_{a,y}(t)\big\}\Big\}\Big]\Big\|_{L^2}\\ \notag
        &+ \frac{1}{\widetilde{n}_a}\sum_{y \in \{0,1\}}\widetilde{n}_{a,y} \sqrt{\int \int \big\{\widehat{\mu}_{a,y}(s)-\mu_{a,y}(s)\big\}^2\big\{\widehat{\mu}_{a,y}(t)-\mu_{a,y}(t)\big\}^2 \;\mathrm{d}s\;\mathrm{d}t}\\ \notag
        =\;&\Big\|\frac{1}{\widetilde{n}_a}\sum_{y \in \{0,1\}} \sum_{i=1}^{n_{a,y}}\Big[\big\{\widetilde{X}^i_{a,y}(s) - \mu_{a,y}(s)\big\}\big\{\widetilde{X}^i_{a,y}(t) - \mu_{a,y}(t)\big\}\\ \notag
        & \hspace{5cm} - \mathbb{E}\Big\{\big\{\widetilde{X}^i_{a,y}(s) - \mu_{a,y}(s)\big\}\big\{\widetilde{X}^i_{a,y}(t) - \mu_{a,y}(t)\big\}\Big\}\Big]\Big\|_{L^2}\\ \notag
        &+\frac{1}{\widetilde{n}_a}\sum_{y \in \{0,1\}}\widetilde{n}_{a,y} \|\widehat{\mu}_{a,y}- \mu_{a,y}\|_{L^2}^2\\ \label{l_average_cov_3}
        =\;& (I) + (II).
    \end{align}
    To control $(I)$, note we have for any $\ell, k \in \mathbb{N}_+$ that 
    \begin{align*}
        &\int \int \Big[\big\{\widetilde{X}^i_{a,y}(s) - \mu_{a,y}(s)\big\}\big\{\widetilde{X}^i_{a,y}(t) - \mu_{a,y}(t)\big\}\\
        & \hspace{5cm}- \mathbb{E}\Big\{\big\{\widetilde{X}^i_{a,y}(s) - \mu_{a,y}(s)\big\}\big\{\widetilde{X}^i_{a,y}(t) - \mu_{a,y}(t)\big\}\Big\}\Big] \phi_{a,\ell}(s)\phi_{a,k}(t)\;\mathrm{d}s\;\mathrm{d}t\\
        =\;& \frac{1}{\widetilde{n}_a} \sum_{y \in \{0,1\}}\sum_{i=1}^{n_{a,y}}\big(\xi^i_{a,y,\ell}\xi^i_{a,y,k} -  \lambda_{a,\ell}\indc\{\ell=k\}\big),
    \end{align*}
    where for $\ell \in \mathbb{N}_+$, $\xi^i_{a,y,\ell} = \int \{\widetilde{X}^i_{a,y}(s) - \mu_{a,y}(s)\}\phi_{a,\ell}(s) \;\mathrm{d}s$. Therefore, by \Cref{l_bosq_thm2.5} and a similar argument as the one in Lemma 1 in \citet{zapata2022partial} or Lemma 6 in \citet{qiao2019functional}, we have that for some small $0< \epsilon \lesssim 1$,
    \begin{equation} \label{l_average_cov_1}
        \mathbb{P}\Big((I) \geq \epsilon\Big) \lesssim \exp(-\widetilde{n}_a\epsilon^2). 
    \end{equation}
    To control $(II)$, by \Cref{l_mean_cov_estimation}, it holds that for $y \in \{0,1\}$,
    \begin{equation}\label{l_average_cov_2}
        \mathbb{P}\Big(\|\widehat{\mu}_{a,y}- \mu_{a,y}\|_{L^2} \geq \sqrt{\frac{\widetilde{n}_a\epsilon^2}{\widetilde{n}_{a,y}}} \Big)  \lesssim \exp (-\widetilde{n}_a\epsilon^2).
    \end{equation}
    Applying a union bound argument and substituting \eqref{l_average_cov_1} and \eqref{l_average_cov_2} into \eqref{l_average_cov_3}, we have that with probability at least $1- \exp (-\widetilde{n}\epsilon^2)$,
    \begin{align*}
        \|\widehat{K}_a - K_a\|_{L^2}  \lesssim \epsilon + \frac{1}{\widetilde{n}_a}\sum_{y \in \{0,1\}} \frac{\widetilde{n}_{a,y}\widetilde{n}_a\epsilon^2}{\widetilde{n}_{a,y}} =  \epsilon + \frac{2\widetilde{n}_a\epsilon^2}{\widetilde{n}_a} \lesssim \epsilon + \epsilon^2 \lesssim \epsilon,
    \end{align*}
    where the last inequality follows from the fact that $\epsilon \lesssim 1$. Thus, the Lemma follows. 
\end{proof}

\subsection{Eigenfunction}
\begin{lemma} \label{l_eigenfunc_estimation}
    Assume Assumptions \ref{a_class_prob} and \ref{a_data}\ref{a_data_cov}~hold. For any $j \in \mathbb{N}_+$ such that $j \leq J_{\widetilde{n}}$ where $J_{\widetilde{n}}>0$ is a function of $\widetilde{n}$ and any small $0 < \epsilon \lesssim jJ_{\widetilde{n}}^{-(\alpha+1)}$, we have that
    \[\mathbb{P}\Big(\|\widehat{\phi}_{a,j} - \phi_{a,j}\|_{L^2} \geq \epsilon\Big) \lesssim \exp\Big(-\frac{\widetilde{n}_a\epsilon^2}{j^2}\Big).\]
\end{lemma}

\begin{proof}
    The proofs share the same spirit as the one in \citet{hall2007methodology}.  Note that by \Cref{l_eigen_split} and (5.16) in \citet{hall2007methodology}, it holds that
    \begin{align} \notag
        &\|\widehat{\phi}_{a,j} - \phi_{a,j}\|_{L^2}\\ \notag
        \leq\;& \sqrt{2} \Big\|\sum_{k:k \neq j} \phi_{a,k}(\cdot)(\widehat{\lambda}_{a,j} - \lambda_{a,k})^{-1}\int\int\big\{\widehat{K}_a(s,t)-K_a(s,t)\big\}\widehat{\phi}_{a,j}(s)\phi_{a,k}(t)\;\mathrm{d}s\;\mathrm{d}t\Big\|_{L^2}\\ \notag
        \leq\;& \sqrt{2}\Big\|\sum_{k:k \neq j} \phi_{a,k}(\cdot)(\widehat{\lambda}_{a,j} - \lambda_{a,k})^{-1}\int\int\big\{\widehat{K}_a(s,t)-K_a(s,t)\big\}\phi_{a,j}(s)\phi_{a,k}(t)\;\mathrm{d}s\;\mathrm{d}t\Big\|_{L^2}\\ \label{eq_eigenfunc_estimation_3}
        &+\sqrt{2}\Big\|\sum_{k:k \neq j} \phi_{a,k}(\cdot)(\widehat{\lambda}_{a,j} - \lambda_{a,k})^{-1}\int\int\big\{\widehat{K}_a(s,t)-K_a(s,t)\big\}\big\{\widehat{\phi}_{a,j}(s)- \phi_{a,j}(s)\big\}\phi_{a,k}(t)\;\mathrm{d}s\;\mathrm{d}t\Big\|_{L^2},
    \end{align}
    where the second inequality follows from the triangle inequality. Next, consider the event $\mathcal{E}_1 = \{\|\widehat{K}_a - K_a\|_{L^2} \leq \epsilon/j\}$ and we have that $\mathbb{P}(\mathcal{E}_1) \gtrsim 1- \exp(-\widetilde{n}\epsilon^2/j^2) $ by \Cref{l_average_cov}. The rest of the proof is constructed conditioning on $\mathcal{E}_1$. Construct another event
    \begin{equation} \label{eq_l_eigenfunc_estimation_7}
        \mathcal{E}_2 = \big\{(\widehat{\lambda}_{a,j} - \lambda_{a,k})^{-2} \leq 2 (\lambda_{a,j} - \lambda_{a,k})^{-2} \lesssim J_{\widetilde{n}}^{2(\alpha+1)}, \quad \text{for}\; k,j \in [J_{\widetilde{n}}]\big\}.
    \end{equation}
    We want to show  $\mathcal{E}_2$ holds. By Weyl's inequality, it holds that $|\widehat{\lambda}_{a,j}-\lambda_{a,j}| \leq \|\widehat{K}_a - K\|_{L^2} \leq \epsilon/j$ for any $j \in \mathbb{N}_+$. This then implies that
    \begin{align} \notag
        |\widehat{\lambda}_{a,j} - \lambda_{a,k}| &\geq |\lambda_{a,k} - \lambda_{a,j}|-|\widehat{\lambda}_{a,j} - \lambda_{a,j}|  \geq |\lambda_{a,k} - \lambda_{a,j}| - \epsilon/j\\ \label{eq_l_eigenfunc_estimation_1}
        & \geq |\lambda_{a,k} - \lambda_{a,j}| - (1-2^{-1/2})|\lambda_{a,k} - \lambda_{a,j}| \geq 2^{-1/2}|\lambda_{a,k} - \lambda_{a,j}|,
    \end{align}
    where the first inequality follows from triangle inequality and the third inequality follows from \Cref{a_data}\ref{a_data_cov}~and the fact that for $0 < \epsilon \leq (1-2^{-1/2})jJ_{\widetilde{n}}^{-(\alpha+1)}$, we have that 
    \[\frac{\epsilon}{j} \leq (1-2^{-1/2})J_{\widetilde{n}}^{-(\alpha+1)} \lesssim (1-2^{-1/2})|\lambda_{a,k} - \lambda_{a,j}|.\]
    Thus, from \eqref{eq_l_eigenfunc_estimation_1}, we have that conditioning on $\mathcal{E}_1$, $\mathcal{E}_2$ holds with probability $1$. The rest of the proof is constructed conditioning on both $\mathcal{E}_1$ and $\mathcal{E}_2$. To control \eqref{eq_eigenfunc_estimation_3}, we can further upper bound it by
    \begin{align*}
        \eqref{eq_eigenfunc_estimation_3} =\;& \sqrt{2}\Big[\sum_{k:k \neq j} (\widehat{\lambda}_{a,j} - \lambda_{a,k})^{-2}\Big\{\int\int\big\{\widehat{K}_a(s,t)-K_a(s,t)\big\}\phi_{a,j}(s)\phi_{a,k}(t)\;\mathrm{d}s\;\mathrm{d}t\Big\}^2\Big]^{\frac{1}{2}}\\ \notag
        &+\sqrt{2}\Big[\sum_{k:k \neq j} (\widehat{\lambda}_{a,j} - \lambda_{a,k})^{-2}\Big\{\int\int\big\{\widehat{K}_a(s,t)-K_a(s,t)\big\}\big\{\widehat{\phi}_{a,j}(s)- \phi_{a,j}(s)\big\}\phi_{a,k}(t)\;\mathrm{d}s\;\mathrm{d}t\Big\}^2\Big]^{\frac{1}{2}}\\
        \leq \;& 2\Big[\sum_{k:k \neq j} (\lambda_{a,j} - \lambda_{a,k})^{-2}\Big\{\int\int\big\{\widehat{K}_a(s,t)-K_a(s,t)\big\}\phi_{a,j}(s)\phi_{a,k}(t)\;\mathrm{d}s\;\mathrm{d}t\Big\}^2\Big]^{\frac{1}{2}}\\ \notag
        &+2\Big[\sum_{k:k \neq j} (\lambda_{a,j} - \lambda_{a,k})^{-2}\Big\{\int\int\big\{\widehat{K}_a(s,t)-K_a(s,t)\big\}\big\{\widehat{\phi}_{a,j}(s)- \phi_{a,j}(s)\big\}\phi_{a,k}(t)\;\mathrm{d}s\;\mathrm{d}t\Big\}^2\Big]^{\frac{1}{2}}\\
        = \;& 2\Big\|\sum_{k:k \neq j} \phi_{a,k}(\cdot)(\lambda_{a,j} - \lambda_{a,k})^{-1}\int\int\big\{\widehat{K}_a(s,t)-K_a(s,t)\big\}\phi_{a,j}(s)\phi_{a,k}(t)\;\mathrm{d}s\;\mathrm{d}t\Big\|_{L^2}\\ \notag
        &+2\Big\|\sum_{k:k \neq j} \phi_{a,k}(\cdot)(\lambda_{a,j} - \lambda_{a,k})^{-1}\int\int\big\{\widehat{K}_a(s,t)-K_a(s,t)\big\}\big\{\widehat{\phi}_{a,j}(s)- \phi_{a,j}(s)\big\}\phi_{a,k}(t)\;\mathrm{d}s\;\mathrm{d}t\Big\|_{L^2}\\
        =\;& 2(I) + 2(II),
    \end{align*}
    where the last inequality follows from $\mathcal{E}_2$ and the orthogonality of eigenfunctions.
    
    \noindent \textbf{Step 1: upper bound on $(I)$.}
    Using the result in \eqref{l_average_cov_4}, we can further write $(I)$ as
    \begin{align} \notag
        (I) \lesssim \;&  \Big\|\frac{1}{\widetilde{n}_a}\sum_{y \in \{0,1\}} \sum_{i=1}^{n_{a,y}} \sum_{k:k \neq j} \phi_{a,k}(\cdot)(\lambda_{a,j} - \lambda_{a,k})^{-1}\int\int\Big[\big\{\widetilde{X}^i_{a,y}(s) - \mu_{a,y}(s)\big\}\big\{\widetilde{X}^i_{a,y}(t) - \mu_{a,y}(t)\big\}\\ \notag
        & \hspace{3cm} - \mathbb{E}\Big\{\big\{\widetilde{X}^i_{a,y}(s) - \mu_{a,y}(s)\big\}\big\{\widetilde{X}^i_{a,y}(t) - \mu_{a,y}(t)\big\}\Big\}\Big]\phi_{a,j}(s)\phi_{a,k}(t)\;\mathrm{d}s\;\mathrm{d}t\Big\|_{L^2}\\ \notag
        &+\Big\|\frac{1}{\widetilde{n}_a}\sum_{y \in \{0,1\}}\sum_{k:k \neq j} \phi_{a,k}(\cdot)(\lambda_{a,j} - \lambda_{a,k})^{-1}\widetilde{n}_{a,y}\\ \notag
        & \hspace{3cm} \cdot\int\int\big\{\widehat{\mu}_{a,y}(s)-\mu_{a,y}(s)\big\}\big\{\widehat{\mu}_{a,y}(t)-\mu_{a,y}(t)\big\}\phi_{a,j}(s)\phi_{a,k}(t)\;\mathrm{d}s\;\mathrm{d}t\Big\|_{L^2}\\ \notag
        =\;& \Big\|\frac{1}{\widetilde{n}_a}\sum_{y \in \{0,1\}} \sum_{i=1}^{n_{a,y}} \sum_{k:k \neq j} \phi_{a,k}(\cdot)(\lambda_{a,j} - \lambda_{a,k})^{-1}\xi_{a,y,j}^i\xi_{a,y,k}^i\Big\|_{L^2}\\\notag
        &+ \Big\|\frac{1}{\widetilde{n}_a}\sum_{y \in \{0,1\}}\sum_{k:k \neq j} \phi_{a,k}(\cdot)(\lambda_{a,j} - \lambda_{a,k})^{-1}\widetilde{n}_{a,y}\Bar{\xi}_{a,y,j}\Bar{\xi}_{a,y,k}\Big\|_{L^2}\\ \notag
        =\;& (I)_1 + (I)_2
    \end{align}
     where for $k \in \mathbb{N}_+$, $\xi^i_{a,y,k} = \int \{\widetilde{X}^i_{a,y}(s) - \mu_{a,y}(s)\}\phi_{a,k}(s) \;\mathrm{d}s \sim \widetilde{n}(0,\lambda_{a,k})$, $\Bar{\xi}_{a,y,k} =\widetilde{n}_{a,y}^{-1}\sum_{i=1}^{\widetilde{n}_{a,y}}\xi^i_{a,y,k}$.

     \noindent \textbf{Step 1-1: upper bound on $(I)_1$.}
    Denote $\{\{z_{i,j}\}_{i=1}^{\widetilde{n}_a}\}_{j \in \mathbb{N}_+}$ and $\{\{z_{i,k}\}_{i=1}^{\widetilde{n}_a}\}_{k \in \mathbb{N}_+}$ two collections of independent standard Gaussian random variables. 
    With the above notation, by \Cref{l_bosq_thm2.5}, to control $(I)_1$, it suffices to control 
    \[\sum_{i=1}^{\widetilde{n}_a} \mathbb{E}\Big\{\Big\|\sum_{k:k \neq j} \phi_{a,k}(\cdot)(\lambda_{a,j} - \lambda_{a,k})^{-1}\sqrt{\lambda_{a,j}\lambda_{a,k}}z_{i,j}z_{i,k}\Big\|_{L^2}^b\Big\}.\]
    Note that for any $i \in [\widetilde{n}]$, we have that 
    \begin{align*}
        &\mathbb{E}\Big\{\Big\|\sum_{k:k \neq j} \phi_{a,k}(\cdot)(\lambda_{a,j} - \lambda_{a,k})^{-1}\sqrt{\lambda_{a,j}\lambda_{a,k}}z_{i,j}z_{i,k}\Big\|_{L^2}^{b}\Big\}\\
        =\;& \mathbb{E}\Big[\Big\{\sum_{k:k \neq j}(\lambda_{a,j} - \lambda_{a,k})^{-2}\lambda_{a,j}\lambda_{a,k} z^2_{i,j}z^2_{i,k}\Big\}^{\frac{b}{2}}\Big]\\
        \lesssim\;& j^{-\frac{b\alpha}{2}}\mathbb{E}\Big[ \Big\{\sum_{k:k \neq j}(\lambda_{a,j} - \lambda_{a,k})^{-2}k^{-\alpha} z^2_{i,j}z^2_{i,k}\Big\}^{\frac{b}{2}}\Big]\\
        =\;& j^{-\frac{b\alpha}{2}}\Big\{\sum_{k:k \neq j} (\lambda_{a,j} - \lambda_{a,k})^{-2}k^{-\alpha}\Big\}^{\frac{b}{2}}\mathbb{E}\Bigg[\Bigg\{\frac{\sum_{k:k \neq j} (\lambda_{a,j} - \lambda_{a,k})^{-2}k^{-\alpha} z^2_{i,j}z^2_{i,k}}{\sum_{k:k \neq j} (\lambda_{a,j} - \lambda_{a,k})^{-2}k^{-\alpha}}\Bigg\}^{\frac{b}{2}}\Bigg]\\
        \leq \;& j^{-\frac{b\alpha}{2}}\Big\{\sum_{k:k \neq j} (\lambda_{a,j} - \lambda_{a,k})^{-2}k^{-\alpha}\Big\}^{\frac{b}{2}}\frac{\sum_{k:k \neq j} (\lambda_{a,j} - \lambda_{a,k})^{-2}k^{-\alpha}\mathbb{E}\big(z^b_{i,j}z^b_{i,k}\big)}{\sum_{k:k \neq j} (\lambda_{a,j} - \lambda_{a,k})^{-2}k^{-\alpha}}\\
        \lesssim \;& j^{-\frac{b\alpha}{2}}(1+j^{\alpha+2})^{\frac{b}{2}-1}\sum_{k:k \neq j} (\lambda_{a,j} - \lambda_{a,k})^{-2}k^{-\alpha}\mathbb{E}\big(z^b_{i,j}z^b_{i,k}\big)\\
        \lesssim \;&j^{-\frac{b\alpha}{2}}(1+j^{\alpha+2})^{\frac{b}{2}-1}\sum_{k:k \neq j} (\lambda_{a,j} - \lambda_{a,k})^{-2}k^{-\alpha} \sqrt{\mathbb{E}\big(z^{2b}_{i,j}\big)\mathbb{E}\big(z^{2b}_{i,k}\big)}\\
        \lesssim \;& j^{-\frac{b\alpha}{2}}(1+j^{\alpha+2})^{\frac{b}{2}-1}2^b b!(1+j^{\alpha+2}) = 2^b b!j^2 j^{b-2},
    \end{align*}
    where the first equality follows from the orthogonality of $\{\phi_{a,k}\}_{k \in \mathbb{N}_+}$, the first inequality follows from \Cref{a_data}\ref{a_data_cov}, the second inequality follows from Jensen's inequality, the third inequality follows from \Cref{l_eigenfunc_sum}, the fourth inequality follows from Cauchy--Schwartz inequality
    and the fifth inequality follows from \Cref{l_eigenfunc_sum} and the fact that 
    \[\mathbb{E}\big(z^{2b}_{i,j}\big) = \pi^{-1/2}2^b\Gamma\Big(\frac{2b+1}{2}\Big) \leq 2^bb!.\]
    
    Therefore, pick $L_1 \asymp j^2$ and $L_2 \asymp j$, we have that
    \[\sum_{i=1}^{\widetilde{n}_a} \mathbb{E}\Big\{\Big\|\sum_{k:k \neq j} \phi_{a,k}(\cdot)(\lambda_{a,j} - \lambda_{a,k})^{-1}\sqrt{\lambda_{a,j}\lambda_{a,k}}z_{i,j}z_{i,k}\Big\|_{L^2}^b\Big\} \lesssim 4b!Nj^2 (2j)^{b-2} \lesssim b!NL_1L_2^{b-2}.\]
    Hence, it holds from \Cref{l_bosq_thm2.5} that
    \begin{align} \label{eq_l_eigenfunc_estimation_3}
        \mathbb{P}\Big\{(I)_1 \geq \epsilon\Big\} \leq 2\exp\Big(-\frac{\widetilde{n}_a\epsilon^2}{2j^2 + 2j\epsilon}\Big) \lesssim \exp\Big(-\frac{\widetilde{n}_a\epsilon^2}{j^2}\Big),
    \end{align}
    whenever $\epsilon \lesssim j$.
    
    \noindent \textbf{Step 1-2: upper bound on $(I)_2$.} To control $(I)_2$, it holds that 
    \begin{align} \notag
        (I)_2 &\leq \sum_{y \in \{0,1\}}\frac{\widetilde{n}_{a,y}}{\widetilde{n}_a}\Big\|\sum_{k:k \neq j} \phi_{a,k}(\cdot)(\lambda_{a,j} - \lambda_{a,k})^{-1}\Bar{\xi}_{a,y,j}\Bar{\xi}_{a,y,k}\Big\|_{L^2}\\ \label{eq_l_eigenfunc_estimation_4}
        & =\sum_{y \in \{0,1\}}\frac{\widetilde{n}_{a,y}}{\widetilde{n}_a}\sqrt{\sum_{k:k \neq j} (\lambda_{a,j} - \lambda_{a,k})^{-2}\Bar{\xi}^2_{a,y,j}\Bar{\xi}^2_{a,y,k}},
    \end{align}
    where the first inequality follows from triangle inequality and the last equality follows from the orthonormality of $\{\phi_{a,k}\}_{k \in \mathbb{N}_+}$. Also, since for any $k \in \mathbb{N}_+$ and $a,y \in \{0,1\}$, we have that $\xi^i_{a,y,k} \stackrel{\text{i.i.d.}}{\sim }N(0,\lambda_{a,k})$ by the independence property. Then, by standard property of independent Gaussian random variables, this implies that $\Bar{\xi}_{a,y,k} \sim N(0,\lambda_{a,k}/\widetilde{n}_{a,y})$.
    
    Moreover, by standard properties of sub-Gaussian random variables and \Cref{l_subWeibuill_property}, we have that $\sum_{k:k \neq j} (\lambda_{a,j} - \lambda_{a,k})^{-2}\Bar{\xi}^2_{a,y,j}\Bar{\xi}^2_{a,y,k}$ follows a sub-Weibull distribution with parameter $1/2$. We next upper bound its sub-Weibull norm. Note that 
    \begin{align*}
        &\Big\|\sum_{k:k \neq j} (\lambda_{a,j} - \lambda_{a,k})^{-2}\Bar{\xi}^2_{a,y,j}\Bar{\xi}^2_{a,y,k}\Big\|_{\psi_{1/2}} \leq \sum_{k:k \neq j} (\lambda_{a,j} - \lambda_{a,k})^{-2}\big\|\Bar{\xi}^2_{a,y,j}\Bar{\xi}^2_{a,y,k}\big\|_{\psi_{1/2}}\\
        \leq \;& \sum_{k:k \neq j} (\lambda_{a,j} - \lambda_{a,k})^{-2}\big\|\Bar{\xi}_{a,y,j}\big\|^2_{\psi_{2}}\big\|\Bar{\xi}_{a,y,k}\big\|^2_{\psi_{2}} \leq \sum_{k:k \neq j} (\lambda_{a,j} - \lambda_{a,k})^{-2} \frac{\lambda_{a,k}\lambda_{a,j}}{\widetilde{n}_{a,y}^2}\\
        \leq \;&  \frac{j^{-\alpha}}{\widetilde{n}_{a,y}^2}\sum_{k:k \neq j}(\lambda_{a,j} - \lambda_{a,k})^{-2}k^{-\alpha} \lesssim \frac{j^2}{\widetilde{n}_{a,y}^2},
    \end{align*}
    where the first inequality follows from triangle inequality, the second inequality follows from standard property of sub-Gaussian random variables \citep[e.g.~Lemma 2.7.7 in][]{vershynin2018high}, the third inequality follows from the fact that for any $k \in \mathbb{N}_0$, $\|\Bar{\xi}_{a,y,k}\big\|_{\psi_{2}} \lesssim \sqrt{\lambda_{a,k}/\widetilde{n}_{a,y}}$, the fourth inequality follows from \Cref{a_data}\ref{a_data_cov} and the last inequality follows from \Cref{l_eigenfunc_sum}. Therefore, by \Cref{l_subWeibuill_property}.1, it holds for any small $\delta >0$ that 
    \begin{align*}
        \mathbb{P}\Big(\Big|\sum_{k:k \neq j} (\lambda_{a,j} - \lambda_{a,k})^{-2}\Bar{\xi}^2_{a,y,j}\Bar{\xi}^2_{a,y,k}\Big| \geq \delta\Big) \lesssim \exp\Big\{-\Big(\frac{\widetilde{n}_{a,y}^2\delta}{j^2}\Big)^{\frac{1}{2}}\Big\}.
    \end{align*}
    Pick $\delta = \widetilde{n}_a^2\epsilon^4/(\widetilde{n}_{a,y}^2j^2)$, we have that 
    \begin{equation} \label{eq_l_eigenfunc_estimation_5}
        \mathbb{P}\Big(\Big|\sum_{k:k \neq j} (\lambda_{a,j} - \lambda_{a,k})^{-2}\Bar{\xi}^2_{a,y,j}\Bar{\xi}^2_{a,y,k}\Big| \geq \frac{\widetilde{n}_a^2\epsilon^4}{\widetilde{n}_{a,y}^2j^2}\Big) \lesssim \exp\Big(-\frac{\widetilde{n}_a\epsilon^2}{j^2}\Big).
    \end{equation}
    Substituting the result in \eqref{eq_l_eigenfunc_estimation_5} into \eqref{eq_l_eigenfunc_estimation_4}, we have that with probability at least $1- \exp(-\widetilde{n}_a\epsilon^2/j^2)$ that 
    \begin{align}\label{eq_l_eigenfunc_estimation_6}
        (I)_2 \lesssim \sum_{y \in \{0,1\}}\frac{\widetilde{n}_{a,y}}{\widetilde{n}_a}\sqrt{\frac{\widetilde{n}_a^2\epsilon^4}{\widetilde{n}_{a,y}^2j^2}}  \lesssim \frac{\epsilon^2}{j} \leq \epsilon,
    \end{align}
    where the last inequality follows from the fact that $j \geq 1$ and $\epsilon <1$.

    \noindent \textbf{Step 2: upper bound on $(II)$.} To control $(II)$, we have that 
    \begin{align} \notag
        (II)^2 & =\sum_{k:k \neq j}(\lambda_{a,j} - \lambda_{a,k})^{-2}\Big[\int\int\big\{\widehat{K}_a(s,t)-K_a(s,t)\big\}\big\{\widehat{\phi}_{a,j}(s)-\phi_{a,j}(s)\big\}\phi_{a,k}(t)\;\mathrm{d}s\;\mathrm{d}t\Big]^2\\ \notag
        &\lesssim J_{\widetilde{n}}^{2(\alpha+1)}\sum_{k =1}^\infty \Big[\int\Big\{ \int\big\{\widehat{K}_a(s,t)-K_a(s,t)\big\}\big\{\widehat{\phi}_{a,j}(s)-\phi_{a,j}(s)\big\}\;\mathrm{d}s \Big\}\phi_{a,k}(t)\;\mathrm{d}t\Big]^2\\ \notag
        & = J_{\widetilde{n}}^{2(\alpha+1)} \int\Big[ \int\big\{\widehat{K}_a(s,t)-K_a(s,t)\big\}\big\{\widehat{\phi}_{a,j}(s)-\phi_{a,j}(s)\big\}\;\mathrm{d}s \Big]^2\;\mathrm{d}t\\ \notag
        & \leq J_{\widetilde{n}}^{2(\alpha+1)} \Big[\int \big\{\widehat{\phi}_{a,j}(s)-\phi_{a,j}(s)\big\}^2 \;\mathrm{d}s\Big] \Big[\int\int \big\{\widehat{K}_a(s,t)-K_a(s,t)\big\}^2\;\mathrm{d}s\;\mathrm{d}t\Big]\\ \notag
        &= J_{\widetilde{n}}^{2(\alpha+1)}\|\widehat{\phi}_{a,j} - \phi_{a,j}\|_{L^2}^2\|\widehat{K}_a - K_a\|_{L^2}^2\\ \notag
        & \leq J_{\widetilde{n}}^{2(\alpha+1)}\frac{\epsilon^2}{j^2}\|\widehat{\phi}_{a,j} - \phi_{a,j}\|_{L^2}^2 \leq (1-2^{-1/2})^2\|\widehat{\phi}_{a,j} - \phi_{a,j}\|_{L^2}^2,
    \end{align}
    where the first inequality follows from $\mathcal{E}_2$ in \eqref{eq_l_eigenfunc_estimation_7}, the first equality follows from Parseval's identity, the second inequality follows from Cauchy--Schwartz inequality, the third inequality follows from $\mathcal{E}_1$ and the last inequality follows for all $\epsilon$ such that $0 < \epsilon \lesssim jJ_{\widetilde{n}}^{-(\alpha+1)}$. Therefore, we have that 
    \begin{equation} \label{eq_l_eigenfunc_estimation_8}
        \mathbb{P}\Big\{(II) \geq (1-2^{-1/2})\|\widehat{\phi}_{a,j} - \phi_{a,j}\|_{L^2}\Big\} \lesssim \exp\Big(-\frac{\widetilde{n}_a\epsilon^2}{j^2}\Big).
    \end{equation}

    \noindent \textbf{Step 3: Combine results together.} Substituting the results in \eqref{eq_l_eigenfunc_estimation_3}, \eqref{eq_l_eigenfunc_estimation_6} and \eqref{eq_l_eigenfunc_estimation_8} into \eqref{eq_eigenfunc_estimation_3} and applying a union bound argument, it holds with probability at least $1-\exp(-\widetilde{n}_a\epsilon^2/j^2)$ that 
    \begin{align*}
        \|\widehat{\phi}_{a,j} - \phi_{a,j}\|_{L^2} \leq \epsilon + \epsilon +(1-2^{-1/2})\|\widehat{\phi}_{a,j} - \phi_{a,j}\|_{L^2},
    \end{align*}
    which implies that $\|\widehat{\phi}_{a,j} - \phi_{a,j}\|_{L^2} \lesssim \epsilon$ and the lemma thus follows.

\end{proof}

\subsection{Eigenvalue}
\begin{lemma} \label{l_eigenvalue}
    Under Assumptions \ref{a_class_prob} and \ref{a_data}\ref{a_data_cov}, for any small constant $\eta \in (0,1/2)$, it holds with probability at least $1-\eta$ that, for any $j \in [J]$ such that $J^{2\alpha+2}\log(\widetilde{n}/\eta) \lesssim \widetilde{n}_a$,
    \begin{align*}
        |\widehat{\lambda}_{a,j} - \lambda_{a,j}| \lesssim \sqrt{\frac{j^{-2\alpha}\log(\widetilde{n}_a/\eta)}{\widetilde{n}_a}}.
    \end{align*}
\end{lemma}
\begin{proof}
     By \Cref{l_eigen_split} and the triangle inequality, for any $j \in [J]$, it holds that
    \begin{align*}
       &\big(1-\|\widehat{\phi}_{a,j}-\phi_{a,j}\|_{L^2}\big)|\widehat{\lambda}_{a,j} - \lambda_{a,j}|\\
       \leq\;& \Big|\int\int\big\{\widehat{K}_a(s,t)-K_a(s,t)\big\}\phi_{a,j}(s)\phi_{a,j}(t)\;\mathrm{d}s\;\mathrm{d}t\Big|\\
       &+ \|\widehat{\phi}_{a,j}-\phi_{a,j}\|_{L^2}\Big\|\int\big\{\widehat{K}_a(s,t)-K_a(s,t)\big\}\phi_{a,j}(s)\;\mathrm{d}s\Big\|_{L^2}\\
       \leq\;& \Big|\int\int\big\{\widehat{K}_a(s,t)-K_a(s,t)\big\}\phi_{a,j}(s)\phi_{a,j}(t)\;\mathrm{d}s\;\mathrm{d}t\Big| + \|\widehat{\phi}_{a,j}-\phi_{a,j}\|_{L^2}\Big\|\widehat{K}_a-K_a\Big\|_{L^2},
    \end{align*}
    where the last inequality follows from the fact that 
    \[\sup_{j\in [J]} \Big\|\int\big\{\widehat{K}_a(s,t)-K_a(s,t)\big\}\phi_{a,j}(s)\;\mathrm{d}s\Big\|_{L^2} \leq \Big\|\widehat{K}_a-K_a\Big\|_{L^2}.\]
    Therefore, by Lemmas \ref{l_average_cov}, \ref{l_eigenfunc_estimation} and \ref{l_cov_proj_1} and a union bound argument, we have that with probability at least $1-3\eta$, 
    \begin{align} \label{eq_eigenvalue_error}
        |\widehat{\lambda}_{a,j} - \lambda_{a,j}| \lesssim \sqrt{\frac{j^{-2\alpha}\log(\widetilde{n}_a/\eta)}{\widetilde{n}_a}} +\sqrt{\frac{j^2\log^2(\widetilde{n}_a/\eta)}{\widetilde{n}_a^2}} \lesssim \sqrt{\frac{j^{-2\alpha}\log(\widetilde{n}_a/\eta)}{\widetilde{n}_a}}, \quad \text{for all} \; j\in [J],
    \end{align}
    where the last inequality follows since $J^{2\alpha+2}\log(\widetilde{n}_a/\eta) \lesssim \widetilde{n}_a$.
\end{proof}

\subsection{Projection of difference between covariance function and its estimator}
\begin{lemma} \label{l_cov_proj_1}
    Assume Assumptions \ref{a_class_prob} and \ref{a_data}\ref{a_data_cov}~hold. Then for $\ell,k \in \mathbb{N}_+$ and $0<\epsilon \lesssim \sqrt{\widetilde{n}_{a,y}(\ell k)^{-\alpha}/\widetilde{n}_a}$, it holds that 
    \[\mathbb{P}\Big\{\Big|\int \int \big\{\widehat{K}_a(s,t) - K_a(s,t)\big\}\phi_{a,\ell}(s) \phi_{a,k}(t) \;\mathrm{d}s\;\mathrm{d}t\Big| \geq \epsilon\Big\} \lesssim \exp\Big(-\frac{\widetilde{n}_a\epsilon^2}{(\ell k)^{-\alpha}}\Big).\]
\end{lemma}
\begin{proof}
    Note that 
    \begin{align*}
        \widehat{K}_a (s,t) - K_a(s,t) = \frac{1}{\widetilde{n}_a}\sum_{y \in \{0,1\}} &\sum_{i=1}^{n_{a,y}}\big\{\widetilde{X}^i_{a,y}(s) - \widehat{\mu}_{a,y}(s)\big\}\big\{\widetilde{X}^i_{a,y}(t) - \widehat{\mu}_{a,y}(t)\big\} \\
        &- \mathbb{E}\Big\{\big\{\widetilde{X}^i_{a,y}(s) - \mu_{a,y}(s)\big\}\big\{\widetilde{X}^i_{a,y}(t) - \mu_{a,y}(t)\big\}\Big\}
    \end{align*}
    We have that for any $k, \ell \in \mathbb{N}_+$, 
    \begin{align}  \notag
        &\int \int \big\{\widehat{K}_a(s,t) - K_a(s,t)\big\}\phi_{a,\ell}(s) \phi_{a,k}(t) \;\mathrm{d}s\;\mathrm{d}t\\ \notag
        \asymp \;& \frac{1}{\widetilde{n}_a}\sum_{y \in \{0,1\}}\sum_{i=1}^{n_{a,y}} (\xi^i_{a,y,\ell}-\Bar{\xi}_{a,y,\ell})(\xi^i_{a,y,k}-\Bar{\xi}_{a,y,k}) - \lambda_{a,\ell}\indc\{\ell=k\}\\ \label{l_cov_proj_1_eq1}
        \asymp \;& \sum_{y \in \{0,1\}}\frac{\widetilde{n}_{a,y}}{\widetilde{n}_a} \Big\{\frac{1}{\widetilde{n}_{a,y}}\sum_{i=1}^{n_{a,y}}\big(\xi^i_{a,y,\ell}\xi^i_{a,y,k} -  \lambda_{a,\ell}\indc\{\ell=k\}\big)\Big\} - \sum_{y \in \{0,1\}}\frac{\widetilde{n}_{a,y}}{\widetilde{n}_a}\Big\{\frac{1}{\widetilde{n}_{a,y}}\sum_{i=1}^{\widetilde{n}_{a,y}}\xi^i_{a,y,\ell}\Big\}\Big\{\frac{1}{\widetilde{n}_{a,y}}\sum_{j=1}^{\widetilde{n}_{a,y}}\xi^j_{a,y,k}\Big\},
    \end{align}
    where for $\ell \in \mathbb{N}_+$, $\xi^i_{a,y,\ell} = \int \{\widetilde{X}^i_{a,y}(s) - \mu_{a,y}(s)\}\phi_{a,\ell}(s) \;\mathrm{d}s \sim N(0,\lambda_{a,\ell})$, $\Bar{\xi}_{a,y,\ell} =\widetilde{n}_{a,y}^{-1}\sum_{i=1}^{\widetilde{n}_{a,y}}\xi^i_{a,y,\ell}$. By standard properties of sub-Gaussian random variables \citep[e.g.~Lemma 2.7.7 in][]{vershynin2018high} and \Cref{a_data}\ref{a_data_cov}, we have that $\|\xi^i_{a,y,\ell}\xi^i_{a,y,k}\|_{\psi_1} \leq \sqrt{\lambda_{a,\ell}\lambda_{a,k}} \asymp (\ell k)^{-\alpha/2}$ for any $y \in \{0,1\}$ and $i \in [\widetilde{n}_{a,y}]$. Using the standard property of the covariance operator, it holds that
    \begin{align*}
        \mathbb{E}\big(\xi^i_{a,y,\ell}\xi^i_{a,y,k}\big) &= \mathbb{E}\Big[\int \int \{\widetilde{X}^i_{a,y}(s) - \mu_{a,y}(s)\}\{\widetilde{X}^i_{a,y}(t) - \mu_{a,y}(t)\}\phi_{a,\ell}(s)\phi_{a,\ell}(t)\;\mathrm{d}s\;\mathrm{d}t\Big]\\
        &= \int \int K_a(s,t)\phi_{a,\ell}(t)\;\mathrm{d}s\;\mathrm{d}t = \lambda_{a,\ell}\indc\{\ell=k\}.
    \end{align*}
    Hence, for $y \in \{0,1\}$ it holds from Bernstein’s inequality \citep[e.g.~Theorem 2.8.1 in][]{vershynin2018high}, we have that for any $\delta_{1,y} >0$,
    \[\mathbb{P}\Big\{\Big|\frac{1}{\widetilde{n}_{a,y}}\sum_{i=1}^{n_{a,y}}\big(\xi^i_{a,y,\ell}\xi^i_{a,y,k} -  \lambda_{a,\ell}\indc\{\ell=k\}\big)\Big| \geq \delta_{1,y} \Big\} \lesssim \exp\Big\{-\Big(\frac{\widetilde{n}_{a,y}\delta_{1,y} ^2}{(\ell k)^{-\alpha}} \wedge \frac{\widetilde{n}_{a,y}\delta_{1,y} }{(\ell k)^{-\alpha/2}}\Big)\Big\}.\]
    Moreover, it holds from General Hoeffding inequality \citep[e.g.~Theorem 2.6.2 in][]{vershynin2018high} that for any $\delta_{2,y},\delta_{3,y} > 0$, 
    \[\mathbb{P}\Big\{\Big|\frac{1}{\widetilde{n}_{a,y}}\sum_{i=1}^{\widetilde{n}_{a,y}}\xi^i_{a,y,\ell}\Big| \geq \delta_{2,y}\Big\} \lesssim \exp\Big(-\frac{\widetilde{n}_{a,y}\delta^2_{2,y}}{\ell^{-\alpha}}\Big),\]
    and
    \[\mathbb{P}\Big\{\Big|\frac{1}{\widetilde{n}_{a,y}}\sum_{i=1}^{\widetilde{n}_{a,y}}\xi^i_{a,y,k}\Big| \geq \delta_{3,y}\Big\} \lesssim \exp\Big(-\frac{\widetilde{n}_{a,y}\delta^2_{3,y}}{k^{-\alpha}}\Big).\]
    Pick
    \[\delta_{1,y} = \epsilon\sqrt{\frac{\widetilde{n}_a}{\widetilde{n}_{a,y}}}, \quad \delta_{2,y} = \epsilon\sqrt{\frac{\widetilde{n}_a}{\widetilde{n}_{a,y}k^{-\alpha}}},\quad \text{and} \quad \delta_{3,y} = \epsilon\sqrt{\frac{\widetilde{n}_a}{\widetilde{n}_{a,y}\ell^{-\alpha}}},\]
    and by a union-bound argument, we have that 
    \begin{align*}
        \mathbb{P}\Big\{\Big|\int \int \big\{\widehat{K}_a(s,t) - K_a(s,t)\big\}\phi_{a,\ell}(s) \phi_{a,k}(t) \;\mathrm{d}s\;\mathrm{d}t\Big| \leq \sum_{y\in \{0,1\}}& \frac{\widetilde{n}_{a,y}}{\widetilde{n}_a}\big(\delta_{1,y} + \delta_{2,y}\delta_{3,y}\big)\Big\} \\
        &\gtrsim 1-\exp\Big(-\frac{\widetilde{n}_a\epsilon^2}{(\ell k)^{-\alpha}}\Big).
    \end{align*}
    Note that
    \begin{align*}
        \sum_{a,y\in \{0,1\}}\frac{\widetilde{n}_{a,y}}{\widetilde{n}_a}\big(\delta_{1,y} + \delta_{2,y}\delta_{3,y}\big) &= \sum_{a,y\in \{0,1\}}\frac{\widetilde{n}_{a,y}}{\widetilde{n}_a}\Big\{\epsilon\sqrt{\frac{\widetilde{n}_a}{\widetilde{n}_{a,y}}}+\epsilon^2\frac{\widetilde{n}_a}{\widetilde{n}_{a,y}}\sqrt{\frac{1}{(\ell k)^{-\alpha}}}\Big\}\\
        & \lesssim \epsilon + \epsilon^{3/2} \sqrt{\frac{\epsilon}{(\ell k)^{-\alpha}}} \lesssim \epsilon,
    \end{align*}
    where the last inequality follows from the fact that $\epsilon \lesssim \sqrt{\widetilde{n}_{a,y}(\ell k)^{-\alpha}/\widetilde{n}_a}$, hence 
    \[\sqrt{\frac{\epsilon}{(\ell k)^{-\alpha}}} \lesssim \sqrt{\frac{\widetilde{n}_{a,y}}{\widetilde{n}_a}}\lesssim 1\]
    and the lemma follows.
\end{proof}

\begin{lemma} \label{l_cov_proj_2}
    Assume Assumptions \ref{a_class_prob} and \ref{a_data}\ref{a_data_cov}~hold. For $j \in \mathbb{N}_+$ and $0< \epsilon \lesssim j^{-\alpha/2}$, we have that 
    \[\mathbb{P}\Big\{\Big\|\int\big\{ K_a(s,t)-\widehat{K}_a(s,t) \big\}\phi_{a,j}(t)\;\mathrm{d}t\Big\|_{L^2} \geq \epsilon \Big\} \lesssim \exp\Big(-\frac{\widetilde{n}_a\epsilon^2}{j^{-\alpha}}\Big).\]
\end{lemma}
\begin{proof}
    The proof follows from a similar and even simpler argument as the one used in \textbf{Step 1-1} in the proof of \Cref{l_eigenfunc_estimation}. We only include the difference in the following. With the same notation as the one used in the proof of \Cref{l_eigenfunc_estimation}, it holds that 
    \begin{align*}
        &\mathbb{E}\Big\{\Big\|\big\{\widetilde{X}_{a,y}^i(\cdot)-\mu_{a,y}(\cdot)\big\}\xi_{a,y,j}^i-\lambda_{a,j}\phi_{a,j}(\cdot)\Big\|_{L^2}^b\Big\} = \mathbb{E}\Big[\Big\{\sum_{\ell=1}^\infty \big(\xi_{a,y,\ell}^i\xi_{a,y,j}^i - \lambda_{a,j}\indc\{\ell=j\}\big)^2\Big\}^{\frac{b}{2}}\Big]\\
        = \;&\mathbb{E}\Big[\Big\{\sum_{\ell=1}^\infty \lambda_{a,j}\lambda_{a,\ell}\big(Z_{i\ell}Z_{ij} - \indc\{\ell=j\}\big)^2\Big\}^{\frac{b}{2}}\Big] \leq j^{-\frac{\alpha b}{2}}\Big(\sum_{\ell=1}^\infty \lambda_{a,\ell}\Big)^{\frac{b}{2}-1}\sum_{\ell=1}^\infty \mathbb{E}\Big\{\Big(Z_{i\ell}Z_{ij} - \indc\{\ell=j\}\Big)^b\Big\}\\
        \lesssim\;& b!j^{-\alpha}(j^{-\frac{\alpha}{2}})^{b-2}.
    \end{align*}
    Therefore, by picking $L_1 = j^{-\alpha}$ and $L_2 = j^{-\frac{\alpha}{2}}$ in \Cref{l_bosq_thm2.5}, we have that 
    \[\mathbb{P}\Big\{\Big\|\int\big\{ K_a(s,t)-\widehat{K}_a(s,t) \big\}\phi_{a,j}(t)\;\mathrm{d}t\Big\|_{L^2} \geq \epsilon \Big\} \lesssim \exp\Big(-\frac{\widetilde{n}_a\epsilon^2}{j^{-\alpha} + j^{-\frac{\alpha}{2}}\epsilon}\Big).\]
    The lemma thus holds by picking $\epsilon$ satisfying $j^{-\alpha} \gtrsim j^{-\frac{\alpha}{2}}\epsilon$.
\end{proof}

\subsection{Projection score}
\begin{lemma}\label{l_score_const}
    Under  Assumptions \ref{a_class_prob} and \ref{a_data}\ref{a_data_cov}, we have with probability at least $1-\eta$ that, for any $j \in [J]$ such that $J^{2\alpha+2}\log^2(J)\log(\widetilde{n}/\eta) \lesssim \widetilde{n}$,
    \begin{equation*}
    |\widehat{\theta}_{a,1,j} -\widehat{\theta}_{a,0,j}-(\theta_{a,1,j} - \theta_{a,0,j})| \lesssim \begin{cases}
        \sqrt{\frac{j^{2-2\beta}\log^2(j)\log(\widetilde{n}/\eta)}{\widetilde{n}}}, \quad &\text{when}\; \frac{\alpha+1}{2} < \beta \leq \frac{\alpha+2}{2},\\
        \sqrt{\frac{j^{-\alpha}\log(\widetilde{n}/\eta)}{\widetilde{n}}},  \quad &\text{when}\;  \beta > \frac{\alpha+2}{2}.
    \end{cases}
\end{equation*}
\end{lemma}

\begin{proof}
    Applying a union bound argument to the arguments in Lemmas \ref{l_mean_diff_score_1}, \ref{l_mean_diff_score_2} and \ref{l_mean_diff_score_3}, we have with probability at least $1-\eta$ that
\begin{align*}
    &|\widehat{\theta}_{a,1,j} -\widehat{\theta}_{a,0,j}-(\theta_{a,1,j} - \theta_{a,0,j})|\\
    \leq\;& \Big|\int \big\{\widehat{\mu}_{a,1}(t) - \mu_{a,1}(t)-\widehat{\mu}_{a,0}(t) +\mu_{a,0}(t)\big\}\phi_{a,j}(t) \;\mathrm{d}t \Big| + \Big|\int \big\{\widehat{\mu}_{a,1}(t)-\widehat{\mu}_{a,0}(t)\big\}\big\{\widehat{\phi}_{a,j}(t) - \phi_{a,j}(t)\big\}\;\mathrm{d}t \Big|\\
    \leq\;&\sum_{y \in \{0,1\}}\Big|\frac{1}{\widetilde{n}_{a,y}}\sum_{i=1}^{\widetilde{n}_{a,y}}\int \big\{\widetilde{X}^i_{a,y}(t) - \mu_{a,y}(t)\big\}\phi_{a,j}(t) \;\mathrm{d}t\Big|\\
    &+ \sum_{y \in \{0,1\}}\Big|\frac{1}{\widetilde{n}_{a,y}}\sum_{i=1}^{\widetilde{n}_{a,y}} \int \big\{\widetilde{X}^i_{a,y}(t)-\mu_{a,y}(t)\big\}\big\{\widehat{\phi}_{a,j}(t) - \phi_{a,j}(t)\big\}\;\mathrm{d}t\Big|\\
    & + \Big|\int\big\{\mu_{a,1}(t)- \mu_{a,0}(t)\big\} \big\{\widehat{\phi}_{a,j}(t) - \phi_{a,j}(t)\big\}\;\mathrm{d}t\Big|\\
    \lesssim \; &\sqrt{\frac{j^{2-2\beta}\log(\widetilde{n}/\eta)}{\widetilde{n}}} + \sqrt{\frac{j^2\log(\widetilde{n}/\eta)\log(1/\eta)}{\widetilde{n}^2}}\big\{1 \vee j^{1+\alpha-\beta}\log(j)\big\}\\
    &+\sqrt{\frac{j^{-\alpha}\log(\widetilde{n}/\eta)}{\widetilde{n}}}\big\{1 \vee j^{\frac{\alpha}{2}-\beta+1}\log(j)\big\}.
\end{align*}
whenever $J^2\log^2(J)\log(\widetilde{n}/\eta)\lesssim \widetilde{n}$. Thus, the result follows. 
\end{proof}

\begin{lemma} \label{l_mean_diff_score_1}
    Under Assumptions \ref{a_class_prob} and \ref{a_data}\ref{a_data_cov}, for any small $\epsilon >0$, we have that
    \begin{align*}
        \mathbb{P}\Big\{\Big|\frac{1}{\widetilde{n}_{a,y}}\sum_{i=1}^{\widetilde{n}_{a,y}}\int \big\{\widetilde{X}^i_{a,y}(t) - \mu_{a,y}(t)\big\}\phi_{a,j}(t) \;\mathrm{d}t \Big| \geq \epsilon\Big\} \lesssim \exp\Big(-\frac{\widetilde{n}_{a,y}\epsilon^2}{j^{-\alpha}}\Big).
    \end{align*}
\end{lemma}
\begin{proof}
    By the standard property of Gaussian processes, we have that for each $i \in [\widetilde{n}_{a,y}]$, 
    \[\int \big\{\widetilde{X}^i_{a,y}(t) - \mu_{a,y}(t)\big\}\phi_{a,j}(t)\mathrm{d}t \stackrel{\text{i.i.d.}}{\sim} N(0,\lambda_{a,j}).\]
    Therefore, the lemma follows by \Cref{a_data}\ref{a_data_cov}~and General Hoeffding inequality \citep[e.g.~Theorem 2.6.2 in][]{vershynin2018high}.
\end{proof}

\begin{lemma} \label{l_mean_diff_score_2}
    Under Assumptions \ref{a_class_prob} and \ref{a_data}\ref{a_data_cov}, for any small constant $\eta \in (0,1/2)$, it holds with probability at least $1-\eta$ that 
    \[\Big|\frac{1}{\widetilde{n}_{a,y}}\sum_{i=1}^{\widetilde{n}_{a,y}} \int \big\{\widetilde{X}^i_{a,y}(t)-\mu_{a,y}(t)\big\}\big\{\widehat{\phi}_{a,j}(t) - \phi_{a,j}(t)\big\}\;\mathrm{d}t\Big| \lesssim \sqrt{\frac{j^{2-\alpha}\log^2(j)\log^2(\widetilde{n}/\eta)}{\widetilde{n}^2}},\]
    for any $j \in [J]$, $a,y \in \{0,1\}$, with $J^{2\alpha+2}\log(1/\eta) \lesssim \widetilde{n}$.
\end{lemma}

\begin{proof}
    Consider the following events: 
    \begin{align*}
        \mathcal{E}_1 = \big\{|\widehat{\lambda}_{a,j} - \lambda_{a,k}|^{-1} \leq \sqrt{2} |\lambda_{a,j} - \lambda_{a,k}|^{-1}, \quad \text{for}\; a\in\{0,1\}, k,j \in\mathbb{N}_+ \big\},
    \end{align*}
    \begin{align*}
        \mathcal{E}_2 = \Big\{\|\widehat{K}_a - K_a\|_{L^2} \lesssim \sqrt{\frac{\log(1/\eta)}{\widetilde{n}}},\quad a\in\{0,1\}\Big\},
    \end{align*}
    \begin{align*}
        \mathcal{E}_3  = \Big\{\|\widehat{\phi}_{a,j} - \phi_{a,j}\|_{L^2} \lesssim \sqrt{\frac{j^2\log(\widetilde{n}/\eta)}{\widetilde{n}}}, \quad \text{for}\; j \in [J],a\in\{0,1\}\Big\},
    \end{align*}
    \begin{align*}
        \mathcal{E}_4 = \Big\{\Big|\frac{1}{\widetilde{n}_{a,y}}\sum_{i=1}^{\widetilde{n}_{a,y}} \int \big\{\widetilde{X}_{a,y}^i(t)-\mu_{a,y}(t)\big\}\phi_{a,j}(t)\;\mathrm{d}t\Big|\lesssim \sqrt{\frac{j^{-\alpha}\log(\widetilde{n}/\eta)}{\widetilde{n}}},\quad \text{for}\; j \in [J], \; a,y \in \{0,1\}\Big\},
    \end{align*}
    and 
    \begin{align*}
        \mathcal{E}_5 = \Big\{\widetilde{n}_{a,y} \asymp \widetilde{n}, \;\; \text{for any}\;\; a,y \in \{0,1\}\Big\}.
    \end{align*}
    Applying a union bound argument, it holds from Lemmas \ref{l_n_ay}, \ref{l_average_cov}, \ref{l_eigenfunc_estimation} \ref{l_mean_diff_score_1} and a similar argument as the one leads to \eqref{eq_l_eigenfunc_estimation_7} and the event $\mathcal{F}_m$ after (5.16) in \citet{hall2007methodology} that $\mathbb{P}(\mathcal{E}_1\cap \mathcal{E}_2 \cap \mathcal{E}_3 \cap \mathcal{E}_4 \cap \mathcal{E}_5) \geq 1-\eta/4$, for some small $\eta \in (0,1/2)$. The rest of the proof is constructed conditioning on $\mathcal{E}_1\cap \mathcal{E}_2 \cap \mathcal{E}_3 \cap \mathcal{E}_4 \cap \mathcal{E}_5$ happening.

    Note that by \Cref{l_eigen_split}, we have that 
    \begin{align} \notag
        &\frac{1}{\widetilde{n}_{a,y}}\sum_{i=1}^{\widetilde{n}_{a,y}} \int \big\{\widetilde{X}^i_{a,y}(t)-\mu_{a,y}(t)\big\}\big\{\widehat{\phi}_{a,j}(t) - \phi_{a,j}(t)\big\}\;\mathrm{d}t  \\ \notag
        =\;& \frac{1}{\widetilde{n}_{a,y}}\sum_{i=1}^{\widetilde{n}_{a,y}} \int \big\{\widetilde{X}^i_{a,y}(t)-\mu_{a,y}(t)\big\}\phi_{a,j}(t)\;\mathrm{d}t\int\big\{\widehat{\phi}_{a,j}(s) - \phi_{a,j}(s)\big\}\phi_{a,j}(s)\;\mathrm{d}s\\\notag
        &+\sum_{k:k \neq j}(\widehat{\lambda}_{a,j} - \lambda_{a,k})^{-1}\Big\{\frac{1}{\widetilde{n}_{a,y}}\sum_{i=1}^{\widetilde{n}_{a,y}}\int \big\{\widetilde{X}^i_{a,y}(t)-\mu_{a,y}(t)\big\}\phi_{a,k}(t)\;\mathrm{d}t\Big\} \\ \notag
        &  \hspace{7cm} \cdot \int\int\big\{\widehat{K}_a(s,\ell)-K_a(s,\ell)\big\}\widehat{\phi}_{a,j}(s)\phi_{a,k}(\ell)\;\mathrm{d}s\;\mathrm{d}\ell\\\notag
        =\; & \frac{1}{\widetilde{n}_{a,y}}\sum_{i=1}^{\widetilde{n}_{a,y}} \int \big\{\widetilde{X}^i_{a,y}(t)-\mu_{a,y}(t)\big\}\phi_{a,j}(t)\;\mathrm{d}t\int\big\{\widehat{\phi}_{a,j}(s) - \phi_{a,j}(s)\big\}\phi_{a,j}(s)\;\mathrm{d}s\\\notag
        & + \sum_{k:k \neq j}(\widehat{\lambda}_{a,j} - \lambda_{a,k})^{-1}\Big\{\frac{1}{\widetilde{n}_{a,y}}\sum_{i=1}^{\widetilde{n}_{a,y}}\int \big\{\widetilde{X}^i_{a,y}(t)-\mu_{a,y}(t)\big\}\phi_{a,k}(t)\;\mathrm{d}t\Big\} \\\notag
        & \hspace{6cm}\cdot \int\int\big\{\widehat{K}_a(s,\ell)-K_a(s,\ell)\big\}\big\{\widehat{\phi}_{a,j}(s)-\phi_{a,j}(s)\big\}\phi_{a,k}(\ell)\;\mathrm{d}s\;\mathrm{d}\ell\\\notag
        & + \sum_{k:k \neq j}(\widehat{\lambda}_{a,j} - \lambda_{a,k})^{-1}\Big\{\frac{1}{\widetilde{n}_{a,y}}\sum_{i=1}^{\widetilde{n}_{a,y}}\int \big\{\widetilde{X}^i_{a,y}(t)-\mu_{a,y}(t)\big\}\phi_{a,k}(t)\;\mathrm{d}t\Big\}\\\notag
        & \hspace{7cm} \cdot\int\int\big\{\widehat{K}_a(s,\ell)-K_a(s,\ell)\big\}\phi_{a,j}(s)\phi_{a,k}(\ell)\;\mathrm{d}s\;\mathrm{d}\ell\\ \label{l_mean_diff_score_2_eq6}
        =\;& (I) + (II) + (III).
    \end{align}
    \noindent \textbf{Step 1 : upper bound on $(I)$.} To control $(I)$, we have that with probability at least $1-\eta/4$ that, for any $j \in [J]$,
    \begin{align} \notag
        |(I)| &\leq \Big|\frac{1}{\widetilde{n}_{a,y}}\sum_{i=1}^{\widetilde{n}_{a,y}} \int \big\{\widetilde{X}^i_{a,y}(t)-\mu_{a,y}(t)\big\}\phi_{a,j}(t)\;\mathrm{d}t\Big| \cdot \Big|\int\big\{\widehat{\phi}_{a,j}(s) - \phi_{a,j}(s)\big\}\phi_{a,j}(s)\;\mathrm{d}s\Big|\\ \notag
        & \leq\Big|\frac{1}{\widetilde{n}_{a,y}}\sum_{i=1}^{\widetilde{n}_{a,y}} \int \big\{\widetilde{X}^i_{a,y}(t)-\mu_{a,y}(t)\big\}\phi_{a,j}(t)\;\mathrm{d}t\Big|\sqrt{\int\big\{\widehat{\phi}_{a,j}(s) - \phi_{a,j}(s)\big\}^2\;\mathrm{d}s} \sqrt{\int \phi_{a,j}^2(s) \;\mathrm{d}s}\\ \label{l_mean_diff_score_2_eq3}
        & \lesssim  \sqrt{\frac{j^{-\alpha}\log(\widetilde{n}/\eta)}{\widetilde{n}}} \sqrt{\frac{j^2\log(\widetilde{n}/\eta)}{\widetilde{n}}} \asymp \sqrt{\frac{j^{2-\alpha}\log^2(\widetilde{n}/\eta)}{\widetilde{n}^2}},
    \end{align}
    where the second inequality follows from Cauchy--Schwarz inequality and the last inequality follows from $\mathcal{E}_3$ and $\mathcal{E}_4$.

    \noindent \textbf{Step 2: upper bound on $(II)$.} To control $(II)$, we have that 
    \begin{align} \notag
        |(II)| &\leq \|\widehat{\phi}_{a,j}(s)-\phi_{a,j}(s)\|_{L_2}\|\widehat{K}_a-K_a\|_{L_2}\\ \notag
        & \hspace{3cm} \cdot \sum_{k:k \neq j}|\widehat{\lambda}_{a,j} - \lambda_{a,k}|^{-1}\Big|\frac{1}{\widetilde{n}_{a,y}}\sum_{i=1}^{\widetilde{n}_{a,y}}\int \big\{\widetilde{X}^i_{a,y}(t)-\mu_{a,y}(t)\big\}\phi_{a,k}(t)\;\mathrm{d}t\Big| \\  \notag
        & \lesssim \|\widehat{\phi}_{a,j}(s)-\phi_{a,j}(s)\|_{L_2}\|\widehat{K}_a-K_a\|_{L_2}\\ \label{l_mean_diff_score_2_eq1}
        & \hspace{3cm} \cdot \sum_{k:k \neq j}|\lambda_{a,j} - \lambda_{a,k}|^{-1}\Big|\frac{1}{\widetilde{n}_{a,y}}\sum_{i=1}^{\widetilde{n}_{a,y}}\int \big\{\widetilde{X}^i_{a,y}(t)-\mu_{a,y}(t)\big\}\phi_{a,k}(t)\;\mathrm{d}t\Big|,
    \end{align}
    where the first inequality follows from Cauchy--Schwarz inequality and the second inequality follows from $\mathcal{E}_1$.
    Note that by the standard property of Gaussian process, we have that 
    \begin{align*}
        \frac{1}{\widetilde{n}_{a,y}}\sum_{i=1}^{\widetilde{n}_{a,y}}\int \big\{\widetilde{X}^i_{a,y}(t) - \mu_{a,y}(t)\big\}\phi_{k}(t)\mathrm{d}t \stackrel{\text{i.i.d.}}{\sim} N\Big(0,\;\frac{\lambda_{a,k}}{\widetilde{n}_{a,y}}\Big).
    \end{align*}
    Therefore, it holds from standard properties of sub-Gaussian norms that 
    \begin{align*}
        &\Big\|\sum_{k:k \neq j}|\lambda_{a,j} - \lambda_{a,k}|^{-1}\Big|\frac{1}{\widetilde{n}_{a,y}}\sum_{i=1}^{\widetilde{n}_{a,y}}\int \big\{\widetilde{X}^i_{a,y}(t)-\mu_{a,y}(t)\big\}\phi_{a,k}(t)\;\mathrm{d}t\Big|\Big\|_{\psi_2}\\
        \leq \;& \sum_{k:k \neq j}|\lambda_{a,j} - \lambda_{a,k}|^{-1} \Big\|\frac{1}{\widetilde{n}_{a,y}}\sum_{i=1}^{\widetilde{n}_{a,y}}\int \big\{\widetilde{X}^i_{a,y}(t)-\mu_{a,y}(t)\big\}\phi_{a,k}(t)\;\mathrm{d}t\Big\|_{\psi_2}\\
        \leq\;& \sum_{k:k \neq j}|\lambda_{a,j} - \lambda_{a,k}|^{-1} \sqrt{\frac{\lambda_{a,k}}{\widetilde{n}_{a,y}}} \lesssim \sqrt{\frac{1}{\widetilde{n}_{a,y}}}\sum_{k:k \neq j}|\lambda_{a,j} - \lambda_{a,k}|^{-1}k^{-\alpha/2} \lesssim \sqrt{\frac{j^{2+\alpha}\log^2(j)}{\widetilde{n}_{a,y}}},
    \end{align*}
    where the last inequality follows from \Cref{l_eigenfunc_sum}. Thus, by standard properties of sub-Gaussian random variables \citep[e.g.~Proposition 2.5.2 in][]{vershynin2018high}, we have that for any $\delta_1 >0$,
    \begin{align*}
        \mathbb{P}\Big\{\sum_{k:k \neq j}|\lambda_{a,j} - \lambda_{a,k}|^{-1}\Big|\frac{1}{\widetilde{n}_{a,y}}\sum_{i=1}^{\widetilde{n}_{a,y}}\int \big\{\widetilde{X}^i_{a,y}(t)-\mu_{a,y}(t)\big\}\phi_{a,k}(t)\;\mathrm{d}t\Big| \geq \delta_1\Big\} \lesssim \exp\Big(-\frac{\delta_1^2 \widetilde{n}_{a,y}}{j^{2+\alpha}\log^2(j)}\Big).
    \end{align*}
    Pick $\delta_1 = \sqrt{j^{2+\alpha}\log^2(j)\log(\widetilde{n}/\eta)/\widetilde{n}_{a,y}} \asymp \sqrt{j^{2+\alpha}\log^2(j)\log(\widetilde{n}/\eta)/\widetilde{n}} $, we then have that with probability at least $1-\eta/4$ that
    \begin{align}\label{l_mean_diff_score_2_eq2}
        \sum_{k:k \neq j}|\lambda_{a,j} - \lambda_{a,k}|^{-1}\Big|\frac{1}{\widetilde{n}_{a,y}}\sum_{i=1}^{\widetilde{n}_{a,y}}\int \big\{\widetilde{X}^i_{a,y}(t)-\mu_{a,y}(t)\big\}\phi_{a,k}(t)\;\mathrm{d}t\Big| \lesssim \sqrt{\frac{j^{2+\alpha}\log^2(j)\log(\widetilde{n}/\eta)}{\widetilde{n}}}.
    \end{align}
    Substituting \eqref{l_mean_diff_score_2_eq2} into \eqref{l_mean_diff_score_2_eq1}, it holds from a union-bound argument that with probability at least $1-\eta/4$ that 
    \begin{align} \label{l_mean_diff_score_2_eq4}
        |(II)| \lesssim \sqrt{\frac{j^2\log(\widetilde{n}/\eta)}{\widetilde{n}}}\sqrt{\frac{\log(1/\eta)}{\widetilde{n}}}\sqrt{\frac{j^{2+\alpha}\log^2(j)\log(\widetilde{n}/\eta)}{\widetilde{n}}} \asymp \sqrt{\frac{j^{4+\alpha}\log^2(j)\log^2(\widetilde{n}/\eta)\log(1/\eta)}{\widetilde{n}^3}},
    \end{align}
    for any $j \in [J]$.

    \noindent \textbf{Step 3: Upper bound on $(III)$.} To control $(III)$, firstly, note that under $\mathcal{E}_1$, it holds that 
    \begin{align*}
        |(III)| &\lesssim \sum_{k:k \neq j}|\widehat{\lambda}_{a,j} - \lambda_{a,k}|^{-1}\Big|\frac{1}{\widetilde{n}_{a,y}}\sum_{i=1}^{\widetilde{n}_{a,y}}\int \big\{\widetilde{X}^i_{a,y}(t)-\mu_{a,y}(t)\big\}\phi_{a,k}(t)\;\mathrm{d}t\Big|\\
        & \hspace{5cm}\cdot \Big|\int\int\big\{\widehat{K}_a(s,\ell)-K_a(s,\ell)\big\}\phi_{a,j}(s)\phi_{a,k}(\ell)\;\mathrm{d}s\;\mathrm{d}\ell\Big|\\
        & \lesssim \sum_{k:k \neq j}|\lambda_{a,j} - \lambda_{a,k}|^{-1}\Big|\frac{1}{\widetilde{n}_{a,y}}\sum_{i=1}^{\widetilde{n}_{a,y}}\int \big\{\widetilde{X}^i_{a,y}(t)-\mu_{a,y}(t)\big\}\phi_{a,k}(t)\;\mathrm{d}t\Big|\\
        & \hspace{5cm}\cdot \Big|\int\int\big\{\widehat{K}_a(s,\ell)-K_a(s,\ell)\big\}\phi_{a,j}(s)\phi_{a,k}(\ell)\;\mathrm{d}s\;\mathrm{d}\ell\Big|.
    \end{align*}
    Next, we control its $\psi_1$-orcliz norm and we have that 
    \begin{align*}
        \|(III)\|_{\psi_1} \leq &\; \sum_{k:k \neq j}(\lambda_{a,j} - \lambda_{a,k})^{-1}\Big\|\frac{1}{\widetilde{n}_{a,y}}\sum_{i=1}^{\widetilde{n}_{a,y}}\int \big\{\widetilde{X}^i_{a,y}(t)-\mu_{a,y}(t)\big\}\phi_{a,k}(t)\;\mathrm{d}t\Big\|_{\psi_2}\\
        & \hspace{4cm}\cdot\Big\|\int\int\big\{\widehat{K}_a(s,\ell)-K_a(s,\ell)\big\}\phi_{a,j}(s)\phi_{a,k}(\ell)\;\mathrm{d}s\;\mathrm{d}\ell\Big\|_{\psi_2}\\
        \lesssim &\;\sum_{k:k \neq j}(\lambda_{a,j} - \lambda_{a,k})^{-1}\sqrt{\frac{k^{-\alpha}}{\widetilde{n}}} \sqrt{\frac{j^{-\alpha}k^{-\alpha}}{\widetilde{n}}} \lesssim \frac{j^{-\frac{\alpha}{2}}}{\widetilde{n}}\sum_{k:k \neq j}|\lambda_{a,j} - \lambda_{a,k}|^{-1}k^{-\alpha}\\
        \lesssim &\;\frac{j^{1-\frac{\alpha}{2}}\log(j)}{\widetilde{n}},
    \end{align*}
    where the first inequality follows from the triangle inequality and Lemma 2.7.7 in \cite{vershynin2018high}, the second inequality follows from Lemmas \ref{l_cov_proj_1} and \ref{l_mean_diff_score_1} and the last inequality follows from \Cref{l_eigenfunc_sum}. Consequently, by standard properties of sub-Exponential random variables \citep[e.g.~Proposition 2.7.1 in][]{vershynin2018high}, it holds for any $\delta_2 >0$ that 
    \begin{align*}
        \mathbb{P}\Big\{|(III)| \geq \delta_2\Big\} \lesssim \exp\Big\{-\frac{\delta_2\widetilde{n}}{j^{1-\frac{\alpha}{2}}\log(j)}\Big\}
    \end{align*}
    By a union bound argument and by picking $\delta_2 = j^{1-\alpha/2}\log(j)\log(\widetilde{n}/\eta)/\widetilde{n}$, we have that with probability at least $1-\eta/4$ that 
    \begin{align} \label{l_mean_diff_score_2_eq5}
        |(III)| \lesssim \frac{j^{1-\frac{\alpha}{2}}\log(j)\log(\widetilde{n}/\eta)}{\widetilde{n}},
    \end{align}
    for any $ j \in [J]$. 
    
    \noindent \textbf{Step 4: Combine results.} Substituting the results in \eqref{l_mean_diff_score_2_eq3}, \eqref{l_mean_diff_score_2_eq4} and \eqref{l_mean_diff_score_2_eq5} into \eqref{l_mean_diff_score_2_eq6} and applying a union bound argument, we have with probability at least $1-\eta$ that 
    \begin{align*}
        &\frac{1}{\widetilde{n}_{a,y}}\sum_{i=1}^{\widetilde{n}_{a,y}} \int \big\{\widetilde{X}^i_{a,y}(t)-\mu_{a,y}(t)\big\}\big\{\widehat{\phi}_{a,j}(t) - \phi_{a,j}(t)\big\}\;\mathrm{d}t \\
        \lesssim\;& \sqrt{\frac{j^{2-\alpha}\log^2(\widetilde{n}/\eta)}{\widetilde{n}^2}} + \sqrt{\frac{j^{4+\alpha}\log^2(j)\log^3(\widetilde{n}/\eta)}{\widetilde{n}^3}}+\frac{j^{1-\frac{\alpha}{2}}\log(j)\log(\widetilde{n}/\eta)}{\widetilde{n}}\\
        \lesssim \;& \sqrt{\frac{j^{2-\alpha}\log^2(j)\log^2(\widetilde{n}/
        \eta)}{\widetilde{n}^2}},
    \end{align*}
    for any $j\in [J]$ such that $J^{2\alpha+2}\log(1/\eta) \lesssim \widetilde{n}$. 

\end{proof}

\begin{lemma} \label{l_mean_diff_score_3}
    Under Assumptions \ref{a_class_prob} and \ref{a_data}\ref{a_data_cov}, for any small constant $\eta \in (0,1/2)$, it holds with probability at least $1-\eta$ that 
    \begin{align*}
       &\Big|\int\big\{\mu_{a,1}(t)- \mu_{a,0}(t)\big\} \big\{\widehat{\phi}_{a,j}(t) - \phi_{a,j}(t)\big\}\;\mathrm{d}t\Big|\\
       \lesssim \;&\sqrt{\frac{j^{2-2\beta}\log(\widetilde{n}/\eta)}{\widetilde{n}}} + \sqrt{\frac{j^2\log(\widetilde{n}/\eta)\log(1/\eta)}{\widetilde{n}^2}}\big\{1+j^{1+\alpha-\beta}\log(j)\big\}\\
        &+\sqrt{\frac{j^{-\alpha}\log(\widetilde{n}/\eta)}{\widetilde{n}}}\big\{1+j^{\frac{\alpha}{2}-\beta+1}\log(j)\big\}
    \end{align*}
    for any $j \in [J]$ with $J^{2\alpha+2} \lesssim_{\log} \widetilde{n}$.
\end{lemma}

\begin{proof}
    The proof follows using a similar and simpler argument as the one used in the proof of \Cref{l_mean_diff_score_2}. We only include the difference here. 

    By \Cref{l_eigen_split}, we have that 
    \begin{align} \notag
        &\int \big\{\mu_{a,1}(t)- \mu_{a,0}(t)\big\}\big\{\widehat{\phi}_{a,j}(t) - \phi_{a,j}(t)\big\}\;\mathrm{d}t  \\ \notag
        =\;&\int \big\{\mu_{a,1}(t)- \mu_{a,0}(t)\big\}\phi_{a,j}(t)\;\mathrm{d}t\int\big\{\widehat{\phi}_{a,j}(s) - \phi_{a,j}(s)\big\}\phi_{a,j}(s)\;\mathrm{d}s\\\notag
        &+\sum_{k:k \neq j}(\widehat{\lambda}_{a,j} - \lambda_{a,k})^{-1}\int \big\{\mu_{a,1}(t)- \mu_{a,0}(t)\big\}\phi_{a,k}(t)\;\mathrm{d}t \int\int\big\{\widehat{K}_a(s,\ell)-K_a(s,\ell)\big\}\widehat{\phi}_{a,j}(s)\phi_{a,k}(\ell)\;\mathrm{d}s\;\mathrm{d}\ell\\\notag
        =\; & \int \big\{\mu_{a,1}(t)- \mu_{a,0}(t)\big\}\phi_{a,j}(t)\;\mathrm{d}t\int\big\{\widehat{\phi}_{a,j}(s) - \phi_{a,j}(s)\big\}\phi_{a,j}(s)\;\mathrm{d}s\\\notag
        & + \sum_{k:k \neq j}(\widehat{\lambda}_{a,j} - \lambda_{a,k})^{-1}\int \big\{\mu_{a,1}(t)- \mu_{a,0}(t)\big\}\phi_{a,k}(t)\;\mathrm{d}t \\\notag
        & \hspace{6cm} \cdot \int\int\big\{\widehat{K}_a(s,\ell)-K_a(s,\ell)\big\}\big\{\widehat{\phi}_{a,j}(s)-\phi_{a,j}(s)\big\}\phi_{a,k}(\ell)\;\mathrm{d}s\;\mathrm{d}\ell\\\notag
        & + \sum_{k:k \neq j}(\widehat{\lambda}_{a,j} - \lambda_{a,k})^{-1}\int \big\{\mu_{a,1}(t)- \mu_{a,0}(t)\big\}\phi_{a,k}(t)\;\mathrm{d}t\int\int\big\{\widehat{K}_a(s,\ell)-K_a(s,\ell)\big\}\phi_{a,j}(s)\phi_{a,k}(\ell)\;\mathrm{d}s\;\mathrm{d}\ell\\ \label{l_mean_diff_score_3_eq1}
        =\;& (I) + (II) + (III).
    \end{align}
    \noindent \textbf{Step 1: Upper bound on $(I)$}
    Using a similar argument as the one used in \textbf{Step 1} in the proof of \Cref{l_mean_diff_score_2}, we have with probability at least $1-\eta/3$ that 
    \begin{align} \notag
        |(I)| &\leq \Big|\int \big\{\mu_{a,1}(t)- \mu_{a,0}(t)\big\}\phi_{a,j}(t)\;\mathrm{d}t\Big| \cdot \Big|\int\big\{\widehat{\phi}_{a,j}(s) - \phi_{a,j}(s)\big\}\phi_{a,j}(s)\;\mathrm{d}s\Big|\\ \notag
        & \leq\Big|\int \big\{\mu_{a,1}(t)- \mu_{a,0}(t)\big\}\phi_{a,j}(t)\;\mathrm{d}t\Big|\sqrt{\int\big\{\widehat{\phi}_{a,j}(s) - \phi_{a,j}(s)\big\}^2\;\mathrm{d}s} \sqrt{\int \phi_{a,j}^2(s) \;\mathrm{d}s}\\ \label{l_mean_diff_score_3_eq2}
        & \lesssim j^{-\beta}\sqrt{\frac{j^2\log(\widetilde{n}/\eta)}{\widetilde{n}}} \asymp \sqrt{\frac{j^{2-2\beta}\log(\widetilde{n}/\eta)}{\widetilde{n}}},
    \end{align}
    where the third inequality follows from \Cref{a_data}\ref{a_data_mean_decay}.

    \noindent \textbf{Step 2: Upper bound on $(II)$} To control $(II)$, then using a similar argument as the one used in \textbf{Step 2} in the proof of \Cref{l_mean_diff_score_2}, we have with probability at least $1-\eta/3$ that
    \begin{align} \notag
        |(II)| &\lesssim  \|\widehat{\phi}_{a,j}(s)-\phi_{a,j}(s)\|_{L_2}\|\widehat{K}_a-K_a\|_{L_2}\sum_{k:k \neq j}|\lambda_{a,j} - \lambda_{a,k}|^{-1}\Big|\int \big\{\mu_{a,1}(t)- \mu_{a,0}(t)\big\}\phi_{a,k}(t)\;\mathrm{d}t\Big|\\ \notag
        & \lesssim \sqrt{\frac{ j^2\log(\widetilde{n}/\eta)\log(1/\eta)}{\widetilde{n}^2}}\sum_{k:k \neq j}|\lambda_{a,j} - \lambda_{a,k}|^{-1}k^{-\beta}\\\label{l_mean_diff_score_3_eq3}
        & \lesssim \sqrt{\frac{ j^2\log(\widetilde{n}/\eta)\log(1/\eta)}{\widetilde{n}^2}}\big\{1+j^{1+\alpha-\beta}\log(j)\big\}.
    \end{align}

    \noindent \textbf{Step 3: Upper bound on $(III)$} To control $(III)$, then using a similar argument as the one used in \textbf{Step 3} in the proof of \Cref{l_mean_diff_score_2}, we have with probability at least $1-\eta/3$ that
    \begin{align}\notag
        |(III)| &\lesssim \sum_{k:k \neq j}|\lambda_{a,j} - \lambda_{a,k}|^{-1}\Big|\int \big\{\mu_{a,1}(t)- \mu_{a,0}(t)\big\}\phi_{a,k}(t)\;\mathrm{d}t\Big|\\ \notag
        & \hspace{5cm}\cdot\Big|\int\int\big\{\widehat{K}_a(s,\ell)-K_a(s,\ell)\big\}\phi_{a,j}(s)\phi_{a,k}(\ell)\;\mathrm{d}s\;\mathrm{d}\ell\Big|\\ \label{l_mean_diff_score_3_eq4}
        & \lesssim \sqrt{\frac{j^{-\alpha}\log(\widetilde{n}/\eta)}{\widetilde{n}}}\big\{1+j^{\frac{\alpha}{2}-\beta+1}\log(j)\big\}.
    \end{align}

    \noindent \textbf{Step 4: Combine results.} Substituting the results in \eqref{l_mean_diff_score_3_eq2}, \eqref{l_mean_diff_score_3_eq3} and \eqref{l_mean_diff_score_3_eq4} into \eqref{l_mean_diff_score_3_eq1} and applying a union bound argument, it holds with probability at least $1-\eta$ that 
    \begin{align*}
        &\Big|\int\big\{\mu_{a,1}(t)- \mu_{a,0}(t)\big\} \big\{\widehat{\phi}_{a,j}(t) - \phi_{a,j}(t)\big\}\;\mathrm{d}t\Big| \\
        \lesssim &\;  \sqrt{\frac{j^{2-2\beta}\log(\widetilde{n}/\eta)}{\widetilde{n}}} + \sqrt{\frac{j^2\log(\widetilde{n}/\eta)\log(1/\eta)}{\widetilde{n}^2}}\big\{1+j^{1+\alpha-\beta}\log(j)\big\}\\
        &+\sqrt{\frac{j^{-\alpha}\log(\widetilde{n}/\eta)}{\widetilde{n}}}\big\{1+j^{\frac{\alpha}{2}-\beta+1}\log(j)\big\}.
    \end{align*}
\end{proof}

\section{Technical lemmas} \label{appendix_technical}
For completeness, we provide all technical lemmas in this section.
\begin{lemma}[Generalized Neyman--Pearson lemma, e.g.~Lemma 3.1 in \citealp{zeng2024bayes}]\label{lem:GNP}
    Let $\phi_0, \phi_1, \ldots,$ \\$ \phi_m$ be $m+1$ real-valued functions defined on a Euclidean space $\mathcal{X}$. Assume they are $\nu$-integrable for a $\sigma$-finite measure $\nu$. Let $f^* \in \mathcal{F}$ be any function of the form
    \[
    f^*(x) = 
    \begin{cases}
    1, & \phi_0(x) > \sum_{i=1}^m c_i \phi_i(x), \\
    \tau(x), & \phi_0(x) = \sum_{i=1}^m c_i \phi_i(x), \\
    0, & \phi_0(x) < \sum_{i=1}^m c_i \phi_i(x),
    \end{cases}
    \]
    where $0 \leq \tau(x) \leq 1$ for all $x \in \mathcal{X}$. For given constants $t_1, \ldots, t_m \in \mathbb{R}$, let $\mathcal{F}_{\leq}$ be the class of measurable functions $f: \mathcal{X} \to \mathbb{R}$ satisfying
    \begin{equation}\label{lem:GNP_eq1}
        \int_{\mathcal{X}} f \phi_i \, d\nu \leq t_i, \quad i \in \{1, 2, \ldots, m\},
    \end{equation}
    and let $\mathcal{F}_{=}$ be the set of functions in $\mathcal{F}_{\leq}$ satisfying \eqref{lem:GNP_eq1} with all inequalities replaced by equalities.
    
    \begin{itemize}
    \item[(1)] If $f^* \in \mathcal{F}_{=}$, then
    \[
    f^* \in \arg\max_{f \in \mathcal{F}_{=}} \int_{\mathcal{X}} f \phi_0 \, d\nu.
    \]
    Moreover, if $\nu(\{x : \phi_0(x) = \sum_{i=1}^m c_i \phi_i(x) \}) = 0$, for all $f' \in \arg\max_{f \in \mathcal{F}_{=}} \int_{\mathcal{X}} f \phi_0 \, d\nu$, $f' = f^*$ almost everywhere with respect to $\nu$.
    
    \item[(2)] Moreover, if $c_i \geq 0$ for all $i = 1, \ldots, m$, then
    \[
    f^* \in \arg\max_{f \in \mathcal{F}_{\leq}} \int_{\mathcal{X}} f \phi_0 \, d\nu.
    \]
    Moreover, if $\nu(\{x : \phi_0(x) = \sum_{i=1}^m c_i \phi_i(x) \}) = 0$, for all $f' \in \arg\max_{f \in \mathcal{F}_{\leq}} \int_{\mathcal{X}} f \phi_0 \, d\nu$, we have $f'(x) = f^*(x)$ almost everywhere with respect to $\nu$.
    \end{itemize}
\end{lemma}

\begin{lemma}[Lemma 7 in \citealp{dou2012estimation}] \label{l_eigenfunc_sum}
    Under \Cref{a_data}\ref{a_data_cov}, for each $r \geq 1$, there exists a constant $C_r >0$ depending on $r$ such that 
    \[ \sum_{k \in \mathbb{N}} \indc_{\{j \neq k\}} \frac{k^{-\gamma}}{|\lambda_{a,j} - \lambda_{a,k}|^r} \leq 
    \begin{cases} 
    C_r (1 + j^{r(1+\alpha)-\gamma}), & \text{if } r > 1, \\
    C_1 (1 + j^{1+\alpha-\gamma} \log j), & \text{if } r = 1,
    \end{cases}
    \]
    for all $a\in \{0,1\}$ and $j\in \mathbb{N}_+$.

\end{lemma}

\begin{lemma}[Theorem 2.5 in \citealp{bosq2000linear}] \label{l_bosq_thm2.5}
    Let $\{X_i\}_{i=1}^n$ be independent random variables in a separable Hilbert space with norm $\|\cdot\|$. If $\mathbb{E}[X_i] = 0$ for all $i \in [n]$ and 
    \[\sum_{i=1}^n \mathbb{E}(\|X_i\|^b) \leq  \frac{b!}{2}nL_1L_2^{b-2}, \quad \text{for} \; b = 2,3, \ldots,\]
    with $L_1,L_2 >0$ being two constants, then for all $\epsilon >0$, it holds that 
    \[\mathbb{P}\Big(\Big\|\frac{1}{n}\sum_{i=1}^n X_i\Big\| \geq \epsilon\Big) \leq 2 \exp\Big(-\frac{n\epsilon^2}{2L_1+2L_2\epsilon}\Big).\]
\end{lemma}

\begin{lemma}[Lemma 5.1 in \citealp{hall2007methodology}] \label{l_eigen_split}
    If we can write 
    \[K(s,t) = \sum_{j=1}^\infty \lambda_j\phi_j(s)\phi_j(t) \quad \text{and} \quad \widehat{K}(s,t) = \sum_{j=1}^\infty \widehat{\lambda}_{j}\widehat{\phi}_j(s)\widehat{\phi}_j(t),\]
    then it holds for any $j \in \mathbb{N}_+$ that 
    \begin{align*}
        \Big|\widehat{\lambda}_{j} - \lambda_j - \int\int\big\{\widehat{K}(s,t)-&K(s,t)\big\}\phi_j(s)\phi_j(t)\;\mathrm{d}s\;\mathrm{d}t\Big|\\
        &\leq \|\widehat{\phi}_j-\phi_j\|_{L^2}\Big(|\widehat{\lambda}_{j} - \lambda_j|+\Big\|\int\big\{\widehat{K}(s,t)-K(s,t)\big\}\phi_j(s)\;\mathrm{d}s\Big\|_{L^2}\Big).
    \end{align*}
    Moreover, if $\inf_{k\neq j} |\widehat{\lambda}_{j}-\lambda_k|>0$, it holds for $\ell \in [0,1]$ that
    \begin{align*}
        \widehat{\phi}_j(\ell) - \phi_j(\ell) =\;& \phi_j(\ell)\int\big\{\widehat{\phi}_j(s) - \phi_j(s)\big\}\phi_j(s)\;\mathrm{d}s\\
        &+\sum_{k:k \neq j} \phi_k(\ell)(\widehat{\lambda}_{j} - \lambda_k)^{-1}\int\int\big\{\widehat{K}(s,t)-K(s,t)\big\}\widehat{\phi}_j(s)\phi_k(t)\;\mathrm{d}s\;\mathrm{d}t.
    \end{align*}
\end{lemma}

\begin{lemma}[\citealp{wong2020lasso}, Sub-Weibull properties]
\label{l_subWeibuill_property}
Let $X$ be a random variable. Then the following statements are equivalent for every $\alpha>0$. The constants $C_1, C_2, C_3 >0$ differ at most by a constant depending only on $\alpha$.
\begin{enumerate}
    \item The tail of $X$ satisfies 
    \begin{align*}
        \mathbb{P}\big\{|X|>t\big\} \leq 2 \exp \big\{-(t / C_1)^\alpha\big\}, \; \text{for all} \;\; t \geq 0.
    \end{align*}
    \item The moments of $X$ satisfy
    \begin{align*}
        \|X\|_p:=\big(\mathbb{E}\big[|X|^p\big]\big)^{1 / p} \leq C_2 p^{1 / \alpha}, \; \text{for all} \;\; p \geq 1 \wedge \alpha.
    \end{align*}
    \item  The moment generating function of $|X|^\alpha$ is finite at some point; namely
    \begin{align*}
        \mathbb{E}\big[\exp (|X| / C_3)^\alpha\big] \leq 2.
    \end{align*}
\end{enumerate}
We further call a random variable $X$ which satisfies any of the properties above a sub-Weibull random variable with parameter $\alpha$.
\end{lemma}

\end{appendices}

\end{document}